\documentclass[OPT,biber]{nowfnt} % creates the journal version, needs biber version
\usepackage[utf8]{inputenc}

\usepackage[intlimits,nosumlimits]{amsmath}
\usepackage{amsfonts,amssymb}
\usepackage{amsthm}
\DeclareSymbolFontAlphabet{\mathbb}{AMSb}
\usepackage{apalike}
\usepackage{float}
\usepackage{fancyhdr}
\usepackage{fancybox}
\usepackage{ifthen}
\usepackage{url}
\usepackage{lscape,afterpage}
\usepackage{xspace}
\usepackage{epstopdf} 
\usepackage{array}
\usepackage{multirow}
\usepackage[mathscr]{euscript}
\usepackage{subcaption}
\usepackage{bookmark}
\usepackage{algorithm}
\usepackage{algpseudocode}
\usepackage{booktabs}       
\usepackage{nicefrac}     
\usepackage{microtype}  
\usepackage{xcolor}
\usepackage{color}
\usepackage{pgfplots}
\usepackage{pgfplotstable}
\usepackage{hhline}
\usepackage{stmaryrd}
\usepackage{mathtools}
\usepackage{enumerate}

\usepackage{graphicx}
\usepackage{appendix}

\newcommand{\vertiii}[1]{{\left\vert\kern-0.25ex\left\vert\kern-0.25ex\left\vert #1 
		\right\vert\kern-0.25ex\right\vert\kern-0.25ex\right\vert}}
\newcommand{\benon}{\begin{equation*}}  
\newcommand{\bemuln}[1]{\begin{multline}\label{#1}}
\newcommand{\bemul}{\begin{multline*}}
\newcommand{\been}[1]{\begin{eqnarray}\label{#1}}
\newcommand{\eeen}{\end{eqnarray}}
\newcommand{\began}[1]{\begin{gather}\label{#1}}
\newcommand{\bega}{\begin{gather*}}
\newcommand{\bealn}[1]{\begin{align}\label{#1}}
\newcommand{\beal}{\begin{align*}}
\newcommand{\bealatn}[2]{\begin{alignat}{#1}\label{#2}}
\newcommand{\bealat}{\begin{alignat*}}
\newcommand{\bexalatn}[1]{\begin{xalignat}\label{#1}}
\newcommand{\bexalat}{\begin{xalignat*}}
\newcommand{\ra}{\rightarrow}

\newcommand{\lb}{\llbracket}
\newcommand{\rb}{\rrbracket}
\newcommand{\diag}{\text{diag}}
\newcommand{\tr}{\text{tr}}

\def\argmax{\mathop{\rm arg\,max}}

\newcommand{\mb}{\mathbf}

\newcommand{\mbb}{\mathbb}

\newtheorem{thm}{Theorem}[section]

\newtheorem{lem}[thm]{Lemma}
\newtheorem{col}[thm]{Corollary}
\newtheorem{defi}{Definition}
\newtheorem{ass}{Assumption} 
 
\def\ba{{\mathbf a}}
\def\bb{{\mathbf b}}

\def\bg{{\mathbf g}}

\def\bp{{\mathbf p}}
\def\bq{{\mathbf q}}
\def\br{{\mathbf r}}

\def\bt{{\mathbf t}}
\def\bu{{\mathbf u}}
\def\bv{{\mathbf v}}
\def\bw{{\mathbf w}}
\def\bx{{\mathbf x}}  
\def\by{{\mathbf y}}
\def\bz{{\mathbf z}}
\def\bA{{\mathbf A}}
\def\bB{{\mathbf B}}
\def\bC{{\mathbf C}}
\def\bD{{\mathbf D}}

\def\bH{{\mathbf H}}
\def\bI{{\mathbf I}}

\def\bM{{\mathbf M}}

\def\bS{{\mathbf S}}

\def\bY{{\mathbf Y}}
\def\bX{{\mathbf X}}

\def\bW{{\mathbf W}}
\def\bZ{{\mathbf Z}}

\def\texitem#1{\par\smallskip\noindent\hangindent 25pt
               \hbox to 25pt {\hss #1 ~}\ignorespaces}

\newcommand{\bzero}{{\mathbf{0}}}

\newcommand{\scrA}{\mathcal{A}}
\newcommand{\scrB}{\mathcal{B}}
\newcommand{\scrC}{\mathcal{C}}
\newcommand{\scrD}{\mathcal{D}}

\newcommand{\scrF}{\mathcal{F}}

\newcommand{\scrH}{\mathcal{H}}
\newcommand{\scrI}{\mathcal{I}}

\newcommand{\scrM}{\mathcal{M}}
\newcommand{\scrN}{\mathcal{N}}

\newcommand{\scrP}{\mathcal{P}}
\newcommand{\scrQ}{\mathcal{Q}}
\newcommand{\scrR}{\mathcal{R}}
\newcommand{\scrS}{\mathcal{S}}
\newcommand{\scrT}{\mathcal{T}}
\newcommand{\scrU}{\mathcal{U}}
\newcommand{\scrV}{\mathcal{V}}
\newcommand{\scrW}{\mathcal{W}}
\newcommand{\scrX}{\mathcal{X}}
\newcommand{\scrY}{\mathcal{Y}}
\newcommand{\scrZ}{\mathcal{Z}}

\newcommand{\bbeta}{\boldsymbol{\beta}}
\newcommand{\bGamma}{\boldsymbol{\Gamma}}
\newcommand{\bDelta}{\boldsymbol{\Delta}}
\newcommand{\bgamma}{\boldsymbol{\gamma}}
\newcommand{\bdelta}{{\boldsymbol{\delta}}}
\newcommand{\bepsilon}{\boldsymbol{\epsilon}}
\newcommand{\bzeta}{\boldsymbol{\zeta}}

\newcommand{\btheta}{\boldsymbol{\theta}}
\newcommand{\bTheta}{\boldsymbol{\Theta}}

\newcommand{\blambda}{\boldsymbol{\lambda}}

\newcommand{\bmu}{\boldsymbol{\mu}} 
\newcommand{\bnu}{{\boldsymbol{\nu}}}
\newcommand{\bpi}{{\boldsymbol{\pi}}}
\newcommand{\bPi}{{\boldsymbol{\Pi}}}

\newcommand{\bSigma}{{\boldsymbol{\Sigma}}}
\newcommand{\bphi}{{\boldsymbol{\phi}}}

\newcommand{\bxi}{\boldsymbol{\xi}}

\newcommand{\bpsi}{\boldsymbol{\psi}}

\newcommand{\ie}{\emph{i.e.}}
\newcommand{\eg}{\emph{e.g.}}

\title{Distributionally Robust Learning}

\maintitleauthorlist{
Ruidi Chen \\
Boston University \\
rchen15@bu.edu
\and
Ioannis Ch. Paschalidis \\
Boston University \\
yannisp@bu.edu
}

\issuesetup
{%
 copyrightowner={R.~Chen and I.~Ch.~Paschalidis},
 volume        = 4,
 issue         = 1--2,
 pubyear       = 2020,
 isbn          = 978-1-68083-772-8,
 eisbn         = 978-1-68083-773-5,
 doi           = 10.1561/2400000026,
 firstpage     = 1, %Explain
 lastpage      = 243
 }

\usepackage{mwe}

\author[1]{Chen,Ruidi}
\author[2]{Paschalidis,Ioannis Ch.}

\affil[1]{Boston University; rchen15@bu.edu}
\affil[2]{Boston University; yannisp@bu.edu}

\articledatabox{\nowfntstandardcitation}

\begin{document}

\makeabstracttitle

\begin{abstract}
    This monograph develops a comprehensive statistical learning framework that is robust to (distributional) perturbations in the data using {\em Distributionally Robust Optimization (DRO)} under the Wasserstein metric. Beginning with fundamental properties of the Wasserstein metric and the DRO formulation, we explore duality to arrive at tractable formulations and develop finite-sample, as well as asymptotic, performance guarantees. We consider a series of learning problems, including $(i)$ distributionally robust linear regression; $(ii)$ distributionally robust regression with group structure in the predictors; $(iii)$ distributionally robust multi-output regression and multiclass classification, $(iv)$ optimal decision making that combines distributionally robust regression with nearest-neighbor estimation; $(v)$ distributionally robust semi-supervised learning, and $(vi)$ distributionally robust reinforcement learning.
    A tractable DRO relaxation for each problem is being derived, establishing a connection between robustness and regularization, and obtaining bounds on the prediction and estimation errors of the solution. Beyond theory, we include numerical experiments and case studies using synthetic and real data. The real data experiments are all associated with various health informatics problems, an application area which provided the initial impetus for this work. 
\end{abstract}

\chapter{Introduction}    \label{chapt:intro}

A central problem in {\em machine learning} is to learn from data (``big'' or
``small'') how to predict outcomes of interest. Outcomes can be {\em binary} or {\em
	discrete}, such as an event or a category, or {\em continuous}, e.g., a real
value. In either case, we have access to a number $N$ of examples from which we can
learn; each example is associated with a potentially large number $p$ of {\em
	predictor} variables and the ``ground truth'' discrete or continuous outcome. This
form of learning is called {\em supervised}, because it relies on the existence of
known examples associating predictor variables with the outcome. In the case of a
binary/discrete outcome the problem is referred to as {\em classification}, while for
continuous outcomes we use the term {\em regression}. 

There are many methods to solve such supervised learning problems, from ordinary
(linear) least squares regression, to logistic regression, Classification And
Regression Trees (CART) \citep{breiman2017classification}, ensembles of decision
trees~\citep{breiman2001random, chen2016xgboost}, to modern deep learning
models~\citep{goodfellow2016deep}. Whereas the nonlinear models (random forests,
gradient boosted trees, and deep learning) perform very well in many specific
applications, they have two key drawbacks: $(i)$ they produce predictive models that
lack {\em interpretability} and $(ii)$ they are hard to analyze and do not give rise
to rigorous mathematical results characterizing their performance and important
properties. In this monograph, we will mainly focus on the more classical linear
models, allowing for some nonlinear extensions.

Clearly, there is a plethora of application areas where such models have been
developed and used. A common thread throughout this monograph is formed by
applications in medicine and health care, broadly characterized by the term {\em
	predictive health analytics}. While in principle these applications are not
substantially different from other domains, they have important salient features that
need to be considered. These include: 
\begin{enumerate}
	\item {\em Presence of outliers.} Medical data often contain outliers, which may be
	caused by medical errors, erroneous or missing data, equipment and lab configuration
	errors, or even different interpretation/use of a variable by different physicians
	who enter the data.
	
	\item {\em Risk of ``overfitting'' from too many variables.} For any individual and
	any outcome we wish to predict, using all predictor variables may lead to {\em
		overfitting} and large generalization errors (out-of-sample). The common practice
	is to seek {\em sparse} models, using the fewest variables possible without
	significantly compromising accuracy. In some settings, especially when genetic
	information is included in the predictors, the number of predictors can exceed the
	training sample size, further stressing the need for {\em sparsity}. Sparse
	regression models originated in the seminar work on the {\em Least Absolute
		Shrinkage and Selection Operator}, better known under the acronym
	LASSO~\citep{tibshirani1996regression}.  
	
	\item {\em Lack of linearity.} In some applications, the linearity of regression or
	logistic regression may not fully capture the relationship between predictors and
	outcome. While kernel methods \citep{friedman2001elements} can be used to employ
	linear models in developing nonlinear predictors, other choices include combining
	linear models with nearest neighbor ideas to essentially develop {\em piecewise
		linear} models. 
\end{enumerate}

To formulate the learning problems of interest more concretely, let
$\bx=(x_1,\ldots,x_p)\in \mbb{R}^p$ denote a column vector with the predictors and
let $y\in \mbb{R}$ be the outcome or response. In the classification problem, we have
$y \in \{-1, +1\}$. We are given training data $(\bx_i,y_i)$, $i \in \lb N \rb$, where
$\lb N \rb \stackrel{\triangle}{=} 1,\ldots,N$, from
which we want to ``learn'' a function $f(\cdot)$ so that $f(\bx_i)=y_i$ for most
$i$. Further, we want $f(\cdot)$ to generalize well to new samples (i.e., to have
good out-of-sample performance).

In the regression problem, we view the $\bx_i$'s as independent variables (predictor
vectors) and $y_i$ as the real-valued dependent variable.  We still want to determine
a function $f(\bx)$ that predicts $y$. In linear regression, $f(\bx) = \bbeta'\bx$,
where $\bbeta$ is a coefficient vector, prime denotes transpose, and we assume one of
the elements of $\bx$ is equal to one with the corresponding coefficient being the
{\em intercept} (of the regression function at zero). Both classification and
regression problems can be formulated as:
\begin{equation} \label{loss}
\textstyle\min_{\bbeta} \mbb{E}^{\mbb{P}^*} [h_{\bbeta}(\bx,y)],
\end{equation}
where $\mbb{P}^*$ is the probability distribution of $(\bx, y)$, $\mbb{E}^{\mbb{P}^*}$ stands for the expectation under $\mbb{P}^*$, and $h_{\bbeta}(\bx,y)$ is a {\em loss} function penalizing differences between
$f(\bx)$ and $y$. This formulation is known as {\em expected risk minimization}. {\em
	Ordinary Least Squares (OLS)} uses a squared loss
$h_{\bbeta}(\bx,y)=(f(\bx)-y)^2$ while logistic regression uses the {\em logloss}
function $h_{\bbeta}(\bx,y)=\log(1+\exp\{-yf(\bx)\})$. Since $\mbb{P}^*$ is typically unknown, a common practice is to approximate it using the {\em empirical distribution} $\hat{\mbb{P}}_N$ which assigns equal probability to each training sample, leading to the following {\em empirical risk minimization} formulation:
\begin{equation*}
\textstyle\min_{\bbeta} \frac{1}{N} \sum_{i=1}^N h_{\bbeta}(\bx_i,y_i).
\end{equation*}

\begin{figure}[ht]
	\begin{center}
		\includegraphics[width=0.55\textwidth]{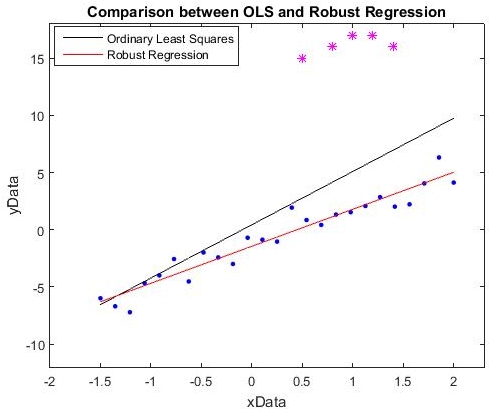}
	\end{center}
	\caption{Regression example.}
	\label{fig:rex}
\end{figure}
One of the well known issues of OLS regression is that the regression function can be
particularly sensitive to outliers. To illustrate this with a simple example,
consider a case of regression with a single predictor; see Fig.~\ref{fig:rex}. Points
in the training set are shown as blue dots. Suppose we include in the training set
some outliers depicted as magenta stars. OLS regression results in the black
line. Notice how much the slope of this line has shifted away from the blue dots to
accommodate the outliers. This skews future predictions but also our ability to
identify new outlying observations. Several approaches have been introduced to
address this issue \citep{huber1964robust,huber1973robust} and we discuss them in
more detail in Section~\ref{chapt:dro}.

The main focus of this monograph is to develop {\em robust learning} methods for a
variety of learning problems. To introduce robustness into the generic problem, we will
use ideas from {\em robust optimization} and formulate a robust version of the
expected risk minimization Problem~(\ref{loss}). We will further focus on {\em
	distributional robustness}. The problems we will formulate are $\min$-$\max$
versions of Problem~(\ref{loss}) where one minimizes a worst case estimate of the loss
over some appropriately defined ambiguity set. Such $\min$-$\max$ formulations have a
long history, going back to the origins of game theory \citep{von1944theory}, where
one can view the problem as a game between an adversary who may affect the training
set and the optimizer who responds to the worst-case selection by the adversary. They
also have strong connections with $\scrH_\infty$ and robust control theory
\citep{zames1981feedback, zhou1998essentials}. 

To avoid being overly broad, we will restrict our attention to the intersection of
statistical learning and {\em Distributionally Robust Optimization (DRO)} under the
Wasserstein metric \citep{esfahani2018data, gao2016distributionally,
	chen2018robust}. Even this more narrow area has generated a lot of interest and
recent work. While we will cover several aspects, we will not cover a number of
topics, including: 
\begin{itemize}
	\item the integration of DRO with different optimization schemes, e.g., inverse optimization \citep{esfahani2018inverse}, polynomial optimization \citep{mevissen2013data}, multi-stage optimization \citep{zhao2015data, hanasusanto2018conic},  and chance-constrained optimization \citep{xie2019distributionally, ji2020data};
	
	\item the application of DRO to stochastic control problems, see, e.g., \cite{van2015distributionally, yang2018wasserstein, yang2018dynamic}, and statistical hypothesis testing \citep{gao2018robust};
	
	\item the combination of DRO with general estimation techniques, see, e.g., \cite{nguyen2019bridging} for distributionally robust Minimum Mean Square Error Estimation, and \cite{nguyen2018distributionally} for distributionally robust Maximum Likelihood Estimation.
\end{itemize}

Most of the learning problems we consider, except for Section~\ref{dro-rl}, are static {\em single-period} problems where the data are assumed to be independently and identically distributed. For extensions of DRO to a dynamic setting where the data come in a sequential manner, we refer to \cite{abadeh2018wasserstein} for a distributionally robust Kalman filter model, \cite{hanasusanto2013robust, yang2018dynamic, duque2020distributionally} for robust dynamic programming, and \cite{sinha2020formulazero} for a distributionally robust online adaptive algorithm.

In this monograph, we focus mainly on linear predictive models, with the exception of
Section~\ref{ch:presp}, where the non-linearity is captured by a non-parametric {\em
	K-Nearest Neighbors (K-NN)} model. For extensions of robust optimization to
non-linear settings, we refer to \cite{shang2017data} for robust kernel methods,
\cite{fathony2018distributionally} for distributionally robust graphical models, and
\cite{sinha2017certifiable} for distributionally robust deep neural networks.

In the remainder of this Introduction, we will present a brief outline of robust
optimization in Section~\ref{sec:ro} and distributionally robust optimization in
Section~\ref{sec:dro-overview}. In Section~\ref{sec:outline} we provide an outline of
the topics covered in the rest of the monograph. Section~\ref{sec:notation}
summarizes our notational conventions and Section~\ref{sec:abbrv} collects all
abbreviations we will use.

\section{Robust Optimization} \label{sec:ro}
{\em Robust optimization} \citep{ben2009robust, bertsimas2011theory} provides a way of modeling uncertainty in the data without the use of probability distributions. It restricts data perturbations to be within a deterministic uncertainty set, and seeks a solution that is optimal for the worst-case realization of this uncertainty. Consider a general optimization problem:
\begin{equation} \label{ro} 
\displaystyle\min\limits_{\bbeta} \ h_{\bbeta}(\bz),
\end{equation}
where $\bbeta$ is a vector of decision variables, $\bz$ is a vector of given parameters, and $h$ is a real-valued function. Assuming that the values of $\bz$ lie within some uncertainty set $\scrZ$, a robust counterpart of Problem (\ref{ro}) can be written in the following form:
\begin{equation} \label{robustro} 
\min\limits_{\bbeta} \max\limits_{\bz \in \scrZ}  \ h_{\bbeta}(\bz).
\end{equation} 
Problem (\ref{robustro}) is computationally tractable for many classes of uncertainty
sets $\scrZ$. For a detailed overview of robust optimization we refer to
\cite{ben2009robust, ben2008selected, bertsimas2011theory}. 

There has been an increasing interest in using robust optimization to develop machine learning algorithms that are immunized against data perturbations; see, for example, \cite{LAU97, xu2009robust, yang2013unified, bertsimas2017characterization} for regression, and \cite{el2003robust, trafalis2006robust, liu2014robust, bertsimas2018robust} for classification methods. \cite{bertsimas2018robust} considered both feature uncertainties:
\begin{equation*}
\scrZ_{\bx} \triangleq \Big\{\bDelta\bX \in \mbb{R}^{N \times p}: \ \|\bDelta \bx_i\|_q \le \rho, \ i \in \lb N \rb \Big\},
\end{equation*} 
where 
$\bDelta\bX$ can be viewed as a feature perturbation matrix on $N$ samples
with $p$ features,
$\|\cdot\|_q$ is the $\ell_q$ norm, and $\bDelta \bx_i \in \mbb{R}^p, i \in \lb N \rb$,
are the rows of $\bDelta\bX$, as well as label uncertainties: 
\begin{equation*}
\scrZ_{y} \triangleq \Big\{ \bDelta \by \in \{0,1\}^N: \ \sum_{i=1}^N \Delta y_i \le \Gamma \Big\},
\end{equation*}
where $\Delta y_i \in \{0, 1\}$, with $1$ indicating that the label was incorrect and has in fact been flipped, and $0$ otherwise, and $\Gamma$ is an integer-valued parameter controlling the number of data points that are allowed to be mislabeled.
They solved various robust classification models under these uncertainty sets. As an example, the robust Support Vector Machine (SVM) \citep{cortes1995support} problem was formulated as:
\begin{equation*} 
\min\limits_{\bw, b} \max\limits_{\bDelta \by \in \scrZ_{y}} \max\limits_{\bDelta\bX \in \scrZ_{\bx}} \sum_{i=1}^N \max\Big\{1-y_i(1-2 \Delta y_i) (\bw'(\bx_i + \bDelta \bx_i) - b), 0 \Big \}.
\end{equation*} 
\cite{xu2009robust} studied a robust linear regression problem with feature-wise disturbance:
\begin{equation*}
\min_{\bbeta} \max_{\bDelta \bX \in \scrZ_{\bx}} \|\by - (\bX + \bDelta \bX) \bbeta\|_2,
\end{equation*}
where $\bbeta$ is the vector of regression coefficients, and the uncertainty set
\begin{equation*}
\scrZ_{\bx} \triangleq \Big\{\bDelta\bX \in \mbb{R}^{N \times p}: \ \|\bDelta \tilde{\bx}_i\|_2 \le c_i, \ i \in \lb p \rb \Big\},
\end{equation*}
where $\bDelta \tilde{\bx}_i \in \mbb{R}^N, i \in \lb p \rb$, are the columns of $\bDelta\bX$. They showed that such a robust regression problem is equivalent to the following $\ell_1$-norm regularized regression problem:
\begin{equation*}
\min_{\bbeta} \|\by - \bX \bbeta\|_2 + \sum_{i=1}^p c_i |\beta_i|.
\end{equation*} 

\section{Distributionally Robust Optimization} \label{sec:dro-overview}

Different from robust optimization, {\em Distributionally Robust Optimization (DRO)}
treats the data uncertainty in a probabilistic way. It minimizes a worst-case
expected loss function over a probabilistic ambiguity set that is constructed from
the observed samples and characterized by certain known properties of the true
data-generating distribution. DRO has been an active area of research in recent
years, due to its probabilistic interpretation of the uncertain data, tractability
when assembled with certain metrics, and extraordinary performance observed on
numerical examples, see, for example, \cite{gao2016distributionally,
	gao2017wasserstein, shafieezadeh2017regularization, esfahani2018data,
	chen2018robust}.  DRO can be interpreted in two related ways: it refers to $(i)$ a
robust optimization problem where a worst-case loss function is being hedged against;
or, alternatively, $(ii)$ a stochastic optimization problem where the expectation of
the loss function with respect to the probabilistic uncertainty of the data is being
minimized. Figure~\ref{fig:ro-dro} provides a schematic comparison of various
optimization frameworks.
\begin{figure}[ht]
	\centering
	\includegraphics[height = 2.2in]{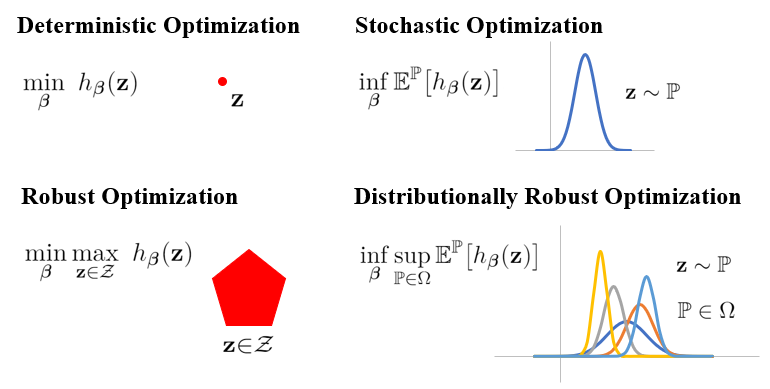}
	\caption{Comparison of robust optimization with distributionally robust optimization.}
	\label{fig:ro-dro}
\end{figure}

To formulate a DRO version of the expected risk minimization problem (\ref{loss}), consider the
stochastic optimization problem: 
\begin{equation} \label{sto}
\inf \limits_{\bbeta}
\mbb{E}^{\mbb{P}^*}\big[ h_{\bbeta}(\bz)\big], 
\end{equation} 
where we set $\bz=(\bx,y)\in \scrZ\subseteq \mbb{R}^d$ in (\ref{loss}), $\bbeta \in
\mbb{R}^p$ is a vector of coefficients to be learned, $h_{\bbeta}(\bz): \scrZ \times
\mbb{R}^p \ra \mbb{R}$ is the loss function of applying $\bbeta$ on a sample $\bz \in
\scrZ$, and $\mbb{P}^*$ is the underlying true probability distribution of $\bz$. The
DRO formulation for (\ref{sto}) minimizes the worst-case expected loss over a
probabilistic ambiguity set $\Omega$:
\begin{equation} \label{dro}
\inf\limits_{\bbeta}\sup\limits_{\mbb{Q}\in \Omega}
\mbb{E}^{\mbb{Q}}\big[ h_{\bbeta}(\bz)\big].
\end{equation}
The existing literature on DRO can be split into two main branches,
depending on the way in which $\Omega$ is defined. One is through a
moment ambiguity set, which contains all distributions that satisfy
certain moment constraints \citep{San11, popescu2007robust, Ye10,
	goh2010distributionally, zymler2013distributionally, Sim14}. In many
cases it leads to a tractable DRO problem but has been criticized for
yielding overly conservative solutions \citep{wang2016likelihood}. The
other is to define $\Omega$ as a ball of distributions:
\begin{equation*}
\Omega \triangleq \Big\{\mbb{Q}\in \scrP(\scrZ): D(\mathbb{Q},\ \mathbb{P}_0) \le \epsilon \Big\},
\end{equation*}
where $\scrZ$ is the set of possible values for $\bz$; $\scrP(\scrZ)$ is
the space of all probability distributions supported on $\scrZ$;
$\epsilon$ is a pre-specified radius of the set $\Omega$; and
$D(\mbb{Q},\ \mbb{P}_0)$ is a probabilistic distance function that measures the distance between
$\mbb{Q}$ and a nominal distribution $\mbb{P}_0$.  

The nominal distribution $\mbb{P}_0$ is typically chosen as the empirical distribution on the observed samples $\{\bz_1, \ldots, \bz_N\}$:
$$\mbb{P}_0 = \hat{\mbb{P}}_{N} \triangleq \frac{1}{N} \sum_{i=1}^N
\delta_{\bz_i}(\bz),$$ where $\delta_{\bz_i}(\cdot)$ is the Dirac density assigning
probability mass equal to $1$ at $\bz_i$;  see \cite{esfahani2018data,
	abadeh2015distributionally, chen2018robust}. There are also works employing a
nonparametric kernel density estimation method to obtain a continuous density
function for the nominal distribution, when the underlying true distribution is
continuous, see \cite{jiang2018risk, zhao2015data2}. The kernel density estimator is
defined as:
\begin{equation*}
f_0(\bz) = \frac{1}{N|\bH|^{1/2}} \sum\limits_{i=1}^N K\big(\bH^{-1/2}(\bz - \bz_i)\big),
\end{equation*} 
where $f_0$ represents the density function of the nominal distribution $\mbb{P}_0$, i.e., $f_0 = \mathrm{d}\mbb{P}_0/\mathrm{d}\bz$, $\bH \in \mbb{R}^{d \times d}$ represents a symmetric and positive definite bandwidth matrix, and $K(\cdot): \mbb{R}^d \rightarrow \mbb{R}^+$ is a symmetric kernel function satisfying $K(\cdot) \ge 0, \int_{\mbb{R}^d} K(\bz) \mathrm{d}\bz = 1$, and $\int_{\mbb{R}^d} K(\bz)\bz \mathrm{d}\bz = \mathbf{0}$.

An example of the probabilistic distance function $D(\cdot,\cdot)$
is the $\phi$-divergence
\citep{bayraksan2015data}:
\begin{equation*}
D(\mbb{Q},\ \mbb{P}_0) = \mbb{E}^{\mbb{P}_0}\Big[\phi\Big(\frac{\mathrm{d}\mbb{Q}}{\mathrm{d}\mbb{P}_0}\Big)\Big],
\end{equation*}
where $\phi(\cdot)$ is a convex function satisfying $\phi(1)=0$. For example, if
$\phi(t)=t\log t$, we obtain the Kullback-Leibler (KL) divergence \citep{Hu13,
	jiang2015data}. The definition of the $\phi$-divergence requires that $\mbb{Q}$ is
absolutely continuous with respect to $\mbb{P}_0$. If we take the empirical measure
to be the nominal distribution $\mbb{P}_0$, this implies that the support of
$\mbb{Q}$ must be a subset of the empirical examples. This constraint could
potentially hurt the generalization capability of DRO.

Other choices for $D(\cdot,\cdot)$ include the Prokhorov
metric \citep{erdougan2006ambiguous}, and the Wasserstein distance
\citep{esfahani2018data, gao2016distributionally, zhao2015data, luo2017decomposition,
	blanchet2016quantifying}. DRO with the Wasserstein metric has been extensively
studied in the machine learning community; see, for example, \cite{chen2018robust,
	blanchet2019robust} for robustified regression models,
\cite{sinha2017certifiable} for adversarial training in neural networks, and
\cite{abadeh2015distributionally} for distributionally robust logistic
regression. \cite{shafieezadeh2017regularization, gao2017wasserstein} provided a
comprehensive analysis of the Wasserstein-based distributionally robust statistical
learning problems with a scalar (as opposed to a vector) response. In recent work,
\cite{blanchet2019multivariate} proposed a DRO formulation for convex regression
under an absolute error loss. 

In this monograph we adopt the Wasserstein metric to define a data-driven DRO problem. Specifically, the ambiguity set $\Omega$ is defined as:
\begin{equation} \label{Omega}
\Omega \triangleq \Big\{\mbb{Q}\in \scrP(\scrZ):
W_{s,t}(\mathbb{Q},\ \hat{\mathbb{P}}_N) \le \epsilon \Big\}, 
\end{equation}
where $\hat{\mbb{P}}_N$ is the uniform empirical distribution over $N$ training
samples $\bz_i$, $i \in \lb N \rb$, and $W_{s,t}(\mbb{Q},\ \hat{\mbb{P}}_N)$ is the
order-$t$ Wasserstein distance ($t \ge 1$) between $\mbb{Q}$ and $\hat{\mbb{P}}_N$
defined as:
\begin{equation} \label{wass_p}
W_{s,t}(\mbb{Q}, \hat{\mbb{P}}_N) \triangleq 
\biggl(\min\limits_{\pi \in \scrP(\scrZ \times \scrZ)} \int\nolimits_{\scrZ \times \scrZ} \big(s(\bz_1, \bz_2)\big)^t \mathrm{d}\pi \bigl(\bz_1, \bz_2\bigr)\biggr)^{1/t},
\end{equation}
where $s$ is a metric on the data space $\scrZ$, and $\pi$ is the joint distribution of $\bz_1$ and $\bz_2$ with
marginals $\mbb{Q}$ and $\hat{\mbb{P}}_N$, respectively. The Wasserstein
distance between two distributions represents the cost of an
optimal mass transportation plan, where the cost is measured through the
metric $s$. 

We choose the Wasserstein metric for two main reasons. On one hand, the Wasserstein
ambiguity set is rich enough to contain both continuous and discrete relevant
distributions, while other metrics such as the KL divergence, exclude all continuous
distributions if the nominal distribution is discrete \citep{esfahani2018data,
	gao2016distributionally}. Furthermore, considering distributions within a KL
distance from the empirical, does not allow for probability mass outside the support of
the empirical distribution.

On the other hand, measure concentration results
guarantee that the Wasserstein set contains the true data-generating
distribution with high confidence for a sufficiently large sample size
\citep{Four14}. Moreover, the Wasserstein metric takes into account the closeness between support points while other metrics such as the $\phi$-divergence only consider the probabilities on these points. An image retrieval example in
\cite{gao2016distributionally} suggests that the probabilistic ambiguity set constructed based on the KL divergence prefers the pathological distribution to the true
distribution, whereas the Wasserstein
distance does not exhibit such a problem. The reason lies in that
the $\phi$-divergence does not incorporate a notion of closeness between two
points, which in the context of image retrieval represents the
perceptual similarity in color.

\section{Outline} \label{sec:outline}

The goal of this monograph is to develop a comprehensive robust statistical learning framework using a Wasserstein-based DRO as the modeling tool. Specifically, 
\begin{itemize}
	\item we provide background knowledge on the basics of DRO and the Wasserstein metric, and show its robustness inducing property through discussions on the Wasserstein ambiguity set and the property of the DRO solution;
	
	\item we cover a variety of predictive and prescriptive models that can be posed and solved using the Wasserstein DRO approach, and show novel problem-tailored theoretical results and real world applications, strengthening the notion of robustness through these discussions; 
	
	\item we consider a variety of synthetic and real world case studies of the respective models, which validate the theory and the proposed DRO approach and highlight its advantages compared to several alternatives. This could potentially $(i)$ ease the understanding of the model and approach; and $(ii)$ attract practitioners from various fields to put these models into use.  
\end{itemize}

Robust models can be useful when $(i)$ the training data is contaminated with noise,
and we want to learn a model that is immunized against the noise; or $(ii)$ the
training data is pure, but the test set is contaminated with outliers. In both
scenarios we require the model to be insensitive to the data
uncertainty/unreliability, which is characterized through a probability distribution
that resides in a set consisting of all distributions that are within a pre-specified
distance from a nominal distribution. The learning problems that are studied in this
monograph include: 
\begin{itemize}
	\item {\em Distributionally Robust Linear Regression (DRLR)}, which estimates a robustified linear regression plane by minimizing the worst-case expected absolute loss over a probabilistic ambiguity set characterized by the Wasserstein metric;
	
	\item {\em Groupwise Wasserstein Grouped LASSO (GWGL)}, which aims at inducing sparsity at a group level when there exists a predefined grouping structure for the predictors, through defining a specially structured Wasserstein metric for DRO;
	
	\item {\em Distributionally Robust Multi-Output Learning}, which solves a DRO problem with a multi-dimensional response/label vector, generalizing the single-output model addressed in DRLR.
	
	\item Optimal decision making using {\em DRLR informed K-Nearest Neighbors (K-NN) estimation}, which selects among a set of actions the optimal one through predicting the outcome under each action using K-NN with a distance metric weighted by the DRLR solution;
	
	\item {\em Distributionally Robust Semi-Supervised Learning}, which estimates a robust classifier with partially labeled data, through $(i)$ either restricting the marginal distribution to be consistent with the unlabeled data, $(ii)$ or modifying the structure of DRO by allowing the center of the ambiguity set to vary, reflecting the uncertainty in the labels of the unsupervised data.
	
	\item {\em Distributionally Robust Reinforcement Learning}, which considers {\em
		Markov Decision Processes (MDPs)} and seeks to inject robustness into the
	probabilistic transition model, deriving a lower bound for the {\em
		distributionally robust} value function in a regularized form.
	
\end{itemize}

The remainder of this monograph is organized as follows. Section~\ref{chapt:wass}
presents basics and key properties for the Wasserstein
metric. Section~\ref{chapt:solve} discusses how to solve a general Wasserstein DRO
problem, the structure of the worst-case distribution, and the performance guarantees
of the DRO estimator. The rest of the sections are dedicated to specific learning
problems that can be posed as a DRO problem.

In Section~\ref{chapt:dro}, we develop the Wasserstein DRO formulation for linear
regression under an absolute error loss.  Section~\ref{chap:group} discusses
distributionally robust grouped variable selection, and develops the {\em Groupwise
	Wasserstein Grouped LASSO (GWGL)} formulation under the absolute error loss and
log-loss.  In Section~\ref{chap:multi}, we generalize the single-output model and
develop distributionally robust multi-output learning models under Lipschitz
continuous loss functions and the multiclass log-loss.  Section~\ref{ch:presp}
presents an optimal decision making framework which selects among a set of actions
the best one, using predictions from {\em K-Nearest Neighbors (K-NN)} with a metric
weighted by the Wasserstein DRO solution.  Section~\ref{chapt:adv} covers a number of
active research topics in the domain of DRO under the Wasserstein metric, including
$(i)$ DRO in {\em Semi-Supervised Learning (SSL)} with partially labeled datasets; $(ii)$
DRO in {\em Reinforcement Learning (RL)} with temporal correlated data.  We close the
monograph by discussing further potential research directions in
Section~\ref{chapt:concl}. 

\section{Notational Conventions} \label{sec:notation}

\noindent {\bf Vectors} 

\begin{itemize}
	
	\item Boldfaced lowercase letters denote vectors, ordinary lowercase letters
	denote scalars, boldfaced uppercase letters denote matrices, and calligraphic
	capital letters denote sets. 
	
	\item $\mathbf{e}_i$ denotes the $i$-th unit vector, $\mathbf{e}$ or $\mathbf{1}$ the vector of ones,
	and $\bzero$ a vector of zeros.  
	
	\item All vectors are column vectors. For space saving
	reasons, we write $\bx=(x_1, \ldots, x_{\text{dim}(\bx)})$ to denote the column
	vector $\bx$, where $\text{dim}(\bx)$ is the dimension of $\bx$. 
	
\end{itemize}

\noindent {\bf Sets and functions}

\begin{itemize}
	
	\item We use $\mbb{R}$ to denote the set of real numbers, and $\mbb{R}^+$ the set of
	non-negative real numbers. 
	
	\item For a set $\scrX$, we use $|\scrX|$ to denote its cardinality.
	
	\item We write $\text{cone}\{\bv \in \scrV\}$ for a cone that is generated from the
	set of vectors $\bv \in \scrV$. 
	
	\item $\mathbf{1}_{\scrA}(\bx)$ denotes the indicator function, i.e., $\mathbf{1}_{\scrA}(\bx) = 1$ if $\bx \in \scrA$, and 0 otherwise. 
	
	\item For $\bz \triangleq (\bx, y) \in \scrX \times \scrY$ and a function $h$, the notations
	$h(\bz)$ and $h(\bx, y)$ are used interchangeably, and $\scrZ \triangleq \scrX \times \scrY$.
	
	\item $\scrB(\scrZ)$ denotes the set of Borel measures supported on $\scrZ$, and $\scrP(\scrZ)$ denotes the set of Borel probability measures
	supported on $\scrZ$.
	
	\item For any integer $n$ we write $\lb n \rb$ for the set $\{1,\ldots,n\}$. Hence,
	$\scrP(\lb n \rb)$ denotes the $n$-th dimensional probability simplex. 
	
\end{itemize}

\noindent {\bf Matrices}

\begin{itemize}
	
	\item $\bI$ denotes the identity matrix.
	
	\item Prime denotes transpose. Specifically, $\bA'$ denotes the transpose of a matrix
	$\bA$.
	
	\item For a matrix $\bA\in \mbb{R}^{m\times n}$, we will denote by
	$\bA=(a_{ij})_{i\in \lb m \rb}^{j\in \lb n \rb}$ the elements of $\bA$, by
	$\ba_1,\ldots,\ba_m$ the rows of $\bA$, and, with some abuse of our notation which denotes
	vectors by lowercase letters, we will denote by $\bA_1,\ldots,\bA_n$ the columns
	of $\bA$. 
	
	\item For a symmetric matrix $\bA$, we write $\bA \succ 0$ to denote a positive definite matrix, and $\bA \succcurlyeq 0$ a positive semi-definite matrix. 
	
	\item  $\diag(\bx)$ denotes a diagonal matrix whose main
	diagonal consists of the elements of $\bx$ and all off-diagonal elements are zero.  
	
	\item $\tr(\bA)$ denotes the trace (i.e., sum of the diagonal elements) of a square
	matrix $\bA\in \mbb{R}^{n\times n}$. 
	
	\item $|\bA|$ denotes the determinant of a square
	matrix $\bA\in \mbb{R}^{n\times n}$.
	
\end{itemize}

\noindent {\bf Norms}

\begin{itemize}
	
	\item $\|\bx\|_p\triangleq(\sum_i |x_i|^p)^{1/p}$ denotes the $\ell_p$ norm with $p \ge
	1$, and $\|\cdot\|$ the general vector norm that satisfies the following
	properties:
	\begin{enumerate}
		\item $\|\bx\| = 0$ implies $\bx=\mathbf{0}$;
		\item $\|a\bx\| = |a| \|\bx\|$, for any scalar $a$; 
		\item $\|\bx + \by\| \le \|\bx\| + \|\by\|$;
		\item $\|\bx\| = \||\bx|\|$, where $|\bx| = (|x_1|, \ldots, |x_{\text{dim}(\bx)}|)$;
		\item $\|(\bx, \mathbf{0})\| = \|\bx\|$, for an arbitrarily long vector $\mathbf{0}$.
	\end{enumerate}
	
	\item Any $\bW$-weighted $\ell_p$ norm defined as 
	\[ 
	\|\bx\|_p^{\bW} \triangleq
	\big((|\bx|^{p/2})'\bW |\bx|^{p/2}\big)^{1/p}
	\] 
	with a positive definite matrix $\bW$ satisfies the above conditions, where
	$|\bx|^{p/2} = (|x_1|^{p/2}, \ldots, |x_{\text{dim}(\bx)}|^{p/2})$.
	
	\item For a matrix $\bA \in \mbb{R}^{m \times n}$, we use $\|\bA\|_p$ to denote its induced $\ell_p$ norm that is defined as $\|\bA\|_p \triangleq \sup_{\bx \neq \mathbf{0}} \|\bA \bx\|_p/\|\bx\|_p$.

\end{itemize}

\noindent {\bf Random variables}

\begin{itemize}
	
	\item For two random variables $w_1$ and $w_2$, we say that $w_1$ is stochastically
	dominated by $w_2$, denoted by $w_1 \overset{\mathclap{\mbox{\tiny D}}}{\le} w_2$, if $\mbb{P}(w_1 \ge x) \le \mbb{P}(w_2 \ge x)$ for all $x \in \mbb{R}$.
	
	\item For a dataset $\scrD \triangleq \{\bz_1, \ldots, \bz_N\}$, we use
	$\hat{\mbb{P}}_{N}$ to denote the empirical measure supported on $\scrD$, i.e.,
	$\hat{\mbb{P}}_{N} \triangleq \frac{1}{N} \sum_{i=1}^N \delta_{\bz_i}(\bz)$, where
	$\delta_{\bz_i}(\bz)$ denotes the Dirac delta function at point $\bz_i \in \scrZ$.
	
	\item The $N$-fold product of a distribution $\mbb{P}$ on $\scrZ$ is denoted by $\mbb{P}^N$, which represents a distribution on the Cartesian product space $\scrZ^N$. We write $\mbb{P}^{\infty}$ to denote the limit of $\mbb{P}^N$ as $N \rightarrow \infty$.
	
	\item $\mathbb{E}^{\mbb{P}}$ denotes the expectation under a probability distribution
	$\mbb{P}$. 
	
	\item For a random vector $\bx$, $\text{cov}(\bx)$ will denote its covariance. 
	
	\item $\scrN_p(\mathbf{0}, \bSigma)$ denotes the $p$-dimensional Gaussian distribution with mean $\mathbf{0}$ and covariance matrix $\bSigma$.
	
	\item For a distribution $\mbb{P} \in \scrP(\scrX \times \scrY)$,
	$\mbb{P}_{\scrX}(\cdot) \triangleq \sum_{y \in \scrY} \mbb{P}( \cdot,y)$ denotes
	the marginal distribution over $\scrX$, and $\mbb{P}_{|\bx} \in \scrP^{\scrX}(\scrY)$ is
	the conditional distribution over $\scrY$ given $\bx \in \scrX$, where $\scrP^{\scrX}(\scrY)$ denotes the set of all conditional distributions supported on $\scrY$, given features in $\scrX$. 
	
	\item $W_{s,t}(\mbb{P}, \mbb{Q})$ denotes the order-$t$ Wasserstein distance between
	measures $\mbb{P}, \mbb{Q}$ under a cost metric $s$. For ease of notation and when
	the cost metric is clear from the context we will be writing $W_{t}(\mbb{P},
	\mbb{Q})$.
	
	\item $\Omega_{\epsilon}^{s,t}(\mbb{P})$ denotes the set of probability distributions whose order-$t$ Wasserstein distance under a cost metric $s$ from the distribution $\mbb{P}$ is less than or equal to $\epsilon$, i.e.,
	\[
	\Omega_{\epsilon}^{s,t}(\mbb{P}) \triangleq \{\mbb{Q}\in \scrP(\scrZ):
	W_{s,t}(\mathbb{Q},\ \mbb{P}) \le \epsilon \}.
	\]
	For ease of notation, when the cost metric is clear from the context and $t=1$, we will be writing $\Omega_{\epsilon}(\mbb{P})$, or simply $\Omega$ when the center distribution $\mbb{P}$ is clear from the context.
	
\end{itemize}

\newpage

\section{Abbreviations} \label{sec:abbrv} 

\begin{center}
	\begin{tabular}{lll}
		\hspace*{2em} & \hspace*{1in} & \hspace*{4.5in} \\
		ACE & \dotfill & Angiotensin-Converting Enzyme \\
		ACS & \dotfill & American College of Surgeons\\
		AD & \dotfill & Absolute Deviation \\
		ARB & \dotfill & Angiotensin Receptor Blockers \\
		a.s. & \dotfill &  almost surely \\
		AUC & \dotfill & Area Under the ROC Curve \\
		BMI & \dotfill & Body Mass Index \\
		CART & \dotfill & Classification And Regression Trees \\
		CCA & \dotfill & Canonical Correlation Analysis \\
		CCR & \dotfill & Correct Classification Rate \\
		CI & \dotfill & Confidence Interval \\
		CT & \dotfill & Computed Tomography \\
		CTDI & \dotfill & CT Dose Index \\
		CVaR & \dotfill & Conditional Value at Risk \\
		C\&W & \dotfill & the Curds and Whey procedure \\
		DRLR   & \dotfill & Distributionally Robust Linear\\
		&&  Regression \\ 
		DRO  & \dotfill & Distributionally Robust Optimization \\
		EHRs & \dotfill & Electronic Health Records \\
		EN & \dotfill & Elastic Net \\
		FA & \dotfill & False Association \\
		FD & \dotfill & False Disassociation \\
		FES & \dotfill & Factor Estimation and Selection \\
		GLASSO & \dotfill & Grouped LASSO\\
		GSRL & \dotfill & Grouped Square Root LASSO \\
		GWGL  & \dotfill & Groupwise Wasserstein Grouped\\
		&&  LASSO  \\
		HbA\textsubscript{1c} & \dotfill & hemoglobin A1c \\
		HIPAA & \dotfill & Health Insurance Portability and \\
		& & Accountability Act \\
		ICD-9 & \dotfill & International Classification of \\
		& & Diseases, Ninth Revision \\ 
	\end{tabular}
\end{center}

\begin{center}
	\begin{tabular}{lll}
		\hspace*{2em} & \hspace*{1in} & \hspace*{4.5in} \\
		%DRO-MLG & \dotfill & Distributionally Robust Optimization for Multiclass Logistic Regression \\
		%DRO-MLR & \dotfill & Distributionally Robust Optimization-Multivariate Linear Regression \\

		%GWGL-LG & \dotfill & Groupwise Wasserstein Grouped LASSO for Logistic Regression \\
		%GWGL-LR & \dotfill & Groupwise Wasserstein Grouped LASSO for Linear Regression \\
		i.i.d. & \dotfill & independently and identically \\
		&& distributed \\
		IRB & \dotfill & Institutional Review Board \\
		IRLS & \dotfill & Iteratively Reweighted Least Squares \\
		KL  & \dotfill & Kullback-Leibler \\
		K-NN & \dotfill & K-Nearest Neighbors \\
		LAD & \dotfill & Least Absolute Deviation \\
		LASSO & \dotfill & Least Absolute Shrinkage and \\
		& & Selection Operator \\
		LG & \dotfill & Logistic Regression \\
		LHS & \dotfill & Left Hand Side \\
		LMS & \dotfill & Least Median of Squares \\
		LOESS & \dotfill & LOcally Estimated Scatterplot\\
		&&  Smoothing \\
		LTS & \dotfill & Least Trimmed Squares \\
		MAD & \dotfill & Median Absolute Deviation \\
		MCC & \dotfill & MultiClass Classification \\
		MDP & \dotfill & Markov Decision Process \\
		MeanAE & \dotfill & Mean Absolute Error \\
		min-max  & \dotfill & minimization-maximization \\
		MLE & \dotfill & Maximum Likelihood Estimator \\
		MLG  & \dotfill & Multiclass Logistic Regression \\
		%MLG-1S & \dotfill & Multiclass Logistic Regression with the $L_{1,s}$ matrix norm regularizer \\
		%MLG-SR & \dotfill & Multiclass Logistic Regression with the $L_{s,r}$ matrix norm regularizer \\
		MLR  & \dotfill & Multi-output Linear Regression \\
		%MLR-1S & \dotfill & Multivariate Linear Regression with the $L_{1,s}$ matrix norm regularizer \\
		%MLR-SR & \dotfill & Multivariate Linear Regression with the $L_{s,r}$ matrix norm regularizer \\mmHg & \dotfill & millimeter of mercury \\
		MPD & \dotfill & Minimal Perturbation Distance \\
		MPI & \dotfill & Maximum Percentage Improvement \\
		MPMs & \dotfill & Minimax Probability Machines \\
		MSE & \dotfill & Mean Squared Error \\
		NPV & \dotfill & Negative Predictive Value\\
		NSQIP & \dotfill & National Surgical Quality \\
		& & Improvement Program \\
		OLS & \dotfill & Ordinary Least Squares \\
		PCR & \dotfill & Principal Components Regression \\
		PPV & \dotfill & Positive Predictive Value \\
		PVE & \dotfill & Proportion of Variance Explained \\		
	\end{tabular}
\end{center}
%LG-LASSO
%LG-Ridge
%LG-EN
%GR & \dotfill & Growth Rate \\
%log-loss \\
\begin{center}
	\begin{tabular}{lll}
		\hspace*{2em} & \hspace*{1in} & \hspace*{4.5in} \\
		RBA & \dotfill & Robust Bias-Aware \\
		RHS & \dotfill & Right Hand Side \\
		RL & \dotfill & Reinforcement Learning \\
		ROC & \dotfill & Receiver Operating Characteristic \\
		RR & \dotfill & Relative Risk \\
		RRR & \dotfill & Reduced Rank Regression \\
		RTE & \dotfill & Relative Test Error \\
		SNR & \dotfill & Signal to Noise Ratio \\
		SR & \dotfill & Squared Residuals \\
		SSL & \dotfill & Semi-Supervised Learning \\
		std & \dotfill & standard deviation \\
		SVM  & \dotfill & Support Vector Machine \\ 
		TA & \dotfill & True Association \\
		TD & \dotfill & True Disassociation \\
		TAR & \dotfill & True Association Rate \\		
		TDR & \dotfill & True Disassociation Rate \\
		WGD & \dotfill & Within Group Difference \\
		w.h.p. & \dotfill &  with high probability \\
		WMSE & \dotfill & Weighed Mean Squared Error \\
		w.p.$1$ & \dotfill &  with probability $1$ \\
		w.r.t. & \dotfill &  with respect to \\
	\end{tabular}
\end{center}

\chapter{The Wasserstein Metric} \label{chapt:wass}

In this section, we outline basic properties of the Wasserstein distance. A
definition in the case of discrete measures is provided in
Section~\ref{sec:wass-basic}. Section~\ref{sec:wmetric} establishes that it is a
proper distance metric. A dual formulation and a generalization to arbitrary measures
are presented in Section~\ref{sec:wdual}. Special cases are described in
Section~\ref{sec:wspecial}. A discussion on how to set the Wasserstein underlying
transport cost function in the context of robust learning is in
Section~\ref{sec:wtransport}. A related robustness-inducing property of the
Wasserstein metric is shown in Section~\ref{sec:robust} and a discussion on how to
set the radius of the Wasserstein ambiguity set is included in
Section~\ref{sec:wass-radius}.

\section{Basics} \label{sec:wass-basic}

We start by reviewing basic properties of the Wasserstein metric defined in
Section~\ref{chapt:intro} (cf. Eq.~(\ref{wass_p})). We will define the metric and
establish key results, first using discrete probability distributions, and then state
how the definitions and results generalize to arbitrary probability measures.

Consider two discrete probability distributions $\mbb{P}=\{p_1,\ldots,p_m\}$ and
$\mbb{Q}=\{q_1,\ldots,q_n\}$, where $p_i, q_j\geq 0$, for all $i,j$, and
$\sum_{i=1}^m p_i=\sum_{j=1}^n q_j=1$. For convenience, let us write
$\bp=(p_1,\ldots,p_m)$ and $\bq=(q_1,\ldots,q_n)$ for the corresponding column
vectors. Define a metric (cost) between points in the support of $\mbb{P}$ and
$\mbb{Q}$ by $s(i,j)$, $i\in \lb m\rb$, $j\in \lb n\rb$, and collect all these
quantities in an $m\times n$ matrix $\bS=(s(i,j))$ whose $(i,j)$ element is
$s(i,j)$. Consider the {\em Linear Programming (LP)} problem:
\begin{equation} \label{transport}
\begin{array}{rl}
W_{\bS,1} (\mbb{P},\mbb{Q}) = \min_\bpi & \sum_{i=1}^m
\sum_{j=1}^n \pi(i,j) s(i,j)\\ 
\text{s.t.} & \sum_{i=1}^m \pi(i,j) = q_j,\quad j\in \lb n\rb, \\
& \sum_{j=1}^n \pi(i,j) = p_i,\quad i\in \lb m\rb, \\
& \pi(i,j) \geq 0,\quad \forall i,j,
\end{array}
\end{equation}
where $\bpi=(\pi(i,j);\, \forall i,j)$ is the decision vector. Notice that according
to the definition in Eq.~(\ref{wass_p}), the objective value is the order-$1$
Wasserstein distance between distributions $\mbb{P}$ and $\mbb{Q}$. In $W_{\bS,1}
(\mbb{P},\mbb{Q})$ we have inserted the subscript $\bS$ to explicitly denote the
dependence on the cost matrix. Similarly, by defining a cost matrix
$\bS^t=((s(i,j))^t)$, the order-$t$ Wasserstein distance, denoted by
$W_{\bS,t}(\cdot,\cdot)$, can be obtained as the $t$-th root of the optimal value of
the same LP with cost matrix $\bS^t$; namely, 
\begin{equation} \label{t-toLP}
W_{\bS,t}(\mbb{P},\mbb{Q}) = \big(W_{\bS^t,1} (\mbb{P},\mbb{Q})\big)^{1/t}. 
\end{equation}

The LP formulation in (\ref{transport}) is equivalent to the well-known {\em
	transportation problem}~\citep{bets-lp} and can be interpreted as the cost of
transporting probability mass from the support points of $\mbb{P}$ to those of
$\mbb{Q}$. Specifically, the problem corresponds to the bipartite graph in
Figure~\ref{fig:earth-mover} with nodes $\{u_1,\ldots,u_m\}$ representing the support of
$\mbb{P}$, nodes $\{v_1,\ldots,v_n\}$ representing the support of $\mbb{Q}$, $p_i$
being the supply at node $u_i$, $q_j$ the demand at node $v_j$, and $\pi(i,j)$ the
flow of material (probability mass) from node $u_i$ to node $v_j$ incurring a transportation cost of
$s(i,j)$ per unit of material. 
\begin{figure}[ht]
	\begin{center}
		\resizebox{0.5\linewidth}{!}{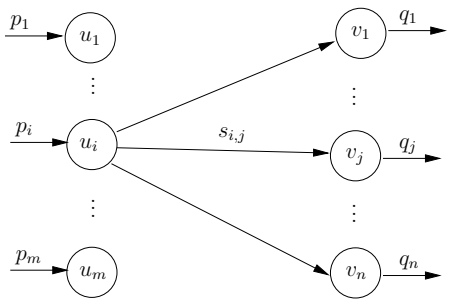}
	\end{center}
	\caption{The transportation problem for computing the Wasserstein distance $W_{\bS,1} (\mbb{P},\mbb{Q})$.}
	\label{fig:earth-mover}
\end{figure}

The formulation in (\ref{transport}) has a long history, starting with Monge
\citep{monge1781memoire} who formulated a problem of optimally transferring material
extracted from a mining site to various construction sites; hence, the terms {\em
	optimal mass transport} and {\em earth mover's distance}. In Monge's formulation,
all material from a source node $u_i$ gets ``assigned'' to a destination node
$v_j$. Kantorovich \citep{kantorovich1942transfer,kantorovich1948monge} relaxed the
problem by allowing sources to split their material to several destination nodes. For
Kantorovich, this was an application of an LP he had earlier defined for production
planning problems~(\cite{kantorovich1939mathematical}, later translated in English in
\cite{kantorovich1960mathematical}) and a method (and a duality theorem) he had
developed for these problems~\citep{kantorovich1940one}. Definitive references on
optimal mass transport are \cite{villani2008optimal}, and, focusing more on
computational aspects, \cite{peyre2019computational}. In presenting some of the key
properties and duality we will follow the approach of \cite{peyre2019computational}
which presents the theory for discrete probability distributions.

\section{A Distance Metric} \label{sec:wmetric}

In this section we establish that the Wasserstein distance
$W_{\bS,t}(\mbb{P},\mbb{Q})$ is a distance metric, assuming that the underlying cost
$s(i,j)$ is a proper distance metric.

\begin{ass} \label{ass:smetric} Let $n=m$ and assume
	\begin{enumerate}
		\item $s(i,j)\geq 0$, with $s(i,j)=0$ if and only if $i=j$. 
		\item $s(i,j)=s(j,i)$ for $i\neq j$. 
		\item For any triplet $i,j,k\in \lb n\rb$, $s(i,k)\leq s(i,j)+s(j,k)$. 
	\end{enumerate}
\end{ass}

\begin{thm}\label{thm:wmetric}
	Under Assumption~\ref{ass:smetric}, the order-$t$ Wasserstein distance ($t\geq 1$) is
	a metric, \ie,
	\begin{enumerate}
		\item $W_{\bS,t}(\mbb{P},\mbb{Q})\geq 0$ for any $\mbb{P}, \mbb{Q}\in \scrP(\lb
		n\rb)$, with $W_{\bS,t}(\mbb{P},\mbb{Q})=0$ if and only if $\mbb{P}=\mbb{Q}$.
		\item $W_{\bS,t}(\mbb{P},\mbb{Q}) = W_{\bS,t}(\mbb{Q},\mbb{P})$ for any $\mbb{P},
		\mbb{Q}\in \scrP(\lb n\rb)$.
		\item For any triplet $\mbb{P},
		\mbb{Q}, \mbb{V} \in \scrP(\lb n\rb)$, $W_{\bS,t}(\mbb{P},\mbb{V})\leq
		W_{\bS,t}(\mbb{P},\mbb{Q})+W_{\bS,t}(\mbb{Q},\mbb{V})$.  
	\end{enumerate}
\end{thm}
\begin{proof}
	Recall Eq.~(\ref{t-toLP}) that relates $W_{\bS,t}(\mbb{P},\mbb{Q})$ to $W_{\bS^t,1}
	(\mbb{P},\mbb{Q})$. The latter quantity can be obtained as the optimal value of the
	LP in (\ref{transport}) using the cost metric $\bS^t$.
	\begin{enumerate} 
		\item The non-negativity follows directly from the formulation in (\ref{transport})
		since $s(i,j)\geq 0$ (by Assumption~\ref{ass:smetric}), hence $(s(i,j))^t\geq 0$,
		and any feasible solution satisfies $\pi(i,j)\geq 0$. In addition,
		$W_{\bS,t}(\mbb{P},\mbb{P})=0$, because, in this case, the optimal solution in
		formulation (\ref{transport}) satisfies $\pi(i,j)=0$, if $i\neq j$, and $\pi(i,i) =
		p_i$, for all $i$. Since $(s(i,i))^t= 0$ (due to Assumption~\ref{ass:smetric}), the
		optimal value of the LP in (\ref{transport}) is zero. Further, if $\mbb{P}\neq
		\mbb{Q}$, there should be flow $\pi(i,j)>0$ for some $i\neq j$, and since
		$s(i,j)>0$ for those $i,j$ (due to Assumption~\ref{ass:smetric}), the optimal value
		of the LP is positive.
		
		\item To establish symmetry, consider $W_{\bS,t}(\mbb{P},\mbb{Q})$ and compare it
		with $W_{\bS,t}(\mbb{Q},\mbb{P})$. It suffices to compare $W_{\bS^t,1}
		(\mbb{P},\mbb{Q})$ with $W_{\bS^t,1} (\mbb{Q},\mbb{P})$. To that end, notice that
		given an optimal solution $\pi_f(i,j)$, for all $i,j$, for $W_{\bS^t,1}
		(\mbb{P},\mbb{Q})$ computed from the LP in (\ref{transport}), we can obtain an
		optimal solution $\pi_b(i,j)$ for $W_{\bS^t,1} (\mbb{Q},\mbb{P})$ simply by
		reversing the flows, \ie, $\pi_b(j,i)=\pi_f(i,j)$, for all $i,j$. Given the
		symmetry of the cost $s(i,j)$ due to Assumption~\ref{ass:smetric}, the result
		follows.
		
		\item To establish the triangle inequality, fix $\mbb{P}, \mbb{Q}, \mbb{V} \in
		\scrP(\lb n\rb)$ and consider $W_{\bS,t}(\mbb{P},\mbb{Q})$ and
		$W_{\bS,t}(\mbb{Q},\mbb{V})$. Let $\bPi_1 = (\pi_1(i,j))_{i,j\in \lb n\rb}$ and
		$\bPi_2 = (\pi_2(i,j))_{i,j\in \lb n\rb}$ be the optimal solutions of the LPs
		corresponding to $W_{\bS^t,1}(\mbb{P},\mbb{Q})$ and $W_{\bS^t,1}(\mbb{Q},\mbb{V})$,
		respectively. Define a $\tilde{\mbb{Q}}$ such that $\tilde{q}_i=q_i$, if
		$q_i>0$, and $\tilde{q}_i=1$, otherwise. Let $\tilde{\bq}$ be the corresponding column
		vector. Define $\bD=\diag{(1/\tilde{q}_1,\ldots,1/\tilde{q}_n)}$. 
		
		Consider next $W_{\bS,t}(\mbb{P},\mbb{V})$ and the LP corresponding to
		$W_{\bS^t,1}(\mbb{P},\mbb{V})$. We will first argue that $\bPi_{1,2}\stackrel{\triangle}{=}\bPi_1 \bD
		\bPi_2$ forms a feasible solution to that LP. Specifically, recalling that $\mb{e}$
		is the vector of all ones, 
		\[ 
		\bPi_{1,2} \mb{e} = \bPi_1 \bD \bPi_2 \mb{e} = \bPi_1 \bD \bq = \bPi_1
		\mb{e}_{\mbb{Q}} = \bp, 
		\]
		where we used the feasibility of $\bPi_1, \bPi_2$, and $\mb{e}_{\mbb{Q}}$ is a vector
		whose $i$th element is set to $1$ if $q_i>0$, and to zero, otherwise. Similarly, we
		can also show $\mb{e}' \bPi_{1,2} = \bv'$, where $\bv$ is the column vector
		corresponding to $\mbb{V}$. 
		
		Letting $\bPi_{1,2} = (\pi_{1,2}(i,j))_{i,j\in \lb n\rb}$, we have
		\begin{align} 
		& W_{\bS,t}(\mbb{P},\mbb{V}) = (W_{\bS^t,1}(\mbb{P},\mbb{V}))^{1/t} \notag\\
		\leq & \left( \sum_{i,j} (s(i,j))^t \pi_{1,2}(i,j) \right)^{1/t} \label{wtriangle-1}\\
		= &\left( \sum_{i,j} (s(i,j))^t \sum_k \frac{\pi_1(i,k) \pi_2(k,j)}{\tilde{q}_k}
		\right)^{1/t} \notag\\
		\leq  & \left( \sum_{i,j,k} (s(i,k)+s(k,j))^t  \frac{\pi_1(i,k) \pi_2(k,j)}{\tilde{q}_k}
		\right)^{1/t} \label{wtriangle-2}\\
		=  & \left( \sum_{i,j,k} \left[ s(i,k) \left(\frac{\pi_1(i,k)
			\pi_2(k,j)}{\tilde{q}_k}\right)^{1/t} \right. \right. \notag \\  
		& \left. \left. \phantom{pppppp} + s(k,j) \left(\frac{\pi_1(i,k)
			\pi_2(k,j)}{\tilde{q}_k}\right)^{1/t}  \right]^t \right)^{1/t} \notag\\
		\leq & \left( \sum_{i,j,k} (s(i,k))^t \frac{\pi_1(i,k)
			\pi_2(k,j)}{\tilde{q}_k} \right)^{1/t} \notag \\ 
		&\phantom{ppppp} + \left( \sum_{i,j,k} (s(k,j))^t
		\frac{\pi_1(i,k)\pi_2(k,j)}{\tilde{q}_k}  \right)^{1/t} \label{wtriangle-3}\\
		= & \left( \sum_{i,k} (s(i,k))^t \pi_1(i,k) \sum_j \frac{
			\pi_2(k,j)}{\tilde{q}_k} \right)^{1/t} \notag \\ 
		&\phantom{ppppp} + \left( \sum_{j,k} (s(k,j))^t \pi_2(k,j) 
		\sum_i \frac{\pi_1(i,k)}{\tilde{q}_k}  \right)^{1/t} \notag\\
		= & \left( \sum_{i,k} (s(i,k))^t \pi_1(i,k) \right)^{1/t} + \left( \sum_{j,k} (s(k,j))^t \pi_2(k,j) 
		\right)^{1/t} \label{wtriangle-4}\\
		= & W_{\bS,t}(\mbb{P},\mbb{Q})+W_{\bS,t}(\mbb{Q},\mbb{V}), \notag 
		\end{align}
		where (\ref{wtriangle-1}) follows from the feasibility (and potential suboptimality)
		of $\pi_{1,2}(i,j)$, (\ref{wtriangle-2}) follows from the triangle inequality for
		$s(i,j)$, (\ref{wtriangle-3}) is due to the Minkowski inequality, and
		(\ref{wtriangle-4}) uses the feasibility of $\bPi_1, \bPi_2$.
	\end{enumerate}
\end{proof}

As a final comment in this section, we note that the order-$1$ Wasserstein distance
$W_{\bS,1}(\mbb{P},\mbb{Q})$, viewed as a function of the vectors $\bp$ and $\bq$
corresponding to $\mbb{P}$ and $\mbb{Q}$, is a convex function. This follows from the
LP formulation (\ref{transport}), where the optimal value is a convex function of the
RHS of the constraints~\citep[Sec. 5.2]{bets-lp}.  

\section{The Dual Problem} \label{sec:wdual}

In this section, we derive the dual of the mass transportation problem in
(\ref{transport}). Let $f_j$ be the dual variable corresponding to the flow
conservation constraint for $q_j$ and $g_i$ the dual variable corresponding to the
flow conservation constraint for $p_i$. We write $\mb{f}\in \mbb{R}^n$ and $\bg\in
\mbb{R}^m$ for the corresponding dual vectors. Using LP duality, the dual of
(\ref{transport}) takes the form:
\begin{equation} \label{dual-transport}
\begin{array}{rl}
W_{\bS,1} (\mbb{P},\mbb{Q}) = \max_{\mb{f},\bg} & \sum_{i=1}^m g_i p_i + 
\sum_{j=1}^n f_j q_j\\ 
\text{s.t.} & f_j+g_i\leq s(i,j),\quad i\in \lb m\rb,\ j\in \lb n\rb. 
\end{array}
\end{equation}
The optimal value is equal to the primal optimal value due to the LP strong duality. 
The complementary slackness conditions suggest that 
\begin{equation} \label{w-cs}
\text{if $\pi(i,j)>0$ then $f_j+g_i= s(i,j)$.} 
\end{equation}
Necessary and sufficient conditions for a primal solution $\bPi$ to be primal optimal
and dual solutions $\mb{f}$ and $\bg$ to be dual optimal are: $(i)$ primal
feasibility, $(ii)$ dual feasibility, and $(iii)$ the complementary slackness
condition in (\ref{w-cs}). 

The primal and dual problems can be interpreted as follows. The primal problem is the
problem of minimizing transportation cost for a transporter of mass across the
bipartite graph in Figure~\ref{fig:earth-mover}. The transporter faces a cost of
$s(i,j)$ per unit of mass transported on link $(i,j)$. Suppose now that the
transporter, instead of carrying out the transportation plan, hires another shipping
company (\eg, a company like UPS, DHL, or Fedex). This shipping company charges a
price of $g_i$ for picking one unit of mass from node $u_i$ and a price of $f_j$ for
delivering one unit of mass to node $v_j$. The dual problem is then the problem
solved by the shipping company to maximize its revenue by carrying out the
transportation of mass. Strong duality simply states that there should not be an
``arbitrage'' opportunity and the transportation cost must be the same irrespective
of whether the transporter of mass hires a shipping company or not. In other words,
if the price offered by the shipping company was strictly less than the
transportation cost, then the mass transporter would be able to make money just by
outsourcing shipping. Furthermore, the market conditions would be ripe for another
middleperson to come into the market, offer the shipping company higher prices, while
still making it profitable for the transporter to use the middleperson's services. More
specifically, the complementary slackness conditions (\ref{w-cs}) suggest that if
there is mass transported along link $(i,j)$, the cost of transporting the mass
through the shipping company must equal the transportation cost faced by the
transporter across that link.

A different interpretation of the primal and the dual can be obtained through an
analogy with electrical circuits. Let us treat $p_i$ as current flowing {\em into} node
$u_i$. Similarly, $q_j$ is current flowing {\em out of} node $v_j$, or, equivalently,
the inflow into $v_j$ is equal to $\hat{q}_j=-q_j$. Rewriting the dual problem
(\ref{dual-transport}) using the $\hat{q}_i$'s and changing variables from $f_j$ to
$\hat{f}_j=-f_j$ yields: 
\begin{equation} \label{dual2-transport}
\begin{array}{rl}
W_{\bS,1} (\mbb{P},\mbb{Q}) = \max_{\mb{\hat{f}},\bg} & \sum_{i=1}^m g_i p_i + 
\sum_{j=1}^n \hat{f}_j \hat{q}_j\\ 
\text{s.t.} & g_i- \hat{f}_j\leq s(i,j),\quad i\in \lb m\rb,\ j\in \lb n\rb. 
\end{array}
\end{equation}
In this context, the constraints of the primal can be viewed as Kirchoff's current
law and the dual variables ($g_i$ at nodes $u_i$ and $\hat{f}_j$ at nodes $v_j$) can
be interpreted as electric potentials (voltages with respect to the ground) at the
nodes. The complementary slackness conditions state that if there is current flowing
from node $u_i$ to $v_j$, the voltage, or potential difference among these nodes,
must equal $s(i,j)$. More simply put, for one unit of flow (current), the voltage
must be equal to the ``resistor'' $s(i,j)$, which corresponds to Ohm's law. These
node potentials are known as Kantorovich potentials~\citep{peyre2019computational}.

\subsection{Arbitrary Measures and Kantorovich Duality}

The primal problem we defined in (\ref{transport}) can be generalized to arbitrary
measures as defined in Eq.~(\ref{wass_p}). Consider two Polish (\ie, complete,
separable, metric) probability spaces $(\scrZ_1, \mbb{P})$ and $(\scrZ_2, \mbb{Q})$
and a lower semicontinuous cost function $s: \scrZ_1\times \scrZ_2 \ra \mbb{R}\cup
\{+\infty\}$. Then, the order-$1$ Wasserstein distance can be defined as the optimal
value of the primal problem:
\begin{equation} \label{transport-cont}
W_{s,1} (\mbb{P},\mbb{Q}) =  \min_\pi \int\nolimits_{\scrZ_1\times \scrZ_2} s(\bz_1,\bz_2)
\mathrm{d}\pi(\bz_1,\bz_2), 
\end{equation}
where $\pi\in \scrP(\scrZ_1\times \scrZ_2)$ is a joint probability distribution of $\bz_1, \bz_2$
with marginals $\mbb{P}$ and $\mbb{Q}$. The order-$t$ Wasserstein distance can be
obtained as: 
\begin{equation} \label{t-toLP-cont}
W_{s,t}(\mbb{P},\mbb{Q}) = \big(W_{s^t,1} (\mbb{P},\mbb{Q})\big)^{1/t}, 
\end{equation}
where $s^t(\bz_1,\bz_2)=(s(\bz_1,\bz_2))^t$. 

The dual problem, known as the Kantorovich
dual~\citep[Thm. 5.10]{villani2008optimal}, analogously to Problem
(\ref{dual-transport}) can be written as:
\begin{equation} \label{dual-transport-cont}
\begin{array}{rl}
W_{s,1} (\mbb{P},\mbb{Q}) = \displaystyle\sup_{f,g} & \int_{\scrZ_1} g(\bz_1) \mathrm{d}\mbb{P}(\bz_1) + 
\int_{\scrZ_2} f(\bz_2) \mathrm{d}\mbb{Q}(\bz_2) \\ 
\text{s.t.} & f(\bz_2)+g(\bz_1)\leq s(\bz_1,\bz_2),\; \bz_1\in \scrZ_1,\ \bz_2\in \scrZ_2, 
\end{array}
\end{equation}
where $f$ and $g$ are absolutely integrable under $\mbb{Q}$ and $\mbb{P}$,
respectively. By the Kantorovich-Rubinstein Theorem~\citep{villani2008optimal}, when
$s(\bz_1, \bz_2)$ is a distance metric on a Polish space $\scrZ_1$,
(\ref{dual-transport-cont}) can be simplified to
\begin{equation} \label{kan-rub}
\begin{array}{rl}
W_{s,1} (\mbb{P},\mbb{Q}) = \displaystyle\sup_{g} & \int_{\scrZ_1} g(\bz_1) \mathrm{d}\mbb{P}(\bz_1) - \int_{\scrZ_2} g(\bz_2) \mathrm{d}\mbb{Q}(\bz_2) \\ 
\text{s.t.} & |g(\bz_1)-g(\bz_2)|\leq s(\bz_1,\bz_2),\; \bz_1\in \scrZ_1,\ \bz_2\in \scrZ_2. 
\end{array}
\end{equation}

\section{Some Special Cases} \label{sec:wspecial}

\subsection{One-Dimensional Cases}

Suppose $\mbb{P}$ and $\mbb{Q}$ are discrete distributions on $\mbb{R}$. Let
$\mbb{P}$ have mass of $1/n$ at each of the points $x_i\in \mbb{R},\ i\in \lb
n\rb$, where $x_1\leq x_2\leq \cdots \leq x_n$. Similarly, $\mbb{Q}$ assigns
mass of $1/n$ at each of the points $y_i\in \mbb{R},\ i\in \lb n\rb$, where
$y_1\leq y_2\leq \cdots \leq y_n$. Then, with $s(x,y)=|x-y|$, the order-$t$
Wasserstein distance can be obtained as:
\begin{equation} \label{1d-discrete}
W_{s,t}(\mbb{P},\mbb{Q}) = \left( \frac{1}{n} \sum_{i=1}^n |x_i-y_i|^t
\right)^{1/t}.
\end{equation}
This can be easily obtained by solving the corresponding formulation in
(\ref{transport}).

For continuous one-dimensional distributions on $\mbb{R}$, let $F_{\mbb{P}}$ denote the
{\em Cumulative Distribution Function (CDF)} of $\mbb{P}$, namely,
\[
F_{\mbb{P}}(x) = \int\nolimits_{-\infty}^x \mathrm{d}\mbb{P}, \qquad x\in \mbb{R}.
\]
Define the inverse CDF or quantile function $F^{-1}_{\mbb{P}}(p)$ as
\[
F^{-1}_{\mbb{P}}(p) = \min \{x\in \mbb{R}\cup \{-\infty\}:\ F_{\mbb{P}}(x)\geq
p\}, \qquad p\in [0,1]. 
\]
Let $F_{\mbb{Q}}$ and $F^{-1}_{\mbb{Q}}$ be the corresponding quantities for
$\mbb{Q}$. Then, using again the metric $s(x,y)=|x-y|$, for $x,y\in \mbb{R}$,
the order-$t$
Wasserstein distance can be computed as~\citep{peyre2019computational}:
\begin{equation} \label{1d-cont}
W_{s,t}(\mbb{P},\mbb{Q}) = \left( \int\nolimits_{0}^{1} \Big|F^{-1}_{\mbb{P}}(p) -
F^{-1}_{\mbb{Q}}(p) \Big|^t \mathrm{d}p \right)^{1/t}.
\end{equation}

\subsection{Sliced Wasserstein Distance}

The fact that Wasserstein distances can be easily computed for one-dimensional
distributions on $\mbb{R}$ has led to the following approximation of the
Wasserstein distance between distributions $\mbb{P}$ and $\mbb{Q}$ on
$\mbb{R}^d$. Specifically, for any direction $\btheta$ on the ball
$\scrS^d=\{\btheta \in \mbb{R}^d: \ \|\btheta\|_2=1\}$, let $T_{\btheta}: \bx\in
\mbb{R}^d \ra \mbb{R}$ be the projection from $\mbb{R}^d$ to $\mbb{R}$. Let
$T_{\btheta,\# \mbb{P}}$ be the so-called push-forward measure satisfying
\[
T_{\btheta,\# \mbb{P}} (\scrA) = \mbb{P}(\{\bx\in \mbb{R}^d:\ T_{\btheta}(\bx)\in
\scrA\}), \qquad \scrA\subseteq \mbb{R}.
\]

Define $T_{\btheta,\# \mbb{Q}}$ similarly. Then, the so-called sliced
Wasserstein distance~\citep{bonneel2015sliced,liutkus2019sliced} can be defined as:
\begin{equation} \label{sliced-W}
SW_{s,2}=\int\nolimits_{\scrS^d} W_{s,2}(T_{\btheta,\# \mbb{P}}, T_{\btheta,\#
	\mbb{Q}}) \mathrm{d}\btheta,
\end{equation}
where $s(x,y)=|x-y|$, for $x,y\in \mbb{R}$.  Such an integral can be
approximated using Monte-Carlo integration, giving rise to a computational
method for computing Wasserstein distances between distributions in
$\mbb{R}^d$.

\subsection{Gaussian Distributions}
We next consider the case of two Gaussian distributions. Let $\mbb{P} \sim
\scrN(\bmu_1, \bSigma_1)$ be a $d$-dimensional Gaussian distribution with mean
$\bmu_1$ and covariance $\bSigma_1$. Similarly, let $\mbb{Q} \sim
\scrN(\bmu_2, \bSigma_2)$. Define the metric $s(\bx_1,
\bx_2)=\|\bx_1-\bx_2\|_2$. Then, the order-$2$ Wasserstein distance between
$\mbb{P}$ and $\mbb{Q}$ is given in closed-form
by~\citep{delon2020wasserstein,dowson1982frechet}: 
\[
W_{s,2} (\mbb{P},\mbb{Q}) = \|\bmu_1-\bmu_2\|^2_2 +
\tr\left(\bSigma_1+\bSigma_2 - 2\left(\bSigma_1^{1/2} \bSigma_2
\bSigma_1^{1/2}\right)^{1/2}\right).
\]

\section{The Transport Cost Function} \label{sec:wtransport}

In this monograph, we are focusing on the use of the Wasserstein metric in the
context of robust learning, specifically the DRO problem we defined in
Eq.~(\ref{dro}). As a result, the cost function $s$ used in defining the Wasserstein
metric should reflect any implicit knowledge we have on the nature of the data
$\bz=(\bx,y)$. Without loss of generality, suppose that the data have already been
{\em standardized}, specifically, for all data points $\bz_i=(\bx_i, y_i)$, $i\in \lb
N\rb$, in the training set, we have normalized every variable (coordinate) in $\bx_i$
by subtracting the empirical mean and dividing by the sample standard
deviation. Then, an element of $\bx_i$ will have a large absolute value if the
corresponding variable deviates substantially from the empirical mean. Below, we
discuss a number of different scenarios on what may be known regarding the data and
the implied appropriate corresponding cost function.

\begin{enumerate}
	\item Suppose we know that the model we are seeking is {\em sparse}, \ie, there are
	few variables, and in the extreme case one, that determine the output $y$. In this
	case, an appropriate cost function is an $\ell_\infty$ norm in the $\bz=(\bx,y)$
	space. In particular, given two data points $\bz_1=(\bx_1,y_1)$ and
	$\bz_2=(\bx_2,y_2)$, if $y_1\neq y_2$ and $\|\bx_1-\bx_2\|_{\infty}<|y_1-y_2|$, the
	distance between $\bz_1$ and $\bz_2$ is equal to $|y_1-y_2|$. If, however,
	$y_1\approx y_2$, then the distance between $\bz_1$ and $\bz_2$ is determined by
	the absolute difference in the most deviating variable, that is,
	$\|\bx_1-\bx_2\|_{\infty}$.
	
	\item Suppose now that we believe the model to be {\em dense}, implying that almost
	all variables are relevant and predictive of the output $y$. Then, an appropriate
	distance metric between two points $\bz_1=(\bx_1,y_1)$ and $\bz_2=(\bx_2,y_2)$ is
	the $\ell_2$ norm $\|\bz_1-\bz_2\|_2$, where all $\bx$ coordinates and $y$ are
	weighted equally. More generally, one can introduce weights and use a
	$\bW$-weighted $\ell_p$ norm defined as $\|\bz\|_p^{\bW}= \left((|\bz|^{p/2})'\bW
	|\bz|^{p/2}\right)^{1/p}$ with a positive definite weight matrix $\bW$.
	
	\item As one more example, suppose that the data $\bz$ are organized into a set of
	(overlapping or non-overlapping) groups according to $\bz = (\bz^1, \ldots,
	\bz^L)$. To reflect this group structure, we can define a $(q, t)$-norm, with
	$q,t\geq 1$, as:
	\begin{equation*}
	\|\bz\|_{q, t} = \left(\sum_{l=1}^L
	\bigl(\|\bz^l\|_q\bigr)^t\right)^{1/t}. 
	\end{equation*}
	Notice that the $(q, t)$-norm of $\bz$ is actually the $\ell_t$-norm of the vector
	$(\|\bz^1\|_q, \ldots, \|\bz^L\|_q)$, which represents each group vector $\bz^l$ in a
	concise way via the $\ell_q$-norm. A special case is the $(2, \infty)$-norm on the
	weighted predictor-response vector 
	\[ 
	\bz_{\bw} \triangleq \left(\frac{1}{\sqrt{p_1}}\bx^1,
	\ldots, \frac{1}{\sqrt{p_L}}\bx^L, M y\right),
	\] 
	where the weight vector is
	\[ 
	\bw=\left(\frac{1}{\sqrt{p_1}}, \ldots, \frac{1}{\sqrt{p_L}}, M\right),
	\]
	and $M$ is a positive weight assigned to the response. Specifically,
	\[
	\|\bz_{\bw}\|_{2, \infty} = \max \left\{\frac{1}{\sqrt{p_1}}\|\bx^1\|_2,
	\ldots, \frac{1}{\sqrt{p_L}}\|\bx^L\|_2, M |y|\right\},
	\]
	where different groups are scaled by the number of variables they contain. The
	$\ell_2$ norm at the individual group level reflects the intuition that all variables
	in a group are relevant, whereas the $\ell_{\infty}$ norm among groups reflects the
	intuition that there is a dominant group predictive of the response, just like the
	situation we outlined in Item~1 above. As we will see later, such a norm imposes a
	{\em group sparsity} structure.
\end{enumerate}

\subsection{Transport Cost Function via Metric Learning}
We now discuss a metric learning approach for determining the weighted transport cost function we outlined in Item 2 above, following the line of work in \cite{blanchet2019data}. The intuition is to calibrate a cost function $s(\cdot)$ which assigns a high transportation cost to a pair of data points $(\bz_1, \bz_2)$ if transporting mass between these locations significantly impacts the performance. 

Consider a classification problem where we observe $N$ (predictor, label) pairs $ \{(\bx_1, y_1), \ldots, (\bx_N, y_N)\}$, and $y_i \in \{-1, +1\}$. Suppose we use a weighted $\ell_2$ norm as the distance metric on the space of predictors: $$s_{\bW}(\bx_1, \bx_2) = \sqrt{(\bx_1-\bx_2)'\bW (\bx_1-\bx_2)},$$
where the weight matrix $\bW$ is symmetric and positive semi-definite. The goal is to 
inform the selection of $\bW$ through recognizing the pairs of samples that are similar/dissimilar to each other. In a classification setting, the labels form a natural separation plane for the observed samples. We define two sets:
\begin{equation*}
\scrM \triangleq \Big\{ (i,j): \ \text{$\bx_i$ and $\bx_j$ are close to each other and $y_i=y_j$}\Big\},
\end{equation*}
\begin{equation*}
\scrN \triangleq \Big\{ (i,j): \ \text{$\bx_i$ and $\bx_j$ are far away from each other} \Big\},
\end{equation*}
where the closeness between $\bx$ can be evaluated using an appropriate norm, e.g., the $\ell_2$ norm. $\bx_i$ and $\bx_j$ are considered to be close if one is among the $k$ nearest neighbors of the other, in the sense of the $\ell_2$ norm, with $k$ being pre-specified. We aim to automatically determine the weight $\bW$ in a data-driven fashion through minimizing the distances on the set $\scrM$ and maximizing the distances on $\scrN$, which yields the following {\em Absolute Metric Learning} formulation:
\begin{equation} \label{abs-metric-learn}
\begin{array}{rl}
\min\limits_{\bW \succcurlyeq 0} & \sum\limits_{(i,j) \in \scrM} s_{\bW}^2(\bx_i, \bx_j) \\
\text{s.t.} & \sum\limits_{(i,j) \in \scrN} s_{\bW}^2(\bx_i, \bx_j) \ge 1.
\end{array}
\end{equation}

A slightly different formulation considers the relative distance between predictors. Define a set
$$\scrT \triangleq \Big\{(i,j,k): \ \text{$s_{\bW}(\bx_i, \bx_j)$ should be smaller than $s_{\bW}(\bx_i, \bx_k)$}\Big\},$$
where $s_{\bW}(\bx_i, \bx_j)$ is considered to be smaller than $s_{\bW}(\bx_i, \bx_k)$ if any of the following holds:
\begin{enumerate}
	\item $y_i = y_j$ and $y_i \neq y_k$;
	\item $y_i = y_j = y_k$ and $\|\bx_i - \bx_j\|_2 < \|\bx_i - \bx_k\|_2$;
	\item $y_i \neq y_j$ and $y_i \neq y_k$ and $\|\bx_i - \bx_j\|_2 < \|\bx_i - \bx_k\|_2$.
\end{enumerate}
The {\em Relative Metric Learning} formulation minimizes the difference of distances on these triplets:
\begin{equation} \label{relative-metric-learn}
\begin{array}{rl}
\min\limits_{\bW \succcurlyeq 0} & \sum\limits_{(i,j,k) \in \scrT} \max \Big(s_{\bW}^2(\bx_i, \bx_j) - s_{\bW}^2(\bx_i, \bx_k) + 1, 0\Big).\\
\end{array}
\end{equation}

To hedge against potential noise in the predictors, \cite{blanchet2019data} proposed to robustify (\ref{abs-metric-learn}) and (\ref{relative-metric-learn}) using robust optimization, and learn a robust data-driven transport cost function. Specifically, for the absolute metric learning formulation, suppose the sets $\scrM$ and $\scrN$ are noisy or inaccurate at level $\alpha$, i.e., $\alpha \cdot 100\%$ of their elements are incorrectly assigned. We construct robust uncertainty sets $\scrW(\alpha)$ and $\scrV(\alpha)$ as follows:
\begin{multline*}
\scrW(\alpha) = \Big\{\boldsymbol{\eta} = (\eta_{i,j};\ {(i,j) \in \scrM}): \ 0 \le
\eta_{i,j} \le 1,\\ \sum\limits_{(i,j) \in \scrM}\eta_{i,j} \le (1-\alpha) |\scrM|
\Big\},
\end{multline*}
\[
\scrV(\alpha) = \Big\{\boldsymbol{\xi} = (\xi_{i,j};\ {(i,j) \in \scrN}): \ 0 \le
\xi_{i,j} \le 1, \ \sum\limits_{(i,j) \in \scrN}\xi_{i,j} \ge (1-\alpha) |\scrN|
\Big\}.
\]
We then formulate the robust counterpart of the Absolute Metric Learning formulation
(\ref{abs-metric-learn}) as: 
\begin{equation}  \label{robust-abs-metric}
\begin{array}{rl}
\min\limits_{\bW \succcurlyeq 0} \max\limits_{\lambda \ge 0}
\max\limits_{\substack{\boldsymbol{\eta} \in \scrW(\alpha)\\ \boldsymbol{\xi}
		\in \scrV(\alpha)}} & \bigg[ \sum\limits_{(i,j) \in \scrM}
\eta_{i,j}s_{\bW}^2(\bx_i, \bx_j) 
\\ & + \lambda\Big(1-\sum\limits_{(i,j) \in \scrN} \xi_{i,j}s_{\bW}^2(\bx_i, \bx_j)\Big)\bigg],\\
\end{array}
\end{equation}
where we robustify the Lagrangian dual problem of (\ref{abs-metric-learn}), which is
formed by bringing the constraint into the objective function via a dual variable $\lambda$, using uncertain parameters $\boldsymbol{\eta}$ and $\boldsymbol{\xi}$. Similarly, for the relative metric learning formulation, suppose the set $\scrT$ is inaccurate at level $\alpha$,
the robust counterpart of the Relative Metric Learning formulation (\ref{relative-metric-learn}) can be formulated as:
\begin{equation} \label{robust-rel-metric}
\begin{array}{rl}
\min\limits_{\bW \succcurlyeq 0} \max\limits_{\bq \in \scrQ(\alpha)}& \sum\limits_{(i,j,k) \in \scrT} q_{i,j,k}\max \Big(s_{\bW}^2(\bx_i, \bx_j) - s_{\bW}^2(\bx_i, \bx_k) + 1, 0\Big),\\
\end{array}
\end{equation}
where the uncertainty set $\scrQ(\alpha)$ is defined as:
\begin{multline*}
\scrQ(\alpha) = \Big\{\bq = (q_{i,j,k};\ {(i,j,k) \in \scrT}): \ 0 \le q_{i,j,k} \le 1, \\
\sum\limits_{(i,j,k) \in \scrT}q_{i,j,k} \le (1-\alpha) |\scrT| \Big\}.
\end{multline*}
For solving the robust optimization problems (\ref{robust-abs-metric}) and (\ref{robust-rel-metric}), we refer the reader to \cite{blanchet2019data} for a sequential iterative algorithm that alternates between optimizing over the weight matrix $\bW$ and the uncertain parameters $\boldsymbol{\eta}$, $\boldsymbol{\xi}$ (or $\bq$).

\section{Robustness of the Wasserstein Ambiguity Set} \label{sec:robust}

The ultimate goal of using DRO is to eliminate the effect of perturbed samples and
produce an estimator that is consistent with the underlying true (clean)
distribution. When the data $\bz=(\bx, y)$ are corrupted by outliers, the observed
samples are not representative enough to encode the true underlying uncertainty of
the data. Instead of equally weighting all the samples as in the empirical
distribution, we may wish to include more informative distributions that ``drive
out'' the corrupted samples. DRO realizes this through hedging the expected loss
against a family of distributions that include the true data-generating mechanism
with a high confidence. In this section, we will provide evidence on the robustness of
DRO under the Wasserstein metric, by showing that the ambiguity set defined via the
Wasserstein metric is able to retain the good (clean) distribution while excluding
the bad (outlying) one; thus, producing an estimator that is robust to outliers.

We make the assumption that the training data $(\bx, y)$ are drawn from a mixture of two
distributions, with probability $q$ from the outlying distribution $\mathbb{P}_{\text{out}}$
and with probability $1-q$ from the true (clean) distribution $\mathbb{P}$. All the
$N$ training samples $(\bx_i,y_i)$, $i\in \lb N\rb$, are independent and identical
realizations of $(\bx, y)$. Recall that $\hat{\mathbb{P}}_N$ is the discrete uniform
distribution over the $N$ samples. We claim that when $q$ is small, if the
Wasserstein ball radius $\epsilon$ is chosen judiciously, the true distribution
$\mathbb{P}$ will be included in the $\epsilon$-Wasserstein ball $\Omega$
(cf. (\ref{Omega}))
\[ 
\Omega = \{\mbb{Q}\in \scrP(\scrZ): W_{s,1}(\mathbb{Q},\ \hat{\mathbb{P}}_N) \le \epsilon\},
\]
while the outlying distribution $\mathbb{P}_{\text{out}}$ will be excluded.
Theorem~\ref{mixture} proves this claim.

\begin{thm} \label{mixture}
	Suppose we are given two probability distributions $\mathbb{P}$ and
	$\mathbb{P}_{\text{out}}$, and the mixture distribution $\mathbb{P}_{\text{mix}}$ is
	a convex combination of the two: $\mbb{P}_{\text{mix}} = q \mathbb{P}_{\text{out}} +
	(1-q)\mathbb{P}$. Then, for any cost function $s$,
	\begin{equation*}
	\frac{W_{s,1}(\mathbb{P}_{\text{out}},
		\mathbb{P}_{\text{mix}})}{W_{s,1}(\mathbb{P}, \mathbb{P}_{\text{mix}})} =
	\frac{1- q}{q}. 
	\end{equation*}
\end{thm}
\begin{proof} 
	As we indicated in Section~\ref{chapt:intro}, and for ease of notation, we will
	suppress the dependence of $W_{s,1}$ on the cost metric $s$. In addition, without
	loss of generality, we will assume that the probability distributions $\mathbb{P}$,
	$\mathbb{P}_{\text{out}}$, $\mathbb{P}_{\text{mix}}$, and any joint distributions
	have densities. From the definition of the Wasserstein distance,
	$W_1(\mathbb{P}_{\text{out}}, \mathbb{P}_{\text{mix}})$ is the optimal value of the
	following optimization problem:
	\begin{equation} \label{wassQ1}
	\begin{aligned}
	\min\limits_{\pi \in \scrP(\scrZ \times \scrZ)} & \quad \int\nolimits_{\scrZ \times \scrZ} s(\bz_1, \bz_2) \ \mathrm{d}\pi \bigl(\bz_1, \bz_2\bigr) \\
	\text{s.t.} & \quad \int\nolimits_{\scrZ}\pi \bigl(\bz_1, \bz_2\bigr) \mathrm{d}\bz_2 = \mathbb{P}_{\text{out}}(\bz_1),\; \forall \bz_1 \in \scrZ, \\
	& \quad \int\nolimits_{\scrZ}\pi \bigl(\bz_1, \bz_2\bigr) \mathrm{d}\bz_1 = q \mathbb{P}_{\text{out}}(\bz_2) + (1-q)\mathbb{P}(\bz_2),\; \forall \bz_2 \in \scrZ.
	\end{aligned}
	\end{equation}
	Similarly, $W_1 (\mathbb{P}, \mathbb{P}_{\text{mix}})$ is the optimal value of the following optimization problem:
	\begin{equation} \label{wassQ2}
	\begin{aligned}
	\min\limits_{\pi \in \scrP(\scrZ \times \scrZ)} & \quad \int\nolimits_{\scrZ \times \scrZ} s(\bz_1, \bz_2) \ \mathrm{d}\pi \bigl(\bz_1, \bz_2\bigr) \\
	\text{s.t.} & \quad \int\nolimits_{\scrZ}\pi \bigl(\bz_1, \bz_2\bigr) \mathrm{d}\bz_2 = \mathbb{P}(\bz_1),\; \forall \bz_1 \in \scrZ, \\
	& \quad \int\nolimits_{\scrZ}\pi \bigl(\bz_1, \bz_2\bigr) \mathrm{d}\bz_1 = q \mathbb{P}_{\text{out}}(\bz_2) + (1-q)\mathbb{P}(\bz_2),\; \forall \bz_2 \in \scrZ.
	\end{aligned}
	\end{equation}
	
	We propose a decomposition strategy. For Problem (\ref{wassQ1}), decompose the joint
	distribution $\pi$ as $\pi = (1-q)\pi_1 + q \pi_2$, where $\pi_1$ and $\pi_2$ are two joint
	distributions of $\bz_1$ and $\bz_2$. The first set of constraints in Problem
	(\ref{wassQ1}) can be equivalently expressed as:
	\begin{multline*} 
	(1-q)\int\nolimits_{\scrZ}\pi_1\bigl(\bz_1, \bz_2\bigr)\mathrm{d}\bz_2 + q
	\int\nolimits_{\scrZ}\pi_2\bigl(\bz_1, \bz_2\bigr) \mathrm{d}\bz_2 \\
	= (1-q)\mathbb{P}_{\text{out}}(\bz_1) + q \mathbb{P}_{\text{out}}(\bz_1),\; \forall \bz_1 \in \scrZ,
	\end{multline*}
	which is satisfied if
	$$\int\nolimits_{\scrZ}\pi_1\bigl(\bz_1, \bz_2\bigr)\mathrm{d}\bz_2 = \mathbb{P}_{\text{out}}(\bz_1), \quad \int\nolimits_{\scrZ}\pi_2\bigl(\bz_1, \bz_2\bigr)\mathrm{d}\bz_2 = \mathbb{P}_{\text{out}}(\bz_1),\; \forall \bz_1 \in \scrZ.$$	
	The second set of constraints can be expressed as:
	\begin{multline*}
	(1-q)\int\nolimits_{\scrZ}\pi_1\bigl(\bz_1, \bz_2\bigr)\mathrm{d}\bz_1 + q \int\nolimits_{\scrZ}\pi_2\bigl(\bz_1, \bz_2\bigr)  \mathrm{d}\bz_1 \\= q \mathbb{P}_{\text{out}}(\bz_2) + (1-q)\mathbb{P}(\bz_2),\; \forall \bz_2 \in \scrZ,
	\end{multline*}
	which is satisfied if
	$$\int\nolimits_{\scrZ}\pi_1\bigl(\bz_1, \bz_2\bigr)\mathrm{d}\bz_1 = \mathbb{P}(\bz_2), \quad \int\nolimits_{\scrZ}\pi_2\bigl(\bz_1, \bz_2\bigr) \mathrm{d}\bz_1 = \mathbb{P}_{\text{out}}(\bz_2),\; \forall \bz_2 \in \scrZ.$$
	The objective function can be decomposed as:
	\begin{equation*}
	\begin{aligned}
	\int\nolimits_{\scrZ \times \scrZ} s(\bz_1, \bz_2) \ \mathrm{d}\pi \bigl(\bz_1, \bz_2\bigr)  = & \ (1-q)\int\nolimits_{\scrZ \times \scrZ} s(\bz_1, \bz_2) \ \mathrm{d}\pi_1\bigl(\bz_1, \bz_2\bigr) \\
	& + q \int\nolimits_{\scrZ \times \scrZ} s(\bz_1, \bz_2)\mathrm{d}\pi_2 \bigl(\bz_1, \bz_2\bigr).
	\end{aligned}
	\end{equation*}
	Therefore, Problem (\ref{wassQ1}) can be decomposed into the following two subproblems.
	\[ \text{Subproblem 1:} \quad \begin{array}{rl}
	\min\limits_{\pi_1 \in \scrP(\scrZ \times \scrZ)} &  \int_{\scrZ \times \scrZ} s(\bz_1, \bz_2) \ \mathrm{d}\pi_1\bigl(\bz_1, \bz_2\bigr) \\
	\text{s.t.} &  \int_{\scrZ}\pi_1\bigl(\bz_1, \bz_2\bigr)\mathrm{d}\bz_2 = \mathbb{P}_{\text{out}}(\bz_1),\; \forall \bz_1 \in \scrZ,\\
	& \int_{\scrZ}\pi_1\bigl(\bz_1, \bz_2\bigr)\mathrm{d}\bz_1 = \mathbb{P}(\bz_2),\; \forall \bz_2 \in \scrZ.
	\end{array}
	\]
	\[ \text{Subproblem 2:} \quad \begin{array}{rl}
	\min\limits_{\pi_2 \in \scrP(\scrZ \times \scrZ)} &  \int_{\scrZ \times \scrZ} s(\bz_1, \bz_2) \ \mathrm{d}\pi_2\bigl(\bz_1, \bz_2\bigr) \\
	\text{s.t.} &  \int_{\scrZ}\pi_2\bigl(\bz_1, \bz_2\bigr)\mathrm{d}\bz_2 =
	\mathbb{P}_{\text{out}}(\bz_1),\;  \forall \bz_1 \in \scrZ,\\
	& \int_{\scrZ}\pi_2\bigl(\bz_1, \bz_2\bigr)\mathrm{d}\bz_1 = \mathbb{P}_{\text{out}}(\bz_2),\; \forall \bz_2 \in \scrZ.
	\end{array}
	\]		
	Assume that the optimal solutions to the two subproblems are $\pi_1^*$ and
	$\pi_2^*$, respectively. We know $\pi_0 = (1-q)\pi_1^* + q \pi_2^*$ is a
	feasible solution to Problem (\ref{wassQ1}). Therefore,
	\begin{equation} \label{out-mix}
	\begin{aligned}
	W_1 (\mathbb{P}_{\text{out}}, \mathbb{P}_{\text{mix}}) & \le \int\nolimits_{\scrZ \times \scrZ} s(\bz_1, \bz_2) \ \mathrm{d}\pi_0 \bigl(\bz_1, \bz_2\bigr) \\
	& = (1-q)W_1 (\mathbb{P}_{\text{out}}, \mathbb{P}) + q W_1 (\mathbb{P}_{\text{out}}, \mathbb{P}_{\text{out}})\\
	& = (1-q)W_1 (\mathbb{P}_{\text{out}}, \mathbb{P}).
	\end{aligned}
	\end{equation}
	Similarly, 
	\begin{equation} \label{true-mix}
	W_1 (\mathbb{P}, \mathbb{P}_{\text{mix}}) \le q W_1 (\mathbb{P}_{\text{out}}, \mathbb{P}).
	\end{equation}
	(\ref{out-mix}) and (\ref{true-mix}) imply that
	\begin{equation*}
	W_1 (\mathbb{P}_{\text{out}}, \mathbb{P}_{\text{mix}}) + W_1 (\mathbb{P}, \mathbb{P}_{\text{mix}}) \le W_1 (\mathbb{P}_{\text{out}}, \mathbb{P}).
	\end{equation*}
	On the other hand, using the triangle inequality for the Wasserstein metric,
	we have,
	$$W_1 (\mathbb{P}_{\text{out}}, \mathbb{P}_{\text{mix}}) + W_1 (\mathbb{P}, \mathbb{P}_{\text{mix}}) \ge W_1 (\mathbb{P}_{\text{out}}, \mathbb{P}).$$
	We thus conclude that 
	\begin{equation} \label{equal}
	W_1 (\mathbb{P}_{\text{out}}, \mathbb{P}_{\text{mix}}) + W_1 (\mathbb{P}, \mathbb{P}_{\text{mix}}) = W_1 (\mathbb{P}_{\text{out}}, \mathbb{P}).
	\end{equation}
	To achieve the equality in (\ref{equal}), (\ref{out-mix}) and (\ref{true-mix}) must be equalities, i.e.,
	$$W_1 (\mathbb{P}_{\text{out}}, \mathbb{P}_{\text{mix}}) = (1-q)W_1 (\mathbb{P}_{\text{out}}, \mathbb{P}),$$
	and, 
	\begin{equation} \label{wass-alt}
	W_1 (\mathbb{P}, \mathbb{P}_{\text{mix}}) = q W_1 (\mathbb{P}_{\text{out}},
	\mathbb{P}). 
	\end{equation}
	Thus,
	\begin{equation*}
	\frac{W_1 (\mathbb{P}_{\text{out}}, \mathbb{P}_{\text{mix}})}{W_1 (\mathbb{P}, \mathbb{P}_{\text{mix}})} = \frac{(1-q)W_1 (\mathbb{P}_{\text{out}}, \mathbb{P})}{q W_1 (\mathbb{P}_{\text{out}}, \mathbb{P})} 
	=  \frac{1-q}{q}.
	\end{equation*} 
\end{proof} 

\section{Setting the Radius of the Wasserstein Ball} \label{sec:wass-radius}

\begin{figure}[ht]
	\centering
	\resizebox{0.35\linewidth}{!}{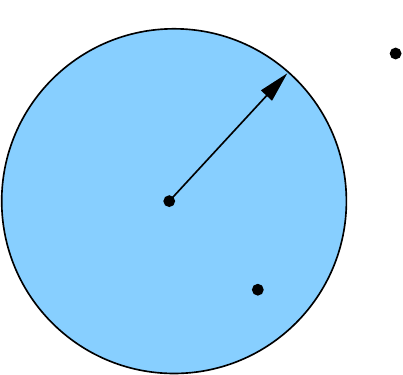} $\qquad$ $\qquad$
	\resizebox{0.35\linewidth}{!}{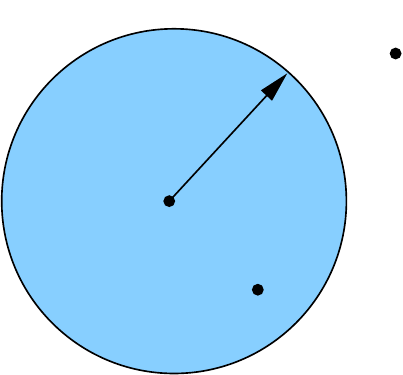} 
	\caption{Left: Training with a contaminated training set drawn from
		$\mathbb{P}_{\text{mix}}$. Right: Training with a pure training set drawn from
		$\mathbb{P}$.}
	\label{fig:wass-radius}
\end{figure}
Theorem~\ref{mixture} provides some guidance on setting the radius $\epsilon$ of the
Wasserstein ball $\Omega$. Figure~\ref{fig:wass-radius} (Left) provides a graphical
interpretation. As seen in the figure, the ball $\Omega$ is centered at
$\mathbb{P}_{\text{mix}}$ because we assume that the training set is drawn from this
distribution. According to Theorem~\ref{mixture}, when $q<0.5$ we have
$W_1(\mathbb{P}, \mathbb{P}_{\text{mix}}) \le \epsilon< W_1(\mathbb{P}_{\text{out}},
\mathbb{P}_{\text{mix}})$. Thus, for a large enough sample size (so that
$\hat{\mbb{P}}_N$ is a good approximation of $\mathbb{P}_{\text{mix}}$), the set
$\Omega$ will include the true distribution and exclude the outlying one, which
provides protection against these outliers.

To provide numerical evidence, consider a simple example where $\mathbb{P}$ is a
discrete distribution that assigns equal probability to $10$ data points equally
spaced between $0.1$ and $1$, and $\mathbb{P}_{\text{out}}$ assigns probability $0.5$ to two
data points $1$ and $2$. We generate $100$ samples and plot the order-1 Wasserstein distances
from $\hat{\mathbb{P}}_N$ for both $\mathbb{P}$ and $\mathbb{P}_{\text{out}}$, under the distance metric $s(z_1, z_2) = |z_1 - z_2|$. 
\begin{figure}[ht]
	\centering
	\includegraphics[height = 2.3in]{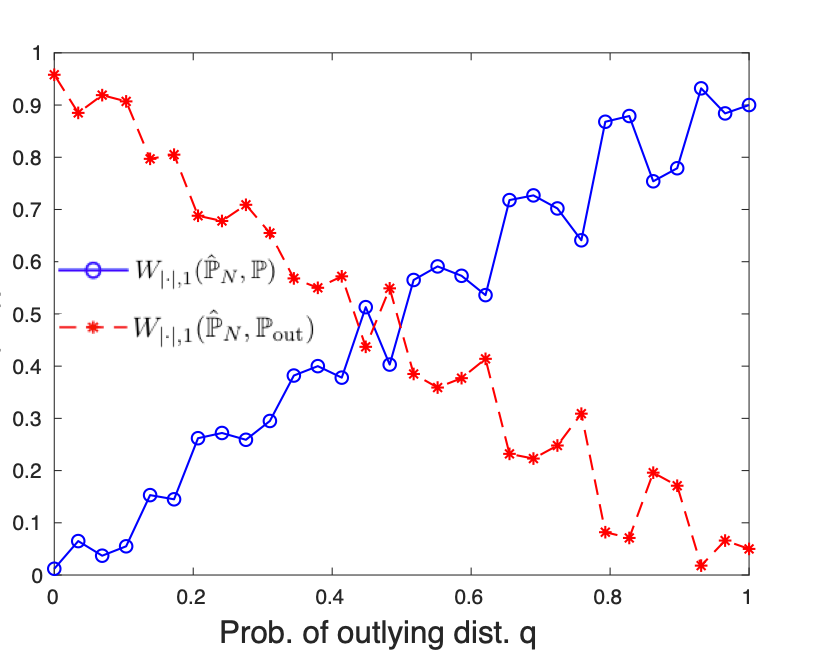}
	\caption{The order-$1$ Wasserstein distances from the empirical distribution.}
	\label{wassdis}
\end{figure}
From Figure~\ref{wassdis} we observe that for $q$ below $0.5$, the true distribution
$\mathbb{P}$ is closer to $\hat{\mathbb{P}}_N$ whereas the outlying distribution
$\mathbb{P}_{\text{out}}$ is further away. If the radius $\epsilon$ is chosen between the
red ($\ast-$) and blue ($\circ-$) lines, the Wasserstein ball that we are hedging
against will exclude the outlying distribution and the resulting estimator will be
robust to the perturbations. Moreover, as $q$ becomes smaller, the gap between the
red and blue lines becomes larger. One implication from this observation is that as
the data becomes purer, the radius of the Wasserstein ball tends to be smaller, and
the confidence in the observed samples is higher. For large $q$ values, the DRO
formulation seems to fail. However, as outliers are defined to be the data points
that do not conform to the majority of data, if $q>0.5$ then
$\mathbb{P}_{\text{out}}$ becomes the distribution of the majority and data generated
from $\mathbb{P}$ can be treated as outliers. Thus, without loss of generality, 
we can safely treat
$\mathbb{P}_{\text{out}}$ as the distribution of the minority and assume $q$ is always below $0.5$.

An alternative use of the DRO learning approach can be seen in
Figure~\ref{fig:wass-radius} (Right). Here, we assume that the training set is pure,
thus, given enough samples, the empirical distribution on which the ball $\Omega$ is
centered is close to $\mbb{P}$. Consider applying the model to a test set which is
contaminated with outliers. Notice from the proof of Theorem~\ref{mixture} that $W_1
(\mathbb{P}, \mathbb{P}_{\text{mix}}) = q W_1 (\mathbb{P}_{\text{out}}, \mathbb{P})$
(cf. Eq.~(\ref{wass-alt})). This implies that the smaller $q$ is, and for a properly
selected $\epsilon$, the distribution from which the test set is drawn
($\mathbb{P}_{\text{mix}}$) is within the ball $\Omega$ and the model has the
potential to generalize well in the test set, tolerating some outliers. In contrast,
the outlying distribution $\mathbb{P}_{\text{out}}$ lies outside the set $\Omega$, which
suggests that the model does not ``adjust'' to samples generated from
$\mathbb{P}_{\text{out}}$. According to this reasoning, and based on Eq.~(\ref{wass-alt}),
$\epsilon$ should be set so that $q W_1 (\mathbb{P}_{\text{out}}, \mathbb{P}) <
\epsilon < W_1 (\mathbb{P}_{\text{out}}, \mathbb{P})$.

The above discussions provide some insights on the optimal selection of the radius, but could be hard to implement due to the unknown $\mbb{P}$ and $\mbb{P}_{\text{out}}$. In practice cross-validation is usually adopted, but could be computationally expensive. In the next two subsections we discuss two practical radius selection approaches that produce the smallest Wasserstein ball which contains the true distribution with high confidence.

\subsection{Measure Concentration} \label{sec:mea-con}
In this subsection we study an optimal radius selection method that originates from the measure concentration theory. As will be seen in Section~\ref{sec:finite-sample-perf}, it leads to an asymptotic consistent DRO estimator that generalizes well out-of-sample.

Suppose $\bz_i, \ i \in \lb N \rb$, are $N$ realizations of $\bz$ which follows an unknown distribution $\mbb{P}^*$. One of the prerequisites for ensuring a good generalization performance of Wasserstein DRO requires that the ambiguity set $\Omega_{\epsilon}(\hat{\mbb{P}}_N)$ includes the true data distribution $\mbb{P}^*$. This implies that the radius $\epsilon$ should be chosen so that
\begin{equation} \label{p-pn-dist}
W_{s, 1} (\mbb{P}^*, \hat{\mbb{P}}_N) \leq \epsilon.
\end{equation}
A measure concentration result developed in \cite{Four14}, which characterizes the rate at which the empirical distribution $\hat{\mbb{P}}_N$ converges to the true distribution $\mbb{P}^*$ in the sense of the Wasserstein metric, can be used as a guidance on the optimal selection of the radius for the Wasserstein ambiguity set. In the following discussion we assume $s$ is a norm, and the true data distribution $\mbb{P}^*$ satisfies the light tail condition stated in Assumption~\ref{light-tail-dist}.

\begin{ass}[Light-tailed distribution] \label{light-tail-dist}
	There exists an exponent $a>1$ such that
	\begin{equation} \label{light-tail-eq}
	A \triangleq \mbb{E}^{\mbb{P}^*}\big[\exp{(\|\bz\|^a)}\big] = \int\nolimits_{\scrZ} \exp{(\|\bz\|^a)} \mathrm{d}\mbb{P}^*(\bz) < \infty.
	\end{equation}
\end{ass}

\begin{thm}[Measure concentration; \cite{Four14}, Theorem 2] \label{measure-con}
	Suppose the Wasserstein metric is induced by some norm $\|\cdot\|$, i.e., $s(\bz_1, \bz_2) = \|\bz_1 - \bz_2\|$. Under Assumption~\ref{light-tail-dist}, we have
	\begin{equation} \label{measure-eqn}
	\mbb{P}^N \Big( W_{s, 1} (\mbb{P}^*, \hat{\mbb{P}}_N) \ge \epsilon\Big) \le 
	\begin{cases}
	c_1 \exp{\big( -c_2 N \epsilon^{\max(d, 2)}\big)}, & \text{if $\epsilon \le 1$}, \\
	c_1 \exp{\big( -c_2 N \epsilon^a\big)}, & \text{if $\epsilon > 1$},
	\end{cases}
	\end{equation}
	for all $N \ge 1, d \neq 2$, and $\epsilon>0$, where $N$ is the size of the observed training set, $d$ is the dimension of $\bz$, $a$ is defined in (\ref{light-tail-eq}), and $c_1, c_2$ are positive constants that only depend on $a, A$, and $d$.
\end{thm}

From Theorem~\ref{measure-con} we can derive the smallest possible $\epsilon$ so that the true distribution is contained in the Wasserstein ambiguity set with high confidence. Given some prescribed $\alpha \in (0,1)$, it is desired that 
\begin{equation*}
\mbb{P}^N \Big( W_{s, 1} (\mbb{P}^*, \hat{\mbb{P}}_N) \le \epsilon\Big) \ge 1-\alpha.
\end{equation*}
Equating the RHS of (\ref{measure-eqn}) to $\alpha$ and solving for $\epsilon$ yields 
\begin{equation} \label{radius-FG}
\epsilon_N(\alpha) = 
\begin{cases}
\Big( \frac{\log(c_1 \alpha^{-1})}{c_2 N}\Big)^{1/\max(d,2)}, & \text{if $N \ge \frac{\log(c_1 \alpha^{-1})}{c_2}$}, \\
\Big( \frac{\log(c_1 \alpha^{-1})}{c_2 N}\Big)^{1/a}, & \text{if $N < \frac{\log(c_1 \alpha^{-1})}{c_2}$}.
\end{cases}
\end{equation}
Notice that Eq.~(\ref{radius-FG}) depends on the unknown constants $c_1$ and $c_2$, and does not make use of the available training data, which could potentially result in a conservative estimation of the radius and is not of practical use \citep{esfahani2018data}. 
By recognizing these issues, some researchers have proposed to choose the radius without relying on exogenous constants, see \cite{ji2018data, zhao2015data2}.

By using an extension of Sanov's theorem which identifies the rate function, in the form of the KL divergence, for large deviations of the empirical measure from the true measure \citep{sanov1958probability}, \cite{ji2018data} derived a closed-form expression for computing the size of the Wasserstein ambiguity set, when the support of $\bz$ is finite and bounded, and the true distribution is discrete.
The reason for restricting to a discrete true distribution lies in that the convergence rate of the empirical measure (in the sense of the Wasserstein distance) is characterized by the KL divergence \citep{wang2010sanov}, which diverges when the true distribution $\mbb{P}^*$ is continuous, and the empirical distribution $\hat{\mbb{P}}_N$ is discrete. 

\begin{thm} [\cite{ji2018data}, Theorem 2]
	Suppose the random vector $\bz$ is supported on a finite Polish space $(\scrZ, s)$, and is distributed according to a discrete true distribution $\mbb{P}^*$. Assume there exists some $\bz_0 \in \scrZ$ such that the following condition holds:
	\begin{equation} \label{jl-cond}
	\log \int\nolimits_{\scrZ} \exp{(as(\bz, \bz_0))} \mathrm{d}\mbb{P}^*(\bz) < \infty, \quad \forall a>0.
	\end{equation}
	Define $B$ as the diameter of the $d$-dimensional compact set $\scrZ$:
	\begin{equation} \label{diam-B}
	B \triangleq \sup \{ s(\bz_1, \bz_2): \bz_1, \bz_2 \in \scrZ\}.
	\end{equation}
	Construct an empirical distribution $\hat{\mbb{P}}_{N}$ based on $N$ i.i.d. samples of $\bz$. A lower bound on the probability that the Wasserstein distance between the empirical distribution $\hat{\mbb{P}}_{N}$ and the true distribution $\mbb{P}^*$ does not exceed $\epsilon$ is given by:
	{\small \begin{equation*}
		\mbb{P}^N \Big( W_{s, 1} (\mbb{P}^*, \hat{\mbb{P}}_N) \le \epsilon\Big) \ge 1 - \exp{\bigg( -N \bigg( \frac{\sqrt{4\epsilon (4B+3) + (4B+3)^2}}{4B+3} - 1\bigg)^2\bigg)}.
		\end{equation*} }
	Furthermore, if 
	\begin{equation*}
	\epsilon \ge \Big(B+\frac{3}{4}\Big) \Big(-\frac{1}{N}\log(\alpha) + 2 \sqrt{-\frac{1}{N} \log(\alpha)}\Big),
	\end{equation*}
	then
	$$\mbb{P}^N \Big( W_{s, 1} (\mbb{P}^*, \hat{\mbb{P}}_N) \le \epsilon\Big) \ge 1-\alpha.$$
\end{thm}

\cite{zhao2015data2} derived a more general formula for computing the Wasserstein set radius, without imposing the exponential integrability condition (\ref{jl-cond}), resulting in a slower convergence rate for the radius, $\epsilon = O(\sqrt{1/N})$.

\begin{thm} [\cite{zhao2015data2}, Proposition 3]
	Assume the support $\scrZ$ is bounded and finite, and the true distribution
	$\mbb{P}^*$ is discrete. 
	%Define $B$ as the diameter of the $d$-dimensional compact set $\scrZ$:
	%    $$B \triangleq \sup \{ s(\bz_1, \bz_2): \bz_1, \bz_2 \in \scrZ\}.$$
	We have,
	\begin{equation*}
	\mbb{P}^N \Big( W_{s, 1} (\mbb{P}^*, \hat{\mbb{P}}_N) \le \epsilon\Big) \ge 1
	- \exp \Big(-\frac{N \epsilon^2}{2B^2} \Big), 
	\end{equation*}
	where $B$ is as in (\ref{diam-B}).
	Moreover, if we set 
	\begin{equation*}
	\epsilon \ge B \sqrt{\frac{2 \log (1/\alpha)}{N}},
	\end{equation*}
	then
	$$\mbb{P}^N \Big( W_{s, 1} (\mbb{P}^*, \hat{\mbb{P}}_N) \le \epsilon\Big) \ge 1-\alpha.$$
\end{thm}

\subsection{Robust Wasserstein Profile Inference}
In this subsection we introduce a different approach proposed by \cite{blanchet2019robust} for optimally selecting the size of the Wasserstein ambiguity set. This method combines the information of the structure of the ambiguity set and the loss function that is being minimized. Unlike Section~\ref{sec:mea-con} where large deviation theory is adopted to describe the closeness between the empirical measure and the true measure, here the true measure is characterized indirectly via the first-order optimality condition of the loss function.

Recall the Wasserstein DRO formulation:
\begin{equation} \label{dro-rwpi}
\inf\limits_{\bbeta }\sup\limits_{\mbb{Q}\in \Omega}
\mbb{E}^{\mbb{Q}}\big[ h_{\bbeta}(\bz)\big],
\end{equation}
where the ambiguity set is defined as:
\begin{equation*}
\Omega = \Omega_{\epsilon}^{s,t}(\hat{\mathbb{P}}_{N}) \triangleq \{\mbb{Q}\in \scrP(\scrZ): W_{s,t}(\mathbb{Q},\ \hat{\mathbb{P}}_{N}) \le \epsilon\}.
\end{equation*}
We will suppress the dependence of $\Omega$ on $s, t, \epsilon, \hat{\mathbb{P}}_{N}$ for ease of notation. For every $\mbb{Q} \in \Omega$, there is an optimal choice $\bbeta = \bbeta(\mbb{Q})$ which minimizes the risk $\mbb{E}^{\mbb{Q}}\big[ h_{\bbeta}(\bz) \big]$, i.e., 
$$\bbeta(\mbb{Q}) = \arg \min_{\bbeta} \mbb{E}^{\mbb{Q}}\big[ h_{\bbeta}(\bz) \big].$$
We define $\scrS_{\bbeta}(\Omega) \triangleq \{\bbeta(\mbb{Q}): \mbb{Q} \in \Omega \}$ to be the set of plausible selections of the parameter $\bbeta$. If the true measure $\mbb{P}^* \in \Omega$, then $\bbeta^* = \bbeta(\mbb{P}^*) \in \scrS_{\bbeta}(\Omega)$.

% \begin{equation*} 
% \inf\limits_{\bbeta \in \mbb{R}^d}\sup\limits_{\mbb{Q}\in \Omega}
% \mbb{E}^{\mbb{Q}}\big[ h_{\bbeta}(\bx, y)\big] = \inf\limits_{\bbeta \in \scrB(\Omega)}\sup\limits_{\mbb{Q}\in \Omega}
% \mbb{E}^{\mbb{Q}}\big[ h_{\bbeta}(\bx, y)\big],
% \end{equation*}

We say that $\bbeta^*$ is \textbf{plausible} with $(1-\alpha)$ confidence if $\bbeta^* \in \scrS_{\bbeta}(\Omega)$ with probability at least $1-\alpha$. We want to choose $\epsilon>0$ as small as possible so that the underlying true parameter $\bbeta^*$ is plausible with $(1-\alpha)$ confidence.

%We will illustrate the selection of $\epsilon$ under the squared loss function, i.e., $h_{\bbeta}(\bx, y) = (y-\bbeta'\bx)^2$. 
For any given $\mbb{Q}$, the optimal solution $\bbeta(\mbb{Q})$ is characterized by the following first-order condition:
\begin{equation} \label{opt-eqn}
\mbb{E}^{\mbb{Q}}[\nabla_{\bbeta} h_{\bbeta(\mbb{Q})}(\bz)] = \mathbf{0},
\end{equation}
where $\nabla_{\bbeta}h_{\bbeta(\mbb{Q})}(\bz)$ is the partial derivative of $h_{\bbeta}(\bz)$ w.r.t. $\bbeta$ evaluated at $\bbeta = \bbeta(\mbb{Q})$. Define the {\em Robust Wasserstein Profile (RWP)} function associated with the estimation equation (\ref{opt-eqn}) as:
\begin{equation*}
R(\bbeta) = \inf\nolimits_{\mbb{Q}}\big\{\big(W_{s,t}(\mbb{Q}, \hat{\mathbb{P}}_{N})\big)^t: \mbb{E}^{\mbb{Q}}[\nabla_{\bbeta} h_{\bbeta}(\bz)] = \mathbf{0}\big\}.
\end{equation*}
$R(\bbeta)$ evaluates the minimal distance to the empirical distribution, for all distributions such that $\bbeta$ is the minimizer of the expected loss. Note that $R(\bbeta)$ is a random quantity due to the randomness in the observed samples, which is reflected in $\hat{\mathbb{P}}_{N}$. For $\bbeta^* \in \scrS_{\bbeta}(\Omega)$ to hold, it is required that there exists at least one 
$$\mbb{Q} \in \{\mbb{Q}: \mbb{E}^{\mbb{Q}}[\nabla_{\bbeta} h_{\bbeta^*}(\bz)] = \mathbf{0}\},$$
such that $\mbb{Q} \in \Omega$. This equivalently translates into the the condition that 
$$R(\bbeta^*) \le \epsilon^t.$$
Therefore, $\bbeta^*$ is plausible with $(1-\alpha)$ confidence if and only if
\begin{equation*}
\mbb{P}(R(\bbeta^*) \le \epsilon^t) \ge 1-\alpha.
\end{equation*}
The optimal choice of $\epsilon$ is thus $\chi_{1-\alpha}^{1/t}$, where $\chi_{1-\alpha}$ is the $1-\alpha$ quantile of $R(\bbeta^*)$. Moreover, 
\begin{equation*}
\mbb{P}\big(\bbeta^* \in \scrS_{\bbeta}(\chi_{1-\alpha})\big) = \mbb{P}\big(R(\bbeta^*) \le \chi_{1-\alpha}\big) = 1-\alpha,
\end{equation*}
where $\scrS_{\bbeta}(\chi_{1-\alpha}) \triangleq \{ \bbeta: R(\bbeta) \le \chi_{1-\alpha}\}$. Therefore, $\scrS_{\bbeta}(\chi_{1-\alpha})$ is a $(1-\alpha)$ confidence region for $\bbeta^*$.

The problem of optimal radius selection now reduces to finding the quantile of $R(\bbeta^*)$. Since $\bbeta^*$ is unknown, we need to come up with a way of estimating the distribution of the RWP function $R(\bbeta^*)$. \cite{blanchet2019robust} developed an asymptotic analysis of the RWP function, and established that as $N \rightarrow \infty$,
\begin{equation*}
N^{t/2}R(\bbeta^*) \xrightarrow[]{\ \text{d} \ } \bar{R}(t),
\end{equation*}
for a suitably defined random variable $\bar{R}(t)$, where $\xrightarrow[]{ \ \text{d} \ }$ means convergence in distribution. We first state a number of assumptions that are needed to establish this convergence in distribution. 
\begin{ass} \label{rwpi-cost-metric}
	The cost function is the $\ell_r$ norm: $s(\bz_1, \bz_2) = \|\bz_1 - \bz_2\|_r$, where $r \ge 1$. Let $s$ be such that $1/r+1/s=1$.
\end{ass}
\begin{ass} \label{beta-star}
	The true parameter $\bbeta^*$ satisfies $\mbb{E}^{\mbb{P}^{*}}[\nabla_{\bbeta} h_{\bbeta^*}(\bz)] = \mathbf{0}$, and $\mbb{E}^{\mbb{P}^{*}} \|\nabla_{\bbeta} h_{\bbeta^*}(\bz)\|_2^2 < \infty$, where $\mbb{P}^{*}$ is the underlying true distribution of $\bz$.
\end{ass}
\begin{ass} \label{diff-z}
	The function $\nabla_{\bbeta} h_{\bbeta^*}(\bz)$ is continuously differentiable
	w.r.t. $\bz$ with gradient $\nabla_{\bbeta,\bz}h_{\bbeta^*}(\bz)$. 
	%\triangleq \frac{\partial^2 h_{\bbeta}(\bz)}{\partial \bbeta \partial \bz} \Big |_{\bbeta = \bbeta^*}$.
\end{ass}
\begin{ass} \label{psd-diff}
	\begin{equation*}
	\mbb{E}^{\mbb{P}^{*}}[\nabla_{\bbeta,\bz}h_{\bbeta^*}(\bz) \nabla_{\bbeta,\bz}h_{\bbeta^*}(\bz)'] \succ 0.
	\end{equation*}
\end{ass}
\begin{ass} \label{rwpi-growth}
	There exists $\kappa>0$ such that for $\|\bz\|_r \ge 1$,
	\begin{equation*}
	\|\nabla_{\bbeta,\bz}h_{\bbeta^*}(\bz)\|_s \le \kappa \|\bz\|_r^{t-1},
	\end{equation*}
	where the LHS denotes the induced $\ell_s$ norm of the matrix $\nabla_{\bbeta,\bz}h_{\bbeta^*}(\bz)$.
\end{ass}
\begin{ass} \label{rwpi-lipschitz}
	There exists a function $c: \mbb{R}^d \ra [0, \infty)$ such that,
	\begin{equation*}
	\|\nabla_{\bbeta,\bz}h_{\bbeta^*}(\bz+ \bdelta) - \nabla_{\bbeta,\bz}h_{\bbeta^*}(\bz)\|_s \le c(\bz) \|\bdelta\|_r,
	\end{equation*}
	for $\|\bdelta\|_r \le 1$, $\mbb{E}^{\mbb{P}^*}[c(\bz)^a]<\infty$, and $a \le \max(2, t/(t-1))$.
\end{ass}

\begin{thm} [\cite{blanchet2019robust}, Theorem 3]\label{rwpi-limit}
	When $t>1$, under Assumptions \ref{rwpi-cost-metric}, \ref{beta-star}, \ref{diff-z}, \ref{psd-diff}, \ref{rwpi-growth}, \ref{rwpi-lipschitz}, as $N \rightarrow \infty$,
	\begin{equation*}
	N^{t/2}R(\bbeta^*) \xrightarrow[]{ \ \text{d} \ } \bar{R}(t),
	\end{equation*}
	where
	\begin{equation*}
	\bar{R}(t) = \max_{\bzeta} \Big\{ t \bzeta'\br - (t-1) \mbb{E}^{\mbb{P}^*} \|\bzeta'\nabla_{\bbeta,\bz}h_{\bbeta^*}(\bz)\|_s^{t/(t-1)}\Big\}.
	\end{equation*}
	When $t=1$, suppose that $\bz$ has a positive density almost everywhere w.r.t. the Lebesgue measure. Then, under Assumptions \ref{rwpi-cost-metric}, \ref{beta-star}, \ref{diff-z}, \ref{psd-diff},
	\begin{equation*}
	N^{1/2}R(\bbeta^*) \xrightarrow[]{ \ \text{d} \ } \bar{R}(1),
	\end{equation*}
	where
	\begin{equation*}
	\begin{aligned}
	\bar{R}(1) = \ & \max_{\bzeta} \ \bzeta'\br \\
	& \ \text{s.t.} \quad \mbb{P}^* \big( \|\bzeta'\nabla_{\bbeta,\bz}h_{\bbeta^*}(\bz)\|_s > 1\big) = 0,
	\end{aligned}
	\end{equation*}
	with $\br \sim \scrN(\mathbf{0}, \mbb{E}[\nabla_{\bbeta}h_{\bbeta^*}(\bz) \nabla_{\bbeta}h_{\bbeta^*}(\bz)'])$.
\end{thm}

The proof of Theorem \ref{rwpi-limit} uses the dual representation of the RWP
function, and proceeds by showing that $\bar{R}(t)$ is both an asymptotic stochastic
upper bound and a lower bound of $N^{t/2}R(\bbeta^*)$ (refer to \cite{blanchet2019robust} for details). Notice that the limiting random variable $\bar{R}(t)$ still depends on the unknown parameter $\bbeta^*$ and the unobservable true distribution $\mbb{P}^*$. When using Theorem \ref{rwpi-limit} in practice, some further relaxations for $\bar{R}(t)$ are needed to get rid of the unknown parameters. We next illustrate this idea using an example of distributionally robust logistic regression.

\subsubsection{Example: Optimal Radius Selection for Wasserstein Distributionally Robust Logistic Regression Using RWP Inference} 

In this example we show how to use the RWP function and its limiting variable $\bar{R}(t)$ to select the optimal radius for the Wasserstein ambiguity set in a distributionally robust logistic regression problem.

Let $\bx \in \mbb{R}^{d}$ denote the predictor and $y \in \{-1, +1\}$ the associated binary label to be predicted. In logistic regression, the conditional distribution of $y$ given $\bx$ is modeled as
\begin{equation*}
\mbb{P}(y|\bx) = \big(1+\exp(-y \bbeta'\bx)\big)^{-1},
\end{equation*}
where $\bbeta$ is the unknown coefficient vector (classifier) to be estimated. The {\em Maximum Likelihood Estimator (MLE)} of $\bbeta$ is found by minimizing the {\em negative log-likelihood (logloss)} 
\begin{equation*}
h_{\bbeta}(\bx, y) = \log(1+\exp(-y \bbeta'\bx)).
\end{equation*}
We define the distance metric on the predictor-response space as follows.
\begin{equation} \label{metric-lg-rwpi}
s((\bx_1, y_1), (\bx_2, y_2)) \triangleq 
\begin{cases}
\|\bx_1 - \bx_2\|_r, & \text{if $y_1 = y_2$,} \\
\infty, & \text{otherwise}.
\end{cases} 
\end{equation}
The distributionally robust logistic regression problem is formulated as:
\begin{equation} \label{dro-lg-rwpi}
\inf\limits_{\bbeta}\sup\limits_{\mbb{Q}\in \Omega}
\mbb{E}^{\mbb{Q}}\big[ \log(1+\exp(-y \bbeta'\bx))\big], 
\end{equation}
where the order-1 Wasserstein metric is used to define the set $\Omega$:
\begin{equation*}
\Omega \triangleq \{\mbb{Q}\in \scrP(\scrZ): W_{s,1}(\mathbb{Q},\ \hat{\mathbb{P}}_{N}) \le \epsilon\}.
\end{equation*}
We apply Theorem \ref{rwpi-limit} with $t=1$ to derive the optimal radius $\epsilon$. Note that
\begin{equation*}
\nabla_{\bbeta}h_{\bbeta}(\bx, y) = \frac{-y\bx}{1+\exp(y \bbeta'\bx)}.
\end{equation*}
Then, for $t=1$, as $N \rightarrow \infty$,
\begin{equation*}
\sqrt{N} R(\bbeta^*) \xrightarrow[]{ \ \text{d} \ } \bar{R}(1),
\end{equation*}
where 
$$\bar{R}(1) = \ \displaystyle\sup\nolimits_{\bzeta \in \scrA} \ \bzeta' \br,$$ and 
$$\br \sim \scrN\bigg(\mathbf{0}, \mbb{E}^{\mbb{P}^*} \bigg[ \frac{\bx \bx'}{\big(1+\exp(y \bx'\bbeta^*)\big)^2}\bigg]\bigg),$$ 
$$\scrA \triangleq \Big\{ \bzeta \in \mbb{R}^{d}: \  \sup\nolimits_{(\bx, y)}\|\bzeta'\nabla_{\bbeta,\bx}h_{\bbeta^*}(\bx, y)\|_s \le 1 \Big \},$$
where $s$ satisfies $1/r+1/s=1$.

Note that $\bar{R}(1)$ still depends on $\bbeta^*$ and $\mbb{P}^*$ which are both unknown. We need to find a stochastic upper bound of $\bar{R}(1)$ (for a conservative selection of the radius) that is independent of the unknown quantities. By noting that $\scrA$ is a subset of 
$$\{\bzeta \in \mbb{R}^d: \ \|\bzeta\|_s \le 1\},$$ 
and that 
$$\mbb{E}^{\mbb{P}^*_{\scrX}} [ \bx \bx'] - \mbb{E}^{\mbb{P}^*} \bigg[ \frac{\bx \bx'}{\big(1+\exp(y \bx'\bbeta^*)\big)^2}\bigg]$$ is positive definite, where $\mbb{P}^*_{\scrX}$ denotes the marginal distribution of $\bx$ under $\mbb{P}^*$, we have:
\begin{equation*}
\bar{R}(1) \overset{\mathclap{\mbox{\tiny D}}}{\le} \|\tilde{\br}\|_r,
\end{equation*}
where $\tilde{\br} \sim \scrN\big(\mathbf{0}, \mbb{E}^{\mbb{P}^*_{\scrX}}
[\bx\bx']\big)$, and $\overset{\mathclap{\mbox{\tiny D}}}{\le}$ denotes stochastic dominance.

The size of the Wasserstein ambiguity set for distributionally robust logistic regression can thus be chosen by the following procedure.
\begin{enumerate}
	\item Estimate the $(1-\alpha)$ quantile of $\|\tilde{\br}\|_r$, where $\tilde{\br} \sim \scrN\big(\mathbf{0}, \mbb{E}^{\mbb{P}^*_{\scrX}} [\bx\bx']\big)$. Denote the estimated quantile by $\hat{\chi}_{1-\alpha}$.
	\item Choose the radius $\epsilon$ to be $\epsilon = \hat{\chi}_{1-\alpha}/\sqrt{N}$.
\end{enumerate}

\chapter{Solving the Wasserstein DRO Problem}  \label{chapt:solve}
In this section we discuss how to solve the Wasserstein DRO problem, as well as the performance of the DRO estimator. A Lagrangian dual method is presented in Section~\ref{sec:dual-solver}, for a DRO model with an ambiguity set centered at a general nominal distribution. Section~\ref{sec:extreme-dist} discusses the existence and the structure of the extreme distribution that achieves the optimal value of the inner maximization problem of DRO. In Section~\ref{sec:empirical-dist-solver}, we apply the dual method to DRO models with an ambiguity set centered at the discrete empirical distribution. Sections~\ref{sec:finite-sample-perf} and \ref{sec:asymp-perf} study the finite sample and asymptotic performance of the DRO estimator, respectively. 

\section{Dual Method} \label{sec:dual-solver}
The main obstacle to solving the DRO problem (\ref{dro}) lies in the inner infinite dimensional maximization problem
\begin{equation}  \label{inner-dro}
\sup\nolimits_{\mbb{Q}\in \Omega} \
\mbb{E}^{\mbb{Q}}\big[ h(\bz)\big],
\end{equation}
where we suppress the dependence of $h$ on $\bbeta$ for ease of notation, and the ambiguity set is defined as:
\begin{equation*}
\Omega = \Omega_{\epsilon}^{s,t}(\mathbb{P}_0) \triangleq \{\mbb{Q}\in \scrP(\scrZ): W_{s,t}(\mathbb{Q},\ \mathbb{P}_0) \le \epsilon\}.
\end{equation*}
We will suppress the dependence of $\Omega$ on $s, t, \epsilon, \mathbb{P}_0$ for notational convenience.
To transform Problem (\ref{inner-dro}) into a finite dimensional problem, researchers
have resorted to Lagrangian duality, see \cite{esfahani2018data, gao2016distributionally}. Write Problem (\ref{inner-dro}) in the following form:
\begin{equation} \label{dro-primal}
\text{Primal:} \qquad \begin{array}{rl}
v_P = \displaystyle\sup_{\mbb{Q} \in \scrP(\scrZ)} & \Big\{ \int_{\scrZ} h(\bz) \mathrm{d}\mbb{Q}(\bz): W_{s,t}(\mbb{Q}, \mbb{P}_0) \leq \epsilon \Big\}.
\end{array}
\end{equation}
\cite{gao2016distributionally} derived the Lagrangian dual of (\ref{dro-primal}) as follows:
\begin{equation} \label{dro-dual}
\text{Dual:} \qquad \begin{array}{rl}
v_D = \displaystyle\inf_{\lambda \geq 0} & \Big\{ \lambda \epsilon^t - \int_{\scrZ} \inf_{\bz \in \scrZ} \big[\lambda s^t(\bz, \bz_0) - h(\bz)\big] \mathrm{d}\mbb{P}_0(\bz_0) \Big\},
\end{array}
\end{equation}
when the {\em growth rate} of the loss function $h(\bz)$, which, given an unbounded set $\scrZ$ and a fixed $\bz_0 \in \scrZ$, is defined as:
\begin{equation} \label{growth-rate-solver}
\text{GR}_h \triangleq \limsup\limits_{s(\bz, \bz_0) \rightarrow \infty} \frac{ h(\bz) - h(\bz_0)}{s^t(\bz, \bz_0)},
\end{equation}
is finite. Note that if $\scrZ$ is bounded, by convention we set $\text{GR}_h=0$. 
The value of $\text{GR}_h$ does not depend on the choice of $\bz_0$ \citep{gao2016distributionally}.
\medskip

\noindent \textbf{Remark:} 
Define a function
$$\phi(\lambda, \bz_0) \triangleq \inf_{\bz \in \scrZ} \big[\lambda s^t(\bz, \bz_0) - h(\bz)\big].$$ 
The dual objective function
\begin{equation*} 
v_D(\lambda) \triangleq   \lambda \epsilon^t - \int\nolimits_{\scrZ} \phi(\lambda, \bz_0) \mathrm{d}\mbb{P}_0(\bz_0), \qquad \lambda \geq 0,
\end{equation*}
is the sum of a linear function and an extended real-valued convex function $- \int_{\scrZ} \phi(\lambda, \bz_0) \mathrm{d}\mbb{P}_0(\bz_0)$. The convexity comes from the concavity of $\phi(\lambda, \bz_0)$ w.r.t. $\lambda$. To see this, for $q \in [0,1]$, and a fixed $\bz_0 \in \scrZ$,
\begin{equation*}
\begin{aligned}
& \quad \ \phi(q\lambda_1 + (1-q) \lambda_2, \bz_0) \\
& = \inf_{\bz \in \scrZ} \Big[\big(q\lambda_1 + (1-q) \lambda_2\big) s^t(\bz, \bz_0) - h(\bz) \Big] \\
& = \big(q\lambda_1 + (1-q) \lambda_2\big) s^t(\bz^*, \bz_0) - h(\bz^*) \\
& = q\Big[\lambda_1 s^t(\bz^*, \bz_0) - h(\bz^*)\Big] + (1-q)\Big[\lambda_2 s^t(\bz^*, \bz_0) - h(\bz^*)\Big] \\
& \geq q \phi(\lambda_1, \bz_0) + (1-q) \phi(\lambda_2, \bz_0),
\end{aligned}
\end{equation*}
where the first step uses the definition of $\phi$, $\bz^* = \arg\min_{\bz \in \scrZ} \big[\big(q\lambda_1 + (1-q) \lambda_2\big) s^t(\bz, \bz_0) - h(\bz) \big]$, and the last step is due to the fact that $\bz^* \in \scrZ$.

Thus, $v_D(\lambda)$ is a convex function on $[0, \infty)$.
Moreover, as $\lambda \rightarrow \infty$, $v_D(\lambda) \rightarrow \infty$, since $v_D(\lambda) \geq \lambda \epsilon^t + \int_{\bz_0 \in \scrZ} h(\bz_0) \mathrm{d} \mbb{P}_0(\bz_0)$, where the RHS is obtained through taking $\bz = \bz_0$ in the definition of $\phi$.

To see the necessity of having a finite growth rate, note that to ensure Problem (\ref{inner-dro}) has a finite optimal value, it is required that
\begin{equation*} 
\mbb{E}^{\mbb{Q}}\big[ h(\bz) \big] < \infty, \quad \forall \mbb{Q}\in \Omega.
\end{equation*}
This can be equivalently expressed as
\begin{equation} \label{finitedif}
\Bigl|\mbb{E}^{\mbb{Q}}\big[ h(\bz)\big] - \mbb{E}^{\mbb{P}_0}\big[ h(\bz)\big]\Bigr|< \infty, \quad \forall \mbb{Q}\in \Omega.
\end{equation}
The following Theorem~\ref{Lip-wass} implies that, if the growth rate of $h$ is infinite, (\ref{finitedif}) will be violated. Moreover, as we will see later, when the growth rate of the loss
function is infinite, strong duality for Problem (\ref{dro-primal}) fails to hold, in which case the DRO problem becomes intractable. 
In the sequel, we assume $h$ is upper semi-continuous and $\text{GR}_h < \infty$.

\begin{thm} \label{Lip-wass}
	Suppose a function $h: \scrZ \ra \mbb{R}$ defined on two metric spaces $(\scrZ, s)$ and $(\mbb{R}, |\cdot|)$, has a finite growth rate:
	$$\frac{\bigl|h(\bz_1)-h(\bz_2)\bigr|}{s^t(\bz_1, \bz_2)} \leq L, \qquad \forall \bz_1, \bz_2 \in \scrZ.$$
	Then, for any two distributions $\mbb{Q}_1$ and $\mbb{Q}_2$ supported on $\scrZ$,
	\begin{equation*}
	\Bigl|\mbb{E}^{\mbb{Q}_1}\big[ h(\bz)\big] -  \mbb{E}^{\mbb{Q}_2}\big[ h(\bz)\big]\Bigr| \leq L W^t_{s,t}(\mbb{Q}_1, \mbb{Q}_2).
	\end{equation*}
\end{thm}

\begin{proof}
	\begin{equation*} 
	\begin{aligned}
	\Bigl| \mbb{E}^{\mbb{Q}_1} & \big[ h(\bz)\big] -  \mbb{E}^{\mbb{Q}_2}\big[
	h(\bz)\big]\Bigr|  \\ 
	= & \ \biggl|\int\nolimits_{\scrZ} h(\bz_1) \mathrm{d}\mbb{Q}_1(\bz_1) -
	\int\nolimits_{\scrZ} h(\bz_2) \mathrm{d}\mbb{Q}_2(\bz_2) \biggr| \\ 
	= & \ \biggl|\int\nolimits_{\scrZ} h(\bz_1) \int\nolimits_{\bz_2\in
		\scrZ}\mathrm{d}\pi_0(\bz_1, \bz_2) - \int\nolimits_{\scrZ} h(\bz_2)
	\int\nolimits_{\bz_1\in \scrZ}\mathrm{d}\pi_0(\bz_1, \bz_2) \biggr| \\
	\le & \ \int\nolimits_{\scrZ \times \scrZ} \bigl|h(\bz_1)-h(\bz_2)\bigr| \mathrm{d}\pi_0(\bz_1,\bz_2) \\
	= & \int\nolimits_{\scrZ \times \scrZ} \frac{\bigl|h(\bz_1)-h(\bz_2)\bigr|}{s^t(\bz_1, \bz_2)} s^t(\bz_1, \bz_2) \mathrm{d}\pi_0(\bz_1, \bz_2) \\
	\le & \ \int\nolimits_{\scrZ \times \scrZ}  L  s^t(\bz_1, \bz_2) \mathrm{d}\pi_0(\bz_1,\bz_2) \\
	= & \ L W^t_{s,t}(\mbb{Q}_1, \mbb{Q}_2),
	\end{aligned}
	\end{equation*}
	where $\pi_0$ is the joint distribution of $\bz_1$ and $\bz_2$ with marginals $\mbb{Q}_1$ and $\mbb{Q}_2$ that achieves the optimal value of (\ref{wass_p}). 
\end{proof}

\subsection{Weak Duality}

The following Theorem~\ref{dro-duality} establishes weak duality for Problem (\ref{dro-primal}). Later we will show that strong duality also holds, i.e., $v_P = v_D$.

\begin{thm} [\cite{gao2016distributionally}, Proposition 1] \label{dro-duality}
	Suppose the loss function $h$ has a finite growth rate: $\text{GR}_h<\infty$. Then $v_P \leq v_D$, where $v_P$ and $v_D$ are defined in (\ref{dro-primal}) and (\ref{dro-dual}), respectively.
\end{thm}

\begin{proof}
	By weak duality, we have that:
	\begin{equation} \label{weak-dual}
	\begin{array}{rl}
	v_P \le \displaystyle\inf_{\lambda \geq 0} & \bigg\{ \lambda \epsilon^t + \sup\limits_{\mbb{Q} \in \scrP(\scrZ)} \Big\{\int_{\scrZ} h(\bz)\mathrm{d}\mbb{Q}(\bz) - \lambda W_{s,t}^t (\mbb{Q}, \mbb{P}_0) \Big\} \bigg\},
	\end{array}
	\end{equation}
	where the RHS is the Lagrangian dual of (\ref{dro-primal}). Using Kantorovich duality (\ref{dual-transport-cont}), we obtain
	\begin{equation*} 
	\begin{array}{rl}
	& \sup\limits_{\mbb{Q} \in \scrP(\scrZ)} \Big\{\int\nolimits_{\scrZ} h(\bz)\mathrm{d}\mbb{Q}(\bz) - \lambda W_{s,t}^t (\mbb{Q}, \mbb{P}_0) \Big\} \\
	= & \sup\limits_{\mbb{Q} \in \scrP(\scrZ)} \bigg\{\int\nolimits_{\scrZ}
	h(\bz)\mathrm{d}\mbb{Q}(\bz)\\ 
	& \qquad \qquad \qquad - \lambda  \sup\limits_{f, g} \Big \{ \int\nolimits_{\scrZ} f(\bz) \mathrm{d} \mbb{Q}(\bz) + \int\nolimits_{\scrZ} g(\bz_0) \mathrm{d} \mbb{P}_0(\bz_0): \\
	& \qquad \qquad \qquad \qquad \quad g(\bz_0) \le \inf\limits_{\bz \in \scrZ} \big[ s^t(\bz, \bz_0) - f(\bz)\big], \; \forall \bz_0 \in \scrZ \Big\} \bigg\} \\
	\le & -\int_{\scrZ} \inf\limits_{\bz \in \scrZ} \big[ \lambda s^t(\bz, \bz_0) - h(\bz)\big] \mathrm{d} \mbb{P}_0(\bz_0),
	\end{array}
	\end{equation*}
	where the second inequality is obtained through setting $f(\bz)=h(\bz)/\lambda$, for $\lambda>0$, which is absolutely integrable due to $\text{GR}_h<\infty$, and is thus a feasible solution to the inner supremum of the second line. For $\lambda=0$, the inequality also holds since,
	\begin{equation*} 
	\begin{array}{rl}
	\sup\limits_{\mbb{Q} \in \scrP(\scrZ)} \Big\{\int_{\scrZ} h(\bz)\mathrm{d}\mbb{Q}(\bz) \Big\} \le \sup\limits_{\bz \in \scrZ} h(\bz)
	=  -\int_{\scrZ} \inf\limits_{\bz \in \scrZ} \big[  - h(\bz)\big] \mathrm{d} \mbb{P}_0(\bz_0).
	\end{array}
	\end{equation*}
	Combining with (\ref{weak-dual}) we arrive at the conclusion that $v_P \le v_D$.
\end{proof}

\subsection{Strong Duality}

We next show that $v_P = v_D$, through constructing a feasible solution to the primal problem (\ref{dro-primal}) whose objective function value coincides with the dual objective. We first define the push-forward measure that will be used to construct a primal feasible distribution $\mbb{Q}^*$.

\begin{defi}[Push-Forward Measure]
	Given measurable spaces $\scrZ$ and $\scrZ'$, a measurable function $T: \scrZ \ra \scrZ'$, and a measure $\mbb{P} \in \scrB(\scrZ)$, define the push-forward measure of $\mbb{P}$ through $T$, denoted by $T_{\# \mbb{P}} \in \scrB(\scrZ')$, as
	\begin{equation*}
	T_{\# \mbb{P}}(\scrA) \triangleq \mbb{P}(T^{-1}(\scrA)) = \mbb{P}\{\bz \in \scrZ: \ T(\bz) \in \scrA\}, \qquad \scrA\subseteq \scrZ'.
	\end{equation*}
\end{defi}

Construct a distribution $\mbb{Q}^*$ as a convex combination of two distributions, each of which is a perturbation of the nominal distribution $\mbb{P}_0$:
\begin{equation} \label{ext-dist-form}
\mbb{Q}^* = q \underline{T}_{\# \mbb{P}_0} + (1-q) \overline{T}_{\# \mbb{P}_0},
\end{equation}
where the functions $\underline{T}, \overline{T} : \scrZ \ra \scrZ$ produce the minimizer to $\phi(\lambda^*, \bz_0)$, where $\lambda^*$ is the optimal solution to the dual problem (\ref{dro-dual}), i.e., 
\begin{equation} \label{T-func-def}
\underline{T}(\bz_0), \overline{T}(\bz_0) \in \Big\{\bz \in \scrZ: \ \lambda^* s^t(\bz, \bz_0) - h(\bz) = \phi(\lambda^*, \bz_0) \Big\},
\end{equation}
and $q \in [0,1]$ is chosen such that 
\begin{equation} \label{q-form}
q \int\nolimits_{\scrZ} s^t\big(\underline{T}(\bz_0), \bz_0\big) \mathrm{d}\mbb{P}_0(\bz_{0}) + (1-q) \int\nolimits_{\scrZ} s^t\big(\overline{T}(\bz_0), \bz_0\big) \mathrm{d}\mbb{P}_0(\bz_{0}) = \epsilon^t.
\end{equation}
We choose $\underline{T}, \overline{T}$ to satisfy the following conditions
\begin{equation*}
\int\nolimits_{\scrZ} s^t\big(\underline{T}(\bz_0), \bz_0\big) \mathrm{d}\mbb{P}_0(\bz_{0}) \le \epsilon^t,
\end{equation*}
\begin{equation*}
\int\nolimits_{\scrZ} s^t\big(\overline{T}(\bz_0), \bz_0\big) \mathrm{d}\mbb{P}_0(\bz_{0}) \geq \epsilon^t,
\end{equation*}
in order to ensure the existence of such a $q$.

We first show that $\mbb{Q}^*$ is primal feasible. Notice that 
\begin{equation*}
\begin{aligned}
& W_{s,t}^t(\mbb{Q}^*, \mbb{P}_0)\\ 
= &  \sup\limits_{f, g} \Big \{ \int\nolimits_{\scrZ} f(\bz) \mathrm{d} \mbb{Q}^*(\bz)  + \int\nolimits_{\scrZ} g(\bz_0) \mathrm{d} \mbb{P}_0(\bz_0): \\
& \qquad \qquad \qquad \qquad \  f(\bz) \le \inf\limits_{\bz_0 \in \scrZ} \big[ s^t(\bz, \bz_0) - g(\bz_0)\big], \; \forall \bz \in \scrZ \Big\} \\
= & \sup\limits_{f, g} \Big \{ q\int\nolimits_{\scrZ} f(\bz)  \mathrm{d}\mbb{P}_0 \big (\underline{T}^{-1}(\bz)\big) + (1-q) \int\nolimits_{\scrZ} f(\bz) \mathrm{d}\mbb{P}_0 \big (\overline{T}^{-1}(\bz) \big) \\ 
& \qquad + \int\nolimits_{\scrZ}  g(\bz_0) \mathrm{d} \mbb{P}_0(\bz_0):  
f(\bz) \le \inf\nolimits_{\bz_0 \in \scrZ} \big[ s^t(\bz, \bz_0) - g(\bz_0)\big], \; \forall \bz \in \scrZ \Big\} \\
\le & \sup\limits_{g} \Big \{ q\int\nolimits_{\scrZ} \Big( s^t\big(\underline{T}(\bz_0), \bz_0\big) - g(\bz_0)\Big)  \mathrm{d}\mbb{P}_0 \big (\bz_0\big) \\
& \qquad + (1-q)\int\nolimits_{\scrZ} \Big( s^t\big(\overline{T}(\bz_0), \bz_0\big) - g(\bz_0)\Big) \mathrm{d}\mbb{P}_0 \big (\bz_0 \big) 
+ \int\nolimits_{\scrZ}  g(\bz_0) \mathrm{d} \mbb{P}_0(\bz_0)\Big\} \\
= & \ q\int\nolimits_{\scrZ} s^t\big(\underline{T}(\bz_0), \bz_0\big)  \mathrm{d}\mbb{P}_0 \big (\bz_0\big) + (1-q)\int\nolimits_{\scrZ} s^t\big(\overline{T}(\bz_0), \bz_0\big) \mathrm{d}\mbb{P}_0 \big (\bz_0 \big) \\
= & \ \epsilon^t,
\end{aligned}
\end{equation*}
where the first step uses the Kantorovich
duality (\ref{dual-transport-cont}), the second step uses the structure of $\mbb{Q}^*$ in (\ref{ext-dist-form}), the third step replaces $f(\bz)$ by its upper bound $s^t\big(\bz, \bz_0\big) - g(\bz_0)$, and the last step uses the definition of $q$ in (\ref{q-form}).

Now that the feasibility of $\mbb{Q}^*$ has been established, we next prove that
$\mbb{Q}^*$ is the primal optimal solution by showing that its objective function
value matches the optimal dual value.
\begin{equation*}
\begin{aligned}
\int\nolimits_{\scrZ} h(\bz) \mathrm{d} \mbb{Q}^*(\bz) & = q\int\nolimits_{\scrZ} h(\bz) \mathrm{d}\mbb{P}_0 (\underline{T}^{-1}(\bz)) + (1-q) \int\nolimits_{\scrZ} h(\bz) \mathrm{d}\mbb{P}_0 (\overline{T}^{-1}(\bz))\\
& = q\int\nolimits_{\scrZ} \Big( \lambda^* s^t(\underline{T}(\bz_0), \bz_0) - \phi(\lambda^*, \bz_0)\Big) \mathrm{d}\mbb{P}_0 (\bz_0) \\
& \quad + (1-q) \int\nolimits_{\scrZ} \Big( \lambda^* s^t(\overline{T}(\bz_0), \bz_0) - \phi(\lambda^*, \bz_0)\Big) \mathrm{d}\mbb{P}_0 (\bz_0)\\
& = q\lambda^*\int\nolimits_{\scrZ}  s^t(\underline{T}(\bz_0), \bz_0) \mathrm{d}\mbb{P}_0 (\bz_0) -\int\nolimits_{\scrZ} \phi(\lambda^*, \bz_0) \mathrm{d}\mbb{P}_0 (\bz_0) \\
& \quad + (1-q) \lambda^*  \int\nolimits_{\scrZ}  s^t(\overline{T}(\bz_0), \bz_0)  \mathrm{d}\mbb{P}_0 (\bz_0) \\
& = \lambda^* \epsilon^t -\int\nolimits_{\scrZ} \phi(\lambda^*, \bz_0) \mathrm{d}\mbb{P}_0 (\bz_0) \\
& = v_D,
\end{aligned}
\end{equation*}
where the first step uses the structure of $\mbb{Q}^*$ in (\ref{ext-dist-form}), the second step uses the definition of $\underline{T}, \overline{T}$ in (\ref{T-func-def}), the fourth step uses the definition of $q$ in (\ref{q-form}), and the last step results from the optimality of $\lambda^*$. We are now ready to state the strong duality result.

\begin{thm} [\cite{gao2016distributionally}, Theorem 1]\label{strong-duality}
	Suppose that $\text{GR}_h<\infty$. The dual problem (\ref{dro-dual}) always admits a minimizer $\lambda^*$, and strong duality holds: $v_P = v_D < \infty$. 
\end{thm}

\noindent \textbf{Remark:} The dual problem (\ref{dro-dual}) admits a minimizer 
$$\lambda^* \in [\max(0, \text{GR}_h), \infty).$$
To see this, notice that for all $\lambda < \text{GR}_h$, $\phi(\lambda, \bz_0) = -\infty$, since 
\begin{equation*}
\begin{aligned}
& \limsup\limits_{s(\bz, \bz_0) \rightarrow \infty} \Big[\lambda s^t(\bz, \bz_0) - h(\bz)\Big] \\
= & \limsup\limits_{s(\bz, \bz_0) \rightarrow \infty} \Big(\lambda-\frac{ h(\bz) - h(\bz_0)}{s^t(\bz, \bz_0)}\Big) s^t(\bz, \bz_0) - h(\bz_0) \\
= & -\infty,
\end{aligned}
\end{equation*}
in which case $v_D(\lambda) = \infty$. We conclude that $\lambda^* \ge \text{GR}_h$.

By using duality, \cite{esfahani2018data, gao2016distributionally, zhao2015data} proposed tractable convex reformulations for the DRO problem (\ref{dro}). For Lipschitz continuous loss functions, the duality result leads to an equivalent formulation for the Wasserstein DRO as a regularized empirical loss minimization problem, where the regularizer is related to the Lipschitz constant of the loss, see \cite{gao2017wasserstein, shafieezadeh2017regularization}. This connection between robustness and regularization has also been established in \cite{abadeh2015distributionally, chen2018robust}. We will discuss it in further details in Section~\ref{chapt:dro}.

\section{The Extreme Distribution} \label{sec:extreme-dist}
Section \ref{sec:dual-solver} reveals the structure of the primal optimal solution $\mbb{Q}^*$ (the extreme distribution) in (\ref{ext-dist-form}). We summarize the discussions on the existence and the form of the extreme distribution in the following theorem.

\begin{thm} [\cite{gao2016distributionally}, Corollary 1] \label{thm:ext-dist}
	Suppose $\scrZ = \mbb{R}^d$. The worst-case distribution exists if there exists a dual minimizer $\lambda^*$, and the set
	$\{\bz \in \scrZ: \ \lambda^* s^t(\bz, \bz_0) - h(\bz) = \phi(\lambda^*, \bz_0)\}$ is non-empty $\mbb{P}_0$-almost everywhere, and 
	\begin{equation*}
	\int\nolimits_{\scrZ} \underline{s}^t(\lambda^*, \bz_0) \mathrm{d}\mbb{P}_0(\bz_{0}) \le \epsilon^t,
	\end{equation*}
	\begin{equation*}
	\int\nolimits_{\scrZ} \overline{s}^t(\lambda^*, \bz_0) \mathrm{d}\mbb{P}_0(\bz_{0}) \geq \epsilon^t,
	\end{equation*}
	where
	\begin{equation*}
	\underline{s}(\lambda, \bz_0) \triangleq \min_{\bz \in \scrZ} \{ s(\bz, \bz_0): \ \lambda s^t(\bz, \bz_0) - h(\bz) = \phi(\lambda, \bz_0)\},
	\end{equation*}
	and,
	\begin{equation*}
	\overline{s}(\lambda, \bz_0) \triangleq \max_{\bz \in \scrZ} \{ s(\bz, \bz_0): \ \lambda s^t(\bz, \bz_0) - h(\bz) = \phi(\lambda, \bz_0)\}.
	\end{equation*}
	Whenever the worst-case distribution exists, there exists one which can be represented as a convex combination of two distributions, each of which is a perturbation of the nominal distribution:
	\begin{equation*}
	\mbb{Q} = q \underline{T}_{\# \mbb{P}_0} + (1-q) \overline{T}_{\# \mbb{P}_0},
	\end{equation*}
	where $q \in [0,1]$, and $\underline{T}, \overline{T} : \scrZ \ra \scrZ$ satisfy
	\begin{equation*}
	\underline{T}(\bz_0), \overline{T}(\bz_0) \in \{\bz \in \scrZ: \ \lambda^* s^t(\bz, \bz_0) - h(\bz) = \phi(\lambda^*, \bz_0)\}.
	\end{equation*}
\end{thm}

\section{A Discrete Empirical Nominal Distribution} \label{sec:empirical-dist-solver}
In this section we apply the strong duality result developed in previous sections to the scenario where the discrete empirical distribution $\hat{\mbb{P}}_N$ is used as the center of the ambiguity set.
\begin{col} [\cite{gao2016distributionally}, Corollary 2] \label{col-discrete-dual}
	Suppose we use the empirical distribution 
	$$\hat{\mbb{P}}_{N} \triangleq \frac{1}{N} \sum_{i=1}^N \delta_{\bz_i}(\bz)$$
	as the center of the ambiguity set, i.e., $\mbb{P}_0 = \hat{\mbb{P}}_N$, where $\bz_i, i \in \lb N \rb$, are the observed realizations of $\bz$.
	Assume $\text{GR}_h<\infty$. Then,
	
	\noindent $(i)$ The primal problem (\ref{dro-primal}) has a strong dual problem
	\begin{equation} \label{discrete-dro-dual}
	v_P = v_D = \displaystyle \min_{\lambda \geq 0} \Big\{ \lambda \epsilon^t + \frac{1}{N} \sum\limits_{i=1}^N \sup_{\bz \in \scrZ} \big[ h(\bz) - \lambda s^t(\bz, \bz_i)\big]\Big\}.
	\end{equation}
	Moreover, $v_P, v_D$ are also equal to
	\begin{equation}  \label{discrete-worst-formul}
	\begin{array}{rl}
	\displaystyle \sup\limits_{\underline{\bz}_i, \overline{\bz}_i, q_1, q_2} & \Big \{ \frac{1}{N} \sum\limits_{i=1}^N \big[q_1 h(\underline{\bz}_i) + q_2 h(\overline{\bz}_i) \big]\Big\} \\
	\text{s.t.} & \frac{1}{N} \sum\limits_{i=1}^N \big[ q_1 s^t(\underline{\bz}_i, \bz_i) + q_2 s^t(\overline{\bz}_i, \bz_i)\big] \le \epsilon^t, \\
	\qquad & q_1 + q_2 \le 1, \\
	\qquad & q_1, q_2 \geq 0.
	\end{array}
	\end{equation}
	$(ii)$ When $\scrZ$ is convex and $h$ is concave, (\ref{discrete-dro-dual}) could be reduced to
	\begin{equation} \label{concave-loss-dual}
	\begin{array}{rl}
	\displaystyle \sup_{\tilde{\bz}_i \in \scrZ} &  \frac{1}{N} \sum\limits_{i=1}^N h(\tilde{\bz}_i)\\
	\text{s.t.} & \frac{1}{N} \sum\limits_{i=1}^N s^t(\bz_i, \tilde{\bz}_i) \le \epsilon^t.
	\end{array}
	\end{equation}
	$(iii)$ Whenever the worst-case distribution exists, there exists one which is supported on at most $N+1$ points and has the form
	\begin{equation*}
	\mbb{Q}^* = \frac{1}{N}\sum\limits_{i \neq i_0} \delta_{\bz_i^*}(\bz) + \frac{q}{N} \delta_{\underline{\bz}_{i_0}^*}(\bz) + \frac{1-q}{N} \delta_{\overline{\bz}_{i_0}^*}(\bz),
	\end{equation*}
	where $1 \leq i_0 \leq N$, $q \in [0,1]$, $\underline{\bz}_{i_0}^*, \overline{\bz}_{i_0}^* \in \arg\min_{\bz \in \scrZ} \{\lambda^* s^t(\bz, \bz_{i_0}) - h(\bz) \}$, and $\bz_i^* \in \arg\min_{\bz \in \scrZ} \{\lambda^* s^t(\bz, \bz_{i}) - h(\bz) \}$ for all $i \neq i_0$.
	\begin{proof}
		(\ref{discrete-dro-dual}) comes directly from (\ref{dro-dual}). For (\ref{discrete-worst-formul}), recall that the worst-case distribution can be expressed as a convex combination of two perturbed versions of the empirical distribution, see (\ref{ext-dist-form}) and (\ref{T-func-def}). Thus, $\mbb{Q}^*$ is supported on $2N$ points $\underline{\bz}_i, \overline{\bz}_i, \ i \in \lb N \rb$, with probabilities $q/N$ and $(1-q)/N$, respectively. Problem (\ref{discrete-worst-formul}) finds the worst-case expected loss by imposing such a structure on the distribution $\mbb{Q}$. 
		
		Part  $(ii)$ can be proved by noticing that Problem (\ref{discrete-dro-dual}) is the Lagrangian dual of (\ref{concave-loss-dual}), which is a convex problem due to the concavity of $h$.
		
		To prove $(iii)$, consider problem (\ref{discrete-worst-formul}) by replacing
		$q_1$ with $q_{i}$ and $q_2$ with $1-q_{i}$, i.e., we allow $q$ to vary over
		samples. (\ref{discrete-worst-formul}) is a linear program in $q_i$ and has
		an optimal solution which has at most one fractional point (i.e., $\exists i_0,\ \text{s.t.} \ q_{i_0}>0;\ q_i = 0,\ \forall i \neq i_0$). Therefore, there exists a worst-case distribution which is supported on at most $N+1$ points. 
	\end{proof}
\end{col}

\subsection{A Special Case} \label{sec:solver-special-case}
We study a special case where the loss function $h(\bz)$ is convex in $\bz$. We will
show that Problem (\ref{inner-dro}) can be relaxed to the summation of the empirical
loss and a regularizer, where the regularization strength is equal to the size of the
ambiguity set, and the regularizer is defined by the dual norm.

Before we present this result, we start with two definitions and a well-known property. 
\begin{defi}[Dual norm]
	Given a norm $\|\cdot\|$ on $\mbb{R}^d$, the dual norm
	$\|\cdot\|_*$ is defined as: 
	\begin{equation} \label{dnorm}
	\|\btheta\|_* \triangleq \sup_{\|\bz\|\le 1}\btheta'\bz.
	\end{equation}
\end{defi}
It can be shown from (\ref{dnorm}) that for any vectors $\btheta, \bz$,
the following H\"{o}lder's inequality
holds.
%\begin{equation*} 
%|\btheta'\bz| \le \|\btheta\|_* \|\bz\|.
%\end{equation*}
\begin{thm} [H\"{o}lder's inequality] \label{holder}
	Suppose we have two scalars $r, s >1$ and $1/r + 1/s =1$. For any two vectors $\btheta = (\theta_1, \ldots, \theta_n)$ and $\bz = (z_1, \ldots, z_n)$, the following holds:
	\begin{equation*}
	\sum_{i=1}^n |\theta_i z_i| \le \|\btheta\|_r \|\bz\|_s.
	\end{equation*}
\end{thm}

\begin{defi}[Conjugate function]
	For a function $h(\bz)$, its convex conjugate $h^*(\cdot)$ is defined as:
	\begin{equation} \label{conj}
	h^*(\btheta) \triangleq
	\sup_{\bz \in \text{dom} \ h}\ \{\btheta'\bz-h(\bz)\},
	\end{equation}
	where $\text{dom} \ h$ denotes the domain of the function $h$. 
\end{defi}
If $h$ is convex, then the convex conjugate of $h^*$ is $h$, and $h$ and $h^*$ are
called convex duals~\citep{rock}. In particular, 
\begin{equation} \label{conj-dual}
h(\bz) = \sup_{\btheta \in \bTheta} [ \btheta' \bz -  h^*(\btheta)], 
\end{equation}
where $\bTheta \triangleq \{ \btheta: h^*(\btheta) < \infty\}$ denotes the effective
domain of the conjugate function $h^*$.

\begin{thm} [\cite{esfahani2018data}]  \label{discrete-dual-convex-thm}
	Suppose the loss function $h(\bz)$ is convex in $\bz\in \scrZ\subseteq
	\mbb{R}^d$, and the set $\scrZ$ is closed and convex. Define an ambiguity set
	around the empirical distribution which is supported on $N$ samples $\bz_i, i \in
	\lb N \rb$, i.e., 
	$$\Omega = \Big\{\mbb{Q}\in \scrP(\scrZ): W_{\|\cdot\|,1}(\mathbb{Q},\ \hat{\mathbb{P}}_N) \le \epsilon \Big\},$$
	where the order-1 Wasserstein metric (\ref{wass_p}) is induced by some norm $\|\cdot\|$.
	Problem (\ref{inner-dro}) can be relaxed to:
	\begin{equation} \label{kk}
	\sup\limits_{\mbb{Q}\in \Omega}\mbb{E}^{\mbb{Q}}[h(\bz)]\le
	\kappa\epsilon+\frac{1}{N}\sum\limits_{i=1}^Nh(\bz_i),
	\end{equation}
	where 
	$$\kappa=\sup\{\|\btheta\|_*: \ h^*(\btheta)<\infty\},$$ 
	where $\|\cdot\|_*$ stands for the dual norm as defined in (\ref{dnorm}), and $h^*(\cdot)$ is the convex conjugate function of $h(\bz)$ as defined in (\ref{conj}). Furthermore, (\ref{kk}) becomes an equality when
	$\scrZ=\mathbb{R}^{d}$.
\end{thm}

\begin{proof}
	Corollary \ref{col-discrete-dual} suggests that 
	\begin{equation} \label{discrete-convex-dual}
	\sup\limits_{\mbb{Q}\in \Omega}\mbb{E}^{\mbb{Q}}[h(\bz)] = \displaystyle \min_{\lambda \geq 0} \Big\{ \lambda \epsilon + \frac{1}{N} \sum\limits_{i=1}^N \sup_{\bz \in \scrZ} \big[ h(\bz) - \lambda \|\bz - \bz_i \|\big]\Big\}.
	\end{equation}
	Using (\ref{conj-dual}), we may write the inner maximization in (\ref{discrete-convex-dual}) as:
	\begin{equation} \label{discrete-convex-dual-inner-sup}
	\begin{aligned}
	\sup_{\bz \in \scrZ} \big[ h(\bz) - \lambda \|\bz - \bz_i \|\big] & = \sup_{\bz \in \scrZ} \sup_{\btheta \in \bTheta} \big[ \btheta' \bz -  h^*(\btheta) - \lambda \|\bz - \bz_i \|\big] \\
	& = \sup_{\bz \in \scrZ} \sup_{\btheta \in \bTheta} \inf_{\|\br\|_* \le \lambda} \big[ \btheta' \bz -  h^*(\btheta) + \br' (\bz - \bz_i)\big] \\
	& = \sup_{\btheta \in \bTheta} \inf_{\|\br\|_* \le \lambda} \sup_{\bz \in \scrZ}   \big[ (\btheta + \br)' \bz -  h^*(\btheta) - \br' \bz_i\big] \\
	& \le \sup_{\btheta \in \bTheta} \inf_{\|\br\|_* \le \lambda} \sup_{\bz \in \mbb{R}^d}   \big[ (\btheta + \br)' \bz -  h^*(\btheta) - \br' \bz_i\big],
	\end{aligned}
	\end{equation}
	where the second equality follows from the definition of the dual norm and the third
	equality uses duality. The inner maximization over $\bz \in \mbb{R}^d$
	achieves $\infty$ unless $\br = -\btheta$.

	Note that if $\sup\{\|\btheta\|_* : \btheta \in
	\bTheta\} > \lambda$, then one can pick some $\btheta \in \bTheta$ such that $\|\btheta\|_* > \lambda$, in which case the inner maximization over $\bz \in \mbb{R}^d$ in (\ref{discrete-convex-dual-inner-sup}) achieves $\infty$ since $\br \neq -\btheta$. 
	
	When $\sup\{\|\btheta\|_* : \btheta \in
	\bTheta\}\leq \lambda$, by taking $\br = -\btheta$, we have:
	\begin{equation} \label{discrete-convex-dual-inner-sup-bound}
	\begin{aligned}
	\sup_{\bz \in \scrZ} \big[ h(\bz) - \lambda \|\bz - \bz_i \|\big] & 
	\le  \sup_{\btheta \in \bTheta } \big[ -  h^*(\btheta) + \btheta' \bz_i\big] \\
	& = h(\bz_i).
	\end{aligned}
	\end{equation}
	Plugging (\ref{discrete-convex-dual-inner-sup-bound}) into (\ref{discrete-convex-dual}), we obtain
	\begin{equation*}
	\sup\limits_{\mbb{Q}\in \Omega}\mbb{E}^{\mbb{Q}}[h(\bz)] \le \frac{1}{N} \sum\limits_{i=1}^N h(\bz_i) + \kappa \epsilon,
	\end{equation*}
	where $\kappa=\sup\{\|\btheta\|_*: \ h^*(\btheta)<\infty\}$.
\end{proof}

\section{Finite Sample Performance} \label{sec:finite-sample-perf}
In this section we discuss the finite sample out-of-sample performance of the DRO
estimator. Recall the stochastic optimization problem  defined in (\ref{sto}):
\begin{equation} \label{sto-opt}
J^* \triangleq \displaystyle \inf_{\bbeta} \mbb{E}^{\mbb{P}^*} \big [
h_{\bbeta}(\bz)\big]  = \inf_{\bbeta} \int\nolimits_{\scrZ}
h_{\bbeta}(\bz) \mathrm{d} \mbb{P}^*(\bz). 
\end{equation}
Since the true measure $\mbb{P}^*$ is unknown, Problem (\ref{sto-opt}) is not directly solvable. We solve its DRO counterpart (\ref{dro}) using the available training data $\bz_i, i \in \lb N \rb$, with an effort to implicitly optimize over the true measure that is included in the ambiguity set with high confidence. Suppose $\hat{J}_N$ and $\hat{\bbeta}_N$ are respectively the optimal value and optimal solution to the DRO problem (\ref{dro}), i.e.,
\begin{equation} \label{dro-perf}
\hat{J}_N \triangleq \inf\limits_{\bbeta}\sup\limits_{\mbb{Q}\in \Omega}
\mbb{E}^{\mbb{Q}}\big[ h_{\bbeta}(\bz)\big] = \sup\limits_{\mbb{Q}\in \Omega}
\mbb{E}^{\mbb{Q}}\big[ h_{\hat{\bbeta}_N}(\bz)\big],
\end{equation}
where the ambiguity set is defined as
\begin{equation} \label{omega-perf}
\Omega = \Omega_{\epsilon}(\hat{\mathbb{P}}_N) \triangleq \Big\{\mbb{Q}\in \scrP(\scrZ): W_{\|\cdot\|,1}(\mathbb{Q},\ \hat{\mathbb{P}}_N) \le \epsilon \Big\}.
\end{equation}
To evaluate the quality of the DRO estimator $\hat{\bbeta}_N$, we study its {\em out-of-sample} performance on a new sample $\bz$ drawn from $\mbb{P}^*$,
\begin{equation} \label{out-of-sample-loss-dro}
\mbb{E}^{\mbb{P}^*} \big [ h_{\hat{\bbeta}_N}(\bz)\big].
\end{equation}
We want to investigate whether the {\em out-of-sample} loss (\ref{out-of-sample-loss-dro}) can be meaningfully bounded from above by some {\em certificate}. Specifically, if we can show that with a high probability, the {\em out-of-sample} loss (\ref{out-of-sample-loss-dro}) does not exceed the training loss $\hat{J}_N$,
\begin{equation*}
\mbb{P}^N \Big \{ \mbb{E}^{\mbb{P}^*} \big [ h_{\hat{\bbeta}_N}(\bz)\big] \le \hat{J}_N\Big\} \geq 1-\alpha,
\end{equation*}
where $\alpha \in (0,1)$ is a significance parameter w.r.t. the distribution $\mbb{P}^N$, which governs both $\hat{\bbeta}_N$ and $\hat{J}_N$, then we can claim that $\hat{\bbeta}_N$ generalizes well out-of-sample. The following theorem, which follows directly from the measure concentration Theorem~\ref{measure-con}, establishes the result.

\begin{thm} [\cite{esfahani2018data}, Theorem 3.5]\label{finite-sample-perf}
	Suppose Assumption \ref{light-tail-dist} holds, and $\hat{J}_N$ and $\hat{\bbeta}_N$ are respectively the optimal value and optimal solution to the DRO problem (\ref{dro}) with an ambiguity set specified in (\ref{omega-perf}). Set the radius $\epsilon = \epsilon_N(\alpha)$ as defined in (\ref{radius-FG}), where $\alpha \in (0,1)$. Then we have
	\begin{equation*}
	\mbb{P}^N \Big \{ \mbb{E}^{\mbb{P}^*} \big [ h_{\hat{\bbeta}_N}(\bz)\big] \le \hat{J}_N\Big\} \geq 1-\alpha.
	\end{equation*}
\end{thm}
\begin{proof}
	The claim follows immediately from the measure concentration result presented in
	Theorem \ref{measure-con}, which establishes that
	\begin{equation*}
	\mbb{P}^N \Big( W_{\|\cdot\|, 1} (\mbb{P}^*, \hat{\mbb{P}}_N) \ge \epsilon_N(\alpha)\Big) \le \alpha,
	\end{equation*}
	and therefore,
	\begin{equation*}
	\mbb{P}^N \Big \{ \mbb{E}^{\mbb{P}^*} \big [ h_{\hat{\bbeta}_N}(\bz)\big] \leq \hat{J}_N\Big\} \geq \mbb{P}^N \Big \{ \mbb{P}^*\in \Omega_{\epsilon_N(\alpha)}(\hat{\mbb{P}}_N) \Big\} \geq 1-\alpha.
	\end{equation*}
\end{proof}
Note that Theorem \ref{finite-sample-perf} establishes the out-of-sample performance of the DRO estimator for an order-1 Wasserstein ambiguity set. For a general Wasserstein metric with order $t>1$, please refer to \cite{Four14} for a general measure concentration result.

\section{Asymptotic Consistency} \label{sec:asymp-perf}

In addition to the finite sample result established in
Section~\ref{sec:finite-sample-perf}, we are also interested in the asymptotic
behavior of 
$\hat{J}_N$ and $\hat{\bbeta}_N$, as the sample size $N$ goes to infinity. We want to
establish that, if the significance level $\alpha = \alpha_N$ converges to zero at a
carefully chosen rate, then the optimal value and solution of the DRO problem
(\ref{dro}) with an ambiguity set of size $\epsilon = \epsilon_N (\alpha_N)$,
converge to the optimal value and solution of the original stochastic optimization
problem (\ref{sto-opt}), respectively. The following Theorem~\ref{asym-dro-perf}
formalizes this statement.

\begin{thm} [\cite{esfahani2018data}, Theorem 3.6]\label{asym-dro-perf}
	Suppose Assumption \ref{light-tail-dist} holds and the significance parameter $\alpha_N \in (0,1)$ satisfies 
	\begin{itemize}
		\item $\sum_{N=1}^{\infty} \alpha_N < \infty$;
		\item $\lim_{N \rightarrow \infty} \epsilon_N(\alpha_N) = 0$.
	\end{itemize}
	Assume the loss function $h_{\bbeta}(\bz)$ is Lipschitz continuous in $\bz$ with
	a Lipschitz constant $L_{\bbeta}$. Denote by $\hat{J}_N$ and $\hat{\bbeta}_N$ the
	optimal value and optimal solution to the DRO problem (\ref{dro}), respectively,
	with an ambiguity set specified in (\ref{omega-perf}) with $\epsilon = \epsilon_N
	(\alpha_N)$, where $\epsilon_N (\alpha_N)$ is defined in (\ref{radius-FG}), and
	$J^*$ is the optimal value of the original stochastic optimization problem
	(\ref{sto-opt}). Then, 
	
	\noindent $(i)$ $\hat{J}_N$ converges to $J^*$ a.s.,
	\begin{equation*}
	\mbb{P}^{\infty} \Big \{ \limsup_{N \rightarrow \infty} \hat{J}_N = J^*\Big\} = 1.
	\end{equation*}
	
	\noindent $(ii)$ If $h_{\bbeta}(\bz)$ is lower semicontinuous in $\bbeta$ for every
	$\bz \in \scrZ$, and \linebreak[4] $\lim_{N \rightarrow \infty} \hat{\bbeta}_N = \bbeta_0$, then $\bbeta_0$ is $\mbb{P}^{\infty}$-almost surely an optimal solution to (\ref{sto-opt}).
\end{thm}

\begin{proof}
	$(i)$ Theorem \ref{finite-sample-perf} implies that
	\begin{equation} \label{consist-proof-out-of-perf}
	\mbb{P}^N \Big \{ J^* \leq \mbb{E}^{\mbb{P}^*} \big [ h_{\hat{\bbeta}_N}(\bz)\big] \le \hat{J}_N\Big\} \geq \mbb{P}^N \Big \{ \mbb{P}^*\in \Omega_{\epsilon_N(\alpha_N)}(\hat{\mbb{P}}_N) \Big\} \geq 1-\alpha_N.
	\end{equation}
	As $\sum_{N=1}^{\infty} \alpha_N < \infty$, the Borel-Cantelli Lemma \citep{borel1909probabilites,francesco1917cantelli} implies that,
	\begin{equation*}
	\mbb{P}^{\infty} \Big \{ \limsup_{N \rightarrow \infty} \hat{J}_N \geq J^* \Big\} = 1.
	\end{equation*}
	It remains to show that 
	\begin{equation} \label{jN-less-than-jstar}
	\mbb{P}^{\infty} \Big \{ \limsup_{N \rightarrow \infty} \hat{J}_N \leq J^* \Big\} = 1.
	\end{equation}
	% Choose any $\delta>0$, find a $\delta$-optimal solution to Problem (\ref{sto-opt}), denoted by $\bbeta_{\delta}$, satisfying
	% \begin{equation*}
	%     \mbb{E}^{\mbb{P}^*} \big [ h_{\bbeta_{\delta}}(\bz)\big] \leq J^* + \delta.
	% \end{equation*}
	Let $\hat{\mbb{Q}}_N \in \Omega_{\epsilon_N(\alpha_N)}(\hat{\mbb{P}}_N)$ be the
	optimal solution to the inner supremum (\ref{inner-dro}) corresponding to $\bbeta
	= \bbeta^*$, where $\bbeta^*$ is the optimal solution to (\ref{sto-opt}). Then, 
	\begin{equation*}
	\mbb{E}^{\hat{\mbb{Q}}_N}\big[ h_{\bbeta^*}(\bz)\big] = \sup\nolimits_{\mbb{Q}\in \Omega_{\epsilon_N(\alpha_N)}(\hat{\mbb{P}}_N)} \mbb{E}^{\mbb{Q}}\big[ h_{\bbeta^*}(\bz)\big] .
	\end{equation*}
	According to Theorem \ref{Lip-wass}, and due to the Lipschitz continuity of $h_{\bbeta}(\bz)$, we know that
	\begin{equation*} 
	\begin{aligned}
	& \ \Bigl|\mbb{E}^{\mbb{Q}_1}\big[ h_{\bbeta}(\bz)\big] -  \mbb{E}^{\mbb{Q}_2}\big[ h_{\bbeta}(\bz)\big]\Bigr| 
	% 	= & \ \biggl|\int\nolimits_{\scrZ} h_{\bbeta}(\bz_1) \mathrm{d}\mbb{Q}_1(\bz_1) - \int\nolimits_{\scrZ} h_{\bbeta}(\bz_2) \mathrm{d}\mbb{Q}_2(\bz_2) \biggr| \\
	% 	= & \ \biggl|\int\nolimits_{\scrZ} h_{\bbeta}(\bz_1) \int\nolimits_{\scrZ}\mathrm{d}\pi_0(\bz_1, \bz_2) - \int\nolimits_{\scrZ} h_{\bbeta}(\bz_2) \int\nolimits_{\scrZ}\mathrm{d}\pi_0(\bz_1, \bz_2) \biggr| \\
	% 	\le & \ \int\nolimits_{\scrZ \times \scrZ} \bigl|h_{\bbeta}(\bz_1)-h_{\bbeta}(\bz_2)\bigr| \mathrm{d}\pi_0(\bz_1,\bz_2) \\
	% 	\le & \ \int\nolimits_{\scrZ \times \scrZ}  L_{\bbeta} \bigl\| \bz_1 -\bz_2\bigr\| \mathrm{d}\pi_0(\bz_1,\bz_2) \\
	\leq \ L_{\bbeta} W_{\|\cdot\|,1}(\mbb{Q}_1, \mbb{Q}_2).
	\end{aligned}
	\end{equation*}
	% where $\pi_0$ is the joint distribution of $\bz_1$ and $\bz_2$ with marginals $\mbb{Q}_1$ and $\mbb{Q}_2$ that achieves the optimal value of (\ref{wass_p}). 
	Then,
	\begin{equation*}
	\begin{aligned}
	\hat{J}_N & \le \sup\nolimits_{\mbb{Q} \in \Omega_{\epsilon_N(\alpha_N)}(\hat{\mbb{P}}_N)} \mbb{E}^{\mbb{Q}} [h_{\bbeta^*}(\bz)]\\
	& = \mbb{E}^{\hat{\mbb{Q}}_N} [h_{\bbeta^*}(\bz)] \\
	& \le \mbb{E}^{\mbb{P}^*} [h_{\bbeta^*}(\bz)] + L_{\bbeta^*} W_{\|\cdot\|,1}(\mbb{P}^*, \hat{\mbb{Q}}_N) \\
	& = J^* + L_{\bbeta^*} W_{\|\cdot\|,1}(\mbb{P}^*, \hat{\mbb{Q}}_N),
	\end{aligned}
	\end{equation*}
	where the first step is due to the feasibility of $\bbeta^*$ to (\ref{dro-perf}), and the third step is due to Theorem~\ref{Lip-wass}. In order to prove (\ref{jN-less-than-jstar}), we only need to show that 
	\begin{equation} \label{pstar-qn-zero}
	\mbb{P}^{\infty} \Big \{ \limsup_{N \rightarrow \infty} W_{\|\cdot\|,1}(\mbb{P}^*, \hat{\mbb{Q}}_N) = 0\Big\} = 1.
	\end{equation}
	The triangle inequality of the Wasserstein metric (cf. Theorem~\ref{thm:wmetric})
	ensures that 
	\begin{equation*}
	\begin{aligned}
	W_{\|\cdot\|,1}(\mbb{P}^*, \hat{\mbb{Q}}_N) & \le W_{\|\cdot\|,1}(\mbb{P}^*, \hat{\mbb{P}}_N) + W_{\|\cdot\|,1}(\hat{\mbb{P}}_N, \hat{\mbb{Q}}_N) \\
	& \le W_{\|\cdot\|,1}(\mbb{P}^*, \hat{\mbb{P}}_N) + \epsilon_N(\alpha_N).
	\end{aligned}
	\end{equation*}
	From Theorem \ref{measure-con} we know that 
	\begin{equation*}
	\mbb{P}^N \Big( W_{\|\cdot\|, 1} (\mbb{P}^*, \hat{\mbb{P}}_N) \le \epsilon_N(\alpha_N) \Big) \geq 1 - \alpha_N.
	\end{equation*}
	Therefore, by the Borel-Cantelli Lemma~\citep{francesco1917cantelli},
	\begin{equation*}
	\mbb{P}^{\infty} \bigg( \limsup_{N \rightarrow \infty} \Big \{ W_{\|\cdot\|, 1} (\mbb{P}^*, \hat{\mbb{P}}_N) \le \epsilon_N(\alpha_N) \Big\} \bigg) = 1.
	\end{equation*}
	Since $\lim_{N \rightarrow \infty} \epsilon_N(\alpha_N) = 0$, (\ref{pstar-qn-zero}) follows.
	
	\noindent $(ii)$ We need to show that $\bbeta_0$ achieves the optimal value of (\ref{sto-opt}), i.e.,
	$$\mbb{E}^{\mbb{P}^*} [h_{\bbeta_0}(\bz)] = J^*.$$
	Note that,
	\begin{equation*}
	\begin{aligned}
	J^* & \leq \mbb{E}^{\mbb{P}^*} [h_{\bbeta_0}(\bz)] \\
	& \leq \mbb{E}^{\mbb{P}^*} \big[ \liminf_{N \rightarrow \infty} h_{\hat{\bbeta}_N}(\bz)\big] \\
	& \leq \liminf_{N \rightarrow \infty} \mbb{E}^{\mbb{P}^*} \big[ h_{\hat{\bbeta}_N}(\bz)\big] \\
	& \leq \limsup_{N \rightarrow \infty} \mbb{E}^{\mbb{P}^*} \big[ h_{\hat{\bbeta}_N}(\bz)\big] \\
	& \leq \limsup_{N \rightarrow \infty} \hat{J}_N \\
	& = J^*,
	\end{aligned}
	\end{equation*}
	where the first inequality is due to the feasibility of $\bbeta_0$ to (\ref{sto-opt}), the second inequality follows from the lower semicontinuity of $h$ in $\bbeta$, the third inequality is due to Fatou's lemma, and the fifth inequality holds $\mbb{P}^{\infty}$-almost surely due to (\ref{consist-proof-out-of-perf}). We thus conclude that $\hat{\bbeta}_N$ converges to the optimal solution of (\ref{sto-opt}) a.s.
\end{proof}

\chapter{Distributionally Robust Linear Regression} \label{chapt:dro}
In this section, we introduce the Wasserstein DRO formulation for linear
regression. The focus is to estimate a robustified linear regression plane that is
immunized against potential outliers in the data. Classical approaches, such as
robust regression \citep{huber1964robust,huber1973robust}, remedy this problem by
fitting a weighted least squares that downweights the contribution of atypical data
points. By contrast, the DRO approach mitigates the impact of outliers through
hedging against a family of distributions on the observed data, some of which assign
very low probabilities to the outliers.

\section{The Problem and Related Work}
Consider a linear regression model with response $y \in \mbb{R}$, predictor vector $\bx \in \mbb{R}^{p}$, regression coefficient
$\bbeta^* \in \mbb{R}^{p}$, and error $\eta \in \mbb{R}$:
\begin{equation*}
y = \bx' \bbeta^* + \eta.
\end{equation*}
Given potentially corrupted samples $(\bx_i, y_i), i \in \lb N \rb$, we are interested in obtaining an estimator of $\bbeta^*$ that is robust with respect to the perturbations in the data.
Popular robust estimators include:
\begin{itemize}
	\item {\em Least Absolute Deviation (LAD)}, which minimizes the sum of absolute residuals $\sum_{i=1}^N |y_i - \bx_i' \bbeta|$, and
	\item  M-estimation \citep{huber1964robust, huber1973robust}, which minimizes a symmetric loss function $\rho(\cdot)$ of the residuals in the form $\sum_{i=1}^N \rho(y_i - \bx_i' \bbeta)$, downweighting the influence of samples with large absolute residuals.
\end{itemize}
Several choices for $\rho(\cdot)$ include the Huber function \citep{huber1964robust, huber1973robust}, the Tukey's Biweight function \citep{PJ05}, the logistic function \citep{coleman1980system}, the Talwar function \citep{hinich1975simple}, and the Fair function \citep{fair1974robust}.

Both LAD and M-estimation are not resistant to large deviations in the predictors. For contamination present in the predictor space, high breakdown value methods are required. The breakdown value is the smallest proportion of observations in the dataset that need to be replaced to carry the estimate arbitrarily far away. Examples of high breakdown value methods include the {\em Least Median of Squares (LMS)} \citep{rousseeuw1984least}, which minimizes the median of the absolute residuals, the {\em Least Trimmed Squares (LTS)} \citep{rousseeuw1985multivariate}, which minimizes the sum of the $q$ smallest squared residuals, and S-estimation \citep{rousseeuw1984robust}, which has a higher statistical efficiency than LTS with the same breakdown value. A combination of the high breakdown value method and M-estimation is the MM-estimation \citep{yohai1987high}. It has a higher statistical efficiency than S-estimation. We refer the reader to the book of \cite{PJ05} for an elaborate description of these robust regression methods.

The aforementioned robust estimation procedures focus on modifying the objective
function in a heuristic way with the intent of minimizing the effect of outliers. A
more rigorous line of research explores the underlying stochastic optimization problem that leads to the sample-based estimation procedures. For example, the OLS objective can be viewed as minimizing the expected squared residual under the uniform empirical distribution over the samples. It has been well recognized that optimizing under the empirical distribution yields estimators that are sensitive to perturbations in the data and suffer from overfitting. Instead of equally weighting all the samples as in the empirical distribution, one may wish to include more informative distributions that ``drive out'' the corrupted samples. DRO realizes this through hedging the expected loss against a family of
distributions that includes the true data-generating mechanism with high confidence (cf. Theorem~\ref{measure-con}). Compared to the single distribution-based stochastic optimization, DRO often results in better
out-of-sample performance due to its distributional robustness.

We consider a DRO problem with an ambiguity set containing
distributions that are close to the discrete empirical distribution in
the sense of Wasserstein distance. We adopt the absolute residual loss $|y - \bx' \bbeta|$ for the purpose of enhancing robustness.
By exploiting duality, we relax the Wasserstein DRO formulation to a convex optimization problem which encompasses a
class of regularized regression models, providing new insights into the regularizer, and establishing the connection between the amount of
ambiguity allowed and a regularization penalty term. We provide justifications for the $\ell_1$-loss based DRO learning by establishing novel performance guarantees on both the out-of-sample loss (prediction bias)
and the discrepancy between the estimated
and the true regression coefficients (estimation bias). 
Extensive numerical
results demonstrate the superiority of the DRO model to a host of regression models, in terms of the prediction and estimation accuracies. We also consider the application of the DRO model to outlier detection, and show that it achieves a much higher AUC (Area Under the ROC Curve) than M-estimation \citep{huber1964robust,huber1973robust}.

The rest of this section is organized as follows. In Section~\ref{sec:2-2}, we
introduce the Wasserstein DRO
formulation in a linear regression setting. Section~\ref{sec:2-3}
establishes performance guarantees for the solution to DRO relaxation. The numerical results on the performance of DRO regression are presented in
Section~\ref{sec:2-4}. An application of DRO regression to outlier detection is discussed in Section~\ref{sec:2-5}. We conclude in Section~\ref{sec:2-6}.

\section{The Wasserstein DRO Formulation for Linear Regression} \label{sec:2-2}
We consider an $\ell_1$-loss function $h_{\bbeta}(\bx, y) \triangleq |y - \bx'\bbeta|$, motivated by the observation that the absolute loss function is more forgiving (hence, robust) to large residuals than the squared loss (see Fig.~\ref{absloss}). The Wasserstein DRO problem using the $\ell_1$-loss function is formulated as:
\begin{equation} \label{dro-lr}
\inf\limits_{\bbeta}\sup\limits_{\mbb{Q}\in \Omega}
\mbb{E}^{\mbb{Q}}\big[ |y-\bx'\bbeta|\big], 
\end{equation}
where $\Omega$ is
defined as:
\begin{equation*}
\Omega = \Omega_{\epsilon}^{s,t}(\hat{\mathbb{P}}_N) \triangleq \{\mbb{Q}\in \scrP(\scrZ): W_{s,t}(\mathbb{Q},\ \hat{\mathbb{P}}_N) \le \epsilon\},
\end{equation*}
and $W_{s,t}(\mbb{Q},\ \hat{\mbb{P}}_N)$ is the order-$t$ Wasserstein distance
between $\mbb{Q}$ and $\hat{\mbb{P}}_N$ under a distance metric $s$ (see definition in (\ref{wass_p})), with $\hat{\mbb{P}}_N$ the uniform empirical distribution over $N$ samples. The formulation
in (\ref{dro-lr}) is robust since it minimizes over the regression
coefficients the worst case expected loss, that is, the expected loss
maximized over all probability distributions in the ambiguity set
$\Omega$.
\begin{figure}[h]
	\centering
	\includegraphics[height = 2.1in]{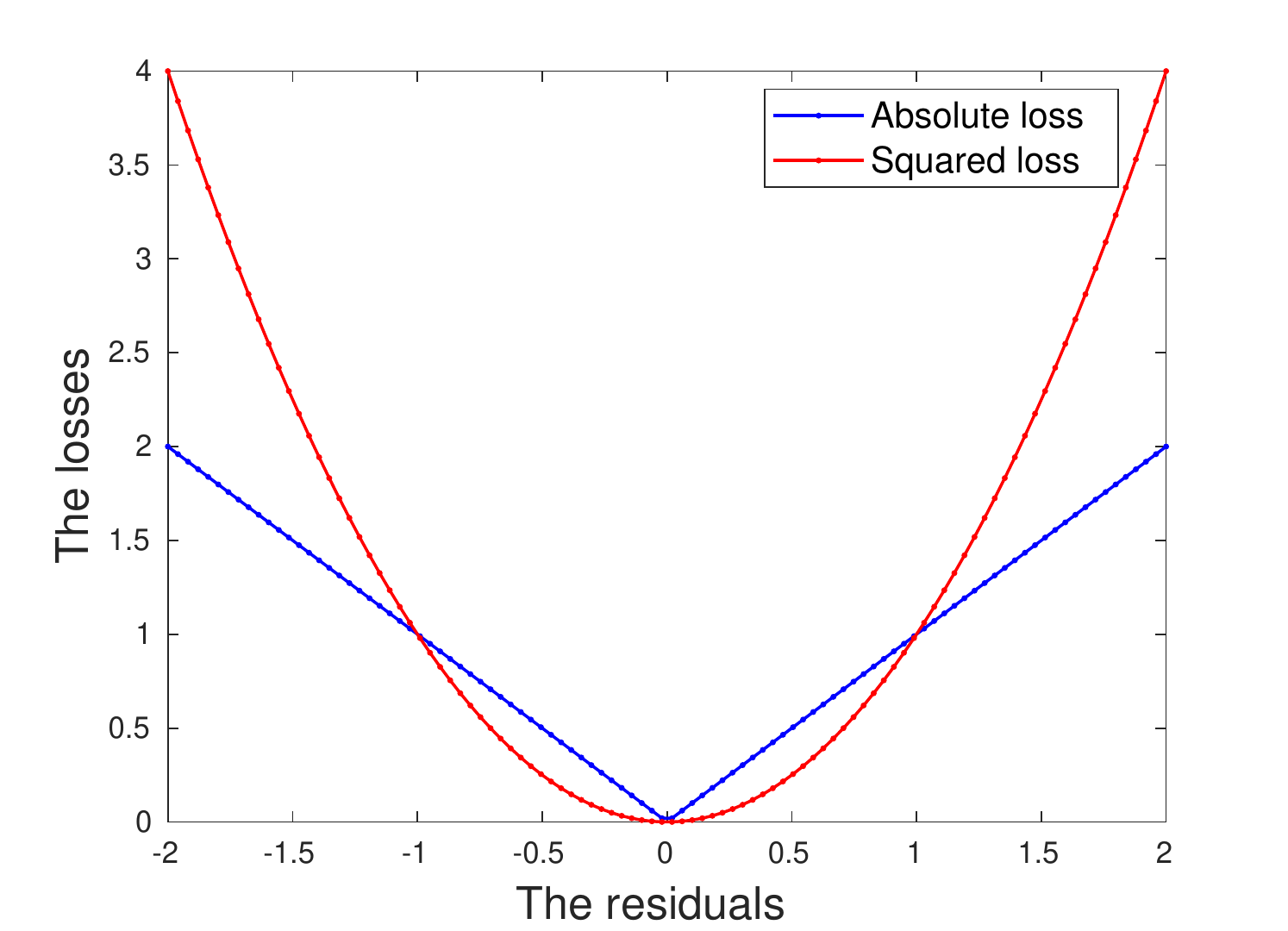}
	\caption{The comparison between $\ell_1$ and $\ell_2$ loss functions.}
	\label{absloss}
\end{figure}

We first decide an appropriate order $t$ for the Wasserstein metric. Based on the discussion in Section~\ref{sec:dual-solver}, it is required that the loss function $h$ has a finite growth rate. Assuming that the metric $s$ is induced by some norm $\|\cdot\|$, the bounded growth rate requirement is expressed as follows:
\begin{equation} \label{gr}
\begin{split}
&  \limsup_{\|(\bx_1, y_1)-(\bx_2, y_2)\| \rightarrow \infty} \frac{|h_{\bbeta}(\bx_1, y_1) - h_{\bbeta}(\bx_2, y_2)|}{\|(\bx_1, y_1)-(\bx_2, y_2)\|^t} \\
\le & \limsup_{\|(\bx_1, y_1)-(\bx_2, y_2)\| \rightarrow \infty} \frac{|y_1 - \bx_1' \bbeta - (y_2 - \bx_2' \bbeta)|}{\|(\bx_1, y_1)-(\bx_2, y_2)\|^t} \\
\le & \limsup_{\|(\bx_1, y_1)-(\bx_2, y_2)\| \rightarrow \infty} \frac{\|(\bx_1, y_1)-(\bx_2, y_2)\| \|(-\bbeta, 1)\|_*}{\|(\bx_1, y_1)-(\bx_2, y_2)\|^t} \\
< & \ \infty,\\
\end{split}
\end{equation}
where $\|\cdot\|_*$ is the dual norm of $\|\cdot\|$, and the second inequality is due
to H\"{o}lder's inequality (cf. Theorem~\ref{holder}). Notice that by taking $t=1$,
(\ref{gr}) is equivalently translated into the condition that $\|(-\bbeta, 1)\|_* <
\infty$, which,  as we will see in Section~\ref{sec:2-3}, is an essential requirement to
guarantee a good generalization performance for the Wasserstein DRO estimator. The
growth rate essentially reveals the underlying metric space used by the Wasserstein
distance. Taking $t>1$ leads to zero growth rate in the limit of (\ref{gr}), which is
not desirable since it removes the Wasserstein ball structure from the formulation
and renders it an optimization problem over a singleton distribution. We thus choose
the order-$1$ Wasserstein metric with $s$ being induced by some norm $\|\cdot\|$ to
define our DRO problem.

Next, we will discuss how to convert (\ref{dro-lr}) into a tractable formulation. Suppose we have $N$ independently and identically distributed
realizations of $(\bx, y)$, denoted by $(\bx_i,y_i), i \in \lb N \rb$. 
Since the loss function is convex in $(\bx, y)$, using the result in Section~\ref{sec:solver-special-case}, the inner supremum of (\ref{dro-lr}) can be relaxed
to the right hand side of (\ref{kk}).  In Theorem~\ref{kappa}, we compute the value of $\kappa$ in (\ref{kk}) for the specific $\ell_1$ loss function we use. 

\begin{thm} \label{kappa} Define $\kappa(\bbeta) = \sup\{\|\btheta\|_*:
	h_{\bbeta}^*(\btheta)<\infty\}$, where $\|\cdot\|_*$ is the dual norm
	of $\|\cdot\|$, and $h_{\bbeta}^*(\cdot)$ is the conjugate
	function of $h_{\bbeta}(\cdot)$.  When the
	loss function is $h_{\bbeta}(\bx, y) = |y-\bx'\bbeta|$, we have $\kappa(\bbeta) =
	\|(-\bbeta, 1)\|_*$.
\end{thm}

\begin{proof}
	We will adopt the notation $\bz \triangleq (\bx, y), \tilde{\bbeta} \triangleq (-\bbeta, 1)$ for ease of analysis. First rewrite $\kappa(\bbeta)$ as:
	\begin{equation*}
	%	\begin{split}
	\kappa(\bbeta) = \sup %& 
	\Bigl\{\|\btheta\|_*: \sup \limits_{\bz:
		\bz'\tilde{\bbeta}\ge 0}\{(\btheta-\tilde{\bbeta})'\bz\} <
	\infty,\ %\\ 
	%&
	\sup\limits_{\bz: \bz'\tilde{\bbeta}\le 0}
	\{(\btheta+\tilde{\bbeta})'\bz\}< \infty\Bigr\}.
	% \\	      \end{split}
	\end{equation*}
	Consider now the two linear optimization problems A and B:	 
	\[ \text{Problem A:} \qquad \begin{array}{rl}
	\max & (\btheta-\tilde{\bbeta})'\bz \\
	\text{s.t.} & \bz'\tilde{\bbeta}\ge 0.
	\end{array}
	\]
	\[ \text{Problem B:} \qquad \begin{array}{rl}
	\max & (\btheta+\tilde{\bbeta})'\bz \\
	\text{s.t.} & \bz'\tilde{\bbeta}\le 0.
	\end{array}
	\]	
	Form the dual problems using dual variables $r_A$ and $r_B$,
	respectively: 
	\[ \text{Dual-A:} \qquad \begin{array}{rl}
	\min & 0\cdot r_A \\
	\text{s.t.} & \tilde{\bbeta} r_A=\btheta-\tilde{\bbeta},\\
	& r_A\le 0,
	\end{array}
	\]
	\[ \text{Dual-B:} \qquad \begin{array}{rl}
	\min & 0\cdot r_B \\
	\text{s.t.} & \tilde{\bbeta} r_B=\btheta+\tilde{\bbeta},\\
	& r_B\ge 0.
	\end{array}
	\]
	We want to find the set of $\btheta$ such that the optimal values of
	problems $A$ and $B$ are finite. Then, Dual-A and Dual-B need to have
	non-empty feasible sets, which implies the following two conditions:
	\begin{gather} 
	\exists \ r_A\le 0, \quad \text{s.t.} \quad \tilde{\bbeta} r_A=
	\btheta-\tilde{\bbeta}, \label{rA} \\
	\exists \ r_B\ge 0, \quad \text{s.t.} \quad \tilde{\bbeta} r_B=
	\btheta+\tilde{\bbeta}.  \label{rB}
	\end{gather}
	For all $i$ with $\tilde{\beta}_i\le 0$, (\ref{rA}) implies
	$\theta_i-\tilde{\beta}_i\ge 0$ and (\ref{rB}) implies $\theta_i\le
	-\tilde{\beta}_i$. On the other hand, for all $j$ with
	$\tilde{\beta}_j\ge 0$, (\ref{rA}) and (\ref{rB}) imply
	$-\tilde{\beta}_j \le \theta_j\le \tilde{\beta}_j$. It is not hard to
	conclude that:
	\[ 
	|\theta_i| \le |\tilde{\beta}_i|, \quad \forall \ i.
	\]	
	It follows, 
	\[ 
	\kappa(\bbeta)=\sup\{\|\btheta\|_*: |\theta_i| \le |\tilde{\beta}_i|,\ \forall i
	\}=
	\|\tilde{\bbeta}\|_* .
	\]
\end{proof}

Due to Theorem~\ref{kappa} and (\ref{kk}), (\ref{dro-lr}) could be formulated as
the following optimization problem:
\begin{equation} \label{qcp}
\inf\limits_{\bbeta} \frac{1}{N}\sum\limits_{i=1}^N|y_i - \bx_i'\bbeta| + \epsilon\|(-\bbeta, 1)\|_*.
\end{equation}
Notice that
(\ref{qcp}) coincides with the regularized LAD models \citep{pollard1991asymptotics, wang2006regularized}, except that it regularizes a variant of the regression coefficient. The regularization term of (\ref{qcp}) is the product of the {\em growth rate} of the loss and the Wasserstein ball radius. A zero growth rate diminishes the effect of the Wasserstein distributional uncertainty set, and the resulting formulation would simply be an empirical loss minimization problem.
The parameter $\epsilon$ controls the conservativeness of the
formulation, whose selection was discussed in Section~\ref{sec:wass-radius}.

The connection between robustness and regularization has been established in several
works. The earliest one may be credited to \cite{LAU97}, which shows that minimizing
the worst-case squared residual within a Frobenius norm-based perturbation set is
equivalent to Tikhonov regularization. In more recent works, using properly selected
uncertainty sets, \cite{xu2009robust} has shown the equivalence between robust
linear regression and the {\em Least Absolute Shrinkage and Selection Operator
	(LASSO)}. \cite{yang2013unified} extends this to more general LASSO-like procedures,
including versions of the grouped LASSO. \cite{bertsimas2017characterization} gives a
comprehensive characterization of the conditions under which robustification and
regularization are equivalent for regression models. For classification problems,
\cite{xu2009robustness} shows the equivalence between the regularized support vector
machines (SVMs) and a robust optimization formulation, by allowing potentially
correlated disturbances in the covariates. \cite{abadeh2015distributionally} considers
a robust version of logistic regression under the assumption that the probability
distributions under consideration lie in a Wasserstein ball. Recently,
\cite{shafieezadeh2017regularization, gao2017wasserstein} has provided a unified
framework for connecting the Wasserstein DRO with regularized learning procedures,
for various regression and classification models.

Formulation (\ref{qcp}) incorporates a class of models whose specific form depends
on the norm space we choose, which could be application-dependent and
practically useful. For example, when the Wasserstein metric $s$ is induced by $\|\cdot\|_2$, (\ref{qcp}) is a convex quadratic problem which can
be solved to optimality very efficiently. Specifically, it could be
converted to:
\begin{equation}  \label{l2qcp}
\begin{aligned}
\min\limits_{\substack{a, \ b_1, \ldots, b_N, \ \bbeta}} \quad & a\epsilon+\frac{1}{N}\sum_{i=1}^N b_i\\
\text{s.t.} \quad & \|\bbeta\|_2^2 + 1 \le a^2,\\
& y_i - \bx_i'\bbeta \le b_i, \  i \in \lb N \rb,\\
& -(y_i - \bx_i'\bbeta) \le b_i, \  i \in \lb N \rb,\\
& a, \ b_i\ge 0, \ i \in \lb N \rb.\\
\end{aligned}
\end{equation}
When the Wasserstein metric is defined using $\|\cdot\|_1$, (\ref{qcp}) is a linear programming problem:
\begin{equation}  \label{l1qcp}
\begin{aligned}
\min\limits_{\substack{a, \ b_1, \ldots, b_N, \ \bbeta}} \quad & a\epsilon+\frac{1}{N}\sum_{i=1}^N b_i\\
\text{s.t.} \quad & a \ge \bbeta'\mathbf{e}_i, \ i \in \lb p \rb,\\
& a \ge -\bbeta'\mathbf{e}_i, \ i \in \lb p \rb,\\
& y_i - \bx_i'\bbeta \le b_i, \  i \in \lb N \rb,\\
& -(y_i - \bx_i'\bbeta) \le b_i, \  i \in \lb N \rb,\\
& a \ge 1,\\
& b_i\ge 0, \ i \in \lb N \rb.
\end{aligned}
\end{equation}
More generally, when the coordinates of $(\bx, y)$ differ
from each other substantially, a properly chosen, positive definite
weight matrix $\bM \in \mbb{R}^{(p+1) \times (p+1)}$ could scale correspondingly different coordinates of
$(\bx, y)$ by using the $\bM$-weighted norm:
\begin{equation*}
\|(\bx, y)\|_{\bM} = \sqrt{(\bx, y)' \bM (\bx, y)}.
\end{equation*}     
It
can be shown that (\ref{qcp}) in this case becomes:
\begin{equation} \label{wqcp}
\inf\limits_{\bbeta} \frac{1}{N}\sum\limits_{i=1}^N|y_i - \bx_i'\bbeta| + \epsilon \sqrt{(-\bbeta, 1)' \bM^{-1} (-\bbeta, 1)}.
\end{equation}

We would like to highlight
several novel viewpoints that are brought by the Wasserstein DRO framework and justify the value and novelty of
(\ref{qcp}). First, (\ref{qcp}) is obtained as an
outcome of a fundamental DRO formulation, which enables new interpretations of the regularizer from the standpoint of distributional robustness, and provides rigorous theoretical foundation on why the $\ell_2$-regularizer prevents overfitting to the training data. The regularizer
could be seen as a control over the amount of ambiguity in the data 
and reveals the reliability of the
contaminated samples. Second, the geometry of the Wasserstein ball is embedded in the regularization term, which penalizes the regression coefficient on the dual Wasserstein space, with the magnitude of penalty being the radius of the ball. This offers an
intuitive interpretation and provides guidance on how to set the regularization
coefficient. Moreover, different from the traditional regularized LAD models that
directly penalize the regression coefficient $\bbeta$, (\ref{qcp}) regularizes the
vector $(-\bbeta, 1)$, where the $1$ takes into account the transportation cost along
the $y$ direction. Penalizing only $\bbeta$ corresponds to an infinite transportation
cost along $y$. (\ref{qcp}) is more general in this sense, and establishes the
connection between the metric space on the data and the form of the regularizer.  

\section{Performance Guarantees for the DRO Estimator}  \label{sec:2-3}
Having obtained a tractable reformulation for the Wasserstein DRO problem, we next establish guarantees on the predictive power and estimation quality for the solution to (\ref{qcp}). Two types of results will be presented in this section, one of which bounds the prediction bias of the estimator on new, future data (given in Section~\ref{out}). The other one bounds
the discrepancy between the estimated and true regression planes (estimation bias), and is given
in Section~\ref{limit}.

\subsection{Out-of-Sample Performance} \label{out} 
In this subsection, we investigate generalization characteristics of the
solution to (\ref{qcp}), which involves measuring the error generated by
the DRO estimator on a new random sample $(\bx, y)$. We would like to obtain
estimates that not only explain the observed samples well, but, more
importantly, possess strong generalization abilities. The derivation is
mainly based on {\em Rademacher complexity} (see \cite{Peter02}), which is a
measurement of the complexity of a class of functions. We would like to emphasize the applicability of such a proof technique to general loss functions, as long as their empirical Rademacher complexity could be bounded. The bound we derive for the prediction bias
depends on both the sample average loss (the training error) and the dual norm of the
regression coefficient (the regularizer), which corroborates the validity and
necessity of the regularized formulation. Moreover, the generalization result also
builds a connection between the loss function and the form of the regularizer via the
Rademacher complexity, which enables new insights into the regularization term and explains the commonly observed good out-of-sample performance of regularized regression in a rigorous way. 

Suppose the data $(\bx, y)$ is drawn from the probability distribution $\mbb{P}^*$.
We first make several mild assumptions that are needed for the generalization result.

\begin{ass} \label{a1} $\|(\bx, y)\|\le R, \ \text{a.s. under $\mbb{P}^*$}$.
\end{ass}

\begin{ass} \label{a2} $\sup_{\bbeta}\|(-\bbeta, 1)\|_*=\bar{B}$.
\end{ass}

Under these two assumptions, the absolute loss could be bounded via H\"{o}lder's inequality.
\begin{lem} \label{l1}
	For every feasible $\bbeta$, it follows that,
	\begin{equation*}
	|y - \bx'\bbeta|\le \bar{B}R, \quad \text{a.s. under $\mbb{P}^*$}.
	\end{equation*}
\end{lem} 

With the above result, the idea is to bound the generalization error using the empirical {\em Rademacher complexity} of
the following class of loss functions: 
\begin{equation*}
\scrH= \Big \{(\bx, y) \ra h_{\bbeta}(\bx, y): h_{\bbeta}(\bx, y)= |y - \bx'\bbeta| \Big\}. 
\end{equation*}
We need to show that the empirical Rademacher complexity of
$\scrH$, denoted by $\scrR_N(\scrH)$ and defined as:
$$\scrR_N(\scrH) \triangleq \mathbb{E}\Biggl[\sup\limits_{h \in \scrH}
\frac{2}{N}\biggl|\sum\limits_{i=1}^N
\sigma_ih_{\bbeta}(\bx_i, y_i)\biggr|\Biggl|(\bx_1, y_1),
\ldots,(\bx_N, y_N)\Biggr],$$
is upper bounded, where $\sigma_1, \ldots, \sigma_N$ are i.i.d.\ uniform random
variables on $\{1, -1\}$, and $(\bx_i, y_i), i \in \lb N \rb$, are $N$ observed realizations of $(\bx, y)$.
The following result, similar to Lemma 3 in
\cite{Dim14}, provides a bound that is inversely proportional to the square root of the sample size.
\begin{lem} \label{radcom}
	\begin{equation*}
	\scrR_N(\scrH)\le \frac{2\bar{B}R}{\sqrt{N}}.
	\end{equation*}
\end{lem}

\begin{proof}
	Suppose that $\sigma_1, \ldots, \sigma_N$ are i.i.d.\ uniform random
	variables on $\{1, -1\}$. Then, by the definition of the Rademacher
	complexity and Lemma~\ref{l1},
	\begin{equation*}
	\begin{split}
	\scrR_N(\scrH)
	& = \mathbb{E}\Biggl[\sup\limits_{h \in \scrH}
	\frac{2}{N}\biggl|\sum\limits_{i=1}^N
	\sigma_ih_{\bbeta}(\bx_i, y_i)\biggr|\Biggl|(\bx_1, y_1),
	\ldots,(\bx_N, y_N)\Biggr]\\ 
	& \le \mathbb{E}\Biggl[
	\frac{2}{N}\biggl|\sum\limits_{i=1}^N
	\sigma_i \bar{B}R \biggr|\Biggr]\\ 
	& = \mathbb{E}\Biggl[
	\frac{2\bar{B}R}{N}\biggl|\sum\limits_{i=1}^N
	\sigma_i  \biggr|\Biggr]\\ 
	& = \frac{2\bar{B}R}{N}\mathbb{E}\Biggl[\Biggl|\sum\limits_{i=1}^N
	\sigma_i\Biggr|\Biggr]\\ 
	& \le 
	\frac{2\bar{B}R}{N}\mathbb{E} \Biggl[\sqrt{\sum\limits_{i=1}^N\sigma_i^2}\
	\Biggr]\\  
	& = \frac{2\bar{B}R}{\sqrt{N}}.
	\end{split}
	\end{equation*}
\end{proof}

Let $\hat{\bbeta}$ be an optimal solution to (\ref{qcp}), obtained using
the samples $(\bx_i, y_i)$, $i \in \lb N \rb$. Suppose we draw a new i.i.d.\
sample $(\bx,y)$. In Theorem~\ref{t2} we establish bounds on the
error $|y - \bx'\hat{\bbeta}|$.
\begin{thm} \label{t2} Under Assumptions~\ref{a1} and \ref{a2}, for any
	$0<\delta<1$, with probability at least $1-\delta$ with respect to the
	sampling,
	\begin{equation} \label{exp}
	\mathbb{E}^{\mbb{P}^*}[|y - \bx'\hat{\bbeta}|]\le
	\frac{1}{N}\sum_{i=1}^N
	|y_i - \bx_i'\hat{\bbeta}|+\frac{2\bar{B}R}{\sqrt{N}}+ 
	\bar{B}R\sqrt{\frac{8\log(2/\delta)}{N}}\ ,
	\end{equation}
	and for any $\zeta>\frac{2\bar{B}R}{\sqrt{N}}+
	\bar{B}R\sqrt{\frac{8\log(2/\delta)}{N}}$,
	\begin{multline} \label{prob}
	\mathbb{P}\Bigl(|y - \bx'\hat{\bbeta}| \ge
	\frac{1}{N}\sum_{i=1}^N |y_i - \bx_i'\hat{\bbeta}|+\zeta\Bigr) 
	\le {} \\
	\frac{\frac{1}{N}\sum_{i=1}^N
		|y_i - \bx_i'\hat{\bbeta}|+\frac{2\bar{B}R}{\sqrt{N}}+ 
		\bar{B}R\sqrt{\frac{8\log(2/\delta)}{N}}}{\frac{1}{N}\sum_{i=1}^N
		|y_i - \bx_i'\hat{\bbeta}|+\zeta}. 
	\end{multline}
\end{thm}

\begin{proof}
	We use Theorem 8 in \cite{Peter02}, which we state for convenience as follows. 
	\begin{thm}[Theorem 8 in \cite{Peter02}] \label{Peter}
		Consider a loss function $L: \scrY \times \scrA \rightarrow [0,1]$ and a dominating cost function $\phi: \scrY \times \scrA \rightarrow [0,1]$. Let $\scrF$ be a class of functions mapping from $\scrX$ to $\scrA$ and let $(\bx_i, y_i)_{i=1}^N$ be independently selected according to the probability measure $\mbb{P}^*$. Then, for any integer $N$ and any $0<\delta<1$, with probability at least $1-\delta$ over samples of length $N$, every $f$ in $\scrF$ satisfies
		\begin{equation*}
		\mbb{E}^{\mbb{P}^*} \big [L (y, f(\bx)) \big] \le \frac{1}{N} \sum_{i=1}^N \phi(y_i, f(\bx_i)) + R_N(\tilde{\phi} \circ \scrF) + \sqrt{\frac{8 \log(2/\delta)}{N}},
		\end{equation*}
		where $\tilde{\phi} \circ \scrF = \{(\bx, y) \ra \phi(y, f(\bx)) - \phi(y,0): f \in \scrF \}$.
	\end{thm} 
	We set the following
	correspondences with the notation used in Theorem~\ref{Peter}: $f(\bx) = \bx'\bbeta$, and $L(y, f(\bx)) =
	\phi(y, f(\bx)) = |y - f(\bx)|$. This yields the
	bound (\ref{exp}) on the expected loss.  For Eq. (\ref{prob}), we
	apply Markov's inequality to obtain:
	\begin{multline*}
	\mathbb{P}\Bigl(|y - \bx'\hat{\bbeta}|  \ge  \frac{1}{N}\sum\limits_{i=1}^N |y_i - \bx_i'\hat{\bbeta}|+\zeta \Bigr)
	\le \frac{\mathbb{E}[|y - \bx'\hat{\bbeta}|]}{\frac{1}{N}\sum_{i=1}^N |y_i - \bx_i'\hat{\bbeta}|+\zeta} {}\\
	\le \frac{\frac{1}{N}\sum_{i=1}^N |y_i - \bx_i'\hat{\bbeta}|+\frac{2\bar{B}R}{\sqrt{N}}+
		\bar{B}R\sqrt{\frac{8\log(2/\delta)}{N}}}{\frac{1}{N}\sum_{i=1}^N |y_i - \bx_i'\hat{\bbeta}|+\zeta}.
	\end{multline*}
\end{proof}

There are two probability measures in the statement of Theorem~\ref{t2}. One is
related to the new data $(\bx, y)$, while the other is 
related to the samples $(\bx_i, y_i),i \in \lb N \rb$. The expectation in
(\ref{exp}) (and the probability in (\ref{prob})) is taken w.r.t. the new
data $(\bx, y)$. For a given set of samples, (\ref{exp}) (and (\ref{prob}))
holds with probability at least $1-\delta$ w.r.t.\ the measure of
samples. Theorem~\ref{t2} essentially says that given typical samples,
the expected loss on new data using the Wasserstein DRO estimator could
be bounded above by the average sample loss plus extra terms that depend on the supremum of $\|(-\bbeta, 1)\|_*$ (the regularizer), and are
proportional to $1/\sqrt{N}$. This result validates the dual norm-based regularized regression from the perspective of generalization ability, and could be generalized to any bounded loss function. It also provides implications on the form of the regularizer. For example, if given an $\ell_2$-loss function, the dependency on $\bar{B}$ for the generalization error bound will be of the form $\bar{B}^2$, which suggests using $\|(-\bbeta, 1)\|_*^2$ as a regularizer, reducing to a variant of ridge regression \citep{hoerl1970ridge} for the $\ell_2$-norm-induced Wasserstein metric. 

We also note that the upper bounds in (\ref{exp}) and (\ref{prob}) do not depend
on the dimension of $(\bx, y)$. This dimensionality-free characteristic
implies direct applicability of the Wasserstein approach to
high-dimensional settings and is particularly useful in many real
applications where, potentially, hundreds of features may be present. Theorem~\ref{t2} also provides guidance on the number of
samples that are needed to achieve satisfactory out-of-sample performance. 
\begin{col} \label{samplesize1}
	Suppose $\hat{\bbeta}$ is the optimal solution to (\ref{qcp}). For a
	fixed confidence level $\delta$ and some threshold parameter $\tau \ge
	0$, if the sample size $N$ satisfies
	\begin{equation} \label{N1}
	N\ge \biggl[\frac{2(1+\sqrt{2\log(2/\delta)}\ )}{\tau}\biggr]^2,
	\end{equation}
	then the percentage difference between the expected
	absolute loss on new data and the sample average loss is less than
	$\tau$, that is,
	\begin{equation*}
	\frac{\mathbb{E}[|y - \bx'\hat{\bbeta}|]-\frac{1}{N}\sum_{i=1}^N |y_i - \bx_i'\hat{\bbeta}|}{\bar{B}R}\le \tau.
	\end{equation*} 
\end{col}

\begin{proof}
	The percentage difference requirement can be translated into:
	\begin{equation*}
	\frac{2}{\sqrt{N}}+
	\sqrt{\frac{8\log(2/\delta)}{N}} \le \tau,
	\end{equation*}
	from which (\ref{N1}) can be easily derived.
\end{proof}

\begin{col} \label{samplesize2}
	Suppose $\hat{\bbeta}$ is the optimal solution to (\ref{qcp}). For a
	fixed confidence level $\delta$, some $\tau \in (0, 1)$ and $\gamma
	\ge 0$ such that $\tau\gamma+\tau-1>0$, if the sample size $N$ satisfies
	\begin{equation} \label{N2}
	N\ge 
	\biggl[\frac{2(1+\sqrt{2\log(2/\delta)}\ )}{\tau\gamma+\tau-1}\biggr]^2,
	\end{equation}
	then, 
	\begin{equation*}
	\mathbb{P}\Bigl(\frac{|y - \bx'\hat{\bbeta}|-\frac{1}{N}\sum_{i=1}^N |y_i - \bx_i'\hat{\bbeta}|}{\bar{B}R}\ge \gamma \Bigr)\le \tau.
	\end{equation*}
\end{col}

\begin{proof}
	Based on Theorem 3.3, we just need the following inequality to hold:
	\begin{equation*}
	\frac{\frac{1}{N}\sum_{i=1}^N |y_i - \bx_i'\hat{\bbeta}|+\frac{2\bar{B}R}{\sqrt{N}}+
		\bar{B}R\sqrt{\frac{8\log(2/\delta)}{N}}}{\frac{1}{N}\sum_{i=1}^N |y_i - \bx_i'\hat{\bbeta}|+\gamma\bar{B}R}\le \tau,
	\end{equation*}
	which is equivalent to:
	\begin{equation} \label{A1}
	\frac{\gamma\bar{B}R-\frac{2\bar{B}R}{\sqrt{N}}-
		\bar{B}R\sqrt{\frac{8\log(2/\delta)}{N}}}{\frac{1}{N}\sum_{i=1}^N |y_i - \bx_i'\hat{\bbeta}|+\gamma\bar{B}R}\ge 1-\tau. 
	\end{equation}
	We cannot obtain a lower bound for $N$ by directly solving (\ref{A1}) since $N$ appears in a summation operator. A proper relaxation to (\ref{A1}) is: 
	\begin{equation} \label{A2}
	\frac{\gamma-\frac{2}{\sqrt{N}}-
		\sqrt{\frac{8\log(2/\delta)}{N}}}{1+\gamma}\ge 1-\tau, 
	\end{equation}
	due to the fact that $\frac{1}{N}\sum_{i=1}^N
	|y_i - \bx_i'\hat{\bbeta}|\le \bar{B}R$. By solving (\ref{A2}), we
	obtain (\ref{N2}).
\end{proof}

In Corollaries~\ref{samplesize1} and \ref{samplesize2}, the sample size is inversely proportional to both
$\delta$ and $\tau$, which is reasonable
since the more confident we want to be, the more samples we
need. Moreover, the smaller $\tau$ is, the stricter a requirement we
impose on the performance, and thus more samples are needed.

\subsection{Discrepancy between Estimated and True Regression Planes} \label{limit} 
In addition to the generalization performance, we are also interested in the accuracy of the estimator. In this subsection, we seek to bound the difference between the estimated
and true regression coefficients, under a certain distributional assumption on $(\bx, y)$. Throughout this
subsection we will use $\hat{\bbeta}$ to denote the estimated regression
coefficients, obtained as an optimal solution to (\ref{infty}), and
$\bbeta^*$ for the true (unknown) regression coefficients.
The bound we will derive turns out to be related to the uncertainty in the
data $(\bx, y)$, and the geometric structure of the
true regression coefficients. 

To facilitate the analysis, we will use the following equivalent form of Problem (\ref{qcp}):
\begin{equation} \label{infty}
\begin{aligned}
\min_{\bbeta} & \quad \|(-\bbeta, 1)\|_* \\
\text{s.t.} & \quad \|(-\bbeta, 1)'\bZ\|_1\le \gamma_N,\\
\end{aligned}
\end{equation}
where $\bZ=[(\bx_1, y_1), \ldots, (\bx_N, y_N)]$ is the matrix with columns $(\bx_i, y_i)$, 
$i \in \lb N \rb$, and $\gamma_N$ is some exogenous parameter related to
$\epsilon$. One can show that for properly chosen $\gamma_N$, (\ref{infty}) produces the same solution with (\ref{qcp}) \citep{bertsekas1999nonlinear}. (\ref{infty}) is similar to (11) in
\cite{chen2016alternating}, with the difference lying in that we impose
a constraint on the error instead of the gradient, and we consider a more general notion of norm on the coefficient. On the other hand,
due to their similarity, we will follow the line of development in
\cite{chen2016alternating}. Still, our analysis is self-contained and
the bound we obtain is in a different form, which provides meaningful
insights into our specific problem. We list below the relevant definitions and assumptions that are needed to bound
the estimation error.

\begin{defi}[Sub-Gaussian random variable]
	A random variable $z$ is sub-Gaussian if the $\psi_2$-norm defined
	below is finite, i.e.,
	\begin{equation*}
	\vertiii{z}_{\psi_2}\triangleq \sup_{q \ge 1}
	\frac{(\mbb{E}|z|^q)^{1/q}}{\sqrt{q}} < +\infty. 
	\end{equation*}
\end{defi}
We do not require sub-Gaussian variables to have zero mean values. It is though worth noting that the $\psi_2$-norm $\vertiii{z}_{\psi_2}$ depends on the mean $\mbb{E}(z)$. An equivalent property for sub-Gaussian random variables is that their
tail distribution decays as fast as a Gaussian, namely, 
\[ 
\mbb{P}(|z - \mbb{E}(z)|\geq t) \leq 2 \exp\{-t^2/C^2\},\quad \forall t \geq 0, 
\] 
for some constant $C$. 

A random vector $\bz \in \mbb{R}^{p+1}$ is sub-Gaussian if $\bz'\bu$ is
sub-Gaussian for any $\bu \in \mbb{R}^{p+1}$. The $\psi_2$-norm of a vector
$\bz$ is defined as: 
\[
\vertiii{\bz}_{\psi_2} \triangleq \sup\limits_{\bu\in
	\scrS^{p+1}}\vertiii{\bz'\bu}_{\psi_2},
\] 
where $\scrS^{p+1}$ denotes the unit sphere in the $(p+1)$-dimensional
Euclidean space. For the properties of sub-Gaussian random
variables/vectors, please refer to \cite{RV17}.
\begin{defi}[Gaussian width]
	For any set $\scrA \subseteq \mbb{R}^{p+1}$, its Gaussian width is defined as:
	\begin{equation} \label{gw}
	w(\scrA) \triangleq \mbb{E}\Bigl[\sup_{\bu \in \scrA} \bu'\bg\Bigr],
	\end{equation}
	where $\bg\sim {\cal N}(\bzero,\bI)$ is a $(p+1)$-dimensional standard
	Gaussian random vector.
\end{defi}

%\begin{ass} \label{2norm}
%	The $\ell_2$ norm of $(-\bbeta, 1)$ is bounded
%	above within the feasible region, namely,
%	\begin{equation*}
%	\sup\limits_{\bbeta\in \scrB}\|(-\bbeta, 1)\|_2=\bar{B}_2.
%	\end{equation*}
%\end{ass}

\begin{ass}[Restricted Eigenvalue Condition] \label{RE} 
	For some set $\scrA(\bbeta^*)=\text{cone}\{\bv: \|(-\bbeta^*, 1) + \bv\|_* \le\|(-\bbeta^*, 1)\|_*\}\cap \scrS^{p+1}$ and some positive scalar $\underline{\alpha}$, where $\scrS^{p+1}$ is the unit sphere in the $(p+1)$-dimensional Euclidean space,
	\begin{equation*}
	\inf\limits_{\bv\in \scrA(\bbeta^*)}\bv'\bZ\bZ'\bv\ge\underline{\alpha}.
	\end{equation*}
\end{ass}

\begin{ass} \label{adm}
	The true coefficient $\bbeta^*$ is a feasible solution to (\ref{infty}), i.e.,
	\begin{equation*}
	\|\bZ'(-\bbeta^*, 1)\|_1\le \gamma_N. 
	\end{equation*}
\end{ass}

\begin{ass} \label{subgaussian}
	$(\bx, y)$ is a centered sub-Gaussian random vector, i.e., it has
	zero mean and satisfies the following condition: 
	\begin{equation*}
	\vertiii{(\bx, y)}_{\psi_2}=\sup\limits_{\bu\in
		\scrS^{p+1}}\vertiii{(\bx, y)'\bu}_{\psi_2}\le \mu. 
	\end{equation*}
\end{ass}

\begin{ass} \label{eigen}
	The covariance matrix of $(\bx, y)$ has bounded positive
	eigenvalues. Set $\bGamma=\mathbb{E}[(\bx, y) (\bx, y)']$; then, 
	\begin{equation*}
	0<\lambda_{\text{min}} \triangleq \lambda_{\text{min}}(\bGamma)\le\lambda_{\text{max}}(\bGamma)\triangleq \lambda_{\text{max}}<\infty.
	\end{equation*}
\end{ass}

Notice that both $\underline{\alpha}$ in Assumption~\ref{RE} and $\gamma_N$ in Assumption~\ref{adm} are related to the random observation matrix $\bZ$. A probabilistic description for these two quantities will be provided later. We next present a preliminary result, similar to Lemma 2 in
\cite{chen2016alternating}, that bounds the $\ell_2$-norm of the estimation
bias in terms of a quantity that is related to the geometric structure
of the true coefficients. This result gives a rough idea on the factors that affect the estimation error. The bound derived in Theorem~\ref{mainresult} is 
crude in the sense that it is a function of several random parameters that are related to the random observation matrix $\bZ$. This randomness will be described in a probabilistic way in the subsequent analysis. 

\begin{thm} \label{mainresult} Suppose the true regression coefficient
	vector is $\bbeta^*$ and the solution to (\ref{infty}) is
	$\hat{\bbeta}$. For the set $\scrA(\bbeta^*)=\text{cone}\{\bv: \|(-\bbeta^*, 1) + \bv\|_* \le\|(-\bbeta^*, 1)\|_*\}\cap \scrS^{p+1}$,
	under Assumptions~\ref{a1}, \ref{RE}, and \ref{adm}, we have:
	\begin{equation} \label{l2norm}
	\|\hat{\bbeta}-\bbeta^*\|_2\le
	\frac{2R\gamma_N}{\underline{\alpha}}\Psi(\bbeta^*), 
	\end{equation}
	where $\Psi(\bbeta^*)=\sup_{\bv\in \scrA(\bbeta^*)}\|\bv\|_*$.
\end{thm}

\begin{proof}
	For ease of exposition, we will adopt the notation $\bz \triangleq (\bx, y), \ \bz_i \triangleq (\bx_i, y_i), \ \tilde{\bbeta} \triangleq (-\bbeta, 1), \ \tilde{\bbeta}_{\text{est}} \triangleq (-\hat{\bbeta}, 1), \ \tilde{\bbeta}_{\text{true}} \triangleq (-\bbeta^*, 1)$.
	
	Since both $\hat{\bbeta}$ and $\bbeta^*$ are feasible (the latter due to
	Assumption~\ref{adm}), we have:
	\begin{equation*}
	\begin{split}
	\|\bZ'\tilde{\bbeta}_{\text{est}}\|_1 & \le \gamma_N, \\
	\|\bZ'\tilde{\bbeta}_{\text{true}}\|_1 & \le \gamma_N,
	\end{split}
	\end{equation*}
	from which we derive that $\|\bZ'(\tilde{\bbeta}_{\text{est}} - \tilde{\bbeta}_{\text{true}})\|_1\le
	2\gamma_N$. 
	Since $\hat{\bbeta}$ is an optimal solution to (\ref{infty})
	and $\bbeta^*$ a feasible solution, it follows that
	$\|\tilde{\bbeta}_{\text{est}}\|_* \le \|\tilde{\bbeta}_{\text{true}}\|_*$. This implies that
	$\bnu=\tilde{\bbeta}_{\text{est}} - \tilde{\bbeta}_{\text{true}}$ satisfies the condition
	$\|\tilde{\bbeta}_{\text{true}}+\bv\|_* \le\|\tilde{\bbeta}_{\text{true}}\|_*$ included in the
	definition of $\scrA(\bbeta^*)$ and, furthermore, $(\tilde{\bbeta}_{\text{est}} -
	\tilde{\bbeta}_{\text{true}})/\|\tilde{\bbeta}_{\text{est}} - \tilde{\bbeta}_{\text{true}}\|_2 \in \scrA(\bbeta^*)$. Together
	with Assumption~\ref{RE}, this yields
	\begin{equation} \label{re1}
	(\tilde{\bbeta}_{\text{est}}-\tilde{\bbeta}_{\text{true}})'\bZ \bZ'(\tilde{\bbeta}_{\text{est}}-\tilde{\bbeta}_{\text{true}})\ge
	\underline{\alpha}\|\tilde{\bbeta}_{\text{est}}-\tilde{\bbeta}_{\text{true}}\|_2^2. 
	\end{equation}
	On the other hand, from H\"{o}lder's inequality:
	\begin{equation} \label{re2}
	\begin{split}
	& \quad \ (\tilde{\bbeta}_{\text{est}}-\tilde{\bbeta}_{\text{true}})'\bZ \bZ'(\tilde{\bbeta}_{\text{est}}-\tilde{\bbeta}_{\text{true}}) \\
	& \le \|\bZ'(\tilde{\bbeta}_{\text{est}}-\tilde{\bbeta}_{\text{true}})\|_1\|\bZ'(\tilde{\bbeta}_{\text{est}}-\tilde{\bbeta}_{\text{true}})\|_\infty\\
	& \le 2\gamma_N \max_i |\bz_i'(\tilde{\bbeta}_{\text{est}}-\tilde{\bbeta}_{\text{true}})|\\
	& \le 2\gamma_N \max_i \|\tilde{\bbeta}_{\text{est}}-\tilde{\bbeta}_{\text{true}}\|_* \|\bz_i\|\\
	& \le 2R\gamma_N \|\tilde{\bbeta}_{\text{est}}-\tilde{\bbeta}_{\text{true}}\|_*.
	\end{split}
	\end{equation}
	Combining (\ref{re1}) and (\ref{re2}), we have:
	\begin{equation*}
	\begin{split}
	\|\hat{\bbeta} - \bbeta^*\|_2 & = \|\tilde{\bbeta}_{\text{est}}-\tilde{\bbeta}_{\text{true}}\|_2\\
	& \le
	\frac{2R\gamma_N}{\underline{\alpha}}\frac
	{\|\tilde{\bbeta}_{\text{est}}-\tilde{\bbeta}_{\text{true}}\|_*}{\|\tilde{\bbeta}_{\text{est}}-\tilde{\bbeta}_{\text{true}}\|_2}\\ 
	& \le \frac{2R\gamma_N}{\underline{\alpha}}\Psi(\bbeta^*), 
	\end{split}
	\end{equation*}
	where the last step follows from the fact that $(\tilde{\bbeta}_{\text{est}} -
	\tilde{\bbeta}_{\text{true}})/\|\tilde{\bbeta}_{\text{est}} - \tilde{\bbeta}_{\text{true}}\|_2 \in \scrA(\bbeta^*)$.
\end{proof}

As mentioned earlier, (\ref{l2norm}) provides a random upper bound, revealed in $\underline{\alpha}$ and $\gamma_N$, that depends on the randomness in $\bZ$. We therefore would like to replace these two parameters by non-random quantities. The quantity $\underline{\alpha}$ acts as the minimum eigenvalue of the matrix $\bZ \bZ'$ restricted to a subspace of $\mbb{R}^{p+1}$, and thus a proper substitute should be related to the minimum eigenvalue of the covariance matrix of $(\bx, y)$, i.e., the $\bGamma$ matrix (cf. Assumption~\ref{eigen}), given that $(\bx, y)$ is zero mean. See Lemmata~\ref{alphalem}, \ref{gaussianwidthlem} and \ref{alphacol} for the derivation.

\begin{lem} \label{alphalem} Consider $\scrA_{\bGamma}=\{\bw\in
	\scrS^{p+1}: \bGamma^{-1/2}\bw \in \text{cone}(\scrA(\bbeta^*))\}$,
	where $\scrA(\bbeta^*)$ is defined as in Theorem~\ref{mainresult}, and
	$\bGamma=\mathbb{E}[(\bx, y)(\bx, y)']$. Under Assumptions~\ref{subgaussian} and
	\ref{eigen}, when the sample size $N\ge C_1\bar{\mu}^4
	(w(\scrA_{\bGamma}))^2$, where
	$\bar{\mu}=\mu\sqrt{\frac{1}{\lambda_{\text{min}}}}$, and
	$w(\scrA_{\bGamma})$ is the Gaussian width of $\scrA_{\bGamma}$, with
	probability at least $1-\exp(-C_2N/\bar{\mu}^4)$, we have
	\begin{equation*}
	\bv'\bZ\bZ'\bv\ge \frac{N}{2}\bv'\bGamma\bv, \quad \forall \bv\in
	\scrA(\bbeta^*), 
	\end{equation*}
	where $C_1$ and $C_2$ are positive constants.
\end{lem}

\begin{proof}
	Define $\hat{\bGamma}=\frac{1}{N}\sum_{i=1}^N\bz_i\bz_i'$. Consider the
	set of functions $\scrF=\{f_{\bw}(\bz)=\bz'\bGamma^{-1/2}\bw, \ \bw\in
	\scrA_{\bGamma}\}$. Then, for any $f_{\bw}\in \scrF$, 
	\begin{equation*}
	\begin{split}
	\mathbb{E}[f_{\bw}^2]&
	=\mathbb{E}[\bw'\bGamma^{-1/2}\bz\bz'\bGamma^{-1/2}\bw]\\ 
	& = \bw'\bGamma^{-1/2}\mathbb{E}[\bz\bz']\bGamma^{-1/2}\bw\\
	& = \bw'\bw\\
	& = 1,
	\end{split}
	\end{equation*}
	where we used $\bGamma=\mathbb{E}[\bz\bz']$ and the fact that $\bw\in
	\scrA_{\bGamma}$. 	
	
	For any $f_{\bw}\in \scrF$ we have
	\begin{align*}
	\vertiii{f_{\bw}}_{\psi_2} & =\vertiii{\bz'\bGamma^{-1/2}\bw}_{\psi_2}\\
	& = \vertiii{\bz'\bGamma^{-1/2}\bw}_{\psi_2}
	\frac{\|\bGamma^{-1/2}\bw\|_2}{\|\bGamma^{-1/2}\bw\|_2}\\ 
	& = \vertiii{\bz' \frac{\bGamma^{-1/2}\bw}{\|\bGamma^{-1/2}\bw\|_2}}_{\psi_2}
	\|\bGamma^{-1/2}\bw\|_2\\
	& \le \mu \sqrt{\bw'\bGamma^{-1}\bw}\\
	& \le \mu \sqrt{\frac{1}{\lambda_{\text{min}}} \|\bw\|_2^2}\\
	& = \mu \sqrt{\frac{1}{\lambda_{\text{min}}}} = \bar{\mu},
	\end{align*}
	where the first inequality used Assumption~\ref{subgaussian} and the
	second inequality used Assumption~\ref{eigen}. 
	
	Applying Theorem D from \cite{mendelson2007reconstruction}, for any
	$\theta>0$ and when 
	\[
	\tilde{C}_1 \bar{\mu}
	\gamma_2(\scrF,\vertiii{\cdot}_{\psi_2})\leq \theta \sqrt{N},
	\] 
	with probability at least $1-\exp(-\tilde{C}_2\theta^2 N/\bar{\mu}^4)$ we have
	\begin{align}
	\sup\limits_{f_{\bw}\in \scrF}\Bigl|\frac{1}{N}\sum\limits_{i=1}^Nf_{\bw}^2(\bz_i)-\mathbb{E}[f_{\bw}^2]\Bigr|
	& = \sup\limits_{f_{\bw}\in
		\scrF}\Bigl|\frac{1}{N} \sum\limits_{i=1}^N\bw'
	\bGamma^{-1/2}\bz_i\bz_i'\bGamma^{-1/2}\bw-1\Bigr|\notag\\ 
	& = \sup\limits_{\bw\in
		\scrA_{\bGamma}}\Bigl|\bw' \bGamma^{-1/2}
	\hat{\bGamma}\bGamma^{-1/2}\bw-1\Bigr|\notag \\  
	& \le \theta, \label{thmDm} 
	\end{align}
	where $\tilde{C}_1$ is some positive constant and
	$\gamma_2(\scrF,\vertiii{\cdot}_{\psi_2})$ is defined in
	\cite{mendelson2007reconstruction} as a measure of the size of the set
	$\scrF$ with respect to the metric $\vertiii{\cdot}_{\psi_2}$.  Using
	$\theta=1/2$, and properties of
	$\gamma_2(\scrF,\vertiii{\cdot}_{\psi_2})$ outlined in
	\cite{chen2016alternating},
	we can set $N$ to satisfy
	\begin{align*}
	\tilde{C}_1 \bar{\mu} \gamma_2(\scrF,\vertiii{\cdot}_{\psi_2}) & \leq   
	\tilde{C}_1 \bar{\mu}^2 \gamma_2(\scrA_{\bGamma},\|\cdot\|_2)\\
	& \leq \tilde{C}_1 \bar{\mu}^2 C_0 w(\scrA_{\bGamma})\\ 
	& \leq \frac{1}{2} \sqrt{N}, 
	\end{align*}
	for some positive constant $C_0$, where we used Eq.\ (44) in
	\cite{chen2016alternating}. This implies 
	\[ 
	N\geq C_1\bar{\mu}^4 (w(\scrA_{\bGamma}))^2
	\]
	for some positive constant $C_1$. Thus, for such $N$ and with
	probability at least $1-\exp(-C_2N/\bar{\mu}^4)$, for some positive
	constant $C_2$, (\ref{thmDm}) holds with $\theta=1/2$. This implies that
	for all $\bw\in \scrA_{\bGamma}$,
	\[
	\Bigl|\bw' \bGamma^{-1/2} \hat{\bGamma} \bGamma^{-1/2}\bw-1\Bigr| 
	\le \frac{1}{2}
	\] 
	or 
	\[ 
	\bw' \bGamma^{-1/2} \hat{\bGamma} \bGamma^{-1/2}\bw \ge \frac{1}{2} = 
	\frac{1}{2}\bw'\bGamma^{-1/2}\bGamma\bGamma^{-1/2}\bw.
	\]
	By the definition of $\scrA_{\bGamma}$, for any $\bv \in \scrA(\bbeta^*)$,
	$$\bv'\hat{\bGamma}\bv \ge \frac{1}{2}\bv'\bGamma\bv.$$
	Noting that $\hat{\bGamma} = (1/N) \bZ\bZ'$ yields the desired result.
\end{proof}

Note that the sample size requirement stated in Lemma~\ref{alphalem}
depends on the Gaussian width of $\scrA_{\bGamma}$, where $\scrA_{\bGamma}$ relates to $\scrA(\bbeta^*)$. The
following lemma shows that their Gaussian widths are also related. This relation is built upon the square root of the eigenvalues of $\bGamma$, which measures the extent to which $\scrA_{\bGamma}$ expands $\scrA(\bbeta^*)$.
\begin{lem}[Lemma 4 in
	\cite{chen2016alternating}] \label{gaussianwidthlem} Let $\mu_0$ be
	the $\psi_2$-norm of a standard Gaussian random vector $\bg \in
	\mathbb{R}^{p+1}$, and $\scrA_{\bGamma}$, $\scrA(\bbeta^*)$ be defined
	as in Lemma~\ref{alphalem}. Then, under Assumption~\ref{eigen},
	\begin{equation*}
	w(\scrA_{\bGamma})\le C_3\mu_0\sqrt{\frac{\lambda_{\text{max}}}{\lambda_{\text{min}}}}\Bigl(w(\scrA(\bbeta^*))+3\Bigr),
	\end{equation*}
	for some positive constant $C_3$.
\end{lem}

\begin{proof}
	We follow the proof of Lemma 4 in \cite{chen2016alternating}, adapted
	to our setting. We include all key steps for completeness.
	
	Recall the definition of the Gaussian width $w(\scrA_{\bGamma})$
	(cf. (\ref{gw})):
	\[ 
	w(\scrA_{\bGamma}) = \mbb{E}\Bigl[\sup_{\bu \in \scrA_{\bGamma}} \bu'\bg\Bigr],
	\]
	where $\bg\sim {\cal N}(\bzero,\bI)$.  
	We have: 
	\begin{equation*}
	\begin{split}
	\sup_{\bw\in \scrA_{\bGamma}}\bw'\bg & =\sup\limits_{\bw\in
		\scrA_{\bGamma}}\bw'\bGamma^{-1/2}\bGamma^{1/2}\bg\\
	&=  \sup_{\bw\in \scrA_{\bGamma}} \| \bGamma^{-1/2} \bw\|_2 
	\frac{\bw'\bGamma^{-1/2}}{\| \bGamma^{-1/2} \bw\|_2} \bGamma^{1/2}\bg \\
	& \leq \sqrt{\frac{1}{\lambda_{\text{min}}}} \sup_{\bv\in
		\text{cone}(\scrA(\bbeta^*))\cap \scrB^{p+1}}\bv'\bGamma^{1/2}\bg, 
	\end{split}
	\end{equation*}
	where $\scrB^{p+1}$ is the unit ball in the $(p+1)$-dimensional Euclidean
	space and the inequality used Assumption~\ref{eigen} and the fact that
	$$\bw' \bGamma^{-1/2}/ \| \bGamma^{-1/2} \bw\|_2 \in \scrB^{p+1}, \quad 
	\bw\in \scrA_{\bGamma}.$$
	
	Define $\scrT=\text{cone}(\scrA(\bbeta^*))\cap \scrB^{p+1}$, and
	consider the stochastic process $\{S_{\bv}=\bv'\bGamma^{1/2}\bg\}_{\bv
		\in \scrT}$.  For any $\bv_1, \bv_2 \in \scrT$,
	\begin{equation*}
	\begin{split}
	\vertiii{S_{\bv_1}-S_{\bv_2}}_{\psi_2} & =
	\vertiii{(\bv_1-\bv_2)'\bGamma^{1/2}\bg}_{\psi_2}\\ 
	& = \|\bGamma^{1/2} (\bv_1-\bv_2) \|_2
	\vertiii{\frac{(\bv_1-\bv_2)'\bGamma^{1/2}\bg}{\|\bGamma^{1/2}
			(\bv_1-\bv_2) \|_2}}_{\psi_2} \\
	& \leq \|\bGamma^{1/2} (\bv_1-\bv_2) \|_2 \sup_{\bu \in \scrS^{p+1}}
	\vertiii{\bu'\bg}_{\psi_2} \\
	& = \mu_0\|\bGamma^{1/2}(\bv_1-\bv_2)\|_2\\
	& \le \mu_0\sqrt{\lambda_{\text{max}}}\|\bv_1-\bv_2\|_2, 
	\end{split}
	\end{equation*}
	where the last step used Assumption~\ref{eigen}. 
	
	Then, by the tail behavior of sub-Gaussian random variables (see
	Hoeffding bound, Thm. 2.6.2 in \cite{RV17}), we have:
	\begin{equation*}
	\mathbb{P}(|S_{\bv_1}-S_{\bv_2}|\ge \delta) \le 2
	\exp\biggl(-\frac{C_{01}\delta^2}{\mu_0^2\lambda_{\text{max}}
		\|\bv_1-\bv_2\|_2^2}\biggr),  
	\end{equation*}
	for some positive constant $C_{01}$. 
	
	To bound the supremum of $S_{\bv}$, we define the metric
	$s(\bv_1,\bv_2)=\mu_0\sqrt{\lambda_{\text{max}}}\|\bv_1-\bv_2\|_2.$
	Then, by Lemma B in \cite{chen2016alternating},
	\begin{equation*}
	\begin{split}
	\mathbb{E}\biggl[\sup\limits_{\bv \in
		\scrT}\bv'\bGamma^{1/2}\bg\biggr] & \le
	C_{02}\gamma_2(\scrT,s)\\
	& = C_{02}\mu_0\sqrt{\lambda_{\text{max}}}\gamma_2(\scrT,\|\cdot\|_2)\\
	& \le C_{3}\mu_0\sqrt{\lambda_{\text{max}}}w(\scrT),
	\end{split}
	\end{equation*}
	for positive constants $C_{02}, C_3$, where $\gamma_2(\scrT, s)$ is
	the $\gamma_2$-functional we referred to in the proof of
	Lemma~\ref{alphalem}. Since $\scrT=\text{cone}(\scrA(\bbeta^*))\cap 
	\scrB^{p+1}\subseteq \text{conv}(\scrA(\bbeta^*)
	\cup\{\boldsymbol{0}\})$, by Lemma 2 in \cite{maurer2014inequality},
	\begin{equation*}
	\begin{split}
	w(\scrT) & \le w(\text{conv}(\scrA(\bbeta^*) \cup\{\boldsymbol{0}\}))\\
	& = w(\scrA(\bbeta^*) \cup\{\boldsymbol{0}\})\\
	& \le \max\{w(\scrA(\bbeta^*)),w(\{\boldsymbol{0}\})\}+2\sqrt{\ln 4}\\
	& \le w(\scrA(\bbeta^*))+3.
	\end{split}
	\end{equation*}
	Thus,
	\begin{equation*}
	\begin{split}
	w(\scrA_{\bGamma})&=\mathbb{E}\biggl[\sup\limits_{\bw\in \scrA_{\bGamma}}\bw'\bg\biggr]\\
	& \le \sqrt{\frac{1}{\lambda_{\text{min}}}}\mathbb{E}\biggl[\sup\limits_{\bv\in \scrT}\bv'\bGamma^{1/2}\bg\biggr]\\
	& \le C_{3}\sqrt{\frac{1}{\lambda_{\text{min}}}}\mu_0\sqrt{\lambda_{\text{max}}}w(\scrT)\\
	& \le C_{3}\mu_0\sqrt{\frac{\lambda_{\text{max}}}{\lambda_{\text{min}}}}\Bigl(w(\scrA(\bbeta^*))+3\Bigr).
	\end{split}
	\end{equation*}
\end{proof}

Combining Lemmata~\ref{alphalem} and \ref{gaussianwidthlem}, and expressing the covariance matrix $\bGamma$ using its eigenvalues, we arrive at the following result.
\begin{col} \label{alphacol} Under Assumptions~\ref{subgaussian} and
	\ref{eigen}, and the conditions in Lemmata~\ref{alphalem} and
	\ref{gaussianwidthlem}, when 
	\[ N\ge \bar{C_1}\bar{\mu}^4
	\mu_0^2\cdot\frac{\lambda_{\text{max}}}{\lambda_{\text{min}}}\Bigl(w(\scrA(\bbeta^*))+3\Bigr)^2,
	\] 
	with probability at least $1-\exp(-C_2N/\bar{\mu}^4)$,
	\begin{equation*}
	\bv'\bZ\bZ'\bv\ge \frac{N\lambda_{\text{min}}}{2}, \qquad \forall 
	\bv\in \scrA(\bbeta^*), 
	\end{equation*}
	where $\bar{C_1}$ and $C_2$ are positive constants.
\end{col}

\begin{proof}
	Combining Lemmata~\ref{alphalem} and \ref{gaussianwidthlem}, and using the fact that for any $\bv \in \scrA(\bbeta^*)$,
	$$\frac{N}{2}\bv'\bGamma\bv \ge \frac{N\lambda_{\text{min}}}{2},$$
	we can derive the desired result.
\end{proof}

Next we derive the smallest possible value of $\gamma_N$
such that $\bbeta^*$ is feasible. 

\begin{lem} \label{gammalem}
	Under Assumptions~\ref{a1} and \ref{a2}, for any feasible
	$\bbeta$,
	\begin{equation*}
	\|(-\bbeta, 1)'\bZ\|_1\le N \bar{B}R, \ \text{a.s. under $\mbb{P}^*$}.
	\end{equation*}
\end{lem}

Combining Theorem~\ref{mainresult}, Corollary~\ref{alphacol} and
Lemma~\ref{gammalem}, we have the following main performance guarantee result 
that bounds the estimation bias of the solution to (\ref{infty}).
\begin{thm} \label{estthm}
	Under Assumptions~\ref{a1}, \ref{a2}, \ref{RE}, \ref{adm},
	\ref{subgaussian}, \ref{eigen}, and the conditions of
	Theorem~\ref{mainresult}, Corollary~\ref{alphacol} and Lemma~\ref{gammalem}, when 
	\[ N\ge \bar{C_1}\bar{\mu}^4
	\mu_0^2\cdot\frac{\lambda_{\text{max}}}{\lambda_{\text{min}}}
	\Bigl(w(\scrA(\bbeta^*))+3\Bigr)^2,
	\]  
	with probability at least
	$1-\exp(-C_2N/\bar{\mu}^4)$,
	\begin{equation} \label{finalbound} \|\hat{\bbeta}-\bbeta^*\|_2\le
	\frac{4R^2\bar{B}}{\lambda_{\text{min}}}
	\Psi(\bbeta^*).
	\end{equation}
\end{thm}

The estimation error bound in (\ref{finalbound}) depends on the variance of $(\bx, y)$, and the geometrical structure of the true regression coefficient. It does not decay to zero as $N$ goes to infinity. The reason is that the absolute residual $|y - \bbeta'\bx|$ has a nonzero mean, which will be propagated into the estimation bias.

\section{Experiments on the Performance of Wasserstein DRO} \label{sec:2-4}
In this section, we will explore the robustness of the Wasserstein formulation in terms of its {\em Absolute Deviation (AD)} loss function and the dual norm regularizer on the {\em extended regression coefficient} $(-\bbeta, 1)$. Recall that the Wasserstein formulation is in the following form:
\begin{equation} \label{qcp2}
\inf\limits_{\bbeta} \frac{1}{N}\sum\limits_{i=1}^N|y_i - \bx_i'\bbeta| + \epsilon\|(-\bbeta, 1)\|_*.
\end{equation}
We will focus on the following three aspects of this formulation:
\begin{enumerate}
	\item How to choose a proper norm $\|\cdot\|$ for the Wasserstein metric?
	\item Why do we penalize the extended regression coefficient $(-\bbeta, 1)$ rather than $\bbeta$?
	\item What is the advantage of the AD loss compared to the {\em Squared Residuals (SR)} loss?
\end{enumerate} 
To answer Question 1, we will connect the choice of $\|\cdot\|$ for the Wasserstein metric with the characteristics/structures of the data $(\bx, y)$. Specifically, we will design two sets of experiments, one with a dense regression coefficient $\bbeta^*$, where all coordinates of $\bx$ play a role in determining the value of the response $y$, and another with a sparse $\bbeta^*$ implying that only a few predictors are relevant in predicting $y$. Two Wasserstein formulations will be tested and compared, one induced by the $\|\cdot\|_2$ (Wasserstein $\ell_2$), which leads to an $\ell_2$-regularizer in (\ref{qcp2}), and the other one induced by the $\|\cdot\|_{\infty}$ (Wasserstein $\ell_{\infty}$) and resulting in an $\ell_1$-regularizer in (\ref{qcp2}). 

The problem of feature selection can be formulated as an $\ell_0$-norm regularized regression problem, which is NP-hard and is usually relaxed to an $\ell_1$-norm regularized formulation, known as the {\em Least Absolute Shrinkage and Selection Operator (LASSO)}. LASSO enjoys several attractive statistical properties under various conditions on the model matrix \citep{tibshirani2011regression, friedman2001elements}. 
Here, in our context, we try to offer an explanation of the sparsity-inducing property of LASSO from the perspective of the Wasserstein DRO formulation, through projecting the sparsity of $\bbeta^*$ onto the $(\bx, y)$ space and establishing a {\em sparse} distance metric that only extracts a subset of coordinates from $(\bx, y)$ to measure the closeness between samples. 

For the second question, we first note that if the Wasserstein metric is induced by the following metric $s_c$:
\begin{equation*}
s_c(\bx, y) = \|(\bx, cy)\|_2,
\end{equation*}
for a positive constant $c$; then as $c \rightarrow \infty$, the resulting Wasserstein DRO formulation becomes:
\begin{equation*}
\inf\limits_{\bbeta} \frac{1}{N}\sum\limits_{i=1}^N|y_i - \bx_i'\bbeta| + \epsilon\|\bbeta\|_2,
\end{equation*}
which is the $\ell_2$-regularized LAD. This can be proved by recognizing that $s_c(\bx, y) = \|(\bx, y)\|_{\bM}$, with $\bM \in \mbb{R}^{(p+1) \times (p+1)}$ a diagonal matrix whose diagonal elements are $(1, \ldots, 1, c^2)$, and then applying (\ref{wqcp}). Alternatively, if we let 
$s_c(\bx, y) = \|(\bx, cy)\|_{\infty},$ Corollary~\ref{wass-l1-LAD}
shows that as $c \rightarrow \infty$, the corresponding Wasserstein formulation becomes
the $\ell_1$-regularized LAD.

\begin{col} \label{wass-l1-LAD}
	If the Wasserstein metric is induced by the following metric $s$:
	$$s_c(\bx, y) = \|(\bx, cy)\|_{\infty},$$
	with $c$ some positive constant. Then as $c \rightarrow \infty$, the Wasserstein DRO formulation (\ref{qcp2}) reduces to:
	\begin{equation*}
	\inf\limits_{\bbeta} \frac{1}{N}\sum\limits_{i=1}^N|y_i - \bx_i'\bbeta| + \epsilon\|\bbeta\|_1,
	\end{equation*}
	which is the $\ell_1$-regularized LAD. 
\end{col}

\begin{proof}
	We first define a new notion of norm on $(\bx, y)$ where $\bx = (x_1, \ldots, x_{p})$:
	\begin{equation*}
	\|(\bx, y)\|_{\bw, r} \triangleq \|(x_1w_1, \ldots, x_{p} w_{p}, yw_{p+1})\|_{r},
	\end{equation*}
	for some $(p+1)$-dimensional weighting vector $\bw = (w_1, \ldots, w_{p+1})$, and $r \ge 1$. Then, $s_c(\bx, y) = \|(\bx, y)\|_{\bw, \infty}$ with $\bw = (1, \ldots, 1, c)$. To obtain the Wasserstein DRO formulation, the key is to derive the dual norm of $\|\cdot\|_{\bw, \infty}$. 
	H\"{o}lder's inequality \citep{rogers1888extension} will be used for the derivation. 
	We will use the notation $\bz \triangleq (\bx, y)$. Based on the definition of dual norm, we are interested in solving the following optimization problem for $\tilde{\bbeta} \in \mbb{R}^{p+1}$:
	\begin{equation} \label{dualnorm2}
	\begin{aligned}
	\max\limits_{\bz} & \quad \bz' \tilde{\bbeta} \\
	\text{s.t.} & \quad \|\bz\|_{\bw, \infty} \le 1.
	\end{aligned}
	\end{equation}
	The optimal value of Problem (\ref{dualnorm2}), which is a function of $\tilde{\bbeta}$, gives the dual norm evaluated at $\tilde{\bbeta}$. Using H\"{o}lder's inequality, we can write
	\begin{equation*}
	\begin{aligned}
	\bz'\tilde{\bbeta} & = \sum_{i=1}^{p+1} (w_i z_i)\Bigl(\frac{1}{w_i}\tilde{\beta}_i\Bigr) \\
	& \le \|\bz\|_{\bw, \infty} \|\tilde{\bbeta}\|_{\bw^{-1}, 1} \\
	& \le \|\tilde{\bbeta}\|_{\bw^{-1}, 1},
	\end{aligned}
	\end{equation*} 
	where $\bw^{-1} \triangleq (\frac{1}{w_1}, \ldots, \frac{1}{w_{p+1}})$.
	The last inequality is due to the constraint $\|\bz\|_{\bw, \infty} \le 1$. It follows that the dual norm of $\|\cdot\|_{\bw, \infty}$ is just $\|\cdot\|_{\bw^{-1}, 1}$. Back to our problem setting, using $\bw = (1, \ldots, 1, c)$, and evaluating the dual norm at $(-\bbeta, 1)$, we have the following Wasserstein DRO formulation as $c \rightarrow \infty$:
	\begin{equation*}
	\inf\limits_{\bbeta} \frac{1}{N}\sum\limits_{i=1}^N|y_i - \bx_i'\bbeta| + \epsilon\|(-\bbeta, 1)\|_{\bw^{-1}, 1} = \inf\limits_{\bbeta} \frac{1}{N}\sum\limits_{i=1}^N|y_i - \bx_i'\bbeta| + \epsilon\|\bbeta\|_1.
	\end{equation*}
\end{proof}

It follows that regularizing over $\bbeta$ implies an infinite transportation cost along $y$. By contrast, the Wasserstein formulation, which regularizes over the extended regression coefficient $(-\bbeta, 1)$, stems from a finite cost along $y$ that is equally weighted with $\bx$. We will see the disadvantages of penalizing only $\bbeta$ in the analysis of the experimental results.

To answer Question 3, we will compare with several commonly used regression models that employ the SR loss function, e.g., ridge regression \citep{hoerl1970ridge}, LASSO \citep{tibshirani1996regression}, and {\em Elastic Net (EN)} \citep{zou2005regularization}. We will also compare against M-estimation \citep{huber1964robust, huber1973robust}, which uses a variant of the SR loss and is equivalent to solving a weighted least squares problem. These models will be compared under two different experimental setups, one involving perturbations in both $\bx$ and $y$, and the other with perturbations only in $\bx$. The purpose is to investigate the behavior of these approaches when the noise in $y$ is substantially reduced. 

We next describe the data generation process. Each training sample has a probability $q$ of being drawn from the outlying distribution, and a probability $1-q$ of being drawn from the true (clean) distribution. Given the true regression coefficient $\bbeta^*$, we generate the training data as follows:
\begin{itemize}
	\item Generate a uniform random variable on $[0, 1]$. If it is no larger than $1-q$, generate a clean sample as follows:
	\begin{enumerate}
		\item Draw the predictor $\bx \in \mbb{R}^{p}$ from the normal distribution $\scrN(\mathbf{0}, \bSigma)$, where $\bSigma$ 
		is the covariance matrix of $\bx$, which is just the top left block of the matrix $\bGamma$ in Assumption~\ref{eigen}. Specifically, $\bGamma=\mathbb{E}[(\bx, y) (\bx, y)']$ is equal to
		\begin{equation*}
		\bGamma = 
		\begin{bmatrix}
		& \bSigma   & \bSigma \bbeta^* \\
		& (\bbeta^*)'\bSigma & (\bbeta^*)' \bSigma \bbeta^* + \sigma^2
		\end{bmatrix},
		\end{equation*}
		with $\sigma^2$ being the variance of the noise term. In our implementation, $\bSigma$ has diagonal elements equal to $1$ (unit variance) and off-diagonal elements equal to $\rho$, with $\rho$ the correlation between predictors. 
		\item Draw the response variable $y$ from $\scrN(\bx' \bbeta^*, \sigma^2)$.
	\end{enumerate}
	\item Otherwise, depending on the experimental setup, generate an outlier that is either:
	\begin{itemize}
		\item Abnormal in both $\bx$ and $y$, with outlying distribution:
		\begin{enumerate}
			\item $\bx \sim \scrN (\mathbf{0}, \bSigma) + \scrN (5\mathbf{e}, \mathbf{I})$, or $\bx \sim \scrN (\mathbf{0}, \bSigma) + \scrN (\mathbf{0}, 0.25\mathbf{I})$;
			\item $y \sim \scrN(\bx' \bbeta^*, \sigma^2) + 5\sigma$.
		\end{enumerate}
		\item Abnormal only in $\bx$: 
		\begin{enumerate}
			\item $\bx \sim \scrN (\mathbf{0}, \bSigma) + \scrN (5\mathbf{e}, \mathbf{I})$;
			\item $y \sim \scrN(\bx' \bbeta^*, \sigma^2)$. 
		\end{enumerate}
	\end{itemize}	
	\item Repeat the above procedure for $N$ times, where $N$ is the size of the training set.
\end{itemize}
To test the generalization ability of various formulations, we generate a test dataset containing $M$ samples from the clean distribution. We are interested in studying the performance of various methods as the following factors are varied:
\begin{itemize}
	\item {\em Signal to Noise Ratio (SNR)}, defined as:
	\begin{equation*}
	\text{SNR} = \frac{(\bbeta^*)'\bSigma \bbeta^*}{\sigma^2},
	\end{equation*}
	which is equally spaced between $0.05$ and $2$ on a log scale.
	\item The correlation between predictors: $\rho$, which takes values in $(0.1, 0.2, \ldots, 0.9)$.
\end{itemize}  	
The performance metrics we use include:
\begin{itemize}
	\item {\em Mean Squared Error (MSE)} on the test dataset, which is defined to be $\sum_{i=1}^M(y_i - \bx_i'\hat{\bbeta})^2/M$, with $\hat{\bbeta}$ being the estimate of $\bbeta^*$ obtained from the training set, and $(\bx_i, y_i), \ i \in \lb M \rb,$ being the observations from the test dataset;
	\item {\em Relative Risk (RR)} of $\hat{\bbeta}$ defined as:
	\begin{equation*}
	\text{RR}(\hat{\bbeta}) \triangleq \frac{(\hat{\bbeta} - \bbeta^*)'\mathbf{\Sigma}(\hat{\bbeta} - \bbeta^*)}{(\bbeta^*)' \mathbf{\Sigma} \bbeta^*}.
	\end{equation*}
	\item {\em Relative Test Error (RTE)} of $\hat{\bbeta}$ defined as:
	\begin{equation*}
	\text{RTE}(\hat{\bbeta}) \triangleq \frac{(\hat{\bbeta} - \bbeta^*)'\bSigma (\hat{\bbeta} - \bbeta^*) + \sigma^2}{\sigma^2}.
	\end{equation*}
	\item {\em Proportion of Variance Explained (PVE)} of $\hat{\bbeta}$ defined as:
	\begin{equation*}
	\text{PVE}(\hat{\bbeta}) \triangleq 1 - \frac{(\hat{\bbeta} - \bbeta^*)'\bSigma (\hat{\bbeta} - \bbeta^*) + \sigma^2}{(\bbeta^*)' \mathbf{\Sigma} \bbeta^* + \sigma^2}.
	\end{equation*}
\end{itemize} 
For the metrics that evaluate the accuracy of the estimator, i.e., the RR, RTE and PVE, we list below two types of scores, one achieved by the best possible estimator $\hat{\bbeta} = \bbeta^*$, called the perfect score, and the other one achieved by the null estimator $\hat{\bbeta} = 0$, called the null score. 
\begin{itemize}
	\item RR: a perfect score is 0 and the null score is 1.
	\item RTE: a perfect score is 1 and the null score is SNR+1.
	\item PVE: a perfect score is $\frac{\text{SNR}}{\text{SNR}+1}$, and the null score is 0.
\end{itemize} 

All the regularization parameters are tuned on a separate validation dataset using the {\em Median Absolute Deviation (MAD)} as a selection criterion, to hedge against the potentially large noise in the validation samples. As to the range of values for the tuned parameters, we borrow ideas from \cite{hastie2017extended}, where the LASSO was tuned over $50$ values ranging from $\lambda_{max} = \|\bX'\by\|_{\infty}$ to a small fraction of $\lambda_{max}$ on a log scale, with $\bX \in \mbb{R}^{N \times p}$ the design matrix whose $i$-th row is $\bx_i'$, and $\by = (y_1, \ldots, y_N)$ the response vector. In our experiments, this range is properly adjusted for procedures that use the AD loss. Specifically, for Wasserstein $\ell_2$ and $\ell_{\infty}$, $\ell_1$- and $\ell_2$-regularized LAD, the range of values for the regularization parameter is: 
$$\sqrt{\exp\biggl(\text{lin}\Bigl(\log(0.005*\|\bX'\by\|_{\infty}),\log(\|\bX'\by\|_{\infty}),50\Bigr)\biggr)},$$
where $\text{lin}(a, b, n)$ is a function that takes in scalars $a$, $b$ and $n$ (integer) and outputs a set of $n$ values equally spaced between $a$ and $b$; the $\exp$ function is applied elementwise to a vector. The square root operator is in consideration of the AD loss that is the square root of the SR loss if evaluated on a single sample. 

\subsection{Dense $\bbeta^*$, Outliers in both $\bx$ and $y$} \label{densexy}
In this subsection, we choose a dense regression coefficient $\bbeta^*$, set the intercept to $\beta_0^* = 0.3$, and the coefficient for each predictor $x_i$ to be $\beta_i^* = 0.5, i \in \lb 20 \rb$. The perturbations are present in both $\bx$ and $y$. Specifically, the outlying distribution is described by:
\begin{enumerate}
	\item $\bx \sim \scrN (\mathbf{0}, \bSigma) + \scrN (5\mathbf{e}, \mathbf{I})$;
	\item $y \sim \scrN(\bx' \bbeta^*, \sigma^2) + 5\sigma$.
\end{enumerate}
We generate 10 datasets consisting of $N = 100, M = 60$ observations. The probability of a training sample being drawn from the outlying distribution is $q = 30\%$. The mean values of the performance metrics (averaged over the 10 datasets), as we vary the SNR and the correlation between predictors, are shown in Figures~\ref{snr-1} and \ref{corr-1}. Note that when SNR is varied, the correlation between predictors is set to $0.8$ times a random noise uniformly distributed on the interval $[0.2, 0.4]$. When the correlation $\rho$ is varied, the SNR is fixed to $0.5$.

It can be seen that as the SNR decreases or the correlation between the predictors increases, the estimation problem becomes harder, and the performance of all approaches gets worse. In general the Wasserstein formulation with an $\ell_2$-norm transportation cost achieves the best performance in terms of all four metrics. Specifically,
\begin{itemize}
	\item it is better than the $\ell_2$-regularized LAD, which assumes an infinite transportation cost along $y$;
	\item it is better than the Wasserstein $\ell_{\infty}$ and $\ell_1$-regularized LAD which use the $\ell_1$-regularizer;
	\item it is better than the approaches that use the SR loss function.
\end{itemize}

Empirically we have found that in most cases, the approaches that use the AD loss, including the $\ell_1$- and $\ell_2$-regularized LAD, and the Wasserstein $\ell_{\infty}$ formulation, drive all the coordinates of $\bbeta$ to zero, due to the relatively small magnitude of the AD loss compared to the norm of the coefficient. The approaches that use the SR loss, e.g., ridge regression and EN, do not exhibit such a problem, since the squared residuals weaken the dominance of the regularization term. 

Overall the $\ell_2$-regularizer outperforms the $\ell_1$-regularizer, since the true regression coefficient is dense, which implies that a proper distance metric on the $(\bx, y)$ space should take into account all the coordinates. From the perspective of the Wasserstein DRO framework, the $\ell_1$-regularizer corresponds to an $\|\cdot\|_{\infty}$-based distance metric on the $(\bx, y)$ space that only picks out the most influential coordinate to determine the closeness between data points, which in our case is not reasonable since every coordinate plays a role (reflected in the dense $\bbeta^*$). In contrast, if $\bbeta^*$ is sparse, using the $\|\cdot\|_{\infty}$ as a distance metric on $(\bx, y)$ is more appropriate. A more detailed discussion of this will be presented in Sections~\ref{sparsexy} and \ref{sparsex}.
\begin{figure}[p] 
	\begin{subfigure}{.49\textwidth}
		\centering
		\includegraphics[width=0.98\textwidth]{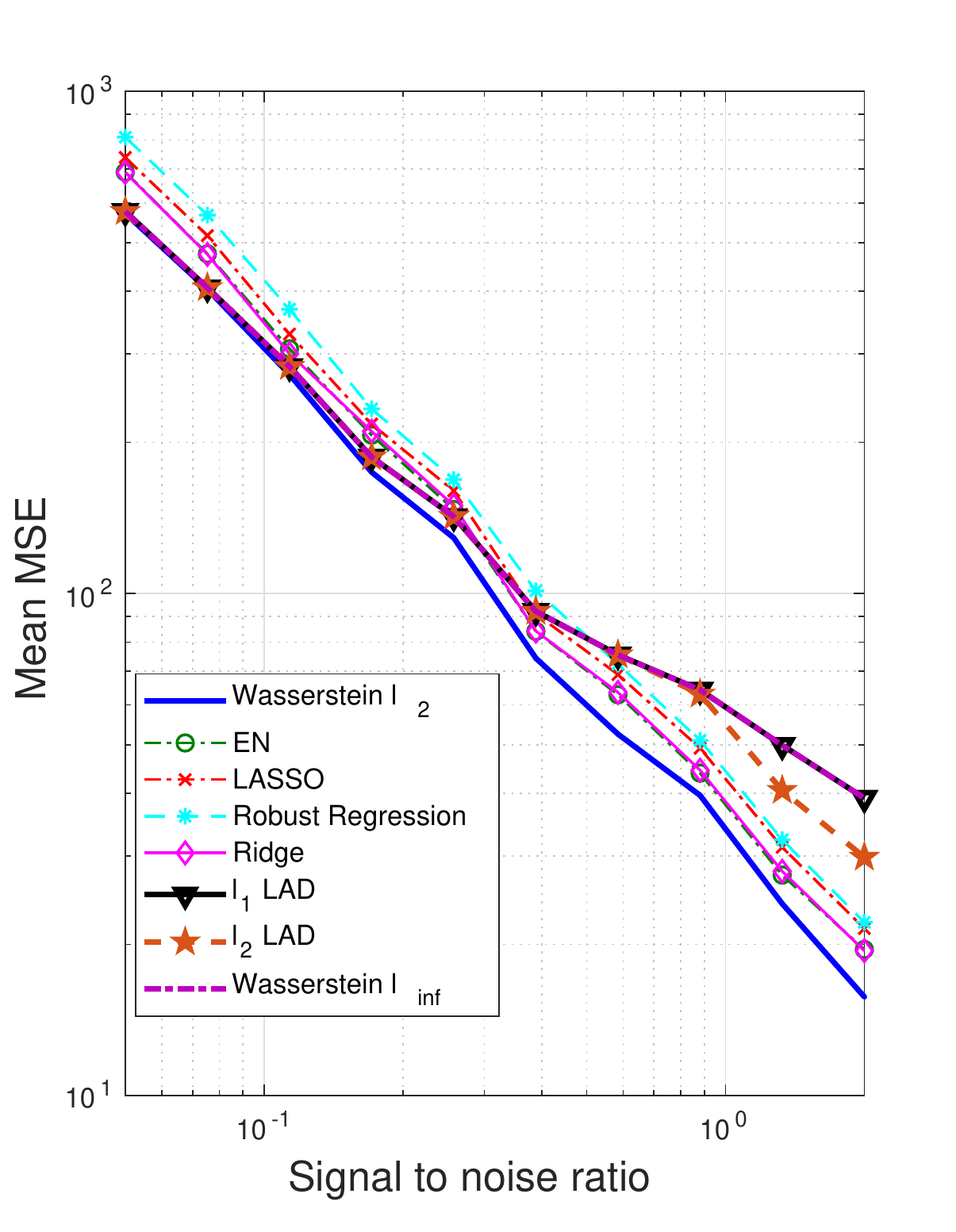}
		\caption{\small{Mean Squared Error.}}
	\end{subfigure}
	\begin{subfigure}{.49\textwidth}
		\centering
		\includegraphics[width=0.98\textwidth]{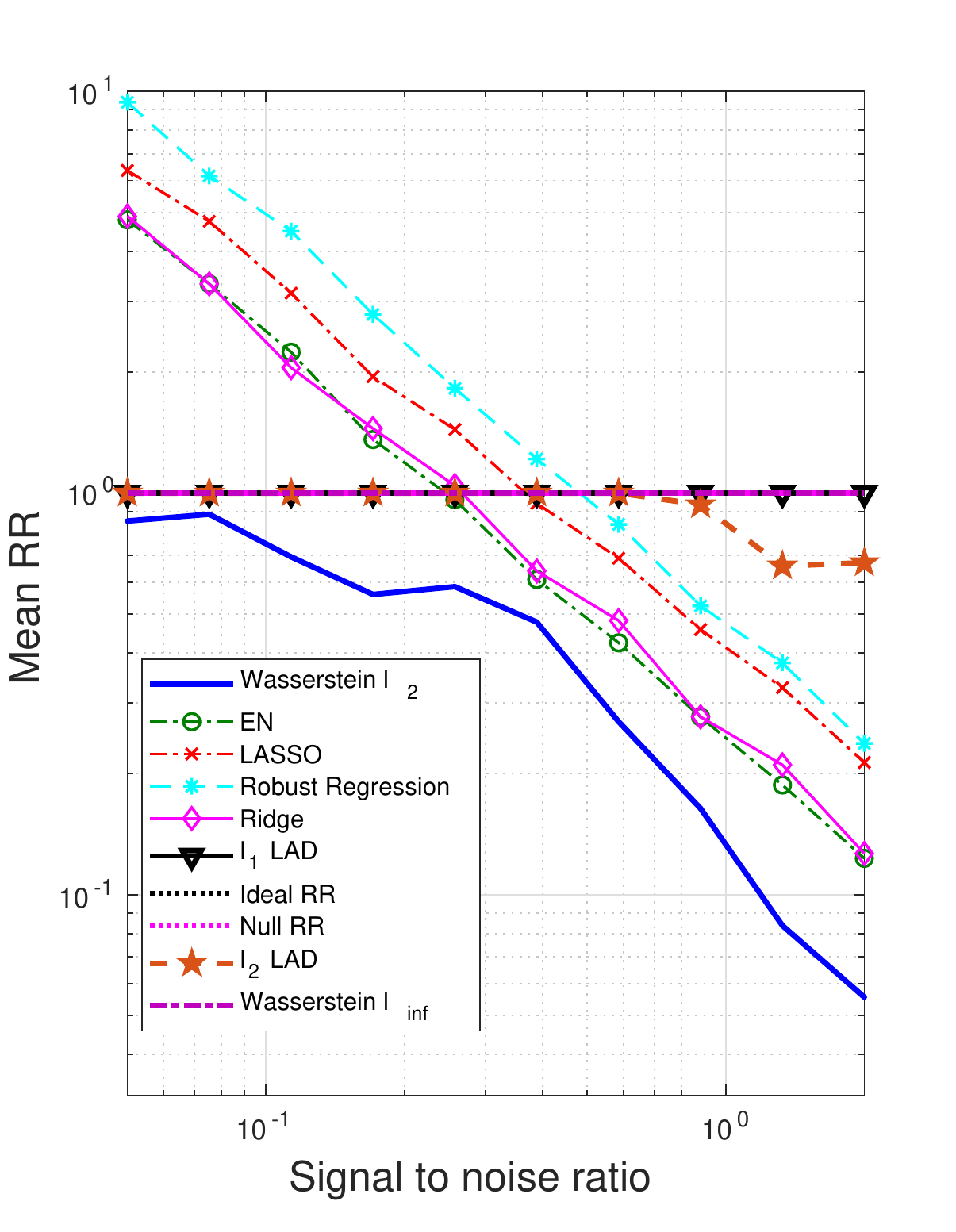}
		\caption{\small{Relative risk.}}
	\end{subfigure}
	
	\begin{subfigure}{.49\textwidth}
		\centering
		\includegraphics[width=0.98\textwidth]{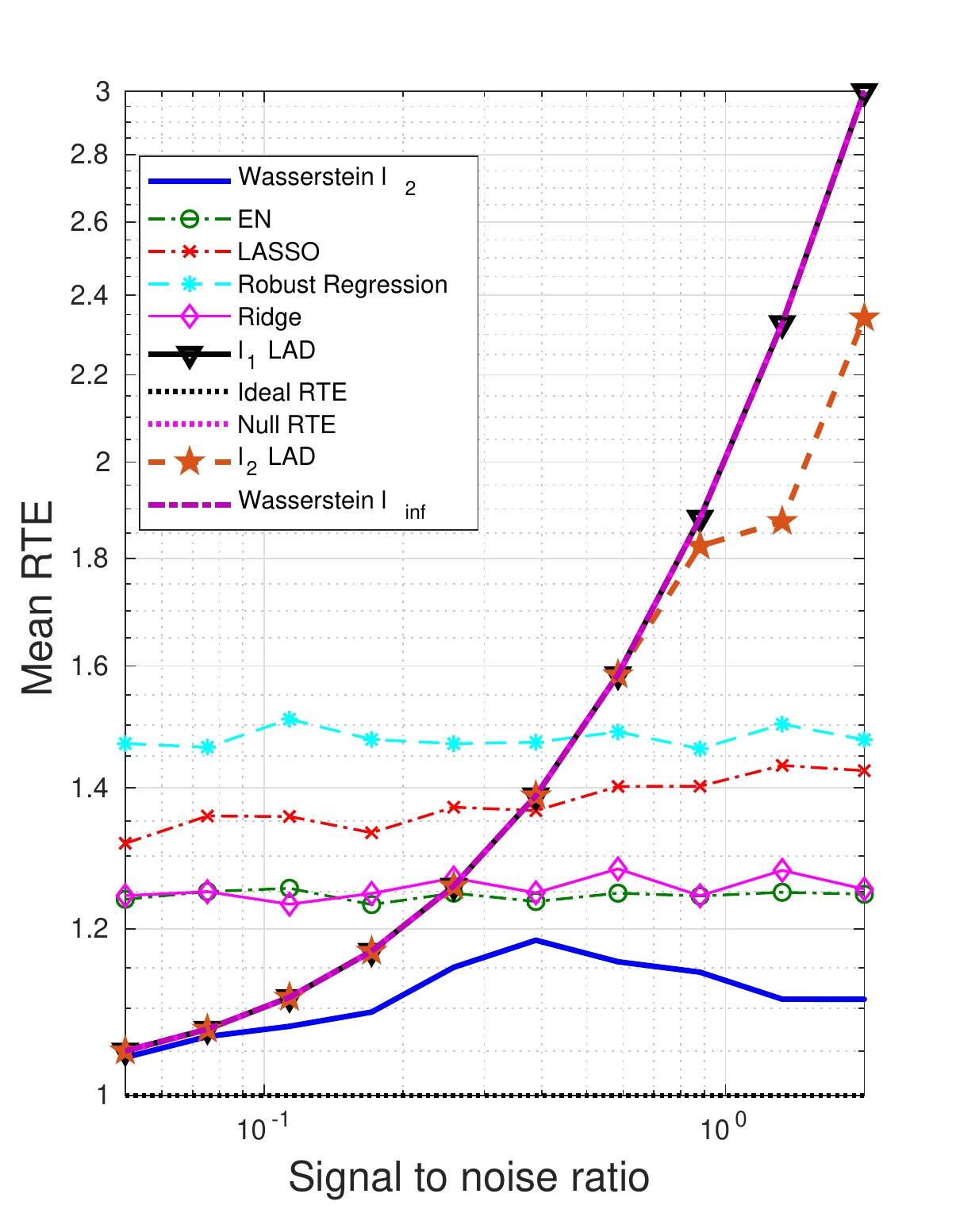}
		\caption{\small{Relative test error.}}
	\end{subfigure}%
	\begin{subfigure}{.49\textwidth}
		\centering
		\includegraphics[width=0.98\textwidth]{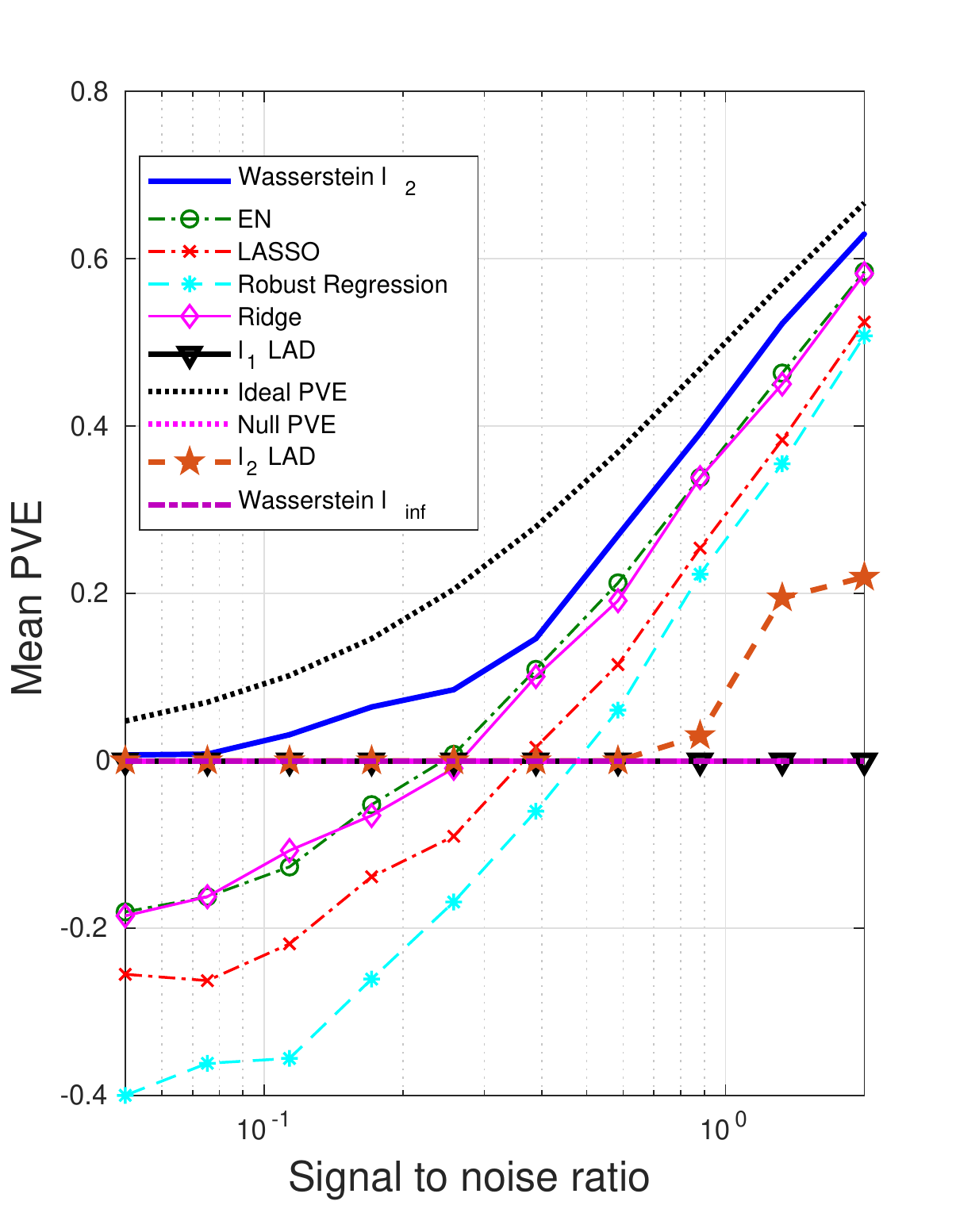}
		\caption{\small{Proportion of variance explained.}}
	\end{subfigure}
	\caption{The impact of SNR on the performance metrics: dense $\bbeta^*$,
		outliers in both $\bx$ and $y$.} 
	\label{snr-1}
\end{figure}

\begin{figure}[p] 
	\begin{subfigure}{.49\textwidth}
		\centering
		\includegraphics[width=.98\textwidth]{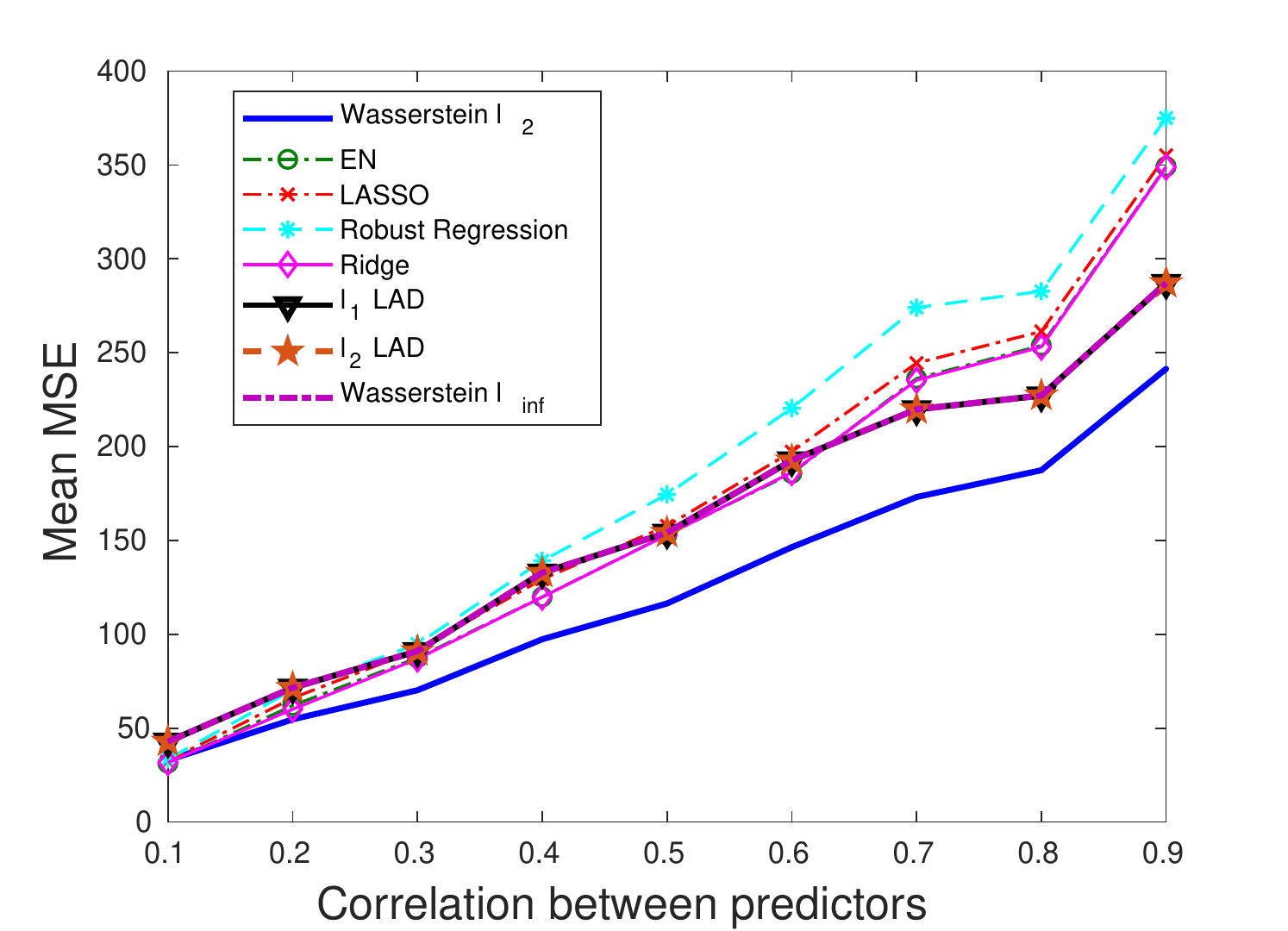}
		\caption{\small{Mean Squared Error.}}
	\end{subfigure}
	\begin{subfigure}{.49\textwidth}
		\centering
		\includegraphics[width=.98\textwidth]{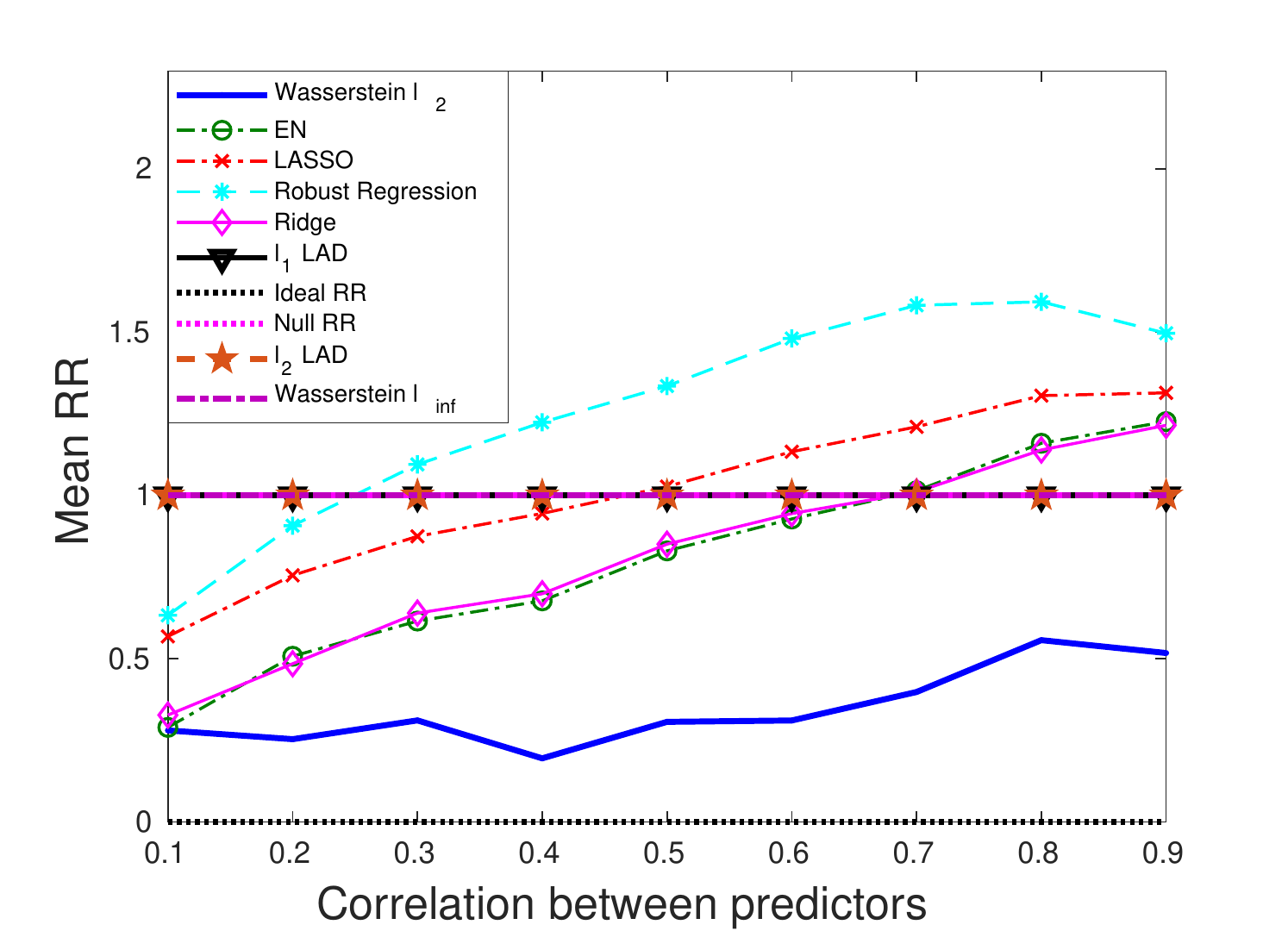}
		\caption{\small{Relative risk.}}
	\end{subfigure}
	
	\begin{subfigure}{.49\textwidth}
		\centering
		\includegraphics[width=.98\textwidth]{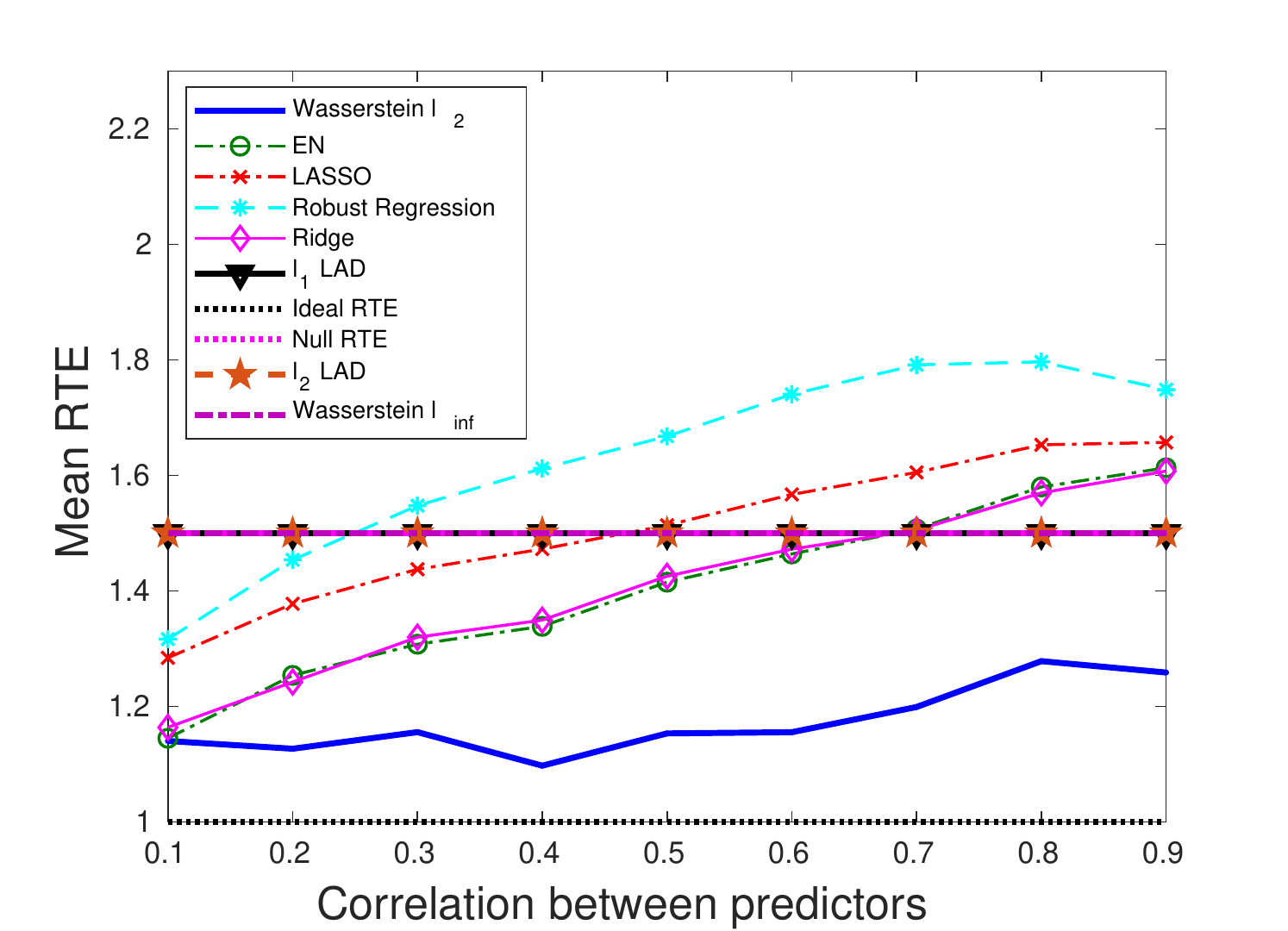}
		\caption{\small{Relative test error.}}
	\end{subfigure}%
	\begin{subfigure}{.49\textwidth}
		\centering
		\includegraphics[width=.98\textwidth]{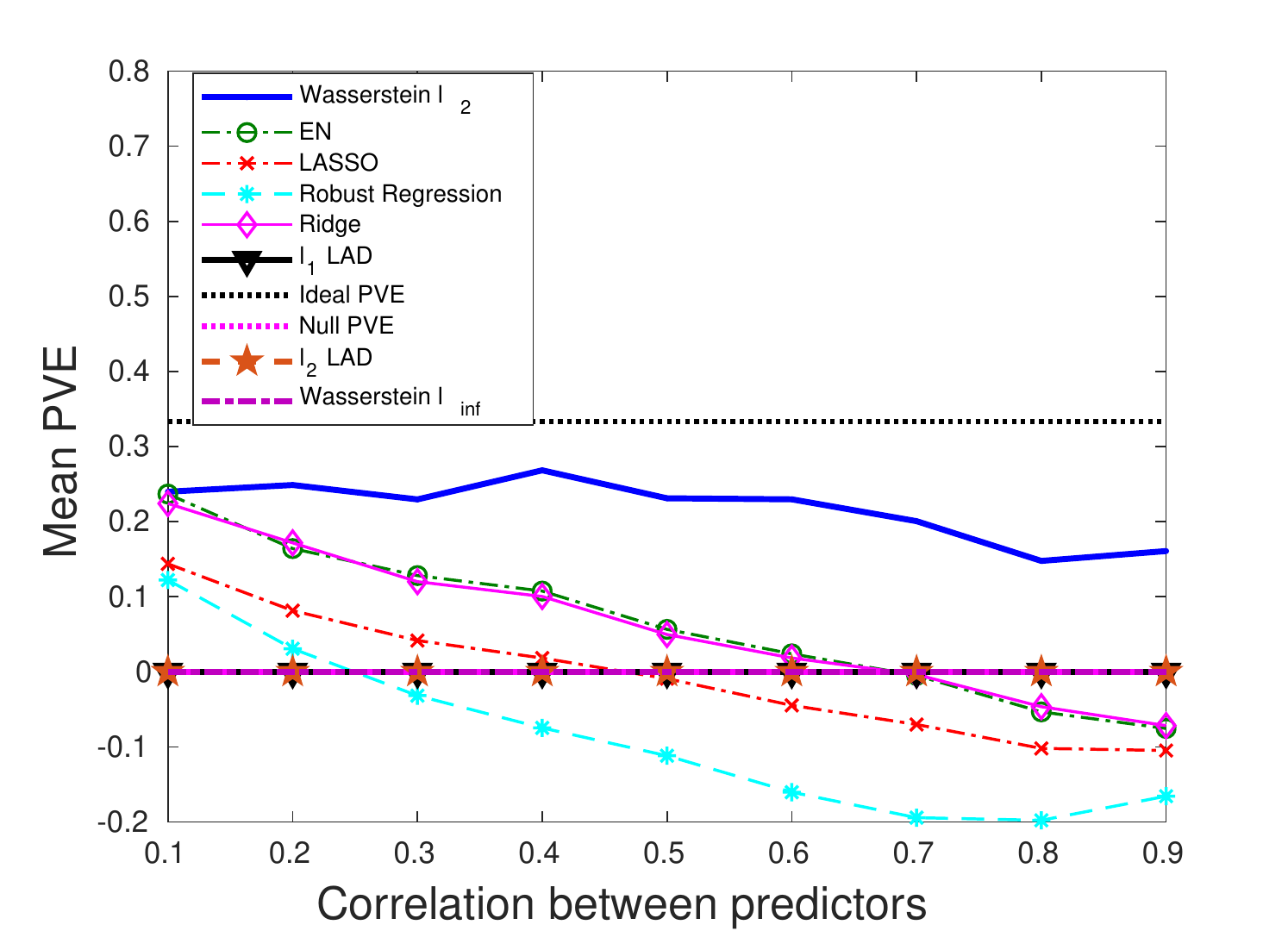}
		\caption{\small{Proportion of variance explained.}}
	\end{subfigure}
	\caption{The impact of predictor correlation on the performance metrics: dense $\bbeta^*$, outliers in both $\bx$ and $y$.}
	\label{corr-1}
\end{figure}

\subsection{Dense $\bbeta^*$, Outliers Only in $\bx$} \label{densex}
In this subsection, we will experiment with the same $\bbeta^*$ as in Section~\ref{densexy}, but with perturbations only in $\bx$. Our goal is to investigate the performance of the Wasserstein formulation when the response $y$ is not subjected to large perturbations. 

Interestingly, we observe that although the $\ell_1$- and $\ell_2$-regularized LAD, as well as the Wasserstein $\ell_{\infty}$ formulation, exhibit unsatisfactory performance, the Wasserstein $\ell_2$, which shares the same loss function with them, is able to achieve a comparable performance with the best among all -- EN and ridge regression (see Figures~\ref{snr-2} and \ref{corr-2}). Notably, the $\ell_2$-regularized LAD, which is just slightly different from the Wasserstein $\ell_2$ formulation, shows a much worse performance. This is because the $\ell_2$-regularized LAD implicitly assumes an infinite transportation cost along $y$, which gives zero tolerance to the variation in the response. Therefore, a reasonable amount of fluctuation, caused by the intrinsic randomness of $y$, would be overly exaggerated by the underlying metric used by the $\ell_2$-regularized LAD. In contrast, the Wasserstein approach uses a proper notion of norm to evaluate the distance in the $(\bx, y)$ space and is able to effectively distinguish abnormally high variations from moderate, acceptable noise.

It is also worth noting that the formulations with the AD loss, e.g., $\ell_2$- and $\ell_1$-regularized LAD, and the Wasserstein $\ell_{\infty}$, perform worse than the approaches with the SR loss. One reasonable explanation is that the AD loss, introduced primarily for hedging against large perturbations in $y$, is less useful when the noise in $y$ is moderate, in which case the sensitivity to response noise is needed. Although the AD loss is not a wise choice, penalizing the extended coefficient vector $(-\bbeta, 1)$ seems to make up, making the Wasserstein $\ell_2$ a competitive method even when the perturbations appear only in $\bx$.

\begin{figure}[p] 
	\begin{subfigure}{.49\textwidth}
		\centering
		\includegraphics[width=0.98\textwidth]{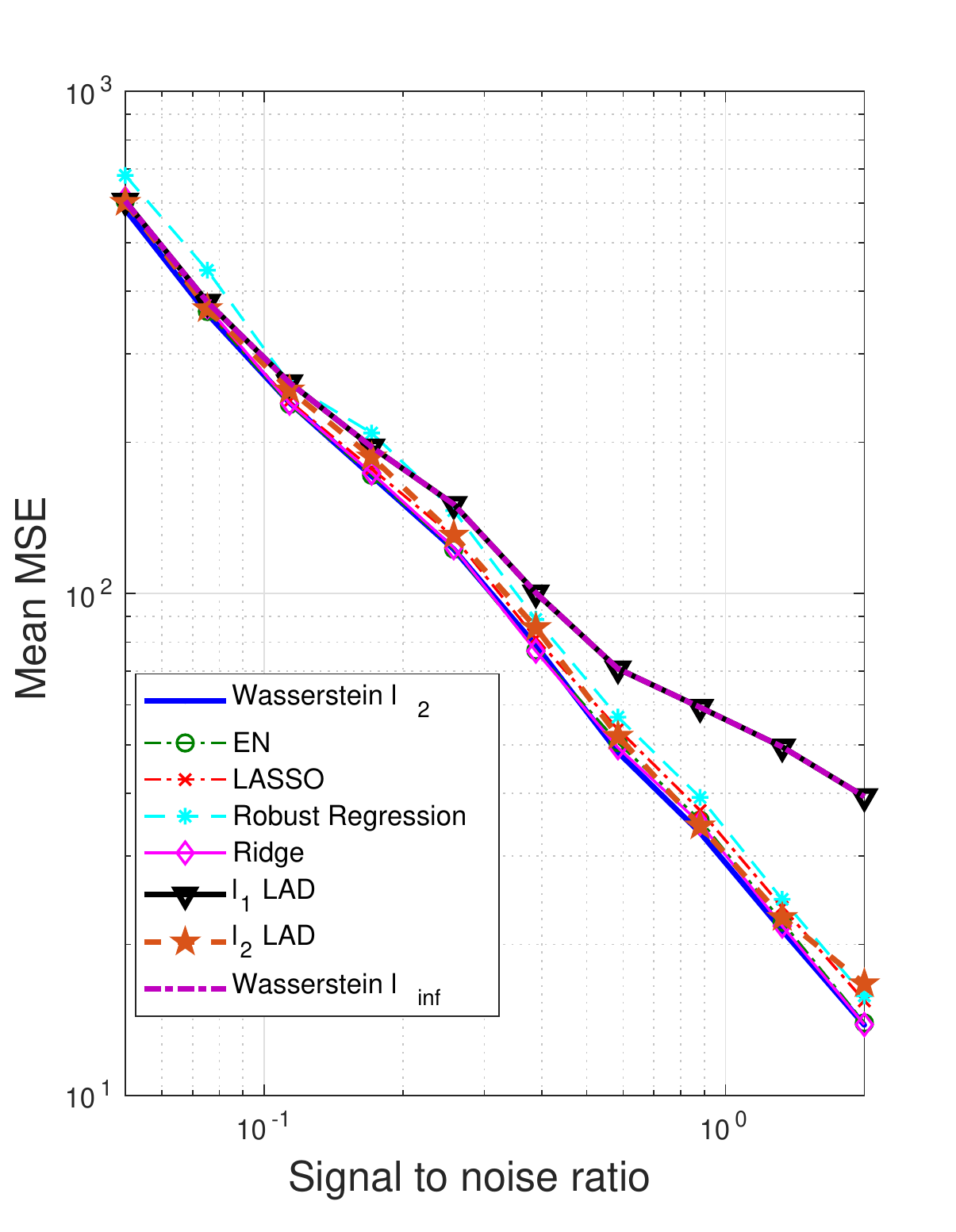}
		\caption{\small{Mean Squared Error.}}
	\end{subfigure}
	\begin{subfigure}{.49\textwidth}
		\centering
		\includegraphics[width=0.98\textwidth]{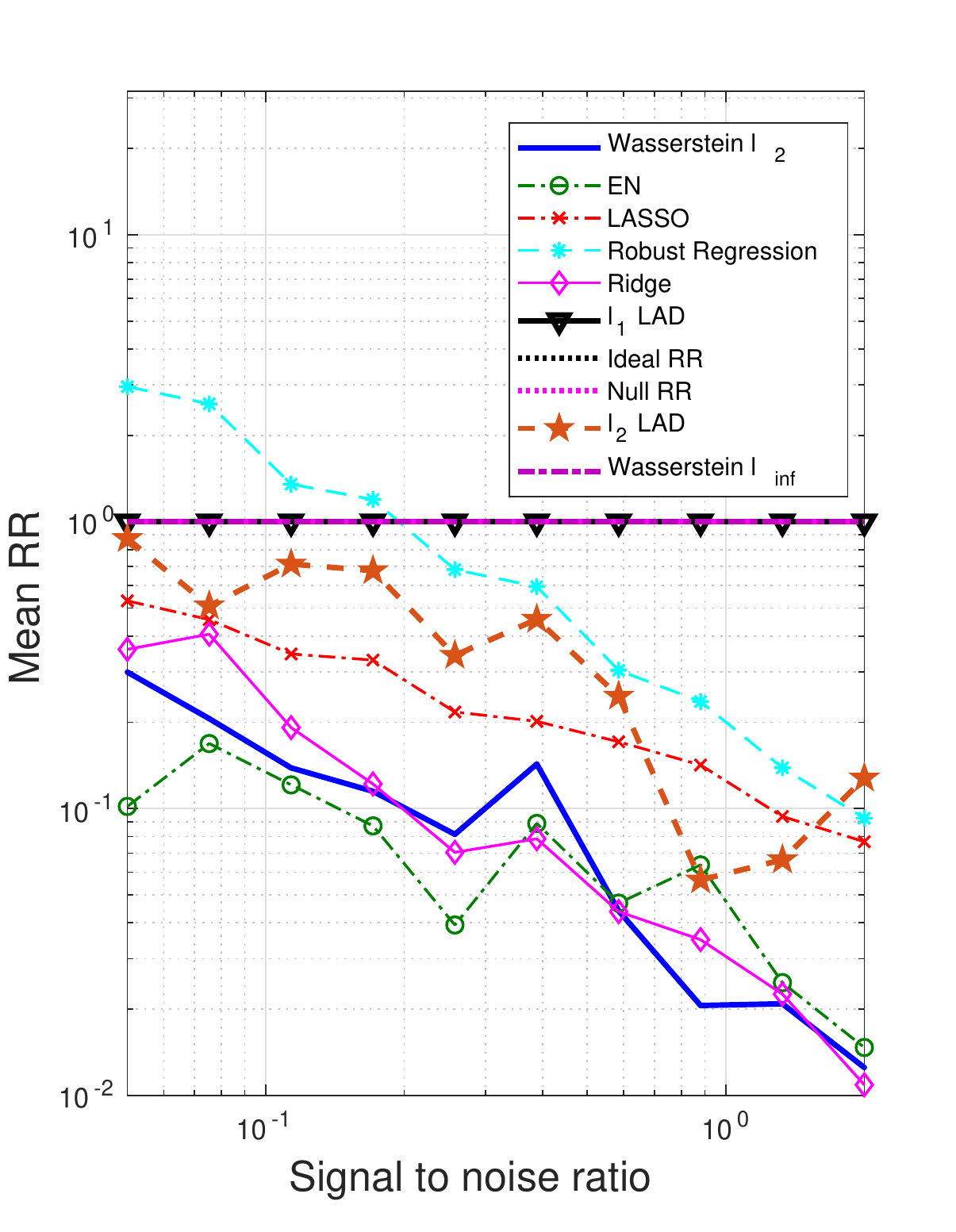}
		\caption{\small{Relative risk.}}
	\end{subfigure}
	
	\begin{subfigure}{.49\textwidth}
		\centering
		\includegraphics[width=0.98\textwidth]{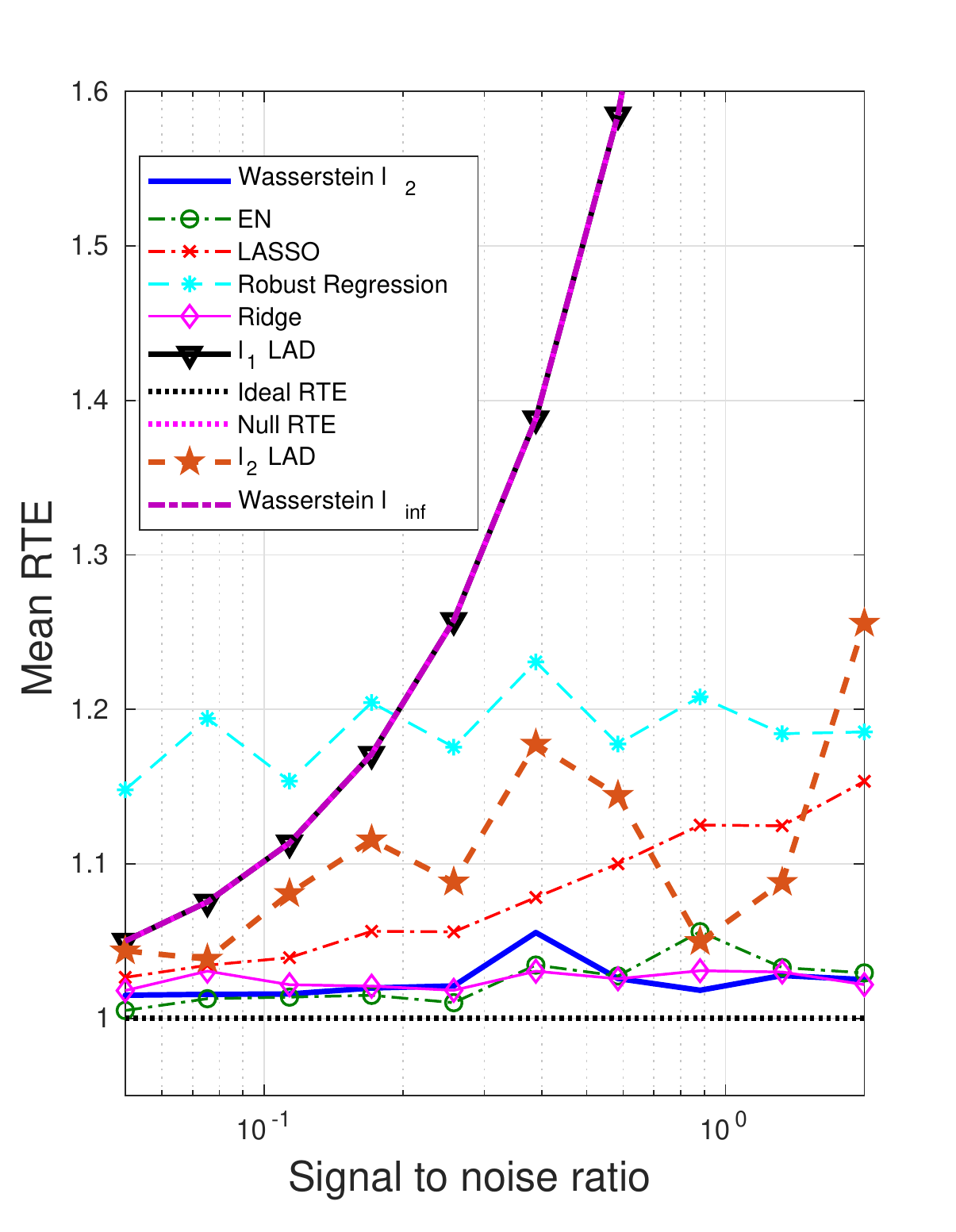}
		\caption{\small{Relative test error.}}
	\end{subfigure}%
	\begin{subfigure}{.49\textwidth}
		\centering
		\includegraphics[width=0.98\textwidth]{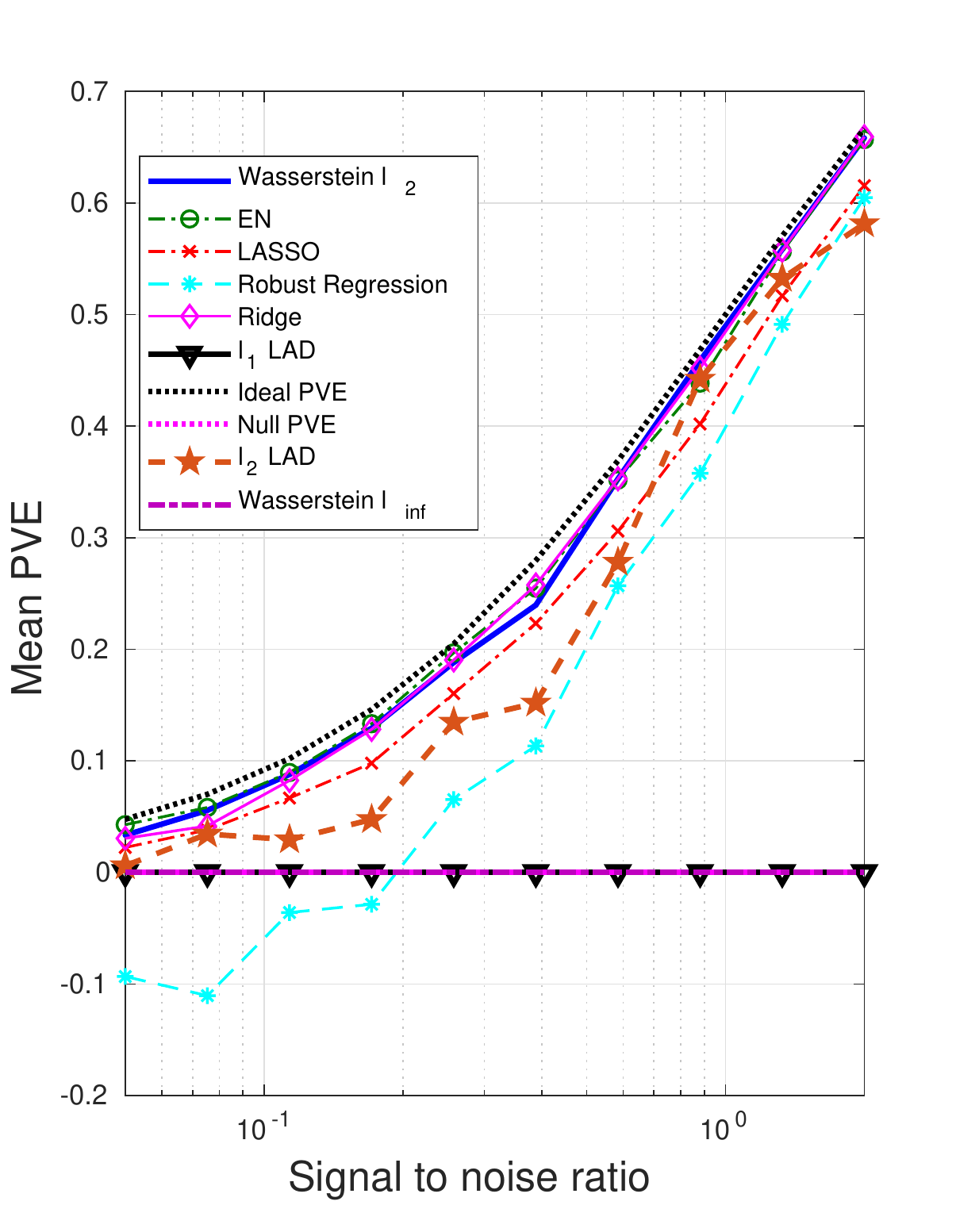}
		\caption{\small{Proportion of variance explained.}}
	\end{subfigure}
	\caption{The impact of SNR on the performance metrics: dense $\bbeta^*$, outliers only in $\bx$.}
	\label{snr-2}
\end{figure}

\begin{figure}[p] 
	\begin{subfigure}{.49\textwidth}
		\centering
		\includegraphics[width=0.98\textwidth]{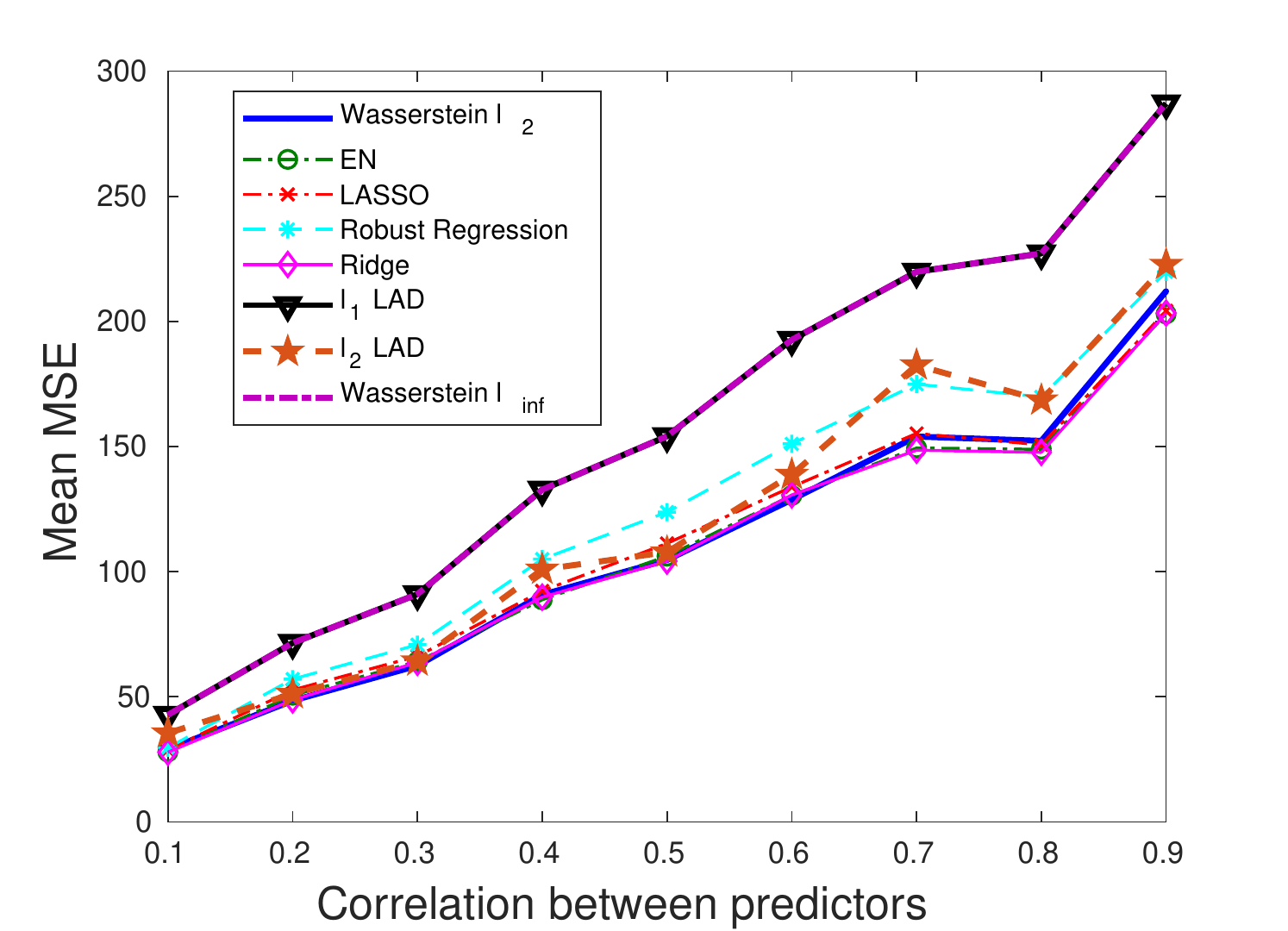}
		\caption{\small{Mean Squared Error.}}
	\end{subfigure}
	\begin{subfigure}{.49\textwidth}
		\centering
		\includegraphics[width=0.98\textwidth]{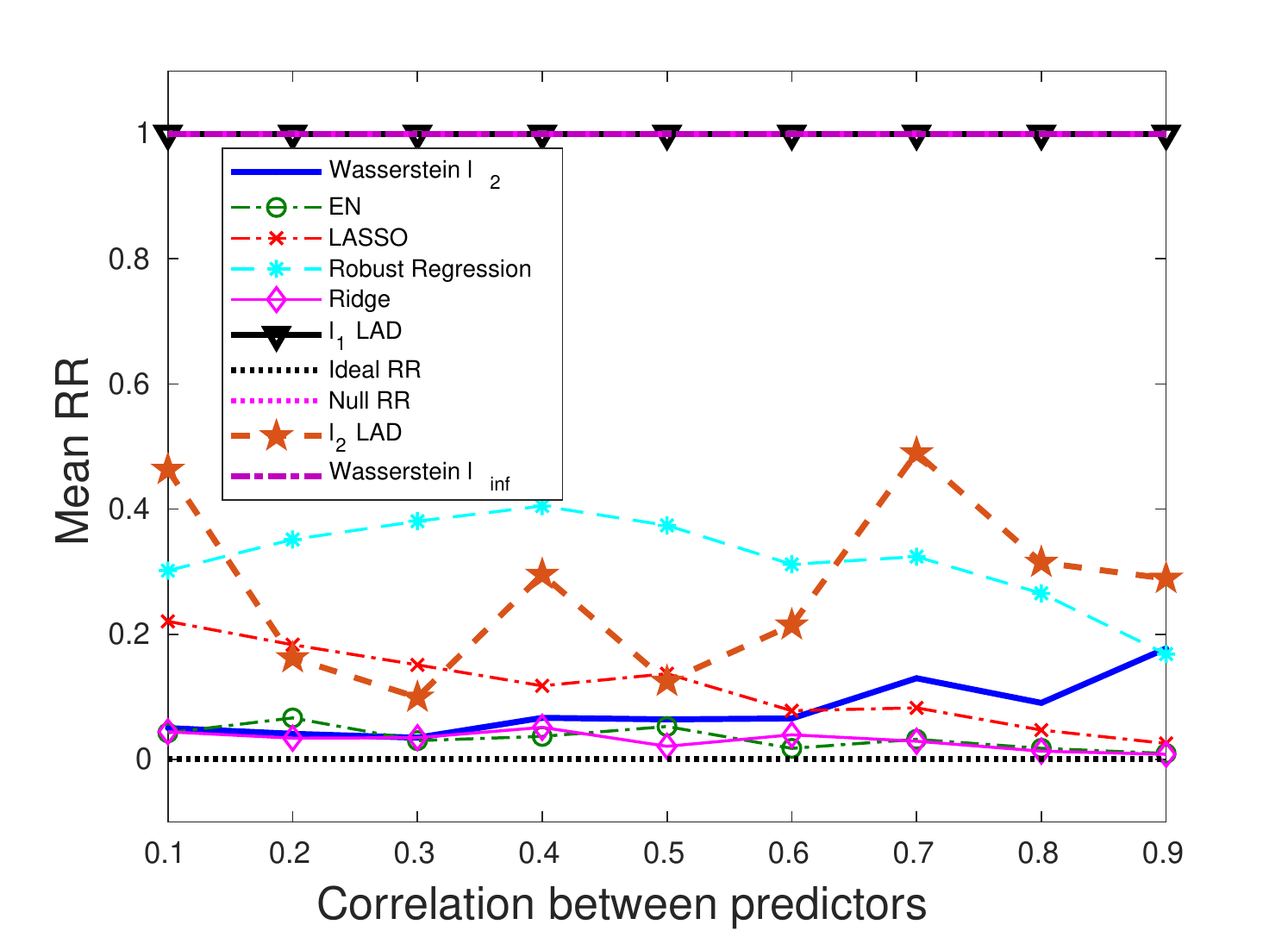}
		\caption{\small{Relative risk.}}
	\end{subfigure}
	
	\begin{subfigure}{.49\textwidth}
		\centering
		\includegraphics[width=0.98\textwidth]{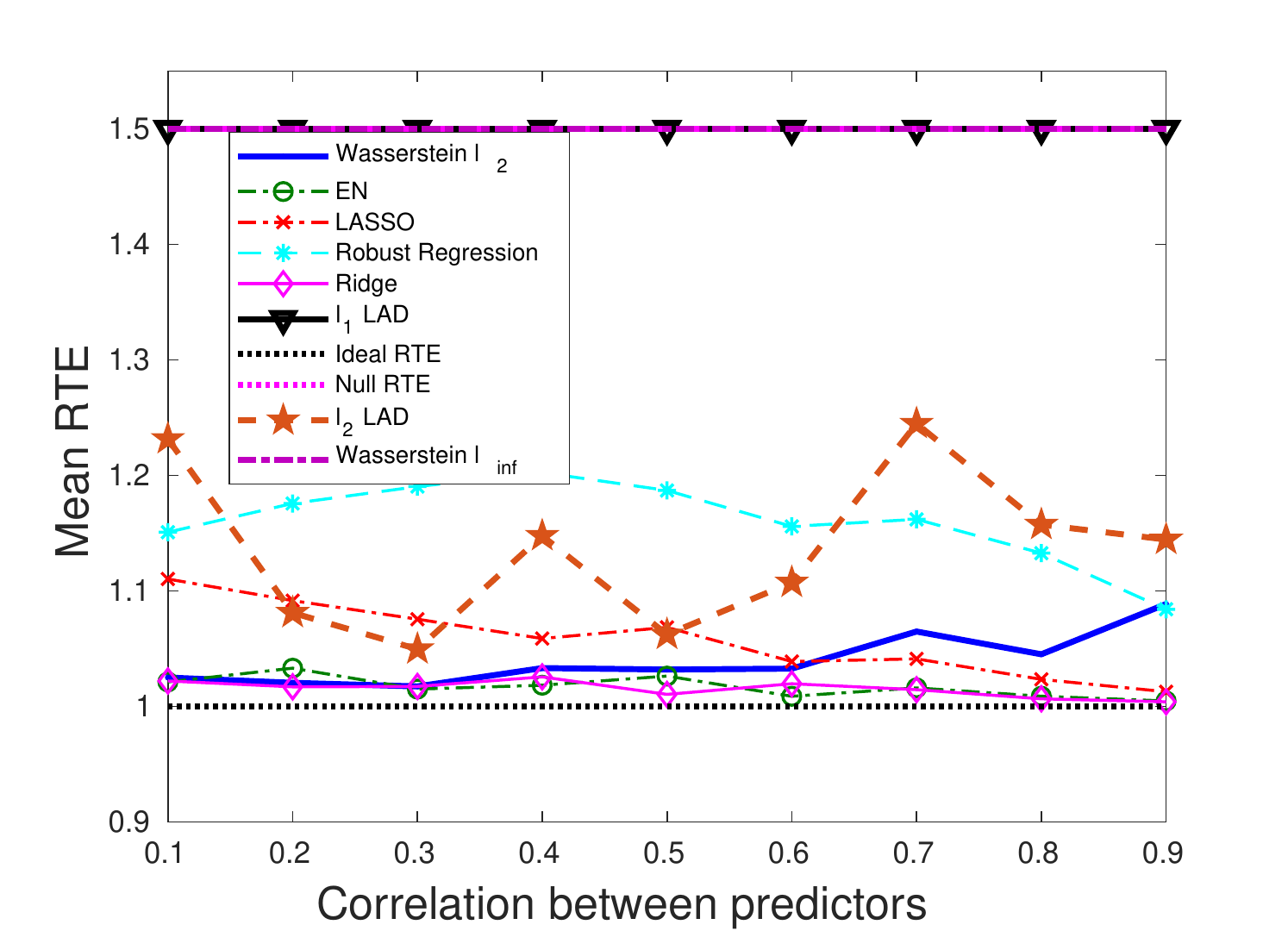}
		\caption{\small{Relative test error.}}
	\end{subfigure}%
	\begin{subfigure}{.49\textwidth}
		\centering
		\includegraphics[width=0.98\textwidth]{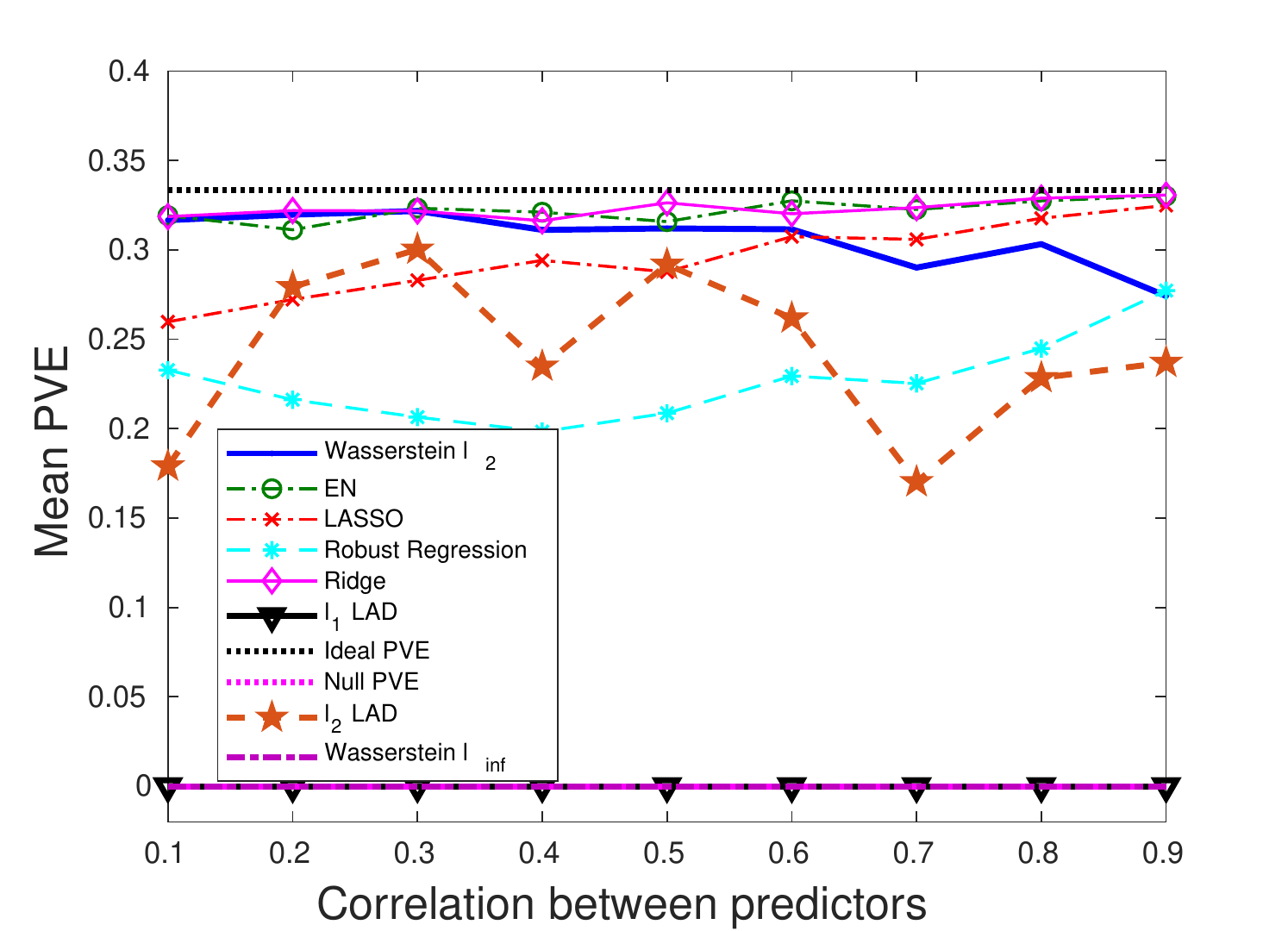}
		\caption{\small{Proportion of variance explained.}}
	\end{subfigure}
	\caption{The impact of predictor correlation on the performance metrics: dense $\bbeta^*$, outliers only in $\bx$.}
	\label{corr-2}
\end{figure}

\subsection{Sparse $\bbeta^*$, Outliers in both $\bx$ and $y$} \label{sparsexy}
In this subsection, we will experiment with a sparse $\bbeta^*$. The intercept is set to $\beta_0^* = 3$, and the coefficients for the $20$ predictors are set to $\bbeta^* = (0.05, 0, 0.006, 0, -0.007, 0, 0.008, 0, \ldots, 0)$. The perturbations are present in both $\bx$ and $y$. Specifically, the distribution of outliers is characterized by:
\begin{enumerate}
	\item $\bx \sim \scrN (\mathbf{0}, \bSigma) + \scrN (\mathbf{0}, 0.25\mathbf{I})$;
	\item $y \sim \scrN (\bx' \bbeta^*, \sigma^2) + 5\sigma$.
\end{enumerate}

Our goal is to study the impact of the sparsity of $\bbeta^*$ on the choice of the norm space for the Wasserstein metric. An intuitively appealing interpretation for the sparsity inducing property of the $\ell_1$-regularizer is made available by the Wasserstein DRO framework, which we explain as follows. The sparse regression coefficient $\bbeta^*$ implies that only a few predictors are relevant to the regression model, and thus when measuring the distance in the $(\bx, y)$ space, we need a metric that only extracts the subset of relevant predictors. The $\|\cdot\|_{\infty}$, which takes only the most influential coordinate of its argument, roughly serves this purpose. Compared to the $\|\cdot\|_2$ which takes into account all the coordinates, most of which are redundant due to the sparsity assumption, $\|\cdot\|_{\infty}$ results in a better performance, and hence, the Wasserstein $\ell_{\infty}$ formulation that induces the  $\ell_1$-regularizer is expected to outperform others.

We note that the $\ell_1$-regularized LAD achieves a similar performance, since replacing $\|\bbeta\|_1$ by $\|(-\bbeta, 1)\|_1$ only adds a constant term to the objective function. The generalization performance (mean MSE) of the AD loss-based formulations is consistently better than those with the SR loss, since the AD loss is less affected by large perturbations in $y$. Also note that choosing a wrong norm for the Wasserstein metric, e.g., the Wasserstein $\ell_2$, could lead to an enormous estimation error, whereas with a right norm space, the Wasserstein formulation is guaranteed to outperform all others. 

\begin{figure}[p] 
	\begin{subfigure}{.49\textwidth}
		\centering
		\includegraphics[width=0.9\textwidth]{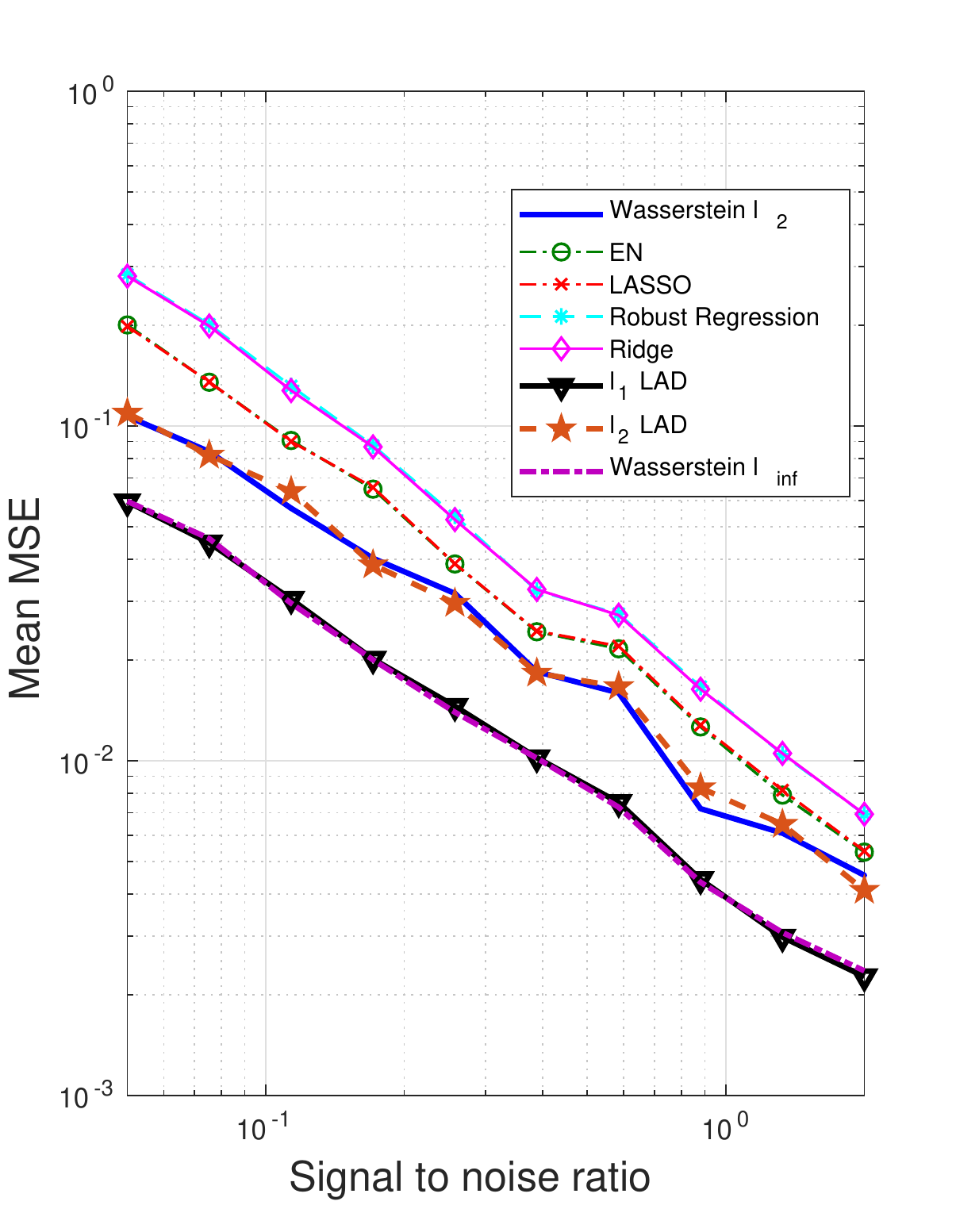}
		\caption{\small{Mean Squared Error.}}
	\end{subfigure}
	\begin{subfigure}{.49\textwidth}
		\centering
		\includegraphics[width=0.9\textwidth]{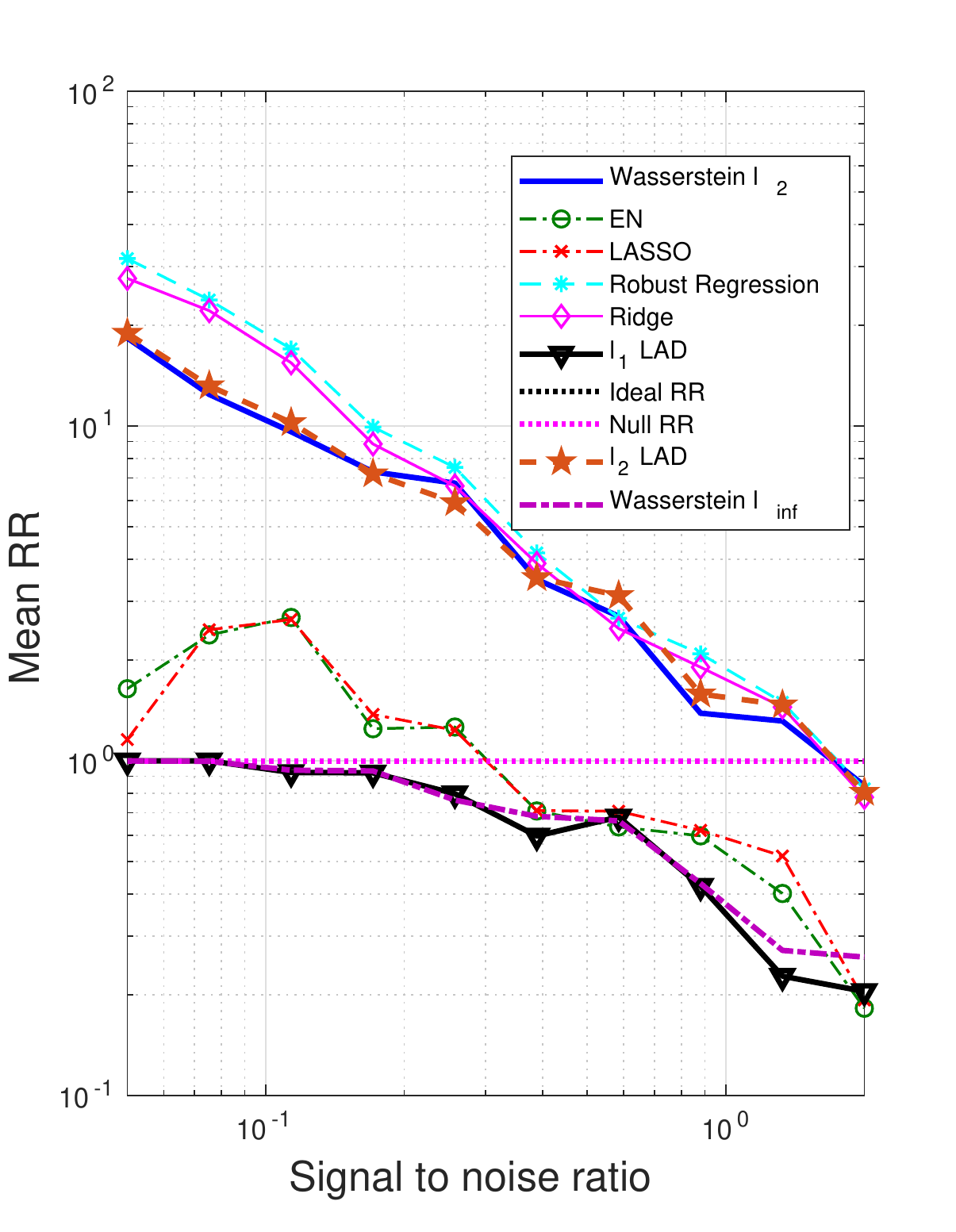}
		\caption{\small{Relative risk.}}
	\end{subfigure}
	
	\begin{subfigure}{.49\textwidth}
		\centering
		\includegraphics[width=0.9\textwidth]{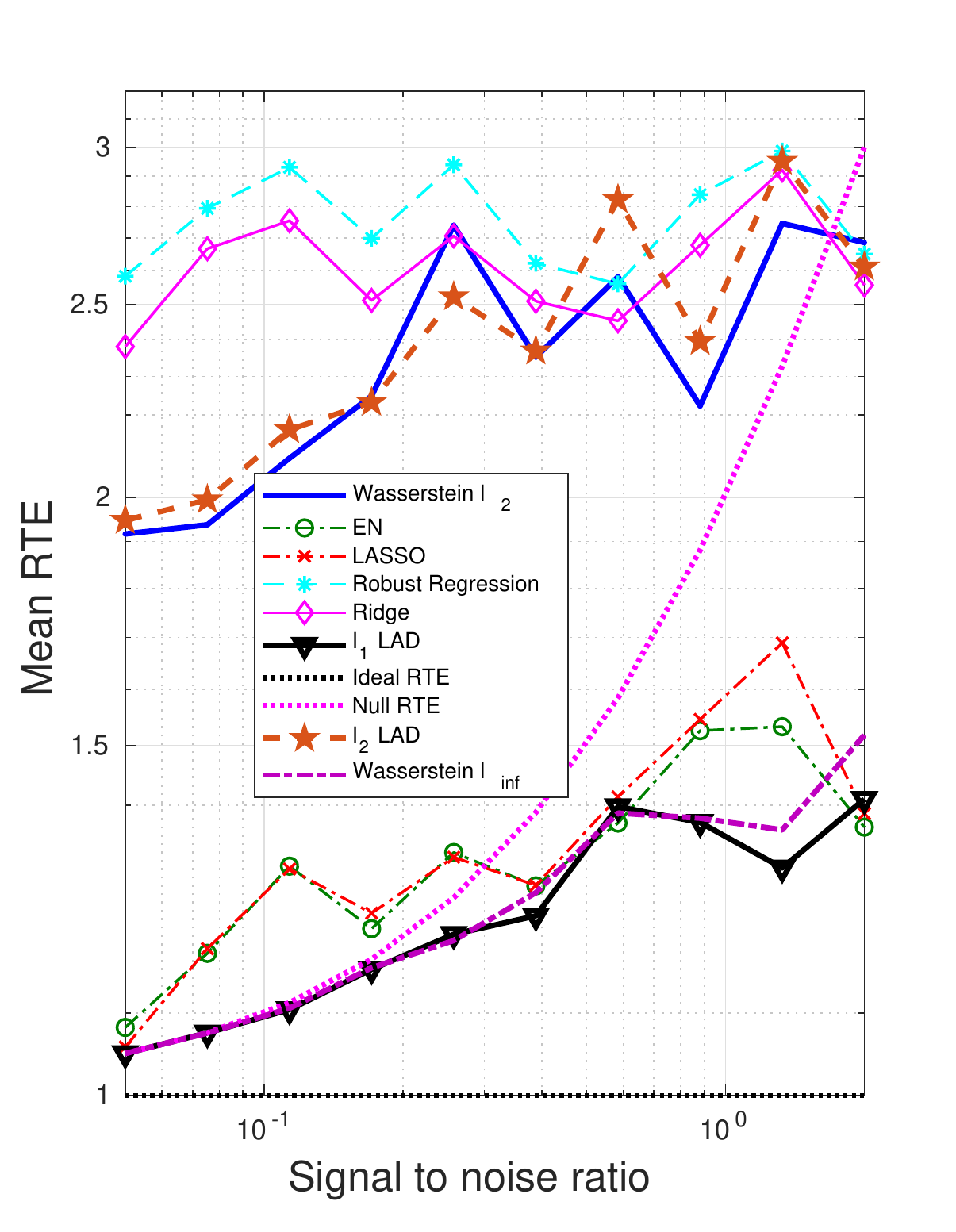}
		\caption{\small{Relative test error.}}
	\end{subfigure}%
	\begin{subfigure}{.49\textwidth}
		\centering
		\includegraphics[width=0.9\textwidth]{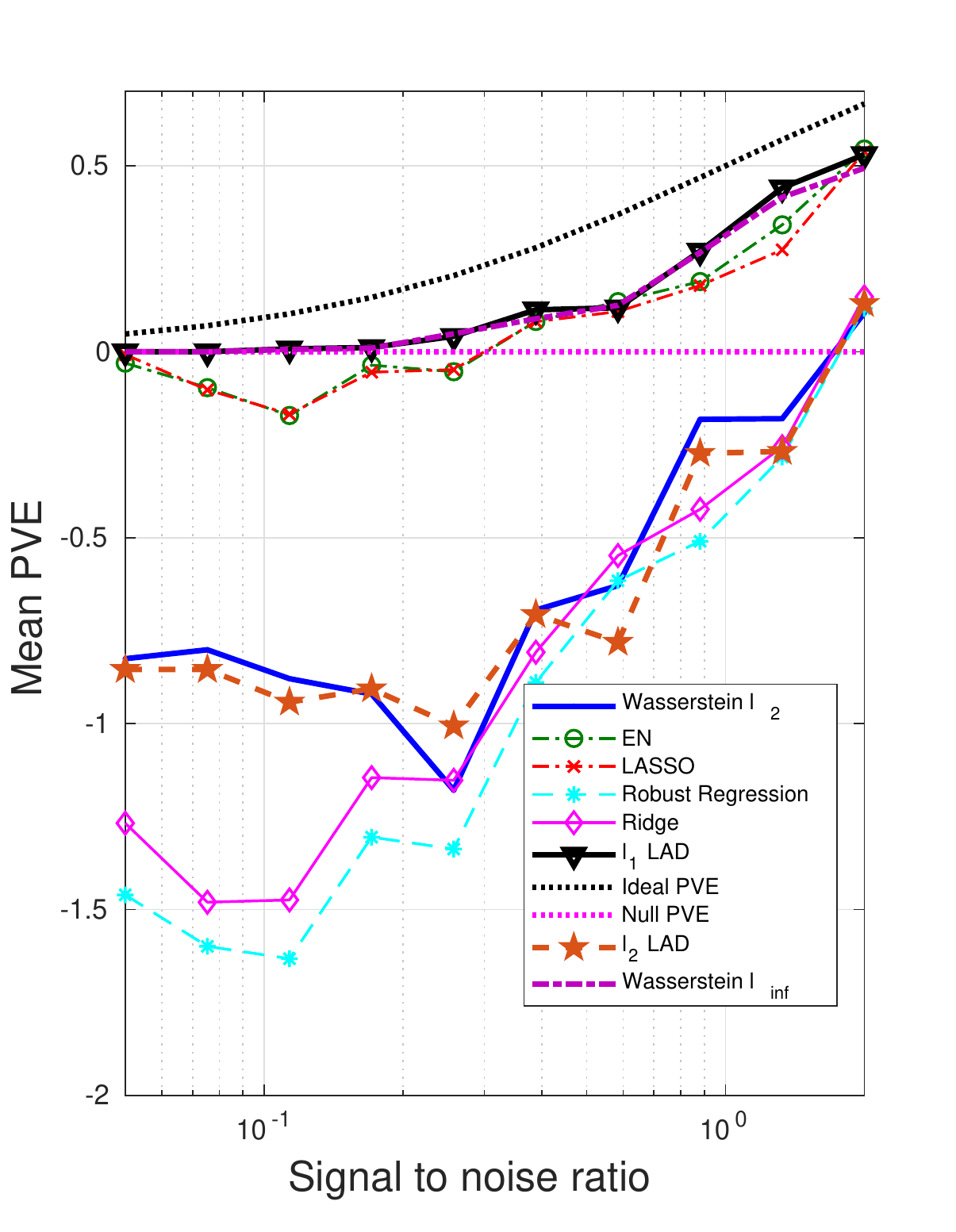}
		\caption{\small{Proportion of variance explained.}}
	\end{subfigure}
	\caption{The impact of SNR on the performance metrics: sparse $\bbeta^*$, outliers in both $\bx$ and $y$.}
	\label{snr-3}
\end{figure}

\begin{figure}[p] 
	\begin{subfigure}{.49\textwidth}
		\centering
		\includegraphics[width=0.98\textwidth]{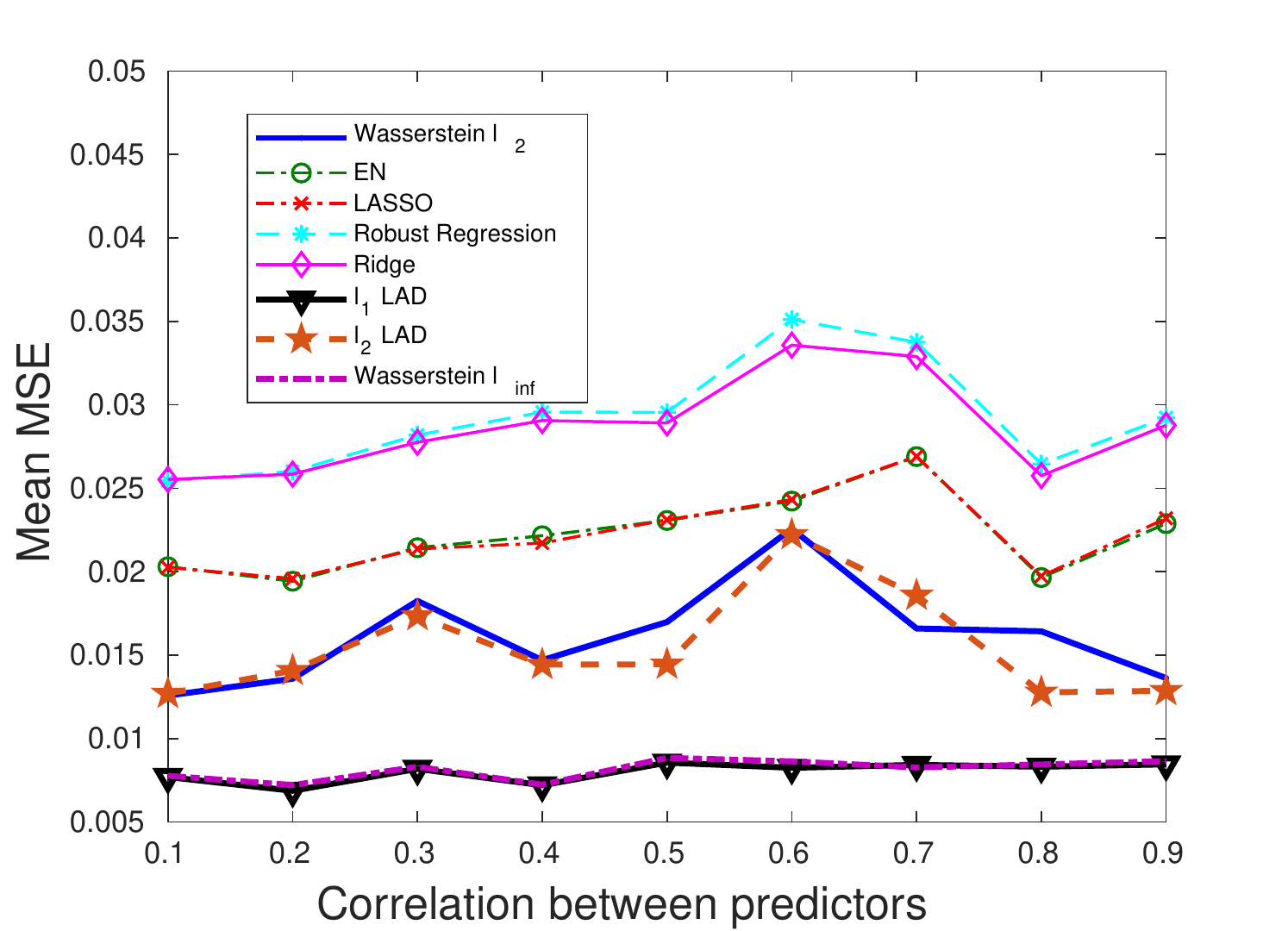}
		\caption{\small{Mean Squared Error.}}
	\end{subfigure}
	\begin{subfigure}{.49\textwidth}
		\centering
		\includegraphics[width=0.98\textwidth]{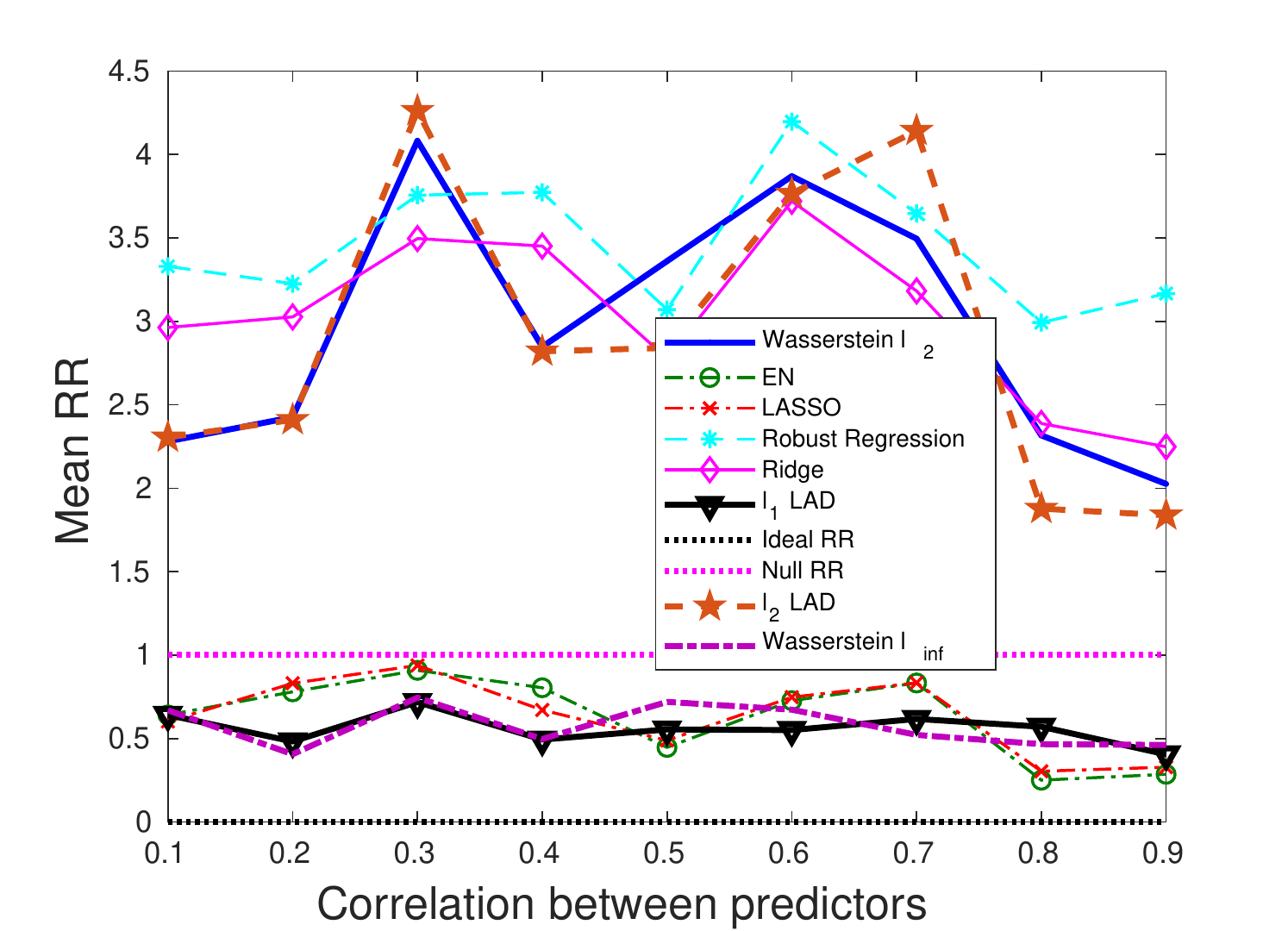}
		\caption{\small{Relative risk.}}
	\end{subfigure}
	
	\begin{subfigure}{.49\textwidth}
		\centering
		\includegraphics[width=0.98\textwidth]{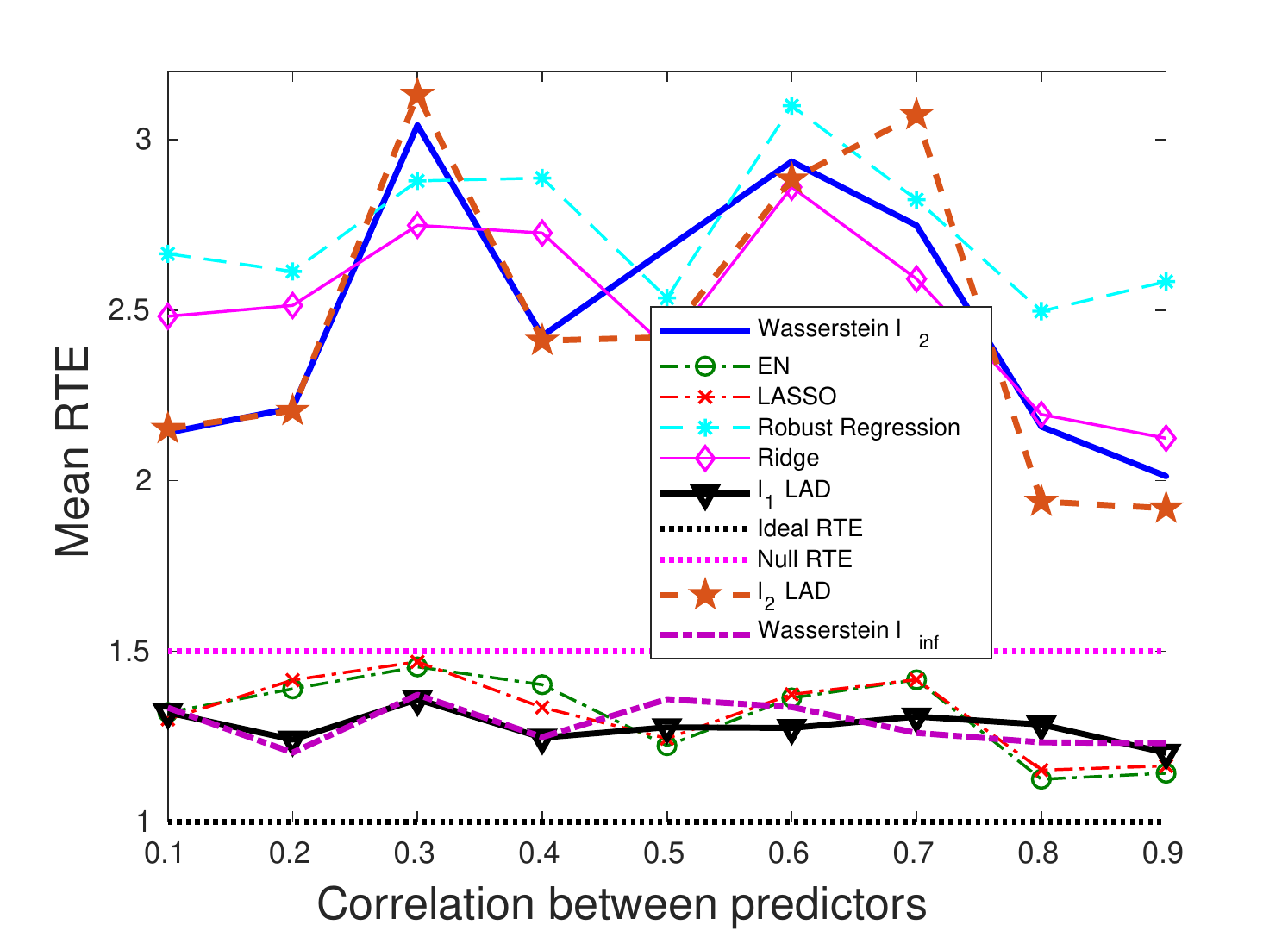}
		\caption{\small{Relative test error.}}
	\end{subfigure}%
	\begin{subfigure}{.49\textwidth}
		\centering
		\includegraphics[width=0.98\textwidth]{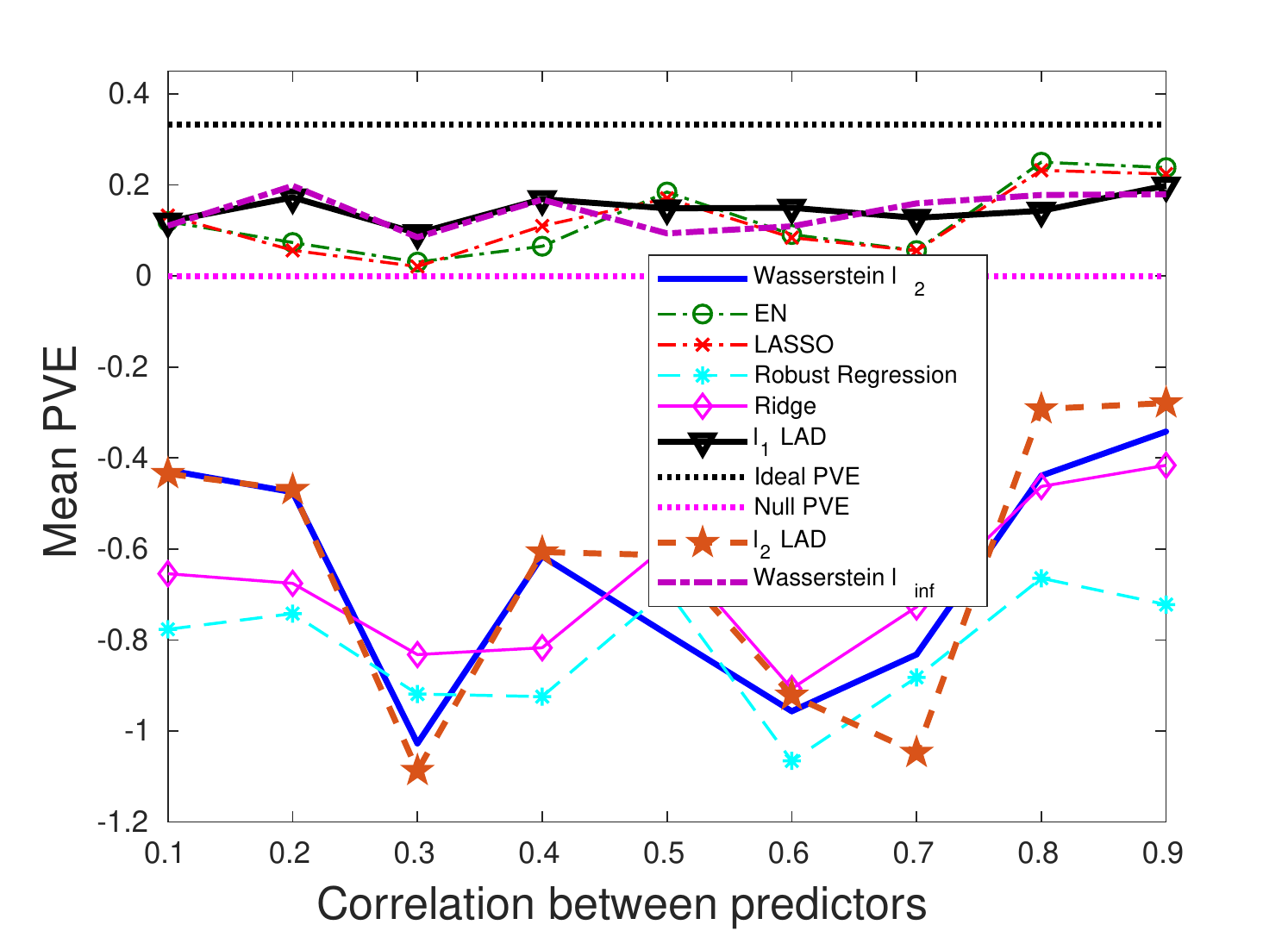}
		\caption{\small{Proportion of variance explained.}}
	\end{subfigure}
	\caption{The impact of predictor correlation on the performance metrics: sparse $\bbeta^*$, outliers in both $\bx$ and $y$.}
	\label{cor-3}
\end{figure}

\subsection{Sparse $\bbeta^*$, Outliers Only in $\bx$} \label{sparsex}
In this subsection, we will use the same sparse coefficient as in Section~\ref{sparsexy}, but the perturbations are present only in $\bx$. Specifically, for outliers, their predictors and responses are drawn from the following distributions:
\begin{enumerate}
	\item $\bx \sim \scrN (\mathbf{0}, \bSigma) + \scrN (5\mathbf{e}, \mathbf{I})$;
	\item $y \sim \scrN (\bx' \bbeta^*, \sigma^2)$. 
\end{enumerate}

Not surprisingly, the Wasserstein $\ell_{\infty}$ and the $\ell_1$-regularized LAD achieve the best performance. Notice that in Section~\ref{sparsexy}, where perturbations appear in both $\bx$ and $y$, the AD loss-based formulations have smaller generalization and estimation errors than the SR loss-based formulations. When we reduce the variation in $y$, the SR loss seems superior to the AD loss, if we restrict attention to the improperly regularized ($\ell_2$-regularizer) formulations (see Figure~\ref{snr-4}). For the $\ell_1$-regularized formulations, the Wasserstein $\ell_{\infty}$ formulation, as well as the $\ell_1$-regularized LAD, is comparable with the EN and LASSO.   

We summarize below our main findings from all sets of experiments we have presented:
\begin{enumerate}
	\item When a proper norm space is selected for the Wasserstein metric, the Wasserstein DRO formulation outperforms all others in terms of the generalization and estimation qualities.
	\item Penalizing the extended regression coefficient $(-\bbeta, 1)$ implicitly assumes a more reasonable distance metric on $(\bx, y)$ and thus leads to a better performance.
	\item The AD loss is remarkably superior to the SR loss when there is large variation in the response $y$.
	\item The Wasserstein DRO formulation shows a more stable estimation performance than others when the correlation between predictors is varied.
\end{enumerate}

\begin{figure}[p] 
	\begin{subfigure}{.49\textwidth}
		\centering
		\includegraphics[width=0.9\textwidth]{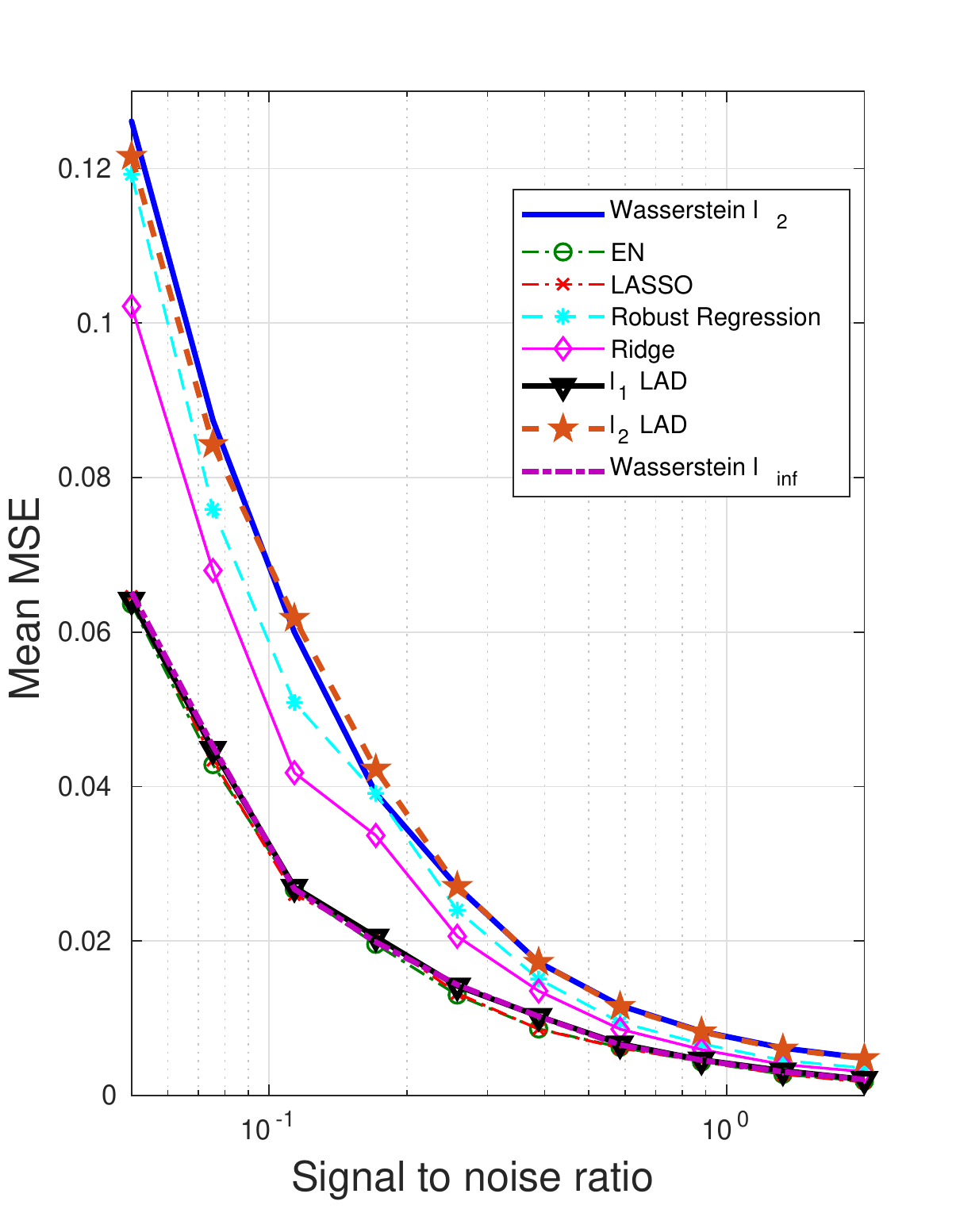}
		\caption{\small{Mean Squared Error.}}
	\end{subfigure}
	\begin{subfigure}{.49\textwidth}
		\centering
		\includegraphics[width=0.9\textwidth]{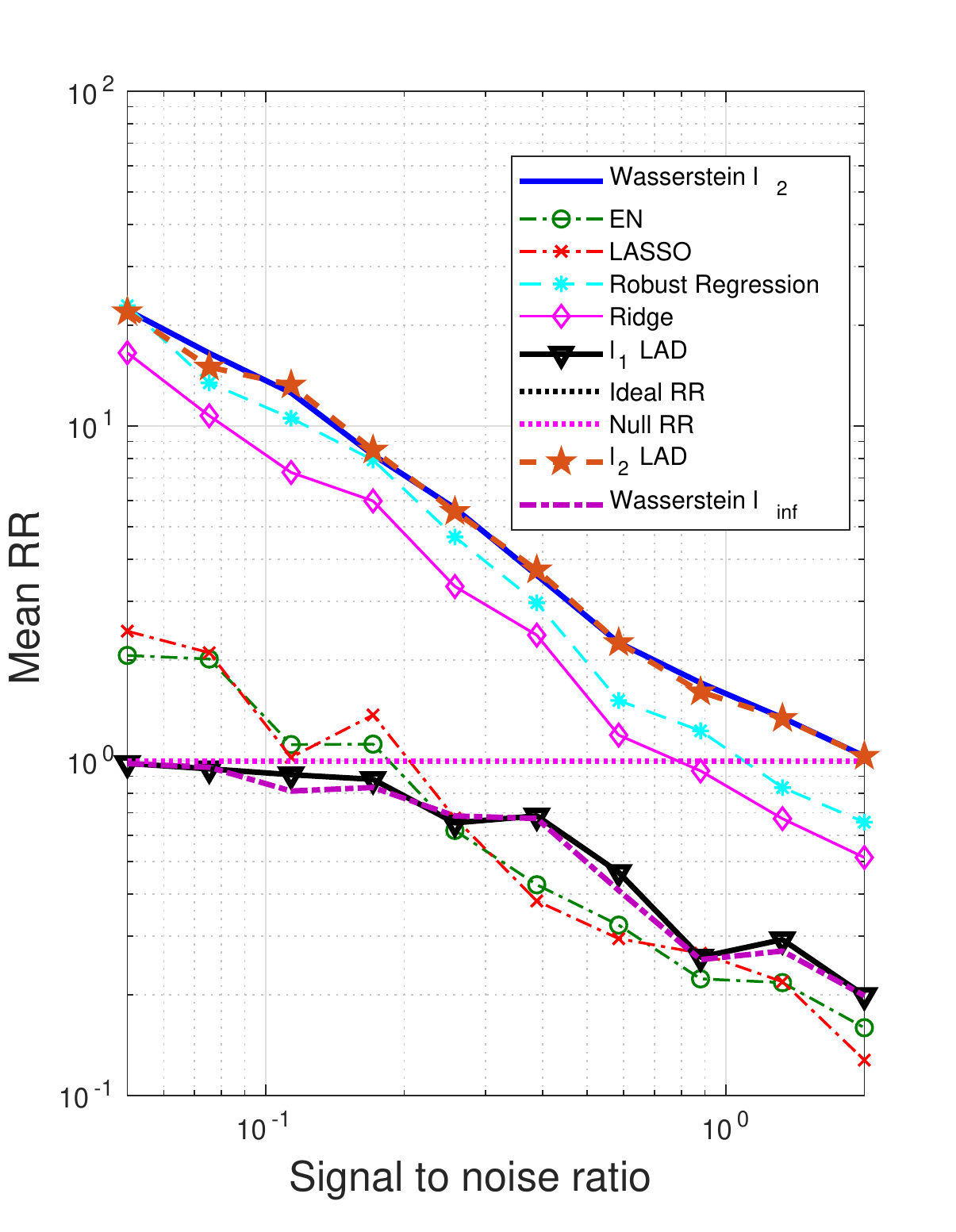}
		\caption{\small{Relative risk.}}
	\end{subfigure}
	
	\begin{subfigure}{.49\textwidth}
		\centering
		\includegraphics[width=0.9\textwidth]{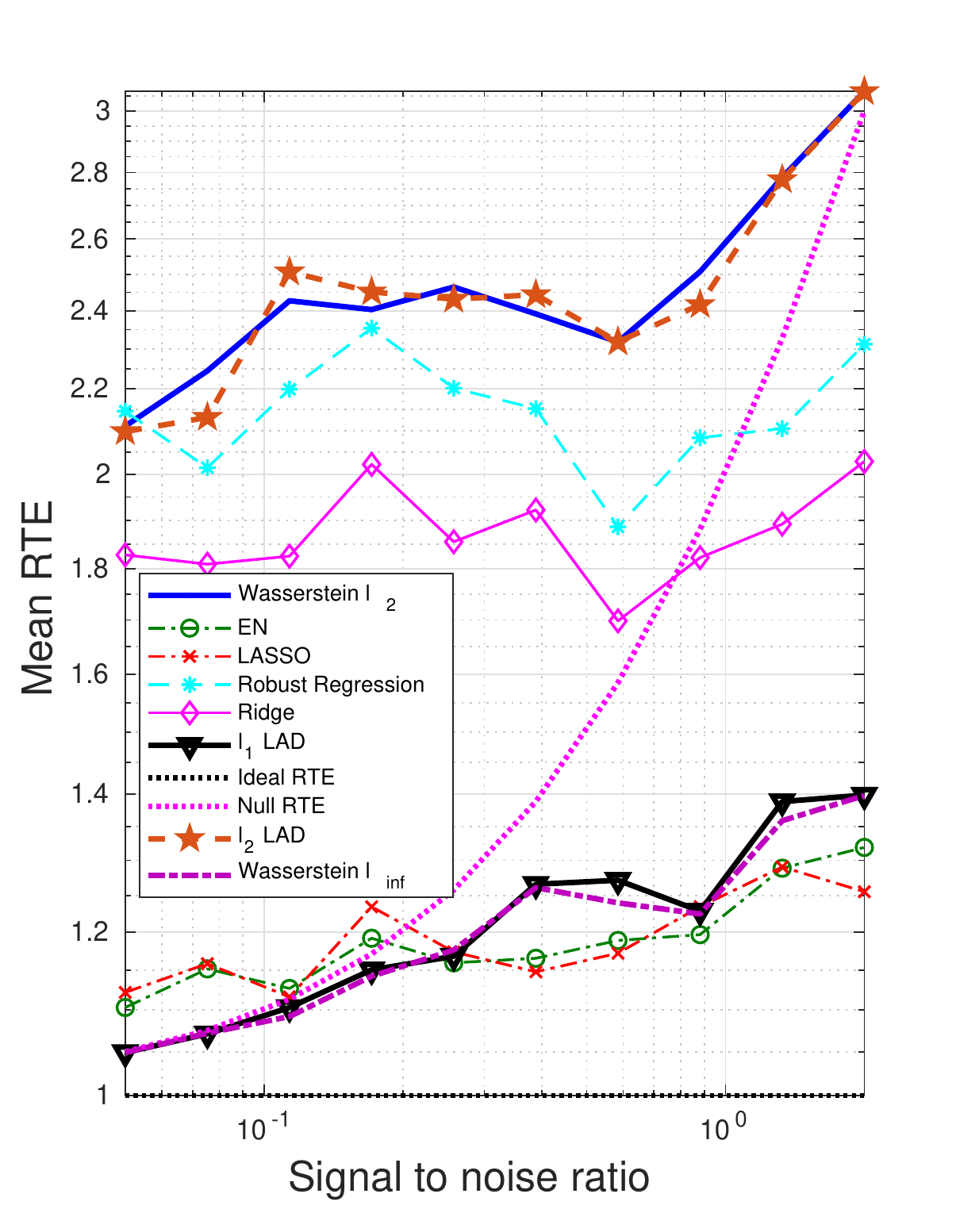}
		\caption{\small{Relative test error.}}
	\end{subfigure}%
	\begin{subfigure}{.49\textwidth}
		\centering
		\includegraphics[width=0.9\textwidth]{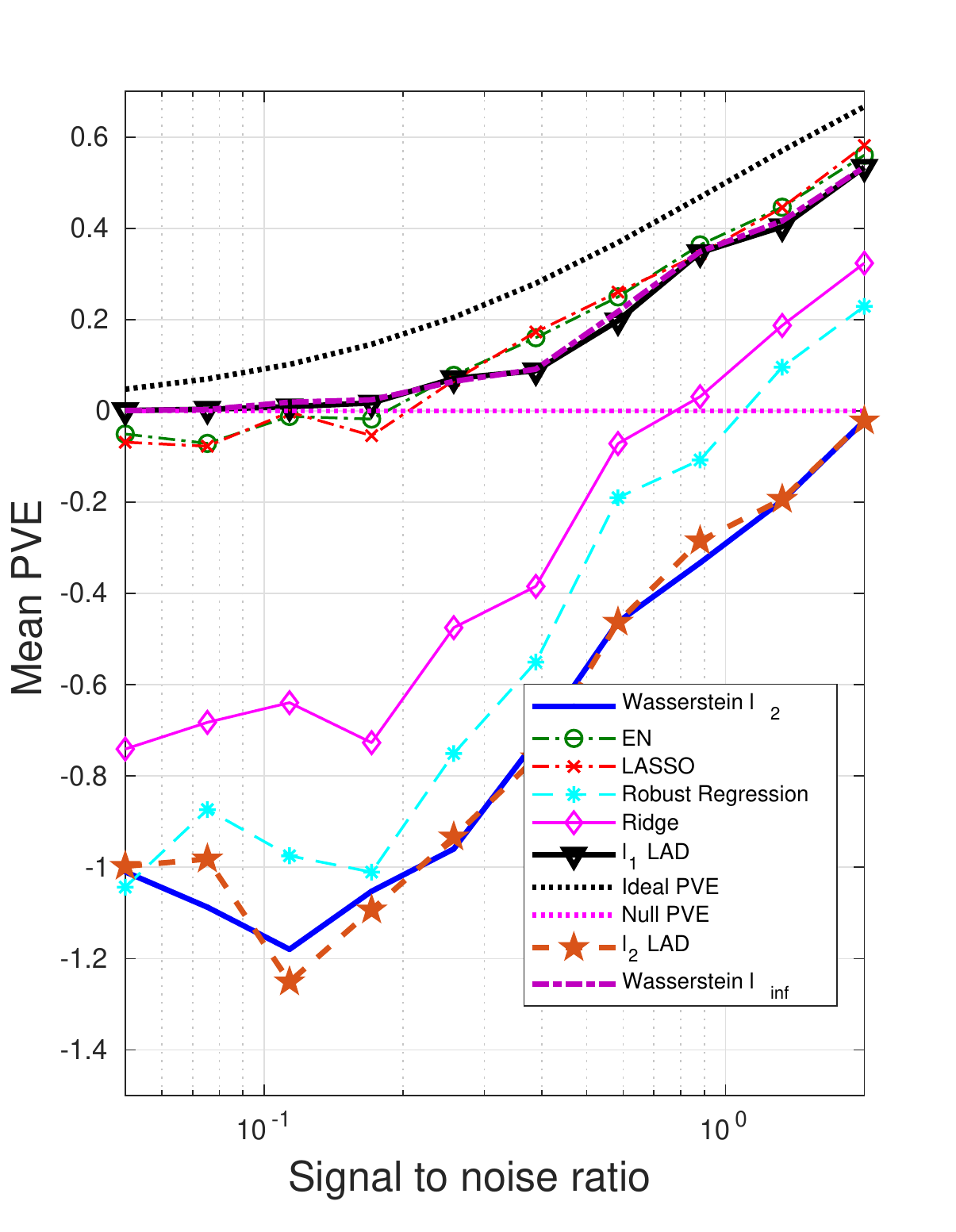}
		\caption{\small{Proportion of variance explained.}}
	\end{subfigure}
	\caption{The impact of SNR on the performance metrics: sparse $\bbeta^*$, outliers only in $\bx$.}
	\label{snr-4}
\end{figure}

\begin{figure}[p] 
	\begin{subfigure}{.49\textwidth}
		\centering
		\includegraphics[width=0.98\textwidth]{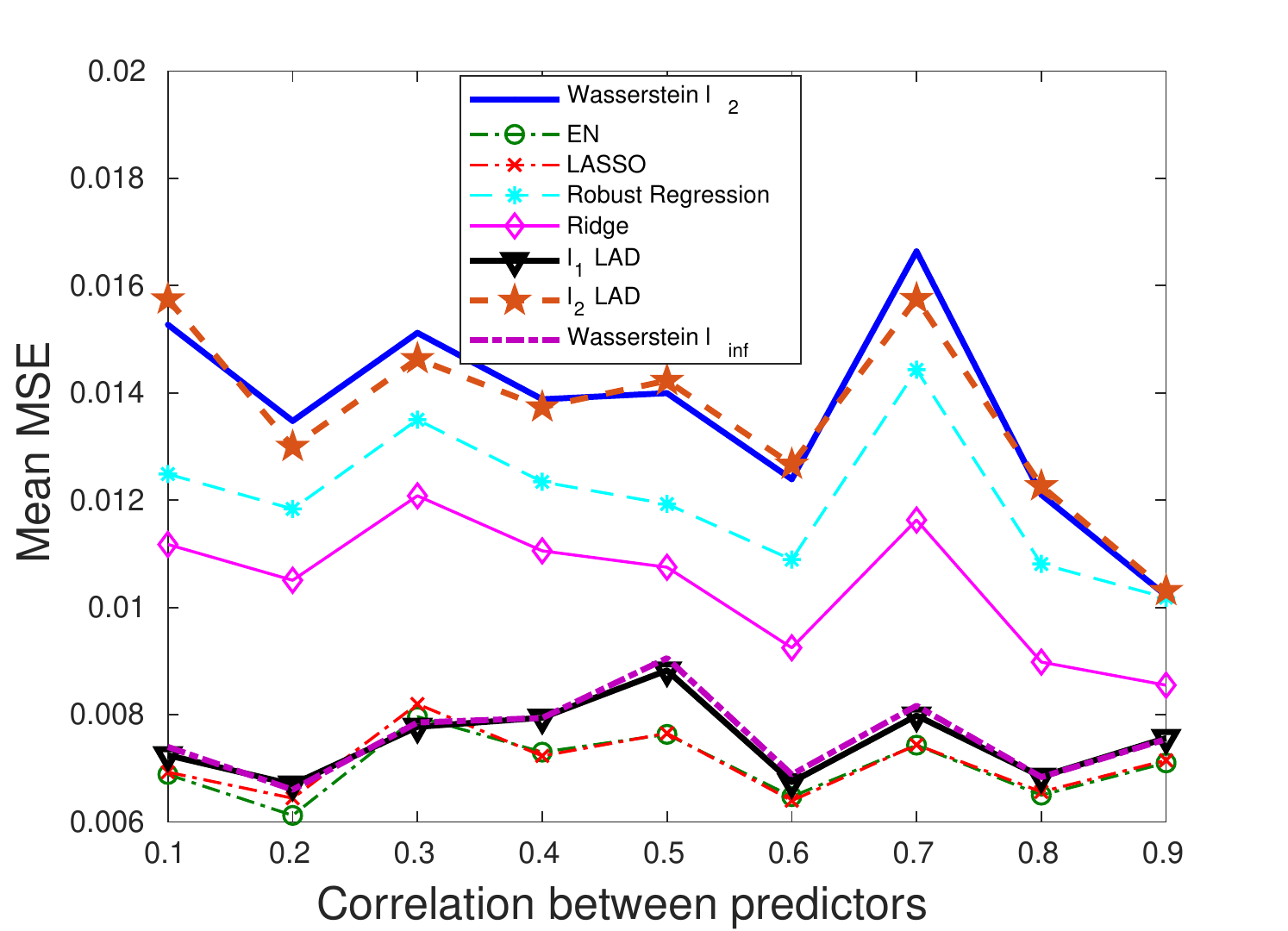}
		\caption{\small{Mean Squared Error.}}
	\end{subfigure}
	\begin{subfigure}{.49\textwidth}
		\centering
		\includegraphics[width=0.98\textwidth]{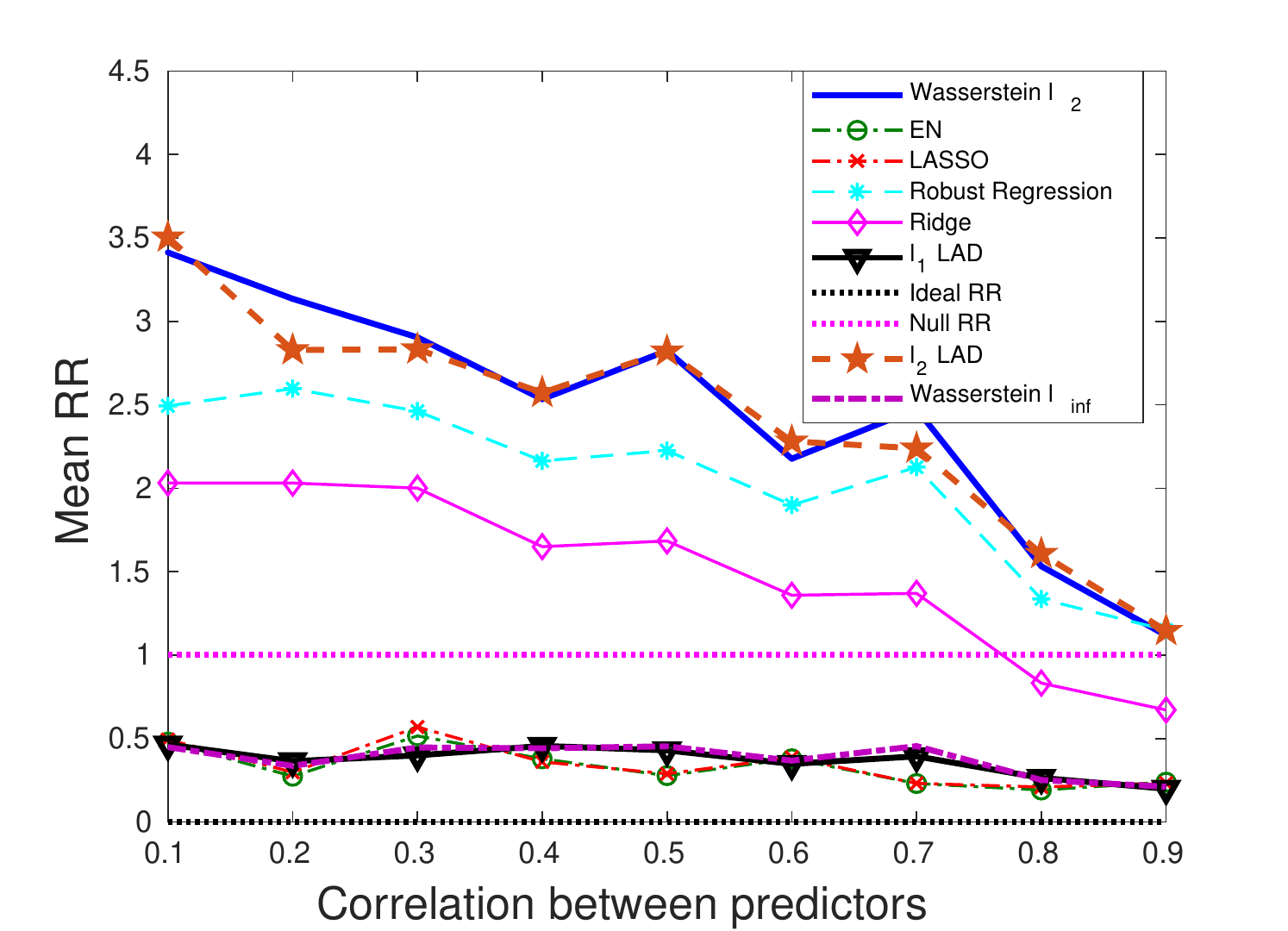}
		\caption{\small{Relative risk.}}
	\end{subfigure}
	
	\begin{subfigure}{.49\textwidth}
		\centering
		\includegraphics[width=0.98\textwidth]{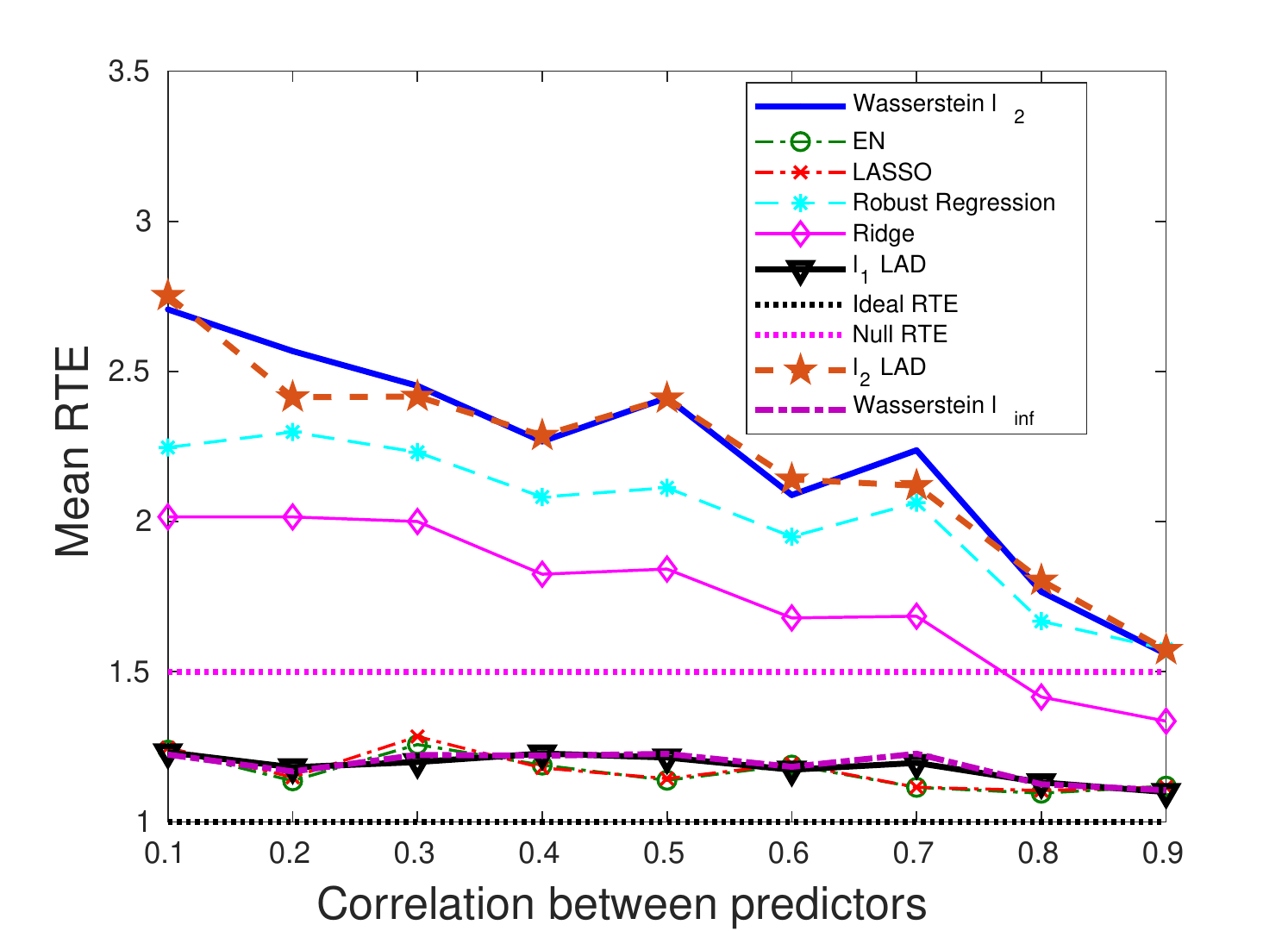}
		\caption{\small{Relative test error.}}
	\end{subfigure}%
	\begin{subfigure}{.49\textwidth}
		\centering
		\includegraphics[width=0.98\textwidth]{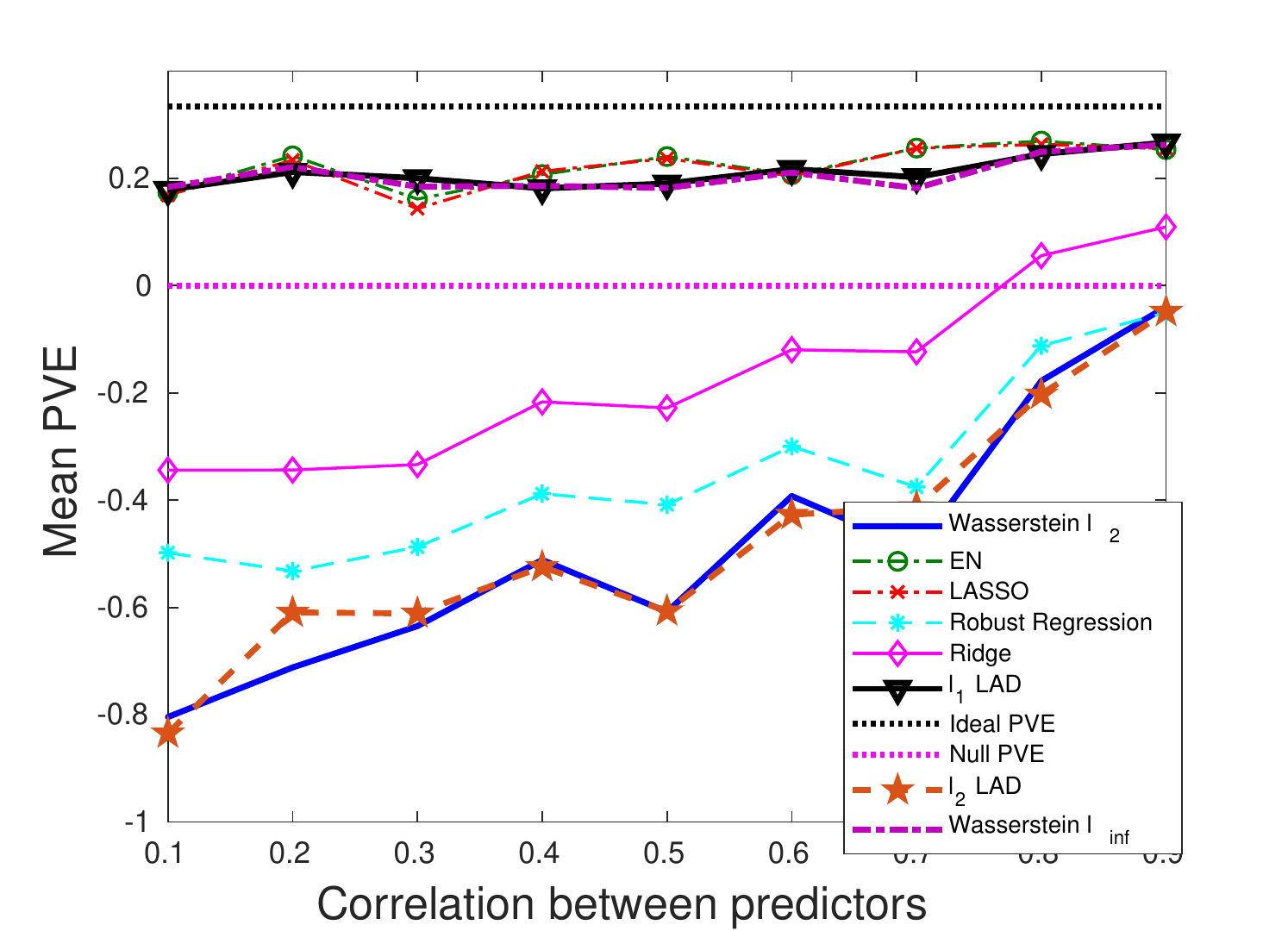}
		\caption{\small{Proportion of variance explained.}}
	\end{subfigure}
	\caption{The impact of predictor correlation on the performance metrics: sparse $\bbeta^*$, outliers only in $\bx$.}
	\label{cor-4}
\end{figure}

\section{An Application of Wasserstein DRO to Outlier Detection} \label{sec:2-5}
As an application, we consider an unlabeled two-class classification problem, where the goal is to identify the abnormal class of data points based on the predictor and response information using the Wasserstein formulation. We do not know a priori whether the samples are normal or abnormal, and thus classification models do not apply. The commonly used regression model for this type of problem is the M-estimation \citep{huber1964robust, huber1973robust}, against which we will compare in terms of the outlier detection capability. 

\subsection{Experiments on Synthetic Data}
We first report results on synthetic datasets that consist of a mixture of clean and outlying examples. For clean samples, all predictors $x_i, i \in \lb 30 \rb$,
come from a normal distribution with mean $7.5$ and standard deviation
$4.0$. The response is a linear function of the predictors with $\beta_0^*=0.3$,
$\beta_1^*=\cdots=\beta_{30}^*=0.5$, plus a Gaussian distributed noise term with zero mean and standard deviation $\sigma$. The outliers concentrate in a cloud that is randomly placed in the
interior of the $\bx$-space. Specifically, their predictors are
uniformly distributed on $(u-0.125, u+0.125)$, where $u$ is a uniform
random variable on $(7.5-3\times4, 7.5+3\times4)$. The response values of the outliers
are at a $\delta_R$ distance off the regression plane.
\begin{equation*}
y=\beta_0^*+\beta_1^*x_1+\cdots+\beta_{30}^*x_{30}+\delta_R.
\end{equation*}

We will compare the performance of the Wasserstein $\ell_2$ formulation (\ref{qcp}) with the
$\ell_1$-regularized LAD and M-estimation with three cost functions
-- Huber \citep{huber1964robust,huber1973robust}, Talwar
\citep{hinich1975simple}, and Fair \citep{fair1974robust}. The performance metrics include
the {\em Receiver Operating Characteristic (ROC)} curve which
plots the true positive rate against the false positive rate, and the related {\em Area Under Curve (AUC)}. 

Notice that all the regression methods under consideration only generate an estimated regression coefficient. The
identification of outliers is based on the residual and estimated
standard deviation of the noise. Specifically,
\begin{equation*}
\text{Outlier}=
\begin{cases}
\text{YES,}&\text{if $|\text{residual}|> \text{threshold}\times\hat{\sigma}$},\\
\text{NO,}&\text{otherwise},
\end{cases}
\end{equation*} 
where $\hat{\sigma}$ is the standard deviation of residuals in the
entire training set. ROC curves are obtained through adjusting the
threshold value. 

The regularization parameters for Wasserstein DRO and regularized LAD are tuned using a separate validation set as done in previous sections. We would like to highlight a salient advantage of the Wasserstein DRO model reflected in its robustness w.r.t. the choice of $\epsilon$. In Figure~\ref{radius} we plot the out-of-sample AUC as the radius $\epsilon$ (regularization parameter) varies, for the $\ell_2$-induced Wasserstein DRO and the $\ell_1$-regularized LAD. For the
Wasserstein DRO curve, when
$\epsilon$ is small, the Wasserstein ball contains the true distribution
with low confidence and thus AUC is low. On the other hand, too large
$\epsilon$ makes the solution overly conservative. 
Note that the robustness of the Wasserstein DRO,
indicated by the flatness of the curve, constitutes
another advantage, whereas the performance of LAD dramatically
deteriorates once the regularizer deviates from the optimum. Moreover,
the maximal achievable AUC for Wasserstein DRO is significantly higher
than LAD.
\begin{figure}[h]
	\centering
	\includegraphics[height = 2.5in]{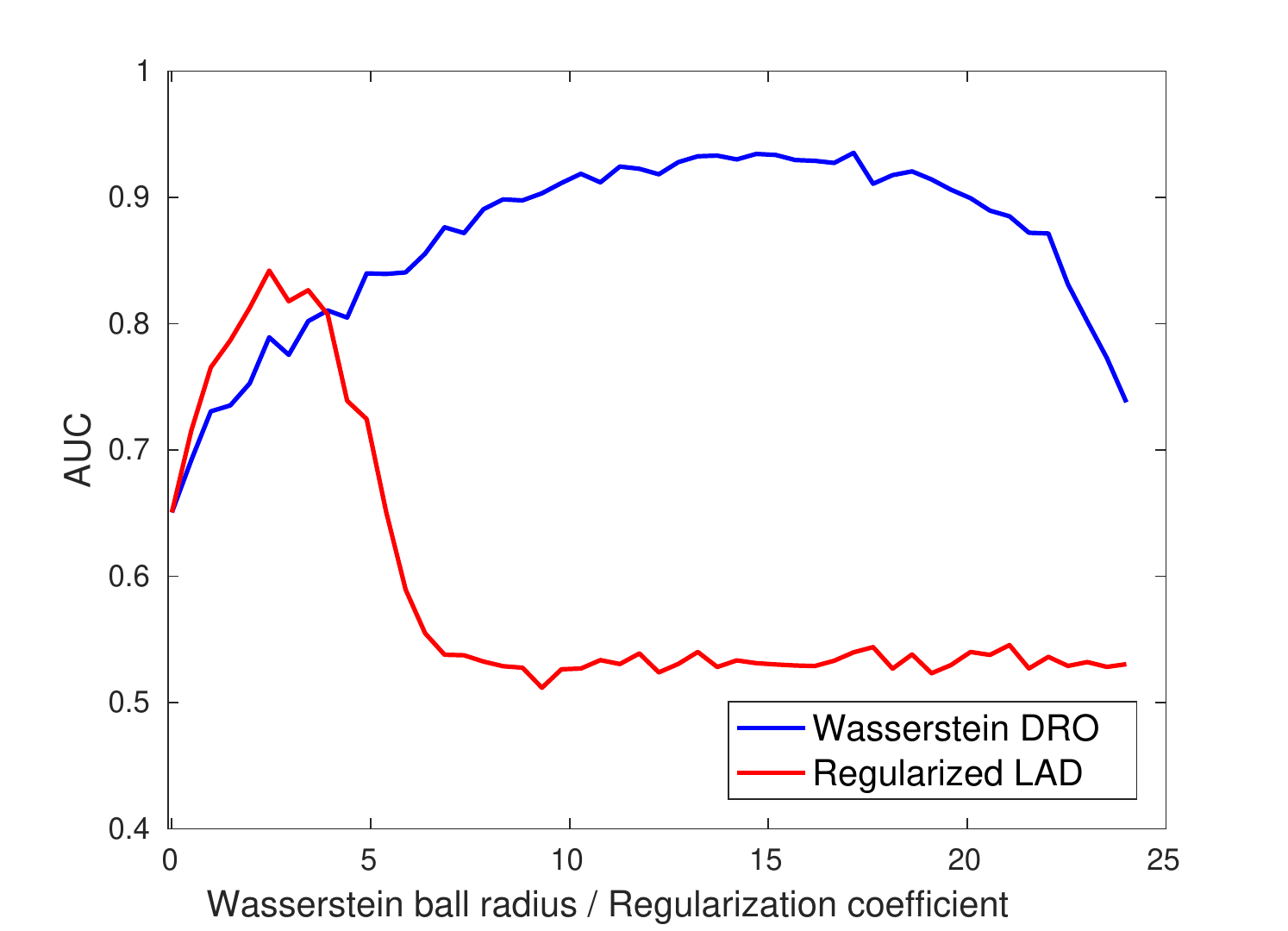}
	\caption{Out-of-sample AUC v.s.\ Wasserstein ball radius (regularization coefficient).}
	\label{radius}
\end{figure}

In Figure~\ref{f15} we show the ROC curves for different approaches, where $q$ represents the percentage of outliers, and $\delta_R$ the outlying distance along $y$. We see that the Wasserstein DRO formulation
consistently outperforms all other approaches, with its ROC curve lying
well above others. The approaches that use the AD loss function (e.g., Wasserstein DRO and regularized LAD) tend to outperform those that adopt the SR loss (e.g., M-estimation which uses a variant of the SR loss). M-estimation adopts
an {\em Iteratively Reweighted Least Squares (IRLS)} procedure which
assigns weights to data points based on the residuals from previous
iterations. With such an approach, there is a chance of exaggerating the
influence of outliers while downplaying the importance of clean
observations, especially when the initial residuals are obtained through
OLS. 

\begin{figure}[h]
	\begin{subfigure}{.49\textwidth}
		\centering
		\includegraphics[width=0.98\textwidth]{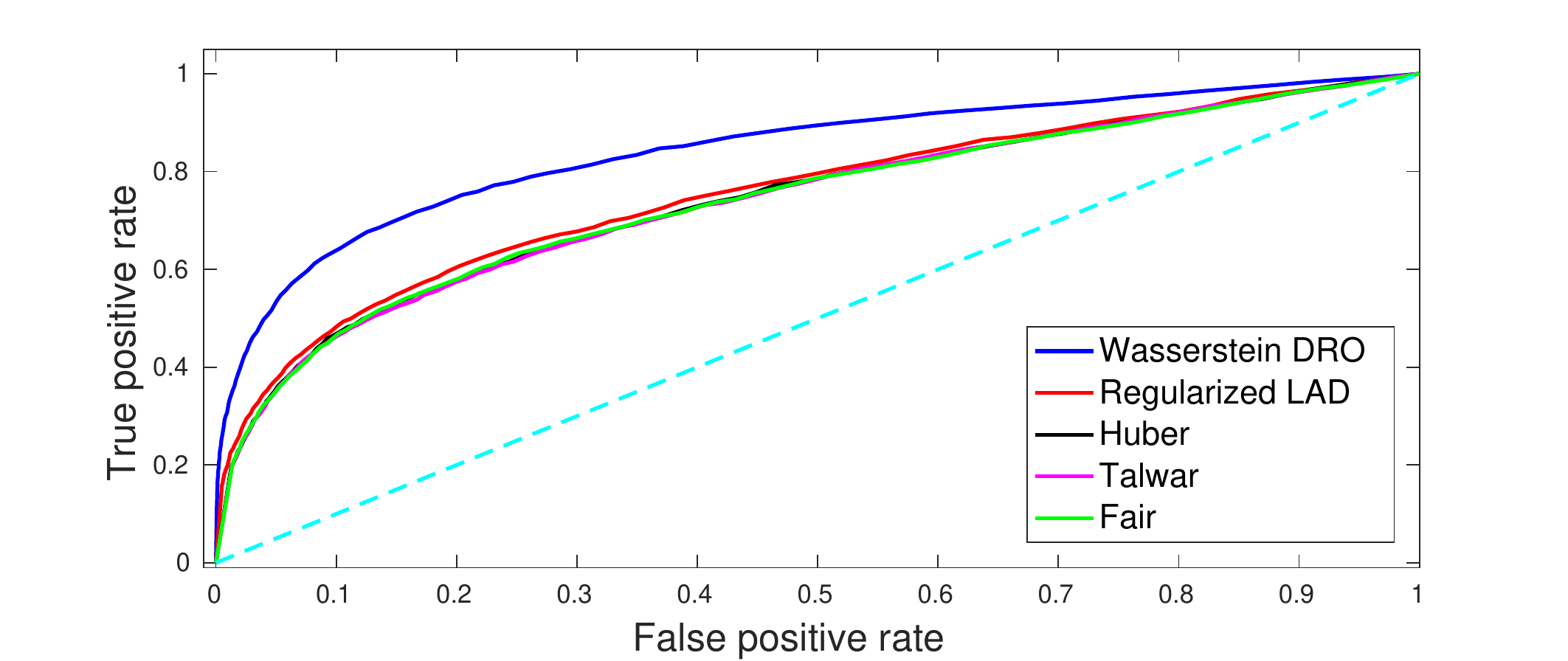}
		\caption{{\small $q=20\%, \delta_R = 3\sigma$}} 
	\end{subfigure}%
	\begin{subfigure}{.49\textwidth}
		\centering
		\includegraphics[width=0.98\textwidth]{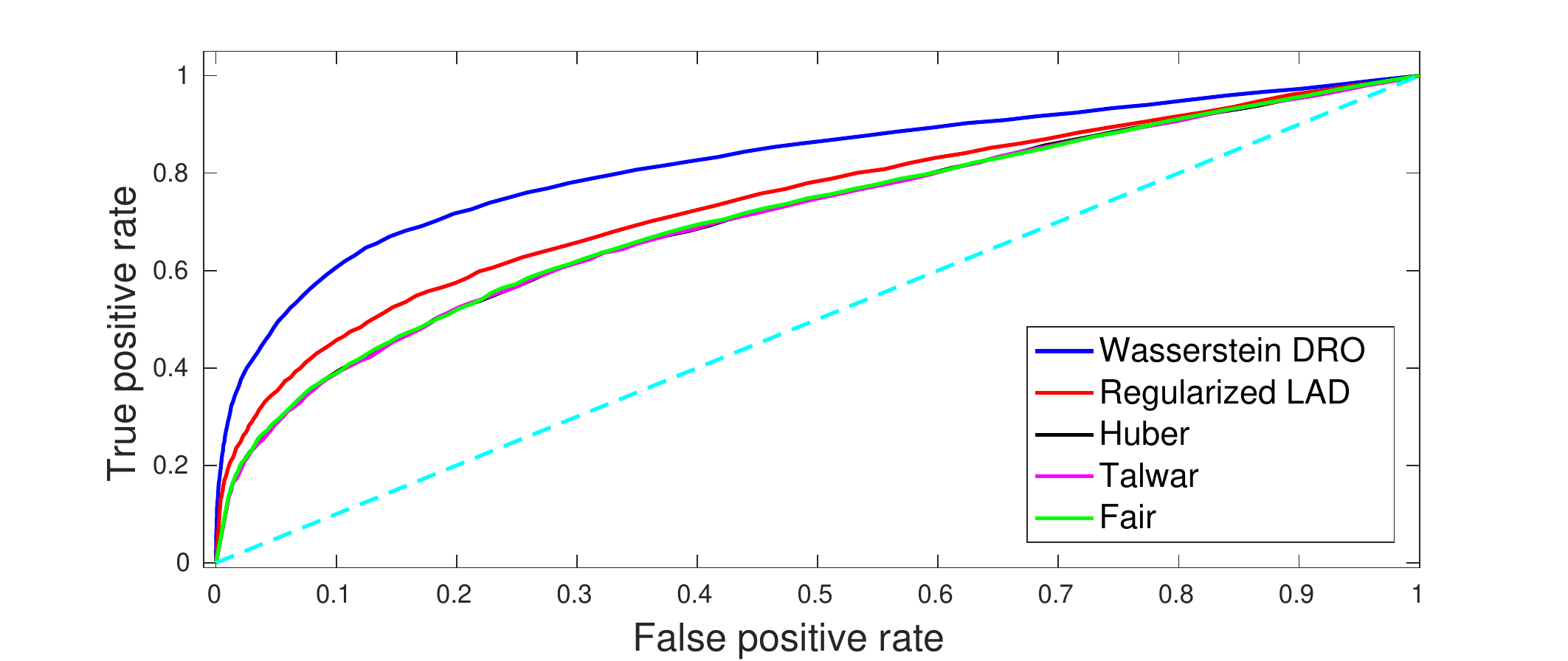}
		\caption{{\small $q=30\%, \delta_R = 3\sigma$}}
	\end{subfigure}
	
	\begin{subfigure}{.49\textwidth}
		\centering
		\includegraphics[width=0.98\textwidth]{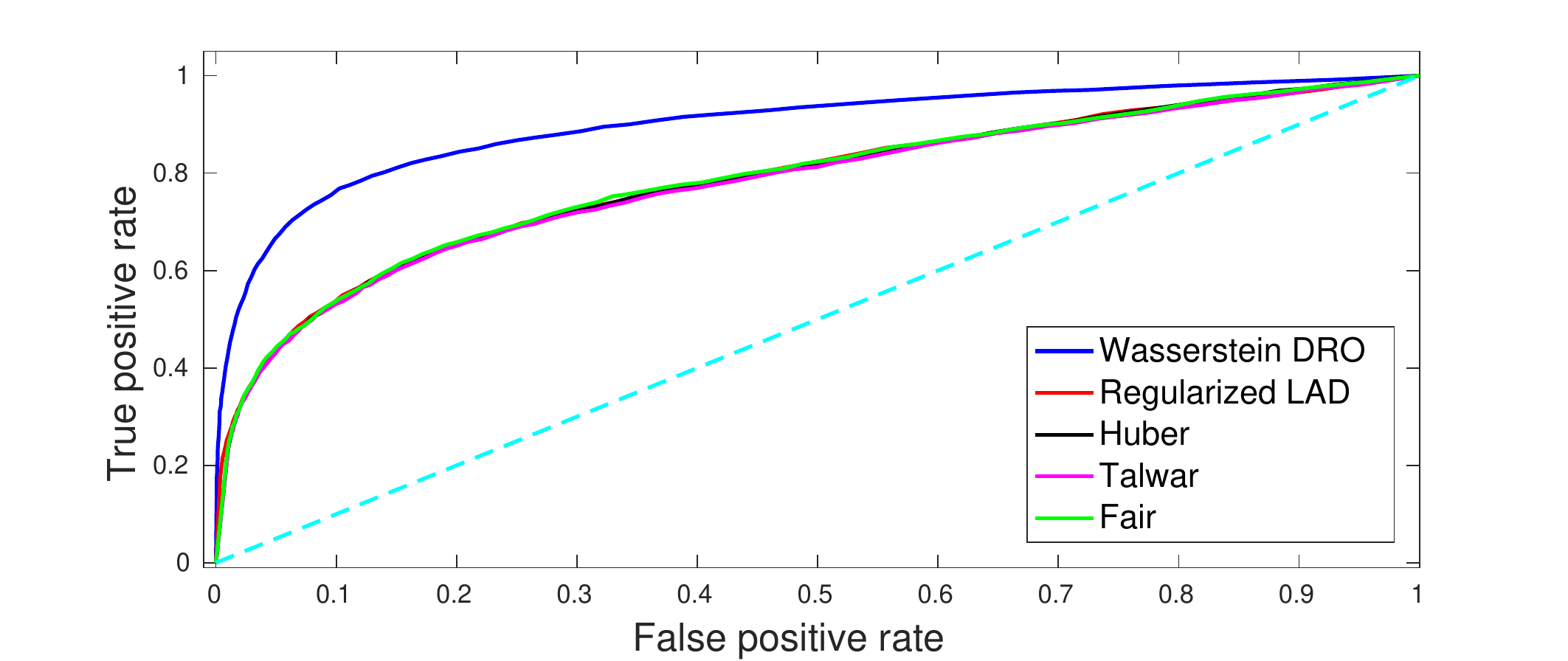}
		\caption{{\small $q=20\%, \delta_R = 4\sigma$}}
	\end{subfigure}
	\begin{subfigure}{.49\textwidth}
		\centering
		\includegraphics[width=0.98\textwidth]{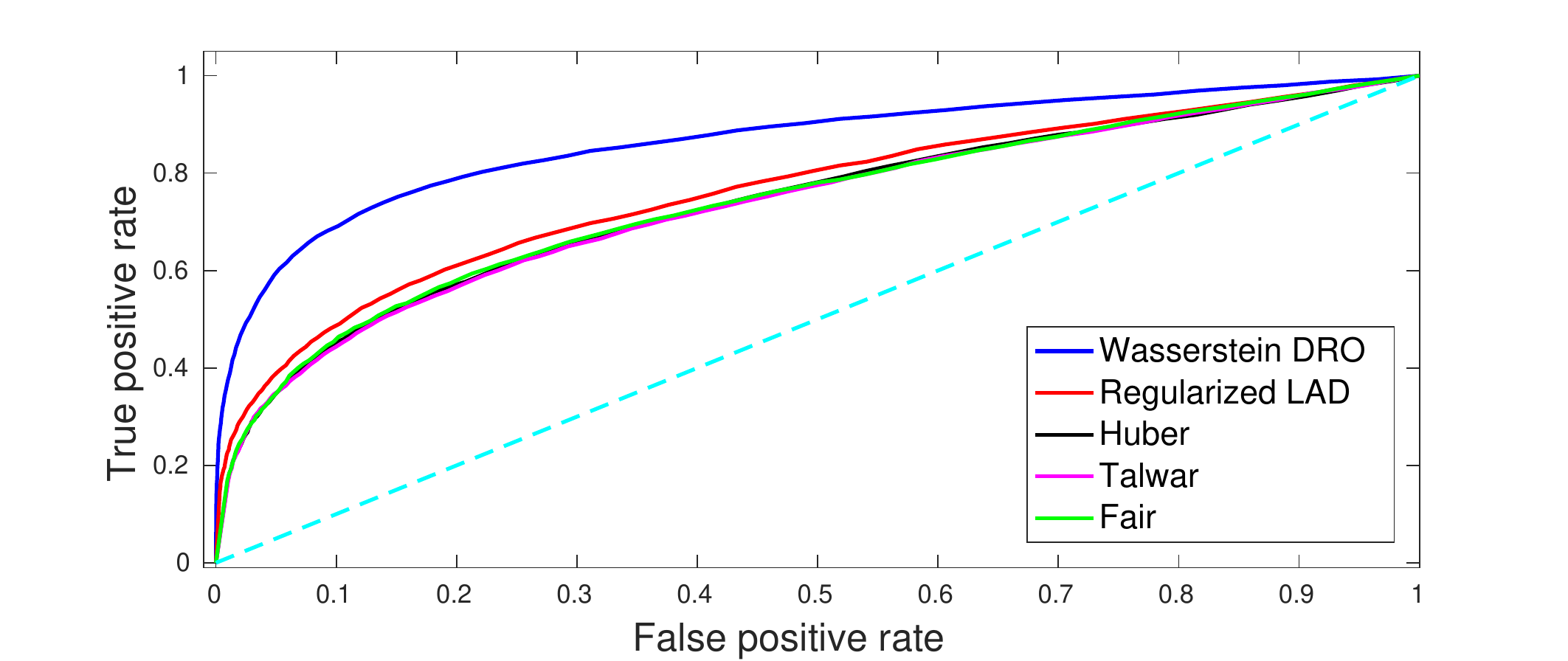}
		\caption{{\small $q=30\%, \delta_R = 4\sigma$}}
	\end{subfigure}
	
	\begin{subfigure}{.49\textwidth}
		\centering
		\includegraphics[width=0.98\textwidth]{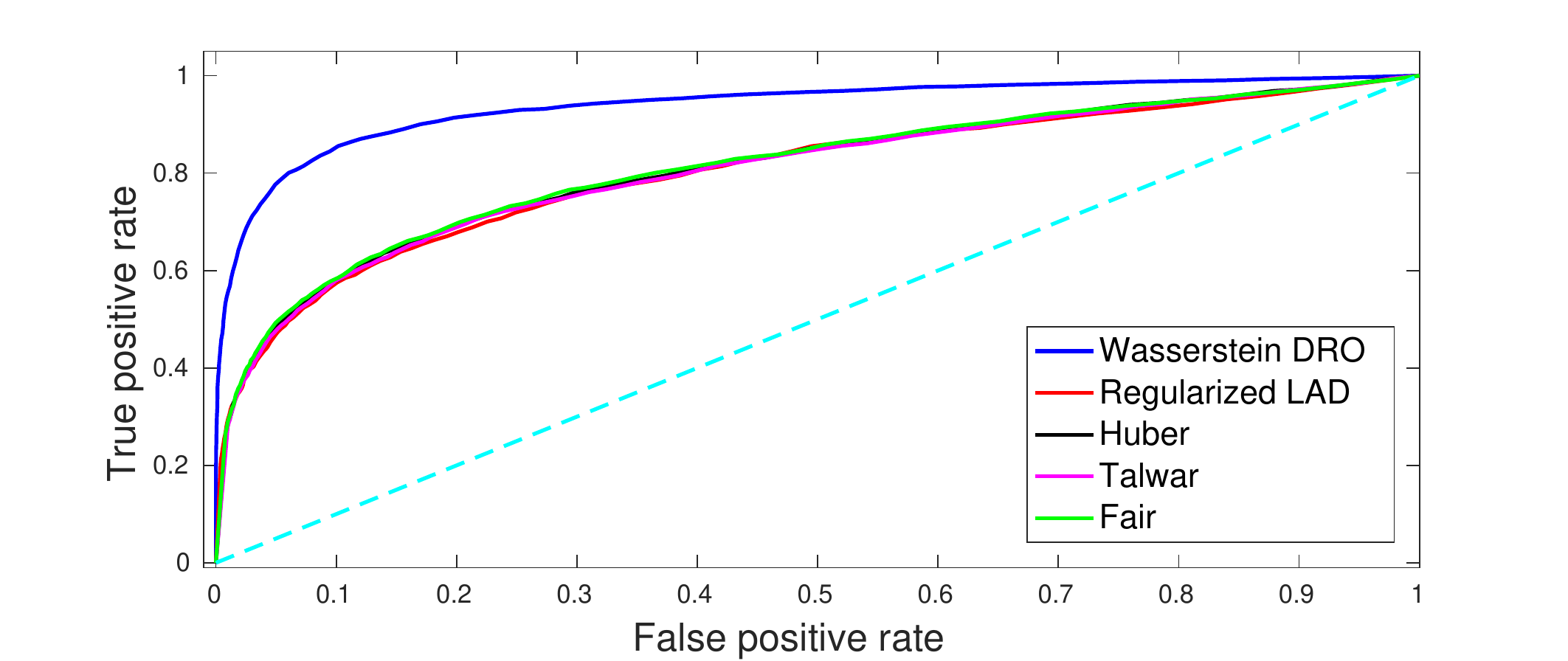}
		\caption{{\small $q=20\%, \delta_R = 5\sigma$}}
	\end{subfigure}%	
	\begin{subfigure}{.49\textwidth}
		\centering
		\includegraphics[width=0.98\textwidth]{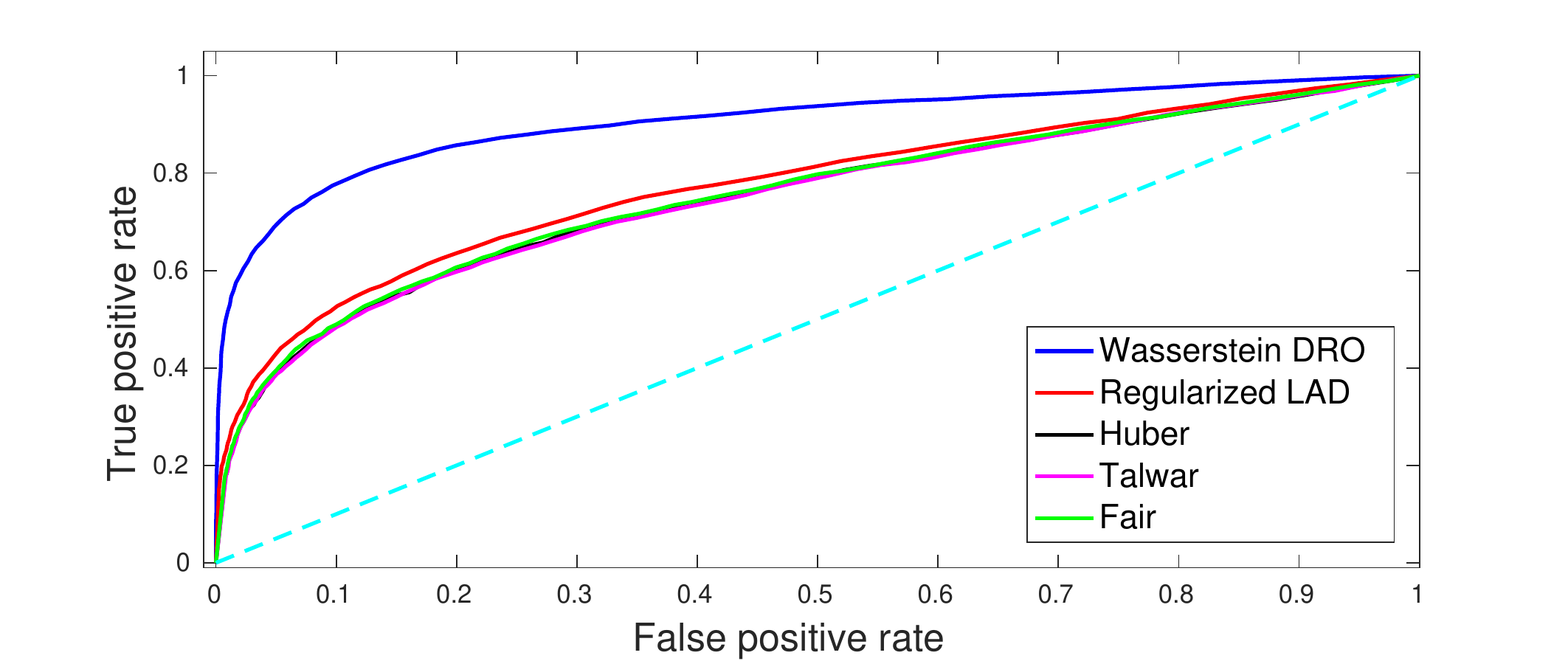}
		\caption{{\small $q=30\%, \delta_R = 5\sigma$}}
	\end{subfigure}
	\caption{ROC curves for outliers in a randomly placed cloud, $N=60, \sigma=0.5$.}
	\label{f15}
\end{figure}

\subsection{CT Radiation Overdose Detection}
In this section we consider an application of Wasserstein DRO regression to CT
radiation overdose detection \citep{chen2019detection}. The goal is to identify all
CT scans with an unanticipated high radiation exposure, given the characteristics of
the patient and the type of the exam. This could be cast as an outlier detection
problem; specifically, estimating a robustified regression plane that is immunized
against outliers and learns the underlying true relationship between radiation dose
and the relevant predictors. Given such a regression plane, abnormal CT scans can be 
identified by the residuals of the regression. 

The data was obtained from a HIPAA-compliant, Institutional Review Board (IRB)-approved retrospective cohort study that was conducted at an academic medical system including a 793-bed quaternary care hospital, and two outpatient imaging facilities. 
%All consecutive CT examinations performed on the 11 operational CT scanners (GE, Philips, and Siemens) were eligible for study inclusion between June 30, 2012 and December 31, 2013, a time-period during which there were no scanner equipment or software changes. 
The original de-identified dataset contained 28 fields for 189,959 CT exams, and the per acquisition CT Dose Index (CTDI), which measures the amount of exposure to CT radiation. Mean patient age was $60.6 \pm 17.1$ years; 54.7\% were females.

The data was pre-processed as follows: $(i)$ patient visits with more than half of the corresponding variables missing, or a missing value for CTDI,  were discarded; $(ii)$ categorical variables were encoded using indicator variables, and categories present only in a small number of exams were deleted; $(iii)$ variables that have low correlation with CTDI were removed from further consideration; $(iv)$ missing values were imputed by the mean (for numerical predictors) or mode (for categorical predictors); $(v)$ all predictors were standardized by subtracting the mean and dividing by the standard deviation.

After pre-processing, we were left with 606 numerically encoded predictors for 88,566 CT exams. 
We first applied the variable selection method LASSO to select important variables for predicting CTDI, and then employed the Wasserstein DRO regression approach (induced by the $\ell_2$ norm) to learn a predictive model of CT radiation doses given important variables identified by LASSO. Patient visits whose predicted radiation dose was statistically different from the radiation dose actually received were identified as outliers. 

To assess the accuracy of the outlier cohort discovery process, we conducted a manual validation in which the results of a human-expert classification were compared to those extracted by the algorithm. A validation sample size of 200 cases were reviewed, yielding specificity of 0.85 [95\% CI 0.78-0.92] and sensitivity of 0.91 [95\% CI 0.85-0.97] (Positive Predictive Value PPV=0.84, Negative Predictive Value NPV=0.92). 

We compared against two alternatives on the same validation set of 200 samples that were reviewed by the human expert. The first alternative method is what we call a “cutoff” method. We computed the average and standard deviation of CTDI over a training set and identified as outlying exams where the CTDI was larger than the average plus 3 times the standard deviation.   
The second alternative method used OLS in lieu of the Wasserstein DRO regression, and the regression residuals (this time from OLS) were used to detect outliers. The results are reported in Table~\ref{ct-exp}, showing an improvement of 72.5\% brought by the Wasserstein DRO method in terms of the F\textsubscript{1} score, which is defined as the harmonic mean of sensitivity and PPV. For an additional point of comparison, we considered the top-40 outliers identified by each method. Among these outliers, 7 of the top-40 OLS outliers (17.5\%) were considered to be ``false positives''; while all the top-40 outliers detected by Wasserstein DRO were real outliers.

\begin{table}[hbt]
	\caption{Comparison of Wasserstein DRO regression against OLS and the cutoff method on CT radiation data.} \label{ct-exp} 
	\begin{center}
		\begin{tabular}{l c c c c c}
			\hline
			& Sensitivity & Specificity   & PPV   & NPV   & F\textsubscript{1} score \\ \hline 
			Wasserstein $\ell_2$ & 0.91        & 0.85          & 0.84  & 0.92  & 0.88   \\ 
			OLS                  & 0.36        & 0.95          & 0.87  & 0.64  & 0.51  \\ 
			Cutoff               & 0.37        & 0.94          & 0.83  & 0.64  & 0.51 \\
			\hline
		\end{tabular}
	\end{center}
\end{table}

\section{Summary} \label{sec:2-6}
In this section, we presented a novel $\ell_1$-loss based robust learning procedure using {\em
	Distributionally Robust Optimization (DRO)} under the Wasserstein metric in a linear regression
setting, through which a delicate connection between the metric space on data and the regularization term has been established. The Wasserstein formulation incorporates a class
of models whose specific form depends on the norm space that the
Wasserstein metric is defined on. We provide out-of-sample generalization
guarantees, and bound the estimation bias
of the general formulation. Extensive
numerical examples demonstrate the superiority of the Wasserstein
formulation and shed light on the advantages of the $\ell_1$-loss, the implication of the regularizer, and the selection of the norm space for the Wasserstein metric. We also presented an outlier detection example as an application of this robust learning procedure. A
remarkable advantage of this approach rests in its flexibility to adjust the form of the regularizer based on the characteristics of the data.

\chapter{Distributionally Robust Grouped Variable Selection} \label{chap:group}
In this section, we will discuss a special case of the general formulation (\ref{qcp}) tailored for selecting grouped variables that are relevant to the response when there exists a predefined grouping structure for the predictors. An example of this is the encoding of a categorical predictor using a group of indicator variables. Jointly selecting/dropping all variables in a group gives rise to more interpretable models. To perform variable selection at a group level, the {\em Grouped LASSO (GLASSO)}
was proposed by \cite{bakin1999adaptive, yuan2006model}, which imposes a block-wise $\ell_2$-normed
penalty for the grouped coefficient
vectors. We will show that by using a special norm ($\|\cdot\|_{2,\infty}$) on the data space, the Wasserstein DRO formulation recovers the GLASSO penalty under the absolute residual loss (regression) and the log-loss (classification).
The resulting model offers
robustness explanations for GLASSO algorithms and highlights the
connection between robustification and regularization. 

\section{The Problem and Related Work}
The {\em Grouped LASSO (GLASSO)}
was first proposed by \cite{bakin1999adaptive, yuan2006model} to induce sparsity at a group level, when there exists a predefined grouping structure for the predictors. Suppose the predictor $\bx = (\bx^1, \ldots, \bx^L)$, and the regression coefficient $\bbeta = (\bbeta^1, \ldots, \bbeta^L)$, where $\bx^l, \bbeta^l \in \mbb{R}^{p_l}, l \in \lb L \rb$, respectively represent the predictor and coefficient for group $l$ which contains $p_l$ predictors. GLASSO minimizes:
\begin{equation*}
\inf\limits_{\bbeta} 
\frac{1}{N}\sum\limits_{i=1}^N (y_i - \bx_i' \bbeta)^2 + \epsilon  
\sum_{l=1}^L \sqrt{p_l}\|\bbeta^l\|_2,
\end{equation*}
where $(\bx_i, y_i), i \in \lb N \rb$, are $N$ observed samples of $(\bx, y)$. Several extensions have been explored. In
particular, \cite{lin2006component, yin2012group} considered grouped variable selection in
nonparametric models. \cite{zhao2009composite, jacob2009group} explored GLASSO for overlapping groups.
The group sparsity in general regression/classification models has also been
investigated in several works, see, for example, \cite{kim2006blockwise, meier2008group, bertsimas2017logistic} for GLASSO in logistic regression,  
and \cite{roth2008group} for GLASSO in generalized linear models. 

Most of the existing works endeavor to modify the GLASSO formulation heuristically to achieve various goals. As an example,
\cite{simon2013sparse} considers a convex combination of the GLASSO and
LASSO penalties, called Sparse Grouped LASSO, to induce
both group-wise and within group sparsity. \cite{bunea2014group} modified the residual sum of squares to
its square root and proposed the {\em Grouped Square Root LASSO (GSRL)}. However, few of those works were able to provide a rigorous explanation or theoretical justification for the form of the penalty term.

In this section, we attempt to fill this gap by casting the problem of grouped variable selection into the Wasserstein DRO framework. We show that in {\em Least Absolute Deviation (LAD)} and {\em Logistic Regression (LG)}, for a specific norm-induced Wasserstein metric, the DRO model can be reformulated as a regularized empirical loss minimization problem, where the regularizer coincides with the GLASSO penalty, and its magnitude is equal to the radius of the distributional ambiguity set. Through such a reformulation we establish a connection between regularization and
robustness and offer new insights into the GLASSO penalty term.

We note that such a connection between robustification and regularization has been
explored in several works (see Section~\ref{sec:2-2}), but none of them
considered grouped variable selection. This section sheds new light on the significance of exploring the group-wise DRO problem. It is worth noting that \cite{blanchet2017distributionally} has studied the group-wise regularization estimator with the square root of the expected loss under the Wasserstein DRO framework and recovered the GSRL. Here, we present a more general framework that includes both the LAD and the negative log-likelihood loss functions, and recover the GLASSO penalty in both cases. Moreover, we point out the potential of generalizing such results to a class of loss functions with a finite growth rate.

The remainder of this section is organized as follows. Section~\ref{sec:3-2} introduces the Wasserstein GLASSO formulations for LAD and LG. Section~\ref{sec:3-3}
% establishes finite-sample
% probabilistic guarantees on the prediction and estimation performance. 
% We also 
establishes a desirable grouping effect, showing
that the difference between coefficients within the same
group converges to zero as $O(\sqrt{1-\rho})$, where $\rho$ is their sample
correlation. 
In light of this result, we use the spectral clustering algorithm to divide the predictors into a pre-specified number of groups. This renders the GLASSO algorithm completely {\em data-driven}, in the sense that no more information other than the data itself is needed.
Section~\ref{sec:3-4} presents numerical results on both synthetic data and a
real very large dataset with surgery-related medical records. 
Conclusions are in Section~\ref{sec:3-5}.

\section{The Groupwise Wasserstein Grouped LASSO} \label{sec:3-2}
In this section we describe the model setup and derive what we call the {\em Groupwise Wasserstein Grouped LASSO (GWGL)} formulation. We will consider a LAD regression model for continuous responses and an LG model for binary categorical responses. In Section~\ref{overlap}, we present a GWGL formulation for overlapping groups.

\subsection{GWGL for Continuous Response Variables}
\label{sec:gwgl_lr} 
We assume that the predictors belong to $L$ prescribed groups with group size $p_l$, $l \in \lb L \rb$, i.e., $\bx = (\bx^1, \ldots,
\bx^L)$, where $\bx^l \in \mbb{R}^{p_l}$ and $\sum_{l=1}^L p_l = p$ (no overlap among
groups). The regression coefficient is $\bbeta = (\bbeta^1, \ldots, \bbeta^L)$, where $\bbeta^l \in \mbb{R}^{p_l}$ denotes the regression
coefficient for group $l$. Similar to Section~\ref{chapt:dro}, we assume
\begin{equation*}
y = \bx' \bbeta^* + \eta.
\end{equation*}
The main assumption we make regarding $\bbeta^*$ is that it is {\em
	group sparse}, i.e., $\bbeta^l = \bzero$ for $l$ in some subset of
$\lb L \rb$. Our goal is to obtain an accurate estimate of
$\bbeta^*$ under perturbations on the data, when the predictors have a predefined grouping structure. We model stochastic disturbances on the data via distributional uncertainty, and apply a Wasserstein DRO framework to inject robustness into the solution. The learning problem is formulated as:
\begin{equation*} 
\inf\limits_{\bbeta}\sup\limits_{\mbb{Q}\in \Omega}
\mbb{E}^{\mbb{Q}}\big[ |y-\bx'\bbeta|\big], 
\end{equation*}
where $\mbb{Q}$ is the probability distribution of $\bz=(\bx, y)$, belonging to some
set $\Omega$ defined as: 
\begin{equation} \label{omega-gwgl-lr}
\Omega = \Omega_{\epsilon}^{s,1}(\hat{\mathbb{P}}_N) \triangleq \{\mbb{Q}\in \scrP(\scrZ): W_{s,1}(\mathbb{Q},\
\hat{\mathbb{P}}_N) \le \epsilon\}, 
\end{equation}
and the order-one Wasserstein distance $W_{s,1}(\mbb{Q},\ \hat{\mbb{P}}_N)$ is
defined on the metric space $(\scrZ, s)$ associated with the data points $\bz$. To
reflect the group structure of the predictors and to take into account the group
sparsity requirement, we adopt a specific notion of norm to define the metric
$s$. Specifically, for a vector $\bz$ with a group structure $\bz = (\bz^1, \ldots,
\bz^L)$, define its $(q, t)$-norm, with $q,t\geq 1$, as:
\begin{equation*}
\|\bz\|_{q, t} = \Bigl(\sum_{l=1}^L
\bigl(\|\bz^l\|_q\bigr)^t\Bigr)^{1/t}. 
\end{equation*}
Notice that the $(q, t)$-norm of $\bz$ is actually the $\ell_t$-norm
of the vector $(\|\bz^1\|_q, \ldots, \|\bz^L\|_q)$, which
represents each group vector $\bz^l$ in a concise way via the
$\ell_q$-norm.

Inspired by the LASSO where the $\ell_1$-regularizer is used to induce sparsity on the individual level, we wish to deduce an $\ell_1$-norm penalty on the group level from (\ref{qcp}) to induce group sparsity on $\bbeta^*$. This motivates the use of the $(2, \infty)$-norm on the
weighted predictor-response vector $$\bz_{\bw} \triangleq \bigg(\frac{1}{\sqrt{p_1}}\bx^1,
\ldots, \frac{1}{\sqrt{p_L}}\bx^L, M y \bigg),$$ where the weight vector is
$$\bw=\bigg(\frac{1}{\sqrt{p_1}}, \ldots, \frac{1}{\sqrt{p_L}}, M \bigg),$$ and $M$
is a positive weight assigned to the response. Specifically,
\begin{equation} \label{2infty}
\|\bz_{\bw}\|_{2, \infty} = \max \left\{\frac{1}{\sqrt{p_1}}\|\bx^1\|_2,
\ldots, \frac{1}{\sqrt{p_L}}\|\bx^L\|_2, M |y|\right\}.
\end{equation}

In (\ref{2infty}) we normalize each group by the number of predictors, to prevent large groups from having a large impact on the distance metric. The $\|\cdot\|_{2, \infty}$ operator computes the maximum of the $\ell_2$ norms of the (weighted) grouped
predictors and the response. It essentially selects the
most influential group when determining the closeness between two points
in the predictor-response space, which is consistent with our group sparsity assumption in that not all groups of predictors contribute to the determination of $y$, and thus a metric that ignores the unimportant groups (e.g., $\|\cdot\|_{2, \infty}$) is desired.

Based on (\ref{qcp}), in order to obtain the GWGL formulation, we need to derive the dual norm of $\|\cdot\|_{2, \infty}$. A
general result that applies to any $(q, t)$-norm is presented in the
following theorem. The dual norm of the $(2, \infty)$-norm is a direct
application of Theorem~\ref{dualnorm}.

\begin{thm} \label{dualnorm} Consider a vector $\bx = (\bx^1,
	\ldots, \bx^L)$, where each $\bx^l \in \mbb{R}^{p_l}$, and $\sum_l p_l
	= p$. Define the weighted $(r, s)$-norm of $\bx$ with the weight
	vector $\bw = (w_1, \ldots, w_L)$ to be:
	\begin{equation*}
	\|\bx_{\bw}\|_{r, s} = \Bigl(\sum_{l=1}^L
	\bigl(\|w_l\bx^l\|_r\bigr)^s\Bigr)^{1/s}, 
	\end{equation*} 
	where $\bx_{\bw} = (w_1\bx^1, \ldots, w_L\bx^L)$, $w_l>0, \forall
	l$, and $r, s \ge 1$. Then, the dual norm of the weighted $(r, s)$-norm with weight $\bw$
	is the $(q, t)$-norm with weight $\bw^{-1}$, where $1/r + 1/q = 1$,
	$1/s+1/t = 1$, and $\bw^{-1} = (1/w_1, \ldots, 1/w_L)$.
\end{thm} 

\begin{proof} 
	The dual norm of $\|\cdot\|_{r, s}$ evaluated at some vector $\bbeta$ is the optimal value of Problem (\ref{dualnorm-3}):
	\begin{equation} \label{dualnorm-3}
	\begin{aligned}
	\max\limits_{\bx} & \quad \bx' \bbeta \\
	\text{s.t.} & \quad \|\bx_{\bw}\|_{r, s} \le 1.
	\end{aligned}
	\end{equation}
	We assume that $\bbeta$ has the same group structure with $\bx$, i.e.,
	$\bbeta = (\bbeta^1, \ldots, \bbeta^L)$. Using H\"{o}lder's inequality, we
	can write
	\begin{equation*}
	\begin{aligned}
	\bx'\bbeta & = \sum_{l=1}^L (w_l\bx^l)'\Bigl(\frac{1}{w_l}\bbeta^l\Bigr) \\
	& \le \sum_{l=1}^L \|w_l\bx^l\|_r \left\|\frac{1}{w_l}\bbeta^l\right\|_q. 
	\end{aligned}
	\end{equation*} 
	Define two new vectors in $\mbb{R}^L$
	\[
	\bx_{new} = (\|w_1\bx^1\|_r, \ldots, \|w_L\bx^L\|_r),
	\] 
	\[
	\bbeta_{new} = \left(\left\|\frac{1}{w_1}\bbeta^1\right\|_q, \ldots,
	\left\|\frac{1}{w_L}\bbeta^L\right\|_q\right).
	\] 
	Applying H\"{o}lder's inequality again to $\bx_{new}$ and
	$\bbeta_{new}$, we obtain:
	\begin{equation*}
	\begin{aligned}
	\bx'\bbeta & \le \bx_{new}'\bbeta_{new} \\
	& \le \|\bx_{new}\|_s
	\|\bbeta_{new}\|_t \\
	& = \Bigl(\sum_{l=1}^L
	\bigl(\|w_l\bx^l\|_r\bigr)^s\Bigr)^{1/s} \left(\sum_{l=1}^L
	\left(\left\|\frac{1}{w_l}\bbeta^l\right\|_q\right)^t\right)^{1/t}. 
	\end{aligned}
	\end{equation*} 
	Therefore,
	\begin{equation*}
	\begin{aligned}
	\bx'\bbeta & \le \|\bx_{\bw}\|_{r, s} \|\bbeta_{\bw^{-1}}\|_{q, t} \\
	& \le
	\|\bbeta_{\bw^{-1}}\|_{q, t}, 
	\end{aligned}
	\end{equation*}
	due to the constraint $\|\bx_{\bw}\|_{r, s} \le 1$. The result then
	follows.
\end{proof}

Now, let us go back to (\ref{2infty}), which is the weighted
$(2, \infty)$-norm of $\bz = (\bx^1, \ldots, \bx^L, y)$ with the weight
$\bw = (\frac{1}{\sqrt{p_1}}, \ldots, \frac{1}{\sqrt{p_L}}, M)$. According
to Theorem~\ref{dualnorm}, the dual norm of the weighted $(2,
\infty)$-norm with weight $\bw$ evaluated at some $\tilde{\bbeta} =
(-\bbeta^1, \ldots, -\bbeta^L, 1)$ is:
\begin{equation*}
\|\tilde{\bbeta}_{\bw^{-1}}\|_{2, 1} = \sum_{l=1}^L \sqrt{p_l}\|\bbeta^l\|_2 + \frac{1}{M},
\end{equation*} 
where $\bw^{-1} = (\sqrt{p_1}, \ldots, \sqrt{p_L}, 1/M)$. Therefore, with $N$ i.i.d. samples $(\bx_i, y_i)$,
$i \in \lb N \rb$,
the GWGL formulation for Linear Regression (GWGL-LR) takes the following form: 
\begin{equation} \label{gwgl-lr} 
\inf\limits_{\bbeta} 
\frac{1}{N}\sum\limits_{i=1}^N|y_i - \bx_i' \bbeta| + \epsilon  
\sum_{l=1}^L \sqrt{p_l}\|\bbeta^l\|_2,
\end{equation}
where the constant term $1/M$ has been removed. We see that
by using the weighted $(2, \infty)$-norm in the predictor-response space, we are able
to recover the commonly used penalty term for GLASSO
\citep{bakin1999adaptive, yuan2006model}. The Wasserstein DRO framework offers new interpretations for the GLASSO penalty from the standpoint of the distance metric on the predictor-response space and establishes the connection between group sparsity and distributional robustness.  

\subsection{GWGL for Binary Response Variables}
\label{sec:gwgl_lg}
In this subsection we will explore the GWGL formulation for binary classification problems. Let $\bx \in \mbb{R}^p$ denote the predictor and $y \in \{-1, +1\}$ the associated binary response/label to be predicted. In LG, the conditional distribution of $y$ given $\bx$ is modeled as
\begin{equation*}
\mbb{P}(y|\bx) = \big(1+\exp(-y \bbeta'\bx)\big)^{-1},
\end{equation*}
where $\bbeta \in \mbb{R}^p$ is the unknown coefficient vector (classifier) to be estimated. The {\em Maximum Likelihood Estimator (MLE)} of $\bbeta$ is found by minimizing the {\em negative log-likelihood (logloss)}: 
\begin{equation*}
h_{\bbeta}(\bx, y) = \log(1+\exp(-y \bbeta'\bx)).
\end{equation*}
To apply the Wasserstein DRO framework, we define the distance metric on the predictor-response space as follows.
\begin{equation} \label{metric-lg}
s((\bx_1, y_1), (\bx_2, y_2)) \triangleq \|\bx_1 - \bx_2\| + M |y_1 - y_2|, \ \forall (\bx_1, y_1), (\bx_2, y_2) \in \scrZ,
\end{equation}
where $M$ is an infinitely large positive number (different from Section~\ref{sec:gwgl_lr} where $M$ could be any positive number), and $\scrZ = \mbb{R}^p \times \{-1, +1\}$. We use a very large weight on $y$ to emphasize its role in determining the distance between data points,
i.e., for a pair $(\bx_i, y_i)$ and $(\bx_j, y_j)$, if $y_i  \neq y_j$,
they are considered to be infinitely far away from each other; otherwise
their distance is determined solely by the predictors. The robust LG problem is modeled as:
\begin{equation} \label{dro-lg}
\inf\limits_{\bbeta}\sup\limits_{\mbb{Q}\in \Omega}
\mbb{E}^{\mbb{Q}}\big[ \log(1+\exp(-y \bbeta'\bx))\big], 
\end{equation}
where $\Omega$ is defined in (\ref{omega-gwgl-lr}) with $s$ specified in (\ref{metric-lg}). Based on the discussion in Section~\ref{sec:dual-solver}, in order to derive a tractable reformulation for (\ref{dro-lg}), 
% we use the result in (\ref{finitedifexp}), which states that for any $\mbb{Q} \in \Omega$,
% \begin{equation} \label{finitedif-lg} 
% \Bigl|\mbb{E}^{\mbb{Q}}\big[ h_{\bbeta}(\bx, y)\big] -  \mbb{E}^{\hat{\mbb{P}}_N}\big[ h_{\bbeta}(\bx, y)\big]\Bigr|  
% 	\le  \int_{\scrZ \times \scrZ} \bigl|h_{\bbeta}(\bx_1, y_1)-h_{\bbeta}(\bx_2, y_2)\bigr| \Pi_0(d(\bx_1, y_1), d(\bx_2, y_2)),
% \end{equation}
% where $\Pi_0$ is the optimal solution of the optimization problem in (\ref{wass_p}), i.e., it is the joint distribution of $(\bx_1, y_1)$ and $(\bx_2, y_2)$ with
% marginals $\mbb{Q}$ and $\hat{\mbb{P}}_N$ that achieves the minimum mass transportation cost. Similar to Section \ref{sec:2-2}, 
we need to bound the growth rate of $h_{\bbeta}(\bx, y)$:
\begin{equation*}
\frac{\bigl|h_{\bbeta}(\bx_1, y_1) - h_{\bbeta}(\bx_2, y_2)\bigr|}{s((\bx_1, y_1), (\bx_2, y_2))},  \quad \forall (\bx_1, y_1), (\bx_2, y_2).
\end{equation*}
To this end, we define a continuous and differentiable univariate function $l(a) \triangleq \log(1+\exp(-a))$, and apply the mean value theorem to it, which yields that for any $a, b\in \mbb{R}$, $\exists c \in (a,b)$ such that:
\begin{equation*}
\biggl|\frac{l(b) - l(a)}{b-a}\biggr|  = \bigl|\nabla l(c)\bigr| 
=  \frac{e^{-c}}{1+e^{-c}} 
\le 1.
\end{equation*}
By noting that $h_{\bbeta}(\bx, y) = l(y\bbeta'\bx)$, we immediately have:
\begin{equation} \label{loss-lg}
\begin{aligned}
\bigl|h_{\bbeta}(\bx_1, y_1) - h_{\bbeta}(\bx_2, y_2)\bigr| & \le \bigl|y_1\bbeta'\bx_1 - y_2\bbeta'\bx_2\bigr| \\
& \le \|y_1\bx_1 - y_2 \bx_2\| \|\bbeta\|_* \\
& \le  s((\bx_1, y_1), (\bx_2, y_2)) \|\bbeta\|_*,
\end{aligned}
\end{equation}
where the second step uses H\"{o}lder's inequality, and the last step is due to the definition of the metric $s$ and the fact that $M$ is infinitely large. Eq. (\ref{loss-lg}) shows that the loss function $h_{\bbeta}(\bx, y)$ is Lipschitz continuous in $(\bx, y)$ with a Lipschitz constant $\|\bbeta\|_*$. Using Theorem~\ref{Lip-wass} with $t=1$, we obtain that for any $\mbb{Q} \in \Omega$,
\begin{equation*}
\Bigl|\mbb{E}^{\mbb{Q}}\big[ h_{\bbeta}(\bx, y)\big] -  \mbb{E}^{\hat{\mbb{P}}_N}\big[ h_{\bbeta}(\bx, y)\big]\Bigr| 
\leq \|\bbeta\|_* W_{s,1} (\mbb{Q}, \ \hat{\mbb{P}}_N) 
\le \epsilon \|\bbeta\|_*.
\end{equation*}
Therefore, Problem (\ref{dro-lg}) can be reformulated as:
\begin{equation} \label{convex-lg}
\inf\limits_{\bbeta} \mbb{E}^{\hat{\mbb{P}}_N}\big[ h_{\bbeta}(\bx, y)\big] + \epsilon \|\bbeta\|_* = \inf\limits_{\bbeta} \frac{1}{N} \sum_{i=1}^N \log\bigl(1+\exp(-y_i \bbeta'\bx_i)\bigr) + \epsilon \|\bbeta\|_*.
\end{equation}

We note that \cite{abadeh2015distributionally, shafieezadeh2017regularization, gao2017wasserstein} arrive at a similar formulation to (\ref{convex-lg}) by other means of derivation. Different from these existing works, we will consider specifically the application of (\ref{convex-lg}) to grouped predictors where the goal is to induce group level sparsity on the coefficients/classifier. As in Section~\ref{sec:gwgl_lr}, we assume that the predictor vector $\bx$ can be decomposed into $L$ groups, i.e., $\bx = (\bx^1, \ldots, \bx^L)$, each $\bx^l$ containing $p_l$ predictors of group $l$, and $\sum_{l=1}^L p_l = p$. To reflect the group sparse structure, we adopt the $(2, \infty)$-norm of the
weighted predictor vector $$\bx_{\bw} \triangleq \bigg(\frac{1}{\sqrt{p_1}}\bx^1,
\ldots, \frac{1}{\sqrt{p_L}}\bx^L \bigg),$$ to define the metric $s$ in (\ref{metric-lg}), where the weight vector is:
$$\bw=\bigg(\frac{1}{\sqrt{p_1}}, \ldots, \frac{1}{\sqrt{p_L}} \bigg).$$
According to Theorem~\ref{dualnorm}, the dual norm of the weighted $(2, \infty)$-norm with weight $\bw=(1/\sqrt{p_1}, \ldots, 1/\sqrt{p_L})$ evaluated at $\bbeta$ is:
\begin{equation*}
\|\bbeta_{\bw^{-1}}\|_{2, 1} = \sum_{l=1}^L \sqrt{p_l}\|\bbeta^l\|_2,
\end{equation*} 
where $\bw^{-1} = (\sqrt{p_1}, \ldots, \sqrt{p_L})$, and $\bbeta^l$ denotes the vector of coefficients corresponding to group $l$. Therefore, the GWGL formulation for LG (GWGL-LG) takes the form: 
\begin{equation} \label{gwgl-lg} 
\inf\limits_{\bbeta} 
\frac{1}{N} \sum_{i=1}^N \log\bigl(1+\exp(-y_i \bbeta'\bx_i)\bigr) + \epsilon  
\sum_{l=1}^L \sqrt{p_l}\|\bbeta^l\|_2.
\end{equation}
The above derivation techniques also apply to other loss functions whose growth rate is finite, e.g., the hinge loss used by SVM, and therefore, the GWGL SVM model can be developed in a similar fashion. 

\subsection{GLASSO with Overlapping Groups} \label{overlap}
In this subsection we will explore the GLASSO formulation with overlapping groups, and show that the Wasserstein DRO framework recovers a latent GLASSO approach that is proposed by \cite{obozinski2011group} to induce a solution with support being the union of predefined overlapping groups of variables.

When the groups overlap with each other, the penalty term used by (\ref{gwgl-lr}) and (\ref{gwgl-lg}) leads to a solution whose support is almost surely the complement of a union of groups, see \cite{jenatton2011structured}. That is to say, setting one group to zero shrinks its covariates to zero even if they belong to other groups, in which case these other groups will not be entirely selected. \cite{obozinski2011group} proposed a latent GLASSO approach where they introduce a set of latent variables that induce a solution vector whose support is a union of groups, so that the estimator would select entire groups of covariates. Specifically, define the latent variables $\bv^l \in \mbb{R}^p$ such that $\text{supp}(\bv^l) \subset g^l, l \in \lb L \rb$, where $\text{supp}(\bv^l) \subset \lb p \rb$ denotes the support of $\bv^l$, i.e., the set of predictors $i \in \lb p \rb$ such that $v_i^l \neq 0$, and $g^l$ denotes the set of predictors that are in group $l$. Our assumption is that $\exists \ l_1, l_2$ such that $g^{l_1} \cap g^{l_2} \neq \emptyset$. The latent GLASSO formulation is in the following form:

\begin{equation} \label{o1}
\begin{aligned}
& \inf\limits_{\bbeta, \bv^1, \ldots, \bv^L} \quad
\frac{1}{N}\sum\limits_{i=1}^N h_{\bbeta}(\bx_i, y_i)+ \epsilon  
\sum\limits_{l=1}^L d_l \|\bv^l\|_2, \\
& \quad \ \text{s.t.} \ \qquad \bbeta = \sum\limits_{l=1}^L \bv^l,
\end{aligned}
\end{equation}
where $d_l$ is a user-specified penalty strength of group $l$.
Notice that (\ref{gwgl-lr}) and (\ref{gwgl-lg}) are special cases of (\ref{o1}) where they require the latent vectors to have the same value at the intersecting covariates. By using the latent vectors $\bv^l$, Formulation (\ref{o1}) has the flexibility of implicitly adjusting the support of the latent vectors such that for any $i \in \text{supp}(\hat{\bv}^l)$ where $\hat{\bv}^l = \mathbf{0}$, it does not belong to the support of any non-shrunk latent vectors, i.e., $i \notin \text{supp}(\hat{\bv}^k)$ where $\hat{\bv}^k \neq \mathbf{0}$. As a result, the covariates that belong to both shrunk and non-shrunk groups would not be mistakenly driven to zero.
Formulation (\ref{o1}) favors solutions which shrink some $\bv^l$ to zero, while the non-shrunk components satisfy $\text{supp}(\bv^l) = g^l$, therefore leading to estimators whose support is the union of a set of groups.

To show that (\ref{o1}) can be obtained from the Wasserstein DRO framework, we consider the following weighted $(2, \infty)$-norm on the predictor space:
\begin{equation} \label{overlapmetric}
s(\bx) = \max_l d_l^{-1} \|\bx^l\|_2.
\end{equation}
For simplicity we treat the response $y$ as a deterministic quantity so that the Wasserstein metric is defined only on the predictor space. The scenario with stochastic responses can be treated in a similar fashion as in Sections~\ref{sec:gwgl_lr} and \ref{sec:gwgl_lg} by introducing some constant $M$. \cite{obozinski2011group} showed that the dual norm of (\ref{overlapmetric}) is: 
\begin{equation*}
\Omega(\bbeta) \triangleq \sum_{l=1}^L d_l \|\bv^l\|_2,
\end{equation*}
with $\bbeta = \sum_{l=1}^L \bv^l$,
and $\bbeta \ra \Omega(\bbeta)$ is a valid norm.
By noting that (\ref{o1}) can be reformulated as:
\begin{equation} \label{o2}
\inf\limits_{\bbeta} \quad
\frac{1}{N}\sum\limits_{i=1}^N l_{\bbeta}(\bx_i, y_i) + \epsilon  
\Omega(\bbeta),
\end{equation}
with 
\begin{equation*}
\Omega(\bbeta) = \min_{\substack{\bv^1, \ldots, \bv^L,\\ 
		\sum_{l=1}^L \bv^l= \bbeta}} \sum_{l=1}^L d_l \|\bv^l\|_2,
\end{equation*}
we have shown that (\ref{o1}) can be derived as a consequence of the Wasserstein DRO formulation with the Wasserstein metric induced by (\ref{overlapmetric}). In fact, \cite{obozinski2011group} pointed out that (\ref{o2}) is equivalent to a regular GLASSO in a covariate space of higher dimension obtained by duplication of the covariates belonging to several groups. For simplicity our subsequent analysis assumes non-overlapping groups.

\section{Performance Guarantees to the DRO Groupwise Estimator} \label{sec:3-3}
In this section we establish several performance guarantees for the
solutions to GWGL-LR and GWGL-LG. We are interested in two types of performance
metrics: 
\begin{enumerate}[(1)]
	\item {\em Prediction quality}, which measures the predictive power of the GWGL
	solutions on new, unseen samples. 
	\item {\em Grouping effect}, which measures the similarity of the estimated
	coefficients in the same group as a function of the sample correlation between
	their corresponding predictors. Ideally, for highly correlated predictors in the
	same group, it is desired that their coefficients are close so that they can be
	jointly selected/dropped (group sparsity).
\end{enumerate}

We note that GWGL-LR is a special
case of the general Wasserstein DRO formulation (\ref{qcp}), and thus the two types of performance
guarantees derived in Section~\ref{sec:2-3}, one for
generalization ability (Theorem~\ref{t2}), and the other for the
estimation accuracy
(Theorem~\ref{estthm}), still apply to the GWGL-LR formulation. For GWGL-LG, we will derive its prediction performance result using similar techniques.

\subsection{Performance Guarantees for GWGL-LR}
The prediction and estimation
performance of the GWGL-LR model can be described by Theorems~\ref{t2} and \ref{estthm}, where the Wasserstein metric is defined using the  weighted $(2, \infty)$-norm with weight $\bw =
(1/\sqrt{p_1}, \ldots, 1/\sqrt{p_L}, M)$. We thus omit the statement of these two results. With Theorem~\ref{estthm}, we are able to provide bounds for the
{\em Relative Risk (RR)}, {\em Relative
	Test Error (RTE)}, and {\em Proportion of Variance Explained (PVE)} that are introduced in Section \ref{sec:2-4}.
All these metrics evaluate the accuracy of
the regression coefficient estimates on a new test sample drawn from the
same probability distribution as the training samples.

Using Theorem~\ref{estthm}, we can bound the term $(\hat{\bbeta}
- \bbeta^*)'\bSigma (\hat{\bbeta} - \bbeta^*)$ as follows:
\begin{equation} \label{metricbound}
\begin{aligned}
(\hat{\bbeta} - \bbeta^*)'\bSigma (\hat{\bbeta} - \bbeta^*) & \le \lambda_{max}(\bSigma) \|\hat{\bbeta} - \bbeta^*\|_2^2 \\
& \le \lambda_{max}(\bSigma) \biggl(\frac{4R^2\bar{B}}{\lambda_{\text{min}}}
\Psi(\bbeta^*)\biggr)^2,
\end{aligned}
\end{equation} 
where $\lambda_{max}(\bSigma)$ is the maximum eigenvalue of
$\bSigma$. Using (\ref{metricbound}), bounds for RR, RTE, and PVE can be
readily obtained and are summarized in the following corollary.
\begin{col} \label{estimation-1rte} 
	Under the specifications in
	Theorem~\ref{estthm}, when the sample size 
	\[ N\ge
	\bar{C_1}\bar{\mu}^4 \mu_0^2
	\frac{\lambda_{\text{max}}}{\lambda_{\text{min}}}
	(w(\scrA(\bbeta^*))+3)^2,
	\]
	with probability at least 
	$1-\exp(-C_2N/\bar{\mu}^4)$,
	\begin{equation*}
	\text{RR}(\hat{\bbeta}) \le \frac{\lambda_{max}(\bSigma)
		\biggl(\frac{4R^2\bar{B}}{\lambda_{\text{min}}}
		\Psi(\bbeta^*)\biggr)^2}{(\bbeta^*)'
		\mathbf{\Sigma} \bbeta^*}, 
	\end{equation*} 
	\begin{equation*}
	\text{RTE}(\hat{\bbeta}) \le \frac{ \lambda_{max}(\bSigma)
		\biggl(\frac{4R^2\bar{B}}{\lambda_{\text{min}}}
		\Psi(\bbeta^*)\biggr)^2 + \sigma^2}{\sigma^2},
	\end{equation*}
	and,
	\begin{equation*}
	\text{PVE}(\hat{\bbeta}) \ge 1 - \frac{\lambda_{max}(\bSigma)
		\biggl(\frac{4R^2\bar{B}}{\lambda_{\text{min}}}
		\Psi(\bbeta^*)\biggr)^2 + \sigma^2}{(\bbeta^*)'
		\mathbf{\Sigma} \bbeta^* + \sigma^2}, 
	\end{equation*}
	where all parameters are defined in the same way as in
	Theorem~\ref{estthm}.
\end{col} 

We next proceed to investigate the grouping effect of the
GWGL-LR estimator. The
next theorem provides a bound on the absolute (weighted) difference between
coefficient estimates as a function of the sample correlation between
their corresponding predictors.
\begin{thm} \label{grouping2} Suppose the predictors are
	standardized (columns of $\bX$ have zero mean and unit variance). Let $\hat{\bbeta}
	\in \mbb{R}^p$ be the optimal solution to (\ref{gwgl-lr}).  If $\bx_{,i}$
	is in group $l_1$ and $\bx_{,j}$ is in group $l_2$, and
	$\|\hat{\bbeta}^{l_1}\|_2 \neq 0$, $\|\hat{\bbeta}^{l_2}\|_2 \neq 0$,
	define:
	\begin{equation*}
	D(i, j) =
	\Biggl|\frac{\sqrt{p_{l_1}}\hat{\beta}_i}{\|\hat{\bbeta}^{l_1}\|_2}
	-\frac{\sqrt{p_{l_2}}\hat{\beta}_j}{\|\hat{\bbeta}^{l_2}\|_2}\Biggr|. 
	\end{equation*} 
	Then,
	$$D(i, j) \le \frac{\sqrt{2(1-\rho)}}{\sqrt{N}\epsilon},$$
	where $\rho = \bx_{,i}'\bx_{,j}$ is the sample correlation, and
	$p_{l_1}, p_{l_2}$ are the number of predictors in groups $l_1$ and
	$l_2$, respectively.
\end{thm}

\begin{proof} By the optimality condition associated with formulation (\ref{gwgl-lr}),
	$\hat{\bbeta}$ satisfies:
	\begin{equation} \label{opt3} \bx_{,i}'\text{sgn}(\by-\bX \hat{\bbeta})
	= N \epsilon \sqrt{p_{l_1}}
	\frac{\hat{\beta}_i}{\|\hat{\bbeta}^{l_1}\|_2},
	\end{equation}
	\begin{equation} \label{opt4} \bx_{,j}'\text{sgn}(\by-\bX \hat{\bbeta})
	= N \epsilon \sqrt{p_{l_2}}
	\frac{\hat{\beta}_j}{\|\hat{\bbeta}^{l_2}\|_2},
	\end{equation}
	where the $\text{sgn}(\cdot)$ function is applied to a vector elementwise. Subtracting (\ref{opt4}) from (\ref{opt3}), we obtain:
	\begin{equation*}
	(\bx_{,i}- \bx_{,j})'\text{sgn}(\by-\bX \hat{\bbeta}) = N \epsilon
	\Biggl( \frac{\sqrt{p_{l_1}}\hat{\beta}_i}{\|\hat{\bbeta}^{l_1}\|_2} -
	\frac{\sqrt{p_{l_2}}\hat{\beta}_j}{\|\hat{\bbeta}^{l_2}\|_2}\Biggr). 
	\end{equation*} 
	Using the Cauchy-Schwarz inequality and $\|\bx_{,i}-
	\bx_{,j}\|_2^2=2(1-\rho)$, we obtain
	\begin{equation*}
	\begin{split}
	D(i, j) & = \Biggl|\frac{\sqrt{p_{l_1}}\hat{\beta}_i}{\|\hat{\bbeta}^{l_1}\|_2} -  \frac{\sqrt{p_{l_2}}\hat{\beta}_j}{\|\hat{\bbeta}^{l_2}\|_2}\Biggr| \\
	& \le  \frac{1}{N \epsilon} \|\bx_{,i}- \bx_{,j}\|_2
	\|\text{sgn}(\by-\bX \hat{\bbeta})\|_2 \\ 
	& \le \frac{\sqrt{2(1-\rho)}}{\sqrt{N}\epsilon}.  
	\end{split}
	\end{equation*} 
\end{proof}

When $\bx_{,i}$ and $\bx_{,j}$ are in the same group $l$ and $\|\hat{\bbeta}^{l}\|_2 \neq 0$, Theorem~\ref{grouping2} yields 
\begin{equation} \label{group1}
|\hat{\beta}_i - \hat{\beta}_j| \le
\frac{\sqrt{2(1-\rho)}\|\hat{\bbeta}^{l}\|_2}{\epsilon \sqrt{N p_l}}.
\end{equation}
From (\ref{group1}) we see that as the within group correlation $\rho$
increases, the difference between $\hat{\beta}_i$ and $\hat{\beta}_j$
becomes smaller. In the extreme case where $\bx_{,i}$ and $\bx_{,j}$ are
perfectly correlated, i.e., $\rho = 1$, $\hat{\beta}_i =
\hat{\beta}_j$. This grouping effect enables recovery of sparsity on a
group level when the correlation between predictors in the same group is
high, and implies the use of predictors' correlation as a grouping criterion. One of the popular clustering algorithms, called {\em spectral clustering} \citep{shi2000normalized, meila2001random, ng2002spectral, ding2004tutorial} performs grouping based on the eigenvalues/eigenvectors of the Laplacian matrix of the similarity graph that is constructed using the {\em similarity matrix} of data (predictors). The similarity matrix measures the pairwise similarities between data points, which in our case could be the pairwise correlations between predictors.

\subsection{Performance Guarantees for GWGL-LG}
In this subsection we establish bounds on the prediction error of the GWGL-LG solution, and explore its grouping effect. We will use the {\em Rademacher complexity} of the class of logloss (negative log-likelihood) functions to bound the generalization error. Suppose $(\bx, y)$ is drawn from the probability measure $\mbb{P}^*$. Two assumptions that impose conditions on the magnitude of the regularizer and the uncertainty level of the predictor are needed.  

\begin{ass} \label{a1-lg}  The weighted $(2, \infty)$-norm of $\bx$ is bounded above, i.e., $\|\bx_{\bw}\|_{2, \infty} \le R_{\bx} \ \text{a.s. under $\mbb{P}^*_{\scrX}$,}$, where $\bw = (1/\sqrt{p_1}, \ldots, 1/\sqrt{p_L})$.
\end{ass}

\begin{ass} \label{a2-lg} The weighted $(2, 1)$-norm of $\bbeta$ with weight $\bw^{-1} = (\sqrt{p_1}, \ldots, \sqrt{p_L})$ is bounded
	above, namely,
	$\sup_{\bbeta}\|\bbeta_{\bw^{-1}}\|_{2, 1}=\bar{B}_1$.
	
\end{ass}

Under these two assumptions, the logloss could be bounded via the definition of dual norm.
\begin{lem}
	Under Assumptions~\ref{a1-lg} and \ref{a2-lg}, it follows that under the probability measure $\mbb{P}^*$,
	\begin{equation*}
	\log\big(1+\exp(-y \bbeta'\bx)\big) \le \log\big(1 + \exp(R_{\bx} \bar{B}_1)\big), \quad \text{a.s.}.
	\end{equation*}
\end{lem} 
Now consider the following class of loss functions: 
\begin{multline*}
\scrH=\Big\{(\bx, y) \ra h_{\bbeta}(\bx, y): h_{\bbeta}(\bx, y)= \log\big(1+\exp(-y \bbeta'\bx)\big),{}\\
\forall \bbeta \text{ s.t. } \|\bbeta_{\bw^{-1}}\|_{2, 1}\le \bar{B}_1 \Big\}.
\end{multline*}
It follows from Lemma~\ref{radcom} that the empirical {\em Rademacher complexity} of $\scrH$, denoted by $\scrR_N(\scrH)$, can be upper bounded by:
\begin{equation*}
\scrR_N(\scrH)\le \frac{2 \log\big(1 + \exp(R_{\bx} \bar{B}_1)\big)}{\sqrt{N}}.
\end{equation*}
Then, applying Theorem~\ref{Peter} (Theorem 8 in \cite{Peter02}), we have the following result on the prediction error of the GWGL-LG estimator.

\begin{thm} \label{prediction-lg} Let $\hat{\bbeta}$ be an optimal solution to (\ref{gwgl-lg}), obtained using
	$N$ training samples $(\bx_i, y_i)$, $i \in \lb N \rb$. Suppose we draw a new i.i.d.\
	sample $(\bx,y)$. Under Assumptions~\ref{a1-lg} and \ref{a2-lg}, for any
	$0<\delta<1$, with probability at least $1-\delta$ with respect to the
	sampling,
	\begin{multline*} 
	\mathbb{E}^{\mbb{P}^*}\big[\log\big(1+\exp(-y \bx'\hat{\bbeta})\big)\big]  \le
	\frac{1}{N}\sum_{i=1}^N
	\log\big(1+\exp(-y_i \bx_i'\hat{\bbeta})\big)+ {}\\
	\frac{2 \log\big(1 + \exp(R_{\bx} \bar{B}_1)\big)}{\sqrt{N}} 
	+ 
	\log\big(1 + \exp(R_{\bx} \bar{B}_1)\big)\sqrt{\frac{8\log(2/\delta)}{N}}\ ,
	\end{multline*}
	and for any $\zeta>\frac{2 \log(1 + \exp(R_{\bx} \bar{B}_1))}{\sqrt{N}} 
	+ 
	\log\big(1 + \exp(R_{\bx} \bar{B}_1)\big)\sqrt{\frac{8\log(2/\delta)}{N}}$,
	\begin{equation*} 
	\begin{aligned}
	& \quad \ \mathbb{P}\Bigl( \log\big(1+\exp( -y \bx'\hat{\bbeta})\big)  \ge
	\frac{1}{N}\sum_{i=1}^N \log\big(1+\exp(-y_i \bx_i'\hat{\bbeta})\big)+\zeta \Bigr) \\
	& \le  \frac{\frac{1}{N}\sum_{i=1}^N
		\log\big(1+\exp(-y_i \bx_i'\hat{\bbeta})\big)+\frac{2 \log(1 + \exp(R_{\bx} \bar{B}_1))}{\sqrt{N}} 
	}{\frac{1}{N}\sum_{i=1}^N
	\log\big(1+\exp(-y_i \bx_i'\hat{\bbeta})\big)+\zeta}
+ \\
& \qquad \qquad \qquad \qquad \qquad \qquad \qquad \frac{\log\big(1 + \exp(R_{\bx} \bar{B}_1)\big)\sqrt{\frac{8\log(2/\delta)}{N}}}{\frac{1}{N}\sum_{i=1}^N
	\log\big(1+\exp(-y_i \bx_i'\hat{\bbeta})\big)+\zeta}. 
\end{aligned}
\end{equation*} 
\end{thm}

% Theorem \ref{prediction-lg} implies that the groupwise regularized LG formulation (\ref{gwgl-lg}) yields a solution with a small generalization error on new i.i.d. samples. 
The next result, similar to Theorem~\ref{grouping2}, establishes the grouping effect of the GWGL-LG estimator.

\begin{thm} \label{grouping-lg} Suppose the predictors are
	standardized (columns of $\bX$ have zero mean and unit variance). Let $\hat{\bbeta}
	\in \mbb{R}^p$ be the optimal solution to (\ref{gwgl-lg}).  If $\bx_{,i}$
	is in group $l_1$ and $\bx_{,j}$ is in group $l_2$, and
	$\|\hat{\bbeta}^{l_1}\|_2 \neq 0$, $\|\hat{\bbeta}^{l_2}\|_2 \neq 0$,
	define:
	\begin{equation*}
	D(i, j) =
	\Biggl|\frac{\sqrt{p_{l_1}}\hat{\beta}_i}{\|\hat{\bbeta}^{l_1}\|_2}
	-\frac{\sqrt{p_{l_2}}\hat{\beta}_j}{\|\hat{\bbeta}^{l_2}\|_2}\Biggr|. 
	\end{equation*} 
	Then,
	$$D(i, j) \le \frac{\sqrt{2(1-\rho)}}{\sqrt{N}\epsilon},$$
	where $\rho = \bx_{,i}'\bx_{,j}$ is the sample correlation between predictors $i$ and $j$, and
	$p_{l_1}, p_{l_2}$ are the number of predictors in groups $l_1$ and
	$l_2$, respectively.
\end{thm}

\begin{proof} By the optimality condition associated with formulation (\ref{gwgl-lg}),
	$\hat{\bbeta}$ satisfies:
	\begin{equation} \label{opt3-lg} 
	\sum_{k=1}^N \frac{\exp(-y_k \bx_k' \hat{\bbeta})}{1+ \exp(-y_k \bx_k' \hat{\bbeta})} y_k x_{k,i}
	= N \epsilon \sqrt{p_{l_1}}
	\frac{\hat{\beta}_i}{\|\hat{\bbeta}^{l_1}\|_2},
	\end{equation}
	\begin{equation} \label{opt4-lg} 
	\sum_{k=1}^N \frac{\exp(-y_k \bx_k' \hat{\bbeta})}{1+ \exp(-y_k \bx_k' \hat{\bbeta})} y_k x_{k,j}
	= N \epsilon \sqrt{p_{l_2}}
	\frac{\hat{\beta}_j}{\|\hat{\bbeta}^{l_2}\|_2},
	\end{equation}
	where $x_{k,i}$ and $x_{k, j}$ denote the $i$-th and $j$-th elements of $\bx_k$, respectively.
	Subtracting (\ref{opt4-lg}) from (\ref{opt3-lg}), we obtain:
	\begin{equation} \label{derdiff}
	\sum_{k=1}^N \frac{\exp(-y_k \bx_k' \hat{\bbeta})}{1+ \exp(-y_k \bx_k' \hat{\bbeta})} \big(y_k x_{k,i} - y_k x_{k,j}\big)
	= N \epsilon
	\Biggl( \frac{\sqrt{p_{l_1}}\hat{\beta}_i}{\|\hat{\bbeta}^{l_1}\|_2} -
	\frac{\sqrt{p_{l_2}}\hat{\beta}_j}{\|\hat{\bbeta}^{l_2}\|_2}\Biggr). 
	\end{equation} 
	Note that the LHS of (\ref{derdiff}) can be written as $\bv_1'\bv_2$, where 
	\begin{equation*}
	\bv_1 = \bigg( \frac{\exp(-y_1 \bx_1' \hat{\bbeta})}{1+ \exp(-y_1 \bx_1' \hat{\bbeta})}, \ldots, \frac{\exp(-y_N \bx_N' \hat{\bbeta})}{1+ \exp(-y_N \bx_N' \hat{\bbeta})}\bigg),
	\end{equation*}
	and,
	\begin{equation*}
	\bv_2 = \big( y_1 ( x_{1,i} - x_{1,j}), \ldots, y_N (x_{N,i} - x_{N,j})\big).
	\end{equation*}
	Using the Cauchy-Schwarz inequality and $\|\bx_{,i}-
	\bx_{,j}\|_2^2=2(1-\rho)$, we obtain
	\begin{equation*}
	\begin{split}
	D(i, j) & = \Biggl|\frac{\sqrt{p_{l_1}}\hat{\beta}_i}{\|\hat{\bbeta}^{l_1}\|_2} -  \frac{\sqrt{p_{l_2}}\hat{\beta}_j}{\|\hat{\bbeta}^{l_2}\|_2}\Biggr| \\
	& \le \frac{1}{N \epsilon} \|\bv_1\|_2 \|\bv_2\|_2 \\
	& \le  \frac{1}{N \epsilon} \sqrt{N} \|\bx_{,i}- \bx_{,j}\|_2
	\\ 
	& = \frac{\sqrt{2(1-\rho)}}{\sqrt{N}\epsilon}.  
	\end{split}
	\end{equation*} 
\end{proof}

We see that Theorem~\ref{grouping-lg} yields the same bound with Theorem~\ref{grouping2}, and for predictors in the same group, their coefficients converge to the same value as $O(\sqrt{1-\rho})$. This encourages group level sparsity if predictor correlation is used as a grouping criterion. 

\section{Numerical Experiments} \label{sec:3-4}
In this section we compare the GWGL formulations with
other commonly used predictive models. In the linear regression setting, we compare GWGL-LR with models that either $(i)$ use a different loss function, e.g., the traditional GLASSO with an
$\ell_2$-loss \citep{yuan2006model}, and the Group Square-Root LASSO (GSRL) \citep{bunea2014group} that minimizes the square root of the $\ell_2$-loss; or $(ii)$ do not make use of the grouping structure of the predictors, e.g.,
the Elastic Net (EN)
\citep{zou2005regularization}, and the LASSO
\citep{tibshirani1996regression}. 
For classification problems, we consider alternatives that minimize the empirical logloss plus penalty terms that do not utilize the grouping structure of the predictors, e.g., the $\ell_1$-regularizer (LG-LASSO), $\ell_2$-regularizer (LG-Ridge), and their combination (LG-EN). 
The results on several synthetic datasets and a real large dataset of surgery-related medical records are shown in the subsequent sections. 

\subsection{GWGL-LR on Synthetic Datasets} \label{gwgl-lr-exp}
In this subsection, we will compare GWGL-LR with the aforementioned models on several synthetic datasets. 
The data generation process is described as follows:
\begin{enumerate}
	\item Generate $\bbeta^*$ based on the following rule:
	\begin{align*}
	(\bbeta^*)^l = 
	\begin{cases}
	0.5 \cdot \mathbf{e}_{p_l}, & \text{if $l$ is even;} \\
	\mathbf{0}, & \text{otherwise},
	\end{cases}
	\end{align*}	
	where $\mathbf{e}_{p_l}$ is the $p_l$-dimensional vector with all ones.
	\item Generate the predictor $\bx \in \mbb{R}^{p}$ from the Gaussian distribution $\scrN_p (0, \bSigma)$, where
	$\bSigma=(\sigma_{i,j})_{i,j=1}^p$ has diagonal elements equal to
	$1$, and off-diagonal elements specified as:
	\begin{equation*}
	\sigma_{i,j}=
	\begin{cases}
	\rho_w,   & \text{if predictors $i$ and $j$ are in the same group}; \\
	0,      & \text{otherwise}.
	\end{cases}
	\end{equation*}
	Here $\rho_w$ is the correlation between predictors in the same group,
	which we call {\em within group correlation}. The correlation between
	different groups is set to zero.
	\item Generate the response $y$ as follows:
	\begin{align*}
	y \sim 
	\begin{cases}
	\scrN(\bx'\bbeta^*, \sigma^2), & \text{if $r \le 1-q$;} \\
	\scrN(\bx'\bbeta^*, \sigma^2) + 5\sigma, & \text{otherwise},
	\end{cases}
	\end{align*}	
	where $\sigma^2$ is the intrinsic variance of $y$, $r$ is a uniform random variable on $[0, 1]$, and $q$ is the probability (proportion) of abnormal samples (outliers). 
\end{enumerate}

We generate 10 datasets consisting of $N = 100, M_t = 60$ observations and
4 groups of predictors, where $N$ is the size of the training set and
$M_t$ is the size of the test set. The number of predictors in each group
is: $p_1 = 1, p_2 = 3, p_3 = 5, p_4 = 7$, and $p=\sum_{i=1}^4
p_i=16$. We are interested in studying the impact of $(i)$ {\em Signal to Noise Ratio (SNR)}, and $(ii)$ the correlation among predictors in the same group ({\em within group correlation}): $\rho_w$. 	
The performance metrics we use are:
\begin{itemize}
	\item {\em Median Absolute Deviation (MAD)} on the test dataset, which
	is defined to be the median value of $|y_i - \bx_i'\hat{\bbeta}|, \
	i \in \lb M \rb$, with $\hat{\bbeta}$ being the estimate of $\bbeta^*$
	obtained from the training set, and $(\bx_i, y_i), \ i \in \lb M \rb,$
	being the observations from the test dataset;
	\item {\em Relative Risk (RR)} of $\hat{\bbeta}$;
	\item {\em Relative Test Error (RTE)} of $\hat{\bbeta}$;
	\item {\em Proportion of Variance Explained (PVE)} of $\hat{\bbeta}$.
\end{itemize}  

All the regularization parameters are tuned using a separate validation
dataset. As to the
range of values for the tuned parameters, we adopt the idea from Section~\ref{sec:2-4} and adjust properly for the GLASSO estimators. Specifically,
\begin{itemize}
	\item For GWGL and GSRL, the range of values for $\epsilon$ or
	$\lambda$ is: 
	\begin{equation*}
	\sqrt{\exp\biggl(\text{lin}\Bigl(\log(0.005\cdot
		\|\bX'\by\|_{\infty}),\log(\|\bX'\by\|_{\infty}),50\Bigr)\biggr)
		\biggl/\max_{l \in \lb L \rb} p_l},
	\end{equation*}
	where $\text{lin}(a, b, n)$ is a
	function that takes in scalars $a$, $b$ and $n$ (integer) and outputs
	a set of $n$ values equally spaced between $a$ and $b$; the $\exp$
	function is applied elementwise to a vector. Compared to LASSO \citep{hastie2017extended}, the
	values are scaled by $\max_{l \in \lb L \rb} p_l$, and the square root
	operation is due to the $\ell_1$-loss function, or the square root of
	the $\ell_2$-loss used in these formulations.
	\item For the GLASSO with $\ell_2$-loss, the range of values for
	$\lambda$ is:
	$$\exp\biggl(\text{lin}\Bigl(\log(0.005\cdot
	\|\bX'\by\|_{\infty}),\log(\|\bX'\by\|_{\infty}),50\Bigr)\biggr)
	\biggl/\sqrt{\max_{l \in \lb L \rb}p_l}.$$ 
\end{itemize}
We note that before solving for the regression coefficients using various GLASSO formulations, the grouping of predictors needs to be determined. Unlike most of the existing works where the grouping structure is assumed to be known or can be obtained from expert knowledge \citep{yuan2006model, ma2007supervised, bunea2014group}, we propose to use a data-driven clustering algorithm to group the predictors based on their sample correlations, as suggested by Theorem~\ref{grouping2}. Specifically, we consider the {\em spectral clustering} \citep{shi2000normalized, meila2001random, ng2002spectral} algorithm with the following Gaussian similarity function
\begin{equation} \label{gs}
\text{Gs}(\bx_{,i}, \bx_{,j}) \triangleq \exp\big( -\|\bx_{,i} - \bx_{,j}\|_2^2/(2\sigma_s^2)\big),
\end{equation}
where $\sigma_s$ is some scale parameter whose selection will be explained
later. Notice that for standardized predictors, (\ref{gs}) captures the sample
pairwise correlations between predictors, since $\|\bx_{,i}-
\bx_{,j}\|_2^2=2\big(1-\text{cor}(\bx_{,i}, \bx_{,j})\big)$, where
$\text{cor}(\bx_{,i}, \bx_{,j}) \triangleq \bx_{,i}'\bx_{,j}$. Using (\ref{gs}), we
can transform the set of predictors into a {\em similarity graph}, whose Laplacian
matrix will be used for spectral clustering. In our implementation, the $k$-nearest
neighbor similarity graph is constructed, where we connect $\bx_{,i}$ and $\bx_{,j}$
with an undirected edge if $\bx_{,i}$ is among the $k$-nearest neighbors of
$\bx_{,j}$ (in the sense of Euclidean distance) {\em or} if $\bx_{,j}$ is among the
$k$-nearest neighbors of $\bx_{,i}$. The parameter $k$ is chosen such that the
resulting graph is connected. The scale parameter $\sigma_s$ in (\ref{gs}) is set to
the mean distance of a point to its $k$-th nearest neighbor \citep{von2007tutorial}.
We assume that the number of clusters is known in order to perform spectral
clustering, but in case it is unknown, the eigengap heuristic \citep{von2007tutorial}
can be used, where the goal is to choose the number of clusters $c$ such that all
eigenvalues $\lambda_1, \ldots, \lambda_c$ of the graph Laplacian are very small, but
$\lambda_{c+1}$ is relatively large. The implementation of spectral clustering uses
the Matlab
package~\footnote{https://www.mathworks.com/matlabcentral/fileexchange/34412-fast-and-efficient-spectral-clustering.}
developed according to the tutorial \cite{von2007tutorial}.

We next present the experimental results. For a percentage of outliers $q = 20\%, 30\%$, we plot two sets of
graphs:
\begin{itemize}
	\item The performance metrics,
	i.e., out-of-sample MAD, RR, RTE, and PVE, v.s. SNR, where the SNR
	values are equally spaced between 0.5 and 2 on a log scale. Note that when SNR is varied, the within group correlation between predictors is set to $0.8$ times a random noise uniformly distributed on the interval $[0.2, 0.4]$. 
	
	\item The performance metrics
	v.s. within group correlation $\rho_w$, where $\rho_w$ takes values in
	$(0.1, 0.2, \ldots, 0.9)$. When $\rho_w$ is varied, SNR is fixed to $1$.
\end{itemize}
Results for varying the SNR are shown in Figures~\ref{snr-20} and \ref{snr-30}. 
Results for varying the within group correlation are shown
in Figures~\ref{cor-20} and \ref{cor-30}. 

\begin{figure}[hbt] 
	\begin{subfigure}{.49\textwidth}
		\centering
		\includegraphics[width=0.9\textwidth]{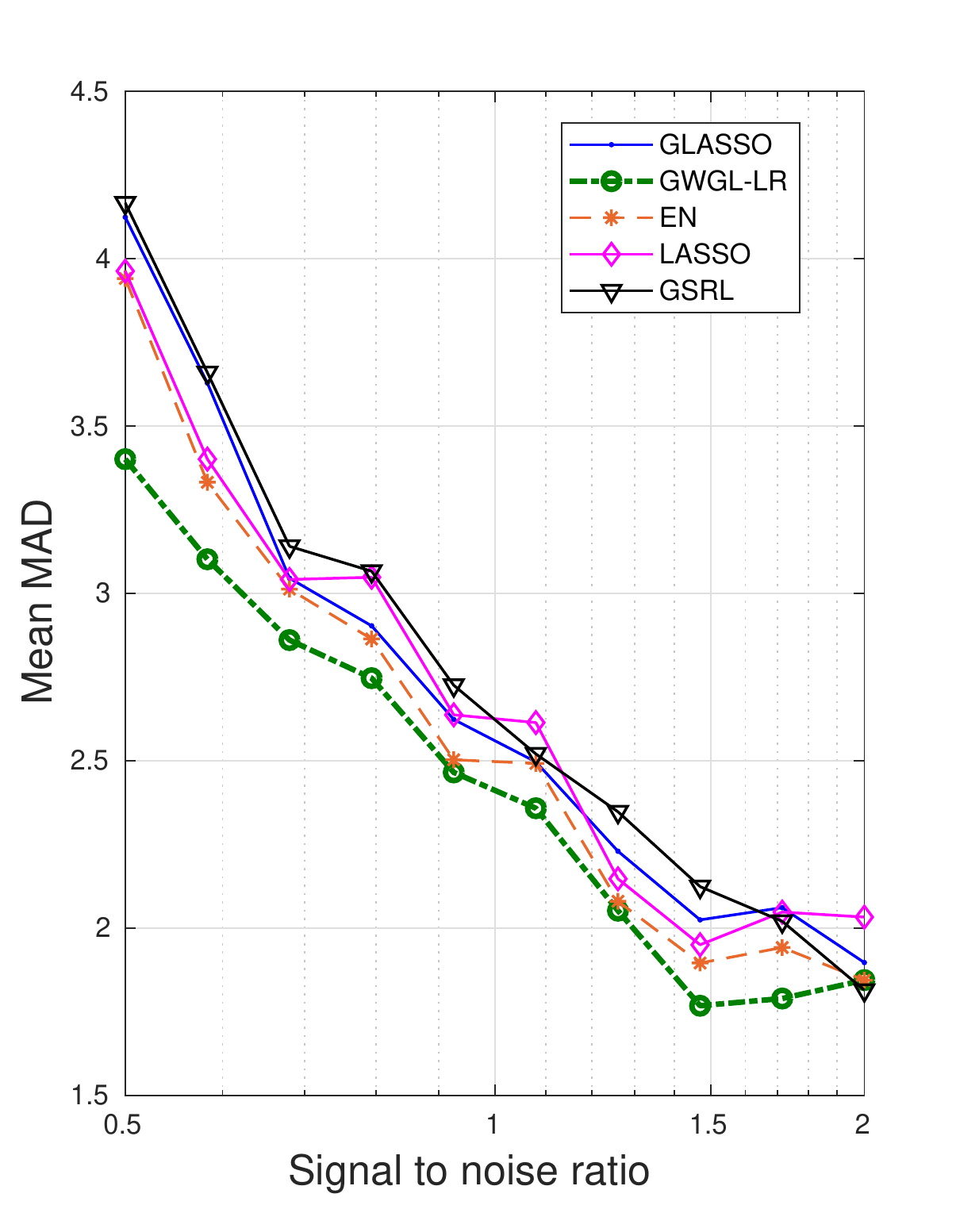}
		\caption{\small{Median Absolute Deviation.}}
	\end{subfigure}
	\begin{subfigure}{0.49\textwidth}
		\centering
		\includegraphics[width=0.9\textwidth]{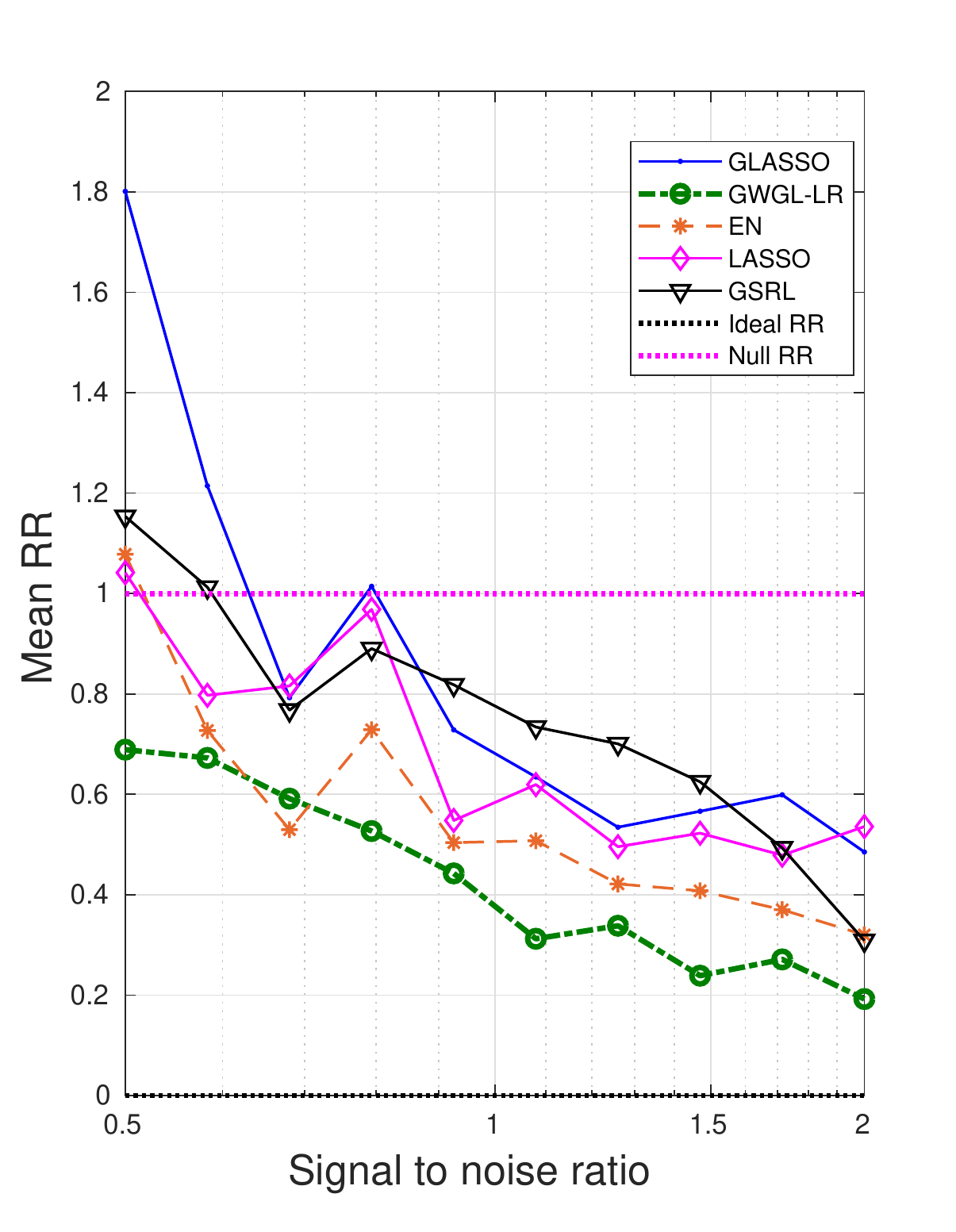}
		\caption{\small{Relative risk.}}
	\end{subfigure}
	
	\begin{subfigure}{0.49\textwidth}
		\centering
		\includegraphics[width=0.9\textwidth]{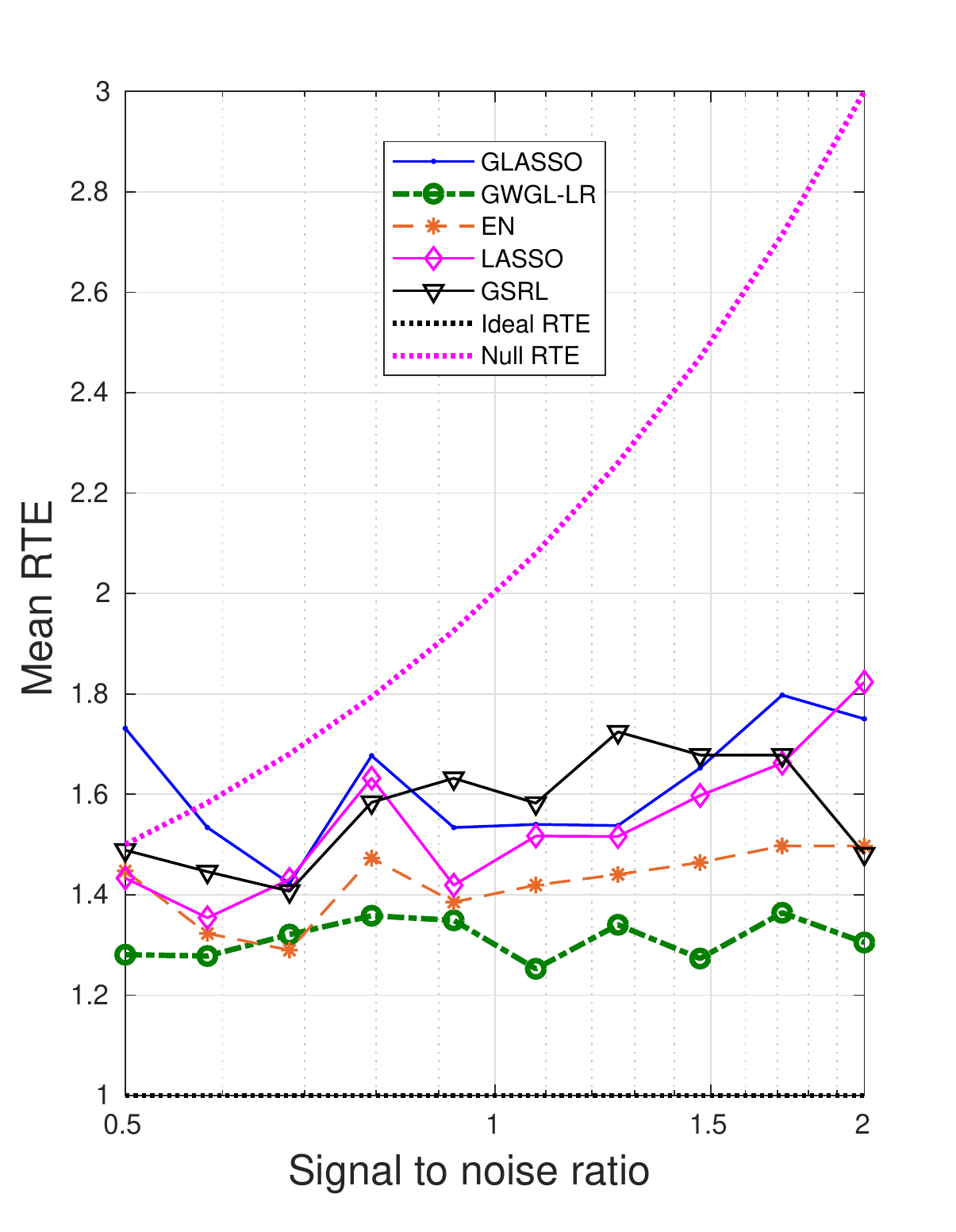}
		\caption{\small{Relative test error.}}
	\end{subfigure}%
	\begin{subfigure}{0.49\textwidth}
		\centering
		\includegraphics[width=0.9\textwidth]{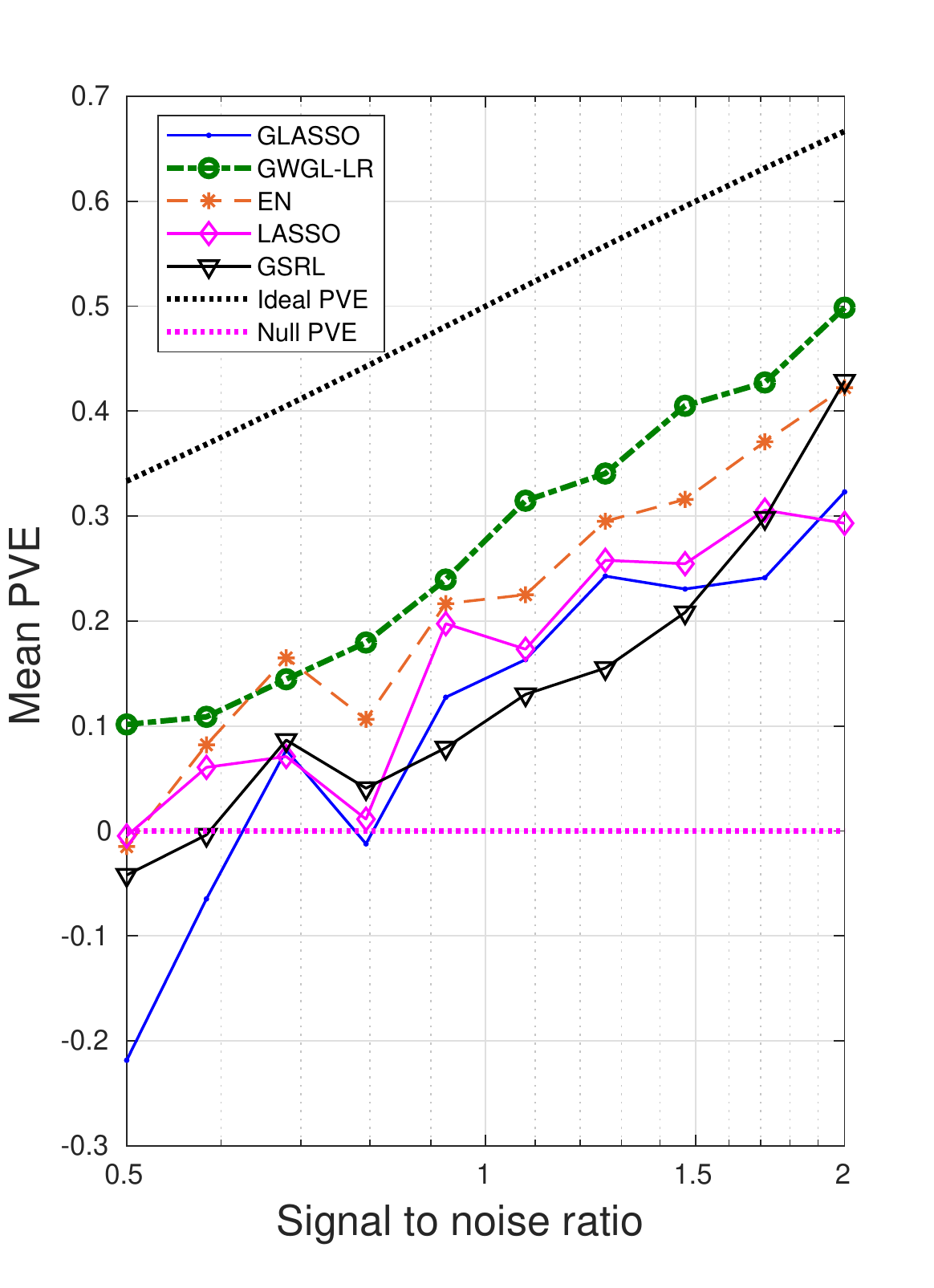}
		\caption{\small{Proportion of variance explained.}}
	\end{subfigure}
	\caption{The impact of SNR on the performance metrics, $q =
		20\%$.} \label{snr-20}
\end{figure}

\begin{figure}[hbt] 
	\begin{subfigure}{.49\textwidth}
		\centering
		\includegraphics[width=0.9\textwidth]{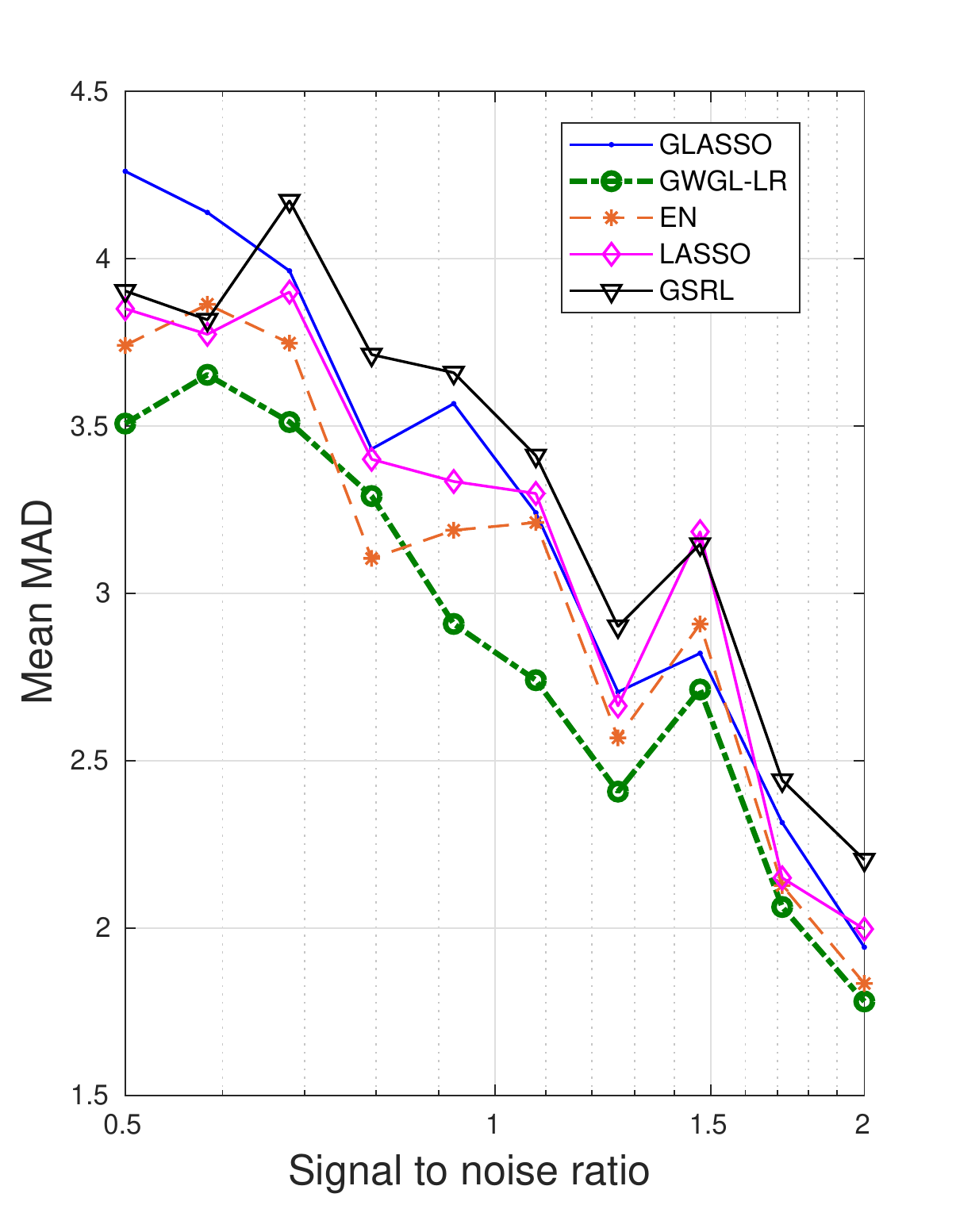}
		\caption{\small{Median Absolute Deviation.}}
	\end{subfigure}
	\begin{subfigure}{0.49\textwidth}
		\centering
		\includegraphics[width=0.9\textwidth]{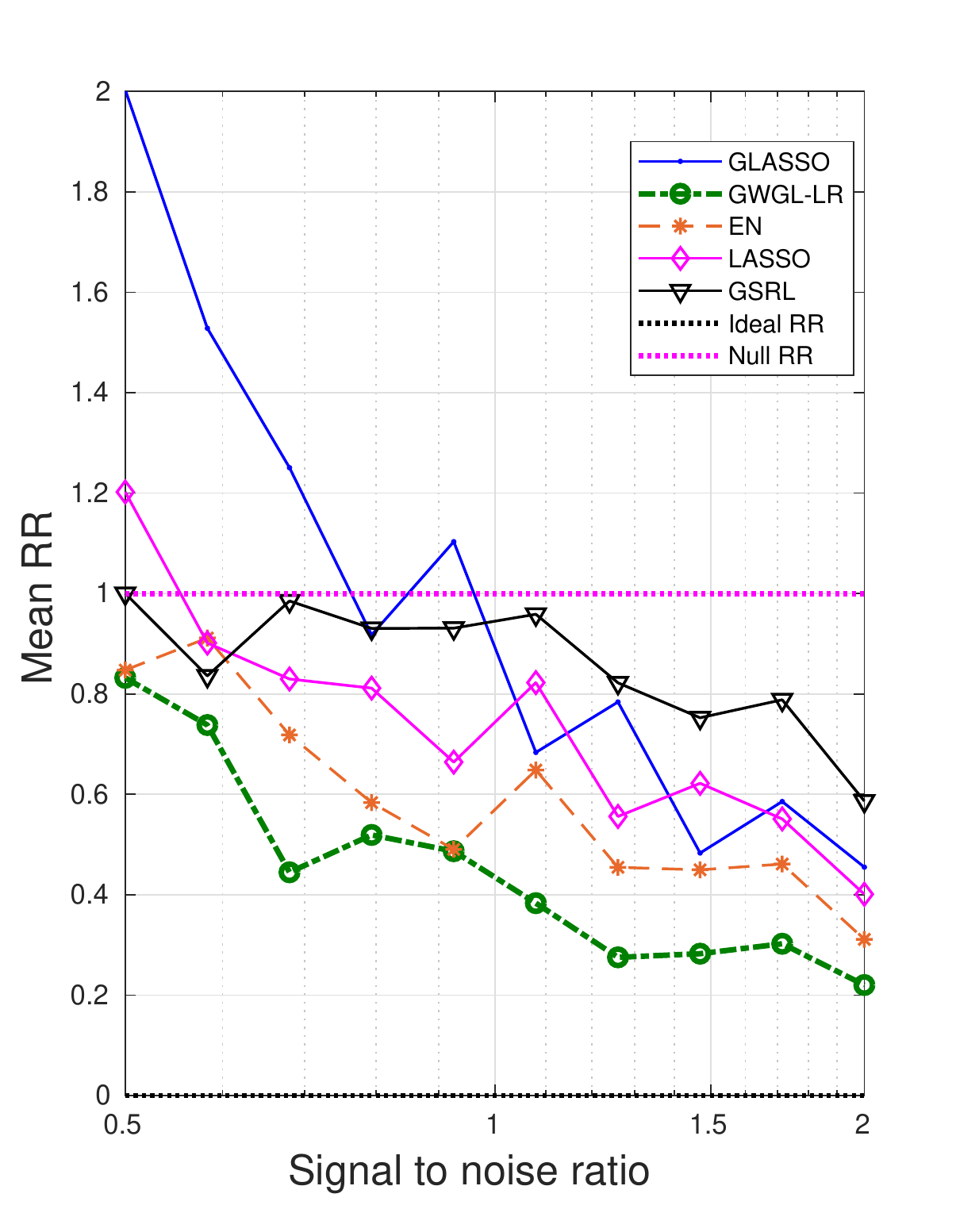}
		\caption{\small{Relative risk.}}
	\end{subfigure}
	
	\begin{subfigure}{0.49\textwidth}
		\centering
		\includegraphics[width=0.9\textwidth]{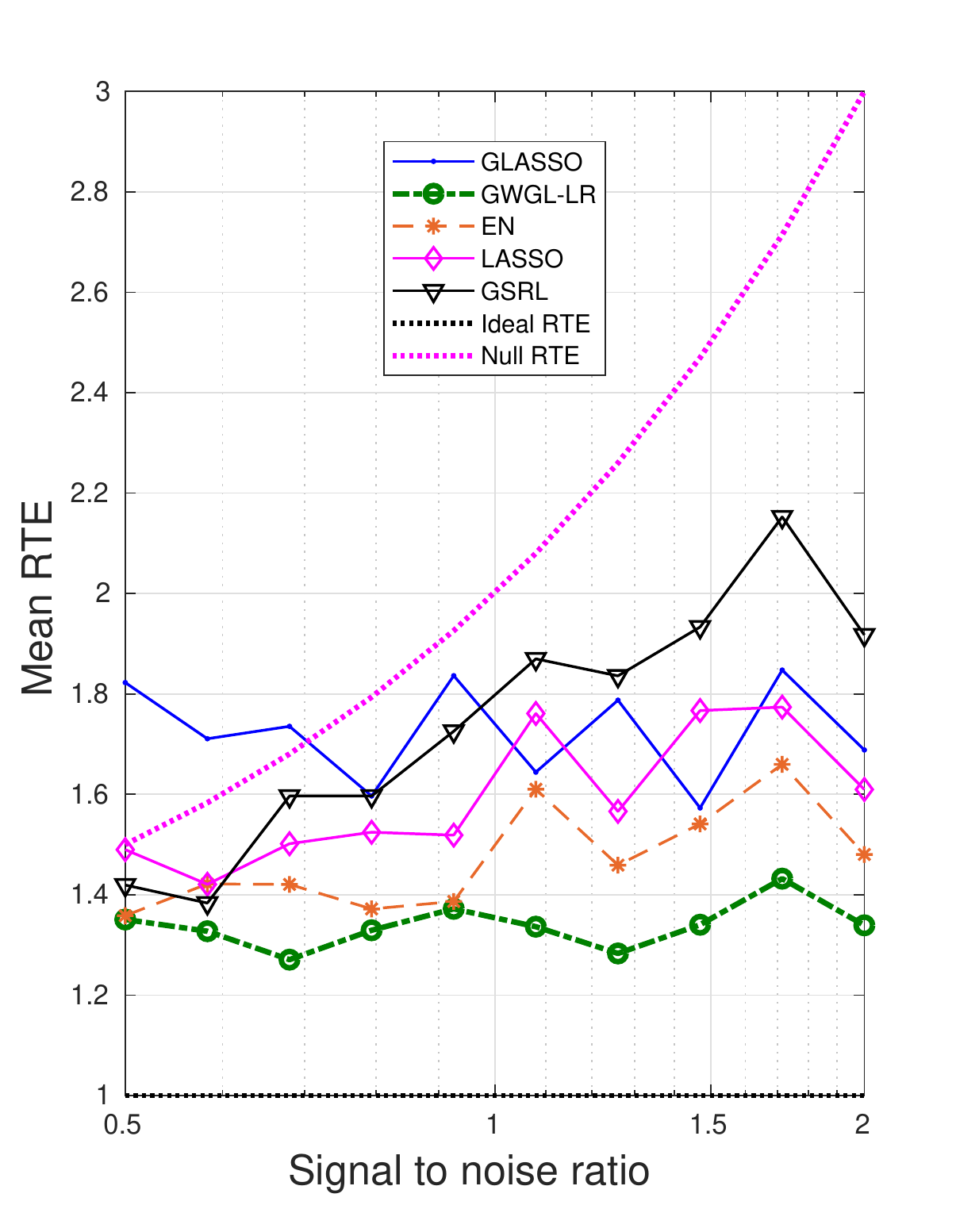}
		\caption{\small{Relative test error.}}
	\end{subfigure}%
	\begin{subfigure}{0.49\textwidth}
		\centering
		\includegraphics[width=0.9\textwidth]{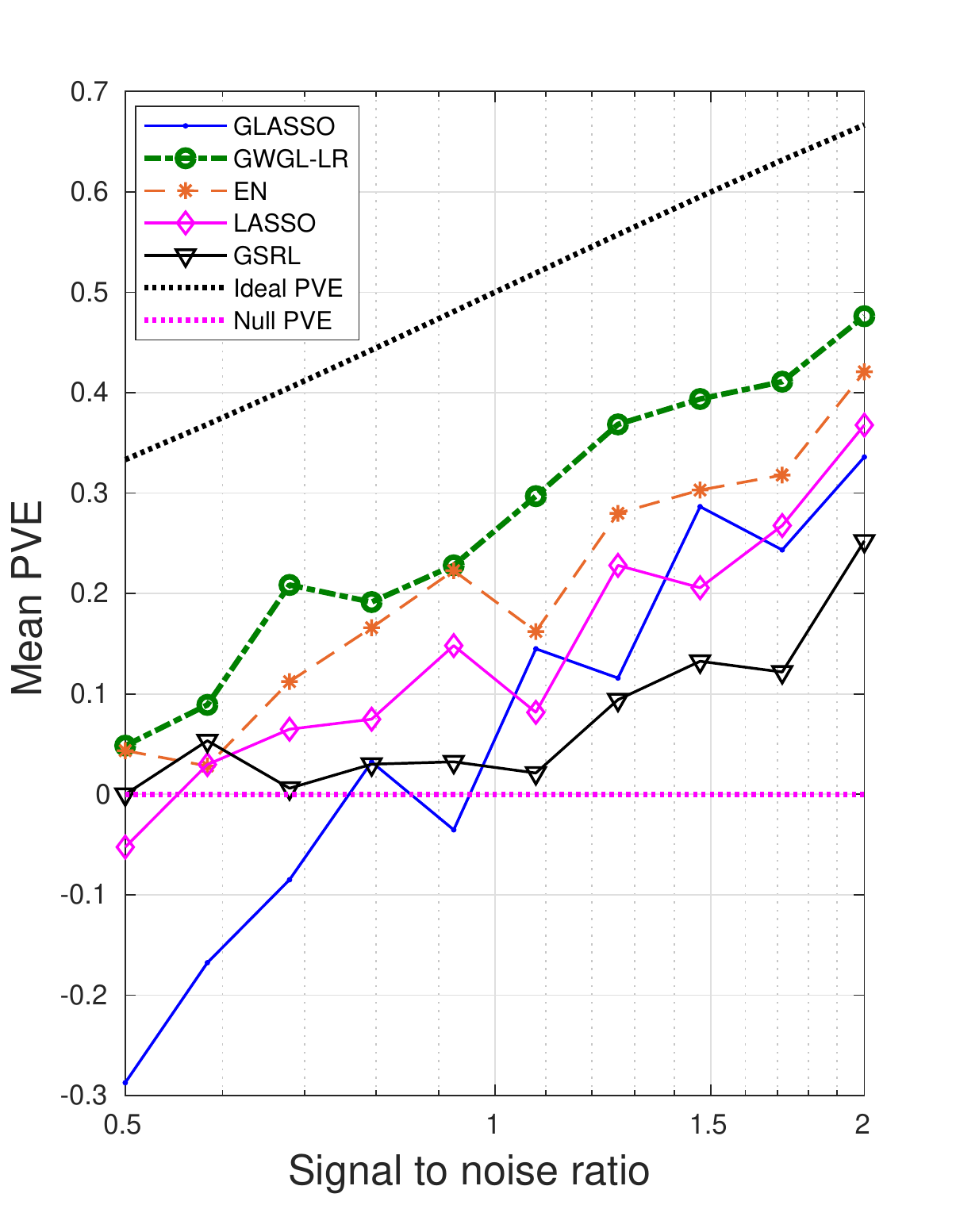}
		\caption{\small{Proportion of variance explained.}}
	\end{subfigure}
	\caption{The impact of SNR on the performance metrics, $q =
		30\%$.} \label{snr-30}
\end{figure}

\begin{figure}[hbt] 
	\begin{subfigure}{.49\textwidth}
		\centering
		\includegraphics[width=1\textwidth]{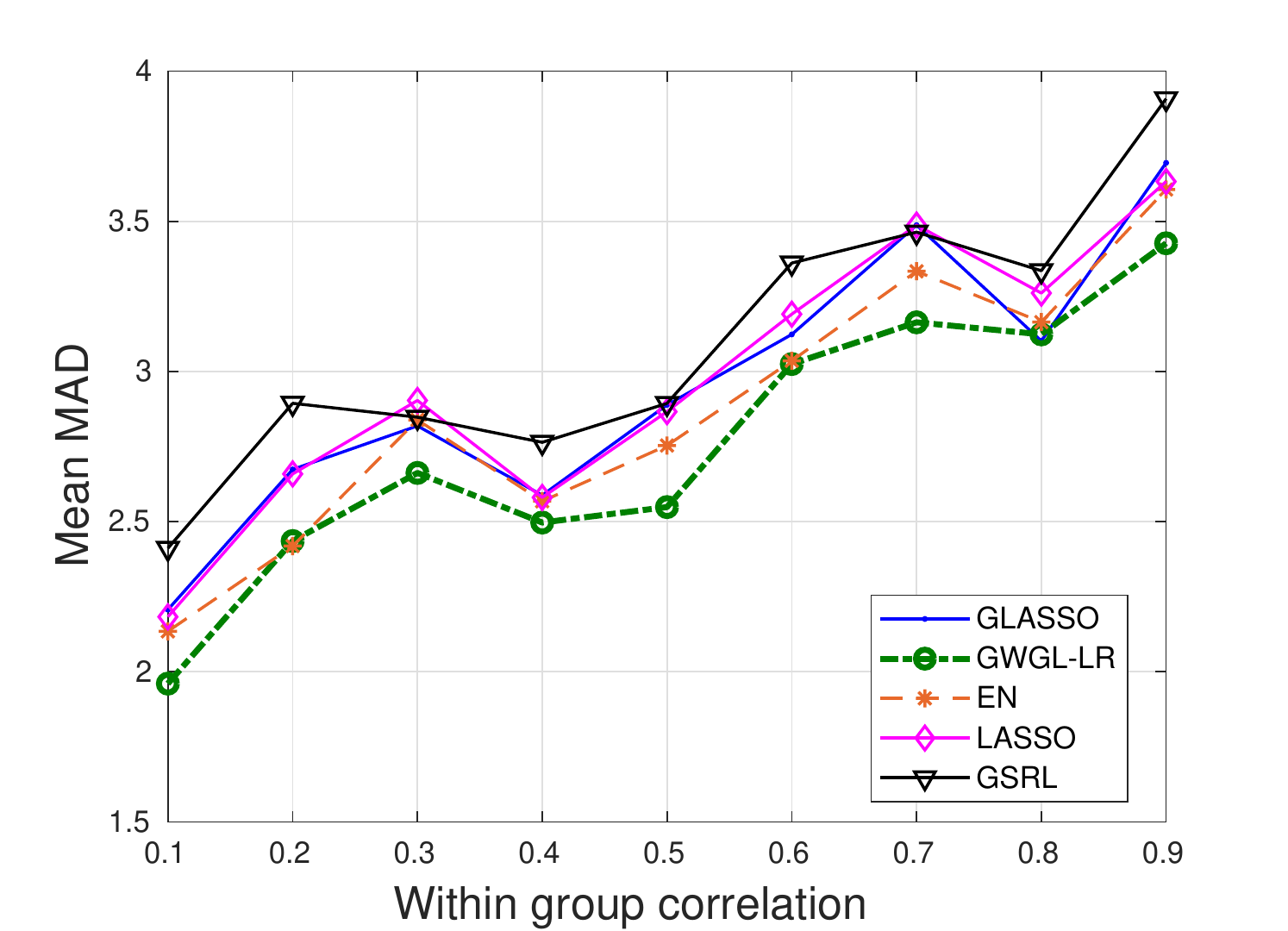}
		\caption{\small{Median Absolute Deviation.}}
	\end{subfigure}
	\begin{subfigure}{0.49\textwidth}
		\centering
		\includegraphics[width=1\textwidth]{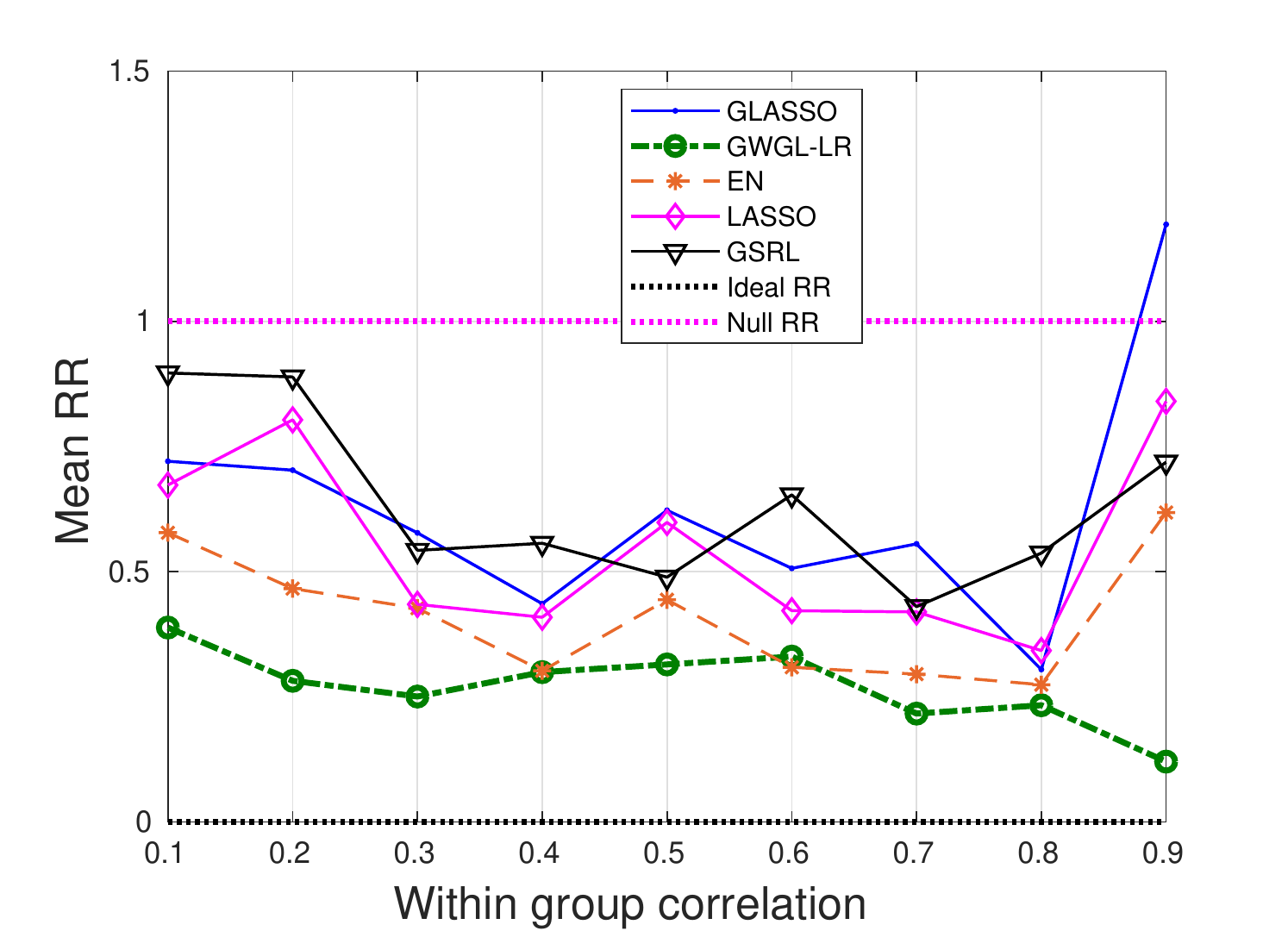}
		\caption{\small{Relative risk.}}
	\end{subfigure}
	
	\begin{subfigure}{0.49\textwidth}
		\centering
		\includegraphics[width=1\textwidth]{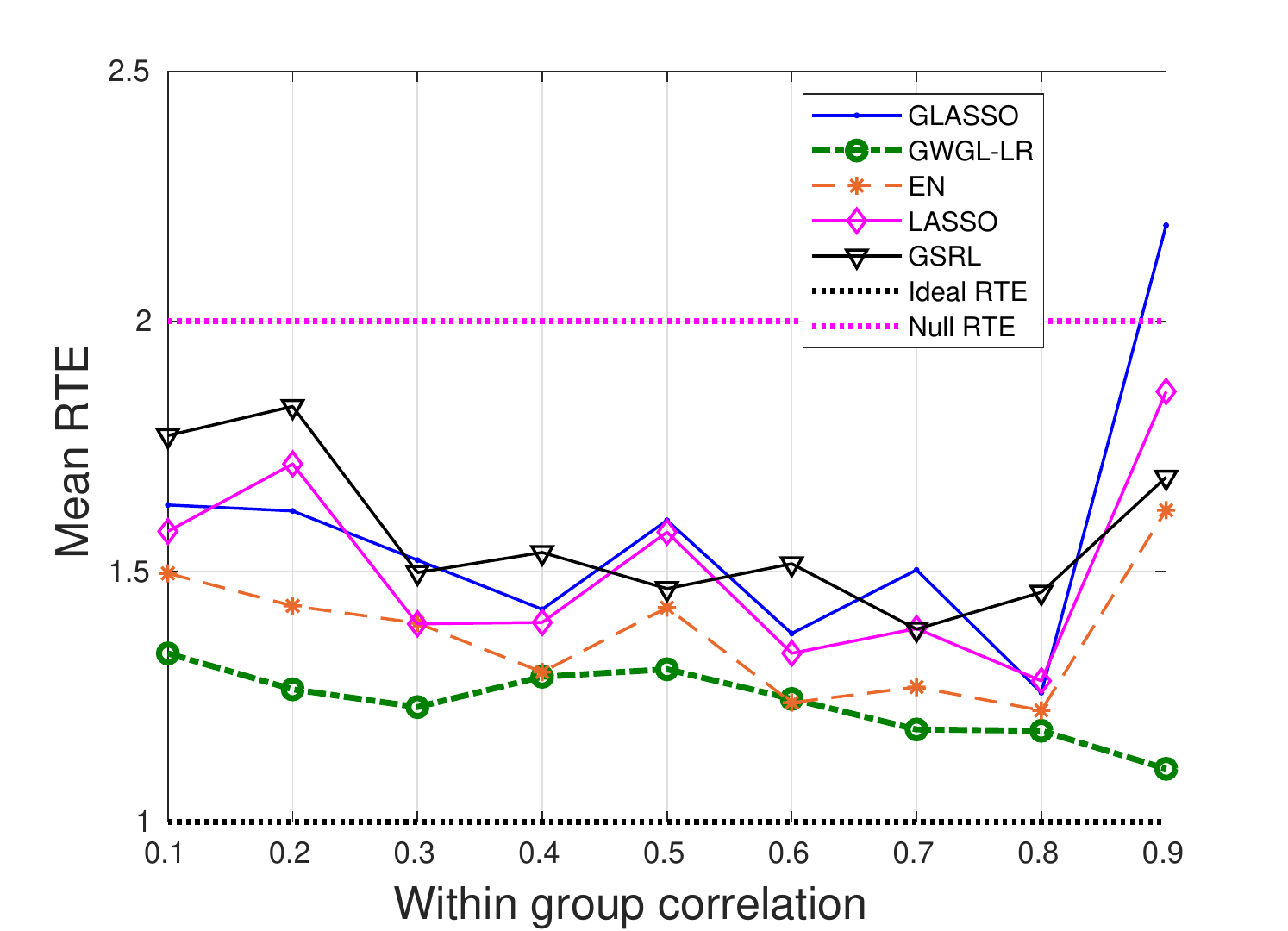}
		\caption{\small{Relative test error.}}
	\end{subfigure}%
	\begin{subfigure}{0.49\textwidth}
		\centering
		\includegraphics[width=1\textwidth]{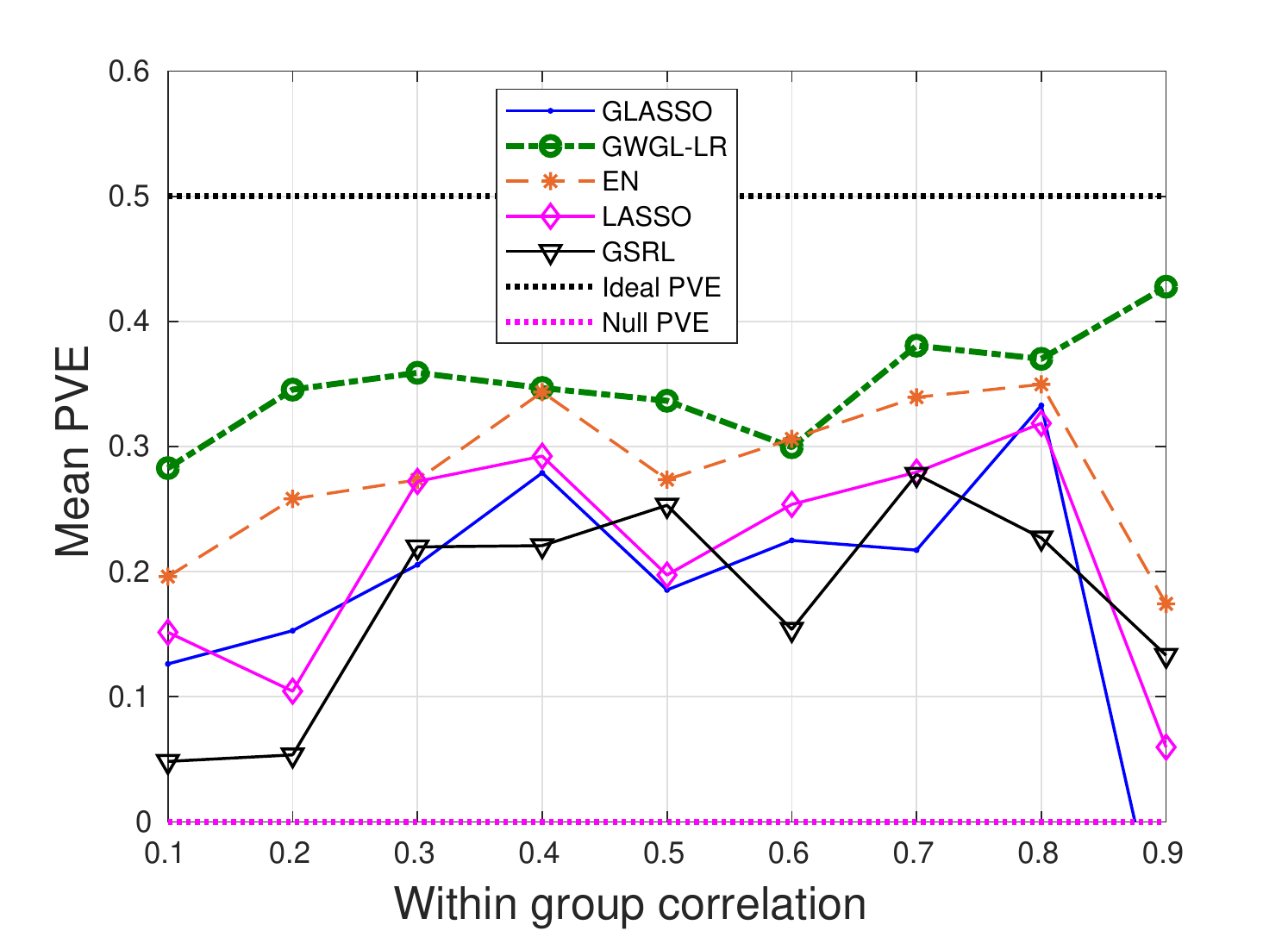}
		\caption{\small{Proportion of variance explained.}}
	\end{subfigure}
	\caption{The impact of within group correlation on the performance metrics, $q =
		20\%$.} \label{cor-20}
\end{figure}

\begin{figure}[hbt] 
	\begin{subfigure}{.49\textwidth}
		\centering
		\includegraphics[width=1\textwidth]{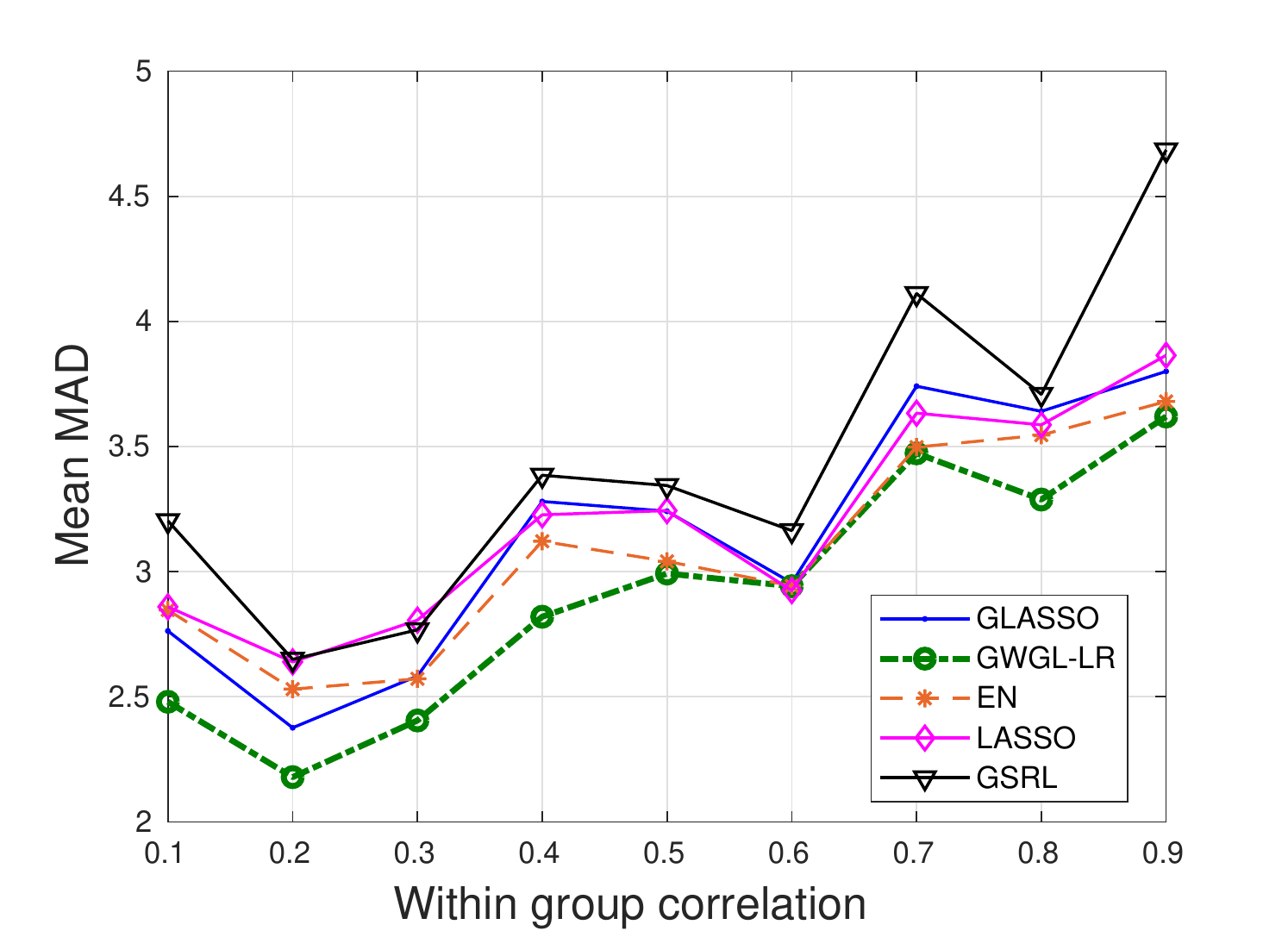}
		\caption{\small{Median Absolute Deviation.}}
	\end{subfigure}
	\begin{subfigure}{0.49\textwidth}
		\centering
		\includegraphics[width=1\textwidth]{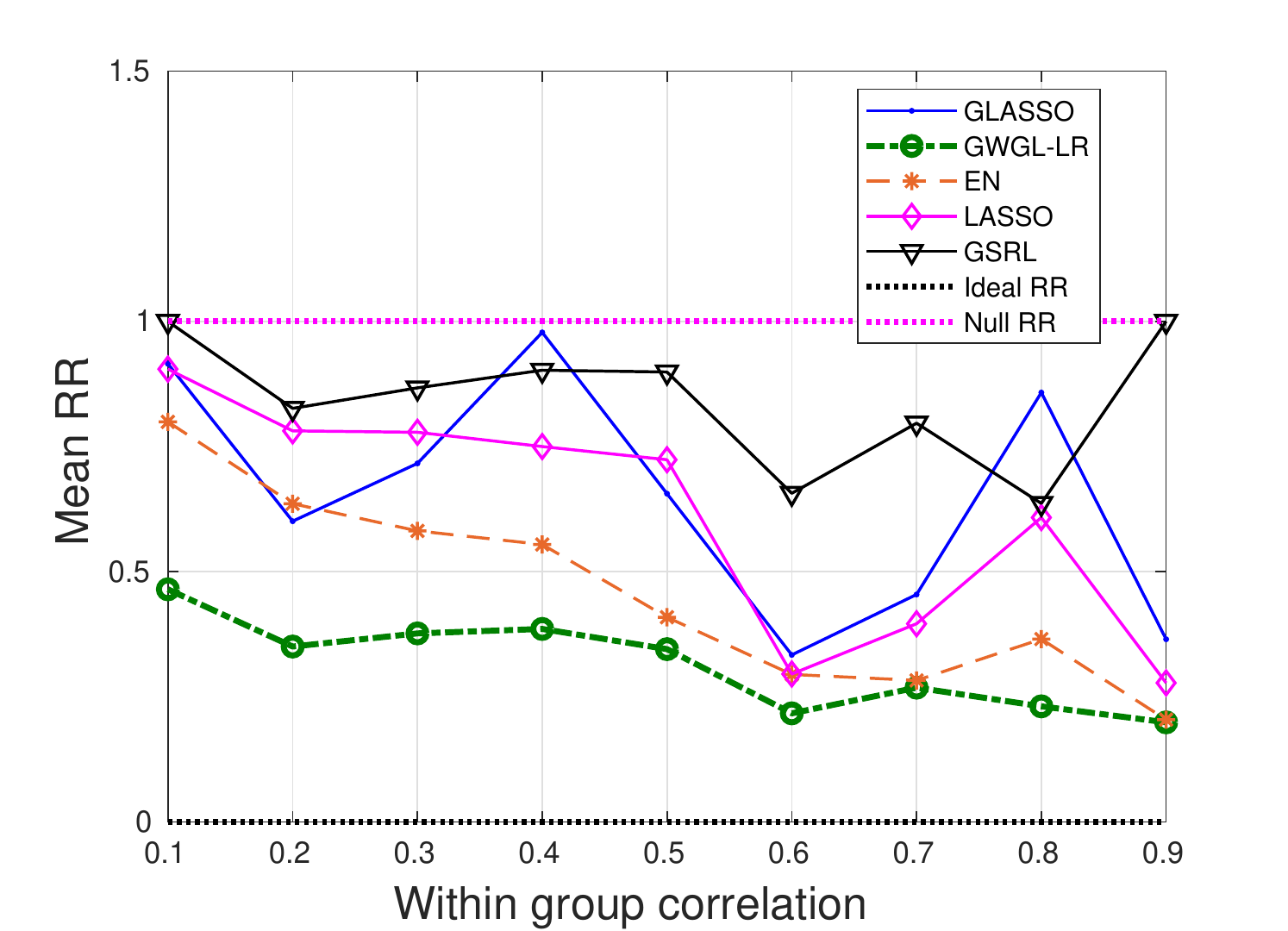}
		\caption{\small{Relative risk.}}
	\end{subfigure}
	
	\begin{subfigure}{0.49\textwidth}
		\centering
		\includegraphics[width=1\textwidth]{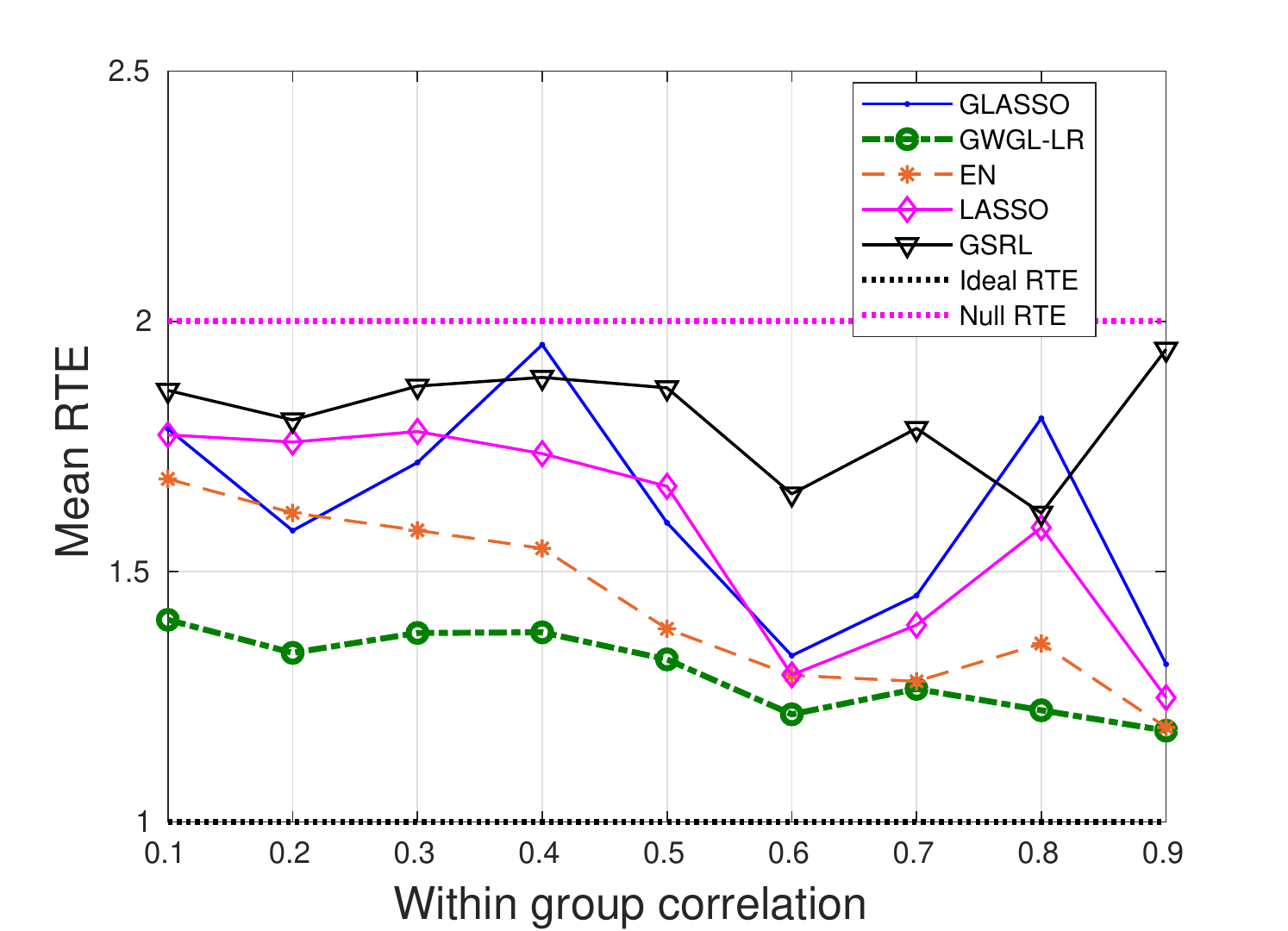}
		\caption{\small{Relative test error.}}
	\end{subfigure}%
	\begin{subfigure}{0.49\textwidth}
		\centering
		\includegraphics[width=1\textwidth]{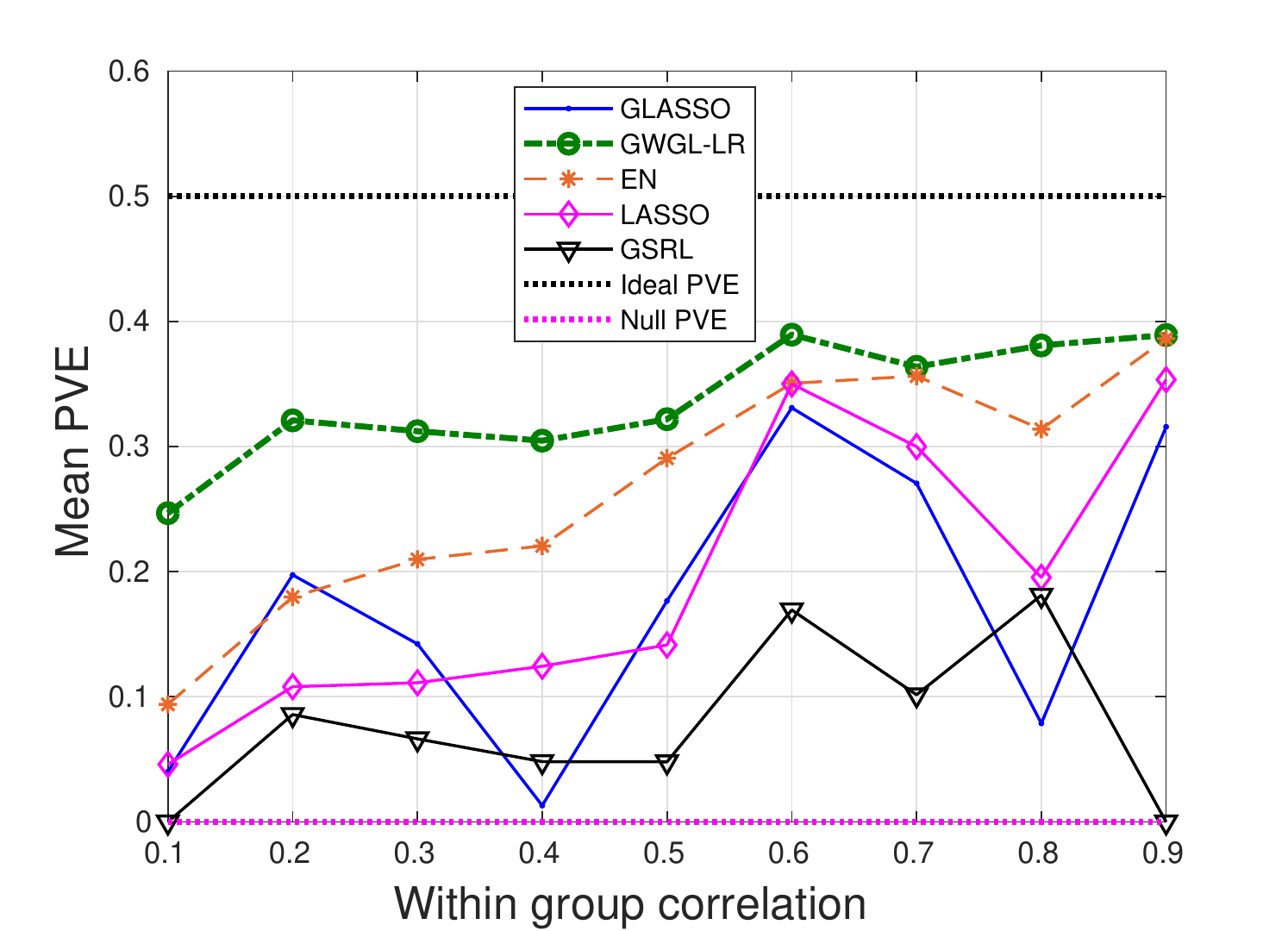}
		\caption{\small{Proportion of variance explained.}}
	\end{subfigure}
	\caption{The impact of within group correlation on the performance metrics, $q =
		30\%$.} \label{cor-30}
\end{figure}

To better highlight the benefits of GWGL-LR, in
Tables~\ref{table1} and \ref{table2} we summarize the
{\em Maximum Percentage Improvement (MPI)} brought about by our methods
compared to other procedures, when varying the SNR and
$\rho_w$, respectively. In all tables, the number outside the parentheses
is the MPI value corresponding to each metric, while the number in the
parentheses indicates the value of SNR/$\rho_w$ at which the MPI
is attained. For each performance metric, the MPI is defined as the
maximum percentage difference of the performance between GWGL-LR and the best among all others.

\begin{table}[hbt]
	\caption{Maximum percentage improvement of all metrics when varying
		the SNR.} \label{table1} 
	\begin{center}
		\begin{tabular}{ c|>{\centering\arraybackslash}p{2cm}|>{\centering\arraybackslash}p{2cm}|>{\centering\arraybackslash}p{2cm}|>{\centering\arraybackslash}p{2cm} } 
			\hline
			& MAD & RR   & RTE   & PVE \\ \hline 
			$q = 20\%$ & 13.7 (0.5) & 41.4 (1.47) & 13.1 (1.47) & 68.9 (0.79)  \\ 
			$q = 30\%$ & 14.7 (1.08) & 40.9 (1.08) &  17 (1.08) & 85.7 (0.68)  \\ 
			\hline
		\end{tabular}
	\end{center}
\end{table}

\begin{table}[hbt]
	\caption{Maximum percentage improvement of all metrics when varying the within group correlation.} \label{table2} 
	\begin{center}
		\begin{tabular}{ c|>{\centering\arraybackslash}p{2cm}|>{\centering\arraybackslash}p{2cm}|>{\centering\arraybackslash}p{2cm}|>{\centering\arraybackslash}p{2cm} } 
			\hline
			& MAD & RR   & RTE   & PVE \\ \hline 
			$q = 20\%$ & 8.2 (0.1) & 80.5 (0.9) & 31.8 (0.9)  & 145.4 (0.9)  \\ 
			$q = 30\%$ & 10.2 (0.1) & 41.9 (0.1) & 16.7 (0.1) & 162.5 (0.1) \\ 
			\hline
		\end{tabular}
	\end{center}
\end{table}

We summarize below our main findings from the results we have presented:
\begin{itemize}
	\item For all approaches under consideration, MAD and RR decrease as the
	data becomes less noisy. PVE increases when the noise is reduced.
	
	\item The GWGL-LR formulation has better prediction and estimation performances than all other approaches under consideration.
	
	\item The relative improvement of GWGL-LR over GLASSO (with an $\ell_2$-loss) is more significant for highly noisy data (with low SNR values or a high percentage of outliers), which can be attributed to the $\ell_1$-loss function it uses. Moreover, GWGL-LR generates more stable estimators than GLASSO.
	
	\item When the within group correlation is varied, GWGL-LR shows a more stable performance than others. 
\end{itemize}

\subsection{GWGL-LG on Synthetic Datasets}
%% Classification
In this subsection we explore the GWGL-LG formulation on synthetic datasets. The data generation process is described as follows:
\begin{enumerate}
	\item Generate $\bbeta^*$ based on the following rule:
	\begin{align*}
	\beta^*_k = 
	\begin{cases}
	\scrU[2.5, 7], & \text{if $\beta^*_k \in (\bbeta^*)^l$ where $l$ is even;} \\
	0, & \text{otherwise},
	\end{cases}
	\end{align*}	
	where $\scrU[2.5, 7]$ stands for a random variable that is uniformly distributed on the interval $[2.5, 7]$.
	\item Generate the predictor $\bx \in \mbb{R}^{p}$ from the Gaussian distribution $\scrN_p (0, \bSigma)$, where
	$\bSigma=(\sigma_{i,j})_{i,j=1}^p$ has diagonal elements equal to
	$1$, and off-diagonal elements specified as:
	\begin{equation*}
	\sigma_{i,j}=
	\begin{cases}
	0.9,   & \text{if predictors $i$ and $j$ are in the same group}; \\
	0,      & \text{otherwise}.
	\end{cases}
	\end{equation*}
	\item Generate the response $y$ as follows:
	\begin{align*}
	y \sim 
	\begin{cases}
	\scrB\big(\big[1+e^{-(\bx'\bbeta^* + \scrN(0, \sigma^2))}\big]^{-1} \big), & \text{if $r \le 1-q$;} \\
	\scrB(0.5), & \text{otherwise},
	\end{cases}
	\end{align*}	
	where $\scrB(p)$ stands for the Bernoulli distribution with the probability of success $p$; $\sigma^2 = (\bbeta^*)'\bSigma \bbeta^*$; $r$ is a uniform random variable on $[0, 1]$; and $q$ is the probability (proportion) of abnormal samples (outliers).  
\end{enumerate}
We generate 10 datasets consisting of $100$ observations and
4 groups of predictors, $80\%$ of which constitute the training dataset, and the remaining forming the test set. The number of predictors in each group
is: $p_1 = 3, p_2 = 4, p_3 = 6, p_4 = 7$, and $p=\sum_{i=1}^4
p_i=20$. The following performance metrics will be used to evaluate the prediction and estimation quality of the solutions:
\begin{itemize}
	\item The Correct Classification Rate (CCR) on the test dataset, which is defined to be the proportion of test set samples that are correctly classified by the classifier $\hat{\bbeta}$, with a threshold $0.5$ on the predicted probability of success.
	\item The {\em AUC (Area Under the ROC Curve)} on the test dataset.
	\item The average logloss on the test set.
	\item The {\em Within Group Difference (WGD)} of the classifier $\hat{\bbeta}$, defined as: 
	\begin{equation*}
	\text{WGD}(\hat{\bbeta}) \triangleq \frac{1}{|\{l: p_l \ge 2\}|} \sum_{l: p_l \ge 2} \frac{1}{\binom{p_l}{2}} \sum_{x_i, x_j \in \bx^l} \biggl|\frac{\hat{\beta}_i - \hat{\beta}_j}{\bx_{,i}'\bx_{,j}}\biggr|,
	\end{equation*}
	where $|\{l: p_l \ge 2\}|$ denotes the cardinality of the set $\{l: p_l \ge 2\}$, and $\bx_{,i}'\bx_{,j}$ measures the sample correlation between predictors $x_i$ and $x_j$ ($\bx_{,i}$ and $\bx_{,j}$ are standardized). $\text{WGD}(\hat{\bbeta})$ essentially evaluates the ability of $\hat{\bbeta}$ to induce group level sparsity. It is desired that the coefficients in the same group are close so that they can be jointly selected/dropped. Theorem~\ref{grouping-lg} implies that the higher the correlation, the smaller the difference between the coefficients, and thus, a smaller WGD value would suggest a stronger ability of grouped variable selection.
\end{itemize}
Notice that the first three metrics mentioned above evaluate the {\em prediction} quality of the classifier, while the last one evaluates its {\em estimation} quality. If the true coefficient vector $\bbeta^*$ is known, we will also use the following {\em confusion matrix} which summarizes the number of zero/nonzero elements in $\hat{\bbeta}$ that are zero/nonzero in the true coefficient $\bbeta^*$. 
\renewcommand{\arraystretch}{1.3}
\begin{table}[H]
	\centering
	\caption{Confusion Matrix.}
	\label{confusionmat}
	{\normalsize \begin{tabular}{c c c }
			\cline{1-3}
			& \multicolumn{2}{c}{$\bbeta^*$}\\
			\cline{2-3}
			$\hat{\bbeta}$ & Nonzero & Zero  \\
			\hline
			\multicolumn{1}{c}{Nonzero} & True Association (TA) & False Association (FA) \\
			\multicolumn{1}{c}{Zero} & False Disassociation (FD) & True Disassociation (TD) \\
			%\multicolumn{1}{c}{Total by True coef.} & TA+FD & FA+TD & \\
			\hline
		\end{tabular}}
	\end{table}
	\renewcommand{\arraystretch}{1}
	Two ratios will be computed using Table~\ref{confusionmat}, the {\em True Association Rate (TAR)} defined as:
	\begin{equation*}
	\text{TAR} = \frac{\text{TA}}{\text{TA}+\text{FD}},
	\end{equation*}
	which calculates the proportion of nonzero coefficients that are correctly discovered by the estimator, and the {\em True Disassociation Rate (TDR)} defined as:
	\begin{equation*}
	\text{TDR} = \frac{\text{TD}}{\text{FA}+\text{TD}},
	\end{equation*}
	which calculates the proportion of zero coefficients that are correctly identified as zero by the estimator.
	
	We compare GWGL-LG with four formulations: the vanilla logistic regression (LG) that minimizes the empirical logloss on the training samples (without penalty), LG-LASSO that imposes an $\ell_1$-norm regularizer on $\bbeta$, LG-Ridge that uses an $\ell_2$-norm regularizer, and LG-EN that uses both the $\ell_1$- and $\ell_2$- regularizers. All the penalty (regularization) parameters are tuned in the same way as Section~\ref{gwgl-lr-exp}. The
	penalty parameter 
	for GWGL-LG is tuned over $50$ values
	ranging from $\lambda_m = \max_{l \in \lb L \rb} \big( \|(\bX^l)'(\by - \bar{y} \mathbf{1})\|_2/\sqrt{p_l}\big)$ to a small fraction of
	$\lambda_m$ on a log scale \citep{meier2008group}, where $\bX^l$ consists of the columns of the design matrix $\bX$ corresponding to group $l$, $\by$ is the vector of training set labels, and $\bar{y} = \by'\mathbf{1}/N$. The maximum penalty parameter $\lambda_m$ for LG-LASSO is computed by recognizing it as a special case of GWGL-LG where each group contains only one predictor. For LG-Ridge, $\lambda_m$ is set to be the square root of the maximum penalty parameter for LG-LASSO, due to the fact that we penalize the square of the $\ell_2$-norm regularizer in LG-Ridge. The range of penalty parameters for LG-EN is set in a similar way.
	
	Similar to Section~\ref{gwgl-lr-exp}, the spectral clustering algorithm with the Gaussian similarity function (\ref{gs}) is used to perform grouping on the predictors. We experiment with two scenarios: $(i)$ $q = 20\%$, and $(ii)$ $q = 30\%$.  The results are shown in Tables~\ref{perf-syn-20} and \ref{perf-syn-30}, where the number outside the parentheses
	is the mean value across $10$ repetitions, and the number in the
	parentheses is the corresponding standard deviation. 
	
	\begin{table}[hbt]
		\caption{The performances of different classification formulations on synthetic datasets, $q = 20\%$.} \label{perf-syn-20} 
		\begin{center}
			{\footnotesize \begin{tabular}{ c|>{\centering\arraybackslash}p{1.1cm}|>{\centering\arraybackslash}p{1.1cm}|>{\centering\arraybackslash}p{1.1cm}|>{\centering\arraybackslash}p{1.1cm}|>{\centering\arraybackslash}p{1.1cm}|>{\centering\arraybackslash}p{1.1cm} } 
					\hline
					& CCR         & AUC          & logloss      & WGD         & TAR         & TDR \\ \hline 
					LG       & 0.62 (0.14) & 0.67 (0.13)  & 0.87 (0.24)  & 1.71 (0.32) & 1.00 (0.00)       & 0.00 (0.00)
					\\ 	\hline 
					LG-LASSO & 0.69 (0.14) & 0.77 (0.12)  & 0.60 (0.13)  & 0.33 (0.19) & 0.54 (0.19) & 0.42 (0.23)  \\ \hline 
					LG-Ridge & 0.67 (0.12) & 0.72 (0.14)  & 0.69 (0.19)  & 0.82 (0.42) & 1.00 (0.00)       & 0.01 (0.04)  \\ \hline
					LG-EN    & 0.70 (0.15) & 0.77 (0.13)  & 0.59 (0.13)  & 0.23 (0.08) & 0.57 (0.21) & 0.42 (0.24) \\ \hline
					GWGL-LG  & 0.68 (0.15) & 0.79 (0.12)  & 0.59 (0.14)  & 0.13 (0.07) & 0.98 (0.04) & 0.28 (0.34)\\
					\hline
				\end{tabular}}
			\end{center}
		\end{table} 
		
		\begin{table}[hbt]
			\caption{The performances of different classification formulations on synthetic datasets, $q = 30\%$.} \label{perf-syn-30} 
			\begin{center}
				{\footnotesize \begin{tabular}{ c|>{\centering\arraybackslash}p{1.1cm}|>{\centering\arraybackslash}p{1.1cm}|>{\centering\arraybackslash}p{1.1cm}|>{\centering\arraybackslash}p{1.1cm}|>{\centering\arraybackslash}p{1.1cm}|>{\centering\arraybackslash}p{1.1cm} } 
						\hline
						& CCR          & AUC          & logloss     & WGD          & TAR         & TDR      \\ \hline 
						LG       & 0.63 (0.09)  & 0.68 (0.08)  & 0.99 (0.21) & 2.68 (0.51)  & 1.00 (0.00) & 0.00 (0.00)     \\ \hline
						LG-LASSO & 0.73 (0.08)  & 0.73 (0.08)  & 0.65 (0.11) & 0.56 (0.37)  & 0.58 (0.21) & 0.43 (0.23)\\ \hline
						LG-Ridge & 0.72 (0.06)  & 0.74 (0.06)  & 0.64 (0.10) & 0.78 (0.58)  & 0.98 (0.04) & 0.00 (0.00)       \\ \hline
						LG-EN    & 0.74 (0.08)  & 0.77 (0.06)  & 0.59 (0.08) & 0.24 (0.09)  & 0.48 (0.14) & 0.42 (0.23)  \\ \hline
						GWGL-LG  & 0.73 (0.08)  & 0.78 (0.06)  & 0.60 (0.09) & 0.21 (0.16)  & 0.99 (0.03)  & 0.18 (0.09)   \\
						\hline
					\end{tabular}}
				\end{center}
			\end{table}
			
			We see that in general, the penalized formulations perform significantly better than the vanilla logistic regression. LG-EN has very similar prediction performance (i.e., CCR, AUC and logloss on the test set) to GWGL-LG, better than LG-LASSO and LG-Ridge. Regarding the estimation quality, the penalized formulations achieve much lower WGD values than LG. LG-Ridge does not induce sparsity, and therefore has the highest WGD among the four regularized models. LG-LASSO shows a relatively small WGD, due to the sparsity inducing (at the individual level) property of the $\ell_1$-regularizer. GWGL-LG achieves the smallest WGD among all (significantly lower than that of LG-EN), which provides empirical evidence on its group sparsity inducing property, and is consistent with our earlier discussion in Theorem~\ref{grouping-lg} that the GLASSO penalty tends to drive the coefficients in the same group to the same value if the within group correlation is high. Moreover, GWGL-LG successfully drops out all the coefficients in the first group, while other formulations are not able to drop any of the four groups. 
			
			Regarding the TAR and TDR, we notice that GWGL-LG obtains very high TAR values, and compared to other formulations that achieve almost perfect TARs (e.g., LG-Ridge and LG), it has a significantly higher TDR. LG-LASSO and LG-EN achieve the highest TDRs, but their TARs are significantly worse. Note that a dense estimator would result in a perfect TAR but a zero TDR, as in LG and LG-Ridge. The higher the TDR, the more parsimonious the model is, but on the other hand, a higher TAR is more appreciated as we do not want to leave out any of the important (effective) predictors. A low TAR means that a substantial proportion of the meaningful predictors are dropped, the cost of which is usually much higher than the cost of wrongly selecting the unimportant ones. Therefore, taking into account both the parsimony and effectiveness of the model, GWGL-LG outperforms all others.  
			
			We also want to highlight the robustness of GWGL-LG to misspecified groups. For example, in the scenario with $q = 30\%$ outliers, even though spectral clustering outputs a wrong grouping structure (it divides the data into four groups with group size being $p_1 = 2, \ p_2 = 3, \  p_3 = 5, \ p_4 = 10$, and the correct group size is $p_1 = 3, \ p_2 = 4, \ p_3 = 6, \ p_4 = 7$), GWGL-LG is still able to achieve a satisfactory prediction performance and an almost perfect TAR with a sparse model (nonzero TDR). 
			
			\subsection{An Application to Hospital Readmission}
			In this section we test the GWGL formulations on a real dataset containing
			medical records of patients who underwent a general surgical
			procedure. In 2005, the American College of Surgeons (ACS) established
			the National Surgical Quality Improvement Program (NSQIP), which
			collects detailed demographic, laboratory, clinical, procedure and
			postoperative occurrence data in several surgical subspecialties. The
			dataset includes $(i)$ baseline demographics; $(ii)$ pre-existing comorbidity information; $(iii)$ preoperative variables; $(iv)$ index admission-related
			diagnosis and procedure information; $(v)$ postoperative events and
			complications, and $(vi)$ additional
			socioeconomic variables. 
			
			In our study, patients who underwent a general surgery procedure over 2011--2014 and were tracked by the NSQIP were
			identified. We will focus on two supervised learning models: $(i)$ a linear regression model whose objective is to predict the
			post-operative hospital length of stay using pre- and intra-operative
			variables, and $(ii)$ an LG model whose objective is to predict the re-hospitalization of patients within 30 days after discharge using the same set of explanatory variables. Both models are extremely useful as they allow hospital staff to predict
			post-operative bed occupancy and prevent costly 30-day readmissions.
			
			Data were pre-processed as follows: $(i)$ categorical variables (such as
			race, discharge destination, insurance type) were numerically encoded
			and units homogenized; $(ii)$ missing values were replaced by the mode;
			$(iii)$ all variables were normalized by subtracting the mean and divided
			by the standard deviation; $(iv)$ patients who died within 30 days of
			discharge or had a postoperative length of stay greater than 30 days
			were excluded. After pre-processing, we were left with a total of
			$2,275,452$ records. 
			
			After encoding the categorical predictors using indicator variables, we have $131$ numerical predictors for the regression model and $132$ for the classification model (the post-operative hospital length of stay is used as a predictor for the 30-day re-hospitalization prediction). The spectral clustering algorithm is used to perform grouping on the predictors, with the number of groups specified as $67$ based on a preliminary analysis of the data. (The eigengap heuristic \citep{von2007tutorial} mentioned in Section~\ref{gwgl-lr-exp} was used.)
			
			For predicting the post-operative hospital length of stay, we report the out-of-sample MAD in
			Table~\ref{table7}, i.e., the median of the absolute difference between the predicted and actual length of stay on the test set. The mean and standard deviation of the MAD are computed across $5$ repetitions, each with a different training set. We see that the GWGL-LR formulation achieves
			the lowest mean MAD with a small variation. Compared to the best
			among others (GLASSO with $\ell_2$-loss), it improves the mean MAD by $7.30\%$. For longer hospital length of stay, this could imply 1 or 2 days improvement in prediction accuracy, which is both clinically and economically meaningful and significant.
			
			%\begin{table}[hbt]
			%	\caption{The Mean and standard deviation of out-of-sample MAD on
			%		the surgery data.} \label{table7}
			%	\begin{center}
			%		{\footnotesize 	 \begin{tabular}{c|>{\centering\arraybackslash}p{3.6cm}|>{\centering\arraybackslash}p{2cm}|>{\centering\arraybackslash}p{2cm}|>{\centering\arraybackslash}p{2cm}|>{\centering\arraybackslash}p{2cm}} 
			%				\hline
			%				& GLASSO with $\ell_2$-loss         & GWGL-LR            & EN          & LASSO     & GSRL  \\ \hline
			%				Mean (Std.) & 0.17 (0.0007)          & 0.16 (0.001)         & 0.17 (0.0009)      & 0.17 (0.0009) &  0.17 (0.0009)   \\ \hline
			%		\end{tabular}}
			%	\end{center}
			%\end{table}
			
			\begin{table}[hbt]
				\caption{The Mean and standard deviation of out-of-sample MAD on
					the surgery data.} \label{table7}
				\begin{center}
					\begin{tabular}{c|>{\centering\arraybackslash}p{2cm}|>{\centering\arraybackslash}p{3.5cm}} 
						\hline
						& Mean  & Standard Deviation \\
						\hline
						GLASSO with $\ell_2$-loss & 0.17  & 0.0007 \\
						GWGL-LR                   & 0.16  & 0.001 \\
						EN                        & 0.17  & 0.0009 \\
						LASSO                     & 0.17  & 0.0009 \\
						GSRL                      & 0.17  & 0.0009 \\
						\hline		
					\end{tabular}
				\end{center}
			\end{table}
			
			For predicting the 30-day re-hospitalization of patients, we notice that the dataset is highly unbalanced, with only $6\%$ of patients being re-hospitalized. To obtain a balanced training set, we randomly draw $20\%$ patients from the positive class (re-hospitalized patients), and sample the same number of patients from the negative class, resulting in a training set of size $53,616$. All the remaining patients go to the test dataset. It turns out that the prediction capabilities of all approaches are very similar. All formulations achieve an average {\em out-of-sample CCR} around $0.62$, an average {\em out-of-sample AUC} of $0.83$, and an average {\em logloss} on the test set ranging from $0.84$ to $0.87$. From Table~\ref{wgd-surgery} we see that GWGL-LG obtains a significantly smaller WGD than others, which implies that the GWGL-LG formulation encourages group level sparsity. This can also be revealed by the number of groups that are dropped by various formulations (see Table~\ref{dropgroup}). Notice that though LG-EN and LG-LASSO obtain the most parsimonious models in terms of the number of dropped features (sparsity at an individual level), GWGL-LG has a stronger ability to induce group level sparsity. 
			
			\begin{table}[hbt]
				\caption{The Within Group Difference (WGD) of the estimators on the surgery data.} \label{wgd-surgery} 
				\begin{center}
					\begin{tabular} { c|>{\centering\arraybackslash}p{2cm}|>{\centering\arraybackslash}p{3.5cm}} 
						\hline
						& Mean & Standard Deviation \\
						\hline
						LG & 23.93 & 1.28 \\
						LG-LASSO   & 16.28 & 0.72 \\
						LG-Ridge &  23.38 & 1.15 \\
						LG-EN & 16.26 & 0.74 \\
						GWGL-LG & 5.04 & 0.45 \\
						\hline
					\end{tabular}
				\end{center}
			\end{table}
			
			\begin{table}[hbt]
				\caption{{\small The number of groups/features dropped by various formulations on the surgery data.}} \label{dropgroup} 
				\begin{center}
					{\small \begin{tabular}{ c|>{\centering\arraybackslash}p{4.3cm}|>{\centering\arraybackslash}p{4.3cm}} 
							\hline
							& Number of dropped groups & Number of dropped features \\
							\hline
							LG         & 1 & 2 \\
							LG-LASSO   & 6 & 24 \\
							LG-Ridge   & 2 & 2 \\
							LG-EN      & 10 & 25 \\
							GWGL-LG    & 16 & 19 \\
							\hline
						\end{tabular}}
					\end{center}
				\end{table}
				
				\section{Summary} \label{sec:3-5}
				In this section we presented a {\em Distributionally Robust Optimization (DRO)} formulation under the Wasserstein
				metric that recovers the GLASSO penalty for {\em Least Absolute Deviation (LAD)} and LG, through which we
				have established a connection between group-sparse regularization and
				robustness and offered new insights into the group sparsity penalty term. We provided insights on the
				grouping effect of the estimators, which suggests the use of spectral clustering with the Gaussian similarity function to perform grouping on the predictors. We established finite-sample
				bounds on the prediction errors, which justify the form of the regularizer and provide guidance
				on the number of training samples needed in order to achieve specific
				out-of-sample accuracy.  
				
				We reported results from several experiments, using both synthetic data
				and a real dataset with surgery-related medical records. It has been observed
				that the GWGL formulations $(i)$ achieve more accurate and stable estimates
				compared to others, especially when the data
				are noisy, or potentially contaminated with outliers; $(ii)$ have a stronger ability of
				inducing group-level sparsity, and thus producing more interpretable models, and
				$(iii)$ successfully identify most of the effective predictors with a reasonably
				parsimonious model. 

\chapter{Distributionally Robust Multi-Output Learning} \label{chap:multi}
In this section, we focus on robust multi-output learning where a multi-dimensional response/label vector is to be learned. The difference from previous sections lies in that we need to estimate a coefficient matrix, rather than a coefficient vector, to explain the dependency of each response variable on the set of predictors. We develop {\em Distributionally Robust Optimization (DRO)} formulations under the Wasserstein metric for {\em Multi-output Linear Regression (MLR)} and {\em Multiclass Logistic Regression (MLG)}, when both the covariates and responses/labels may be contaminated by outliers. Through defining a new notion of matrix norm, we relax the DRO formulation into a regularized learning problem whose regularizer is the norm of the coefficient matrix, establishing a connection between robustness and regularization and generalizing the single-output results presented in Section~\ref{chapt:dro}.

\section{The Problem and Related Work}
We consider the multi-output learning problem under the framework of {\em Distributionally Robust Optimization (DRO)} where the ambiguity set is defined via the Wasserstein metric \citep{gao2016distributionally, gao2017wasserstein, shafieezadeh2017regularization, esfahani2018data}. The term multi-output learning refers to scenarios where multiple correlated responses are to be predicted - {\em Multi-output Linear Regression (MLR)}, or one of multiple classes is to be assigned - {\em MultiClass Classification (MCC)}, based on a linear combination of a set of predictors. Both involve learning a target vector $\by$ from a vector of covariates $\bx$. MLR has many applications in econometrics \citep{zhang2013regression}, health care \citep{hidalgo2013multivariate, peng2010regularized}, and finance \citep{islam2013financial, tsay2013multivariate}, for modeling multiple measurements of a single individual \citep{friston1994statistical}, or evaluating a group of interdependent variables \citep{breiman1997predicting}. MCC has seen wide applications in image segmentation \citep{plath2009multi}, text classification \citep{forman2003extensive}, and bioinformatics \citep{li2004comparative}.

Unlike a single-output learning problem where the response variable is scalar and a coefficient vector representing the dependency of the response on the predictors is to be learned, in the multi-output setting the decision variable is a coefficient matrix $\bB \in \mbb{R}^{p \times K}$ whose $k$-th column explains the variation in the $k$-th coordinate of $\by \in \mbb{R}^K$ that can be attributed to the predictors $\bx \in \mbb{R}^p$, for $k \in \lb K \rb$. Inspired by the DRO relaxation derived in Section~\ref{chapt:dro} for the single-output case, which adds a dual norm regularizer to the empirical loss, we obtain a novel \textbf{matrix norm} regularizer for the multi-output case through reformulating the Wasserstein DRO problem. The matrix norm exploits the geometrical structure of the coefficient matrix, and provides a way of associating the coefficients for the potentially correlated responses through the dual norm of the distance metric in the data space. 

As the simplest MLR model, the multi-output extension of OLS regresses each response variable against the predictors independently, which does not take into account the potential correlation between the responses, and is vulnerable to high correlations existing among the predictors. A class of methods that are used in the literature to overcome this issue is called linear factor regression, where the response $\by$ is regressed against a small number of linearly transformed predictors (factors). Examples include reduced rank regression \citep{izenman1975reduced, velu2013multivariate}, principal components regression \citep{massy1965principal}, and {\em Factor Estimation and Selection (FES)} \citep{yuan2007dimension}. Another type of methods applies multivariate shrinkage by either estimating a linear transformation of the OLS predictions \citep{breiman1997predicting}, or solving a regularized MLR problem, e.g., ridge regression \citep{brown1980adaptive, haitovsky1987multivariate}, and FES \citep{yuan2007dimension}, whose regularizer is the coefficient matrix's Ky Fan norm defined as the sum of its singular values. 

As for the popular MCC models, \cite{aly2005survey} provided a thorough survey on
the existing MCC techniques which can be categorized into: $(i)$ transformation to
binary, e.g., one vs. rest and one vs. one; $(ii)$ extension from binary, e.g.,
decision trees \citep{breiman2017classification}, neural networks
\citep{bishop1995neural}, K-Nearest Neighbor \citep{bay1998combining}, Naive Bayes
classifiers \citep{rish2001empirical}, and Support Vector Machine (SVM) \citep{cortes1995support}; and $(iii)$ hierarchical classification \citep{kumar2002hierarchical}. 

The research on robust classification has mainly focused on binary classifiers. For example, to robustify logistic regression, \cite{feng2014robust} proposed to optimize a robustified linear correlation between the response $y$ and a linear function of $\bx$; \cite{ding2013t} introduced T-logistic regression which replaces the exponential distribution in LG by the t-exponential distribution family; \cite{tibshirani2013robust} introduced a shift parameter for each data point to account for the label error; and \cite{bootkrajang2012label} modeled the label error through flipping probabilities, which can be extended to multiclass LG. 
Another line of research uses a modified loss function that gives less influence to points far from the boundary, e.g., \cite{masnadi2010design} used a tangent loss, \cite{pregibon1982resistant} proposed an M-estimator like loss metric which, however, is not robust to outliers with high leverage covariates. 

None of the aforementioned works, however, explore distributionally robust learning problems with multiple responses, with the exception of \cite{hu2016does}, which considered distributionally robust multiclass classification models under the $\phi$-divergence metric. We fill this gap by developing DRO formulations for both MLR and MCC under the Wasserstein metric. To the best of our knowledge, we are the first to study the robust multi-output learning problem from the standpoint of distributional robustness. Our approach is completely optimization-based, without the need to explicitly model the complicated relationship between different responses, leading to compact and computationally solvable models. It is interesting that a purely optimization-based method that is completely agnostic to the covariate and response correlation structure can be used as a better-performing alternative to statistical approaches that explicitly model this correlation structure. 

The rest of this Section is organized as follows. In Section~\ref{sec:5-2}, we
develop the DRO-MLR and DRO-MLG formulations and introduce the matrix norm that is used to define the regularizer. Section~\ref{sec:5-3}
establishes the out-of-sample performance guarantees for the solutions to DRO-MLR and DRO-MLG. The numerical experimental results are presented in
Section~\ref{sec:5-4}. We conclude in Section~\ref{sec:5-5}.

\section{Distributionally Robust Multi-Output Learning Models} \label{sec:5-2}
In this section we introduce the Wasserstein DRO formulations for MLR and MLG, and offer a dual norm interpretation for the regularization terms.

\subsection{Distributionally Robust Multi-Output Linear Regression} \label{sec:dro-mlr}
We assume the following model for the MLR problem: 
$$\by = \bB'\bx + \boldsymbol{\eta},$$
where $\by = (y_1, \ldots, y_K)$ is the vector of $K$ responses, potentially
correlated with each other; $\bx = (x_1, \ldots, x_p)$ is the vector of $p$
predictors; $\bB = (B_{ij})_{i \in \lb p \rb} ^{j \in \lb K \rb}$ is the $p \times K$
matrix of coefficients, the $j$-th column of which describes the dependency of $y_j$
on the predictors; and $\boldsymbol{\eta}$ is the random error. Suppose we observe $N$ realizations of the data, denoted by $(\bx_i, \by_i), i \in \lb N \rb$, where $\bx_i = (x_{i1}, \ldots, x_{ip}), \by_i = (y_{i1}, \ldots, y_{iK})$.  
The Wasserstein DRO formulation for MLR minimizes the following worst-case expected loss: 
\begin{equation} \label{dro-mlr} 
\inf\limits_{\bB}\sup\limits_{\mbb{Q} \in \Omega}
\mbb{E}^{\mbb{Q}} [h_{\bB}(\bx, \by)],
\end{equation}
where $h_{\bB}(\bx, \by) \triangleq l(\by-\bB'\bx)$, with $l: \mbb{R}^K \rightarrow \mbb{R}$ an $L$-Lipschitz continuous function on the metric spaces $(\scrD, \|\cdot\|_r)$ and $(\scrC, |\cdot|)$, where $\scrD, \scrC$ are the domain and codomain of $l(\cdot)$, respectively; and $\mbb{Q}$ is the probability distribution of the data $(\bx, \by)$, belonging to a set $\Omega$ defined as
\begin{equation*}
\Omega = \Omega_{\epsilon}^{s,1}(\hat{\mathbb{P}}_N) \triangleq \{\mbb{Q}\in \scrP(\scrZ): W_{s,1}(\mathbb{Q},\ \hat{\mathbb{P}}_N) \le \epsilon\},
\end{equation*}
where the order-1 Wasserstein distance $W_{s,1}(\mbb{Q},\ \hat{\mbb{P}}_N)$ is 
induced by the metric $s(\bz_1, \bz_2) \triangleq \| \bz_1 - \bz_2\|_r$. Notice that we use the same norm to define the Wasserstein metric and the metric space on the domain $\scrD$ of $l(\cdot)$.

Write the loss function as $h_{\tilde{\bB}}(\bz) \triangleq l(\tilde{\bB}\bz)$, where $\bz = (\bx, \by)$, and $\tilde{\bB} = [-\bB', \bI_{K}]$. From Theorem~\ref{Lip-wass}, we know that 
to derive a tractable reformulation for (\ref{dro-mlr}), the key is to bound the following 
{\em growth rate} of the loss:
\begin{equation*}
\text{GR}\big(h_{\tilde{\bB}}\big) \triangleq \limsup_{\|\bz_1 - \bz_2\|_r \rightarrow \infty}\frac{\big|h_{\tilde{\bB}}(\bz_1) - h_{\tilde{\bB}}(\bz_2)\big|}{\|\bz_1 - \bz_2\|_r}.
\end{equation*}
Let us first consider the numerator. By the Lipschitz continuity of $l(\cdot)$, we have:
\begin{equation*}
|h_{\tilde{\bB}}(\bz_1) - h_{\tilde{\bB}}(\bz_2)|  = |l(\tilde{\bB}\bz_1) - l(\tilde{\bB}\bz_2)| 
\le L \|\tilde{\bB}(\bz_1 - \bz_2)\|_r.
\end{equation*}
The key is to bound $\|\tilde{\bB}(\bz_1 - \bz_2)\|_r$ in terms of $\|\bz_1 - \bz_2\|_r$. The following lemmata provide three types of bounds whose tightness will be analyzed in the sequel. 

\begin{lem} \label{lr-bound-1}
	For any matrix $\bA \in \mbb{R}^{m \times n}$ and any vector $\bx \in \mbb{R}^n$, we have:
	\begin{equation*} 
	\|\bA \bx\|_r \le \|\bx\|_r \Big(\sum_{i=1}^m\|\ba_i\|_1^r \Big)^{1/r},
	\end{equation*}
	for any $r \ge 1$, where $\ba_i, i \in \lb m \rb$, are the rows of $\bA$.
\end{lem} 

\begin{proof}
	Suppose $\bA = (a_{ij})_{i \in \lb m \rb}^{j \in \lb n \rb}$. Then,
	\begin{equation*}
	\begin{aligned}
	\|\bA \bx\|_r^r & = \left\|
	\begin{bmatrix}
	a_{11}x_1 + \ldots + a_{1n}x_n \\
	\vdots \\
	a_{m1}x_1 + \ldots + a_{mn}x_n 
	\end{bmatrix}
	\right\|_r^r \\
	& = \sum_{i=1}^m |a_{i1}x_1 + \ldots + a_{in}x_n|^r \\
	& \le \|\bx\|_r^r\Bigl(\sum_{i=1}^m \big(|a_{i1}| + \ldots + |a_{in}|\big)^r \Bigr). \\
	\end{aligned}
	\end{equation*}
	where in the second step we use the fact that $|x_i| \le \|\bx\|_r$.
\end{proof}

\begin{lem} \label{lr-bound-2}
	For any matrix $\bA \in \mbb{R}^{m \times n}$ and any vector $\bx \in \mbb{R}^n$, we have:
	\begin{equation*} 
	\|\bA \bx\|_r \le \|\bx\|_r \Bigl(\sum_{i=1}^m \|\ba_i\|_s^r\Bigr)^{1/r},
	\end{equation*}
	for any $r \ge 1$, where $\ba_i, i \in \lb m \rb$, are the rows of $\bA$, and $1/r + 1/s = 1$.
\end{lem}

\begin{proof}
	\begin{equation*}
	\begin{aligned}
	\|\bA \bx\|_r^r & = \sum_{i=1}^m |\ba_i'\bx|^r \\
	& \le \sum_{i=1}^m  \|\bx\|_r^r \|\ba_i\|_s^r \\
	& = \|\bx\|_r^r \sum_{i=1}^m \|\ba_i\|_s^r,
	\end{aligned}
	\end{equation*}
	where $r, s \ge 1$, $1/r+1/s=1$, and the second step uses H\"{o}lder's inequality.
\end{proof}

\begin{lem} \label{lr-bound-3}
	Given an $m \times n$ matrix $\bA = (a_{ij})_{i \in \lb m \rb}^{j \in \lb n \rb}$ and a vector $\bx \in \mbb{R}^n$, the following holds:
	\begin{equation*} 
	\|\bA \bx\|_r \le \|\bx\|_r \|\bv\|_s,
	\end{equation*}
	where $r, s \ge 1$, and $1/r+1/s=1$; $\bv = (v_1, \ldots, v_n)$, with $v_j = \sum_{i=1}^m |a_{ij}|$.	
\end{lem}

\begin{proof}
	Suppose $\bA = (a_{ij})_{i \in \lb m \rb}^{j \in \lb n \rb}$. Then,
	\begin{equation*}
	\begin{aligned}
	\|\bA \bx\|_r & = \left\|
	\begin{bmatrix}
	a_{11}x_1 + \ldots + a_{1n}x_n \\
	\vdots \\
	a_{m1}x_1 + \ldots + a_{mn}x_n 
	\end{bmatrix}
	\right\|_r \\
	& = \Big(\sum_{i=1}^m |a_{i1}x_1 + \ldots + a_{in}x_n|^r \Big)^{1/r}\\
	& \le \bigg(\Big(\sum_{i=1}^m |a_{i1}x_1 + \ldots + a_{in}x_n|\Big)^r\bigg)^{1/r} \\
	& = \sum_{i=1}^m |a_{i1}x_1 + \ldots + a_{in}x_n| \\
	& \le |x_1|\sum_{i=1}^m |a_{i1}| + \ldots + |x_n|\sum_{i=1}^m |a_{in}| \\
	& = |\bx|'\bv \\
	& \le \|\bx\|_r \|\bv\|_s,
	\end{aligned}
	\end{equation*}
	where $r, s\ge 1$, $1/r+1/s=1$, $|\bx| = (|x_1|, \ldots, |x_n|)$, and $\bv = (v_1, \ldots, v_n)$, with $v_j = \sum_{i=1}^m |a_{ij}|$. 
	The last step uses H\"{o}lder's inequality and the fact that $\|\bx\| = \||\bx|\|$. 
\end{proof}

Note that for any $s \ge 1$, we know $\|\ba_i\|_s \le \|\ba_i\|_1$, implying that $\sum_{i=1}^m \|\ba_i\|_s^r \le \sum_{i=1}^m\|\ba_i\|_1^r$. Therefore, Lemma~\ref{lr-bound-2} provides a tighter bound than Lemma~\ref{lr-bound-1}.
The vector $\bv$ in the statement of Lemma~\ref{lr-bound-3} can be written as $\bv = \sum_{i=1}^m |\ba_i|$, where the $|\cdot|$ is applied element-wise to $\ba_i$. We thus have,
$$\|\bv\|_s = \Big\|\sum_{i=1}^m |\ba_i|\Big\|_s \le \sum_{i=1}^m \|\ba_i\|_s.$$
It is clear that when $r=1$, Lemma~\ref{lr-bound-3} gives a tighter bound than Lemma~\ref{lr-bound-2}. However, when $r \ge s$, we claim that Lemma~\ref{lr-bound-2} yields a better bound. To see this, notice that
\begin{equation*}
\begin{aligned}
\|\bv\|_s^r & = \left(\sum_{j=1}^n v_j^s\right)^{r/s} \\
& =  \Big(\sum_{j=1}^n \big(|a_{1j}|+\ldots+|a_{mj}|\big)^s \Big)^{r/s} \\
& \ge \Big(\sum_{j=1}^n \big(|a_{1j}|^s+ \ldots+|a_{mj}|^s \big)\Big)^{r/s} \\
& = \Big(\sum_{i=1}^m \big(|a_{i1}|^s+ \ldots+|a_{in}|^s \big) \Big)^{r/s} \\
& \ge \sum_{i=1}^m \Big(|a_{i1}|^s+ \ldots+|a_{in}|^s\Big)^{r/s} \\
& = \sum_{i=1}^m \|\ba_i\|_s^r,
\end{aligned}	
\end{equation*}
where in the derivation we have used Lemma~\ref{sum_pow}.

\begin{lem} \label{sum_pow}
	For any $k \ge 1$ and $c_i \ge 0$, the following holds:
	$$\left(\sum_{i=1}^m c_i\right)^k \ge \sum_{i=1}^m c_i^k.$$
\end{lem}

\begin{proof}
	If all $c_i=0$, the result obviously holds. Without loss of generality, we assume not all $c_i$ are equal to zero. Let $\lambda = \sum_{i=1}^m c_i$; then $\lambda > 0$. Set $b_i = c_i/\lambda$; then $b_i \in [0, 1]$, and $b_i^k \le b_i$. Together with the fact that $\sum_i b_i = 1$, we have:
	\begin{equation*}
	\begin{aligned}
	\sum_{i=1}^m c_i^k & = \lambda^k \sum_{i=1}^m b_i^k \\
	& \le \lambda^k \sum_{i=1}^m b_i \\
	& = \lambda^k \\
	& = \Bigl(\sum_{i=1}^m c_i\Bigr)^k,
	\end{aligned}
	\end{equation*}
	for any $k \ge 1, c_i \ge 0$.
\end{proof}

We now proceed to obtain a tractable relaxation to formulation (\ref{dro-mlr}). Using Lemma~\ref{lr-bound-2} and Theorem~\ref{Lip-wass}, we have:
\begin{equation*}  
\begin{aligned}
& \qquad \Bigl|\mbb{E}^{\mbb{Q}}\big[ h_{\tilde{\bB}}(\bz)\big] -  \mbb{E}^{\hat{\mbb{P}}_N}\big[ h_{\tilde{\bB}}(\bz)\big]\Bigr| \\
% \le \int\nolimits_{\scrZ \times \scrZ} \bigl|h_{\tilde{\bB}}(\bz_1)-h_{\tilde{\bB}}(\bz_2)\bigr| \mathrm{d}\pi_0(\bz_1, \bz_2) \\
& \le \int\nolimits_{\scrZ \times \scrZ} \frac{\big|h_{\tilde{\bB}}(\bz_1)-h_{\tilde{\bB}}(\bz_2)\big|}{\|\bz_1-\bz_2\|_r} \|\bz_1-\bz_2\|_r \mathrm{d}\pi_0(\bz_1, \bz_2) \\
& \le  \int\nolimits_{\scrZ \times \scrZ} \frac{L \|\tilde{\bB}(\bz_1 - \bz_2)\|_r}{\|\bz_1-\bz_2\|_r} \|\bz_1-\bz_2\|_r \mathrm{d}\pi_0(\bz_1, \bz_2)\\
& \le L \Big(\sum_{i=1}^K \|\bb_i\|_s^r\Big)^{1/r} \int\nolimits_{\scrZ \times \scrZ} \|\bz_1-\bz_2\|_r \mathrm{d}\pi_0(\bz_1, \bz_2) \\ 
& =  L \Big(\sum_{i=1}^K \|\bb_i\|_s^r\Big)^{1/r} W_{s,1} (\mbb{Q}, \ \hat{\mbb{P}}_N) \\
& \le \epsilon L \Big(\sum_{i=1}^K \|\bb_i\|_s^r\Big)^{1/r}, \qquad \forall \mbb{Q} \in \Omega,
\end{aligned}
\end{equation*}
where $\bb_i = (-B_{1i}, \ldots, -B_{pi}, \mathbf{e}_i)$ is the $i$-th row of $\tilde{\bB}$, with $\mathbf{e}_i$ the $i$-th unit vector in $\mbb{R}^K$.
The above derivation implies that when the Wasserstein metric is induced by $\|\cdot\|_r$,
$$\sup\limits_{\mbb{Q} \in \Omega}
\mbb{E}^{\mbb{Q}} [h_{\tilde{\bB}}(\bz)] \le \mbb{E}^{\hat{\mbb{P}}_N}\big[ h_{\tilde{\bB}}(\bz)\big] + \epsilon L \Big(\sum_{i=1}^K \|\bb_i\|_s^r\Big)^{1/r},$$
where $1/r+1/s=1$. This directly yields the following relaxation to (\ref{dro-mlr}):
\begin{equation} \label{relax1}
\inf_{\bB} \frac{1}{N}\sum_{i=1}^N h_{\bB}(\bx_i, \by_i) + \epsilon L \Big(\sum_{i=1}^K \|\bb_i\|_s^r\Big)^{1/r},
\end{equation}
which we call the MLR-SR relaxation. The regularization term in (\ref{relax1}) penalizes the aggregate of the dual norm of the regression coefficients corresponding to each of the $K$ responses. Notice that when $r \neq 1$, (\ref{relax1}) cannot be decomposed into $K$ independent terms. When $s=r=2$, the regularizer is just the Frobenius norm of $\tilde{\bB}$. Using a similar derivation, Lemma~\ref{lr-bound-3} yields the following relaxation to (\ref{dro-mlr}):
\begin{equation} \label{relax2}
\inf_{\bB} \frac{1}{N}\sum_{i=1}^N h_{\bB}(\bx_i, \by_i) + \epsilon L \|\bv\|_s,
\end{equation}
where $\bv \triangleq (v_1, \ldots, v_p, 1, \ldots, 1)$, with $v_i = \sum_{j=1}^K |B_{ij}|$, i.e., $v_i$ is a condensed representation of the coefficients for predictor $i$ through summing over the $K$ coordinates. We call (\ref{relax2}) the MLR-1S relaxation (the naming convention will be more clear after introducing the $L_{r,s}$ matrix norm in Section~\ref{matrixnorm}). When $s \neq 1$, it cannot be decomposed into $K$ subproblems due to the entangling of coefficients in the regularization term. 

Note that when $K=1$, with an 1-Lipschitz continuous loss function, the two regularizers in MLR-SR and MLR-1S reduce to $\epsilon \|(-\bbeta, 1)\|_s$, which coincides with the Wasserstein DRO formulation derived in Section~\ref{chapt:dro} with an absolute error loss.
In both relaxations for MLR, the Wasserstein ball radius $\epsilon$ and the Lipschitz constant $L$ determine the strength of the penalty term. Recall that we assume the loss function is Lipschitz continuous on the same norm space with the one used by the Wasserstein metric. This assumption can be relaxed by allowing a different norm space for the Lipschitz continuous loss function, and the derivation technique can be easily adapted to obtain relaxations to (\ref{dro-mlr}). On the other hand, however, the norm space used by the Wasserstein metric can provide implications on what loss function to choose. For example, if we restrict the class of loss functions $l(\cdot)$ to the norms, our assumption suggests that $l(\bz) = \|\bz\|_r$, which is a reasonable choice since it reflects the distance metric on the data space.

\subsection{A New Perspective on the Formulation} \label{matrixnorm}
In this subsection we will present a matrix norm interpretation for the two relaxations (\ref{relax1}) and (\ref{relax2}). Different from the commonly used matrix norm definitions in the literature, e.g., the vector norm-induced matrix norm $\|\bA\| \triangleq \max_{\|\bx\| \le 1} \|\bA \bx\|$, the entrywise norm that treats the matrix as a vector, and the Schatten–von-Neumann norm that defines the norm on the vector of singular values \citep{tomioka2013convex}, we adopt the $L_{r,s}$ norm, which summarizes each column by its $\ell_r$ norm, and then computes the $\ell_s$ norm of the aggregate vector. The formal definition is described as follows.

\begin{defi}[$L_{r,s}$ Matrix Norm]
	For any $m \times n$ matrix $\bA = (a_{ij})_{i \in \lb m \rb}^{j \in \lb n \rb}$, define its $L_{r,s}$ norm as:
	\begin{equation*}
	\|\bA\|_{r,s} \triangleq \Bigg(\sum_{j=1}^n\bigg(\sum_{i=1}^m |a_{ij}|^r\bigg)^{s/r}\Bigg)^{1/s},
	\end{equation*}
	where $r, s \ge 1$. 
\end{defi}
Note that $\|\bA\|_{r,s}$ can be viewed as the $\ell_s$ norm of a newly defined vector $\bv = (v_1, \ldots, v_n)$, where $v_j = \|\bA_j\|_r$, with $\bA_j$ the $j$-th column of $\bA$. When $r=s=2$, the $L_{r,s}$ norm is the Frobenius norm. Moreover, $\|\bA\|_{r,s}$ is a convex function in $\bA$, which can be shown as follows.
\begin{proof}
	For two matrices $\bA = [\bA_1, \cdots, \bA_n]$, $\bB = [\bB_1, \cdots, \bB_n]$, where $\bA_i, \bB_i$ are the columns of $\bA$ and $\bB$, respectively, consider their convex combination $\lambda \bA + (1-\lambda)\bB$, where $\lambda \in [0, 1]$. Its $L_{r,s}$ norm can be expressed as:
	\begin{equation*}
	\begin{aligned}
	& \quad \ \|\lambda \bA + (1-\lambda)\bB\|_{r,s} \\
	& = \Big\|\Big(\|\lambda\bA_1+(1-\lambda)\bB_1\|_r, \ldots, \|\lambda\bA_n+(1-\lambda)\bB_n\|_r\Big)\Big\|_s \\
	& \le \Big\|\Big(\lambda\|\bA_1\|_r+(1-\lambda)\|\bB_1\|_r, \ldots, \lambda\|\bA_n\|_r+(1-\lambda)\|\bB_n\|_r\Big)\Big\|_s \\
	& = \Big\|\lambda\big(\|\bA_1\|_r, \ldots, \|\bA_n\|_r\big)+(1-\lambda)\big(\|\bB_1\|_r, \ldots, \|\bB_n\|_r\big) \Big\|_s \\
	& \le \lambda \Big \|\big(\|\bA_1\|_r, \ldots, \|\bA_n\|_r\big)\Big\|_s + (1-\lambda) \Big \|\big(\|\bB_1\|_r, \ldots, \|\bB_n\|_r\big)\Big\|_s \\
	& = \lambda \|\bA\|_{r,s} + (1-\lambda)\|\bB\|_{r,s}.
	\end{aligned}
	\end{equation*}
	Therefore, the $L_{r,s}$ norm is convex.
\end{proof}

The $L_{r,s}$ matrix norm depends on the structure of the matrix, and transposing a matrix changes its norm. For example, given $\bA \in \mbb{R}^{n \times 1}$, $\|\bA\|_{r,s} = \|\ba\|_r$, $\|\bA'\|_{r,s} = \|\ba\|_s$, where $\ba$ represents the vectorization of $\bA$. To show the validity of the $L_{r,s}$ norm, we need to verify the following properties:
\begin{enumerate}
	\item $\|\bA\|_{r,s} \ge 0$.
	\item $\|\bA\|_{r,s} = 0$ if and only if $\bA = 0$. 
	\item $\|\alpha \bA\|_{r,s} = |\alpha|\|\bA\|_{r,s}$.
	\item $\|\bA + \bB\|_{r,s} \le \|\bA\|_{r,s} + \|\bB\|_{r,s}$. 
\end{enumerate}

The first three properties are straightforward. To show the sub-additivity property (triangle inequality), assume $\bA = [\bA_1, \cdots, \bA_n]$ and $\bB = [\bB_1, \cdots, \bB_n]$, where $\bA_j, \bB_j, j \in \lb n \rb$, are the columns of $\bA$ and $\bB$, respectively. Define two vectors $\bv\triangleq(\|\bA_1\|_r, \ldots, \|\bA_n\|_r)$, and $\bt \triangleq (\|\bB_1\|_r, \ldots, \|\bB_n\|_r)$, we have:
\begin{equation*}
\begin{aligned}
\|\bA\|_{r,s} + \|\bB\|_{r,s} & = \|\bv\|_s + \|\bt\|_s \\
& \ge \|\bv+\bt\|_s \\	
% & = \Big(\big(\|\bA_1\|_r+\|\bB_1\|_r\big)^s + \ldots + \big(\|\bA_n\|_r+\|\bB_n\|_r\big)^s\Big)^{1/s} \\
& = \Big( \sum_{i=1}^n \big(\|\bA_i\|_r+\|\bB_i\|_r\big)^s \Big)^{1/s} \\
& \ge \Big( \sum_{i=1}^n \|\bA_i+\bB_i\|_r^s \Big)^{1/s} \\
& = \|\bA + \bB\|_{r,s}.
\end{aligned}
\end{equation*}
The $L_{r,s}$ norm also satisfies the following {\em sub-multiplicative} property:
\begin{equation} \label{cmn}
\|\bA \bB\|_{r,s}  \le \|\bA\|_{1, u} \|\bB\|_{t, s},
\end{equation}
for $\bA \in \mbb{R}^{m \times n}, \bB \in \mbb{R}^{n \times K}$, and any $t, u \ge 1$ satisfying $1/t+1/u=1$.
\begin{proof}
	Assume $\bA \in \mbb{R}^{m \times n}, \ \bB \in \mbb{R}^{n \times K}$, and  
	$\bB = [\bB_1, \cdots, \bB_K]$, where $\bB_j, j \in \lb K \rb$, are the columns of $\bB$. Then, 
	$\bA \bB = [\bA \bB_1, \cdots, \bA \bB_K]$, and $\|\bA \bB\|_{r,s} = \|\bw\|_s$, where $\bw = (w_1, \ldots, w_K)$ with $w_j = \|\bA \bB_j\|_r$.
	From the proof of Lemma~\ref{lr-bound-3}, we immediately have:
	$$w_j = \|\bA \bB_j\|_r \le \|\bB_j\|_t \|\bA\|_{1, u},$$
	where $1/t+1/u=1$. We thus have,
	\begin{equation*}
	\begin{aligned}
	\|\bA \bB\|_{r,s} & = \Big(\sum_{j=1}^K w_j^s\Big)^{1/s} \\
	& \le \Big(\sum_{j=1}^K \|\bB_j\|_t^s \|\bA\|_{1, u}^s\Big)^{1/s} \\
	& = \|\bA\|_{1, u} \Big(\sum_{j=1}^K \|\bB_j\|_t^s \Big)^{1/s} \\
	& = \|\bA\|_{1, u} \|\bB\|_{t, s},
	\end{aligned}
	\end{equation*}
	for any $t, u \ge 1$ satisfying $1/t+1/u=1$.
\end{proof}

Next we will reformulate the two relaxations (\ref{relax1}) and (\ref{relax2}) using the $L_{r,s}$ norm. When the Wasserstein metric is defined by $\|\cdot\|_r$, the MLR-SR relaxation can be written as:
\begin{equation*} 
\inf_{\bB} \frac{1}{N}\sum_{i=1}^N h_{\bB}(\bx_i, \by_i) + \epsilon L \|\tilde{\bB}'\|_{s, r}.
\end{equation*}
Similarly, the MLR-1S relaxation can be written as:
\begin{equation*} 
\inf_{\bB} \frac{1}{N}\sum_{i=1}^N h_{\bB}(\bx_i, \by_i) + \epsilon L \|\tilde{\bB}\|_{1, s},
\end{equation*}
where $r, s \ge 1$ and $1/r+1/s=1$. When the loss function is convex, e.g., $h_{\bB}(\bx, \by) = \|\by-\bB'\bx\|$, it is obvious that both MLR-SR and MLR-1S are convex optimization problems. By using the $L_{r,s}$ matrix norm, we are able to express the two relaxations in a compact way, which reflects the role of the norm space induced by the Wasserstein metric on the regularizer, and demonstrates the impact of the size of the Wasserstein ambiguity set and the Lipschitz continuity of the loss function on the regularization strength.

\subsection{Distributionally Robust Multiclass Logistic Regression} \label{sec:dro-mlg}
In this subsection we apply the Wasserstein DRO framework to the problem of {\em Multiclass Logistic Regression (MLG)}. Suppose there are $K$ classes, and we are given a predictor vector $\bx \in \mbb{R}^p$. Our goal is to predict its class label, denoted by a $K$-dimensional binary label vector $\by \in \{0, 1\}^K$, where $\sum_k y_k = 1$, and $y_k=1$ if and only if $\bx$ belongs to class $k$. The conditional distribution of $\by$ given $\bx$ is modeled as
$$p(\by|\bx) = \prod_{i=1}^K p_i^{y_i},$$
where $p_i = e^{\bw_i'\bx}/\sum_{k=1}^K e^{\bw_k'\bx}$, and $\bw_i, i \in \lb K \rb$, are the coefficient vectors to be estimated that account for the contribution of $\bx$ in predicting the class labels. The log-likelihood can be expressed as:
\begin{equation*}
\begin{aligned}
\log p(\by|\bx) & = \sum_{i=1}^K y_i \log(p_i) \\
& = \sum_{i=1}^K y_i \log \frac{e^{\bw_i'\bx}}{\sum_{k=1}^K e^{\bw_k'\bx}} \\
& = \sum_{i=1}^K y_i \bw_i'\bx - \Bigl(\log \sum_{k=1}^K e^{\bw_k'\bx} \Bigr) \sum_{i=1}^K y_i\\
& = \sum_{i=1}^K y_i \bw_i'\bx - \log \sum_{k=1}^K e^{\bw_k'\bx} \\
& = \by'\bB'\bx - \log \mathbf{1}'e^{\bB'\bx},
\end{aligned}
\end{equation*} 
where $\bB \triangleq [\bw_1, \cdots, \bw_K]$, $\mathbf{1}$ is the vector of ones, and the exponential operator is applied element-wise to the exponent vector. The log-loss is defined to be the negative log-likelihood, i.e., $h_{\bB}(\bx, \by) \triangleq \log \mathbf{1}'e^{\bB'\bx} - \by'\bB'\bx$. The Wasserstein DRO formulation for MLG minimizes the following worst-case expected loss: 
\begin{equation} \label{dro-mlg} 
\inf\limits_{\bB}\sup\limits_{\mbb{Q} \in \Omega}
\mbb{E}^{\mbb{Q}} \Big[\log \mathbf{1}'e^{\bB'\bx} - \by'\bB'\bx \Big],
\end{equation}
where $\Omega$ is defined using the order-1 Wasserstein metric induced by: 
$$s(\bz_1, \bz_2) = \|\bx_1 - \bx_2\|_r + Ms_{\by}(\by_1, \by_2),$$
where $\bz_1 = (\bx_1, \by_1), \ \bz_2 = (\bx_2, \by_2)$, $s_{\by}(\cdot, \cdot)$ could be any metric, and M is a very large positive constant. To make (\ref{dro-mlg}) tractable, we need to derive an upper bound for the growth rate of the loss function, which involves bounding the following difference
\begin{equation} \label{log-loss-diff}
\begin{aligned}
\qquad & \ |h_{\bB}(\bx_1, \by_1) - h_{\bB}(\bx_2, \by_2)| \\
= & \ |\log \mathbf{1}'e^{\bB'\bx_1} - \by_1'\bB'\bx_1 - \log \mathbf{1}'e^{\bB'\bx_2} + \by_2'\bB'\bx_2| \\
\le & \ |\log \mathbf{1}'e^{\bB'\bx_1} - \log \mathbf{1}'e^{\bB'\bx_2}| + |\by_1'\bB'\bx_1 - \by_2'\bB'\bx_2|,
\end{aligned}
\end{equation}
in terms of $s(\bz_1, \bz_2)$. Let us examine the two terms in (\ref{log-loss-diff}) separately. For the first term, define a function $g(\ba) = \log \mathbf{1}'e^{\ba}$, where $\ba \in \mbb{R}^K$. Using the mean value theorem, we know for any $\ba, \bb \in \mbb{R}^K$, there exists some $t \in (0, 1)$ such that
\begin{equation} \label{mvt}
|g(\mathbf{b}) - g(\ba)| \le \bigl\|\nabla g\bigl((1-t)\ba + t\mathbf{b}\bigr)\bigr\|_s \|\mathbf{b} - \ba\|_r 
\le K^{1/s} \|\mathbf{b} - \ba\|_r,
\end{equation}
where $r, s \ge 1$, $1/r+1/s=1$, the first inequality is due to H\"{o}lder's inequality, and the second inequality is due to the fact that $\nabla g(\ba) = e^{\ba}/\mathbf{1}'e^{\ba}$, which implies that each element of $\nabla g(\ba)$ is smaller than 1. Based on (\ref{mvt}) we have:
\begin{equation*} 
|\log \mathbf{1}'e^{\bB'\bx_1} - \log \mathbf{1}'e^{\bB'\bx_2}| \le K^{1/s} \|\bB'(\bx_1-\bx_2)\|_r. 
\end{equation*}
We can use Lemma~\ref{lr-bound-2} or \ref{lr-bound-3} to bound $\|\bB'(\bx_1-\bx_2)\|_r$, which respectively leads to the following two results:
\begin{equation} \label{sr-1}
\begin{aligned}
|\log \mathbf{1}'e^{\bB'\bx_1} - \log \mathbf{1}'e^{\bB'\bx_2}| & \le K^{1/s} \|\bx_1-\bx_2\|_r \Bigl(\sum_{i=1}^K \|\bw_i\|_s^r\Bigr)^{1/r} \\
& = K^{1/s} \|\bx_1-\bx_2\|_r \|\bB\|_{s, r},
\end{aligned}
\end{equation}
and 
\begin{equation} \label{1s-1}
|\log \mathbf{1}'e^{\bB'\bx_1} - \log \mathbf{1}'e^{\bB'\bx_2}| \le K^{1/s} \|\bx_1-\bx_2\|_r \|\bB'\|_{1,s},
\end{equation}
where $r, s \ge 1$ and $1/r+1/s=1$. By noting that $\|\bx_1-\bx_2\|_r \le s(\bz_1, \bz_2)$, we obtain the upper bound for the first term in (\ref{log-loss-diff}) in terms of $s(\bz_1, \bz_2)$. For the second term, we have,
\begin{equation} \label{sr-1s-2}
\begin{aligned}
|\by_1'\bB'\bx_1 - \by_2'\bB'\bx_2| 
& = \Big |\sum_{i=1}^K  \bw_i' (y_{1i}\bx_1 -  y_{2i} \bx_2)\Big|  \\
& \le \sum_{i=1}^K |\bw_i'(y_{1i}\bx_1 - y_{2i}\bx_2)| \\
& \le \sum_{i=1}^K \|\bw_i\|_s \|y_{1i}\bx_1 - y_{2i}\bx_2\|_r \\
& \le s(\bz_1, \bz_2)\sum_{i=1}^K \|\bw_i\|_s \\
& = s(\bz_1, \bz_2) \|\bB\|_{s, 1},
\end{aligned}
\end{equation}
where $\by_1 = (y_{11}, \ldots, y_{1K}), \by_2 = (y_{21}, \ldots, y_{2K})$, $1/s+1/r=1$, the second inequality uses the H\"{o}lder's inequality, and the last inequality can be proved by noting that if $y_{1i} = y_{2i}$, $\|y_{1i}\bx_1 - y_{2i}\bx_2\|_r \le s(\bz_1, \bz_2)$; otherwise $s(\bz_1, \bz_2)$ goes to infinity. Suppose $\pi_0$ is the optimal transportation plan that moves the probability mass from $\mbb{Q}$ to $\hat{\mbb{P}}_N$, combining (\ref{sr-1}) with (\ref{sr-1s-2}), we have:
\begin{equation*}  
\begin{aligned}
& \quad \ \ \Bigl|\mbb{E}^{\mbb{Q}}\big[ h_{\bB}(\bx, \by)\big] -  \mbb{E}^{\hat{\mbb{P}}_N}\big[ h_{\bB}(\bx, \by)\big]\Bigr| \\ 
& \le \int\nolimits_{\scrZ \times \scrZ} \bigl|h_{\bB}(\bx_1, 
\by_1)-h_{\bB}(\bx_2, \by_2)\bigr| \mathrm{d}\pi_0(\bz_1, \bz_2) \\
& = \int\nolimits_{\scrZ \times \scrZ} \frac{\big|h_{\bB}(\bx_1, 
	\by_1)-h_{\bB}(\bx_2, \by_2)\big|}{s(\bz_1, \bz_2)} s(\bz_1, \bz_2) \mathrm{d}\pi_0(\bz_1, \bz_2) \\
& \le \int\nolimits_{\scrZ \times \scrZ} \Big(K^{1/s}  \|\bB\|_{s, r}+ \|\bB\|_{s, 1}\Big) s(\bz_1, \bz_2) \mathrm{d}\pi_0(\bz_1, \bz_2) \\
& \le \epsilon \Big(K^{1/s}  \|\bB\|_{s, r}+ \|\bB\|_{s, 1}\Big),
\end{aligned}
\end{equation*}
which yields the following MLG-SR relaxation to (\ref{dro-mlg}):
\begin{equation*} 
\inf_{\bB} \frac{1}{N} \sum_{i=1}^N \Bigl(\log \mathbf{1}'e^{\bB'\bx_i} - \by_i'\bB'\bx_i \Bigr) + \epsilon \Big(K^{1/s}  \|\bB\|_{s, r}+ \|\bB\|_{s, 1}\Big).
\end{equation*}
Similarly, combining (\ref{1s-1}) with (\ref{sr-1s-2}) produces the following MLG-1S relaxation:
\begin{equation*} 
\inf_{\bB} \frac{1}{N} \sum_{i=1}^N \Bigl(\log \mathbf{1}'e^{\bB'\bx_i} - \by_i'\bB'\bx_i \Bigr) + \epsilon \Big(K^{1/s}  \|\bB'\|_{1,s}+ \|\bB\|_{s, 1}\Big).
\end{equation*}

We note that both MLG-SR and MLG-1S are convex optimization problems. The convexity of the regularizer has been shown in Section~\ref{matrixnorm}. The convexity of the log-loss is shown in the following theorem.
\begin{thm}
	The log-loss $h_{\bB}(\bx, \by) \triangleq \log \mathbf{1}'e^{\bB'\bx} - \by'\bB'\bx$ is convex in $\bB$.
\end{thm} 
\begin{proof}
	Since the linear function is convex, we only need to show the convexity of $\log \mathbf{1}'e^{\bB'\bx}$. The following result will be used.
	\begin{col} \label{log-of-sum-exp}
		The function $f(\bx) = \log (\sum_{i=1}^n e^{x_i})$ is a convex function of $\bx \in \mbb{R}^n$.
	\end{col}
	By Corollary~\ref{log-of-sum-exp}, we have for any $\lambda \in [0,1]$, and any two matrices $\bB = [\bB_1, \cdots, \bB_K]$ and $\bC = [\mathbf{C}_1, \cdots, \bC_K]$,
	\begin{equation*}
	\begin{aligned}
	\log \mathbf{1}'e^{(\lambda\bB+(1-\lambda) \bC)'\bx} & = \log \Big(\sum_{i=1}^K e^{\lambda \bB_i'\bx + (1-\lambda)\bC_i'\bx} \Big)\\
	& = f(\lambda \bv_1 + (1-\lambda)\bv_2) \\
	& \le \lambda f(\bv_1) + (1-\lambda)f(\bv_2) \\
	& = \lambda \log \Big(\sum_{i=1}^K e^{\bB_i'\bx} \Big) + (1-\lambda) \log \Big(\sum_{i=1}^K e^{\bC_i'\bx} \Big) \\
	& = \lambda \log \mathbf{1}'e^{\bB'\bx} + (1-\lambda) \log \mathbf{1}'e^{\bC'\bx},
	\end{aligned}
	\end{equation*}
	where $\bv_1 = (\bB_1'\bx, \ldots, \bB_K'\bx)$, and $\bv_2 = (\bC_1'\bx,
	\ldots, \bC_K'\bx)$. Therefore the log-loss is convex. 
\end{proof}

When $K=2$, by taking one of the two classes as a reference, we can set one column of $\bB$ to zero, in which case all three regularizers $\|\bB\|_{s, r}$, $\|\bB\|_{s, 1}$ and $\|\bB'\|_{1,s}$ reduce to $\|\bbeta\|_s$, where $\bB \triangleq [\bbeta, \mathbf{0}]$, and the MLG-SR and MLG-1S relaxations coincide with the regularized logistic regression formulation derived in (\ref{convex-lg}).

We also note that the number of classes $K$, along with the Wasserstein set radius $\epsilon$, determines the regularization magnitude in the two MLG relaxations. There are two terms in the regularizer, one accounting for the predictor/feature uncertainty, and the other accounting for the label uncertainty. In the MLG-SR regularizer, we summarize each column of $\bB$ by its dual norm, and aggregate them by the $\ell_r$ and $\ell_1$ norms to reflect the predictor and label uncertainties, respectively.

\section{The Out-of-Sample Performance Guarantees} \label{sec:5-3}
In this section we will show the out-of-sample performance guarantees for the solutions to the MLR and MLG relaxations, i.e., given a new test sample, what is the expected prediction bias/log-loss. The results are established using the Rademacher complexity \citep{Peter02}, following the line of proof presented in Section~\ref{out}. The resulting bounds shed light on the role of the regularizer in inducing a low prediction error.

\subsection{Performance Guarantees for MLR Relaxations}
\label{mlr-perf}
In this subsection we study the out-of-sample predictive performance of the solutions to (\ref{relax1}) and (\ref{relax2}). Suppose the data $(\bx, \by)$ is drawn from the probability measure $\mbb{P}^*$.
We first make the following assumptions that are essential for deriving the bounds.

\begin{ass} \label{a1-mlr} The $\ell_r$-norm of the data $(\bx, \by)$ is
	bounded above a.s. under the probability measure $\mbb{P}^*$, i.e., $\|(\bx, \by)\|_r \le R, \ \text{a.s.}$.
\end{ass}

\begin{ass} \label{a2-mlr} For any feasible solution to MLR-SR, it holds that
	$\|\tilde{\bB}'\|_{s, r} \le \bar{B}_{s,r}$.
\end{ass}

\begin{ass} \label{a3-mlr} For any feasible solution to MLR-1S, it holds that
	$\|\tilde{\bB}\|_{1, s} \le \bar{B}_{1,s}$.
\end{ass}

\begin{ass} \label{a4-mlr} The loss resulting from $(\bx, \by) = (\mathbf{0}, \mathbf{0})$ is 0, i.e.,
	$h_{\tilde{\bB}}(\mathbf{0})=0$.
\end{ass}

Note that Assumption~\ref{a1-mlr} bounds the magnitude of the data in terms of its $\ell_r$-norm, and $R$ can be assumed to be reasonably small with standardized data input. Assumptions~\ref{a2-mlr} and \ref{a3-mlr} impose restrictions on the norm of the coefficient matrix, which are a result of adding appropriate regularizers into the formulation as in (\ref{relax1}) and (\ref{relax2}). Assumption~\ref{a4-mlr} easily holds when the loss function is defined via some norm, i.e., $h_{\bB}(\bx, \by) \triangleq \|\by-\bB'\bx\|$. Under Assumptions~\ref{a1-mlr}, \ref{a2-mlr} and \ref{a4-mlr}, using Lemma~\ref{lr-bound-2} and the Lipschitz continuity of the loss function, we have 
\begin{equation} \label{mlr-bound-1}
h_{\tilde{\bB}}(\bz)  \le L\|\tilde{\bB} \bz\|_r 
\le L \|\bz\|_r \|\tilde{\bB}'\|_{s, r} 
\le LR\bar{B}_{s,r}. 
\end{equation}
Similarly, under Assumptions~\ref{a1-mlr}, \ref{a3-mlr} and \ref{a4-mlr}, Lemma~\ref{lr-bound-3} yields the following:
\begin{equation} \label{mlr-bound-2}
h_{\tilde{\bB}}(\bz)  \le L\|\tilde{\bB} \bz\|_r 
\le L \|\bz\|_r \|\tilde{\bB}\|_{1, s} 
\le LR\bar{B}_{1,s}.
\end{equation}

With the above results, the idea is to bound the out-of-sample prediction error using the empirical {\em Rademacher complexity} $\scrR_N(\cdot)$ of
the class of loss functions: 
$\scrH=\{\bz \ra h_{\tilde{\bB}}(\bz)\}$, denoted by $\scrR_N(\scrH)$. Using Lemma~\ref{radcom} and the upper bounds in (\ref{mlr-bound-1}) and (\ref{mlr-bound-2}), we arrive at the following result.

\begin{lem} \label{radcom-mlr}
	Under Assumptions~\ref{a1-mlr}, \ref{a2-mlr} and \ref{a4-mlr},
	\begin{equation*}
	\scrR_N(\scrH)\le \frac{2LR\bar{B}_{s,r}}{\sqrt{N}}.
	\end{equation*}
	Under Assumptions~\ref{a1-mlr}, \ref{a3-mlr} and \ref{a4-mlr},
	\begin{equation*}
	\scrR_N(\scrH)\le \frac{2LR\bar{B}_{1,s}}{\sqrt{N}}.
	\end{equation*}
\end{lem}
Using the Rademacher complexity of the class of loss functions, the out-of-sample prediction bias of the solutions to (\ref{relax1}) and (\ref{relax2}) can be bounded by applying Theorem 8 in \cite{Peter02}.

\begin{thm} \label{t1-mlr} Suppose the solution to (\ref{relax1}) is $\hat{\bB}_{s,r}$. Under Assumptions~\ref{a1-mlr}, \ref{a2-mlr} and \ref{a4-mlr}, for any
	$0<\delta<1$, with probability at least $1-\delta$ with respect to the
	sampling,
	\begin{equation*} 
	\mathbb{E}[h_{\hat{\bB}_{s,r}}(\bx, \by)]\le
	\frac{1}{N}\sum_{i=1}^N
	h_{\hat{\bB}_{s,r}}(\bx_i, \by_i)  +\frac{2LR\bar{B}_{s,r}}{\sqrt{N}}
	+ 
	LR\bar{B}_{s,r}\sqrt{\frac{8\log(\frac{2}{\delta})}{N}}\ ,
	\end{equation*}
	and for any $\zeta>\frac{2LR\bar{B}_{s,r}}{\sqrt{N}}+
	LR\bar{B}_{s,r}\sqrt{\frac{8\log(2/\delta)}{N}}$,
	\begin{equation*} 
	\begin{aligned}
	\mathbb{P}\Bigl(h_{\hat{\bB}_{s,r}}(\bx,\by)  & \ge
	\frac{1}{N}\sum_{i=1}^N h_{\hat{\bB}_{s,r}}(\bx_i,\by_i)+\zeta\Bigr) \\
	& \le \frac{\frac{1}{N}\sum_{i=1}^N
		h_{\hat{\bB}_{s,r}}(\bx_i,\by_i)+\frac{2LR\bar{B}_{s,r}}{\sqrt{N}}+ 
		LR\bar{B}_{s,r}\sqrt{\frac{8\log(2/\delta)}{N}}}{\frac{1}{N}\sum_{i=1}^N
		h_{\hat{\bB}_{s,r}}(\bx_i,\by_i)+\zeta}.
	\end{aligned}	
	\end{equation*}
\end{thm}

\begin{thm} \label{t2-mlr} Suppose the solution to (\ref{relax2}) is $\hat{\bB}_{1,s}$. Under Assumptions~\ref{a1-mlr}, \ref{a3-mlr} and \ref{a4-mlr}, for any
	$0<\delta<1$, with probability at least $1-\delta$ with respect to the
	sampling,
	\begin{equation*} 
	\mathbb{E}[h_{\hat{\bB}_{1,s}}(\bx, \by)]\le
	\frac{1}{N}\sum_{i=1}^N
	h_{\hat{\bB}_{1,s}}(\bx_i, \by_i)  +\frac{2LR\bar{B}_{1,s}}{\sqrt{N}}
	+ 
	LR\bar{B}_{1,s}\sqrt{\frac{8\log(\frac{2}{\delta})}{N}}\ ,
	\end{equation*}
	and for any $\zeta>\frac{2LR\bar{B}_{1,s}}{\sqrt{N}}+
	LR\bar{B}_{1,s}\sqrt{\frac{8\log(2/\delta)}{N}}$,
	\begin{equation*} 
	\begin{aligned}
	\mathbb{P}\Bigl(h_{\hat{\bB}_{1,s}}(\bx,\by)   & \ge
	\frac{1}{N}\sum_{i=1}^N h_{\hat{\bB}_{1,s}}(\bx_i,\by_i)+\zeta\Bigr) \\
	& \le \frac{\frac{1}{N}\sum_{i=1}^N
		h_{\hat{\bB}_{1,s}}(\bx_i,\by_i)+\frac{2LR\bar{B}_{1,s}}{\sqrt{N}}+ 
		LR\bar{B}_{1,s}\sqrt{\frac{8\log(2/\delta)}{N}}}{\frac{1}{N}\sum_{i=1}^N
		h_{\hat{\bB}_{1,s}}(\bx_i,\by_i)+\zeta}.
	\end{aligned}
	\end{equation*}
\end{thm}

Theorems~\ref{t1-mlr} and \ref{t2-mlr} present bounds on the out-of-sample prediction errors of the solutions to (\ref{relax1}) and (\ref{relax2}), respectively. The expectations/probabilities are taken w.r.t. the new sample $(\bx, \by)$. The magnitude of the regularizer plays a role in controlling the bias, and a smaller upper bound on the matrix norm leads to a smaller prediction error, suggesting the superiority of MLR-SR for $r \ge 2$, and the superiority of MLR-1S for $r=1$ (see Section~\ref{sec:dro-mlr}). But on the other hand, the prediction error also depends on the sample average loss over the training set, for which there is no guarantee on which model wins out. In practice we suggest trying both models and selecting the one that yields a smaller error on a validation set.

\subsection{Performance Guarantees for MLG Relaxations}
\label{mlg-perf}
In this subsection, we study the out-of-sample log-loss of the solutions to MLG-SR and MLG-1S. Suppose the data $(\bx, \by)$ is drawn from the probability measure $\mbb{P}^*$. We first make several assumptions that are needed to establish the results.

\begin{ass} \label{b1} The $\ell_r$ norm of the predictor $\bx$ is bounded above a.s. under the probability measure $\mbb{P}^*_{\scrX}$, i.e., $\|\bx\|_r \le R_{\bx}, \ \text{a.s.}$.
\end{ass}

\begin{ass} \label{b2} For any feasible solution to MLG-SR, the following holds:
	$$K^{1/s}  \|\bB\|_{s, r}+ \|\bB\|_{s, 1} \le \bar{C}_{s,r}.$$
\end{ass}

\begin{ass} \label{b3} For any feasible solution to MLG-1S, the following holds: 
	$$K^{1/s}  \|\bB'\|_{1,s}+ \|\bB\|_{s, 1} \le \bar{C}_{1,s}.$$
\end{ass}

With standardized predictors, $R_{\bx}$ in Assumption~\ref{b1} can be assumed to be small. The form of the constraints in Assumptions~\ref{b2} and \ref{b3} is consistent with the form of the regularizers in MLG-SR and MLG-1S, respectively. We will see later that the bounds $\bar{C}_{s,r}$ and $\bar{C}_{1,s}$ respectively control the out-of-sample log-loss of the solutions to MLG-SR and MLG-1S, which validates the role of the regularizer in improving the out-of-sample performance. Under Assumptions~\ref{b1} and \ref{b2}, using (\ref{log-loss-diff}), (\ref{sr-1}) and (\ref{sr-1s-2}), we have,
\begin{equation*}
|h_{\bB}(\bx, \by) - h_{\bB}(\mathbf{0}, \by)| 
\le K^{1/s} \|\bx\|_r \|\bB\|_{s, r} + \|\bx\|_r \|\bB\|_{s, 1} 
\le R_{\bx} \bar{C}_{s,r}.
\end{equation*}
By noting that $h_{\bB}(\mathbf{0}, \by) = \log K$, we immediately have,
\begin{equation} \label{mlg-bound-1}
\log K - R_{\bx} \bar{C}_{s,r} \le h_{\bB}(\bx, \by) \le R_{\bx} \bar{C}_{s,r} + \log K.
\end{equation}
Similarly, under Assumptions~\ref{b1} and \ref{b3}, using (\ref{log-loss-diff}), (\ref{1s-1}) and (\ref{sr-1s-2}), we have,
\begin{equation*}
|h_{\bB}(\bx, \by) - h_{\bB}(\mathbf{0}, \by)| 
\le K^{1/s} \|\bx\|_r \|\bB'\|_{1,s} + \|\bx\|_r \|\bB\|_{s, 1} 
\le R_{\bx} \bar{C}_{1,s},
\end{equation*}
which implies that 
\begin{equation} \label{mlg-bound-2}
\log K - R_{\bx} \bar{C}_{1,s} \le h_{\bB}(\bx, \by) \le R_{\bx} \bar{C}_{1,s} + \log K.
\end{equation}

Using (\ref{mlg-bound-1}) and (\ref{mlg-bound-2}), we can now proceed to bound the out-of-sample log-loss using the empirical {\em Rademacher complexity} $\scrR_N(\cdot)$ of
the following class of loss functions: 
\begin{equation*}
\scrH=\{(\bx, \by) \ra h_{\bB}(\bx, \by): h_{\bB}(\bx, \by) = \log \mathbf{1}'e^{\bB'\bx} - \by'\bB'\bx \}.
\end{equation*}

\begin{lem} \label{radcom-lg}
	Under Assumptions~\ref{b1} and \ref{b2},
	\begin{equation*}
	\scrR_N(\scrH)\le \frac{2 (R_{\bx} \bar{C}_{s,r} + \log K)}{\sqrt{N}}.
	\end{equation*}
	Under Assumptions~\ref{b1} and \ref{b3},
	\begin{equation*}
	\scrR_N(\scrH)\le \frac{2(R_{\bx} \bar{C}_{1,s} + \log K)}{\sqrt{N}}.
	\end{equation*}
\end{lem}

Using Lemma~\ref{radcom-lg}, we are able to bound
the out-of-sample log-loss of the solutions to MLG-SR and MLG-1S by applying Theorem 8 in \cite{Peter02}.

\begin{thm} \label{c1} Suppose the solution to MLG-SR is $\hat{\bB}_{s,r}$. Under Assumptions~\ref{b1} and \ref{b2}, for any
	$0<\delta<1$, with probability at least $1-\delta$ with respect to the
	sampling,
	\begin{multline*} 
	\mathbb{E}[\log \mathbf{1}'e^{\hat{\bB}_{s,r}'\bx} - \by'\hat{\bB}_{s,r}'\bx] \le 
	\ \frac{1}{N}\sum_{i=1}^N
	(\log \mathbf{1}'e^{\hat{\bB}_{s,r}'\bx_i} - \by_i'\hat{\bB}_{s,r}'\bx_i) {}\\ 
	+\frac{2 (R_{\bx} \bar{C}_{s,r} + \log K)}{\sqrt{N}} 
	+ 
	(R_{\bx} \bar{C}_{s,r} + \log K)\sqrt{\frac{8\log(\frac{2}{\delta})}{N}}\ ,
	\end{multline*}
	and for any $\zeta>\frac{2 (R_{\bx} \bar{C}_{s,r} + \log K)}{\sqrt{N}}
	+ 
	(R_{\bx} \bar{C}_{s,r} + \log K)\sqrt{\frac{8\log(\frac{2}{\delta})}{N}}$,
	\begin{equation*} 
	\begin{aligned}
	& \quad \ \mathbb{P}\Bigl(\log \mathbf{1}'e^{\hat{\bB}_{s,r}'\bx} - \by' \hat{\bB}_{s,r}'\bx    \ge
	\frac{1}{N}\sum_{i=1}^N (\log \mathbf{1}'e^{\hat{\bB}_{s,r}'\bx_i} - \by_i'\hat{\bB}_{s,r}'\bx_i)+\zeta\Bigr) \\
	& \le \frac{\frac{1}{N}\sum_{i=1}^N
		(\log \mathbf{1}'e^{\hat{\bB}_{s,r}'\bx_i} - \by_i'\hat{\bB}_{s,r}'\bx_i)+\frac{2 (R_{\bx} \bar{C}_{s,r} + \log K)}{\sqrt{N}}
	}{\frac{1}{N}\sum_{i=1}^N
	(\log \mathbf{1}'e^{\hat{\bB}_{s,r}'\bx_i} - \by_i'\hat{\bB}_{s,r}'\bx_i)+\zeta} \\
& \qquad \qquad \qquad  \qquad \qquad \quad + \frac{(R_{\bx} \bar{C}_{s,r} + \log K)\sqrt{\frac{8\log(\frac{2}{\delta})}{N}}}{\frac{1}{N}\sum_{i=1}^N
	(\log \mathbf{1}'e^{\hat{\bB}_{s,r}'\bx_i} - \by_i'\hat{\bB}_{s,r}'\bx_i)+\zeta}.
\end{aligned}
\end{equation*}
\end{thm}

\begin{thm} \label{c2} Suppose the solution to MLG-1S is $\hat{\bB}_{1,s}$. Under Assumptions~\ref{b1} and \ref{b3}, for any
	$0<\delta<1$, with probability at least $1-\delta$ with respect to the
	sampling,
	\begin{multline*} 
	\mathbb{E}[\log \mathbf{1}'e^{\hat{\bB}_{1,s}'\bx} - \by'\hat{\bB}_{1,s}'\bx] \le
	\frac{1}{N}\sum_{i=1}^N
	(\log \mathbf{1}'e^{\hat{\bB}_{1,s}'\bx_i} - \by_i'\hat{\bB}_{1,s}'\bx_i) {}\\
	+\frac{2 (R_{\bx} \bar{C}_{1,s} + \log K)}{\sqrt{N}} 
	+ 
	(R_{\bx} \bar{C}_{1,s} + \log K)\sqrt{\frac{8\log(\frac{2}{\delta})}{N}}\ ,
	\end{multline*}
	and for any $\zeta>\frac{2 (R_{\bx} \bar{C}_{1,s} + \log K)}{\sqrt{N}}
	+ 
	(R_{\bx} \bar{C}_{1,s} + \log K)\sqrt{\frac{8\log(\frac{2}{\delta})}{N}}$,
	\begin{equation*} 
	\begin{aligned}
	& \quad \ \mathbb{P}\Bigl(\log \mathbf{1}'e^{\hat{\bB}_{1,s}'\bx} - \by'\hat{\bB}_{1,s}'\bx    \ge
	\frac{1}{N}\sum_{i=1}^N (\log \mathbf{1}'e^{\hat{\bB}_{1,s}'\bx_i} - \by_i'\hat{\bB}_{1,s}'\bx_i)+\zeta\Bigr) \\
	&	\le \frac{\frac{1}{N}\sum_{i=1}^N
		(\log \mathbf{1}'e^{\hat{\bB}_{1,s}'\bx_i} - \by_i'\hat{\bB}_{1,s}'\bx_i)+\frac{2 (R_{\bx} \bar{C}_{1,s} + \log K)}{\sqrt{N}}
	}{\frac{1}{N}\sum_{i=1}^N
	(\log \mathbf{1}'e^{\hat{\bB}_{1,s}'\bx_i} - \by_i'\hat{\bB}_{1,s}'\bx_i)+\zeta}\\
& \qquad \qquad \qquad  \qquad \qquad \quad + \frac{(R_{\bx} \bar{C}_{1,s} + \log K)\sqrt{\frac{8\log(\frac{2}{\delta})}{N}}}{\frac{1}{N}\sum_{i=1}^N
	(\log \mathbf{1}'e^{\hat{\bB}_{1,s}'\bx_i} - \by_i'\hat{\bB}_{1,s}'\bx_i)+\zeta}.
\end{aligned}
\end{equation*}
\end{thm}

We note that the expected log-loss on a new test sample depends both on the sample average log-loss on the training set, and the magnitude of the regularizer in the formulation. The form of the bounds in Theorems~\ref{c1} and \ref{c2} demonstrates the validity of MLG-SR and MLG-1S in leading to a good out-of-sample performance. For $r \ge 2$, $\bar{C}_{s,r}$ can be considered smaller than $\bar{C}_{1,s}$, while for $r=1$, the reverse holds. We can decide which model to use on a case-by-case basis, by computing their out-of-sample error on a validation set.

\section{Numerical Experiments} \label{sec:5-4}
In this section, we will test the out-of-sample performance of the MLR and MLG relaxations on a number of synthetic datasets, and compare with several commonly used multi-output regression/classification models. 

\subsection{MLR Relaxations} \label{mlr-exp}
In this subsection we will first explore the selection of a proper norm for the regularizer based on an appropriate notion of distance in the data space. To this end, we design two different structures for the true coefficient matrix denoted by $\bB^*$ in order to reflect different distance metrics: 
\begin{enumerate}
	\item $\bB^*$ is drawn from a standard multivariate normal distribution, which corresponds to an $\ell_2$-norm induced Wasserstein metric ($r=2$);
	\item we first generate $\bB^*$ from a standard multivariate normal distribution, and then normalize each row using the softmax function while keeping the sign of each element unchanged. The normalization guarantees an equal row absolute sum for $\bB^*$. This can be thought of as standardizing the effect of each predictor, which is represented by the absolute sum over the $K$ columns of $\bB^*$. Such a coefficient matrix implies an $\ell_1$-norm distance metric in the data space ($r=1$). The reason is that in the dual space ($\|\cdot\|_{\infty}$), the vertex of the constraint set has each coordinate being the same in absolute value, and in our setting each coordinate is represented by the absolute sum over the $K$ columns of $\bB^*$.
\end{enumerate}
The predictor $\bx$ is generated from a multivariate normal distribution with mean zero and covariance $\bSigma_{\bx} = (\sigma^{\bx}_{ij})_{i,j \in \lb p \rb}$, where $\sigma^{\bx}_{ij} = 0.9^{|i-j|}$. The response vector $\by$ is generated as $$\by = (\bB^*)'\bx + \boldsymbol{\eta},$$ where $\boldsymbol{\eta}$ is a standard normal random vector. Throughout the experiments we set $p=5$, and $K=3$.

We adopt a loss function $h_{\bB}(\bx, \by) = \|\by-\bB'\bx\|_r$ that is 1-Lipschitz continuous on $\|\cdot\|_r$. Note that we use the same norm to define the loss function and the Wasserstein metric. We will compare the MLR-SR and MLR-1S relaxations induced by $r=1$ and $r=2$, respectively, in terms of their out-of-sample {\em Weighted Mean Squared Error (WMSE)}, defined as:
\begin{equation*}
\text{WMSE} \triangleq \frac{1}{M} \sum_{i=1}^M (\by_i-\hat{\by}_i)' \hat{\bSigma}^{-1} (\by_i-\hat{\by}_i),
\end{equation*}
where $M$ is the size of the test set, $\by_i$ and $\hat{\by}_i$ are the true and predicted response vectors for the $i$-th test sample, respectively, and $\hat{\bSigma}$ is the covariance matrix of the prediction error on the training set, 
$$\hat{\bSigma} = (\bY-\hat{\bY})'(\bY-\hat{\bY})/(N-pK),$$ where $\bY, \hat{\bY} \in \mbb{R}^{N \times K}$ are the true and estimated response matrices of the training set, respectively, and $N$ is the size of the training set. We will also look at the {\em Conditional Value at Risk (CVaR)} of the WMSE (at the confidence level $\alpha = 0.8$) that quantifies its tail behavior.  

Figures~\ref{fig:mlr-1} and \ref{fig:mlr-2} show the comparison of MLR-SR and MLR-1S formulations derived from the Wasserstein metric induced by the $\ell_r$ norm, with $r=1$ and $r=2$, when the radius of the Wasserstein ball $\epsilon$ is varied. As expected, when $\bB^*$ is a dense matrix, the $\ell_2$ norm is a proper distance metric in the data space, and as a result, the two relaxations with $r=2$ achieve a lower out-of-sample prediction bias. On the other hand, when the structure of $\bB^*$ implies an $\ell_1$-norm distance metric on the data (Figure~\ref{fig:mlr-2}), the formulations with $r=1$ have a better performance.

\begin{figure}[h] 
	\begin{center}
		\begin{subfigure}{.49\textwidth}
			\centering
			\includegraphics[width=1.0\textwidth]{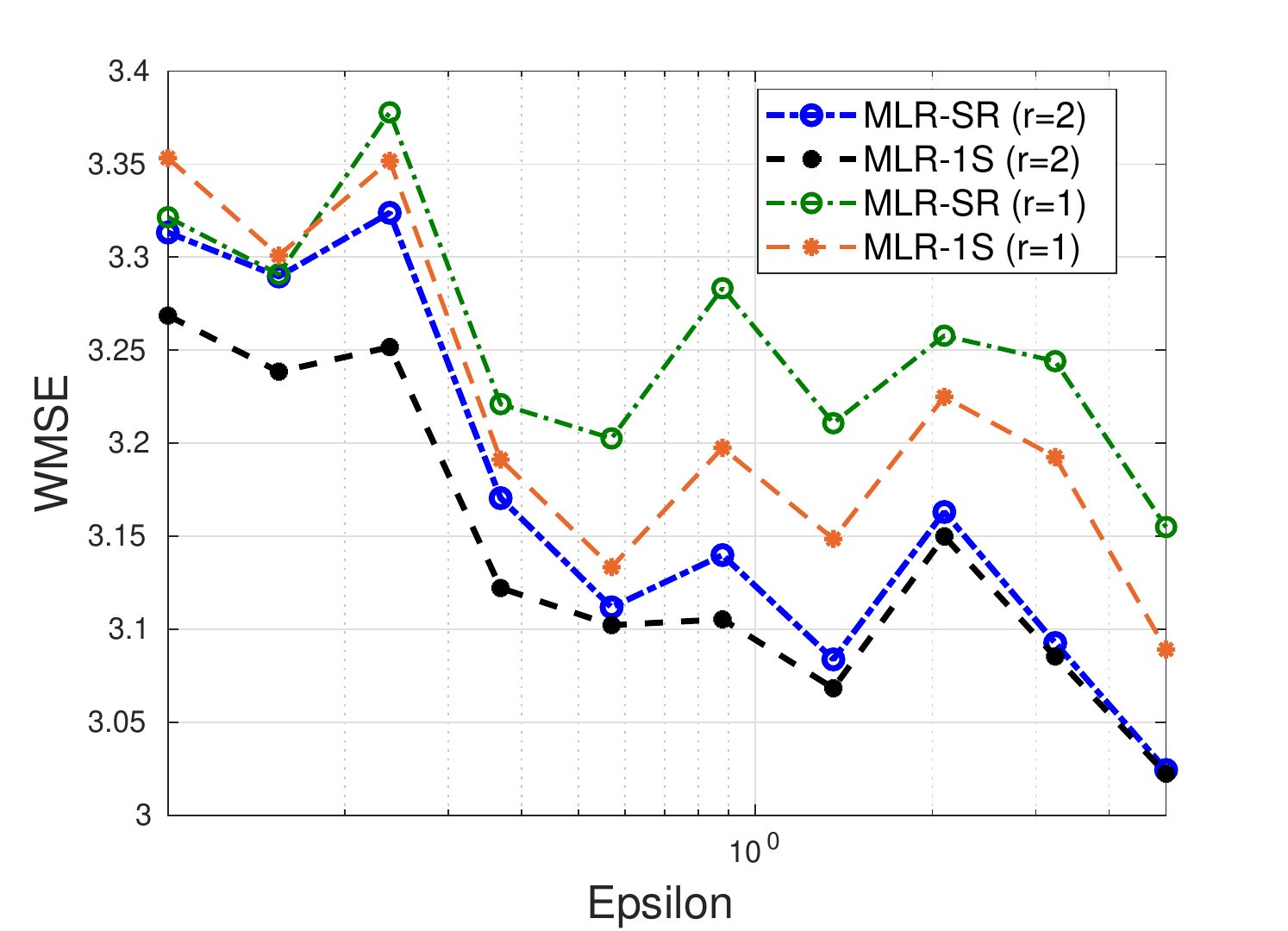}
			\caption{\small{WMSE.}}
		\end{subfigure}
		\begin{subfigure}{.49\textwidth}
			\centering
			\includegraphics[width=1.0\textwidth]{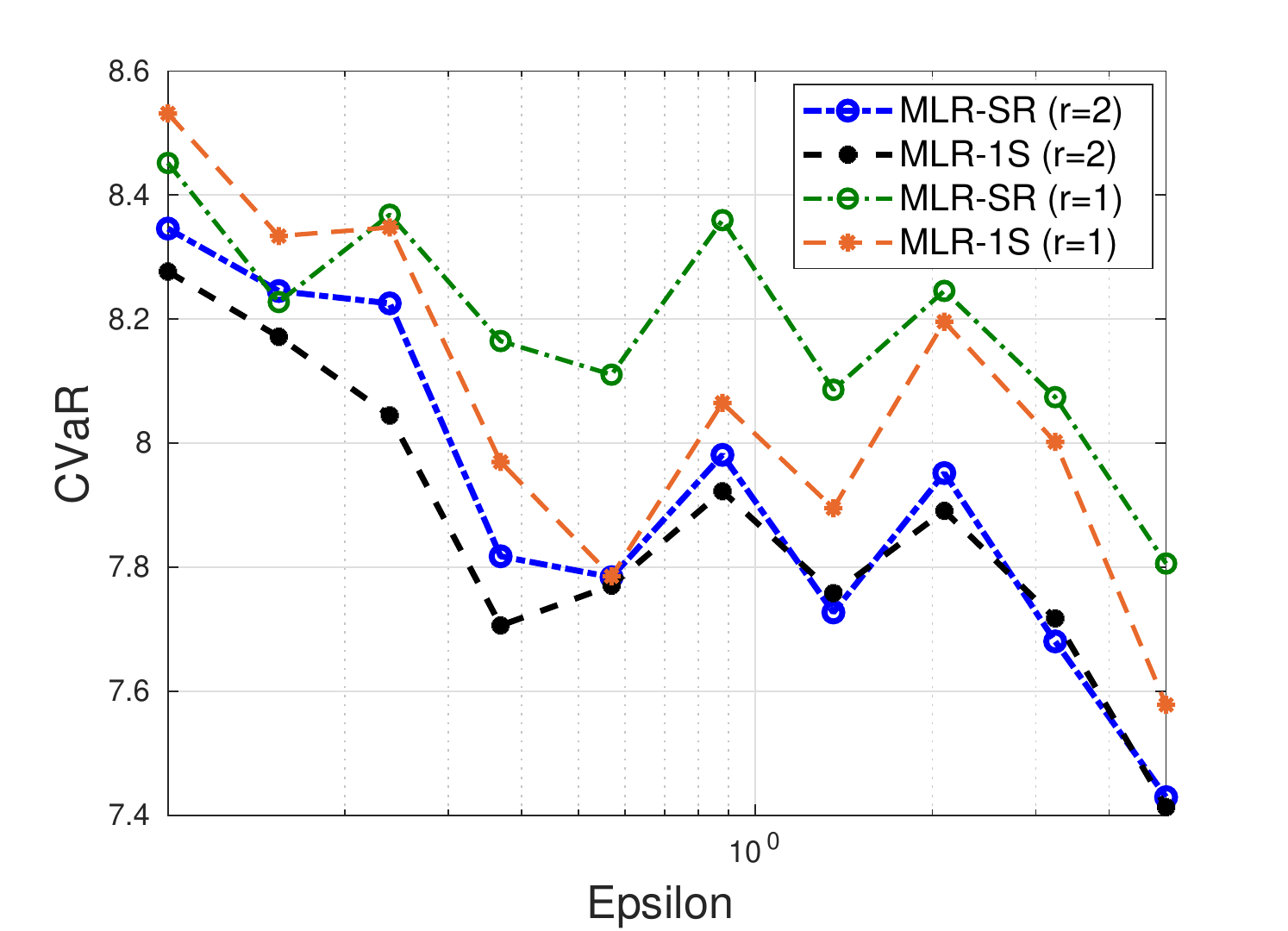}
			\caption{\small{CVaR of WMSE.}}
		\end{subfigure}
	\end{center}
	\caption{The out-of-sample performance of MLR-SR and MLR-1S with normally distributed $\bB^*$.}
	\label{fig:mlr-1}
\end{figure}

\begin{figure}[h] 
	\begin{center}
		\begin{subfigure}{.49\textwidth}
			\centering
			\includegraphics[width=1.0\textwidth]{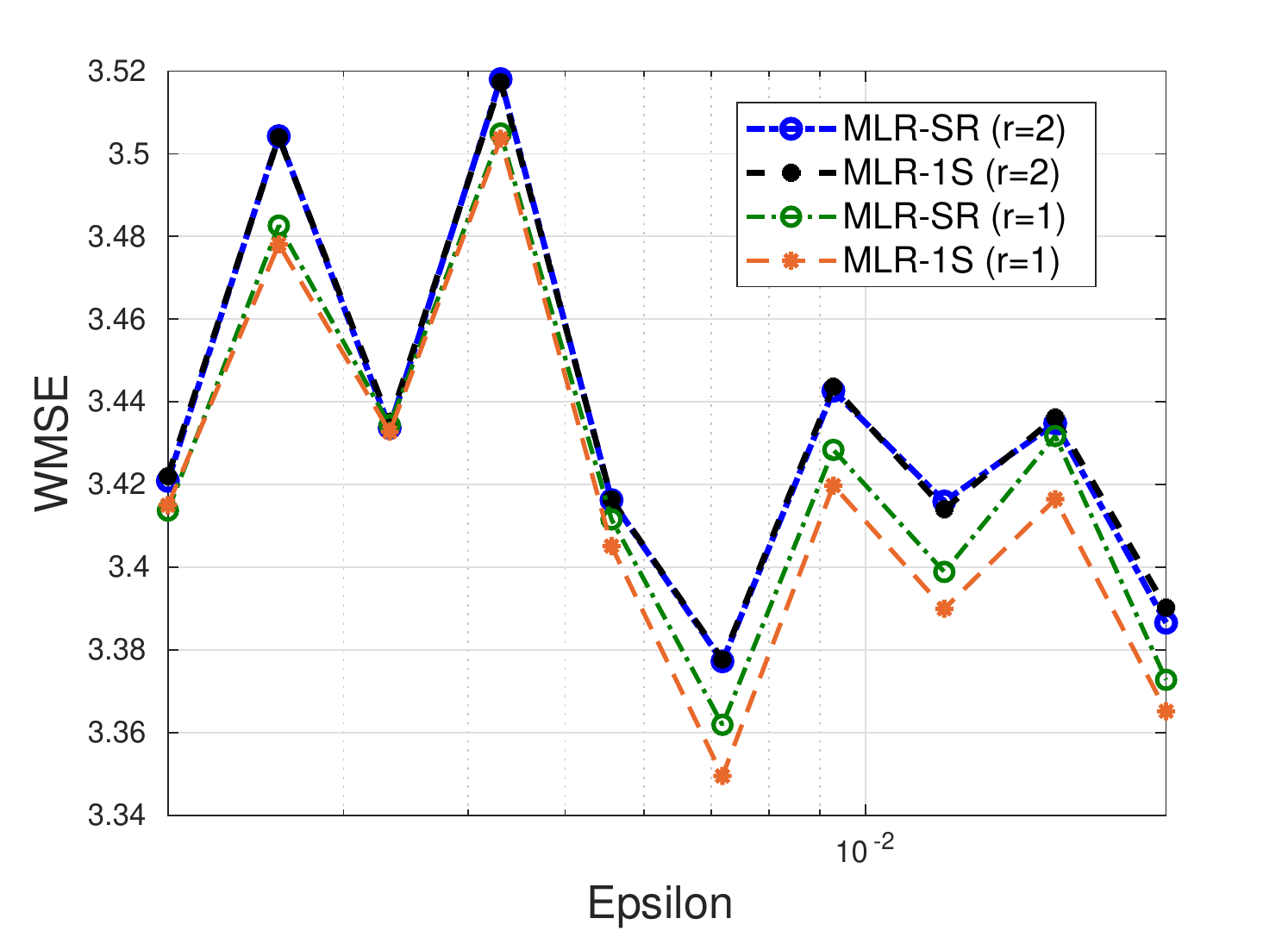}
			\caption{\small{WMSE.}}
		\end{subfigure}
		\begin{subfigure}{.49\textwidth}
			\centering
			\includegraphics[width=1.0\textwidth]{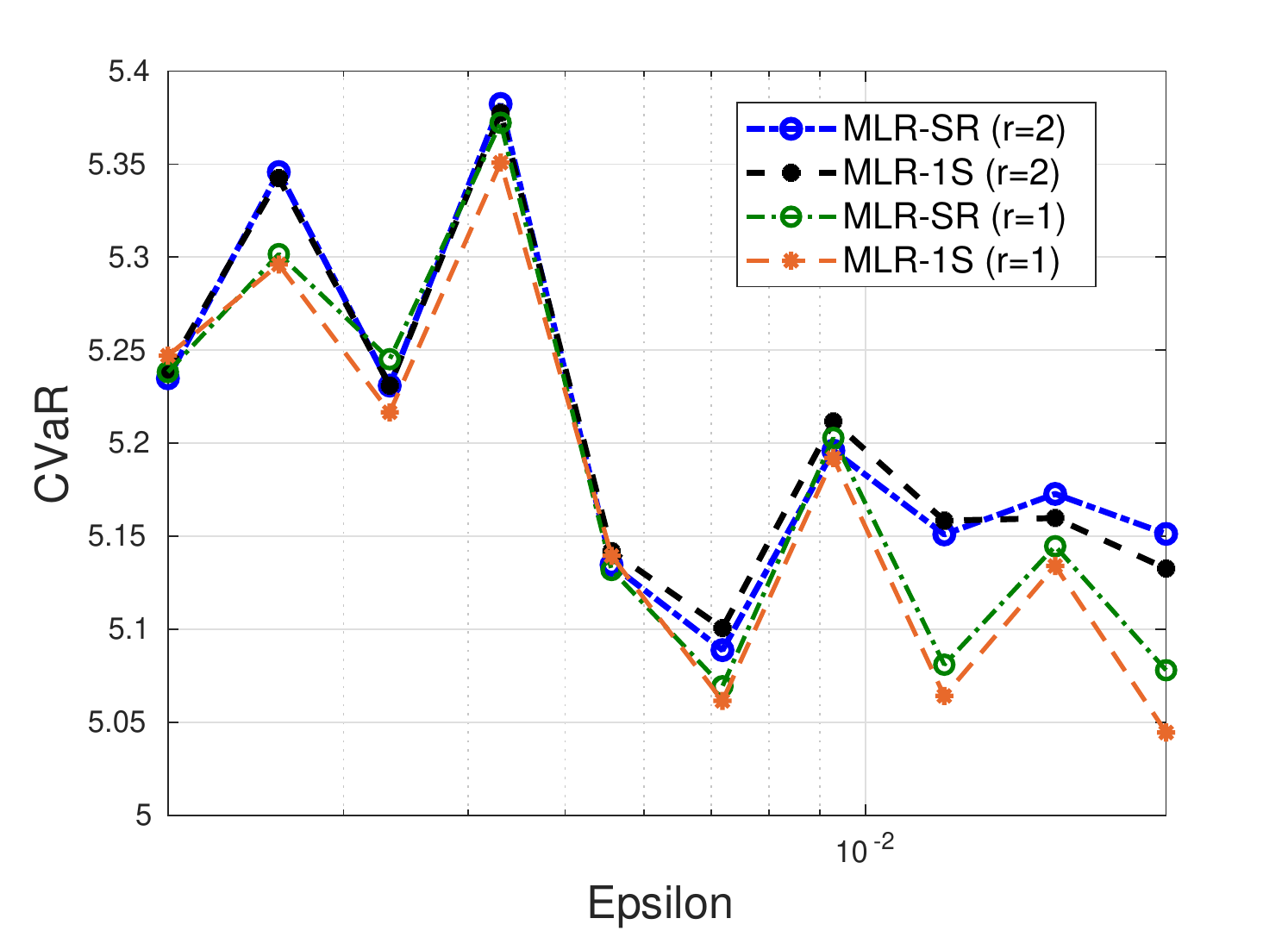}
			\caption{\small{CVaR of WMSE.}}
		\end{subfigure}
	\end{center}
	\caption{The out-of-sample performance of MLR-SR and MLR-1S when $\bB^*$ has an equal row absolute sum.}
	\label{fig:mlr-2}
\end{figure}

We also compare the four MLR relaxations with an optimal $\epsilon$ chosen by cross-validation. Tables~\ref{tab:mlr-1} and \ref{tab:mlr-2} show the mean WMSE and CVaR over 100 repetitions (the numbers inside the parentheses indicate the corresponding standard deviations). Similar conclusions can be drawn from the results in the tables. With a proper choice of $r$, the MLR relaxations are able to achieve a lower prediction error with a smaller variance. For example, in Table~\ref{tab:mlr-1}, compared to MLR-SR (r=1), the two relaxations with $r=2$ improved the WMSE by 4.6\%.

\begin{table}[h]
	\caption{The out-of-sample performance of MLR-SR and MLR-1S with cross-validated $\epsilon$ when $\bB^*$ is normally distributed.} 
	\label{tab:mlr-1}
	\begin{center}
		\begin{tabular} {l| c c }
			\hline
			& WMSE              & CVaR                \\ \hline
			MLR-SR (r=1) & 3.26 (0.48)       & 4.91 (0.70)            \\ 
			MLR-1S (r=1) & 3.21 (0.40)       & 4.93 (0.65)              \\ 
			MLR-SR (r=2) & 3.11 (0.36)       & 4.74 (0.62)                 \\ 
			MLR-1S (r=2) & 3.11 (0.35)       & 4.75 (0.64)             \\  \hline
		\end{tabular}
	\end{center}
\end{table}

\begin{table}[h]
	\caption{The out-of-sample performance of MLR-SR and MLR-1S with cross-validated $\epsilon$ when $\bB^*$ has an equal row absolute sum.} 
	\label{tab:mlr-2}
	\begin{center}
		\begin{tabular} {l| c c }
			\hline
			& WMSE        & CVaR                \\ \hline
			MLR-SR (r=1) & 3.04 (0.52)        & 4.56 (0.83)             \\ 
			MLR-1S (r=1) & 3.05 (0.52)        & 4.60 (0.89)               \\ 
			MLR-SR (r=2) & 3.05 (0.52)        & 4.62 (0.88)               \\ 
			MLR-1S (r=2) & 3.05 (0.52)        & 4.62 (0.89)           \\  \hline
		\end{tabular}
	\end{center}
\end{table}

We next compare the MLR-SR and MLR-1S formulations with several other popular methods for MLR, including \textsl{OLS}, \textsl{Reduced Rank Regression (RRR)} \citep{izenman1975reduced, velu2013multivariate}, \textsl{Principal Components Regression (PCR)} \citep{massy1965principal}, \textsl{Factor Estimation and Selection (FES)} \citep{yuan2007dimension}, the \textsl{Curds and Whey (C\&W)} procedure \citep{breiman1997predicting}, and \textsl{Ridge Regression (RR)} \citep{brown1980adaptive, haitovsky1987multivariate}. We provide a brief outline of these methods.
RRR restricts the rank of $\bB$, and its solution is obtained by a {\em Canonical Correlation Analysis (CCA)} of the response and predictor matrices that finds a sequence of uncorrelated linear combinations of the predictors and a corresponding sequence of uncorrelated linear combinations of the responses such that their correlations are successively maximized. PCR converts the predictors into a set of linearly uncorrelated variables and applies OLS on the transformed variables. Both RRR and PCR form linear combinations of predictors and responses (in the case of RRR) which hurts interpretability since it is not possible to explain an original response via the original predictors. FES penalizes the sum of the singular values of $\bB$. The C\&W procedure shrinks the canonical variates between $\bx$ and $\by$. RR penalizes the sum of the squared elements in $\bB$ (equivalent to multiple independent ridge regression of each coordinate of $\by$).  

To test the robustness of various methods, we inject outliers to the training datasets whose distribution differs from the majority by a normally distributed random quantity. Specifically, the response of outliers is generated as $$\by = (\bB^*)'\bx + \boldsymbol{\eta}+\mathbf{o}_{\boldsymbol{\eta}},$$ 
where $\boldsymbol{\eta} \sim \mathcal{N}(\mathbf{0}, \mathbf{I})$, and $\mathbf{o}_{\boldsymbol{\eta}} \sim \mathcal{N}(\mathbf{0}, \bSigma_{\by})$, where $\bSigma_{\by} = (\sigma^{\by}_{ij})_{i,j \in \lb K \rb}$, with $\sigma^{\by}_{ij} = (-0.9)^{|i-j|}$. Note that the perturbation occurs only on the response variables. 

We generate 20 datasets with a training size of 100 and a test size of 60, and compare the WMSE and CVaR of various models on a clean test set. All the regularization coefficients are tuned through cross-validation. Table~\ref{tab:mlr-3} shows the average performance on datasets with 20\% and 30\% outliers, respectively, when $\bB^*$ is generated from a standard normal distribution. We see that as the proportion of outliers increases, the WMSE and its CVaR increase, and in both scenarios, the MLR-SR and MLR-1S relaxations achieve the smallest out-of-sample prediction error with a small variance. They improve the WMSE by 1\% -- 5\% and 3\% -- 7\% when the proportion of outliers is 20\% and 30\%, respectively. PCR and FES achieve a slightly worse performance, but with a considerably higher variance in the scenario with 20\% outlier. PCR works well with linearly correlated predictors, but could possibly fail when there exists a highly nonlinear relationship among the predictors. 

To further characterize the robustness of various approaches, we compute outlier detection rates on the test set, and draw the {\em Receiver Operating Characteristic (ROC)} curves obtained from varying the threshold values in the outlier detection rule. Note that in this case both the training and test datasets contain outliers. The response of outliers is generated as $$\by = (\bB^*)'\bx +\mathbf{o}_{\boldsymbol{\eta}},$$ where $\mathbf{o}_{\boldsymbol{\eta}} \sim \mathcal{N}(4*\mathbf{1}_K, \bSigma_{\by})$, with $\mathbf{1}_K$ the K-dimensional vector of all ones, and $\bSigma_{\by} = (\sigma^{\by}_{ij})_{i,j \in \lb K \rb}$, with $\sigma^{\by}_{ij} = (-0.9)^{|i-j|}$. The outlier detection criterion is described as follows:
\begin{equation*}
(\bx_i,  \by_i) = 
\begin{cases}
\text{outlier,} & \text{if \ $\br_i'\hat{\bSigma}^{-1}\br_i \ge c,$} \\
\text{not an outlier,} & \text{otherwise,}
\end{cases}
\end{equation*}
where $\br_i = \by_i - \hat{\bB}'\bx_i$ is the estimated residual, $\hat{\bSigma}$ is the covariance matrix of the prediction error on the training set, and $c$ is the threshold value that is varied between 0 and $\chi^2_{0.99}(K)$ (0.99 percentile of the chi-square distribution with $K$ degrees of freedom) to produce the ROC curves. Table~\ref{tab:mlr-3} shows the average {\em Area Under the ROC Curve (AUC)} on the test set over 20 repetitions, and Figure~\ref{fig:mlr-3} shows the ROC curves for different methods with 20\% and 30\% outliers, where the true positive rates and false positive rates are averaged over 20 repetitions. Compared to other methods except FES, the MLR-SR and MLR-1S models improve the AUC by 2\% -- 11\% when we have 20\% outliers, and 6\% -- 17\% when we have 30\% outliers, with a relatively small variability. Notice that FES also achieves a high AUC, but with a worse out-of-sample predictive performance.

\begin{table}[h]
	\caption{The out-of-sample performance of different MLR models, mean (std.).} 
	\label{tab:mlr-3}
	\begin{center}
		{\normalsize \begin{tabular} {l| c c c }
				\hline
				Proportion of Outliers & \multicolumn{3}{c}{20\%}           \\ \hline
				& WMSE                  & CVaR        & AUC            \\ \hline
				MLR-SR (r=2)           & 2.55 (0.26)           & 3.91 (0.51)  &0.89 (0.04)                   \\ 
				MLR-1S (r=2)           & 2.59 (0.21)           & 3.93 (0.43)  & 0.89 (0.03)               \\  
				OLS                    & 2.66 (0.44)           & 4.11 (0.83)  & 0.85 (0.06)          \\
				RR                     & 2.64 (0.42)           & 4.12 (0.82)  &0.87 (0.03)          \\
				RRR                    & 2.68 (0.36)           & 4.03 (0.61)  &0.80 (0.05)         \\
				FES                    & 2.58 (0.29)           & 4.01 (0.55)  &0.89 (0.03)         \\
				C\&W                   & 2.65 (0.42)           & 4.10 (0.80)  &0.86 (0.06)         \\
				PCR                    & 2.61 (0.29)           & 4.01 (0.50)  &0.86 (0.03)          \\ 
				\hline
				Proportion of Outliers & \multicolumn{3}{c}{30\%} \\
				\hline  
				MLR-SR (r=2)           & 2.63 (0.33)           & 4.07 (0.66)   & 0.83 (0.09)  \\
				MLR-1S (r=2)           & 2.57 (0.31)            &4.02 (0.60)  & 0.83 (0.09) \\
				OLS                    & 2.75 (0.33)            &4.14 (0.71) & 0.73 (0.11) \\
				RR                     & 2.72 (0.32)            &4.14 (0.60) & 0.78 (0.10) \\
				RRR                    & 2.76 (0.44)            &4.22 (0.80) & 0.71 (0.12) \\
				FES                    & 2.66 (0.33)           &4.02 (0.61) & 0.83 (0.09) \\
				C\&W                   & 2.74 (0.33)            &4.11 (0.65) & 0.74 (0.11) \\
				PCR                    & 2.68 (0.32)           &4.02 (0.62) & 0.73 (0.11) \\
				\hline
			\end{tabular}}
		\end{center}
	\end{table}
	
	\begin{figure}[h]
		\begin{center}
			\begin{subfigure}{0.9\textwidth}
				\centering
				\includegraphics[width=1.0\textwidth]{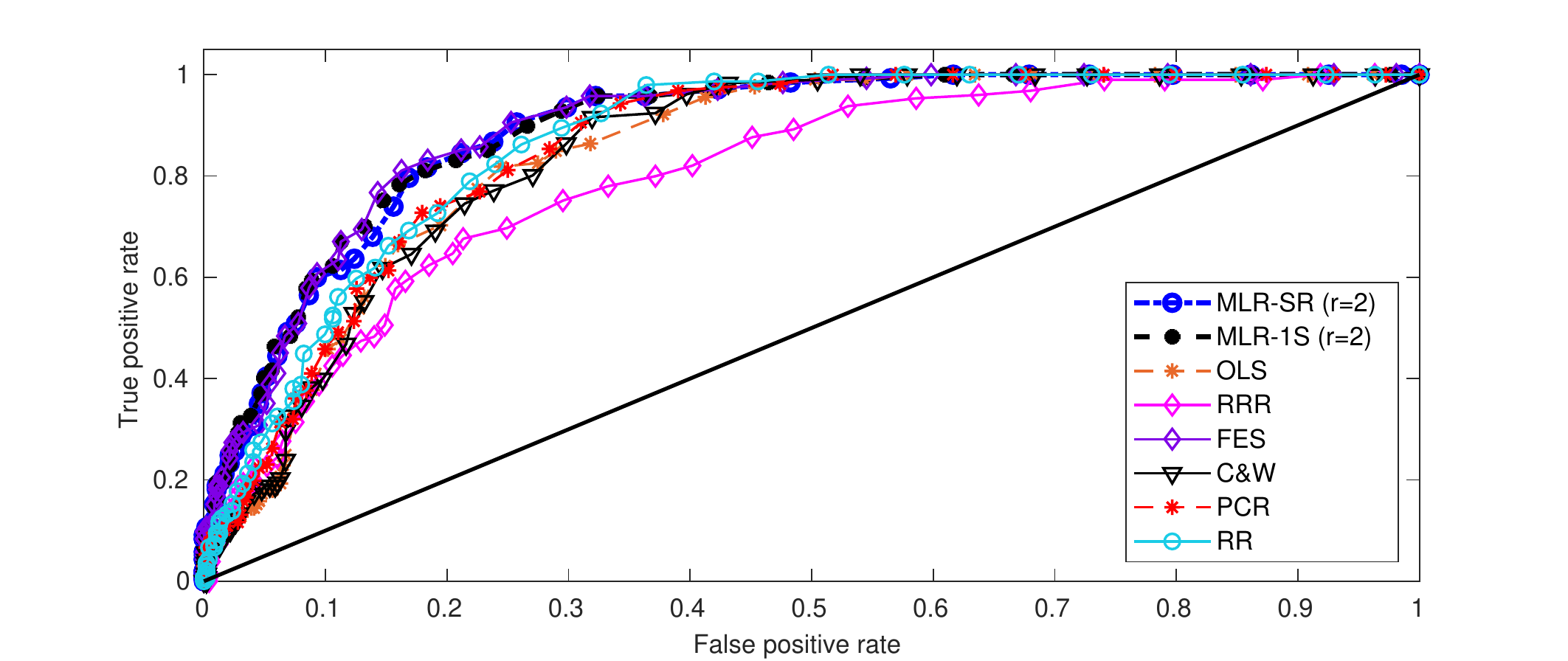}
				\caption{\small{20\% outliers.}}
			\end{subfigure}
			
			\begin{subfigure}{0.9\textwidth}
				\centering
				\includegraphics[width=1.0\textwidth]{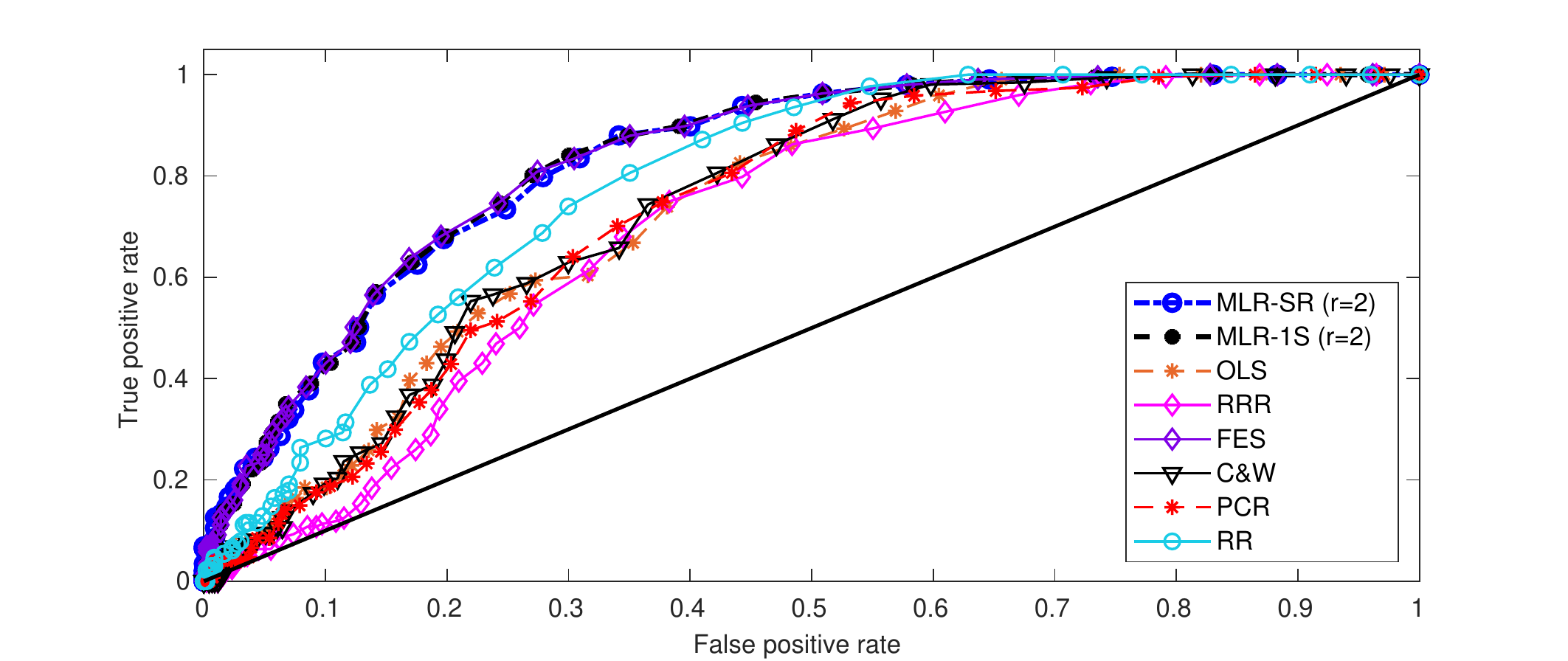}
				\caption{\small{30\% outliers.}}
			\end{subfigure}
		\end{center}
		\caption{The ROC curves of different MLR models.}
		\label{fig:mlr-3}
	\end{figure}
	
	\subsection{MLG Relaxations}
	In this subsection we study the performance of the two DRO-MLG relaxations, and compare them with a number of MLG variants on simulated datasets, in terms of their out-of-sample log-loss and classification accuracy.
	
	We first study the problem of selecting the right regularizer based on the distance metric in the data space. Similar to Section~\ref{mlr-exp}, we experiment with two types of $\bB^*$, one coming from a multivariate normal distribution, and the other normalized to have an equal row absolute sum. They respectively correspond to an $\ell_2$ and $\ell_1$-norm distance metric in the data space.
	The predictor is drawn according to $\bx \sim \scrN(\mathbf{0}, \bSigma_{\bx})$, where $\bSigma_{\bx} = (\sigma^{\bx}_{ij})_{i, j \in \lb p \rb}$, and $\sigma^{\bx}_{ij} = 0.9^{|i-j|}$. The label vector $\by \in \{0, 1\}^K$ is generated from a 
	multinomial distribution with probabilities specified by the softmax normalization of $(\bB^*)'\bx+\boldsymbol{\eta}$, where $\boldsymbol{\eta} \sim \scrN(\mathbf{0},\mathbf{I}_K)$.
	
	We set $p=5, K=3$, and conduct 20 simulation runs, each with a training size of 100 and a test size of 60. The performance metrics we use include: $(i)$ the average log-loss, $(ii)$ the {\em Correct Classification Rate (CCR)}, and $(iii)$ the {\em Conditional Value at Risk (CVaR)} (at the confidence level 0.8) of log-loss, which computes the expectation of extreme log-loss values. The average performance metrics on the test set over 20 replications are reported. 
	
	Figures~\ref{fig:mlg-1} and \ref{fig:mlg-2} show the comparison of the four models as the Wasserstein radius $\epsilon$ is varied. We see that when $\bB^*$ is a dense matrix, the MLG-SR and MLG-1S induced by the $\ell_2$-norm have a higher classification accuracy and a lower log-loss. By contrast, when the structure of $\bB^*$ implies an $\ell_1$-norm distance metric in the data space, the MLG-SR and MLG-1S with $r=1$ perform better. We also validate this conclusion in Tables~\ref{tab:mlg-1} and \ref{tab:mlg-2} where the optimal Wasserstein set radius $\epsilon$ is chosen through cross-validation. With normally distributed $\bB^*$, the formulations with $r=2$ improve the CCR and log-loss by 5\% compared to the ones induced by $r=1$.
	
	\begin{figure}[h] 
		\begin{center}
			\begin{subfigure}{0.6\textwidth}
				\centering
				\includegraphics[width=1.0\textwidth]{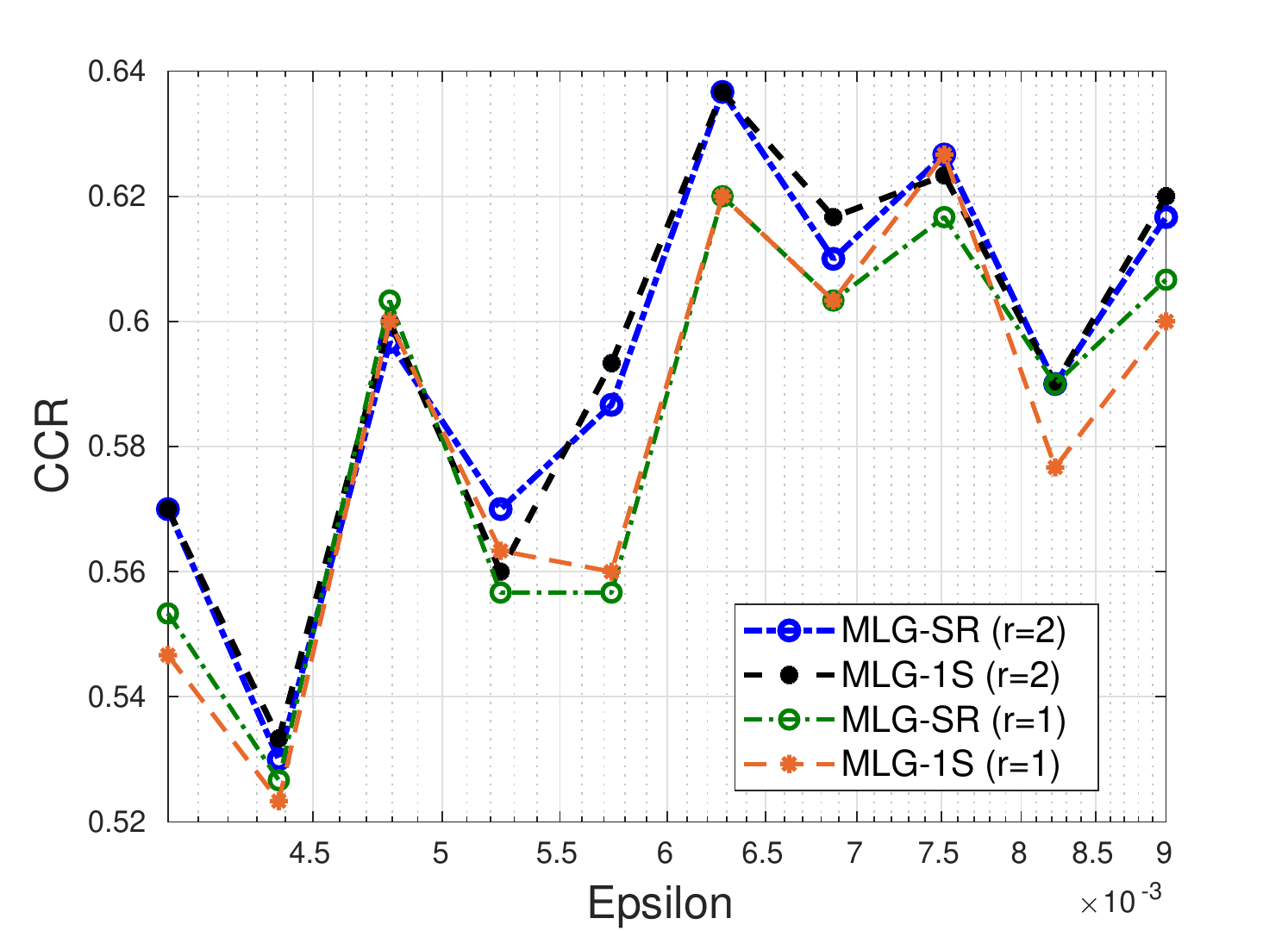}
				\caption{\small{CCR.}}
			\end{subfigure}
			
			\begin{subfigure}{0.6\textwidth}
				\centering
				\includegraphics[width=1.0\textwidth]{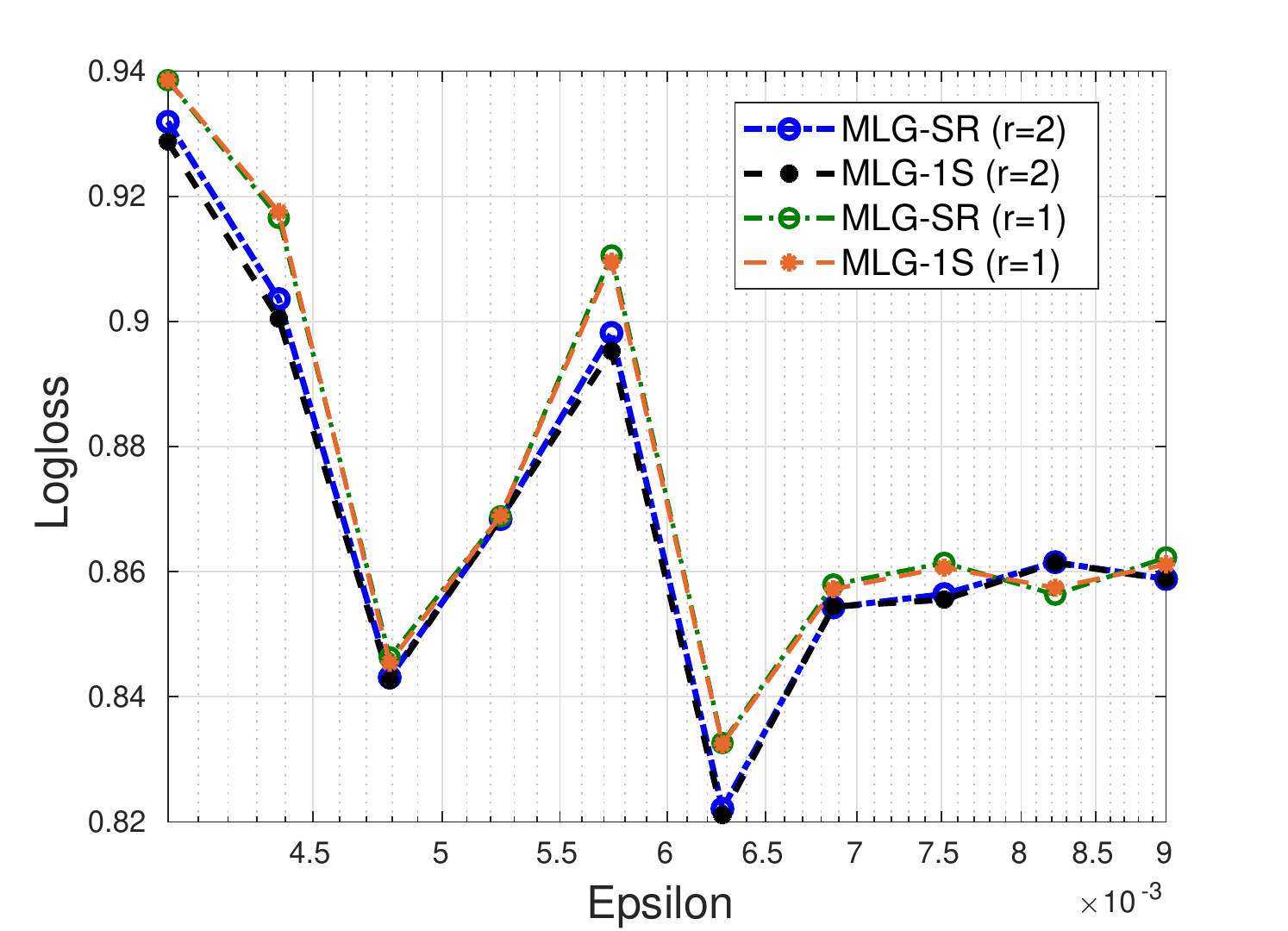}
				\caption{\small{Log-loss.}}
			\end{subfigure}
			
			\begin{subfigure}{0.6\textwidth}
				\centering
				\includegraphics[width=1.0\textwidth]{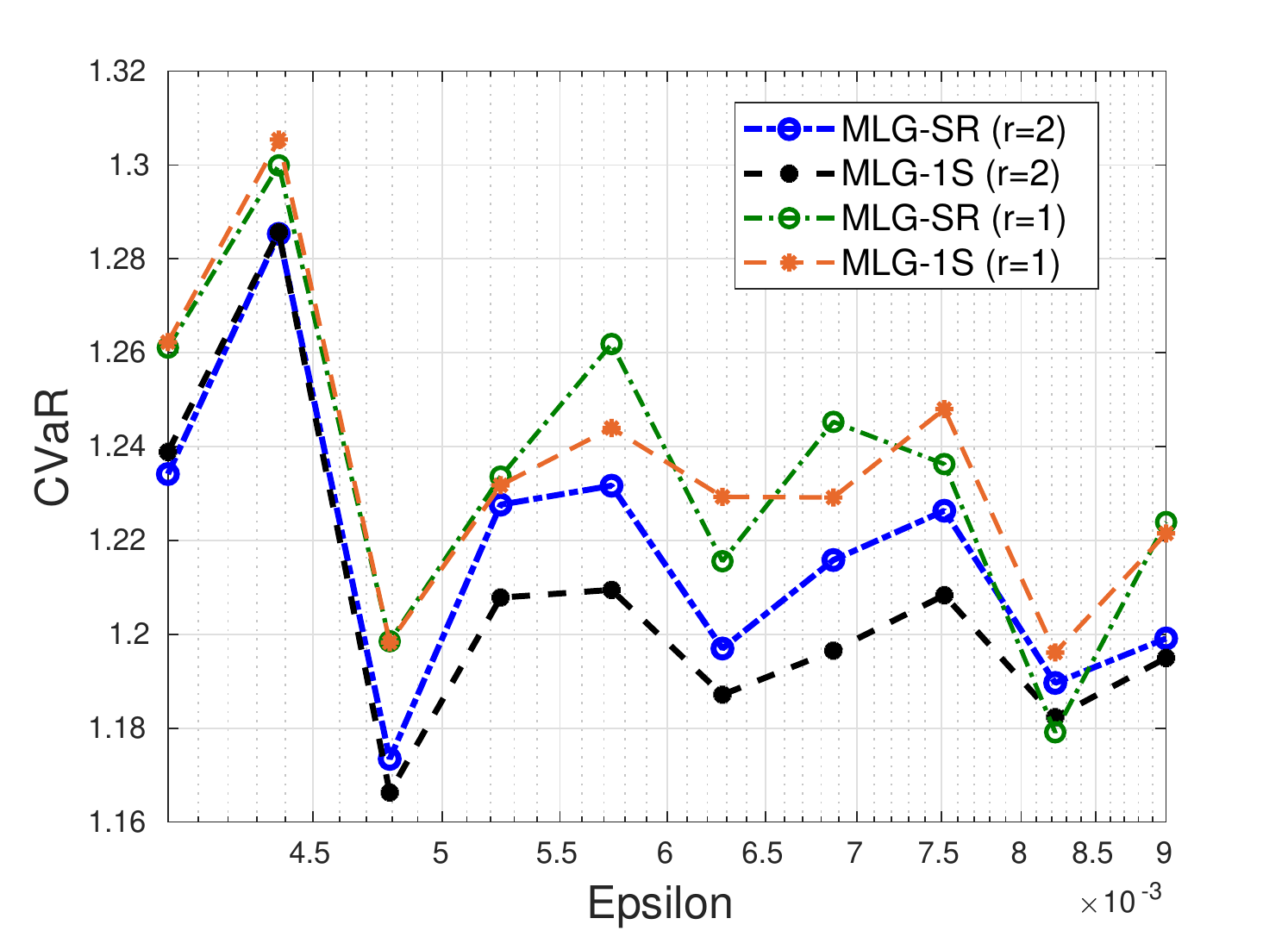}
				\caption{\small{CVaR of log-loss.}}
			\end{subfigure}
		\end{center}
		\caption{The out-of-sample performance of MLG-SR and MLG-1S when $\bB^*$ is normally distributed.}
		\label{fig:mlg-1}
	\end{figure}
	
	\begin{figure}[h] 
		\begin{center}
			\begin{subfigure}{.6\textwidth}
				\centering
				\includegraphics[width=1.0\textwidth]{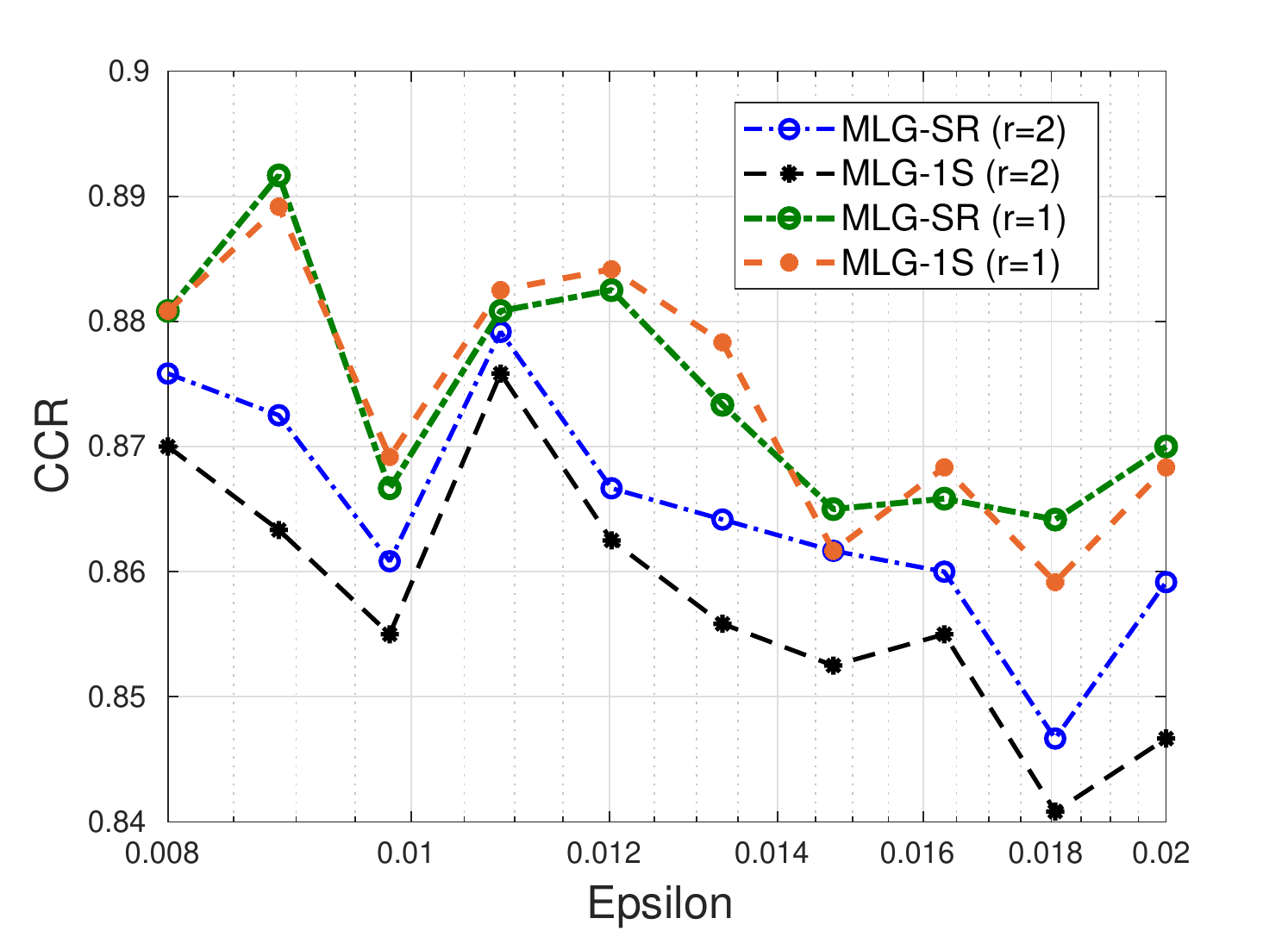}
				\caption{\small{CCR.}}
			\end{subfigure}
			
			\begin{subfigure}{.6\textwidth}
				\centering
				\includegraphics[width=1.0\textwidth]{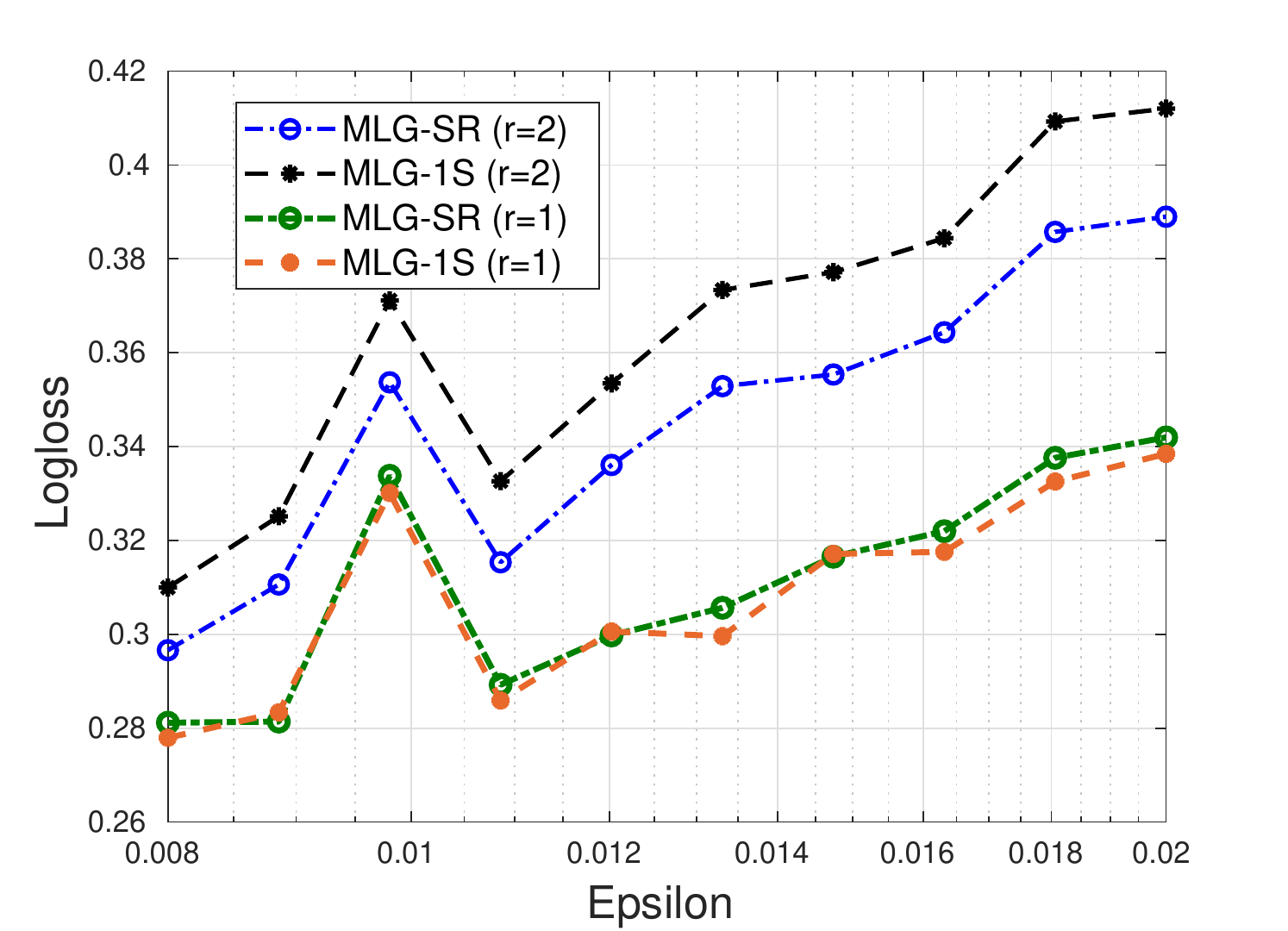}
				\caption{\small{Log-loss.}}
			\end{subfigure}
			
			\begin{subfigure}{.6\textwidth}
				\centering
				\includegraphics[width=1.0\textwidth]{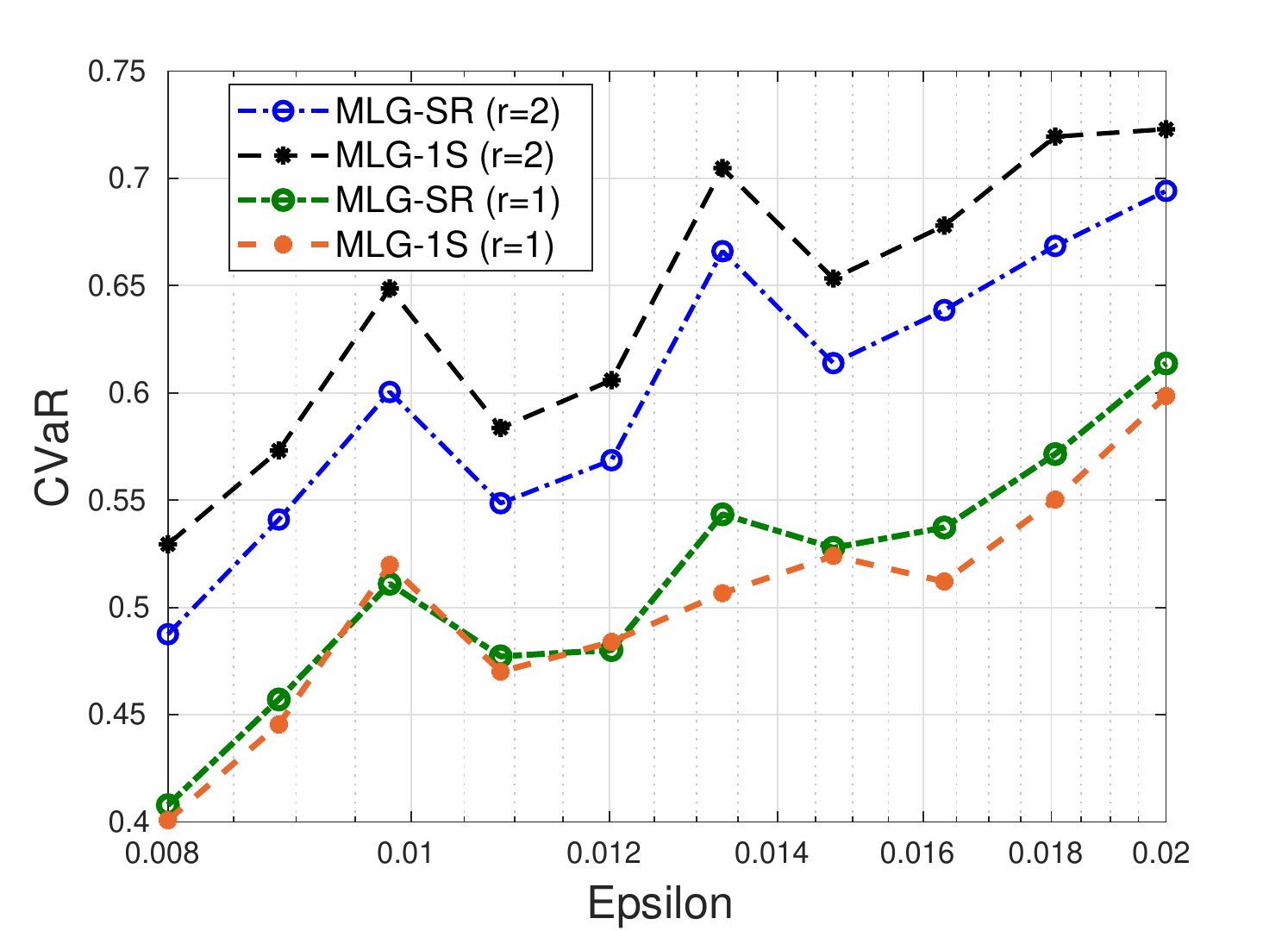}
				\caption{\small{CVaR of log-loss.}}
			\end{subfigure}
		\end{center}
		\caption{The out-of-sample performance of MLG-SR and MLG-1S when $\bB^*$ has an equal row absolute sum.}
		\label{fig:mlg-2}
	\end{figure}

	\begin{table}[h]
		\caption{The out-of-sample performance of MLG-SR and MLG-1S with cross-validated $\epsilon$ when $\bB^*$ is normally distributed.} 
		\label{tab:mlg-1}
		\begin{center}
			\begin{tabular} {l| c c c }
				\hline
				& CCR            & Log-loss                        & CVaR                \\ \hline
				MLG-SR (r=1) & 0.73 (0.06)         &    0.59 (0.09)                           &  1.17 (0.24)                 \\ 
				MLG-1S (r=1) & 0.75 (0.03)          & 0.60 (0.10)                            & 1.16 (0.19)              \\ 
				MLG-SR (r=2) & 0.76 (0.05)         &   0.57 (0.08)                         &  1.16 (0.21)                      \\ 
				MLG-1S (r=2) & 0.76 (0.04)        &   0.57 (0.09)                        &    1.11 (0.13)                 \\  \hline
			\end{tabular}
		\end{center}
	\end{table}
	
	\begin{table}[h]
		\caption{The out-of-sample performance of MLG-SR and MLG-1S with cross-validated $\epsilon$ when $\bB^*$ has an equal row absolute sum.}
		\label{tab:mlg-2} 
		\begin{center}
			\begin{tabular} {l| c c c }
				\hline
				& CCR            & Log-loss                        & CVaR                \\ \hline
				MLG-SR (r=1) & 0.84 (0.04)        &    0.35 (0.05)                          &  0.54 (0.18)                \\ 
				MLG-1S (r=1) & 0.84 (0.04)          & 0.34 (0.07)                           & 0.55 (0.26)              \\ 
				MLG-SR (r=2) & 0.82 (0.05)         &   0.35 (0.12)                         & 0.58 (0.32)                     \\ 
				MLG-1S (r=2) & 0.83 (0.05)         &   0.35 (0.06)                       &    0.58 (0.25)                 \\  \hline
			\end{tabular}
		\end{center}
	\end{table}

	Next we will compare with a number of MLG models, including: $(i)$ Vanilla MLG which minimizes the empirical log-loss with no penalty term, and $(ii)$ Ridge MLG which penalizes the trace of $\bB'\bB$ (the squared Frobenius norm of $\bB$) as in ridge regression \citep{brown1980adaptive, haitovsky1987multivariate}.
	In addition to the three performance metrics used earlier, we introduce another robustness measure that calculates the minimal perturbation needed to fool the classifier. For a given $\bx$ with label $k$, for any $j \neq k$, consider the following optimization problem:
	\begin{equation} \label{perturb}
	\begin{aligned}
	\min_{\tilde{\bx}} & \quad \|\bx - \tilde{\bx}\|_1 \\
	\text{s.t.} & \quad P_j(\tilde{\bx}) \ge P_k(\tilde{\bx}),\\
	& \quad k = \arg\max_i P_i(\bx),
	\end{aligned}
	\end{equation}
	where $P_i(\bx)$ denotes the probability of assigning class label $i$ to $\bx$, which is a function of the trained classifier. Problem (\ref{perturb}) measures the minimal perturbation distance (in terms of the $\ell_1$-norm) that is needed to change the label of $\bx$. Its optimal value evaluates the robustness of a given classifier in terms of the perturbation magnitude. The more robust the classifier, the larger the required perturbation to switch the label, and thus the larger the optimal value. We solve Problem (\ref{perturb}) for every test point $\bx$ and any $j \neq k$, and take the minimum of the optimal values to be the {\em Minimal Perturbation Distance (MPD)} of the classifier.
	
	We experiment with two types of $\bB^*$: 
	\begin{enumerate}
		\item \textbf{type-1}: each column of $\bB^*$, with probability 0.4, is generated from a multivariate normal distribution, and with probability 0.6, it is generated as a sparse vector where only one randomly picked element is set to nonzero; 
		\item \textbf{type-2}: we first generate a $\bB^*$ that is normalized to have an equal row absolute sum as before, and then set each column to a sparse vector with probability 0.6.
	\end{enumerate}
	
	To test the robustness of various methods, we inject outliers to the training datasets. The predictors of the outliers have the same distribution as the clean samples, but their label vector is generated from a 
	multinomial distribution with probabilities specified by the softmax normalization of $1/(\bx'\bB^*+\boldsymbol{\eta})$, where $\boldsymbol{\eta} \sim \scrN(\mathbf{0},\mathbf{I}_K)$. The test set does not contain any outlier. 
	
	Tables~\ref{tab:mlg-3} and \ref{tab:mlg-4} show the average performance of various models over 20 repetitions under different experimental settings. For type-1 $\bB^*$, the MLG-1S (r=2) achieves the highest classification accuracy and the largest MPD, while for type-2 $\bB^*$, the MLG-1S (r=1) excels. Notice that in both cases, the DRO-MLG models lose to vanilla MLG in terms of the log-loss. The reason is that vanilla MLG focuses solely on minimizing the sample average log-loss, while the MLG-SR and MLG-1S models balance between maintaining a low log-loss and achieving a high robustness level. By allowing for a slightly larger log-loss, the DRO-MLG models achieve a considerably higher MPD and a higher classification accuracy. They improve the CCR by 1\% -- 5\% and 5\% -- 6\%, the MPD by 67\% -- 400\% and 50\% -- 650\% in Tables~\ref{tab:mlg-3} and \ref{tab:mlg-4}, respectively. Compared to ridge MLG, they improve the log-loss by 6\% and 5\% in the two tables. 
	%Furthermore, compared to ridge MLG, we improve the CCR by 6\%, the log-loss by 5\%, and the MPD by 50\% in Table \ref{tab:mlg-4}.
	
	\begin{table}[h]
		\caption{The out-of-sample performance of different MLG models trained on datasets with 30\% outliers and type-1 $\bB^*$.} 
		\label{tab:mlg-3}
		\begin{center}
			\resizebox{\linewidth}{!}{		\begin{tabular} {l| c c c c}
					\hline
					& CCR                   & Log-loss                 & CVaR                      & MPD          \\ \hline
					MLG-SR (r=1) & 0.81 (0.05)           &  0.65 (0.09)              & 0.97 (0.10)              &0.04 (0.15)           \\ 
					MLG-1S (r=1) & 0.81 (0.05)          & 0.65 (0.08)               & 0.99 (0.10)              & 0.04 (0.13)           \\ 
					MLG-SR (r=2) & 0.81 (0.05)          & 0.64 (0.04)               & 0.94 (0.11)              & 0.07 (0.02)  \\ 
					MLG-1S (r=2) & 0.82 (0.05)              & 0.66 (0.07)               & 0.95 (0.14)          & 0.10 (0.03)\\ 
					Vanilla MLG  & 0.78 (0.07)          & 0.61 (0.08)               & 0.94 (0.17)             & 0.02 (0.01)  \\ 
					Ridge MLG    & 0.81 (0.04)          & 0.68 (0.08)               & 0.95 (0.13)              & 0.06 (0.05) \\  \hline
				\end{tabular}}
			\end{center}
		\end{table}
		
		\begin{table}[h]
			\caption{The out-of-sample performance of different MLG models trained on datasets with 40\% outliers and type-2 $\bB^*$.} 
			\label{tab:mlg-4}
			\begin{center}
				\resizebox{\linewidth}{!}{		\begin{tabular} {l| c c c c}
						\hline
						& CCR                   & Log-loss                 & CVaR                       & MPD          \\ \hline
						MLG-SR (r=1) &  0.69 (0.14)          & 0.80 (0.06)              & 1.10 (0.17)              & 0.01 (0.007)          \\ 
						MLG-1S (r=1) &  0.69 (0.14)         & 0.82 (0.06)              & 1.11 (0.17)              & 0.03 (0.02)          \\ 
						MLG-SR (r=2) &  0.66 (0.13)          & 0.84 (0.05)              & 1.08 (0.12)              & 0.02 (0.01)               \\ 
						MLG-1S (r=2) &  0.68 (0.13)          & 0.82 (0.05)              & 1.10 (0.13)              & 0.02 (0.02)             \\ 
						Vanilla MLG  &  0.66 (0.06)         & 0.79 (0.06)             & 1.13 (0.14)              & 0.004 (0.002)               \\ 
						Ridge MLG    &  0.65 (0.09)         & 0.84 (0.04)              & 1.10 (0.09)             & 0.02 (0.01)             \\  \hline
					\end{tabular}}
				\end{center}
			\end{table}
			
			\section{Summary} \label{sec:5-5}
			In this section, we developed a {\em Distributionally Robust Optimization (DRO)} based approach under the Wasserstein metric
			to robustify {\em Multi-output Linear Regression (MLR)} and {\em Multiclass Logistic Regression (MLG)}, leading to matrix-norm regularized formulations that establish a connection between robustness and regularization in the multi-output scenario. We
			established out-of-sample performance guarantees for the solutions to the DRO-MLR and DRO-MLG extracts, illustrating the role of the regularizer in controlling the out-of-sample prediction error. 
			We provided empirical evidence showing that the DRO-MLR and DRO-MLG models achieve a comparable (slightly better) out-of-sample predictive performance to others, but a significantly higher robustness to outliers.

\chapter{Optimal Decision Making via Regression Informed K-NN} \label{ch:presp}
In this section, we will develop a prediction-based prescriptive model for optimal decision
making that $(i)$ predicts the outcome under each possible action using a robust
nonlinear model, and $(ii)$ adopts a randomized prescriptive policy determined by the
predicted outcomes. The predictive model combines the Wasserstein DRO regression with the
{\em K-Nearest Neighbors (K-NN)} regression that helps to capture the
nonlinearity embedded in the data. We apply the
proposed methodology in making recommendations for medical prescriptions, using
a diabetes and a hypertension dataset extracted from the Electronic Health
Records (EHRs) of a major safety-net hospital in New England. 

\section{The Problem and Related Work}
Suppose we are given a set of $M$ actions, and our goal is to choose $m \in \lb M \rb$ such that the future outcome $y$ is optimized. We are interested in finding the optimal decision with the aid of auxiliary data $\bx \in \mbb{R}^p$ that is concurrently observed, and correlated with the uncertain outcome $y$.
A main challenge with learning from observational data lies in the lack of counterfactual information. One solution is to predict the effects of counterfactual policies by learning an action-dependent {\em predictive model} that groups the training samples based on their actions, and fits a model in each group between the outcome $y$ and the feature $\bx$. The predictions from this composite model can be used to determine the optimal action to take. The performance of the prescribed decision hinges on the quality of the predictive model. We have observed that $(i)$ there is often significant ``noise'' in the data caused by recording errors, missing values, and large variability across individuals, and $(ii)$ the underlying relationship we try to learn is usually nonlinear and its parametric form is not known a priori. To deal with these issues, a nonparametric robust learning procedure is in need. 

Motivated by the observation that individuals with similar features $\bx$ would have
similar outcomes $y$ if they were to take the same action, we propose a predictive
model that makes predictions based on the outcomes of similar individuals -- to be
called {\em neighbors} -- in each group of the training set. It is a nonlinear and nonparametric estimator which constructs locally linear (constant) curves based on the similarity between individuals. To find reasonable neighbors, we need to accurately identify the set of features that are correlated with the outcome. We use the Wasserstein DRO regression for this task in consideration of the noise that could potentially bias the estimation. Our prescriptive methodology is established on the basis of a regression informed {\em K-Nearest Neighbors (K-NN)} model \citep{altman1992introduction} that evaluates the importance of features through Wasserstein DRO regression, and estimates the outcome by averaging over the neighbors identified by a regression coefficients-weighted distance metric. 

Our framework uses both parametric (Wasserstein DRO regression) and nonparametric
(K-NN) predictive models, producing robust predictions immunized against significant noise and capturing the underlying nonlinearity by utilizing the information of neighbors.
It is more information-efficient and more interpretable than the vanilla K-NN. We
then develop a randomized prescriptive policy that chooses each action $m$ with
probability $e^{-\xi \hat{y}_m(\bx)}/\sum_{j=1}^M e^{-\xi \hat{y}_j(\bx)}$, for some
pre-specified positive constant $\xi$, where $\hat{y}_m(\bx)$ is the predicted future
outcome for $\bx$ under action $m\in \lb M \rb$. We show that this randomized strategy leads to a nearly optimal future outcome by an appropriate choice of $\xi$.

In recent years there has been an emerging interest in combining ideas from machine learning with operations research to develop a framework that uses data to prescribe optimal decisions \citep{bertsimas2014predictive, den2016bridging, bravo2017discovering}. Current research focus has been on applying machine learning methodologies to predict the counterfactuals, based on which optimal decisions can be made. Local learning methods such as K-Nearest Neighbors \citep{altman1992introduction}, LOESS (LOcally Estimated Scatterplot Smoothing) \citep{cleveland1988locally}, CART (Classification And Regression Trees) \citep{breiman2017classification}, and Random Forests \citep{breiman2001random}, have been studied in\cite{bertsimas2014predictive, bertsimas2017personalized, bertsimas2018optimal, dunn2018optimal, biggs2018optimizing}. Extensions to continuous and multi-dimensional decision spaces with observational data were considered in \cite{bertsimas2018optimization}. To prevent overfitting, \cite{bertsimas2017bootstrap} proposed two robust prescriptive methods based on Nadaraya-Watson and nearest-neighbors learning. 
Deviating from such a predict-optimize paradigm,  
\cite{bastani2020online} presented a new bandit algorithm based on the LASSO to learn a model of decision rewards conditional on individual-specific covariates. 

Our problem is closely related to contextual bandits \citep{chu2011contextual, agarwal2014taming, slivkins2014contextual, wu2015algorithms}, where an agent learns a sequence of decisions conditional on the contexts with the aim of maximizing its cumulative reward. It has recently found applications in learning personalized treatment of long-term diseases from mobile health data \citep{tewari2017ads, xia2018price, zhu2018robust}. However, we learn the interaction between the context and rewards in each action group across similar individuals, not over the history of the same individual as in contextual bandits. A contextual bandits framework is most suitable for learning sequential strategies through repeated interactions with the environment, which requires a substantial amount of historical data for exploring the reward function and exploiting the promising actions. In contrast, our method does not require the availability of historical data, but instead learns the payoff function from similar individuals. This can be viewed as a different type of exploration, i.e., when little information can be acquired for the past states of an individual, investigating the behavior of similar subjects may be beneficial. This is essential for learning from the Electronic Health Records (EHRs), where rapid and continuous collection of patient data is not possible. We may observe a very short treatment history for some patients, and the lag between patient visits is usually large.

Our method is similar to K-NN regression with an OLS-weighted metric used in \cite{bertsimas2017personalized} to learn the optimal treatment for type-2 diabetic patients. The key differences lie in that: $(i)$ we adopt a robustified regression procedure that is immunized against outliers and is thus more stable and reliable; $(ii)$ we propose a randomized prescriptive policy that adds robustness to the methodology whereas \cite{bertsimas2017personalized} deterministically prescribed the treatment with the best predicted outcome; $(iii)$ we establish theoretical guarantees on the quality of the predictions and the prescribed actions, and $(iv)$ the prescriptive rule in \cite{bertsimas2017personalized} was activated when the improvement of the recommended treatment over the standard of care exceeded a certain threshold, whereas our method looks into the improvement over the previous regimen. This distinction makes our algorithm applicable in the scenario where the standard of care is unknown or ambiguous. Further, we derive a closed-form expression for the threshold level, which greatly improves the computational efficiency compared to \cite{bertsimas2017personalized}, where a threshold was selected by cross-validation. 

The remainder of this section is organized as follows. In Section~\ref{sec:4-2}, we
introduce the robust nonlinear predictive model and present the performance guarantees on its predictive power. Section~\ref{sec:4-3}
develops the randomized prescriptive policy and proves its optimality in terms of the expected true outcome. The numerical experimental results are presented in
Section~\ref{sec:4-4}. We conclude in Section~\ref{sec:4-5}.

\section{Robust Nonlinear Predictive Model} \label{sec:4-2}
Given a feature vector $\bx \in \mbb{R}^p$, and a set of $M$ available actions, our goal is to predict the future outcome $y_m(\bx)$ under each possible action $m \in \lb M \rb$.
Assume the following relationship between the feature and outcome:
\begin{equation*} 
y_m = \bx_m'\bbeta^*_m + h_m(\bx_m) + \epsilon_m,
\end{equation*}
where $(\bx_m, y_m)$ represents the feature-outcome pair of an individual taking action $m$; $\bbeta^*_m$ is the coefficient that captures the linear trend; $h_m(\cdot)$ is a nonlinear function (whose form is unknown) describing the nonlinear fluctuation in $y_m$, and $\epsilon_m$ is the noise term with zero mean and standard deviation $\eta_m$ that expresses the intrinsic randomness of $y_m$ and is assumed to be independent of $\bx_m$.

Suppose for each $m \in \lb M \rb$, we observe $N_m$ independently and identically distributed (i.i.d.) training samples $(\bx_{mi}, y_{mi}), i \in \lb N_m \rb$, that take action $m$. To estimate $\bbeta^*_m$, we adopt the $\ell_2$-norm induced Wasserstein DRO formulation. A robust model could lead to an improved out-of-sample performance, and accommodate the potential nonlinearity that is not explicitly revealed by the linear coefficient $\bbeta^*_m$, thus, resulting in a more accurate assessment of the features.
Solving the Wasserstein DRO regression model
gives us a robust estimator of the linear regression coefficient
$\bbeta_m^*$, which we denote by $\hat{\bbeta}_m \triangleq
(\hat{\beta}_{m1}, \ldots, \hat{\beta}_{mp})$. The elements of $\hat{\bbeta}_m$
measure the relative significance of the predictors in determining the outcome
$y_m$. We feed the estimator into the nonlinear non-parametric K-NN regression model,
by considering the following $\hat{\bbeta}_m$-weighted metric:
\begin{equation} \label{wtd}
\|\bx - \bx_{mi}\|_{\hat{\bW}_m} = \sqrt{(\bx-\bx_{mi})'\hat{\bW}_m (\bx-\bx_{mi})},
\end{equation} 
where $\hat{\bW}_m = \text{diag}(\hat{\bbeta}_{m}^2)$, and $\hat{\bbeta}_{m}^2 = (\hat{\beta}_{m1}^2, \ldots, \hat{\beta}_{mp}^2)$. For a new test sample $\bx$, within each action group $m$, we
find its $K_m$ nearest neighbors using the weighted distance function (\ref{wtd}). The
predicted future outcome for $\bx$ under action $m$, denoted by $\hat{y}_m(\bx)$, is
computed by 
\begin{equation} \label{knn} 
\hat{y}_m(\bx) = \frac{1}{K_m}\sum_{i=1}^{K_m} y_{m(i)},
\end{equation}
where $y_{m(i)}$ is the outcome of the $i$-th closest individual to $\bx$ in the
training set who takes action $m$. In essence, we compute a K-NN estimate of the future outcome
by using the regression weighted distance function, which can be viewed as a locally
smoothed estimator in the neighborhood of $\bx$. Notice that due to (\ref{wtd}), the
nearest neighbors are similar to $\bx$ in the features that are most predictive of
the outcome. Therefore, their corresponding response values should serve as a good
approximation for the future outcome of $\bx$.

We next show that (\ref{knn}) provides a good prediction in the sense of {\em Mean
	Squared Error (MSE)}. The bias-variance decomposition implies the following:
\begin{equation} \label{vardecomp}
\begin{aligned}
& \ \quad \text{MSE}\bigl(\hat{y}_m(\bx) \bigl| \bx, \bx_{mi}, i \in \lb N_m \rb \bigr) \\
& \triangleq \mbb{E}\Bigl[\bigl(\hat{y}_m(\bx) - y_m(\bx)\bigr)^2 \Bigl| \bx,
\bx_{mi}, i \in \lb N_m \rb \Bigr] \\
& = \mbb{E}\biggl[\Bigl(\frac{1}{K_m}\sum_{j=1}^{K_m} \bigl(\bx_{m(j)}'\bbeta_m^* +
h_m(\bx_{m(j)}) + \epsilon_{m(j)}\bigr) \\
& \qquad \quad  - \bigl(\bx'\bbeta_m^* + h_m(\bx) +
\epsilon_m\bigr)\Bigr)^2 \Bigl| \bx, \bx_{mi}, i \in \lb N_m \rb \biggr]\\ 
& = \Bigl(\bx'\bbeta_m^* + h_m(\bx) - \frac{1}{K_m}\sum\limits_{i=1}^{K_m} \bigl(\bx_{m(i)}'\bbeta_m^* + h_m(\bx_{m(i)})\bigr)\Bigr)^2 + \frac{\eta_m^2}{K_m} + \eta_m^2 \\
& = \Bigl(\frac{1}{K_m}\sum\limits_{i=1}^{K_m} \bigl((\bx - \bx_{m(i)})'\bbeta_m^* +
h_m(\bx) - h_m(\bx_{m(i)})\bigr)\Bigr)^2 + \frac{\eta_m^2}{K_m} + \eta_m^2, \\ 
\end{aligned}
\end{equation}
where $y_m(\bx)$ is the {\em true} future outcome for $\bx$ under action $m$, and
$\bx_{m(i)}$, $\epsilon_{m(i)}$ are the feature vector and noise term corresponding
to the $i$-th closest sample to $\bx$ within group $m$, respectively. For each $m \in
\lb M \rb$, we aim at providing a probabilistic bound for (\ref{vardecomp}) w.r.t. the
measure of the $N_m$ training samples. According to (\ref{vardecomp}), for the MSE to
be small the following three conditions suffice: 
\begin{enumerate}
	\item $\|\bbeta_m^* - \hat{\bbeta}_m\|_2$ is small;
	\item $\|\bx - \bx_{m(i)}\|_{\hat{\bW}_m}$ is small for
	$i \in \lb K_m \rb$;
	\item $h_m(\bx) - h_m(\bx_{m(i)})$ is small for $i \in \lb K_m \rb$.
\end{enumerate}
In other words, to ensure an accurate prediction of the outcome, we require an accurate
estimate of the linear trend and a smooth nonlinear fluctuation, with the selected
neighbors close enough to $\bx$. An upper bound for the MSE follows from bounding
these three quantities. We note that Theorem~\ref{estthm} provides an upper bound on
the estimation bias $\|\bbeta_m^* - \hat{\bbeta}_m\|_2$ in a linear model. If we view
the nonlinear term $h_m(\bx_m)$ as part of the noise, then the bound provided by
Theorem~\ref{estthm} applies to our case. The increased variance of the noise (due to
$h_m(\bx_m)$) is reflected in the eigenvalues of the covariance matrix, which play a
role in the estimation error bound. We provide a simplified version of
Theorem~\ref{estthm} in the following theorem. 
\begin{thm} \label{prop:beta}
	Under Assumptions~\ref{a1}, \ref{a2}, \ref{RE}, \ref{adm},
	\ref{subgaussian}, \ref{eigen}, when the sample size $N_m \ge n_m$, with probability at least $\delta_m$,
	$$\|\bbeta_m^* - \hat{\bbeta}_m\|_2 \le \tau_m.$$
\end{thm}

We next show that the distance between $\bx$ and its $K_m$ nearest neighbors
$\bx_{m(i)}$ could be upper bounded probabilistically. All predictors are assumed to
be centered, and independent from each other.  In Theorem~\ref{prop:mu} we present a lower bound for $\mbb{P}(\|\bx -
\bx_{m(i)}\|_{\bW} \le \bar{w}_m, i\in \lb K_m \rb)$, for any positive definite diagonal
matrix $\bW$.

\begin{thm} \label{prop:mu}
	Suppose we are given $N_m$ i.i.d. samples $(\bx_{mi}, y_{mi})$, $i\in \lb N_m \rb$, drawn
	from some unknown probability distribution with finite fourth moment. Every
	$\bx_{mi}$ has independent, centered coordinates:
	\[ 
	\mbb{E}(\bx_{mi}) = \mathbf{0},  \quad \text{cov}(\bx_{mi}) =
	\text{diag}\Bigl(\sigma_{m1}^2, \ldots, \sigma_{mp}^2\Bigr), \quad \forall i\in \lb N_m \rb.
	\]
	For a fixed predictor $\bx$, and for any given positive definite diagonal matrix $\bW \in
	\mbb{R}^{p \times p}$ with diagonal elements $w_j$, $j\in \lb p \rb$, and $|w_j| \le
	\bar{B}^2$, suppose:
	\[ 
	|(x_{mij}-x_j)^2-(\sigma_{mj}^2 + x_j^2)| \le T_{m},\  \text{a.s.},\ \forall i\in
	\lb N_m \rb,\ j\in \lb p \rb,
	\] 
	where $x_{mij}, x_j$ are the $j$-th components of $\bx_{mi}$ and $\bx$,
	respectively. Under the condition that $\bar{w}_m^2 > \bar{B}^2\sum_{j=1}^p
	(\sigma_{mj}^2 + x_j^2)$, with probability at least $1 - I_{1-p_{m0}}(N_m-K_m+1,
	K_m)$,
	\begin{equation*}
	\|\bx - \bx_{m(i)}\|_{\bW} \le \bar{w}_m, \qquad  i\in \lb K_m \rb,
	\end{equation*}
	where $g(u) = (1+u)\log(1+u)-u$, 
	\begin{multline*}
	I_{1-p_{m0}}(N_m - K_m+1, K_m) \triangleq {} \\ (N_m-K_m+1) \binom{N_m}{K_m-1}
	\int_{0}^{1-p_{m0}} t^{N_m-K_m}(1-t)^{K_m-1}dt,
	\end{multline*}
	\begin{equation*}
	p_{m0} = 1 -
	\exp\biggl(-\frac{\sigma_m^2}{T_{m}^2}g\Bigl(\frac{T_{m}\Bigl(\bar{w}_m^2/\bar{B}^2 -
		\sum_j(\sigma_{mj}^2 + x_j^2)\Bigr)}{\sigma_m^2}\Bigr)\biggr),
	\end{equation*} 
	and,
	\begin{equation*}
	\sigma_m^2 = \sum_{j=1}^p \text{var}\Bigl((x_{mij}-x_j)^2\Bigr).
	\end{equation*}
\end{thm}  

\begin{proof}
	To simplify the notation, we will omit the subscript $m$ in all proofs, e.g., using
	$\bx_i$ and $\bx_{(i)}$ for $\bx_{mi}$ and $\bx_{m(i)}$, respectively, and $N$ for
	$N_m$. Define the event $\scrA_i:= \{\|\bx_i - \bx\|_{\bar{B}^2\bI} \le
	\bar{w}\}$. As long as we can calculate the probability that at least $K$ of
	$\scrA_i$, $i\in \lb N \rb$, occur, we are able to provide a lower bound on $\mbb{P}(\|\bx
	- \bx_{(i)}\|_{\bW} \le \bar{w}, \ i\in \lb K \rb)$. Note that given $\bx$, $\scrA_i$,
	$i\in \lb N \rb$, are independent and equiprobable, since $\bx_i$, $i\in \lb N \rb$, are
	i.i.d. Based on Bennett's inequality \citep{RV17}, we have:
	\begin{equation*}
	\begin{aligned}
	\mbb{P}(\scrA_i) & = \mbb{P}(\|\bx_i - \bx\|_{\bar{B}^2\bI}^2 \le \bar{w}^2) \\
	& = \mbb{P}\Bigl(\bar{B}^2(x_{i1}-x_1)^2+\ldots+\bar{B}^2(x_{ip}-x_p)^2 \le \bar{w}^2\Bigr)\\
	& = \mbb{P}(t_1+\ldots+t_p \le \bar{w}^2/\bar{B}^2)\\
	& = \mbb{P}\biggl(\sum_j\Bigl(t_j - (\sigma_j^2 + x_j^2)\Bigr) \le \bar{w}^2/\bar{B}^2 - \sum_j(\sigma_j^2 + x_j^2)\biggr)\\
	& \ge 1 - \exp\biggl(-\frac{\sigma^2}{T^2}g\Bigl(\frac{T\Bigl(\bar{w}^2/\bar{B}^2 - \sum_j(\sigma_j^2 + x_j^2)\Bigr)}{\sigma^2}\Bigr)\biggr) \\
	& \triangleq  p_{0},
	\end{aligned}	
	\end{equation*}
	where $t_j = (x_{ij}-x_j)^2$, $j\in \lb p \rb$; $\sigma^2 = \sum_j \text{var}(t_j)$. In the
	above derivation, we used the fact that $t_j$, $j\in \lb p \rb$, are independent, and $|t_j
	- \mbb{E}[t_j]| \le T, \ \text{a.s.}, \ \forall j$.
	
	Given the lower bound for $\mbb{P}(\scrA_i)$, we can derive a lower bound for the
	probability that exactly $K$ of $\scrA_i$, $i\in \lb N \rb$, occur. For a given $\bx$,
	$\scrA_i$, $i\in \lb N \rb$, are independent, and thus,
	\begin{equation*}
	\begin{aligned}
	& \quad \ \mbb{P}(\|\bx - \bx_{(i)}\|_{\bW} \le \bar{w}, \ i\in \lb K \rb)  \\
	& \ge \mbb{P}(\text{at
		least $K$ of $\scrA_i, i\in \lb N \rb$ occur}) \\ 
	& = \sum_{k = K}^N\binom{N}{k} \Bigl(\mbb{P}(\scrA_i)\Bigr)^{k} \Bigl(1-\mbb{P}(\scrA_i)\Bigr)^{N-k}\\
	& \ge \sum_{k=K}^N \binom{N}{k} p_{0}^{k} (1-p_{0})^{N-k} \\
	& = 1 - I_{1-p_0}(N-K+1, K),
	\end{aligned}
	\end{equation*}
	where $I_{1-p_0}(N - K+1, K)$ is the {\em regularized incomplete beta
		function} defined as: 
	\[ 
	I_{1-p_0}(N - K+1, K) \triangleq (N-K+1) \binom{N}{K-1} \int_{0}^{1-p_0}
	t^{N-K}(1-t)^{K-1}dt.
	\] 
\end{proof}

Note that $p_{m0}$ is nonnegative, due to the assumption that $\bar{w}_m^2 >
\bar{B}^2\sum_j (\sigma_{mj}^2 + x_j^2)$, and the non-decreasing property of the
function $g(\cdot)$ when its argument is non-negative. We also note that a lower
bound on $\mbb{P}(\|\bx - \bx_{m(i)}\|_2 \le \bar{w}_m,\ i\in \lb K_m \rb)$ can be obtained
by setting $\bar{B} = 1$.

By now we have shown results on the accuracy of $\hat{\bbeta}_m$ and the similarity
between $\bx$ and its neighbors. Notice that for a Lipschitz continuous function
$h_m(\cdot)$ with a Lipschitz constant $L_m$, the difference between $h_m(\bx)$ and
$h_m(\bx_{m(i)})$ can be bounded by $L_m\|\bx - \bx_{m(i)}\|_2$. With these results
we are ready to bound the MSE of $\hat{y}_m(\bx)$.

\begin{thm} \label{thm:main}
	Suppose we are given $N_m$ i.i.d. copies of $(\bx_m, y_m)$, denoted by $(\bx_{mi},
	y_{mi}), i\in \lb N_m \rb$, where $\bx_m$ has independent, centered coordinates, and
	$\text{cov}(\bx_m) = \text{diag}(\sigma_{m1}^2, \ldots, \sigma_{mp}^2).$ We are given
	a fixed predictor $\bx = (x_1, \ldots, x_p)$, a scalar $\bar{w}_m$, and we assume:
	\begin{enumerate}
		\item $h_m(\cdot)$ is Lipschitz continuous with a Lipschitz constant $L_m$ on the
		metric spaces $(\scrX_m, \|\cdot\|_2)$ and $(\scrY_m, |\cdot|)$, where $\scrX_m,
		\scrY_m$ are the domain and codomain of $h_m(\cdot)$, respectively. 
		\item $\bar{w}_m^2 > \bar{B}_m^2\sum_{j=1}^p (\sigma_{mj}^2 + x_j^2)$, where $\bar{B}_m$ is the upper bound on $\|(-\bbeta_m, 1)\|_2$ for any feasible $\bbeta_m$ to (\ref{qcp}).
		\item $|(x_{mij}-x_j)^2-(\sigma_{mj}^2 + x_j^2)|$ is upper bounded a.s. under the probability measure $\mbb{P}^*_{\scrX_m}$ for
		any $i, j$, where $x_{mij}$ is the $j$-th component of $\bx_{mi}$, and $\mbb{P}^*_{\scrX_m}$ is the underlying true probability distribution of $\bx_m$.
		\item The coordinates of any feasible solution to (\ref{qcp}) have absolute values
		greater than or equal to some positive number $b_m$ (dense estimators). 
	\end{enumerate}	
	Under Assumptions~\ref{a1}, \ref{a2}, \ref{RE}, \ref{adm},
	\ref{subgaussian}, \ref{eigen},
	when $N_m \ge n_m$, with probability at least $\delta_m - I_{1-p_{m0}}(N_m-K_m+1,
	K_m)$ w.r.t. the measure of samples,
	\begin{multline} \label{mse}
	\mbb{E}\Bigl[(\hat{y}_m(\bx) - y_m(\bx))^2 \Bigl| \bx, \bx_{mi}, i \in \lb N_m \rb\Bigr] 
	\le {}\\
	\biggl(\frac{\bar{w}_m \tau_m}{b_m} + \sqrt{p}\bar{w}_m +
	\frac{L_m\bar{w}_m}{\bar{B}_m}\biggr)^2 + \frac{\eta_m^2}{K_m} + \eta_m^2, 
	\end{multline}
	and for any $a \ge (\bar{w}_m \tau_m/b_m + \sqrt{p}\bar{w}_m +
	L_m\bar{w}_m/\bar{B}_m)^2 + \eta_m^2/K_m + \eta_m^2$, 
	\begin{multline} \label{prob-mse} 
	\mbb{P}\Bigl(\bigl(\hat{y}_m(\bx) - y_m(\bx)\bigr)^2 \ge a \Bigl| \bx, \bx_{mi}, i\in
	\lb N_m \rb \Bigr) \le {}\\
	\frac{\Bigl(\frac{\bar{w}_m \tau_m}{b_m} + \sqrt{p}\bar{w}_m + \frac{L_m\bar{w}_m}{\bar{B}_m}\Bigr)^2 + \frac{\eta_m^2}{K_m} + \eta_m^2}{a},
	\end{multline}
	where all parameters are set in the same way as in Theorems~\ref{prop:beta} and
	\ref{prop:mu}.
\end{thm}

\begin{proof}
	We omit the subscript $m$ for simplicity. By Theorems~\ref{prop:beta} and \ref{prop:mu}, we know that,
	\begin{equation*}
	\begin{aligned}
	|(\bx - \bx_{(i)})'(\bbeta^* - \hat{\bbeta})| & = |(\bx - \bx_{(i)})'\hat{\bW}^{\frac{1}{2}}\hat{\bW}^{-\frac{1}{2}}(\bbeta^* - \hat{\bbeta})| \\
	& \le \|(\bx - \bx_{(i)})'\hat{\bW}^{\frac{1}{2}}\|_2 \|\hat{\bW}^{-\frac{1}{2}}(\bbeta^* - \hat{\bbeta})\|_2 \\
	& \le  \frac{\bar{w} \tau}{b},
	\end{aligned}
	\end{equation*}	
	where the second inequality used the fact that $\|\hat{\bW}^{-\frac{1}{2}}(\bbeta^* - \hat{\bbeta})\|_2 \le  \tau/b$ if $\|\bbeta^* - \hat{\bbeta}\|_2 \le \tau$, which can be verified by the Courant-Fischer Theorem, and the fact that $\hat{\bW}$ is diagonal with elements $\hat{\beta}_j^2, j \in \lb p \rb$, and $|\hat{\beta}_j| \ge b$.
	Based on the inequality $\Bigl(\sum_{i=1}^n a_i\Bigr)^2 \le n \Bigl(\sum_{i=1}^n a_i^2\Bigr)$, we know:
	\begin{equation*}
	\begin{aligned}
	|(\bx - \bx_{(i)})'\hat{\bbeta}| & = \Bigl|\sum_{j=1}^{p} \hat{\beta}_j (\bx - \bx_{(i)})_j\Bigr| \\
	& \le \sqrt{p \sum_{j=1}^{p} \Bigl(\hat{\beta}_j (\bx - \bx_{(i)})_j\Bigr)^2} \\ 
	& = \sqrt{p (\bx - \bx_{(i)})'\hat{\bW} (\bx - \bx_{(i)})} \\
	& \le \sqrt{p}\bar{w}.
	\end{aligned}
	\end{equation*}
	Therefore,
	\begin{equation*}
	\begin{aligned}
	|(\bx - \bx_{(i)})'\bbeta^*| & = |(\bx - \bx_{(i)})'(\bbeta^* - \hat{\bbeta}) + (\bx - \bx_{(i)})'\hat{\bbeta}|\\
	& \le |(\bx - \bx_{(i)})'(\bbeta^* - \hat{\bbeta})| + |(\bx - \bx_{(i)})'\hat{\bbeta}| \\
	& \le \frac{\bar{w} \tau}{b} + \sqrt{p}\bar{w}.
	\end{aligned}
	\end{equation*}  	
	Thus, for a given $\bx$,
	\begin{equation*}
	\begin{aligned}
	& \quad \ \mbb{E}\Bigl[(\hat{y}(\bx) - y(\bx))^2 \Bigl| \bx, \bx_i\Bigr] \\
	& = \biggl(\frac{1}{K}\sum_{i=1}^{K} \bigl((\bx - \bx_{(i)})'\bbeta^* + h(\bx) - h(\bx_{(i)})\bigr)\biggr)^2 + \frac{\eta^2}{K} + \eta^2 \\
	& \le \biggl(\frac{1}{K}\sum_{i=1}^{K} \bigl(|(\bx - \bx_{(i)})'\bbeta^*| + |h(\bx) - h(\bx_{(i)})|\bigr)\biggr)^2 + \frac{\eta^2}{K} + \eta^2 \\
	& \le \biggl(\frac{\bar{w} \tau}{b} + \sqrt{p}\bar{w} + \frac{L\bar{w}}{\bar{B}}\biggr)^2 + \frac{\eta^2}{K} + \eta^2
	\end{aligned}
	\end{equation*}
	The probability bound can be easily derived using Markov's inequality.
\end{proof}

The expectation in (\ref{mse}) and the probability in (\ref{prob-mse})
are taken w.r.t.\ the measure of the noise $\epsilon_m$. Theorem~\ref{thm:main} essentially
says that for any given predictor $\bx$, with a high probability (w.r.t.\ the measure
of samples), the prediction from our model is close to the true future outcome.  The
prediction bias depends on the sample size, the variation in the predictors and
response, and the smoothness of the nonlinear fluctuation.

The dependence on $b_m$ in the upper bound provided by
(\ref{mse}) is due to the fact that $\hat{\bW}_m$ has diagonal elements
$\hat{\beta}_{mj}^2, j \in \lb p \rb$, which are assumed to be at least $b_m^2$. If we
multiply $\hat{\bW}_m$ by a very large number, the neighbor selection criterion is
not affected, since the relative significance of the predictors stays unchanged, but
the $b_m$ appearing in (\ref{mse}) would be replaced by a very large number,
diminishing the effect of the first term in the parentheses, at the price of
increasing $\bar{B}_m$ and $\bar{w}_m$, which in turn have an effect on the number of
neighbors needed. It might be interesting to explore this implicit trade-off and
optimize $\hat{\bW}_m$ to achieve the smallest MSE. 

\section{Prescriptive Policy Development} \label{sec:4-3}
We now proceed to develop the prescriptive policy with the aim of minimizing the future outcome. A natural idea is to take the action that yields the minimum predicted outcome. To allow for flexibility in exploring alternatives that have a comparable performance, and also to correct for potential prediction errors that might mislead the ranking of actions, we propose a randomized policy that prescribes each action with a probability inversely proportional to its exponentiated predicted outcome. It can be viewed as an offline Hedge algorithm \citep{hazan2016introduction} that increases the robustness of our method through exploration. 

Specifically, given an individual with a feature vector $\bx$, and her predicted future outcome under each action $m$, denoted by $\hat{y}_m(\bx)$, we 
consider a randomized policy that chooses action $m$ with probability 
$e^{-\xi \hat{y}_m(\bx)}/\sum_{j=1}^M e^{-\xi \hat{y}_j(\bx)}$,
with $\xi$ some pre-specified positive constant. The randomness in making decisions might hurt the interpretability of the model. But on the other hand, it presents a range of comparable options that can be assessed subjectively by the decision maker based on her expertise. As $\xi$ goes to infinity, the randomized policy will converge to a deterministic one which selects the action with the lowest predicted outcome. We next establish a related property of the randomized policy in terms of its expected {\em true} outcome. 
%which is done in the following theorem.

\begin{thm} \label{thm:random}
	Given any fixed predictor $\bx \in \mbb{R}^p$, denote its predicted and true future outcome under action $m$ by $\hat{y}_m(\bx)$ and $y_m(\bx)$, respectively. Assume that we adopt a randomized strategy that prescribes action $m$ with probability $e^{-\xi \hat{y}_m(\bx)}/\sum_{j=1}^M e^{-\xi \hat{y}_j(\bx)},$ for some $\xi \ge 0$. Assume $\hat{y}_m(\bx)$ and $y_m(\bx)$ are non-negative, $\forall m \in \lb M \rb$. The expected true outcome under this policy satisfies:
	\begin{multline} \label{eq:random}
	\sum_{m=1}^M \frac{e^{-\xi \hat{y}_m(\bx)}}{\sum_j e^{-\xi \hat{y}_j(\bx)}}  y_m(\bx)  \le  y_k(\bx)  + \bigg(\hat{y}_k(\bx)  - \frac{1}{M} \sum_{m=1}^M \hat{y}_m(\bx)\bigg)\\
	+ \xi \bigg( \frac{1}{M} \sum_{m=1}^M \hat{y}_m^2(\bx) + \sum_{m=1}^M \frac{e^{-\xi \hat{y}_m(\bx)}}{\sum_j e^{-\xi \hat{y}_j(\bx)}}  y_m^2(\bx)\bigg) + \frac{\log M}{\xi}, 
	\end{multline}
	for any $k \in \lb M \rb$.
\end{thm}

\begin{proof}
	The proof borrows ideas from Theorem 1.5 in \cite{hazan2016introduction}. 
	Define $W_m \triangleq e^{-\xi \hat{y}_m(\bx)}/\sum_{j=1}^M e^{-\xi \hat{y}_j(\bx)}$, and $\phi \triangleq \sum_{m=1}^M e^{-\xi \hat{y}_m(\bx)} e^{-\xi y_m(\bx)}$. Then,
	\begin{equation*}
	\begin{aligned}
	\phi & =  \Big(\sum_{j=1}^M e^{-\xi \hat{y}_j(\bx)}\Big)\sum_{m=1}^M W_m e^{-\xi y_m(\bx)} \\
	& \le \Big(\sum_{j=1}^M e^{-\xi \hat{y}_j(\bx)}\Big) \sum_{m=1}^M W_m \big(1-\xi y_m(\bx) + \xi^2 y_m^2(\bx)\big) \\
	& = \Big(\sum_{j=1}^M e^{-\xi \hat{y}_j(\bx)}\Big) \Big( 1 - \xi \sum_{m=1}^M W_m y_m(\bx) + \xi^2 \sum_{m=1}^M W_my_m^2(\bx)\Big) \\
	& \le \Big(\sum_{j=1}^M e^{-\xi \hat{y}_j(\bx)}\Big) e^{- \xi \sum_{m=1}^M W_my_m(\bx) + \xi^2 \sum_{m=1}^M W_m y_m^2(\bx)},
	\end{aligned}
	\end{equation*}
	where the first inequality uses the fact that for $x\ge 0$, $e^{-x} \le 1-x+x^2$, and the last inequality is due to the fact that $1+x \le e^{x}$. Next let us examine the sum of exponentials:
	\begin{equation*}
	\begin{aligned}
	\sum_{j=1}^M e^{-\xi \hat{y}_j(\bx)} & \le \sum_{j=1}^M \Big( 1 - \xi \hat{y}_j(\bx) + \xi^2\hat{y}^2_j(\bx) \Big) \\
	& = M \Big( 1 - \xi \frac{1}{M}\sum_{j=1}^M \hat{y}_j(\bx) + \xi^2 \frac{1}{M}\sum_{j=1}^M \hat{y}_j^2(\bx)\Big) \\
	& \le M e^{- \xi \frac{1}{M}\sum_{j=1}^M \hat{y}_j(\bx) + \xi^2 \frac{1}{M}\sum_{j=1}^M \hat{y}_j^2(\bx)}.
	\end{aligned}
	\end{equation*}
	On the other hand, for any $k \in \lb M \rb$,
	\begin{align} \label{eq1}
	e^{-\xi \hat{y}_k(\bx) -\xi y_k(\bx)}   \le & \phi \notag\\  
	\le & 
	M \exp\bigg\{-  \frac{\xi \sum_{j=1}^M \hat{y}_j(\bx)}{M} +
	\frac{\xi^2 \sum_{j=1}^M \hat{y}_j^2(\bx)}{M} \notag \\
	&  \qquad \qquad - \xi \sum\limits_{m=1}^M W_my_m(\bx) + \xi^2 \sum\limits_{m=1}^M
	W_m y_m^2(\bx)\bigg\}. 
	\end{align}
	Taking the logarithm on both sides of (\ref{eq1}) and dividing by $\xi$, we obtain
	\begin{multline*}
	\frac{1}{M} \sum_{m=1}^M \hat{y}_m(\bx)  +  \sum_{m=1}^M  \frac{e^{-\xi \hat{y}_m(\bx)}}{\sum_j e^{-\xi \hat{y}_j(\bx)}}  y_m(\bx)  \le  \  \hat{y}_k(\bx) + y_k(\bx) \\  
	+ \xi \bigg( \frac{1}{M} \sum_{m=1}^M \hat{y}_m^2(\bx) + \sum_{m=1}^M \frac{e^{-\xi \hat{y}_m(\bx)}}{\sum_j e^{-\xi \hat{y}_j(\bx)}}  y_m^2(\bx)\bigg) + \frac{\log M}{\xi}.
	\end{multline*}
\end{proof}

Theorem~\ref{thm:random} says that the expected {\em true} outcome of the randomized policy is no worse than the true outcome of any action $k$ plus two components, one accounting for the gap between the {\em predicted} outcome under $k$ and the average predicted outcome, and the other depending on the parameter $\xi$. Thinking about choosing $k = \arg \min_m y_m(\bx)$, if $\hat{y}_k(\bx)$ is below the average predicted outcome (which should be true if we have an accurate prediction), it follows from (\ref{eq:random}) that the randomized policy leads to a nearly optimal future outcome by an appropriate choice of $\xi$. 

In the medical applications, when determining the {\em future} prescription for a patient, we usually have access to some auxiliary information such as the {\em current} prescription that she is receiving, and her {\em current} lab results. In consideration of the health care costs and treatment transients, it is not desired to switch patients' treatments too frequently. We thus set a threshold level for the expected improvement in the outcome, below which the randomized strategy will be ``frozen'' and the current therapy will be continued. Specifically, 
\begin{multline*} 
m_{\text{f}}(\bx) \\ = 
\begin{cases}
\text{$m$, w.p. $\frac{e^{-\xi \hat{y}_m(\bx)}}{\sum_{j=1}^M e^{-\xi
			\hat{y}_j(\bx)}}$}, & \text{if $\sum\limits_{k} \frac{e^{-\xi \hat{y}_k(\bx)}}{\sum_j e^{-\xi \hat{y}_j(\bx)}}  \hat{y}_k(\bx) \le x_{\text{co}} - T(\bx)$}, \\
m_{\text{c}}(\bx), & \text{otherwise},
\end{cases}
\end{multline*}
where $m_{\text{f}}(\bx)$ and $m_{\text{c}}(\bx)$ are the future and current prescriptions for patient $\bx$, respectively; $x_{\text{co}}$ represents the current observed outcome (e.g., current blood pressure), which is assumed to be one of the components of $\bx$, and $T(\bx)$ is some threshold level which will be determined later. This prescriptive rule basically says that the randomized strategy will be activated only if the expected improvement relative to the current observed outcome is significant. 

\begin{thm} \label{thm:threshold}
	Assume that the distribution of the predicted outcome $\hat{y}_m(\bx)$ conditional on
	$\bx$, is sub-Gaussian, and its $\psi_2$-norm is equal to $\sqrt{2}C_m(\bx)$, for any
	$m \in [M]$ and any $\bx$. Given a small $0 < \bar{\epsilon} < 1$, to satisfy 
	\begin{equation*}
	\mbb{P}\left(\sum_k \frac{e^{-\xi \hat{y}_k(\bx)}}{\sum_j e^{-\xi \hat{y}_j(\bx)}}
	\hat{y}_k(\bx) > x_{\text{co}} - T(\bx)\right) \le \bar{\epsilon}, 
	\end{equation*}
	it suffices to set a threshold
	\begin{equation*}
	T(\bx) = \max\Bigl(0, \ \min_m \Bigl(x_{\text{co}} - \mu_{\hat{y}_{m}}(\bx) - \sqrt{-2 C_m^2(\bx)\log (\bar{\epsilon}/M)}\Bigr)\Bigr),
	\end{equation*}
	where $\mu_{\hat{y}_m}(\bx) = \mbb{E}[\hat{y}_m(\bx)|\bx]$.
\end{thm}

\begin{proof}
	By the sub-Gaussian assumption we have:
	\begin{equation} \label{tailprobr}
	\begin{aligned}
	& \quad \ \mbb{P}\Big(\sum_k \frac{e^{-\xi \hat{y}_k(\bx)}}{\sum_j e^{-\xi \hat{y}_j(\bx)}}  \hat{y}_k(\bx) > x_{\text{co}} - T(\bx)\Big)  \\
	& \le \mbb{P}\Bigl(\max\limits_k \hat{y}_k(\bx)> x_{\text{co}} - T(\bx)\Bigr) \\
	& =  \mbb{P}\Big(\bigcup\limits_k \big\{\hat{y}_k(\bx)> x_{\text{co}} - T(\bx)\big\}\Big) \\
	& \le \sum\limits_k \mbb{P} \big(\hat{y}_k(\bx)> x_{\text{co}} - T(\bx)\big) \\
	& \le \sum\limits_k \exp \left(-\frac{\bigl(x_{\text{co}} - T(\bx) - \mu_{\hat{y}_k}(\bx)\bigr)^2}{2 C_k^2(\bx)}\right). \\
	\end{aligned}
	\end{equation} 
	Note that the probability in (\ref{tailprobr}) is taken with respect to the measure
	of the training samples. We essentially want to find the largest threshold $T(\bx)$
	such that the probability of the expected improvement being less than $T(\bx)$ is
	small. Given a small $0 < \bar{\epsilon} < 1$ and due to (\ref{tailprobr}), to satisfy 
	\begin{equation*}
	\mbb{P}\left(\sum_k \frac{e^{-\xi \hat{y}_k(\bx)}}{\sum_j e^{-\xi \hat{y}_j(\bx)}}  \hat{y}_k(\bx) > x_{\text{co}} - T(\bx)\right) \le \bar{\epsilon},
	\end{equation*}
	it suffices to set:
	\begin{equation} \label{epsilonbound}
	\sum\limits_k \exp \left(-\frac{\bigl(x_{\text{co}} - T(\bx) - \mu_{\hat{y}_k}(\bx)\bigr)^2}{2 C_k^2(\bx)}\right) \le \bar{\epsilon}.
	\end{equation}
	A sufficient condition for (\ref{epsilonbound}) is:
	\begin{equation*}
	\exp \left((-\frac{\bigl(x_{\text{co}} - T(\bx) -
		\mu_{\hat{y}_m}(\bx)\bigr)^2}{2 C_m^2(\bx)}\right) \le
	\frac{\bar{\epsilon}}{M},\qquad  \forall m\in \lb M \rb, 
	\end{equation*}
	which yields that,
	\begin{equation} \label{T-bound}
	T(\bx) \le x_{\text{co}} - \mu_{\hat{y}_{m}}(\bx) - \sqrt{-2 C_m^2(\bx)\log
		(\bar{\epsilon}/M)}, \qquad \forall m\in \lb M \rb. 
	\end{equation}
	Given that $T(\bx)$ is non-negative, we set the largest possible threshold satisfying
	(\ref{T-bound}) to:
	\begin{equation*}
	T(\bx) = \max\Bigl(0, \ \min_m \Bigl(x_{\text{co}} - \mu_{\hat{y}_{m}}(\bx) -
	\sqrt{-2 C_m^2(\bx)\log (\bar{\epsilon}/M)}\Bigr)\Bigr). 
	\end{equation*}
	When using a deterministic policy ($\xi \to \infty$), for any $m\in \lb M \rb$, we have
	\begin{equation*} 
	\begin{aligned}
	\mbb{P}(\min\limits_m \hat{y}_m(\bx)> x_{\text{co}} - T(\bx)) & = \mbb{P}\Bigl(\bigcap \limits_m \big\{\hat{y}_m(\bx)> x_{\text{co}} - T(\bx)\big\}\Bigr) \\
	& \le  \mbb{P}(\hat{y}_m(\bx)> x_{\text{co}} - T(\bx)) \\
	& \le \exp \left(-\frac{\bigl(x_{\text{co}} - T(\bx) -
		\mu_{\hat{y}_m}(\bx)\bigr)^2}{2 C_m^2(\bx)}\right).\\
	\end{aligned}
	\end{equation*} 
	Similarly, to make
	\begin{equation*}
	\mbb{P}\bigl(\min\limits_m \hat{y}_m(\bx)> x_{\text{co}} - T(\bx)\bigr) \le \bar{\epsilon},
	\end{equation*}
	we set:
	\begin{equation*}
	T(\bx) = \max\Bigl(0, \ \min_m \Bigl(x_{\text{co}} - \mu_{\hat{y}_{m}}(\bx) -
	\sqrt{-2 C_m^2(\bx)\log \bar{\epsilon}}\Bigr)\Bigr), 
	\end{equation*}
	which establishes the desired result.
\end{proof}

Theorem~\ref{thm:threshold} finds the largest threshold $T(\bx)$ such that the
probability of the expected improvement being less than $T(\bx)$ is small. The
parameters $\mu_{\hat{y}_{m}}(\bx)$ and $C_m(\bx)$, for $m \in \lb M \rb$, can be
estimated by simulation through random sampling a subset of the training
examples. Algorithm~\ref{ms} provides the  details. 
\begin{algorithm}[h]
	\caption{Estimating the conditional mean and standard deviation of the predicted outcome.} \label{ms}
	\begin{algorithmic}
		\State {\textbf{Input:} a feature vector $\bx$; $a_m$: the number of subsamples used to compute $\hat{\bbeta}_m$, $a_m < N_m$; $d_m$: the number of repetitions.}	
		\For{$i = 1, \ldots, d_m$}
		\State {Randomly pick $a_m$ samples from group $m$, and use them to estimate a robust regression coefficient $\hat{\bbeta}_{m_i}$ through solving (\ref{qcp})}.
		\State {The future outcome for $\bx$ under action $m$ is predicted as $\hat{y}_{m_i}(\bx) = \bx'\hat{\bbeta}_{m_i}$}.
		\EndFor	
		\State {\textbf{Output:} Estimate the conditional mean of $\hat{y}_m(\bx)$ as: $$\mu_{\hat{y}_{m}}(\bx) = \frac{1}{d_m} \sum_{i=1}^{d_m}\hat{y}_{m_i}(\bx),$$
			and the conditional standard deviation as:
			\begin{equation*}
			C_m(\bx) = \sqrt{\frac{1}{d_m - 1} \sum_{i=1}^{d_m} \Bigl(\hat{y}_{m_i}(\bx) - \mu_{\hat{y}_{m}}(\bx)\Bigr)^2}.
			\end{equation*}}
	\end{algorithmic}
\end{algorithm}

\paragraph{A Special Case}
As $\xi \to \infty$, the randomized policy will assign probability $1$ to the action with the lowest predicted outcome, which is equivalent to the following deterministic policy:
\begin{equation*} 
m_{\text{f}}(\bx) = 
\begin{cases}
\arg \min \limits_{m} \hat{y}_m(\bx), & \text{if $\min\limits_m \hat{y}_m(\bx) \le x_{\text{co}} - T(\bx)$}, \\
m_{\text{c}}(\bx), & \text{otherwise}.
\end{cases}
\end{equation*}
A slight modification to the threshold level $T(\bx)$ is given as follows:
\begin{equation*}
T(\bx) = \max\Bigl(0, \ \min_m \Bigl(x_{\text{co}} - \mu_{\hat{y}_{m}}(\bx) - \sqrt{-2 C_m^2(\bx)\log \bar{\epsilon}}\Bigr)\Bigr).
\end{equation*}

\section{Developing Optimal Prescriptions for Patients} \label{sec:4-4}
In this section, we apply our method to develop optimal prescriptions for patients
with type-2 diabetes and hypertension.  The data used for the study come from the
Boston Medical Center -- the largest safety-net hospital in New England -- and
consist of {\em Electronic Health Records (EHRs)} containing the patients' medical
history in the period 1999--2014. The medical history of each patient includes
demographics, diagnoses, prescriptions, lab tests, and past admission records. We
build two datasets from the EHRs, one containing the medical records of patients with
type-2 diabetes and the other for patients with hypertension. For diabetic patients,
we want to determine the treatment (drug regimen) that leads to the lowest future
HbA\textsubscript{1c}~\footnote{HbA\textsubscript{1c} measures the percentage of
	glycosylated hemoglobin in the total amount of hemoglobin present in the blood. It
	reflects average blood glucose levels over the past 6--8 weeks. The
	normal range is below 5.7\%.} based on the medical histories, while for
hypertension patients, our goal is to find the treatment that minimizes the future
systolic blood pressure.~\footnote{Systolic blood pressure is the maximum arterial
	pressure during contraction of the left ventricle of the heart. It is measured in
	mmHg (millimeters of mercury) and the normal range is below 120.} 

\subsection{Description of the Datasets}
The patients that meet the following criteria are included in the diabetes dataset: 
\begin{itemize}
	\item Patients present in the system for at least 1 year;
	\item Received at least one blood glucose regulation agent, including injectable (e.g., insulin) and oral (e.g., metformin) drugs, etc., and had at least one medical record 100 days before this prescription;
	\item Had at least three measurements of HbA\textsubscript{1c} in the system; and,
	\item Were not diagnosed with type-1 diabetes.
\end{itemize}
Similarly, for the hypertension dataset, the patients that meet the following criteria are included:
\begin{itemize}
	\item Patients present in the system for at least 1 year;
	\item Received at least one type of cardiovascular medications, including ACE inhibitors, Angiotensin Receptor Blockers (ARB), calcium channel blockers, diuretics, $\alpha$-blockers and $\beta$-blockers, and had at least one medical record 10 days before this prescription.
	\item Had at least one recorded diagnosis of hypertension (corresponding to the ICD-9 diagnosis codes 401-405); 
	\item Had at least three measurements of the systolic blood pressure.
\end{itemize}
We have identified 11,230 patients for the diabetes dataset and 49,401 patients for the hypertension dataset. Each patient may have multiple entries in her/his medical record. We define the {\em line of therapy} as a time period (between 200 and 500 days) during which the combination of drugs prescribed to the patient does not change. Each line of therapy is characterized by a drug regimen which is defined as the combination of drugs prescribed to the patient within the first 200 days. The line of therapy intends to capture the period when the patient was experiencing the effect of the drug regimen. 

We define {\em patient visits} within each line of therapy to reflect changes in the features and outcomes. For the diabetic patients, we consider four possible drug regimens (combinations of oral and injectable drugs), while for the hypertension patients, we consider the most frequent 19 of the 32 combinations of drugs and merge all others into one class.

\paragraph{Diabetic Patients.} During each line of therapy, we assume that the patient visits every 100 days, beginning from the start of the therapy and continuing until at least 80 days prior to the end of the therapy. The measurements, lab tests are averaged over the 100 days prior to the visit. We define the {\em current prescription} of each visit as the combination of drugs that was given during the 100 days immediately preceding the visit, and the {\em standard of care} as the drug regimen that is prescribed by the doctors at the time of the visit. If no value exists over the 100 days, we use the neighboring visits to determine the measurements/lab tests (through linear interpolation) and the current prescription. The {\em future} outcome for each visit is computed as the average HbA\textsubscript{1c} 75 to 200 days after the visit. Patient visits that contain missing values for the outcome are dropped. We end up with 12,016 valid visits, which are divided into four groups based on their standard of care.

\paragraph{Hypertension Patients.} During each line of therapy, the patient visits are considered occurring every 70 days, beginning from the start of the therapy and continuing until at least 180 days prior to the end of the therapy. The measurements, lab tests are averaged over the 10 days prior to the visit. We define the {\em current prescription} of each visit as the combination of drugs that was given during the 10 days immediately preceding the visit, and the {\em standard of care} as the drug regimen that is prescribed by the doctors at the time of the visit. We narrow down the time window due to the fact that the blood pressure is usually much more noisy than the HbA\textsubscript{1c}, and thus the features within a smaller time window tend to be more relevant. The {\em future} outcome of the visit is the average systolic blood pressure 70 to 180 days after it. Linear interpolation is used to replace the missing values of the measurements and lab tests. We have obtained 26,128 valid visits, which are divided into 20 groups based on their standard of care.

\paragraph{Prescriptions.} The prescriptions are used to group the patient visits. For the diabetic patients, we consider two types of prescriptions: one includes oral medications, e.g., metformin, pioglitazone, and sitagliptin, etc., and the other type includes injectable medications, e.g., insulin. Typically, injectable medications are prescribed for patients with more advanced disease. For the hypertension patients, six types of prescriptions are considered: ACE inhibitor, Angiotensin Receptor Blockers (ARB), calcium channel blockers, thiazide and thiazide-like diuretics, $\alpha$-blockers and $\beta$-blockers. 

The following sets of features are considered for building the predictive model. The number of features included in both datasets is 63. All features are standardized before fed into our algorithm. 

\paragraph{Demographic information.} Includes sex (male, female and other), age and
race (10 types). We consider the three most frequent races: Caucasian, Black, and
Hispanic, and group all others into one category `other'.

\paragraph{Measurements.} Systolic/diastolic blood pressure (mmHg), Body Mass Index
(BMI) and pulse. 

\paragraph{Lab tests.} Two types of tests considered: blood chemistry tests such as
calcium, carbon dioxide, chloride, potassium, sodium, creatinine, and urea nitrogen;
and hematology tests such as blood glucose, hematocrit, hemoglobin, leukocyte count,
platelet count, and mean corpuscular volume.

\paragraph{Diagnosis history.} The ICD-9 coding system is used to record
diagnoses.   

\subsection{Model Development and Results}

We will compare our prescriptive algorithm with several alternatives that replace our
{\em Distributionally Robust Linear Regression (DRLR)} informed K-NN with a different
predictive model such as LASSO, CART, and OLS informed K-NN
\citep{bertsimas2017personalized}. Both deterministic and randomized prescriptive
policies are considered using predictions from these models. We note a very recent
tree-based algorithm called Optimal Prescription Tree (OPT) developed in
\cite{bertsimas2019optimal}, that uses either constant or linear models in the leaves
of the tree in order to predict the counterfactuals and to assign optimal treatments
to new samples. We do not include it as a comparison in this work, yet, it would be
interesting to do in subsequent work.

\paragraph{Parameter tuning.} Within each prescription group, we randomly split the patient visits into three sets: a training set (80\%), a validation set (10\%), and a test set (10\%). To reflect the dependency of the number of neighbors on the number of training samples, we perform a linear regression between these two quantities, which will be used  to determine the number of neighbors needed in different settings.

To tune the exponent $\xi$ for the randomized strategy, it is necessary to evaluate the effects of counterfactual treatments. We assess the predictive power of a series of robust predictive models in terms of the following metrics:
\begin{itemize}
	\item R-square:
	\begin{equation*}
	\text{R\textsuperscript{2}} (\by, \hat{\by}) = 1-\dfrac{\sum_{i=1}^{N_t} (y_{i}-\hat{y}_{i})^2}{\sum_{i=1}^{N_t} (y_{i}-\bar{y})^2},
	\end{equation*}
	where $\by = (y_{1},\,\ldots,\, y_{N_t})$ and $\hat{\by} = (\hat{y}_{1},\,\ldots,\, \hat{y}_{N_t})$ are the vectors of the true (observed) and predicted outcomes, respectively, with $N_t$ the size of the test set, and $\bar{y} = (1/N_t) \sum_{i=1}^{N_t} y_i$. 
	\item {\em Mean Squared Error (MSE)}: 
	\begin{equation*}
	\text{MSE} (\by, \hat{\by}) = \frac{1}{N_t} \sum_{i=1}^{N_t} (y_{i}-\hat{y}_{i})^2.
	\end{equation*}
	\item {\em Mean Absolute Error (MeanAE)}, which is more robust to large deviations than the MSE in that the absolute value function increases more slowly than the square function over large (absolute) values of the argument.
	\begin{equation*}
	\text{MeanAE} (\by, \hat{\by}) = \frac{1}{N_t} \sum_{i=1}^{N_t} |y_{i}-\hat{y}_{i}|.
	\end{equation*}
	\item MAD, which can be viewed as a robust measure of the MeanAE, computing the median of the absolute deviations:
	\begin{equation*}
	\text{MAD} (\by, \hat{\by}) = \text{Median} \left( |y_{i}-\hat{y}_{i}| , i \in \lb N_t \rb \right).
	\end{equation*}	
\end{itemize} 
The out-of-sample performance metrics of the various models on the two datasets are shown in Tables~\ref{tab:diabetes_preformances} and \ref{tab:hypertension}, where the numbers in the parentheses show the improvement of DRLR informed K-NN compared against other methods. Huber refers to the robust regression method proposed in \cite{huber1964robust, huber1973robust}, and CART refers to the {\em Classification And Regression Trees}. Huber/OLS/LASSO + K-NN means fitting a K-NN regression model with a Huber/OLS/LASSO-weighted distance metric. We note that in order to produce well-defined and meaningful predictive performance metrics, the dataset used to generate Tables~\ref{tab:diabetes_preformances} and \ref{tab:hypertension} did not group the patients by their prescriptions. A universal model was fit to all patients with prescription information being used as one of the predictors. 
Nevertheless, it would still be considered as a fair comparison as all models were
evaluated on the same dataset. The results provide supporting evidence for the
validity of our DRLR+K-NN model that outperforms all others in all metrics, and is
thus used to impute the outcome for an unobservable treatment $m$, through averaging
over the most similar patient visits who have received the prescription $m$ in the
{\em validation set}, where the number of neighbors is selected to fit the size of the validation set. Note that using DRLR+K-NN as an imputation model might cause bias in evaluating the performance of different methods, since it is in favor of the framework that uses the same model (DRLR+K-NN) to predict the future outcome. Using a weighted combination of several different predictive models may alleviate the bias. This could be done in future work. 

\begin{table}[htbp]
	{\centering 
		\caption{Performance of different models for predicting future HbA\textsubscript{1c} for diabetic patients.} 
		\label{tab:diabetes_preformances} 
		{\small \begin{tabular}{cccccc}
				\toprule
				Methods & R\textsuperscript{2} & MSE & MeanAE & MAD\\
				\midrule
				OLS  & 0.52 (2\%) & 1.36 (2\%) & 0.81 (4\%) & 0.55 (11\%)\\
				LASSO  & 0.52 (2\%) & 1.37 (2\%) & 0.80 (3\%) & 0.54 (9\%)\\
				Huber  & 0.36 (47\%) & 1.81 (26\%) & 0.96 (19\%) & 0.70 (30\%)\\
				RLAD  & 0.50 (4\%) & 1.40 (4\%) & 0.78 (1\%) & 0.50 (1\%)\\
				K-NN  & 0.25 (109\%) & 2.11 (37\%) & 1.07 (27\%) & 0.81 (39\%)\\
				OLS+K-NN  & 0.52 (0\%) & 1.34 (0\%) & 0.79 (1\%) & 0.51 (3\%)\\
				LASSO+K-NN  & 0.52 (1\%) & 1.36 (1\%) & 0.79 (1\%) & 0.50 (1\%)\\
				Huber+K-NN  & 0.51 (3\%) & 1.38 (3\%) & 0.81 (3\%) & 0.53 (7\%)\\
				DRLR+K-NN  & 0.52 (N/A) & 1.34 (N/A) & 0.78 (N/A) & 0.49 (N/A)\\
				CART  & 0.49 (7\%) & 1.43 (7\%) & 0.81 (3\%) & 0.50 (2\%)\\
				\bottomrule
			\end{tabular} }\\
		}
	\end{table} 
	
	\begin{table}[htbp]
		{\centering 
			\caption{Performance of different models for predicting future systolic blood pressure for hypertension patients. }
			\label{tab:hypertension} 
			{\small \begin{tabular}{ccccc}
					\toprule
					Methods & R\textsuperscript{2} & MSE & MeanAE & MAD\\
					\midrule
					OLS  & 0.31 (14\%) & 170.80 (6\%) & 10.09 (7\%) & 8.15 (9\%)\\
					LASSO  & 0.31 (14\%) & 170.83 (6\%) & 10.08 (7\%) & 8.22 (10\%)\\
					Huber  & 0.22 (62\%) & 193.54 (17\%) & 10.70 (12\%) & 8.61 (14\%)\\
					RLAD  & 0.30 (18\%) & 173.32 (8\%) & 10.11 (7\%) & 8.28 (11\%)\\
					K-NN  & 0.33 (10\%) & 167.41 (5\%) & 9.62 (2\%) & 7.50 (2\%)\\
					OLS+K-NN  & 0.35 (1\%) & 160.22 (0\%) & 9.42 (0\%) & 7.49 (1\%)\\
					LASSO+K-NN  & 0.32 (12\%) & 169.50 (6\%) & 9.74 (3\%) & 7.73 (5\%)\\
					Huber+K-NN  & 0.32 (10\%) & 167.92 (5\%) & 9.71 (3\%) & 7.84 (6\%)\\
					DRLR+K-NN  & 0.36 (N/A) & 159.74 (N/A) & 9.42 (N/A) & 7.38 (N/A)\\
					CART  & 0.25 (43\%) & 186.23 (14\%) & 10.34 (9\%) & 8.22 (10\%)\\
					\bottomrule
				\end{tabular} }\\
			}
		\end{table}
		
		\paragraph{Model training.} 
		We solve the predictive models on the whole training set with the best tuned parameters, the output of which is used to develop the optimal prescriptions for the test set patients. The parameter $\bar{\epsilon}$ in the threshold $T(\bx)$ is set to $0.1$. For estimating the conditional mean and standard deviation of the predicted outcome using Algorithm~\ref{ms}, we set $a_m = 0.9N_m$, and $d_m = 100$. We compute the average improvement (reduction) in outcomes for patients in the test set, which is defined to be the difference between the (expected) {\em future} outcome under the recommended therapy and the {\em current} observed outcome. If the recommendation does not match the standard of care, its future outcome is estimated through the imputation model that was discussed earlier, where $K_m$ should be selected to fit the size of the {\em test set}.
		
		\paragraph{Results and discussions.} The reductions in outcomes (future minus current) for various models are shown in Table~\ref{tab:results_orig}. The columns indicate the prescriptive policies (deterministic or randomized); the rows represent the predictive models whose outcomes $\hat{y}_m (\bx)$ serve as inputs to the prescriptive algorithm. We test the performance of all algorithms over five repetitions, each with a different training set. The numbers outside the parentheses are the mean reductions in the outcome and the numbers inside the parentheses are the corresponding standard deviations. We note that HbA\textsubscript{1c} is measured in percentage while systolic blood pressure in mmHg. We also list the reductions in outcomes resulted from the {\em standard of care}, and the {\em current prescription} which prescribes $m_{\text{f}}(\bx) = m_{\text{c}}(\bx)$ with probability one, i.e., always continuing the current drug regimen.
		
		\begin{table}[hbt]
			\caption{The reduction in HbA\textsubscript{1c}/systolic blood pressure for various models.}
			\label{tab:results_orig}
			\begin{center}
				{\footnotesize
					\begin{tabular}{>{\centering\arraybackslash}p{2.8cm}>{\centering\arraybackslash}p{1.65cm}>{\centering\arraybackslash}p{1.65cm}>{\centering\arraybackslash}p{1.65cm}>{\centering\arraybackslash}p{1.65cm}}
						\toprule
						& \multicolumn{2}{c}{Diabetes} & \multicolumn{2}{c}{Hypertension} \\ 
						& Deterministic & Randomized & Deterministic & Randomized\\
						\midrule
						LASSO & -0.51 (0.16) & -0.51 (0.16) & -4.71 (1.09) & -4.72 (1.10)\\
						CART & -0.45 (0.13) & -0.42 (0.14) & -4.84 (0.62) & -4.87 (0.66)\\
						OLS+K-NN & -0.53 (0.13) & -0.53 (0.13) & -4.33 (0.46) & -4.33 (0.47)\\
						DRLR+K-NN & -0.56 (0.06) & -0.55 (0.08) & -6.98 (0.86) & -7.22 (0.82) \\  
						\midrule    
						Current prescription & \multicolumn{2}{c} {-0.22 (0.04)} & \multicolumn{2}{c}{-2.52 (0.19)} \\
						\midrule
						Standard of care & \multicolumn{2}{c} {-0.22 (0.03)} & \multicolumn{2}{c}{-2.37 (0.11)} \\
						\bottomrule
					\end{tabular}}
				\end{center}
			\end{table}
			
			Several observations are in order: $(i)$ all models outperform the current prescription
			and the standard of care; $(ii)$ the DRLR-informed K-NN model leads to the largest
			reduction in outcomes with a relatively stable performance; and $(iii)$ the randomized
			policy achieves a similar performance (slightly better on the hypertension dataset)
			to the deterministic one. We expect the randomized strategy to win when the effects
			of several treatments do not differ much, in which case the deterministic algorithm
			might produce misleading results. The randomized policy could potentially improve the
			out-of-sample (generalization) performance, as it gives the flexibility of exploring
			options that are suboptimal on the training set, but might be optimal on the test
			set.  The advantages of the DRLR+K-NN model are more prominent in the hypertension
			dataset, due to the fact that we considered a finer classification of the
			prescriptions for patients with hypertension, while for diabetic patients, we only
			distinguish between oral and injectable prescriptions.
			
			\subsection{Refinement on the DRLR+K-NN Model} 
			Up to now, we used a patient-independent parameter $K_m$ (the number of neighbors in
			group $m$) to predict the effects of treatments on different individuals. Such a
			strategy might improperly utilize less relevant information and lead to inadequate
			predictions. For example, denote by $d_i^m$ the distance between the patient in
			question and her $i$-th closest neighbor in group $m$, and assume there exists a ``big
			jump'' at $d_{j}^m$, i.e., $d_{j}^m - \sum_{i=1}^{j-1}d_i^m/(j-1)$ is large. If
			$K_m \ge j$, we would include the $j$-th closest neighbor in computing the K-NN
			average, resulting in a biased estimate given its dissimilarity to the patient of
			interest.
			
			We thus propose a patient-specific rule to determine the appropriate number of
			neighbors. Specifically, using the notations $d_i^m$ defined above, we know $d_1^m
			\le d_2^m \le \cdots \le d_{K_m}^m$. Define
			$$j_m^* = \arg\max_j \Big(d_j^m - \sum_{i=1}^{j-1}\frac{d_i^m}{j-1}\Big).$$
			The number of neighbors $K_m'$ will be determined as follows:
			\begin{equation*}
			K_m' = 
			\begin{cases}
			j_m^*-1, & \text{if $\frac{d_{j_m^*}^m -
					\sum_{i=1}^{j_m^*-1}\frac{d_i^m}{j_m^*-1}}{\sum_{i=1}^{j_m^*-1}\frac{d_i^m}{j_m^*-1}}>\tilde{T},$} \\ 
			K_m, & \text{otherwise},
			\end{cases}
			\end{equation*}
			where $\tilde{T}$ is some threshold that can be tuned using cross-validation. This
			strategy discards the neighbors that are relatively far away from the patient under
			consideration. We test this strategy on the two datasets, using a cross-validated
			threshold $\tilde{T} = 2.5$ and $1$ for diabetes and hypertension, respectively, and
			show the results in Tables~\ref{tab:diabetes_trunc} and \ref{tab:hypertension_trunc}. Notice that
			such a truncation strategy could affect both the training of DRLR+K-NN and the
			imputation model that is used to evaluate the effects of counterfactual
			treatments. To compare with the original strategy of using a uniform $K_m$ for every
			patient, we list in the left halves of the tables the results from adopting the
			truncation strategy to both training and imputation, and in the right halves the
			results from applying the truncation only to the imputation/evaluation model.  We see
			that using a patient-specific $K_m'$ in general leads to a larger reduction in
			outcomes.
			
			\begin{table}[hbt]
				\caption{The reduction in HbA\textsubscript{1c} for various models ($\tilde{T}=2.5$).}
				\label{tab:diabetes_trunc}
				\begin{center}
					{\footnotesize \begin{tabular}{>{\centering\arraybackslash}p{2.9cm}>{\centering\arraybackslash}p{1.65cm}>{\centering\arraybackslash}p{1.65cm}>{\centering\arraybackslash}p{1.65cm}>{\centering\arraybackslash}p{1.65cm}}
							\toprule
							& \multicolumn{2}{c}{Training with $K_m'$} & \multicolumn{2}{c}{Training with $K_m$} \\ 
							& Deterministic & Randomized & Deterministic & Randomized\\
							\midrule
							LASSO &-0.54 (0.19)  &-0.54 (0.20)  &-0.50 (0.17)  &-0.49 (0.17) \\
							CART &-0.62 (0.32)  &-0.57 (0.27)  & -0.56 (0.19) &-0.53 (0.15) \\
							OLS+K-NN &-0.65 (0.25)  &-0.64 (0.25)  &-0.61 (0.16)  &-0.61 (0.17)  \\
							DRLR+K-NN &-0.68 (0.20)  &-0.67 (0.23)  &-0.61 (0.10)  & -0.59 (0.10) \\  
							\midrule    
							Current prescription & \multicolumn{2}{c} {-0.23 (0.05)} & \multicolumn{2}{c}{-0.22 (0.05)} \\
							\midrule
							Standard of care & \multicolumn{2}{c} {-0.22 (0.03)} & \multicolumn{2}{c}{-0.22 (0.03)}\\
							\bottomrule
						\end{tabular}}
					\end{center}
				\end{table}
				
				\begin{table}[hbt]
					\caption{The reduction in systolic blood pressure for various models ($\tilde{T}=1$).}
					\label{tab:hypertension_trunc}
					\begin{center}
						{\footnotesize \begin{tabular}{>{\centering\arraybackslash}p{2.9cm}>{\centering\arraybackslash}p{1.65cm}>{\centering\arraybackslash}p{1.65cm}>{\centering\arraybackslash}p{1.65cm}>{\centering\arraybackslash}p{1.65cm}}
								\toprule
								& \multicolumn{2}{c}{Training with $K_m'$} & \multicolumn{2}{c}{Training with $K_m$} \\ 
								& Deterministic & Randomized & Deterministic & Randomized\\
								\midrule
								LASSO & -4.34 (0.28) & -4.33 (0.28)  & -4.22 (0.20) & -4.22 (0.19)  \\
								CART & -4.46 (0.46) & -4.49 (0.50) & -4.48 (0.55) & -4.51 (0.49)\\
								OLS+K-NN & -4.30 (0.35) & -4.30 (0.32) &-4.27 (0.32)  & -4.29 (0.31) \\
								DRLR+K-NN & -7.42 (0.46) &  -7.58 (0.51)  & -6.58 (0.70) & -6.78 (0.73) \\  
								\midrule    
								Current prescription & \multicolumn{2}{c} {-2.56 (0.14)} & \multicolumn{2}{c}{-2.50 (0.16)} \\
								\midrule
								Standard of care & \multicolumn{2}{c} {-2.37 (0.11)} & \multicolumn{2}{c}{-2.37 (0.11)}\\
								\bottomrule
							\end{tabular}}
						\end{center}
					\end{table}
					
					\section{Summary} \label{sec:4-5}
					
					We proposed an interpretable robust predictive method by combining ideas from
					distributionally robust optimization with the local learning procedure K-Nearest
					Neighbors, and established theoretical guarantees on its out-of-sample predictive
					performance. We also developed a randomized prescriptive policy based on the robust
					predictions, and proved its optimality in terms of the expected true outcome.  In
					conjunction, we derived a closed-form expression for a clinically meaningful
					threshold that is used to activate the randomized prescriptive policy.  We applied
					the proposed methodology to a diabetes and a hypertension dataset obtained from a
					major safety-net hospital, providing numerical evidence for the predicted improvement
					on outcomes due to our algorithm.

\chapter{Advanced Topics in Distributionally Robust Learning}  \label{chapt:adv}
In this section, we will cover a number of active research topics in the domain of DRO under the Wasserstein metric. Different from previous sections, where we focused on traditional supervised learning models with identically and independently distributed labeled data, here we want to explore how to adapt the DRO framework to more complex data and model regimes. Specifically, we will study:
\begin{itemize}
	\item Distributionally Robust {\em Semi-Supervised Learning (SSL)}, which
	estimates a robust classifier with partially labeled data, through $(i)$ either
	restricting the marginal distribution to be consistent with the unlabeled data,
	$(ii)$ or modifying the structure of DRO by allowing the center of the
	ambiguity set to vary, reflecting the uncertainty in the labels of the
	unsupervised data.
	
	\item DRO in {\em Reinforcement Learning (RL)} with temporally correlated data,
	which considers {\em Markov Decision Processes (MDPs)} and seeks to inject
	robustness into the probabilistic transition model. We will derive a lower
	bound for the {\em distributionally robust} value function in a regularized
	form.   
\end{itemize}

\section{Distributionally Robust Learning with Unlabeled Data}
In this section we study a Distributionally Robust Optimization (DRO) model with the availability of unlabeled data. This problem can be approached with two types of model architectures. One assumes a setting where supervised DRO with labeled data does not ensure a good generalization performance, and explores the role of unlabeled data in enhancing the performance of conventional supervised DRO, while the other is set up in a semi-supervised setting with potential noise on both labeled and unlabeled data, and aims to robustify SSL algorithms by employing the DRO framework. 

Note that the role of the unlabeled data in the two modeling schemes is different, so are the learning objectives. One seeks to utilize the additional information contained in the unlabeled data, while the other seeks immunity to perturbations on both labeled and unlabeled data. As we will see in the subsequent sections, the former objective is realized through confining the elements of the DRO formulation, i.e., the ambiguity set $\Omega$, to digest the additional information brought by the unlabeled data. By contrast, the latter requires modification of the underlying infrastructure of DRO so that it can be adapted to existing SSL algorithms.

Examples of past works that use unlabeled data to improve adversarial robustness include \cite{carmon2019unlabeled, raghunathan2019adversarial, zhai2019adversarially, alayrac2019labels}. For inducing robustness to SSL, \cite{yan2016robust} proposed an ensemble learning approach through label aggregation. Previous works that fall into the intersection of DRO and SSL include \cite{frogner2019incorporating, blanchet2017distributionally-semi,najafi2019robustness}, where the first two study the role of unlabeled data in improving the generalization performance, while the third one focuses on robustifying a well-known SSL framework, called self-training, by using the DRO. 

Throughout this section, we consider a $K$-class classification problem with a dataset $\scrD$ of size $N$ consisting of two non-overlapping sets $\scrD_l$ (labeled) and $\scrD_{ul}$ (unlabeled), with size $N_l$ and $N_{ul}$, respectively, and $N_l + N_{ul} = N$. Denote by $\scrI_l$ and $\scrI_{ul}$ the index sets corresponding to the labeled and unlabeled data points, respectively. Thus, $\scrD_l = \{\bz_i \triangleq (\bx_i, y_i): \ i \in \scrI_l\}$, where $y_i \in \lb K \rb$, and $\scrD_{ul} = \{\bx_i: \ i \in \scrI_{ul}\}$.

\subsection{Incorporating Unlabeled Data into Distributionally Robust Learning} \label{dro-unlabel}

One of the prerequisites for ensuring a good generalization performance of Wasserstein DRO requires that the ambiguity set includes the true data distribution. In a ``medium-data'' regime, where the observed data may be far from the true data distribution, the Wasserstein ball must be extremely large to contain the true data distribution (cf. Theorem~\ref{measure-con}). As a result, the learner has to be robust to an enormous variety of data distributions, preventing it from making a prediction with any confidence \citep{frogner2019incorporating}. To address this problem, a number of works have proposed to use unlabeled data to further constrain the adversary, see \cite{frogner2019incorporating, blanchet2017distributionally-semi}. Recall the general Wasserstein DRO formulation for a supervised learning problem with feature vector $\bx$ and label $y$: 
\begin{equation}  \label{dro-ssl}
\inf\limits_{\bbeta}\sup\limits_{\mbb{Q}\in \Omega}
\mbb{E}^{\mbb{Q}}\big[ h_{\bbeta}(\bx, y)\big], 
\end{equation}
where $h_{\bbeta}(\bx, y)$ is the loss function evaluated at some hypothesis $\bbeta$, and $\mbb{Q}$ is the probability distribution of $(\bx, y)$ belonging to some set $\Omega$ that constrains the distribution to be close to the empirical distribution of the labeled data, denoted by $\hat{\mathbb{P}}_{N_l}$, in the sense of the order-1 Wasserstein metric induced by a cost metric $s$:
\begin{equation*} 
\Omega \triangleq \{\mbb{Q}\in \scrP(\scrX \times \scrY): W_{s,1}(\mathbb{Q},\ \hat{\mathbb{P}}_{N_l}) \le \epsilon\}.
\end{equation*}
To overcome the problem of overwhelmingly-large ambiguity set $\Omega$, \cite{frogner2019incorporating} proposed to remove from $\Omega$ the distributions that are unrealistic in the sense that their marginals in feature space do not resemble the unlabeled data. Specifically, they define the uncertainty set to be
\begin{equation}  \label{omega-marginal}
\Omega \triangleq \{\mbb{Q}\in \scrU(\mbb{P}_{\scrX}, \underline{\mbb{P}}_{\scrY}, \overline{\mbb{P}}_{\scrY}): W_{s,1}(\mathbb{Q},\ \hat{\mathbb{P}}_{N_l}) \le \epsilon \},
\end{equation}
where $\underline{\mbb{P}}_{\scrY}$ and $\overline{\mbb{P}}_{\scrY}$ are two distributions on the label $y$ with probability vectors $\underline{\bp} \triangleq (\underline{p}_1, \ldots, \underline{p}_K)$ and $\overline{\bp} \triangleq (\overline{p}_1, \ldots, \overline{p}_K)$, respectively, and $\scrU(\mbb{P}_{\scrX}, \underline{\mbb{P}}_{\scrY}, \overline{\mbb{P}}_{\scrY})$ is the set of probability measures whose $\bx$-marginal is $\mbb{P}_{\scrX}$ and $y$-marginal is constrained by $[\underline{\bp}, \overline{\bp}]$, i.e., the class $i$ probability $p_i \in [\underline{p}_i, \overline{p}_i], i \in \lb K \rb$. They choose $\mbb{P}_{\scrX}$ to be consistent with the unlabeled data $\bx \in \scrD_{ul}$. The constraint on $\mbb{P}_{\scrY}$ could come from prior knowledge, or could be implied by the labeled training data. 

\cite{blanchet2017distributionally-semi} also constrained the uncertainty set $\Omega$ by incorporating the information of the unlabeled data. Different from (\ref{omega-marginal}) where the marginals are enforced to be consistent with the unlabeled data, they set the joint support of the feature and labels to be confined to the empirical observations. Specifically, they build a ``complete'' unlabeled set by assigning all possible labels to each unlabeled data point: $\mathscr{C}_{ul} \triangleq \bigcup\nolimits_{y=1}^K\{(\bx_i, y): \ i \in \scrI_{ul}\}$, and then construct the full dataset $\scrC = \scrD_l \cup \mathscr{C}_{ul}$. The uncertainty set is restricted to be supported on $\scrC$, namely,
\begin{equation}  \label{omega-joint}
\Omega \triangleq \{\mbb{Q}\in \scrP(\scrC): W_{s,1}(\mathbb{Q},\ \hat{\mathbb{P}}_{N_l}) \le \epsilon \}.
\end{equation}
Compared to (\ref{omega-marginal}), (\ref{omega-joint}) is more restrictive in the sense that it imposes constraints on the joint distribution of the feature and labels, while (\ref{omega-marginal}) only restricts the marginals. Furthermore, it does not allow support points outside the empirical observations, which eliminates one of the major advantages of the Wasserstein metric. In the absence of the unlabeled data, (\ref{omega-joint}) essentially asks the learner to be robust only to distributions with support on $\scrD_l$, which could hurt the generalization capability on unseen data. By contrast, (\ref{omega-marginal}) guarantees robustness to distributions with support on the whole data space.

Note that the DRO formulation with an uncertainty set defined through either (\ref{omega-marginal}) or (\ref{omega-joint}) does not serve the purpose of robustifying an existing SSL model. Rather, it explores ways of improving the generalization performance of a DRO model by utilizing the unlabeled data information.

In the remainder of this section, we will discuss a {\em Stochastic Gradient Descent
	(SGD)} algorithm proposed in \cite{frogner2019incorporating}, in order to solve the Wasserstein DRO formulation assembled with the ambiguity set (\ref{omega-marginal}). The key is to transform the inner infinite-dimensional maximization problem in (\ref{dro-ssl}) into its finite-dimensional dual. Define the worst-case expected loss as
\begin{equation} \label{innerMax-ssl}
v_P(\bbeta) \triangleq \sup\limits_{\mbb{Q}\in \Omega}
\mbb{E}^{\mbb{Q}}\big[ h_{\bbeta}(\bx, y)\big].
\end{equation}
Rewrite (\ref{innerMax-ssl}) by casting it as an optimal transportation problem with a transport plan $\pi \in \scrP(\scrZ \times \scrZ)$:
\begin{equation} \label{primal-ssl}
\begin{array}{rl}    v_P(\bbeta) = 
\sup \limits_{\pi \in \scrP(\scrZ \times \scrZ)} & \ \int\nolimits_{(\scrX \times \scrY)\times \scrZ} h_{\bbeta}(\bx, y)  \mathrm{d}\pi((\bx, y), \bz')\\
\text{s.t.} & \ \int\nolimits_{\scrZ \times \scrZ} s(\bz, \bz')  \mathrm{d}\pi(\bz, \bz') \le \epsilon, \\
& \ \int\nolimits_{\scrZ \times \scrZ} \delta_{\bz_i}(\bz')  \mathrm{d}\pi(\bz,\bz') = \frac{1}{N_l}, \ \forall i \in \scrI_l, \\
& \ \int\nolimits_{(\scrA \times \scrY) \times \scrZ} \mathrm{d}\pi ((\bx, y), \bz') = \mbb{P}_{\scrX}(\scrA), \ \forall \scrA \subseteq \scrX,\\
& \ \int\nolimits_{(\scrX \times \scrY) \times \scrZ} \delta_{i}(y) \mathrm{d}\pi ((\bx, y), \bz') \le \overline{p}_i, \ \forall i \in \lb K \rb, \\
& \ \int\nolimits_{(\scrX \times \scrY) \times \scrZ} \delta_{i}(y) \mathrm{d} \pi ((\bx, y), \bz') \ge \underline{p}_i, \ \forall i \in \lb K \rb,
\end{array}
\end{equation}
where we use $\bz \triangleq (\bx, y)$ to index the support of the worst-case measure and $\bz'$ to index the support of $\hat{\mbb{P}}_{N_l}$. Notice that the constraint on the $\bx$-marginal is infinite dimensional. Through translating (\ref{primal-ssl}) to its dual one can move the infinite dimensional constraint to an expectation under $\mbb{P}_{\scrX}$ in the objective. 
The dual to (\ref{primal-ssl}) can be formulated as
\begin{equation} \label{dual-sll}
\begin{array}{rl}    v_D(\bbeta) = 
\inf \limits_{\alpha, \bgamma, \underline{\blambda}, \overline{\blambda}} & \quad \alpha \epsilon + \frac{1}{N_l} \sum_{i=1}^{N_l} \gamma_i + \sum_{k=1}^K (\overline{\lambda}_k \overline{p}_k - \underline{\lambda}_k \underline{p}_k) + \\
& \quad \mbb{E}^{\mbb{P}_{\scrX}} \Big[\phi(\bx; \bbeta, \alpha, \bgamma, \underline{\blambda}, \overline{\blambda})\Big] \\
\text{s.t.} & \quad \alpha, \underline{\lambda}_k, \overline{\lambda}_k \ge 0, \ \forall k \in \lb K \rb,
\end{array}
\end{equation}
where 
\[
\phi(\bx; \bbeta, \alpha, \bgamma, \underline{\blambda}, \overline{\blambda}) =
\max_{ i \in \lb N_l \rb, k \in \lb K \rb } \phi^{i,k}(\bx; \bbeta, \alpha, \bgamma,
\underline{\blambda}, \overline{\blambda}),
\]
\[ 
\phi^{i,k}(\bx; \bbeta, \alpha, \bgamma, \underline{\blambda}, \overline{\blambda})
\triangleq h_{\bbeta}(\bx, k) - \big(\alpha s \big((\bx, k), \bz_i \big) +
\gamma_i\big) - (\overline{\lambda}_k - \underline{\lambda}_k).
\]
It can be shown that strong duality holds if the primal problem (\ref{primal-ssl}) is feasible. We refer the reader to Theorem 2 of \cite{frogner2019incorporating} for a detailed proof. The DRO problem (\ref{dro-ssl}) reduces to minimizing $v_D(\bbeta)$ w.r.t. $\bbeta$, which can be solved via the stochastic gradient method. The main obstacle to deriving the gradient lies in the expectation in the objective of $v_D(\bbeta)$. By applying the Reynolds Transport Theorem \citep{reynolds1903papers}, one can obtain that 
\begin{equation} \label{grad-exp}
\begin{aligned}
\frac{\partial}{\partial \alpha}  \mbb{E}^{\mbb{P}_{\scrX}} \Big[ \phi(\bx; \bbeta, \alpha, \bgamma, \underline{\blambda}, \overline{\blambda})\Big] & = \mbb{E}^{\mbb{P}_{\scrX}} \Big[ \frac{\partial}{\partial \alpha} \phi(\bx; \bbeta, \alpha, \bgamma, \underline{\blambda}, \overline{\blambda})\Big]. \\
\end{aligned}
\end{equation}
Notice that $\phi$ is defined to be the maximum of a series of functions $\phi^{i,k}$. To evaluate its derivative, we need to partition the feature space $\scrX$ to recognize the set of points $\bx$ where the maximum is achieved at each $(i, k)$. Define 
\begin{equation*}
\scrV^{i,k} \triangleq \Big\{\bx \in \scrX: \phi^{i,k}(\bx; \bbeta, \alpha, \bgamma, \underline{\blambda}, \overline{\blambda}) \ge \phi^{i',k'}(\bx; \bbeta, \alpha, \bgamma, \underline{\blambda}, \overline{\blambda}), \forall i', k' \Big \}.
\end{equation*}
The derivative of $\phi$ can be evaluated as
\begin{equation*}
\frac{\partial}{\partial \alpha} \phi(\bx; \bbeta, \alpha, \bgamma,
\underline{\blambda}, \overline{\blambda}) = -\sum_{i=1}^{N_l}\sum_{k=1}^K
\mathbf{1}_{\scrV^{i,k}}(\bx) s \big((\bx, k), \bz_i\big),
\end{equation*}
where $\mathbf{1}_{\scrV^{i,k}}(\bx)$ denotes the indicator function of the event
$\bx\in \scrV^{i,k}$. 
Similarly, the gradients w.r.t. other parameters are computed as follows.
\begin{equation*}
\frac{\partial}{\partial \gamma_i} \phi(\bx; \bbeta, \alpha, \bgamma, \underline{\blambda}, \overline{\blambda}) = -\sum_{k=1}^K \mathbf{1}_{\scrV^{i,k}}(\bx),
\end{equation*}
\begin{equation*}
\frac{\partial}{\partial \underline{\lambda}_k} \phi(\bx; \bbeta, \alpha, \bgamma, \underline{\blambda}, \overline{\blambda}) = \sum_{i=1}^{N_l} \mathbf{1}_{\scrV^{i,k}}(\bx),
\end{equation*}
\begin{equation*}
\frac{\partial}{\partial \overline{\lambda}_k} \phi(\bx; \bbeta, \alpha, \bgamma, \underline{\blambda}, \overline{\blambda}) = -\sum_{i=1}^{N_l} \mathbf{1}_{\scrV^{i,k}}(\bx),
\end{equation*}
\begin{equation*}
\frac{\partial}{\partial \beta_j} \phi(\bx; \bbeta, \alpha, \bgamma, \underline{\blambda}, \overline{\blambda}) \in \sum_{i=1}^{N_l}\sum_{k=1}^K \mathbf{1}_{\scrV^{i,k}}(\bx) \frac{\partial}{\partial \beta_j} h_{\bbeta}(\bx, k).
\end{equation*}
For $\bx$ lying on the boundary between two of the sets $\scrV^{i,k}$, we can obtain a subgradient by arbitrarily selecting only one of these $\scrV^{i,k}$ to contain $\bx$ when evaluating $\mathbf{1}_{\scrV^{i,k}}(\bx)$. To evaluate the expectation of the gradient under $\mbb{P}_{\scrX}$ on the RHS of (\ref{grad-exp}), one can simulate a series of $\bx$ values, say $\bx_1, \ldots, \bx_{N_b}$, from $\mbb{P}_{\scrX}$, and compute the above gradients by taking the sample average. This is summarized in Algorithm~\ref{sgd-semi}.

\begin{algorithm}[H]
	\caption{SGD for distributionally robust learning with unlabeled data under uncertainty set (\ref{omega-marginal}).} \label{sgd-semi}
	\begin{algorithmic}
		\State {\textbf{Input:} $\epsilon \ge 0, \underline{p}_i, \overline{p}_i \in [0,1], i \in \lb K \rb$, feasible solution $\bbeta_0$, step size $\delta>0$, batch size $N_b$.}	
		\State $\bbeta \leftarrow \bbeta_0, \alpha \leftarrow 0, \bgamma, \underline{\blambda}, \overline{\blambda} \leftarrow \mathbf{0}.$
		\While{not converged}
		\State Sample $\bx_1, \ldots, \bx_{N_b} \sim \mbb{P}_{\scrX}$.
		\State $\bbeta \leftarrow \text{Proj}_{\scrB} \Big[ \bbeta - \frac{\delta}{N_b} \sum_{j=1}^{N_b} \nabla_{\bbeta} \phi(\bx_j; \bbeta, \alpha, \bgamma, \underline{\blambda}, \overline{\blambda})\Big]$
		\State $\alpha \leftarrow \max \Big(0, \alpha - \delta \big[ \epsilon + \frac{1}{N_b} \sum_{j=1}^{N_b} \nabla_{\alpha} \phi(\bx_j; \bbeta, \alpha, \bgamma, \underline{\blambda}, \overline{\blambda})\big]\Big)$
		\State $\bgamma \leftarrow \bgamma - \delta \big[ \frac{1}{N_l} \mathbf{e} + \frac{1}{N_b} \sum_{j=1}^{N_b} \nabla_{\bgamma} \phi(\bx_j; \bbeta, \alpha, \bgamma, \underline{\blambda}, \overline{\blambda})\big]$
		\State $\underline{\blambda} \leftarrow \max \Big(0, \underline{\blambda} - \delta \big[ -\underline{\bp} + \frac{1}{N_b} \sum_{j=1}^{N_b} \nabla_{\underline{\blambda}} \phi(\bx_j; \bbeta, \alpha, \bgamma, \underline{\blambda}, \overline{\blambda})\big]\Big)$
		\State $\overline{\blambda} \leftarrow \max \Big(0, \overline{\blambda} - \delta \big[ \overline{\bp} + \frac{1}{N_b} \sum_{j=1}^{N_b} \nabla_{\overline{\blambda}} \phi(\bx_j; \bbeta, \alpha, \bgamma, \underline{\blambda}, \overline{\blambda})\big]\Big)$
		\EndWhile
	\end{algorithmic}
\end{algorithm}

\subsection{Distributionally Robust Semi-Supervised Learning} \label{robustSSL}
In this subsection we discuss the problem of robustifying existing SSL algorithms via DRO. Different from Section~\ref{dro-unlabel}, the goal here is to induce robustness into conventional SSL models, which requires modification of the DRO infrastructure in order to fit the characteristics of the problem at hand. Note that DRO cannot readily be applied to the partially-labeled setting, since it needs complete knowledge of all the feature-label pairs. 

A well-known family of SSL models is called {\em self-learning}, which first trains a classifier on the labeled portion of a dataset, and then assigns pseudo-labels to the remaining unlabeled samples using the learned rules. The enlarged dataset consisting of both the supervised and artificially-labeled unsupervised samples is used in the final stage of training. To prevent overfitting, instead of assigning a deterministic hard label to the unsupervised data points, one can apply a soft labeling scheme that maintains a level of uncertainty through specifying a probability distribution of the labels. 

To use DRO in a semi-supervised setting, we need to address the uncertainty embedded in the unknown labels of the unsupervised samples. This can be resolved by soft-labeling. Define the consistent set of probability distributions $\hat{\scrP}(\scrD) \subseteq \scrP(\scrZ)$ w.r.t. a partially-labeled dataset $\scrD = \scrD_l \cup \scrD_{ul}$ as
\begin{equation*}
\hat{\scrP}(\scrD) \triangleq \bigg\{ \Big(\frac{N_l}{N}\Big) \hat{\mbb{P}}_{N_l} + \Big(\frac{N_{ul}}{N}\Big) \hat{\mbb{P}}_{N_{ul}} \cdot \mbb{Q}: \ \mbb{Q} \in \scrP^{\scrX}(\scrY) \bigg\},
\end{equation*}
where $\mbb{Q}$ encodes the uncertainty in the labels for the unsupervised dataset $\scrD_{ul}$, and $\scrP^{\scrX}(\scrY)$ denotes the set of all conditional distributions supported on $\scrY$, given features in $\scrX$. Note that the distributions in $\hat{\scrP}(\scrD)$ differ from each other only in the way they assign soft labels to the unlabeled data, and the empirical measure corresponding to the true complete dataset is a member of $\hat{\scrP}(\scrD)$. 

We will illustrate the idea proposed in \cite{najafi2019robustness} for introducing DRO to SSL, where they select a suitable measure from $\hat{\scrP}(\scrD)$, and use it as a proxy of the true empirical probability measure that serves as the center of the Wasserstein ball. The learner essentially aims to hedge against a set of distributions centered at some $\mbb{S} \in \hat{\scrP}(\scrD)$ that is induced by a soft-label distribution $\mbb{Q}$, so that the resulting  
classification rule would show low sensitivity to adversarial perturbations around
the soft-label distribution. The criterion for choosing $\mbb{S}$ is to make the
worst-case expected loss as small as possible. Specifically, the {\em Semi-Supervised
	Distributionally Robust Learning (SSDRO)} model proposed by
\cite{najafi2019robustness} can be formulated as
\begin{equation} \label{dro-self-training}
\inf\limits_{\bbeta}\inf\limits_{\mbb{S} \in \hat{\scrP}(\scrD)} \bigg\{\sup\limits_{\mbb{P}\in \Omega_{\epsilon}(\mbb{S})}
\mbb{E}^{\mbb{P}}\big[ h_{\bbeta}(\bx, y)\big] + \Big(\frac{1-N_l/N}{\lambda}\Big) \mbb{E}^{\hat{\mbb{P}}_{N_{ul}}}[H(\mbb{S}_{|\bx})]\bigg\},
\end{equation}
where $\Omega_{\epsilon}(\mbb{S})$ denotes the set of probability distributions that are close to $\mbb{S}$ by a distance at most $\epsilon$, i.e.,
\[ 
\Omega_{\epsilon}(\mbb{S}) \triangleq \{\mbb{P}\in \scrP(\scrZ):
W_{s,1}(\mathbb{P},\ \mbb{S}) \le \epsilon \},
\]
In (\ref{dro-self-training}), $\mbb{S}_{|\bx}$ is
the conditional distribution over $\scrY$ given $\bx \in \scrX$, $\lambda<0$ is a user-defined parameter, and $H(\cdot)$ denotes the Shannon entropy. 

Notice that for a fixed $\bbeta$, the inner infimum of (\ref{dro-self-training}) guides the learner to pick a soft label distribution that tends to reduce the loss function, which \cite{najafi2019robustness} refers to as an {\em optimistic} learner. Alternatively, one can choose to be {\em pessimistic}, i.e., choosing a $\bbeta$ that hedges against the maximum loss over all possible choices of $\mbb{S}$. To prevent hard labeling of the unsupervised data, $\lambda$ is set to be negative for {\em optimistic} learning, and positive for {\em pessimistic} learning. 

Note also that compared to conventional DRO models, in (\ref{dro-self-training}) we have an additional regularization term that penalizes the Shannon entropy of the conditional label distribution of the unlabeled data. When $\lambda<0$, the regularization term $\Big(\frac{1-N_l/N}{\lambda}\Big) \mbb{E}^{\hat{\mbb{P}}_{N_{ul}}}[H(\mbb{S}_{|\bx})]$ becomes negative. The formulation (\ref{dro-self-training}) essentially promotes softer labels for the unlabeled data by encouraging a larger entropy, implying a higher level of uncertainty in the labels. 

We next discuss how to solve Problem (\ref{dro-self-training}). Using duality,
\cite{najafi2019robustness} was able to transform the inner min-max formulation to
an analytic form whose gradient can be efficiently computed. A Lagrangian relaxation
to (\ref{dro-self-training}) is given in the following theorem.

\begin{thm} [\cite{najafi2019robustness}, Theorem 1] \label{lagrangian-dro-ssl}
	Consider a continuous loss function $h$, and a continuous transportation cost $s$. For a partially-labeled dataset $\scrD$ with size $N$, define the empirical Semi-Supervised Adversarial Risk (SSAR), denoted by $\hat{R}_{\text{SSAR}}(\bbeta; \scrD)$, as 
	\begin{equation} \label{ssar}
	\hat{R}_{\text{SSAR}}(\bbeta; \scrD) \triangleq \frac{1}{N} \sum\limits_{i \in \scrI_l} \phi_{\gamma}(\bx_i, y_i; \bbeta) + \frac{1}{N} \sum\limits_{i \in \scrI_{ul}} \underset{y \in \scrY}{\overset{(\lambda)}{\text{softmin}}}\{\phi_{\gamma}(\bx_i, y; \bbeta)\} + \gamma \epsilon,
	\end{equation}
	where $\gamma \ge 0$, and the adversarial loss $\phi_{\gamma}(\bx, y; \bbeta)$ and the soft-minimum operator are defined as:
	\begin{equation} \label{phi-func}
	\phi_{\gamma}(\bx, y; \bbeta) \triangleq \sup_{\bz' \in \scrZ} h_{\bbeta}(\bz') - \gamma s \big (\bz', (\bx, y) \big),
	\end{equation}
	and
	\begin{equation*}
	\underset{y \in \scrY}{\overset{(\lambda)}{\text{softmin}}}\{q(y)\} \triangleq \frac{1}{\lambda} \log \bigg( \frac{1}{|\scrY|} \sum\limits_{y \in \scrY} e^{\lambda q(y)}\bigg),
	\end{equation*}
	respectively. Let $\bbeta^*$ be a minimizer of (\ref{dro-self-training}) for some given $\epsilon \ge 0$ and $\lambda <0$. Then, there exists $\gamma \ge 0$ such that $\bbeta^*$ is also a minimizer of (\ref{ssar}) with the same parameters $\epsilon$ and $\lambda$.
\end{thm}

According to Theorem~\ref{lagrangian-dro-ssl} our problem is now translated to solving for a $\bbeta$ that minimizes $\hat{R}_{\text{SSAR}}(\bbeta; \scrD)$. To apply SGD, the key is to derive the gradient of the adversarial loss function $\phi_{\gamma}(\bx, y; \bbeta)$, which itself is the output of an optimization problem. The gradient of $\phi$ w.r.t. $\bbeta$ relies on the optimal solution of Problem (\ref{phi-func}), i.e., $\nabla_{\bbeta}\phi_{\gamma}(\bx, y; \bbeta) = \bg_{\bbeta} (\bz^*(\bbeta))$, where $\bg_{\bbeta}(\bz) \triangleq \nabla_{\bbeta}h_{\bbeta}(\bz)$ and $\bz^*(\bbeta)$ is the optimal solution to (\ref{phi-func}). The following lemma specifies a set of sufficient conditions to ensure the uniqueness of the solution.

\begin{lem} [\cite{najafi2019robustness}, Lemma 1] \label{smoothness}
	Assume the loss function $h$ to be differentiable w.r.t. $\bz$, and $\nabla_{\bz}h_{\bbeta}(\bz)$ is $L_z$-Lipschitz w.r.t. $\bbeta$. Also, the cost metric $s$ is $1$-strongly convex in its first argument. If $\gamma > L_{z}$, then Problem (\ref{phi-func}) is $(\gamma - L_z)$-strongly concave for all $(\bx,y) \in \scrZ$.
\end{lem}

Lemma~\ref{smoothness} guarantees the existence and uniqueness of the solution to (\ref{phi-func}). We can thus express the gradients of $\phi$ and $\hat{R}_{\text{SSAR}}$ explicitly as a function of the solution. An efficient computation of the gradient of $\hat{R}_{\text{SSAR}}(\bbeta; \scrD)$ w.r.t. $\bbeta$ is given in the following theorem. 

\begin{thm} [\cite{najafi2019robustness}, Lemma 2] \label{SSDRO-grad}
	Under conditions of Lemma~\ref{smoothness}, assume the loss function $h$ to be differentiable w.r.t. $\bbeta$, and let $\bg_{\bbeta}(\bz) \triangleq \nabla_{\bbeta}h_{\bbeta}(\bz)$. For a fixed $\bbeta$, define 
	\begin{equation} \label{grad-l}
	\bz_i^*(\bbeta) = \argmax_{\bz' \in \scrZ} h_{\bbeta}(\bz') - \gamma s\big(\bz', (\bx_i, y_i)\big), \quad i \in \scrI_l,
	\end{equation}
	and, 
	\begin{equation} \label{grad-ul}
	\bz_i^*(y; \bbeta) = \argmax_{\bz' \in \scrZ} h_{\bbeta}(\bz') - \gamma s\big(\bz', (\bx_i, y)\big), \quad y \in \scrY, \ i \in \scrI_{ul}.
	\end{equation}
	Then, the gradient of (\ref{ssar}) w.r.t. $\bbeta$ can be obtained as
	\begin{equation} \label{ssar-grad}
	\nabla_{\bbeta} \hat{R}_{\text{SSAR}}(\bbeta; \scrD) = \frac{1}{N}\sum\limits_{i \in \scrI_l} \bg_{\bbeta} (\bz_i^*(\bbeta)) + \frac{1}{N} \sum\limits_{i \in \scrI_{ul}} \sum\limits_{y \in \scrY} q(y; \bbeta) \bg_{\bbeta}(\bz_i^*(y; \bbeta)),
	\end{equation}
	where $q(y; \bbeta) \triangleq e^{\lambda \phi_{\gamma}(\bx_i, y; \bbeta)} / \Big( \sum_{y' \in \scrY} e^{\lambda \phi_{\gamma}(\bx_i, y'; \bbeta)}\Big)$.
\end{thm}

Using Theorem~\ref{SSDRO-grad}, we can apply SGD to solve (\ref{ssar}), or equivalently, the SSDRO model (\ref{dro-self-training}). This is summarized in Algorithm~\ref{sgd-ssDRO}. \cite{najafi2019robustness} proved a convergence rate of $O(T^{-1/2})$ for Algorithm~\ref{sgd-ssDRO}, if we assume $\bz_i^*(\bbeta)$ and $\bz_i^*(y; \bbeta)$ can be computed exactly. Nonetheless, the optimality gap $\delta$ can be set infinitesimally small due to the strong concavity of (\ref{grad-l}) and (\ref{grad-ul}) that is shown in Lemma~\ref{smoothness}. The parameters $\gamma$ and $\lambda$ can be tuned via cross-validation.

\begin{algorithm}[H] 
	\caption{Stochastic Gradient Descent for SSDRO.} \label{sgd-ssDRO}
	\begin{algorithmic}
		\State \textbf{Inputs:} $\scrD, \gamma, \lambda, k \le N, \delta, \alpha, T$.	
		\State \textbf{Initialize:} $\bbeta \leftarrow \bbeta_0$, $t \leftarrow 0$. 
		\For{$t=0 \rightarrow T-1$} 
		\State Randomly select index set $\scrI \subseteq \lb N \rb$ with size $k$.
		\For{$i \in \scrI_l \cap \scrI$}
		\State Compute a $\delta$-approx of $\bz_i^*(\bbeta_t)$ from (\ref{grad-l}).
		\EndFor
		\For{$(i, y) \in (\scrI_{ul} \cap \scrI) \times \scrY$}
		\State Compute a $\delta$-approx of $\bz_i^*(y; \bbeta_t)$ from (\ref{grad-ul}).
		\EndFor
		\State Compute the sub-gradient of $\hat{R}_{\text{SSAR}}(\bbeta_t; \scrD)$ from (\ref{ssar-grad}).
		\State \textbf{Update:} $\bbeta_{t+1} \leftarrow  \bbeta_t - \alpha \nabla_{\bbeta} \hat{R}_{\text{SSAR}}(\bbeta_t; \scrD)$.
		\EndFor
		\State \textbf{Output:} $\bbeta^* \leftarrow \bbeta_T$.
	\end{algorithmic}
\end{algorithm}

\section{Distributionally Robust Reinforcement Learning} \label{dro-rl}

So far in this monograph, we considered learning problems where the objective is to
predict an output variable (or vector in the setting of
Section~\ref{chap:multi}). These learning problems were cast as distributionally
robust {\em single-period} optimization problems. Even in the applications of
Section~\ref{ch:presp} involving medical prescriptions, where we considered
information from multiple past time periods to learn actions that optimize an outcome
in the next time period, the resulting optimization problem was single-period. In
this section, we will discuss {\em multi-period} optimization motivated by learning a
{\em policy} for a {\em Markov Decision Process (MDP)}. We will restrict ourselves to
{\em model-based} settings, where there is an explicit model of how the MDP
transitions from state to state under some policy, and seek to inject robustness into
this transition model. The development follows the work in
\cite{derman2020distributional}.

We start by defining a discrete-time MDP. Consider an MDP with a finite state space
$\scrS$, a finite action space $\scrA$, a deterministic reward function $r:
\scrS\times \scrA \ra \mbb{R}$, and a transition probability model $p$ that, given a
state $s_1$ and an action $a$, determines the probability $p(s_2|s_1,a)$ of landing to
the next state $s_2$. A {\em policy} $\pi$ maps states to actions; specifically,
$\pi(a|s)$ denotes the probability of selecting action $a$ in state $s$.  The state
of the MDP evolves dynamically as follows. Suppose that at time $t$ the MDP is in
state $s_t$. According to the policy $\pi$, it selects some action $a_t$, receives a
reward $r(s_t,a_t)$, and transitions to the next state $s_{t+1}$ with probability
$p(s_{t+1}|s_t,a_t)$. In an infinite-horizon discounted reward setting, the objective
is to select a policy $\pi$ that maximizes the {\em expected total discounted reward}
\[ 
\mbb{E}^{\tau\sim \pi}_p \left[\sum_{t=0}^\infty \gamma^t r(s_t,a_t) \right],
\] 
where $\gamma \in [0,1)$ is the discount factor and $\tau\sim \pi$ represents a
random trajectory $\tau=(s_0,a_0,s_1,a_1,\ldots)$ sampled by selecting the initial
state $s_0$ according to some probability distribution $\rho_0(\cdot)\in
\scrP(\scrS)$, sampling actions according to $a_t\sim \pi(\cdot|s_t)$, and states
according to $s_{t+1}\sim p(\cdot|s_t,a_t)$ (hence, the subscript $p$ in the
expectation to denote dependence on the transition model $p$). 

We can now define the state {\em value}, or reward-to-go function, which equals the
future total discounted reward when starting from state $s$, namely, 
\[
v^\pi_p(s) = \mbb{E}_p^{\tau\sim \pi} \left[\sum_{t=0}^\infty \gamma^t r(s_t,a_t)
\mid s_0=s \right].
\]
The value function can be obtained as a solution to the following {\em Bellman
	equation}:
\[ 
v(s) = T_p^\pi v(s) \stackrel{\triangle}{=} \sum_{a\in \scrA} \pi(a|s) \left( r(s,a) +
\gamma \sum_{q\in \scrS} p(q|s,a) v(q) \right). 
\] 
The operator $T_p^\pi$ satisfies a contraction property with
respect to the sup-norm, implying that the Bellman equation has a unique fixed point
denoted by $v^\pi_p(s)$. This can for instance be obtained by successive application
of $T_p^\pi$ to some arbitrary initial solution -- a method known as {\em value
	iteration}.

\subsection{Deterministically Robust Policies} \label{sec:det-rob-mdp}

A number of results in the literature examined how to introduce robustness with
respect to uncertainty on the transition probability model, starting with
\cite{satia1973markovian, white1994markov} and \cite{bagnell2001solving}. A more
complete theory of {\em robust dynamic programming} has been developed in \cite{Iye05} and \cite{NilGha05}. In this work, the transition probability vector
$\bp_{s,a} =(p(q|s,a);\ q\in \scrS)$ at any state-action pair $(s,a)$ belongs to
some ambiguity or uncertainty set $\scrU_{s,a}\subseteq \scrP(\scrS)$. It is assumed
that every time a state-action pair $(s,a)$ is encountered, a potentially different
measure $\bp_{s,a}\in \scrU_{s,a}$ could be applied; this has been termed the {\em
	rectangularity assumption} in \cite{Iye05}. 

In this robust setting, one can define a {\em robust value function} as the
worst-case value function over the uncertainty set, that is,
\begin{equation} \label{rob-v}
v^\pi_\scrU(s) = \inf_{p\in \scrU} \mbb{E}_p^{\tau\sim \pi} \left[\sum_{t=0}^\infty \gamma^t r(s_t,a_t)
\mid s_0=s \right], 
\end{equation}
where the uncertainty set $\scrU$ is the cartesian product of the transition
probability uncertainty sets encountered throughout the trajectory, i.e.,
$\scrU=\prod_{t=0}^\infty \scrU_{s_t,a_t}$. 

\cite{Iye05} and \cite{NilGha05} show that a robust Bellman equation can be written
as:
\begin{align} \label{bellman-det-robust}
v(s) = & T_\scrU^\pi v(s) \\ 
\stackrel{\triangle}{=} & \sum_{a\in \scrA} \pi(a|s) \left( r(s,a) +
\gamma \inf_{\bp_{s,a}\in \scrU_{s,a}} \sum_{q\in \scrS} p(q|s,a) v(q) \right). \notag
\end{align}
As with the non-robust case, the operator $T_\scrU^\pi$ satisfies a
contraction property, implying that the robust Bellman equation has a unique fixed
point which can be computed by successive application of $T_\scrU^\pi$.

\subsection{Distributionally Robust Policies} \label{sec:dro-mdp}

Distributionally robust MDPs can be thought of as a generalization of
deterministically robust MDPs. Instead of selecting transition probabilities out of
the ambiguity set $\scrU$ defined earlier, we can view the transition probability
model as being sampled according to some distribution $\mu\in \scrM \subseteq \scrP(\scrU)$, i.e., $\mu$ is the probability distribution of the transition probability model $p$. Making the
same rectangularity assumption as before, that is, requiring that $\mu$ is a product
of independent distributions over  $\scrU_{s,a}$, we can define a {\em
	distributionally robust value function} similarly to (\ref{rob-v}) as:
\begin{equation} \label{dist-rob-v}
v^\pi_\scrM(s) = \inf_{\mu\in \scrM} \mbb{E}_{p\sim\mu}^{\tau\sim \pi}
\left[\sum_{t=0}^\infty \gamma^t r(s_t,a_t) \mid s_0=s \right].
\end{equation}

\cite{derman2020distributional} introduces Wasserstein distributionally robust MDPs by
defining the set of distributions $\scrM$ as a Wasserstein ball around some nominal
distribution. More specifically, for any state-action pair $(s,a)$, let
$\hat{\mu}_{s,a}\in \scrP(\scrU_{s,a})$ be some {\em nominal} distribution over
$\scrU_{s,a}$. For any distribution $\mu_{s,a}\in \scrP(\scrU_{s,a})$, define the
order-$1$ Wasserstein distance induced by some norm $\|\cdot\|$, and denote it by
$W_1(\hat{\mu}_{s,a}, \mu_{s,a})$. A Wasserstein ball around the nominal distribution
can be defined as:
\begin{equation} \label{rl-wass-ball}
\Omega_{\epsilon_{s,a}}(\hat{\mu}_{s,a}) = \left\{\mu_{s,a} \in
\scrP(\scrU_{s,a})\ :\  W_1(\hat{\mu}_{s,a}, \mu_{s,a}) \leq \epsilon_{s,a} \right\}.
\end{equation}
Under a rectangularity assumption as in Sec.~\ref{sec:det-rob-mdp}, we define the
cartesian product of the sets $\Omega_{\epsilon_{s,a}}(\hat{\mu}_{s,a})$ over all
state-action pairs and denote it by $\Omega_{\bepsilon}(\hat{\mu})=\prod_{(s,a)\in
	\scrS\times \scrA} \Omega_{\epsilon_{s,a}}(\hat{\mu}_{s,a})$, where $\bepsilon$ is
a vector defined as $\bepsilon=(\epsilon_{s,a};\ (s,a)\in \scrS\times \scrA)$, and $\hat{\mu} = \prod_{(s,a)\in \scrS\times \scrA} \hat{\mu}_{s,a}$.

Analogously to (\ref{bellman-det-robust}), the {\em distributionally robust Bellman
	equation} can be written as:
\begin{align} \label{bellman-dist-robust} 
v(s) = & T_{\Omega_{\bepsilon}(\hat{\mu})}^{\pi} v(s) \\
\stackrel{\triangle}{=} & \sum_{a\in \scrA} \pi(a|s) \bigg( r(s,a) \notag \\
& +\gamma 
\inf_{\mu_{s,a}\in \Omega_{\epsilon_{s,a}}(\hat{\mu}_{s,a})}
\int_{\bp_{s,a}\in \scrU_{s,a}} \sum_{q\in \scrS} p(q|s,a) v(q)
d\mu_{s,a}(\bp_{s,a}) \bigg).
\notag
\end{align}       
The operator $T_{\Omega_{\bepsilon}(\hat{\mu})}^{\pi}$ satisfies a contraction
property with respect to the sup-norm, implying that the Bellman equation has a
unique fixed point denoted by $v^{\pi}_{\Omega_{\bepsilon}(\hat{\mu})}(s)$. To find
an optimal policy, consider the operator
\begin{equation} \label{rob-opt-bellman}
T_{\Omega_{\bepsilon}(\hat{\mu})} = \sup_{\pi(\cdot|s) \in \scrP(\scrA)}
T_{\Omega_{\bepsilon}(\hat{\mu})}^{\pi}.  
\end{equation}
As shown in \cite{derman2020distributional, chen2019distributionally}, there exists
a {\em distributionally robust optimal policy} $\pi^*$ and a unique value function
$v_{\Omega_{\bepsilon}(\hat{\mu})}^*(s)$ which is a fixed point of the operator defined by
(\ref{rob-opt-bellman}). In particular, for every $s\in \scrS$,
\[ 
v_{\Omega_{\bepsilon}(\hat{\mu})}^*(s) = \sup_{\pi} \inf_{\mu \in \Omega_{\bepsilon}(\hat{\mu})}  \mbb{E}_{p\sim
	\mu}^{\tau\sim \pi} \left[\sum_{t=0}^\infty \gamma^t r(s_t,a_t) 
\mid s_0=s \right] = v^{\pi^*}_{\Omega_{\bepsilon}(\hat{\mu})}(s).  
\]
The optimal value function can be obtained by value iteration, i.e., successive
application of $T_{\Omega_{\bepsilon}(\hat{\mu})}$ to some arbitrary initial value
function.

\subsubsection{Selecting a Nominal Distribution} \label{sec:rob-nominal}

The nominal distribution $\hat{\mu}$ that serves as the center of the Wasserstein
balls in (\ref{rl-wass-ball}) can be determined as the empirical distribution
computed from a set of different independent episodes of the MDP. Suppose we have in
our disposal $n$ such episodes. Then, for each episode $i\in \lb n \rb$, and using
the observed sequence of states and actions during the episode, we can compute the
empirical transition probability $\hat{p}^{(i)}(q|s,a)$ of transitioning into state $q$
when applying action $a$ in state $s$. The resulting empirical distribution
$\hat{\mu}^n_{s,a}$ assigns mass $1/n$ to each $\hat{p}^{(i)}(\cdot|s,a)$, namely,
\[
\hat{\mu}^n_{s,a} = \frac{1}{n} \sum_{i=1}^n \delta_{\hat{p}^{(i)}(\cdot|s,a)},
\] 
where $\delta_{\hat{p}^{(i)}(\cdot|s,a)}$ is a Dirac function assigning mass $1$ to the
model $\hat{p}^{(i)}(\cdot|s,a)$. Defining a product distribution for each episode $i$
by $\delta_i = \prod_{(s,a)\in \scrS\times \scrA} \delta_{\hat{p}^{(i)}(\cdot|s,a)}$, we
can define the empirical distribution 
\[ 
\hat{\mu}^n = \frac{1}{n} \sum_{i=1}^n \delta_i. 
\]

The model above requires computing an empirical transition probability for each
state-action pair. When the state-action space is very large, this is not
practical. Instead, one can employ some approximation architecture. One possibility
is to use an architecture of the following type 
\[ 
\hat{p}^{(i)}(q|s,a) = 	\frac{\exp\{\bxi_i' \bpsi(s, a, q)\} }
{\sum_{y \in \scrS} \exp\{\bxi_i' \bpsi(s, a, y)\} },
\] 
for some vector of feature functions $\bpsi(s, a, q)$ and a parameter vector
$\bxi_i$; the latter can be learned from the sequence of state-actions corresponding
to episode $i$ by solving a logistic regression problem.

\subsubsection{A Regularization Result for the Distributionally Robust
	MDP} \label{sec:rob-regularization}

\cite{derman2020distributional} obtains a regularization result for the Wasserstein
distributionally robust MDP that is analogous to the dual-norm regularization we
obtained in Section~\ref{chapt:dro}. We will outline some of the key steps, referring
the reader to \cite{derman2020distributional} for the full details. The result
obtains a lower bound on the value function
$v^{\pi}_{\Omega_{\bepsilon}(\hat{\mu}^n)}(s)$. 

To that end, define first the {\em conjugate robust value function} at state $s$ and
under policy $\pi$. Specifically, let $\bp=(p(q|s,a); \forall q,s\in \scrS, a\in
\scrA)$ denote a vectorized form of the transition probability model. For any
$\bz=(z(q|s,a); \forall q,s\in \scrS, a\in \scrA)$, we define the conjugate robust value function as
\begin{equation} \label{conj-ch8} 
v_s^{*,\pi}(\bz) \stackrel{\triangle}{=} \inf_{\bp} (v_{\bp}^\pi(s) - \bz'\bp),    
\end{equation}
and let $\scrD_s = \{ \bz : v_s^{*,\pi}(\bz) > -\infty\}$ be its effective
domain. Note that as defined, $v_s^{*,\pi}(\bz)$ is the negative of the convex
conjugate of the value function as a function of $\bp$~\citep{rock}.

A key result from \cite{derman2020distributional} is in the following theorem. As
discussed earlier, suppose we have data from $n$ episodes from
the MDP and we have constructed the empirical transition probabilities for each
episode. Let $\hat{\bp}^{(i)} = (\hat{p}^{(i)}(q|s,a); \forall q,s\in \scrS, a\in \scrA)$
be the corresponding vector.
\begin{thm} \citep{derman2020distributional} \label{thm:mdp-reg}
	For any policy $\pi$, it holds that 
	\begin{equation*} 
	v^{\pi}_{\Omega_{\bepsilon}(\hat{\mu}^n)}(s) \geq \frac{1}{n} \sum_{i=1}^n
	v_{\hat{\bp}^{(i)}}^\pi(s) - \kappa \alpha_s, 
	\end{equation*}
	where $\alpha_s=\sum_{a\in \scrA} \pi(a|s) \epsilon_{s,a}$, $\kappa=\sup_{\bz\in
		\scrD_s} \|\bz\|_*$, and $\|\cdot\|_*$ is the dual norm to the norm used in
	defining the Wasserstein uncertainty set (cf. (\ref{rl-wass-ball})). 
\end{thm}
\begin{proof} We will provide an outline of the key steps. We start by expressing
	$v^{\pi}_{\Omega_{\bepsilon}(\hat{\mu}^n)}(s)$ using the Bellman equation
	(\ref{bellman-dist-robust}). We have
	\begin{align} \label{reg-1} 
	& v^{\pi}_{\Omega_{\bepsilon}(\hat{\mu}^n)}(s) \notag\\
	= & \sum_{a\in \scrA} \pi(a|s) \bigg( r(s,a) \notag \\
	& +\gamma 
	\inf_{\mu_{s,a}\in \Omega_{\epsilon_{s,a}}(\hat{\mu}^n_{s,a})}
	\int_{\bp_{s,a}} \sum_{q\in \scrS} p(q|s,a) v(q)
	d\mu_{s,a}(\bp_{s,a}) \bigg)
	\notag \\
	= &  \sum_{a\in \scrA} \pi(a|s) \bigg( r(s,a) \notag \\
	& +\gamma 
	\inf_{\{\mu_{s,a}:  W_1(\mu_{s,a},\hat{\mu}^n_{s,a})\leq \epsilon_{s,a}\}}
	\int_{\bp_{s,a}} \sum_{q\in \scrS} p(q|s,a) v(q)
	d\mu_{s,a}(\bp_{s,a}) \bigg)
	\notag\\
	= &  \sum_{a\in \scrA} \pi(a|s) \bigg( r(s,a) \notag \\
	& +\gamma 
	\inf_{\mu_{s,a}} \sup_{\lambda\geq 0} \bigg[
	\int_{\bp_{s,a}} \sum_{q\in \scrS} p(q|s,a) v(q)
	d\mu_{s,a}(\bp_{s,a}) \notag \\
	& 
	+ \lambda W_1(\mu_{s,a},\hat{\mu}^n_{s,a}) -\lambda \epsilon_{s,a}\bigg]\bigg) 
	\notag\\
	\geq &  \sum_{a\in \scrA} \pi(a|s) \bigg( r(s,a) \notag \\
	& +\gamma 
	\sup_{\lambda\geq 0} \inf_{\mu_{s,a}} \bigg[ 
	\int_{\bp_{s,a}} \sum_{q\in \scrS} p(q|s,a) v(q)
	d\mu_{s,a}(\bp_{s,a}) \notag \\
	& 
	+ \lambda W_1(\mu_{s,a},\hat{\mu}^n_{s,a})  -\lambda \epsilon_{s,a}\bigg] \bigg),
	\end{align}       
	where the last inequality used weak duality. 
	
	Next, using the structure of $\hat{\mu}^n$ as an average over $n$ episodes and the
	fact (due to the rectangularity assumption) that the empirical distribution is a
	product distribution over state-action pairs, we can deduce from (\ref{reg-1}) that 
	\begin{equation} \label{reg-2} 
	v^{\pi}_{\Omega_{\bepsilon}(\hat{\mu}^n)}(s) \geq 
	\sup_{\lambda\geq 0} \bigg[ \frac{1}{n} \sum_{i=1}^n \inf_{\bp}( v^{\pi}_{\bp}(s) + \lambda
	\|\bp - \hat{\bp}^{(i)} \| ) - \lambda \alpha_s \bigg], 
	\end{equation}
	where $\alpha_s=\sum_{a\in \scrA} \pi(a|s) \epsilon_{s,a}$. This derivation used
	similar techniques as in Theorem~\ref{dro-duality}. 
	
	Using the definition of the dual norm and for any $\lambda\geq 0$ we have
	\begin{align} \label{reg-3}
	\lambda \|\bp - \hat{\bp}^{(i)} \|  = & \max_{\|\bz_i\|_* \leq 1} \lambda \bz_i' (\bp -
	\hat{\bp}^{(i)}) \notag\\
	= & \max_{\|\lambda \bz_i\|_* \leq \lambda} \lambda \bz_i' (\bp -
	\hat{\bp}^{(i)}) \notag\\
	= & \max_{\|\bu_i\|_* \leq \lambda} \bu_i' (\bp -
	\hat{\bp}^{(i)}). 
	\end{align}
	
	Denote by $\nu_s^\pi: \bp \rightarrow v^{\pi}_{\bp}(s)$ the function that maps the
	transition probability vector $\bp$ to the value function $v^{\pi}_{\bp}(s)$. Let
	$\Breve{\text{cl}}(\nu_s^\pi)$ be its convex closure, i.e., the greatest closed and
	convex function upper bounded by  $\nu_s^\pi$ at any $\bp$. Since
	$\Breve{\text{cl}}(\nu_s^\pi)$ is a lower bound on $v^{\pi}_{\bp}(s)$ and using
	(\ref{reg-3}) and (\ref{reg-2}) we obtain:
	\begin{equation} \label{reg-4} 
	v^{\pi}_{\Omega_{\bepsilon}(\hat{\mu}^n)}(s) \geq 
	\sup_{\lambda\geq 0} \bigg[ \frac{1}{n} \sum_{i=1}^n \inf_{\bp} \max_{\|\bu_i\|_*
		\leq \lambda} ( \Breve{\text{cl}}(\nu_s^\pi)(\bp) + \bu_i' (\bp -
	\hat{\bp}^{(i)})) - \lambda \alpha_s \bigg]. 
	\end{equation}
	
	Using the fact that the convex closure of a function has the same convex dual as the
	function itself, it follows that
	\begin{align} \label{reg-5}
	\Breve{\text{cl}}(\nu_s^\pi)(\bp) = & \Breve{\text{cl}}(\nu_s^\pi)^{**}(\bp) \notag \\
	= & \max_{\bz \in \scrD_s} [ \bz'\bp - \Breve{\text{cl}}(\nu_s^\pi)^{*}(\bz) ] \notag\\
	= &  \max_{\bz \in \scrD_s} [ \bz'\bp - (\nu_s^\pi)^*(\bz) ] \notag\\
	= & \max_{\bz \in \scrD_s} [ \bz'\bp + v_s^{*,\pi}(\bz) ], 
	\end{align}
	where the last equation used the definition of the conjugate robust value function (\ref{conj-ch8}). 
	
	Then, using (\ref{reg-5}), the term inside the summation in the RHS of (\ref{reg-4}) can be written as:
	\begin{align} \label{reg-6}
	\inf_{\bp} & \max_{\|\bu_i\|_*
		\leq \lambda} ( \Breve{\text{cl}}(\nu_s^\pi)(\bp) + \bu_i' (\bp -
	\hat{\bp}^{(i)})) \notag \\
	= & \inf_{\bp} \max_{\bz_i \in \scrD_s} \max_{\|\bu_i\|_*
		\leq \lambda} (v_s^{*,\pi}(\bz_i) + \bz_i'\bp + \bu_i' (\bp -
	\hat{\bp}^{(i)})) \notag \\
	= & \max_{\bz_i \in \scrD_s} \max_{\|\bu_i\|_*
		\leq \lambda}  [ v_s^{*,\pi}(\bz_i) - \bu_i' \hat{\bp}^{(i)} + \inf_{\bp} \bp'(\bz_i + \bu_i)], 
	\end{align}
	where the last equality used duality. Note that the minimization in the RHS of the
	above is over transition probability vectors. We can relax this minimization over all
	real vectors, which would render a lower bound and result in the infimum being $-\infty$ unless
	$\bu_i=-\bz_i$. Note that if $\sup\{\|\bz_i\|_* : \bz_i \in
	\scrD_s\} > \lambda$, then one can pick some $\bz_i \in \scrD_s$ such that $\|\bz_i\|_* > \lambda$, in which case the inner minimization in (\ref{reg-6}) achieves $-\infty$ since $\bu_i \neq -\bz_i$. When $\sup\{\|\bz_i\|_* : \bz_i \in
	\scrD_s\}\leq \lambda$, we have
	\begin{align*} 
	\inf_{\bp} & \max_{\|\bu_i\|_*
		\leq \lambda} ( \Breve{\text{cl}}(\nu_s^\pi)(\bp) + \bu_i' (\bp -
	\hat{\bp}^{(i)})) \notag \\
	\geq & \max_{\bz_i \in \scrD_s} \max_{\|\bz_i\|_* \leq \lambda} [ v_s^{*,\pi}(\bz_i)
	+ \bz_i' \hat{\bp}^{(i)}] \notag \\
	= & v^{\pi}_{\hat{\bp}^{(i)}}(s),
	\end{align*}
	where the second step follows from the fact that $v_s^{*,\pi}(\bz_i)$ is the
	negative of the convex dual of the value function. It follows that  
	\begin{align} \label{reg-7}
	\inf_{\bp} & \max_{\|\bu_i\|_*
		\leq \lambda} ( \Breve{\text{cl}}(\nu_s^\pi)(\bp) + \bu_i' (\bp -
	\hat{\bp}^{(i)})) \notag \\
	\geq & \begin{cases} v^{\pi}_{\hat{\bp}^{(i)}}(s), & \text{if $\sup\{\|\bz_i\|_* : \bz_i \in
		\scrD_s\}\leq \lambda$,}\\
	-\infty, & \text{otherwise.}  \end{cases}
	\end{align}
	Plugging (\ref{reg-7}) in
	(\ref{reg-4}) it follows that 
	\begin{equation*} 
	v^{\pi}_{\Omega_{\bepsilon}(\hat{\mu}^n)}(s) \geq 
	\frac{1}{n} \sum_{i=1}^n v^{\pi}_{\hat{\bp}^{(i)}}(s) - \kappa \alpha_s, 
	\end{equation*}
	where $\kappa= \sup_{\bz\in \scrD_s} \|\bz\|_*$.
\end{proof}

The result of Theorem~\ref{thm:mdp-reg} provides a lower bound on the
distributionally robust value function, which can be used in the RHS of the Bellman
equation and in a value iteration scheme. It can also be used in the same manner in
obtaining a distributionally robust optimal policy. However, this strategy is
applicable in settings where the state-action space is relatively small. For large
state-action spaces, one typically approximates either the value function or the
policy. To that end, the regularization result Theorem~\ref{thm:mdp-reg}
can be extended to cases where the value function is approximated by a linear
function.

In particular, suppose we approximate the value function by $v^{\pi}_{\bp}(s)
\approx \bphi(s)' \bw_\bp$, where $\bphi(s)$ is some feature vector and $\bw_\bp$ a
parameter vector. Similar to (\ref{conj-ch8}) we can define an {\em approximate conjugate
	robust value function} at state $s$ and under policy $\pi$ as: 
\begin{equation} \label{approx-conj} 
w_s^{*,\pi}(\bz) \stackrel{\triangle}{=} \inf_{\bp} ( \bphi(s)' \bw_\bp - \bz'\bp),    
\end{equation}
and let $\scrW_s = \{ \bz : w_s^{*,\pi}(\bz) > -\infty\}$ be its effective
domain. 

\cite{derman2020distributional} provides a result analogous to
Theorem~\ref{thm:mdp-reg}. 
\begin{thm} \citep{derman2020distributional} 
	For any policy $\pi$, it holds that 
	\begin{equation} \label{thm:mdp-reg-approx}
	\inf_{\mu \in \Omega_{\bepsilon}(\hat{\mu}^n)}  \mbb{E}_{p\sim
		\mu}^{\tau\sim \pi} [\bphi(s)' \bw_\bp] \geq \frac{1}{n} \sum_{i=1}^n
	\bphi(s)' \bw_{\hat{\bp}^{(i)}} - \eta \alpha_s, 
	\end{equation}
	where $\alpha_s=\sum_{a\in \scrA} \pi(a|s) \epsilon_{s,a}$ and $\eta=\sup_{\bz\in
		\scrW_s} \|\bz\|_*$.
\end{thm}

\chapter{Discussion and Conclusions} \label{chapt:concl}

% set this to the location of the figures for this chapter. it may
% also want to be ../Figures/2_Body/ or something. make sure that
% it has a trailing directory separator (i.e., '/')!
%\graphicspath{{3_Conclusion/Figures/}}

In this monograph, we developed a Wasserstein-based distributionally robust learning
framework for a comprehensive list of predictive and prescriptive problems, including
$(i)$ Distributionally Robust Linear Regression (DRLR), $(ii)$ Groupwise Wasserstein
Grouped LASSO (GWGL), $(iii)$ Distributionally Robust Multi-Output Learning, $(iv)$
Optimal decision making via DRLR informed K-Nearest Neighbors (K-NN), $(v)$
Distributionally Robust Semi-Supervised Learning, and $(vi)$ Distributionally Robust
Reinforcement Learning.

Starting with the basics of the Wasserstein metric and the DRO formulation, we
explored its robustness inducing properties, discussed approaches for solving the DRO
formulation, and investigated the properties of the DRO solution. Then, we turned our
attention into specific learning problems that can be posed and solved using the
Wasserstein DRO approach. In each case, we derived equivalent regularized empirical
loss minimization formulations and established the robustness of the solutions both
theoretically and empirically. We showed novel theoretical results tailored to each
setting and validated the methods using real world medical applications,
strengthening the notion of robustness through these discussions.

The robustness of the Wasserstein DRO approach hinges on the fact that a family of
distributions that are different from, but close to the empirical measure, are being
hedged against. This data-driven formulation not only utilizes the information
contained in the observed samples, but also generalizes beyond that by allowing
distributions with out-of-sample support. This is a distinguishing feature from DRO
approaches based on alternative 
distance functions, such as $\phi$-divergences, which only consider
distributions whose support is a subset of the observed samples. Such a limitation
could potentially hurt the generalization power of the model. Another salient
advantage of the Wasserstein metric lies in its structure, in particular, encoding a
distance metric in the data space, which makes it possible to link the form of the
regularizer with the growth rate of the loss function and establish a connection
between {\em robustness} and {\em regularization}.

Our results on Wasserstein DRO and its connection to regularization are not
restricted to linear and logistic regression. From the analysis presented in
Section~\ref{chapt:solve}, we see that as long as the growth rate of the loss
function is bounded, the corresponding Wasserstein DRO problem can be made
tractable. We consider both static settings, where all the samples are readily
accessible when solving for the model (Sections~\ref{chapt:dro}, \ref{chap:group},
\ref{chap:multi}), and a dynamic setting where the samples come in a sequential
manner (Section~\ref{dro-rl}). Another example of a dynamic DRO problem is
\cite{abadeh2018wasserstein}, which proposed a distributionally robust Kalman filter
that hedges against model risk; in that setting, the Wasserstein ambiguity set
contains only normal distributions.

More broadly, researchers have proposed distributionally robust versions for general
estimation problems, see, for example, \cite{nguyen2019bridging} for
distributionally robust Minimum Mean Square Error Estimation,
\cite{nguyen2018distributionally} for distributionally robust Maximum Likelihood
Estimation, which was adopted to estimate the inverse covariance matrix of a Gaussian
random vector. We refer the reader to \cite{peyre2019computational} for
computational aspects related to Wasserstein distances and optimal
transport. \cite{kuhn2019wasserstein} and \cite{rahimian2019distributionally} also
provided nice overviews of DRO, the former focusing specifically on the Wasserstein
DRO, covering in detail the theoretical aspects of the general formulation with a
brief discussion on some machine learning applications, while the latter covered DRO
models with all kinds of ambiguity sets.  We summarize our key novel contributions as
follows.

\begin{itemize}
	\item We considered a comprehensive list of machine learning problems, not only
	predictive models, but also prescriptive models, that can be posed and solved
	using the Wasserstein DRO framework.
	
	\item We presented novel performance guarantees tailored to each problem,
	reflecting the particularity of the specific problem and providing
	justifications for using a Wasserstein DRO approach. This is very different
	from \cite{kuhn2019wasserstein}, where a universal performance guarantee result
	was derived. Their result is in general applicable to every single DRO problem,
	but may miss the individual characteristics of the problem at hand.
	
	\item The Wasserstein prescriptive model we presented in Section~\ref{ch:presp}
	is novel. We showed the power of Wasserstein DRO through the K-NN insertion in a
	decision making problem, and demonstrated the benefit of robustness through a
	novel out-of-sample MSE result.
	
	\item The non-trivial extension to multi-output DRO has implications on training
	robust neural networks, e.g., the robustness of the multiclass logistic
	regression classifier to optimized perturbations that are designed to fool the
	classifier, see Section~\ref{sec:dro-mlg}. 
	
	\item Finally, we considered a variety of synthetic and real world case studies
	of the respective models, demonstrating the applications of the DRO framework
	and its superior performance compared to other alternatives, which adds to the
	accessibility and appeal of this work to an application-oriented reader.
\end{itemize}

\begin{acknowledgements}

The authors are grateful to Dimitris Bertsimas, Theodora Brisimi, Christos
Cassandras, David Casta\~{n}\'{o}n, Alex Olshevsky, Venkatesh Saligrama, and Wei Shi,
for their insightful comments and constructive suggestions. 

We are thankful to the Network Optimization and Control Lab at Boston University
for providing computational resources and expertise for some of the case
studies. Collaborations on a number of application fronts have involved Michael
Caramanis and Pirooz Vakili.  

We are grateful to many clinicians and researchers in Boston area hospitals who
provided access to data and collaborated in parts of the work, including: Hiroto
Hatabu, George Kasotakis, Fania Mela, Rebecca Mishuris, Jenifer Siegelman, Vladimir
Valtchinov, and George Velmahos. Particular mention is due to Bill Adams, at Boston
Medical Center, whose efforts to make data available for research have been nothing
short of extraordinary and who was instrumental in engaging the authors in health
analytics research.

We are thankful to the series editors Garud Iyengar, Stephen Boyd, and to the
anonymous reviewers for valuable feedback.

RC is grateful to ICP and David Casta\~{n}\'{o}n who have provided constant support
and encouragement for her, and have been inspirational role models as excellent
researchers and teachers with endless positivity and passion. She is also grateful to
her parents, Xudong and Shouzhen, and her cousins Yingying, Qianqian, and Chunlei,
for their unconditional love, support and company, which have given her the
strength and determination to overcome difficulties and complete this work.

ICP is grateful to Dimitris Bertsimas and John Tsitsiklis for all they have taught
him and for being such inspirational role models for research and the good exposition
of research ideas. He is also grateful to his family (Gina, Aris, Phevos, and
Alexandros) for their love, support, and giving him the time to work on this project.

The authors are also grateful to Ulrike Fischer, who designed the style files, and
Neal Parikh, who laid the groundwork for these style files.

Part of the research included in this monograph has been supported by the NSF under
grants IIS-1914792, DMS-1664644, CNS-1645681, CCF-1527292, and IIS-1237022, by the
ARO under grant W911NF-12-1-0390, by the ONR under grant N00014-19-1-2571, by the NIH
under grants R01 GM135930 and UL54 TR004130, by the DOE under grant DE-AR-0001282, by
the Clinical \& Translational Science Institute at Boston University, by the Boston
University Digital Health Initiative and the Center for Information and Systems
Engineering, and by the joint Boston University and Brigham \& Women's Hospital
program in Engineering and Radiology.
\end{acknowledgements}

%\appendix

\backmatter  

\printbibliography

\end{document}